%% file: main.tex
\documentclass[lettersize,journal]{IEEEtran}
\usepackage{amsmath,amsfonts}
\usepackage{array}
\usepackage{amssymb}
\usepackage{amsthm}
\usepackage{algorithmic}
\usepackage{algorithm}
\newtheorem{theorem}{Theorem}[section]

\usepackage[caption=false]{subfig}
\usepackage{textcomp}
\usepackage{stfloats}
\usepackage{url}
\usepackage{verbatim}
\usepackage{academicons}
\usepackage{graphicx}
\usepackage{tabularx}
\usepackage{comment}
\usepackage{hyperref}
\hypersetup{
    colorlinks=true,
    linkcolor=black,
    citecolor=black,      
    urlcolor=black,
}
\usepackage{orcidlink}
\usepackage[numbers,sort&compress]{natbib}

\usepackage{makecell}
\usepackage{array}
\usepackage{multirow}
\usepackage{ragged2e}
\usepackage{rotating}
\usepackage{colortbl}
\def\BibTeX{{\rm B\kern-.05em{\sc i\kern-.025em b}\kern-.08em
    T\kern-.1667em\lower.7ex\hbox{E}\kern-.125emX}}
\usepackage{balance}
\usepackage[none]{hyphenat}
\usepackage{xr}
\makeatletter

\newcommand*{\addFileDependency}[1]{%
\typeout{(#1)}%
\@addtofilelist{#1}

\IfFileExists{#1}{}{\typeout{No file #1.}}
}\makeatother

\newcommand*{\myexternaldocument}[1]{%
\externaldocument{#1}%
\addFileDependency{#1.tex}%
\addFileDependency{#1.aux}%
}

\myexternaldocument{SupplementaryMaterial}
\usepackage{float}
\usepackage{placeins}

\usepackage{bm}
\usepackage{nomencl}
\makenomenclature

\usepackage{titlesec}
\begin{document}

\title{Semantic-Preserving Feature Partitioning for Multi-View Ensemble Learning}

\author{Mohammad~Sadegh~Khorshidi\orcidlink{0000-0001-6556-2926}
, Navid~Yazdanjue\orcidlink{0000-0001-9670-8422}%
, Hassan~Gharoun\orcidlink{0000-0001-8298-7512}%
, Danial~Yazdani\orcidlink{0000-0002-7799-5013},~\IEEEmembership{Member,~IEEE}%
, Mohammad~Reza~Nikoo\orcidlink{0000-0002-3740-4389}%
, Fang~Chen\orcidlink{0000-0003-4971-8729}%
, and Amir~H.~Gandomi\orcidlink{0000-0002-2798-0104},~\IEEEmembership{Senior Member,~IEEE}
\thanks{Mohammad Sadegh Khorshidi, Navid Yazdanjue, Hassan Gharoun, Danial Yazdani, Fang Chen, and Amir H. Gandomi are with the Faculty of Engineering \& Information Technology, University of Technology Sydney, Ultimo 2007, Australia. e-mails: ms.khorshidi@student.uts.edu.au, navid.yazdanjue@gmail.com, hassan.gharoun@student.uts.edu.au, danial.yazdani@gmail.com, fang.chen@uts.edu.au, gandomi@uts.edu.au}%
\thanks{Mohammad Reza Nikoo is with the Department of Civil and Architectural Engineering, Sultan Qaboos University, Muscat, Oman. (e-mail: m.reza@squ.edu.om}%
\thanks{Amir H. Gandomi is also with the University Research and Innovation Center (EKIK), Obuda University, Budapest 1034, Hungary.}%
\thanks{This work was supported by the Australian Government through the Australian Research Council under Project DE210101808.}
\thanks{Corresponding author: Amir H. Gandomi}}

\maketitle

\begin{abstract}
In machine learning, the exponential growth of data and the associated ``curse of dimensionality'' pose significant challenges, particularly with expansive yet sparse datasets.
Addressing these challenges, multi-view ensemble learning (MEL) has emerged as a transformative approach, with feature partitioning (FP) playing a pivotal role in constructing artificial views for MEL.
Our study introduces the Semantic-Preserving Feature Partitioning (SPFP) algorithm, a novel method grounded in information theory.
The SPFP algorithm effectively partitions datasets into multiple semantically consistent views, enhancing the MEL process.
Through extensive experiments on eight real-world datasets, ranging from high-dimensional with limited instances to low-dimensional with high instances, our method demonstrates notable efficacy.
It maintains model accuracy while significantly improving uncertainty measures in scenarios where high generalization performance is achievable.
Conversely, it retains uncertainty metrics while enhancing accuracy where high generalization accuracy is less attainable.
An effect size analysis further reveals that the SPFP algorithm outperforms benchmark models by large effect size and reduces computational demands through effective dimensionality reduction.
The substantial effect sizes observed in most experiments underscore the algorithm's significant improvements in model performance.
\end{abstract}


\section{Introduction}
\IEEEPARstart{T}{he} contemporary digital era is characterized by an explosive growth in data generation, facilitated by various applications \cite{Fan2014293, Sagiroglu201342}.
The popularity of the Internet of Things (IoT) and web-based platforms has amplified the data creation rate accompanied by a wave of noisy data, compromising the efficacy of machine-learning algorithms \cite{Bhadani20161, Tsai201477}.
This surge in data complexity, also known as ``big data'', not only challenges analytics but also escalates the burden on machine learning tools, especially for classification and pattern recognition tasks \cite{Verleysen2005758, Shanthamallu20171}.

The ``curse of dimensionality'' is ascribed to the dilemma of handling large yet sparse data in machine-learning tasks that undermine the performance of applied algorithms \cite{Donoho2000, Johnstone20094237}.
This phenomenon underlines the challenges associated with navigating large feature spaces.
As the dimensionality of these spaces expands, data tends to become increasingly sparse, rendering traditional algorithms ineffective \cite{Cui2021, Li20161230}.
Several dimensionality reduction approaches have been proposed in the literature to overcome the challenges posed by the curse of dimensionality \cite{khorshidi202310, gharoun2023noiseaugmented}.
Broadly, these strategies are categorized into feature selection (FS) and feature extraction (FE) \cite{Ghojogh2019}.
However, employing these techniques eliminates features and results in loss of information.
It is essential to note that even these seemingly trivial features in one feature set could hold intrinsic information value in another due to statistical interactions.
Machine learning algorithms can predominantly struggle to capture and generalize feature interactions.
Notably, this is often due to a mismatch between the mathematical complexity of these interactions and the structure of the algorithms \cite{panda20219}.

 To address this shortcoming, a distinctive machine-learning method called ``multi-view learning'' (MVL) has emerged in the literature \cite{Zhang20229258396}.
 This paradigm harnesses multiple distinct representations of data to enhance model performance.
 When data is sourced from various sources or observed from different angles, each offers a unique ``view''.
 Relying on the complementary aspects of the information provided by multiple views, the MVL technique leverages distinct views during a collaborative learning process, including nagging, boosting, or stacking to achieve a single model with enhanced performance \cite{xu2013survey}.
 Multi-view ensemble learning (MEL), on the other hand, combines the principles of MVL with ensemble learning.
 The MEL technique aggregates the insights from individual models trained for each view through consensus mechanisms, including model averaging, voting, and meta-learning, aiming to enhance performance and robustness \cite{vanLoon2020113}.
 Both techniques are potent methods to improve the robustness and generalization and reduce overfitting of learning algorithms for high-dimensional and complex data \cite{KUMAR2023101959}.

In the context of MVL (and MEL), views can be categorized into two primary types: natural and artificial views \cite{sun2011multi}.
Natural views arise intrinsically from the data acquisition sources that often result from different sensors, modalities, or feature extraction mechanisms \cite{Fortino201550}.
For instance, MRI and CT scans provide two distinct natural views of the same anatomical structure in medical imaging.
On the other hand, artificial views are constructed post-hoc, typically through various construction or transformation techniques applied to the original data \cite{Flynn20192362}.
These might involve dimensionality reduction, different preprocessing steps, or the application of domain-specific knowledge.
While natural views leverage inherent variability and complementary information from diverse sources, artificial views aim to uncover latent patterns or relationships within the data by introducing new perspectives. 

MVL (and MEL) has diverse applications spanning various domains.
It is employed in tasks including clustering \cite{liu2021multiview}, semi-supervised \cite{jia2020semi} and supervised learning \cite{ye2021multiview}, ensemble and active learning \cite{Di20121942}, regression \cite{yang2019adaptive}, dimensionality reduction \cite{liu2018multiview}, and representation learning \cite{li2018survey}. 

The FP is a key method for constructing artificial views in MVL (and MEL) applied to single-source data.
This technique involves vertically dividing the dataset into multiple views, forming the initial phase of the MEL application.
In FP, the quality of generated views, the quantity of the partitions, and computational efficiency are the three crucial factors significantly influencing MEL performance \cite{zhao2017multi}.
While multi-source data naturally provides diverse views, creating qualitatively rich and appropriately quantified artificial views from single-source data poses a substantial challenge.
Moreover, existing methods in the literature often rely on random feature search and model evaluation to construct artificial views, further intensifying the computational complexity.
 These aspects are the three primary challenges to address for effective MEL implementation.

This study introduces the novel Semantic-Preserving Feature Partitioning (SPFP) algorithm, a mathematically robust approach designed to tackle the three primary challenges in MEL effectively.
The SPFP algorithm offers a systematic method to arbitrarily determine the number of artificial views, ensuring each view maintains informational quality comparable to the entire dataset.
This approach effectively eliminates the need for recursive evaluation of the machine-learning algorithm, thereby enhancing computational efficiency.
Key contributions of the SPFP algorithm include:
\begin{itemize}
\item View quantification: The user can choose their desired number of artificial views for a given dataset while optimizing the balance between computational efficiency and model performance.

\item High-quality view Construction: Ensures that each generated view preserves the original dataset's semantic integrity and informational richness, contributing to more accurate and reliable model predictions.

\item Reduced computational complexity: Streamlines the model evaluation process by negating the need for repetitive algorithm testing, significantly reducing the computational demand and time required for model training and validation.
\end{itemize}

The remainder of this paper is organized as follows. 
Section~\ref{sec:background} briefly and concisely reviews the MEL method and highlights the most relevant research in this domain. 
Section~\ref{sec:spfp} introduces the proposed SPFP algorithm, detailing its mathematical underpinnings.
Section~\ref{sec:experiments} explores the specifics of the experiments conducted, including descriptions of the datasets used, the statistical analyses performed, and a comprehensive discussion of the results obtained. 
Finally, Section~\ref{sec:conclusion} concludes the paper, offering reflections on the study's findings and proposing potential avenues for future research to further enhance and expand upon the work presented here.

\section{Background}
\label{sec:background}
Since the focus of the present paper is FP in the MEL domain, the related works are mainly rooted in two domains of MEL and classical FS.
Thus, this section is divided into two main categories. 
In Section~\ref{sub:fpmethod}, we present and discuss relevant FP methods proposed for MEL. 
This discussion includes an exploration of key terminologies and common concepts from the field of Feature Selection (FS), contextualizing them within the domain of FP.
Section~\ref{sub:iffsmethod} then shifts focus to the foundational classical information-theoretic methods and concepts necessary for introducing our proposed SPFP algorithm, laying the groundwork for its detailed presentation and analysis.
It should be noted that although we aim to offer a brief yet thorough overview of current methodologies, the scope of existing approaches in the literature extends well beyond what is covered in this paper.
We recommend consulting state-of-the-art surveys and reviews in the field for those seeking a more in-depth exploration. Notable references include \cite{xu2013survey, Cai201870, Ghojogh2019, KUMAR2023101959} for comprehensive insights.
The mathematical notations used in this paper are listed in Appendix~\ref{appendix}.

\subsection{Feature Partitioning Methods}
\label{sub:fpmethod}
Numerous FP methods for MEL have been proposed in the literature. A summary of these methods is provided in the following:

Random-based FP methods, such as the Random Split approach, partition features arbitrarily \cite{ho1998random}.
Another notable method in this category is attribute bagging \cite{bryll2003attribute}. 
Being filter-based, they focus on the intrinsic properties of the data, offering simplicity and rapid partitioning.
However, their arbitrary nature often lacks assurance that the resultant views capture meaningful or complementary information.

Pattern-based FP methods, including theme-based FP methods \cite{guggari2019theme}, and round robin or zig-zag \cite{kumar2015multi}, with bell triangle-based FP \cite{guggari2018non} partition features based on specific patterns or themes.
As filter-based methods, they offer computational efficiency due to deterministic partitioning patterns, reducing computational requirements.
The challenge, however, is that fixed patterns might not always capture the underlying data structure or relationships optimally.

Clustering-based FP methods form either homogeneous or heterogeneous clusters of attributes.
Methods including graph coloring based FP \cite{zheng2021feature} and attribute clustering based on the k-mean clustering technique \cite{taheri2023collaboration} are representative.
Being filter-based, they capture inherent groupings or relationships within the data.
However, their effectiveness extensively relies on the clustering algorithm's performance.

Performance-based FP methods, such as the optimal feature set partitioning method \cite{kumar2016multi} and rough set based FP method \cite{saini2019multi}, are iterative and improve ensemble classification accuracy through multiple experimental runs.
As wrapper-based methods, the feature selection process is guided by the performance of a specific algorithm.
While they aim for optimal performance, they can be time-consuming, making them less feasible for large datasets.

Optimization-based FP methods utilize evaluation criteria optimized for the MEL framework, commonly employing metaheuristic optimization approaches.
Being wrapper-based, they fine-tune feature partitioning based on algorithm performance.
They often achieve high performance due to this optimization but are computationally demanding.
The role of optimization in FP is pivotal, aiming to derive subsets of features that are informative and less redundant for classification tasks.
Techniques including genetic algorithm \cite{rokach2008genetic}, particle swarm optimization \cite{kumar2022multi}, and simulated annealing \cite{husin2016ant} are employed to ascertain the best feature subsets.
Through this, FP ensures optimal partitioning of features and an ensemble of classifiers, leading to improved classification performance.
However, the challenge remains in the computational demand of the optimization process.

Examining various FP methods for MVL and MEL reveals that each method addresses the primary challenges—the quality of generated views, appropriate partitioning, and computational demands—to varying degrees.
While filter-based methods, renowned for their computational efficiency, fall short in generating comprehensive views, their wrapper-based counterparts, though proficient at generating high-quality views, struggle with computational intensity.
This intensity often renders the fine-tuning of machine-learning algorithms' hyper-parameters unfeasible due to the high-dimensional nature of the data.

\subsection{Information-based Feature Selection Methods}
\label{sub:iffsmethod}
Information theory is one of the widely used frameworks for FS due to its inherent capability to discern linear and non-linear inter-dependencies among variables.
Moreover, its applicability is independent of any machine-learning algorithm, thus categorized as a filter-based FS method.
The literature predominantly presents heuristic algorithms grounded in information theory, operationalizing three primary objectives for feature selection: maximizing relevance, minimizing redundancy, and maximizing complementarity \cite{li2017feature}.

Entropy is a metric that measures the amount of variability (or uncertainty) in a random variable.
For instance, in discrete form, the entropy, $H(X)$, of a random variable, $X$, can be calculated as follows:
\begin{equation}
\label{eq:ent}
H(X) = \sum_{x_i \in X} -p(x_i)log(p(x_i)),
\end{equation}
where, $x_i$, and $p(x_i)$ denote the $i$th observation of the random variable, $X$, and its probability, respectively.
Given another random variable, namely $Y$, the joint entropy, $H(X,Y)$, and conditional entropy, $H(X|Y)$, can be defined by replacing $p(x_i)$ with their joint probability, $p(x_i,y_j)$, and the conditional probability, $p(x_i|y_j)$, respectively, in Eq. \eqref{eq:ent} as follows:
\begin{equation}
\label{eq:jent}
H(X,Y) = \sum_{y_j \in Y} \sum_{x_i \in X} -p(x_i,y_j)log(p(x_i,y_j)),
\end{equation}
\begin{equation}
\label{eq:cent}
H(X|Y) = \sum_{y_j \in Y} \sum_{x_i \in X} -p(x_i|y_j)log(p(x_i|y_j)).
\end{equation}

The entropy, $H(\cdot)$, is a non-negative quantity.
This is because the probability, $0 \leq p(\cdot) \leq 1$, ensures the logarithmic term in Eqs. \eqref{eq:ent} to \eqref{eq:cent} remains negative. 
The conditional entropy $H(X|Y)$ is always less than or equal to the entropy $H(X)$, as knowing another variable $Y$ cannot increase the uncertainty about $X$, i.e., $H(X|Y) \leq H(X)$ \cite{mackay2003information}.
In contrast, the joint entropy $H(X,Y)$ is always greater than the entropy of each individual variable, signifying that the combined uncertainty of variables $X$ and $Y$ is higher than their individual uncertainties, i.e., $H(X) \leq H(X,Y)$ \cite{mackay2003information}.
Additionally, the relationship $H(X,Y)=H(X|Y)+H(Y)$ holds, and since $H(X|Y)<+H(X)$, it leads to the conclusion that $H(X,Y) \leq H(X)+H(Y)$ \cite{mackay2003information}. 
Hence, the following inequality universally holds due to the inherent properties of entropy in information theory:
\begin{equation}
\label{eq:entinq}
\begin{aligned}
0& \leq H(X|Y) \leq H(X)\\
& \leq H(X,Y) \leq H(X)+H(Y).
\end{aligned}
\end{equation}

Utilizing the three fundamental definitions articulated in Eqs. \eqref{eq:ent} to \eqref{eq:cent}, the criteria essential for Feature Selection (FS), including relevance, redundancy, and complementarity, can be derived directly or indirectly.

Relevance refers to the shared information, or mutual information (MI), $I(X;Y)$, between feature, $X$, and target variable, $Y$:
\begin{equation}
\label{eq:mi}
\begin{aligned}
I(X;Y) &= H(X) - H(X|Y) \\
       &= H(Y) - H(Y|X) \\
       &= H(X) + H(Y) - H(X,Y) \\
       &= \sum_{y_j \in Y} \sum_{x_i \in X} p(x_i,y_j) \log\left(\frac{p(x_i,y_j)}{p(x_i)p(y_j)}\right).
\end{aligned}
\end{equation}

Selecting a feature subset, $S$, solely based on individual MIs does not necessarily guarantee an improvement in the MI between the subset, 
$S$, and the target, $Y$, denoted as $I(S;Y)$.
Therefore, selecting a subset where individual features share minimum information is crucial.
Thus, redundancy pertains to the MI between a selected feature, $f_s$, and a candidate feature, $f_c$, within the pool of features $F$, denoted as $I(f_s;f_c)$.

While relevance and redundancy address pairwise dependencies, they do not capture the interaction among variables.
Therefore, The complementarity criterion pertains to the interaction gain (IG) or the degree of synergy between a selected feature, $f_s$, and a candidate feature, $f_c$, given $Y$, beyond their pairwise mutual information.
The IG, denoted as $I(f_s;f_c;Y)$ can be computed as follows:
\begin{equation}
\label{eq:inter}
I(f_s;f_c;Y) = I(f_s;f_c) - I(f_s;f_c|Y),
\end{equation}
where, $I(f_s;f_c|Y)$ is the conditional mutual information (CMI) between $f_s$ and $f_c$ given $Y$, and can be calculated as follows:
\begin{equation}
\label{eq:cmi}
\begin{aligned}
&I(f_s;f_c|Y)\\
&=H(f_s|Y)+H(f_c|Y)-H(f_s,f_c,Y)\\
&= H(f_s,Y) + H(f_c,Y) - H(f_s,f_c,Y) - H(Y)\\
&=\sum_{f_{s_i} \in S} \sum_{f_{c_j} \in F} \sum_{y_k \in Y} p(f_{s_i},f_{c_j},y_k) \log\left(\frac{p(f_{s_i},f_{c_j}|y_k)}{p(f_{s_i}|y_k)p(f_{c_j}|y_k)}\right)\\
&=\sum_{f_{s_i} \in S} \sum_{f_{c_j} \in F} \sum_{y_k \in Y} p(f_{s_i},f_{c_j},y_k) \log\left(\frac{p(f_{s_i},f_{c_j},y_k)p(y_k)}{p(f_{s_i},y_k)p(f_{c_j},y_k)}\right).
\end{aligned}
\end{equation}

It is worth noting that the inequalities presented in Eq. \eqref{eq:entinq} highlight specific essential properties of MI, CMI, and IG.
Specifically, both MI and CMI are non-negative quantities.
MI is bounded by $\min(H(f_s),H(Y))$, while CMI is confined within the boundaries determined by $\min(H(f_s),H(f_c),H(Y))$.
In contrast, IG can take both negative and positive values between $-\min(H(f_s),H(f_c),H(Y))$ and $\min(H(f_s),H(f_c),H(Y))$.
Negative IG values signify the presence of redundancy between variables, while positive values indicate the existence of interaction or synergy among the variables. 

The relevance, redundancy, and complementarity criteria are integrated within a framework known as the conditional likelihood framework (CLF) in the literature \cite{li2017feature}.
This integration is achieved by using either a linear or non-linear combination of three key elements: the relevance (MI) of a candidate feature, $f_c$, in predicting the target, denoted as $I(f_c;Y)$; the redundancy (MI) of $f_c$ relative to the features already present in the selected feature subset, $f_s$, represented by $I(f_s;f_j)$; and the complementarity (CMI) of $f_c$ and the previously selected features given the target, indicated as $I(f_s;f_c|Y)$. These elements are defined in Eqs. \eqref{eq:mi} and \eqref{eq:cmi}. The general formula for the linear combination is presented as follows \cite{li2017feature}:
\begin{equation}
\label{eq:generalobj}
\begin{aligned}
&J(f_c) = \\
&I(f_c;Y) - \alpha \sum_{f_s \in S} I(f_s;f_c) + \beta \sum_{f_s \in S} I(f_s;f_c|Y),
\end{aligned}
\end{equation}
where $J(f_c)$ represents the score of the candidate feature, $f_c$ is a candidate feature drawn from the pool of unselected features, $f_s$ is a feature already included within the selected feature set, $S$, $\alpha$ and $\beta$ are arbitrary weights for redundancy and complementarity, respectively.
Algorithms that implement the CLF as per Eq. \eqref{eq:generalobj} employ heuristic search strategies to identify the candidate feature, $f_c$, that maximizes the score. 
The unified CLF, as presented, has been acknowledged for its substantial capability in dimensionality reduction and reducing computational load, as evidenced in existing literature.
However, a prominent challenge with this framework is the absence of robust and well-defined stopping criteria, which is essential for effectively applying the algorithm \cite{yu2019simple}.

\section{Proposed Semantic-Preserving Feature Partitioning Method}
\label{sec:spfp}
This section presents the SPFP technique based on the concepts outlined in Section~\ref{sub:iffsmethod}, innovatively modifying the CLF to adeptly overcome the intrinsic challenges of FP, such as identifying the quality and quantity of constructed views.
A meticulously implemented stopping criterion is central to our approach that bridges the gap between conventional information-theoretic FS algorithms and artificial view construction for MEL.

This method systematically measures the information content of the dataset, strategically selecting feature subsets until a saturation point is reached, where either a predefined number of features are chosen, or the cumulative information content matches that of the complete dataset.
This process iterates, allowing for the extraction of multiple subsets until no features remain unselected or a predefined subset limit is attained.
Distinctively, our algorithm embodies a comprehensive insight into the contextual relevance of each feature, recognizing that a feature's significance may vary based on its associated feature subset.

This perspective facilitates a versatile partitioning strategy, enabling the decomposition of the dataset into multiple views, each infused with information content mirroring the original dataset, making them suitable for the MEL task.
The current section is structured into four sub-sections for clarity and depth.
Section~\ref{sub:spfpobj} details the objective function of the SPFP algorithm, which is a modified version of the CLF.
This will include an explanation of how the objective function aligns with the goals of the algorithm.
In Section~\ref{sub:spfpstcri}, we discuss the critical stopping criteria for the SPFP algorithm.
This includes a mathematical justification for these criteria, reinforcing their importance in the algorithm's design and execution.

Section~\ref{sub:spfp} presents the feature partitioning process proposed through the SPFP algorithm.
Here, we explain how the algorithm decomposes a feature set into multiple subsets, each offering informational parity with the complete feature set.
This process underpins the artificial view construction essential for MEL.
Finally, Section~\ref{sub:spfpcond} examines the mathematical aspect of conditional independence, a convention often assumed in MVL literature.
This exploration aims to provide a better understanding of the underlying assumptions in MVL and MEL methodologies.

\subsection{Objective Function}
\label{sub:spfpobj}
As previously mentioned, the SPFP algorithm employs a modified CLF objective function (Eq. \eqref{eq:generalobj}), where coefficients $\alpha$ and $\beta$ are defined as the reverse of the cardinality of the selected feature set, $S$, i.e., $\alpha=\beta=1/\left| S \right|$.

It is essential to highlight that information-theoretic metrics capture complex dependency relationships between a set of variables.
However, the ultimate objective of a learning task is to optimize a usually predefined mapping function between the chosen feature set and the target variable. 
Different machine learning algorithms might face challenges in fully interpreting and integrating the dependency patterns revealed by information-theoretic metrics within their structures \cite{khorshidi202310}.

To address this, in a further refinement of Eq. \eqref{eq:generalobj}, we incorporate a complexity measure into the objective function. 
This involves introducing the Pearson correlation coefficient, as illustrated below:
\begin{equation}
\label{eq:algobj}
\begin{aligned}
&J_{SPFP}(f_c) = \left | R(f_c,Y) \right | + I(f_c;Y) \\
&\quad - \frac{1}{\left | S \right |} \sum_{f_s \in S} I(f_s;f_c) + \frac{1}{\left | S \right |} \sum_{f_s \in S} I(f_s;f_c|Y),
\end{aligned}
\end{equation}
where $J_{SPFP}(f_c)$ is the SPFP objective function of the candidate feature $f_c$; $\left | R(f_c,Y) \right |$ represents the absolute value of the Pearson correlation between the candidate feature, $f_c$, and the target, $Y$; and $f_s$ is a selected feature from the set of selected features, $S$. 

\subsection{Stopping Criteria}
\label{sub:spfpstcri}
The stopping criteria, as previously articulated, can be outlined in mathematical constructs, as illustrated in Eq. \eqref{eq:criteria}:
\begin{equation}
\label{eq:criteria}
\begin{aligned}
&C_1 : \left | S \right | \geq N_F,\\
&C_2 : H(S) = H(F),~ \text{and} ~ \\
&C_3 : H(S,Y) = H(F,Y).
\end{aligned}
\end{equation}

Each criterion described in Eq. \eqref{eq:criteria} elucidates the following aspect of stopping criteria for SPFP:

\begin{enumerate}
    \item \(C_1\) mandates that the cardinality of the selected feature set \(S\), denoted as \(\left | S \right | \), should be at least as large as a predefined threshold \(N_F\).
    \item $C_2$ constrains the information content of the selected subset $S$ to be equivalent to that of the entire feature set $F$, and, 
    \item \(C_3\) implies that the joint entropy of the selected subset and the target variable should equal to that of the entire feature set and the target variable.
\end{enumerate}

Criterion \(C_1\) in Eq. \eqref{eq:criteria} controls the cardinality of the selected feature subset.
Criteria \(C_2\) and \(C_3\) are imperative for maximizing the information content of the selected subset, particularly in conjunction with the target variable.
These criteria derive their significance from Theorems \ref{th:ent}, and \ref{th:mi}, which affirm that the entropy of a dataset and the MI between the entire feature set and the target variable, $Y$, are inherently maximum.
Consequently, no subset of features can surpass the entire set, $F$, in terms of either entropy or MI with the target variable.

\begin{theorem}
\label{th:ent}
For two sets of features, \(F\) and \(S \), where \(S \subset F\), the entropy of the entire feature set \(F\), i.e., \(H(F)\), is always greater than or equal to the entropy of the subset \(S \), \(H(S)\).
\end{theorem}

\begin{proof}
To prove this theorem, consider any disjoint and complementary subsets \(S \) and \(S' \) within the set \(F\) ($S' \cup S = F$), and examine their entropy relationship as follows:
\begin{align*}
p(f) = p(s) \cdot p(s'|s) \Rightarrow H(F) = H(S)+H(S'|S).
\end{align*}

Assume a proposition \(PR_1\) such that \(H(F) \geq H(S)\).
To challenge this, assume the contrary, where \(PR_1\) is false, denoted as \(\sim PR_1\), implying there exists a subset \(S\) such that \(H(F) < H(S)\).
\begin{align*}
\sim PR_1 & : H(F)<H(S) ~\text{and}~ H(F) = H(S)+H(S'|S), \\
& \Rightarrow H(S'|S)<0, \\
& \Rightarrow p(S'|S)>1, \\
& \Rightarrow \bot, \\
& \therefore PR_1: H(F) \geq H(S).
\end{align*}
holds for all $S \subset F$.
\end{proof}

\begin{theorem}
\label{th:mi}
For a target variable, \(Y\), and two sets of predictors, \(F\) and \(S\), where \(S \subset F\), the MI between the feature set \(F\) and \(Y\), \(I(F;Y)\), is always greater than or equal to the MI of the subset \(S\) and \(Y\), \(I(S;Y)\).
\end{theorem}

\begin{proof}
Similar to the Theorem \ref{th:ent}, consider any disjoint and complementary subsets \(S\) and \(S'\) within the set \(F\) ($S \cup S' = F$), and examine their MI relationships with \(Y\) as follows:
\begin{align*}
I(F;Y) &= I(S,S';Y) \\
&= \sum p(s,s',y) \log\left(\frac{p(s,s',y)}{p(s,s')p(y)}\right) \\
&= \sum p(s,s',y) \log\left(\frac{p(s,y)p(s'|s,y)}{p(s)p(s'|s)p(y)}\right) \\
&= \sum p(s,y) \log\left(\frac{p(s,y)}{p(s)p(y)}\right) \\
&\quad + \sum p(s,s',y) \log\left(\frac{p(s,s',y)p(s')}{p(s,s')p(s,y)}\right) \\
&= I(S;Y) + I(S';Y|S).
\end{align*}

As we can see, the MI of set \(F\) and \(Y\) is equal to the MI of \(S\) and \(Y\) (\(I(S;Y)\)) plus the CMI of \(S'\) and \(Y\) conditional on \(S\) (\(I(S';Y|S)\)).
Since \(I(S';Y|S)\) is a non-negative quantity, we can conclude that:
\[ \therefore I(F;Y) \geq I(S;Y), \]
holds for all \(S \subset F\).
\end{proof}

From Theorem \ref{th:ent}, we deduce that the information content of a feature subset, \(S\), is maximized only when it equals the information content of the entire feature set, \(F\), as illustrated in criterion \(C_2\) in Eq. \eqref{eq:criteria}.
In essence, identifying a subset \(S\) where \(S \subset F\) that meets criterion \(C_2\) ensures that the maximum information is utilized for the learning task.

However, satisfying \(C_2\) alone does not ascertain that the subset \(S\) provides the maximum information about the target variable, \(Y\). 
According to the interpretation of Theorem \ref{th:mi}, for a subset \(S\) to effectively replace \(F\) in the learning task without losing pertinent information about \(Y\), it is imperative that \(H(F, Y) = H(S, Y)\).
This condition ensures that the MI \(I(F; Y) = I(S; Y)\), necessitating the fulfilment of criterion \(C_3\).

In essence, criterion \(C_2\) is designed to ensure that the selected subset \(S\) maintains the same semantics as \(F\).
On the other hand, criterion \(C_3\) is dedicated to preserving the quality of \(S\) in predicting the target, \(Y\), ensuring that the chosen subset is reflective of the underlying data structures in learning tasks.

\subsection{Prposed SPFP Algorithm}
\label{sub:spfp}
Algorithm \ref{alg:spfsp} outlines the pseudo-code for the proposed Semantic-Preserving Feature Partitioning (SPFP) algorithm.
The algorithm takes as input the dataset (comprised of $F$ and $Y$), the minimum number of features $N_F$ within the views, $\theta_g$, the number of artificial views $N_\theta$, and the fraction $r$ of features to be eliminated from the feature space, $U$. 

The algorithm initializes by setting up the necessary parameters and structures.
A series of nested loops are utilized for the main computational processes.
The outer loop is responsible for constructing the views until the predefined number of views, $N_\theta$, is achieved.
Within this loop, a ``while'' loop is utilized to manage the feature selection process, ensuring that the stopping criteria defined in Eq. \eqref{eq:criteria} are satisfied. 

An aspect of adaptability is incorporated within the algorithm, where a user-specified proportion of the selected features, determined by $r$, are randomly removed from the feature space $U$.
This is particularly consequential when $r=1$, where all selected features are removed from the feature space, promoting diversity in feature selection. 

The innermost loop calculates the objective function, $J_{SPFP}$ (refer to Eq. \eqref{eq:algobj}), for each candidate feature in a temporary feature space, $U_t$.
The feature, $f_s$, associated with the maximum value of the objective function is included in the view construct, $\theta_g$.
This iterative process continues until the conditions $C_1$, $C_2$, and $C_3$ as per Eq. \eqref{eq:criteria} are satisfied.

\begin{algorithm}
\caption{The pseudo-code for the Semantic-Preserving Feature Partitioning (SPFP) algorithm.}
\label{alg:spfsp}
\begin{algorithmic}[1]
\renewcommand{\algorithmicrequire}{\textbf{Input:}}
\renewcommand{\algorithmicensure}{\textbf{Output:}}
\REQUIRE $F$, $Y$, $N_F$, $N_\theta$, $r$
\ENSURE  $\Theta$
\STATE $\Theta \leftarrow \phi$
\STATE $U \leftarrow F$
\FOR{$g = 1$ to $N_\theta$}
    \STATE $S \leftarrow \phi$
    \STATE $U_t \leftarrow U$
    \WHILE{$\left | S \right | \leq N_F \vee H(S) \leq H(F) \vee H(S,Y) \leq H(F,Y)$}
        \FOR{each $f_c \in U_t$}
            \STATE Calculate $J_{SPFP}(f_c)$ using Eq. \eqref{eq:algobj}
        \ENDFOR
        \STATE $f_s \leftarrow argmax(J_{SPFP}(\cdot))$
        \STATE $S \leftarrow S \cup \{f_s \}$
        \STATE $U_t \leftarrow U_t - \{f_s \}$
    \ENDWHILE
    \STATE $\theta_g \leftarrow S$
    \STATE $\Theta \leftarrow \Theta \cup \{\theta_g\}$
    \STATE Randomly remove $r \times \left | \theta_g \right |$ features from $U$
\ENDFOR
\RETURN $\Theta$
\end{algorithmic}
\end{algorithm}

The SPFP method in Algorithm \ref{alg:spfsp} involves selecting multiple views \(\Theta = \{\theta_g | \theta_g \subset F, \theta_g \models C_1, C_2, C_3,~ \forall g=1,2,3,\ldots\}\).
Each subset \(\theta_g\) acts as an artificial construction of views of \(Y\), ensuring that they independently carry sufficient information for predicting \(Y\).

\subsection{Conditional Independence Assumption in MVL}
\label{sub:spfpcond}
To analyze the validity of the conditional independence of multiple views in MVL, let's consider the conventional assumption that any pair of views, such as \(\theta_1\) and \(\theta_2\), are conditionally independent given \(Y\) \cite{chen2012large, feuz2017collegial, chen2022asm2tv}:

\begin{equation}
\label{eq:assummultiv}
\begin{aligned}
p(\theta_1,\theta_2|y) &= p(\theta_1|y) \cdot p(\theta_2|y) \\
&\Rightarrow \frac{p(\theta_1,s_2|y)}{p(\theta_1|y) \cdot p(\theta_2|y)} = 1.
\end{aligned}
\end{equation}

Taking the logarithm and multiplying both sides by \(p(\theta_1,\theta_2,y)\) transforms Eq. \eqref{eq:assummultiv} into entropy form:
\begin{equation}
\label{eq:assumcmi}
\begin{aligned}
\Rightarrow I(\theta_1;\theta_2|Y) &= H(\theta_1,Y) + H(\theta_2,Y) \\
&- H(\theta_1,\theta_2,Y) - H(Y) = 0.
\end{aligned}
\end{equation}

Equation \eqref{eq:assumcmi} suggests that the CMI of \(\theta_1\) and \(\theta_2\) given \(Y\) is zero.
Using constraints \(C_2\) and \(C_3\) in Eq. \eqref{eq:criteria}, and referring to the Theorems \ref{th:ent}, and \ref{th:mi}, we further explore this assumption through the following equations:
\begin{equation}
\label{eq:inv1}
H(\theta_1,Y) = H(\theta_2,Y) = H(\theta_1,\theta_2,Y) = H(F,Y),
\end{equation}
\begin{equation}
\label{eq:inv2}
\Rightarrow I(\theta_1;\theta_2|Y) = H(F,Y) - H(Y) = 0,
\end{equation}
\begin{equation}
\label{eq:inv3}
\Rightarrow H(F,Y) = H(Y),
\end{equation}
\begin{equation}
\label{eq:inv4}
\Rightarrow I(F;Y) = H(F) + H(Y) - H(F,Y) = H(F).
\end{equation}
\begin{equation}
\label{eq:inv5}
\&~ 0 \leq I(F;Y) \leq \min(H(F),H(Y)),
\end{equation}
\begin{equation}
\label{eq:inv6}
\Rightarrow H(F) \leq H(Y).
\end{equation}

The equality in Eq. \eqref{eq:inv1}, which is evident from \(C_3\) as stated in Eq. \eqref{eq:criteria}, simplifies Eq. \eqref{eq:assumcmi} into Eqs. \eqref{eq:inv2} and \eqref{eq:inv3}.
Additionally, Eq. \eqref{eq:inv4} provides the definition of MI between the entire feature set, \(F\), and the target \(Y\) for a given dataset based on the assumption of conditional Independence. 

\(I(F;Y)\) is a non-negative quantity bounded by \(\min (H(F),H(Y))\) as previously mentioned and shown in Eq. \eqref{eq:inv5}. 
Thus, it is imperative for the equality in Eq. \eqref{eq:inv4} that the information content of the entire feature set should not exceed the entropy of the target variable \(Y\) for the assumption in Eq. \eqref{eq:assummultiv} to be valid, unless there is a compromise in the quality and semantics of the constructed views, as per the violation of constraint \(C_3\) in Eq. \eqref{eq:inv1}.
Such a compromise, however, is unlikely to be feasible in practical, real-world single-source datasets \cite{brefeld2015multi, wen2019unified}.

\section{Experiments}
\label{sec:experiments}
This section is organized into three distinct parts.
Section~\ref{sub:data} offers a detailed description of the datasets utilized in our experiments, which are designed to evaluate the performance of the SPFP algorithm.
Section~\ref{sub:setup} outlines the setup of these experiments, including the chosen performance metrics, the machine-learning algorithms used, and the statistical methods employed in the analysis.
Finally, Section~\ref{sub:analysis} explore the presentation and discussion of the results obtained from these experiments, examining their implications and significance.
Note that a substantial portion of the preliminary results are presented in tables and figures within the Supplementary Document.
Throughout this section, any tables and figures from the Supplementary Document are referenced with an ``S.'' prefix.

\subsection{Data Description}
\label{sub:data}
We evaluate the proposed SPFP algorithm using eight diverse datasets. These datasets are selected to represent a range of domains and complexity levels. The datasets include: APS Failure at Scania Trucks (APSF) \cite{apsf}, Activity Recognition Using Wearable Physiological Measurements (ARWPM) \cite{arwpm}, Gene Expression Cancer RNA-Sequence (GECR) \cite{gecr}, Grammatical Facial Expressions (GFE) \cite{gfe}, Gas Sensor Array Drift Dataset at Different Concentrations (GSAD) \cite{gsad}, Smartphone-Based Recognition of Human Activities and Postural Transitions (HAPT) \cite{hapt}, ISOLET \cite{isolet}, and Parkinson's Disease (PD) \cite{pd}. These datasets are publicly available at the UCI Machine Learning Repository. Table \ref{tab:data} details the number of instances, features, and classes for each dataset.

\begin{table}[!ht]
\caption{Datasets' description.}
\label{tab:data}
\centering
\begin{tabular}{|c|c|c|c|c|}
\hline
\textbf{Dataset} & \textbf{\# Instances} & \textbf{\# Features} & \textbf{\# Classes}\\
\hline
APSF    & \(75,994\)    & \(170\)       & \(2\)  \\
ARWPM   & \(4,480\)     & \(533\)       & \(5\)  \\
GECR    & \(801\)       & \(20,531\)    & \(5\)  \\
GFE     & \(27,965\)    & \(301\)       & \(2\)  \\
GSAD    & \(13,910\)    & \(129\)       & \(6\)  \\
HAPT    & \(10,929\)    & \(561\)       & \(12\) \\
ISOLET  & \(7,797\)     & \(617\)       & \(26\) \\
PD      & \(756\)       & \(753\)       & \(2\)  \\
\hline
\end{tabular}
\end{table}

As indicated in Table \ref{tab:data}, the datasets span binary and multi-class classification tasks. The number of instances in these datasets varies from 756 to 75,994, while the feature dimensionality extends from 129 to 20,531. 

\subsection{Experimental Setup}
\label{sub:setup}
We employed these eight benchmark datasets to rigorously assess the generalization capability of the proposed SPFP algorithm. 
Our evaluation protocol involved a multi-phase iterative process, repeated 30 times for each dataset and classifier to ensure statistical robustness.
Each iteration consists of the following phases: 
\begin{enumerate}
    \item Random partitioning of the dataset into training and testing sets,
    \item Construction of multiple views using the SPFP algorithm on the training set, 
    \item Hyperparameter tuning of the machine-learning models using cross-validation on the training set, 
    \item Training the models on these views, and,
    \item Evaluating the performance of these models on the testing set.
\end{enumerate} 
The algorithm's effectiveness was determined by comparing the generalization performance of models trained on individual views and their ensemble against those trained on the complete dataset.
This comprehensive approach not only tests the generalization capability of the models but also serves as an indirect measure of the SPFP algorithm's efficacy in enhancing model performance.

For experimental validation, we randomly partitioned each dataset into a 67\% training set and a 33\% testing set. 
The SPFP algorithm was then applied to the training data to generate multiple artificial views. 
The parameters for this process were set as follows: the number of artificial views \( N_\theta=5 \), the proportion of features to be randomly excluded from the feature space \( r=0.6 \), and the minimum number of features \( N_F=0.1 \times | F | \) for all datasets except GECR, and \(N_F=0.01 \times |F| \) for GECR dataset (for details, see Section \ref{sub:spfp}).
For the GECR dataset, which has a notably higher number of features compared to the other datasets, we opted for a significantly lower $N_F$ value.

The parameter values for the SPFP algorithm were selected based on informed estimations rather than exhaustive optimization, to reflect a practical scenario where users seek satisfactory generalization from an MEL framework without committing to extensive parameter tuning of the SPFP algorithm.

It is important to note that assigning a large number to $N_\theta$ and setting $r=1$ would lead to the partitioning of all dataset features into numerous artificial views, thereby allowing the SPFP algorithm to function at its full potential.
However, we intentionally selected a fixed number of views with significant potential for overlap ($N_\theta=5$ and $r=0.6$).
This decision was made to prevent the SPFP algorithm from dividing all features in a dataset while enabling it to construct views that represent similar dataset aspects.
Adopting this pragmatic approach allows us to explore and identify possible limitations of the SPFP algorithm and to more thoroughly examine its performance.

To fine-tune the hyperparameters, we utilized stratified 5-fold cross-validation on the training data for both the Extreme Gradient Boosting (XGBoost) and Logistic Regression (LR) models. 
For our study, we chose XGBoost, a complex model incorporating built-in FS and ensemble learning, to assess the capability of the SPFP algorithm in enhancing model performance.
Additionally, we selected LR, a simpler model, to examine the SPFP algorithm's impact on the performance of various types of models.
The range of hyperparameters investigated for the XGBoost and LR models, along with their respective search ranges, are detailed in Table \ref{tab:hyp}.

\begin{table}[!ht]
\caption{The hyperparameters and their range used for fine-tuning the XGBoost and LR Models.}
\label{tab:hyp}
\centering
\begin{tabular}{|c|c|c|}
\hline
\textbf{Model} & \textbf{Hyperparameter} & \textbf{Range}\\
\hline
\multirow{9}{*}{XGBoost} & Learning Rate             & $[0.01,0.2]$\\
                         & Gamma                     & $[0,10]$    \\
                         & Maximum Tree Depth        & $[3,12]$    \\
                         & Minimum Child Weight      & $[1,20]$    \\
                         & Sub-sample                & $[0.1,1]$   \\
                         & Feature Sample by Tree    & $[0.1,1]$   \\
                         & $L_1$ Regularization      & $[0.01,50]$ \\
                         & $L_2$ Regularization      & $[0.01,50]$ \\
                         & Estimators                & $[50,600]$  \\
\hline
\multirow{2}{*}{LR} & Penalty                   & $\{L_1, L_2\}$   \\
                    & C                         & $[0.001,1]$    \\
\hline
\end{tabular}
\end{table}

Subsequent to parameter optimization, the XGBoost and LR models were trained on the entire training set using the identified optimal hyperparameters.
The generalization performance of these models was then evaluated on the test set. 
The evaluation metrics included the micro-averaged \( F_1 \) score, area under the receiver operating characteristic curve (AUC), cross-entropy (or log-loss), the mean entropy of correct predictions (MEC), the mean entropy of wrong predictions (MEW), and models' execution time.

The \( F_1 \) score, combining precision and recall, is particularly effective in assessing model performance on imbalanced datasets by treating each instance equally, regardless of class.
The AUC, measuring the model's ability to discriminate between classes, provides insight into the overall classification effectiveness across various thresholds.  
The log-loss quantifies the model's prediction accuracy, penalizing significantly for confident yet incorrect predictions; it reflects how close the predicted probability distribution is to the true distribution.
In contrast, MEC and MEW focus on the model's certainty in its predictions. 
The MEC measures the average entropy (uncertainty) of the predictions that are correct, while MEW does the same for incorrect predictions. 
Achieving a higher correct confidence (i.e., a lower MEC) and a lower incorrect confidence (i.e., a higher MEW) is crucial, as it implies that the model not only accurately predicts outcomes but also does so with a high degree of certainty in correct predictions and skepticism in incorrect ones, thereby enhancing the reliability and trustworthiness of its decision-making process.

To statistically validate the differences in model performance across various metrics, we employed the Friedman test, a non-parametric statistical test used to detect differences in treatment effects across multiple treatments.
This test was applied to assess whether the median ranks of the evaluation metrics (\( F_1 \) score, AUC, log-loss, MEC, MEW, and execution time) differ significantly across the models at an \(\alpha=0.05\) significance level.
Upon rejection of the null hypothesis indicating no difference, the Conover post-hoc test was utilized as a follow-up analysis.
This test helps identify which specific models demonstrate superior performance by comparing the performance ranks of the models pairwise.

Additionally, we utilized Cliff's Delta (Cliff's $\delta$) analysis to assess the extent of difference in the metrics of the models obtained.
Within the context of Cliff's $\delta$, a difference between two models is categorized as negligible if $|\delta| < 0.147$, small if $0.147 \leq |\delta| < 0.333$, medium if $0.333 \leq |\delta| < 0.474$, and large if $0.474 \leq |\delta| < 1$.
Furthermore, the 95\% confidence intervals for Cliff's $\delta$ were determined using 10,000 bootstrap resampling. 
Detailed statistical analysis and discussions on the control of type I and II errors can be found in Section~\ref{sub:analysis}.
Such a comprehensive statistical approach ensures a rigorous assessment of the SPFP algorithm's impact on the models' performance.

\subsection{Results and Analysis}
\label{sub:analysis}
The performance of the SPFP algorithm across the eight benchmark datasets is summarized in Tables~\ref{tab:spfspresults}, and illustrated in Figures~\ref{fig:sppcommon} and~\ref{fig:sppall}. Table~\ref{tab:spfspresults} provides a detailed overview of the characteristics of the views constructed by the SPFP algorithm.
The average number of features, \( |\theta_g| \), in the constructed views varies between 25.4 and 206, corresponding to ASPF with the second lowest (170) and GECR with the highest (20,531) number of features (refer to Table~\ref{tab:data}).
Considering that \( N_F \) is set to \( 0.1 \times |F| \) and \( 0.01 \times |F| \) for GECR in the SPFP algorithm, the ratio of selected features to the total features, \( \frac{|\theta_g|}{|F|} \), varies between 0.01 and 0.27 for GECR and GASD datasets.
This elucidates the SPFP algorithm's efficiency in reducing dimensionality while maintaining a comprehensive representation of the original feature set.

Most of the obtained ratios exceed the specified \( N_F \), indicating that criteria \( C_2 \) and \( C_3 \) were the controlling factors for terminating the SPFP algorithm.
The standard deviation of the selected features is negligible for all cases except for ASPF (0.02), GFE (0.04), and GSAD (0.06).
This low standard deviation demonstrates the versatility of the SPFP algorithm in consistently identifying the underlying data structure within different feature sets across all runs.
During 30 runs with different training sets, SPFP selected a relatively equal number of features, even though the total number of unique (distinct) features, $\mathbf{\left| \bigcup_{\theta_g \in \Theta} \theta_g \right|}$, used for view construction shows much more diversity (as detailed in Table~\ref{tab:spfspresults}).

The ASPF, GFE, and GSAD datasets exhibit the highest ratio of used features in all views to the total number of features, \( \frac{|\bigcup_{\theta_g \in \Theta} \theta_g|}{|F|} \), among all datasets with 0.51, 0.73, and 0.89 of their total features used in the construction of views, respectively.
This variability highlights the algorithm's adaptability to different dataset complexities and sizes.
The total number of unique features across all views (\( |\bigcup_{\theta_g \in \Theta} \theta_g| \)) and the common features in all views (\( |\bigcap_{\theta_g \in \Theta} \theta_g| \)) provide insights into the diversity and overlap of the features selected by the SPFP algorithm in the construction of the five views.

Despite setting \( r = 0.6 \), implying that the constructed views could potentially have 40\% common features (\( |\bigcap_{\theta_g \in \Theta} \theta_g| \)), there are relatively low overlapping features among all views suggesting the existence of different patterns within the datasets.
However, the ISOLET dataset exhibits a significant number of common features (\(31.2 \pm 2.83\)), indicating a strong correlation between frequently selected features and other features consistently identified by the SPFP algorithm.
Additionally, the time required for the construction of views is reported in the table, with significant variation observed across datasets.
For example, the GECR dataset with 20,531 features required considerably more time (\(2087.2 \pm 327.06\) seconds) compared to other datasets, reflecting the increased complexity and computational demand of handling larger feature sets.


\begin{table*}
\centering
\caption{Summary of views' characteristics constructed by SPFP Algorithm (mean $\pm$ standard deviation), including the number of features in each view (\( |\theta_g| \)), the features across all views (\( |\cup_{\theta_g \in \Theta} \theta_g| \)), the features common to all views (\( |\cap_{\theta_g \in \Theta} \theta_g| \)), the ratio of features in each view to the original feature set (\( \frac{|\theta_g|}{|F|} \)), the ratio of the features used in all views to the original feature set (\( \frac{|\cup_{\theta_g \in \Theta} \theta_g|}{|F|} \)), and the time elapsed for the construction of views (Time in seconds).}
\label{tab:spfspresults}
\begin{tabular}{c|cccccc} 
\hline
Dataset & $\mathbf{\left| \theta_g \right|}$ & $\mathbf{\left| \bigcup_{\theta_g \in \Theta} \theta_g \right|}$ & $\mathbf{\left| \bigcap_{\theta_g \in \Theta} \theta_g \right|}$ & $\mathbf{\frac{\left| \theta_g \right|}{\left| F \right|}}$ & $\mathbf{\frac{\left| \bigcup_{\theta_g \in \Theta} \theta_g \right|}{\left| F \right|}}$ & \textbf{Time (sec)}  \\ 
\hline
APSF & $25.4 \pm 2.7$ & $86.1 \pm 2.92$ & $0.5 \pm 0.72$ & $0.15 \pm 0.02$ & $0.51 \pm 0.02$ & $51.1 \pm 17.3$ \\
ARWPM & $54.0 \pm 0.0$ & $182.7 \pm 1.09$ & $1.3 \pm 1.11$ & $0.1 \pm 0.0$ & $0.34 \pm 0.0$ & $95.5 \pm 9.7$ \\
GECR & $206.0 \pm 0.0$ & $702.1 \pm 0.25$ & $5.7 \pm 2.28$ & $0.01 \pm 0.0$ & $0.03 \pm 0.0$ & $2087.2 \pm 327.06$ \\
GFE & $60.7 \pm 12.48$ & $220.6 \pm 9.06$ & $1.0 \pm 0.84$ & $0.2 \pm 0.04$ & $0.73 \pm 0.03$ & $590.8 \pm 74.16$ \\
GSAD & $34.8 \pm 8.11$ & $115.3 \pm 5.14$ & $0.9 \pm 0.67$ & $0.27 \pm 0.06$ & $0.89 \pm 0.04$ & $53.3 \pm 26.93$ \\
HAPT & $68.0 \pm 0.0$ & $233.8 \pm 1.32$ & $2.1 \pm 1.24$ & $0.12 \pm 0.0$ & $0.42 \pm 0.0$ & $170.7 \pm 23.4$ \\
ISOLET & $124.0 \pm 0.0$ & $272.0 \pm 0.0$ & $31.2 \pm 2.83$ & $0.2 \pm 0.0$ & $0.44 \pm 0.0$ & $352.7 \pm 38.9$ \\
PD & $76.0 \pm 0.0$ & $260.3 \pm 0.68$ & $2.3 \pm 1.7$ & $0.1 \pm 0.0$ & $0.35 \pm 0.0$ & $169.5 \pm 18.18$ \\

\hline
\end{tabular}
\end{table*} 

\begin{figure*}[!t] 
\centering

\subfloat[APSF]{\includegraphics[width=0.24\textwidth]{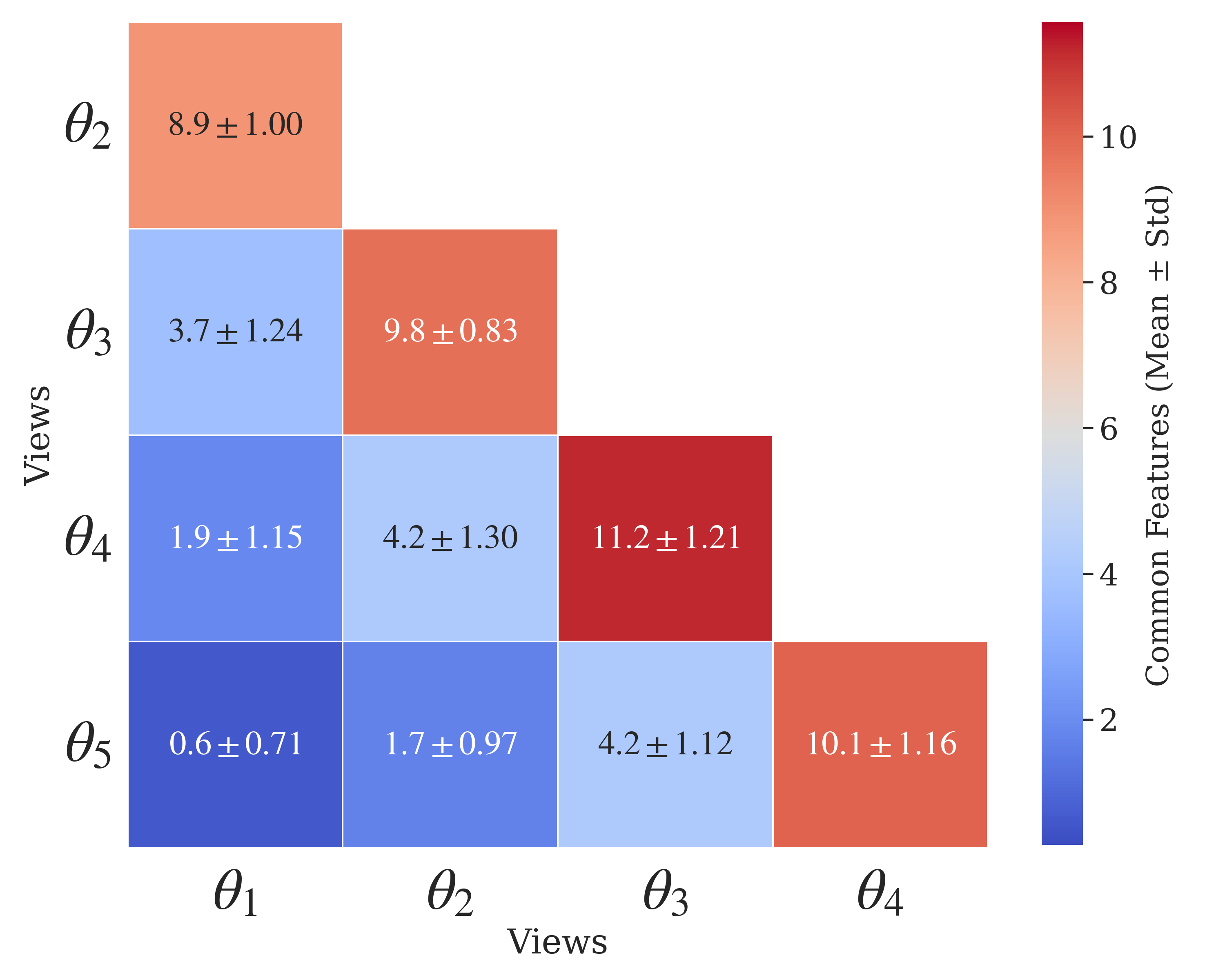}\label{fig:fpafps}}%
\hfill
\subfloat[ARWPM]{\includegraphics[width=0.24\textwidth]{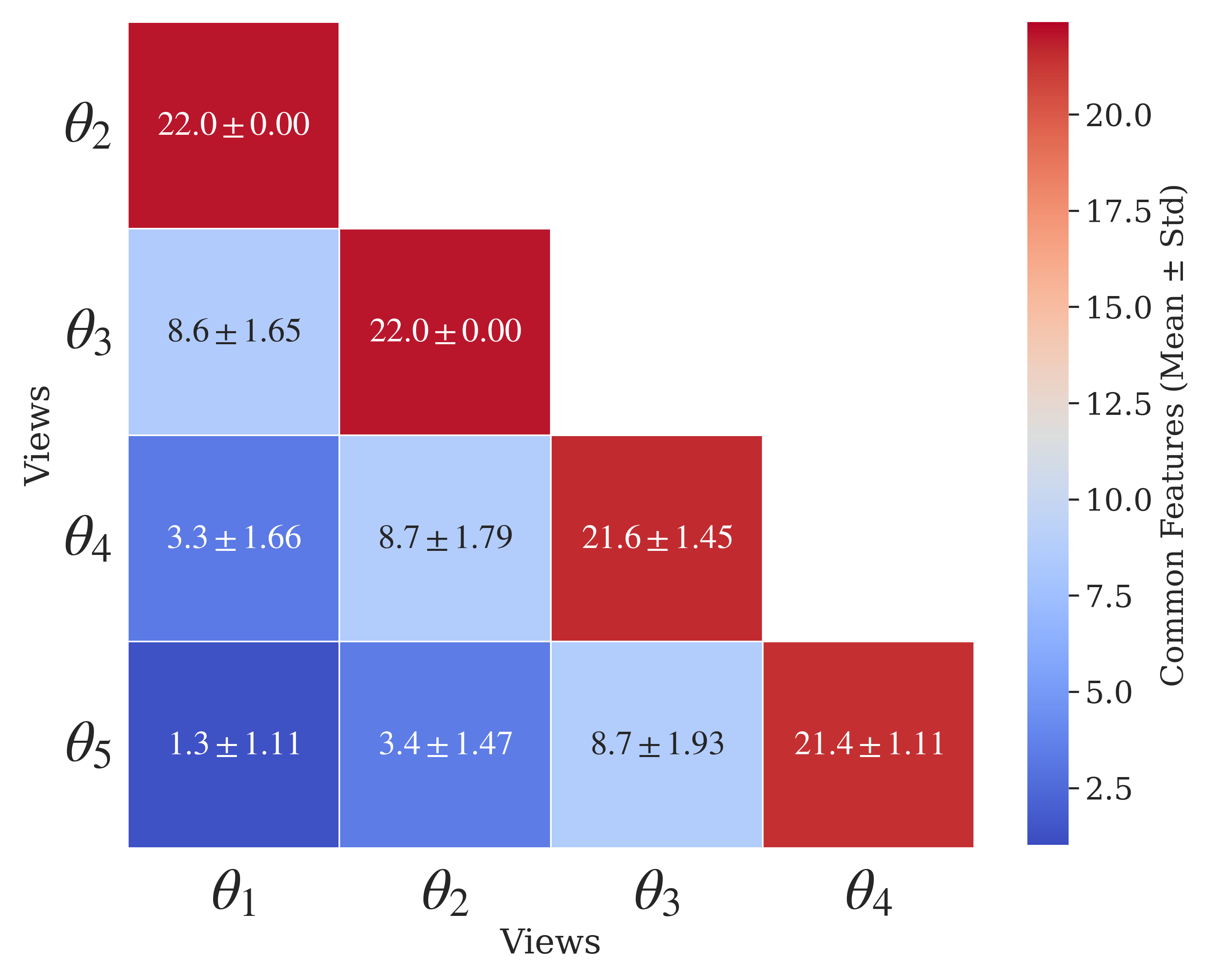}\label{fig:fparwpm}}%
\hfill
\subfloat[GECR]{\includegraphics[width=0.24\textwidth]{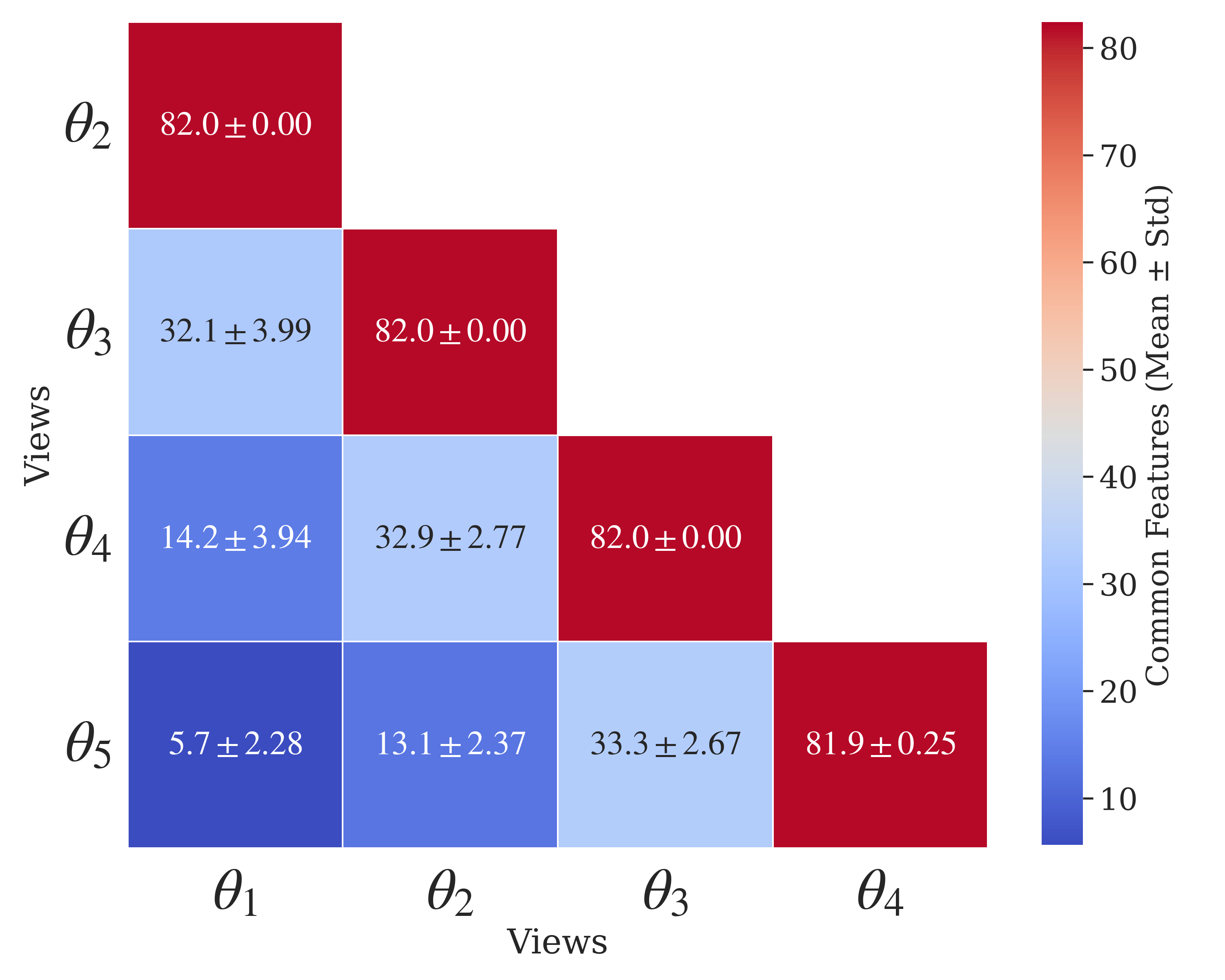}\label{fig:fpgecr}}%
\hfill
\subfloat[GFE]{\includegraphics[width=0.24\textwidth]{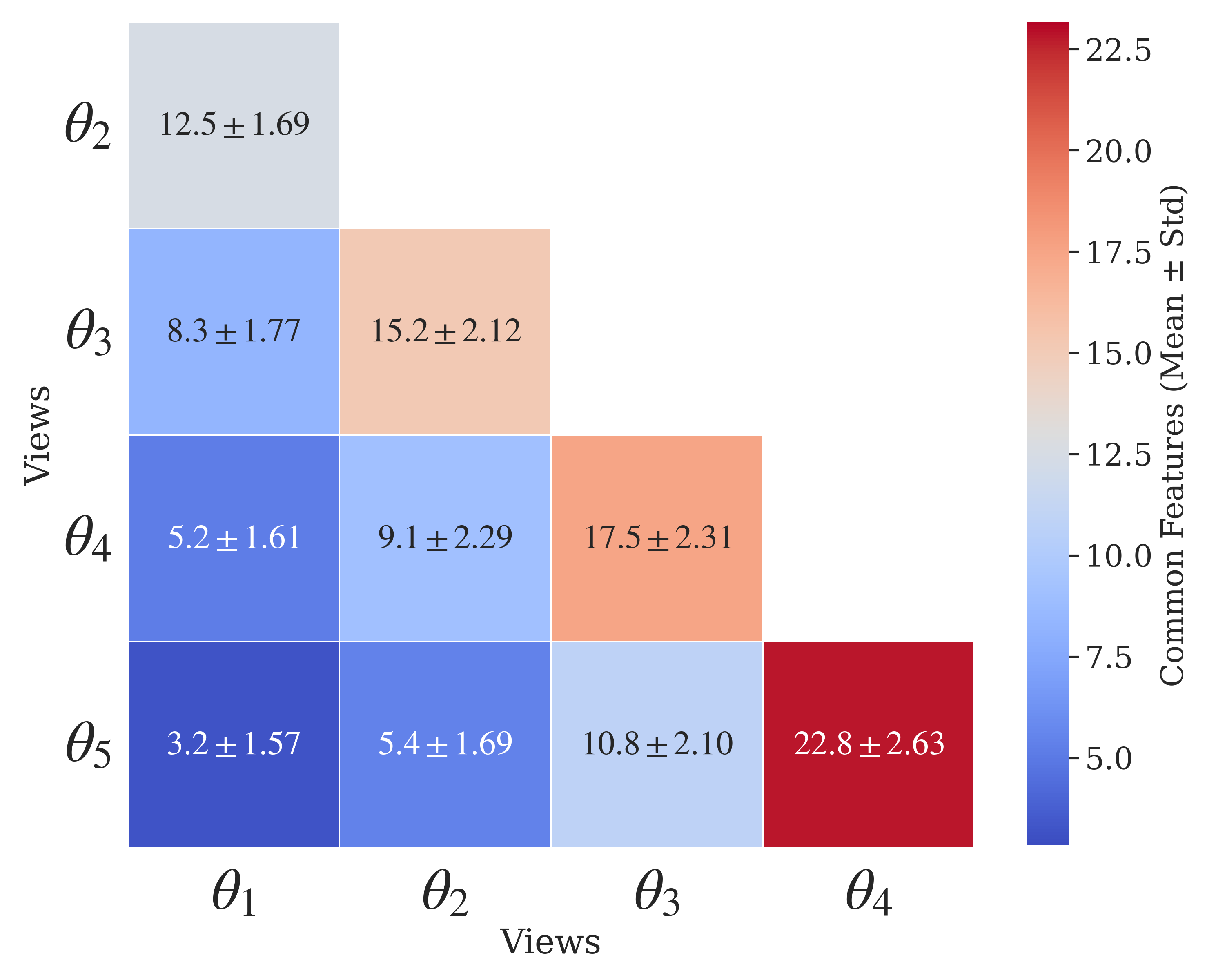}\label{fig:fpgfe}}

\subfloat[GSAD]{\includegraphics[width=0.24\textwidth]{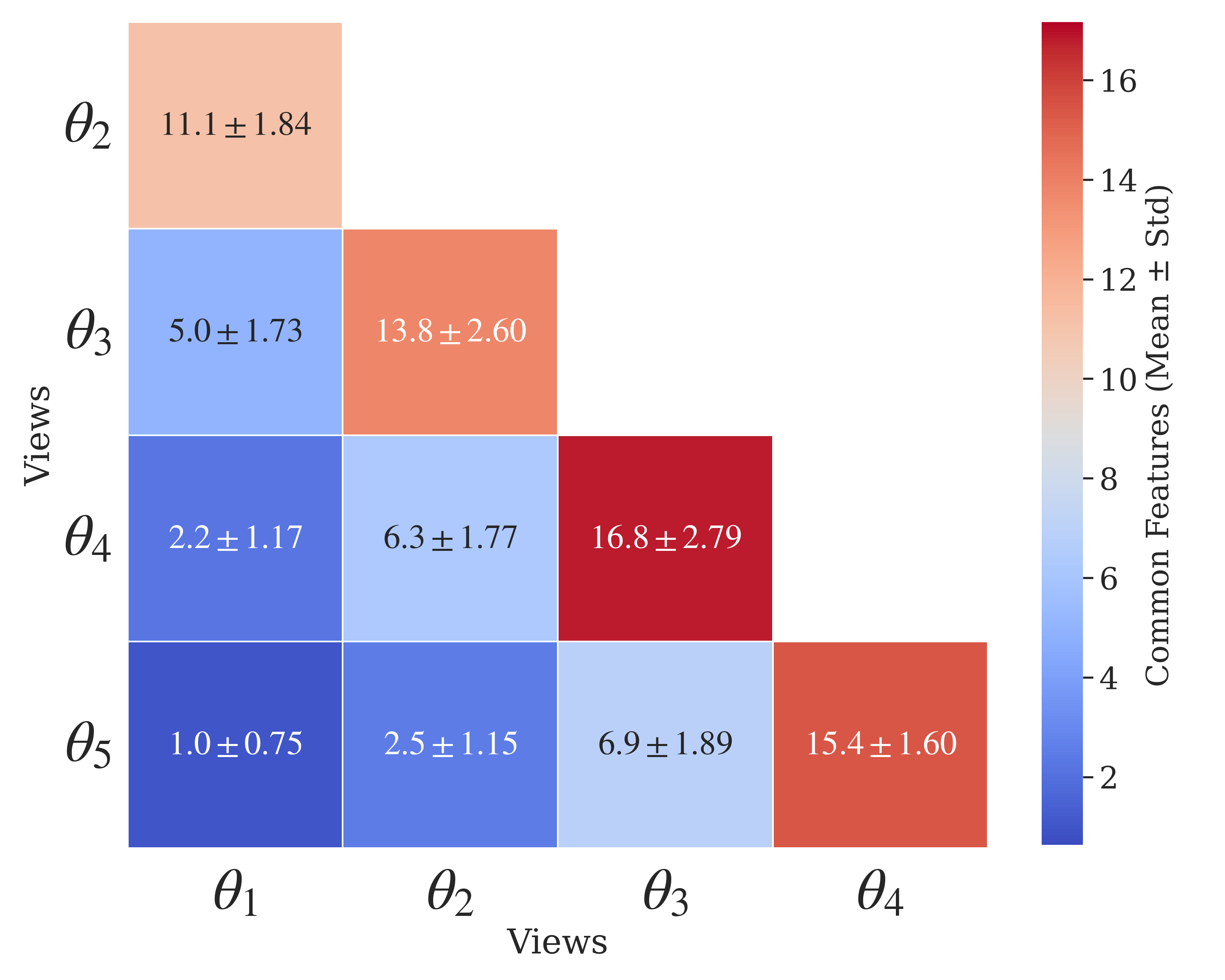}\label{fig:fpgsad}}%
\hfill
\subfloat[HAPT]{\includegraphics[width=0.24\textwidth]{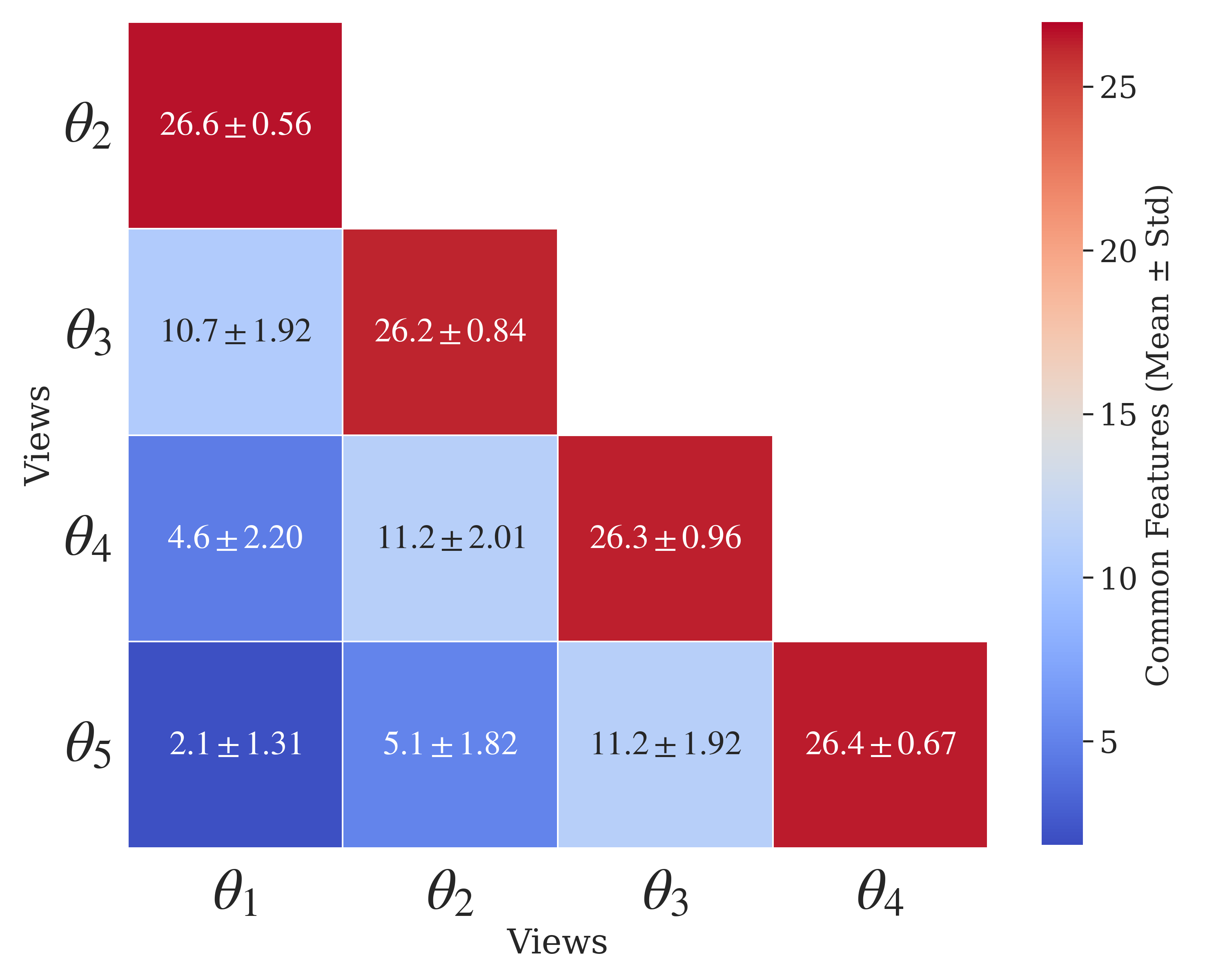}\label{fig:fphapt}}%
\hfill
\subfloat[ISOLET]{\includegraphics[width=0.24\textwidth]{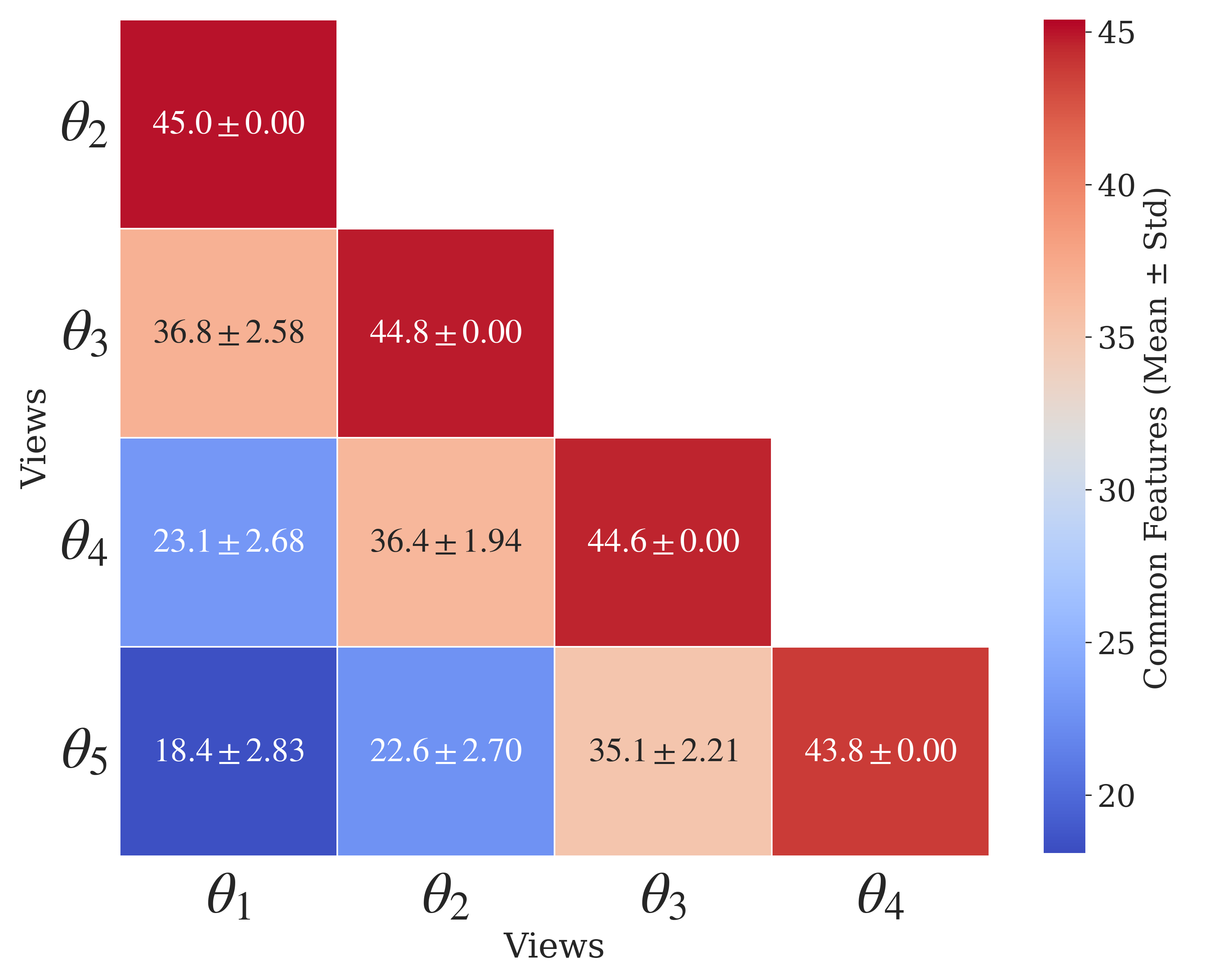}\label{fig:fpisolet}}%
\hfill
\subfloat[PD]{\includegraphics[width=0.24\textwidth]{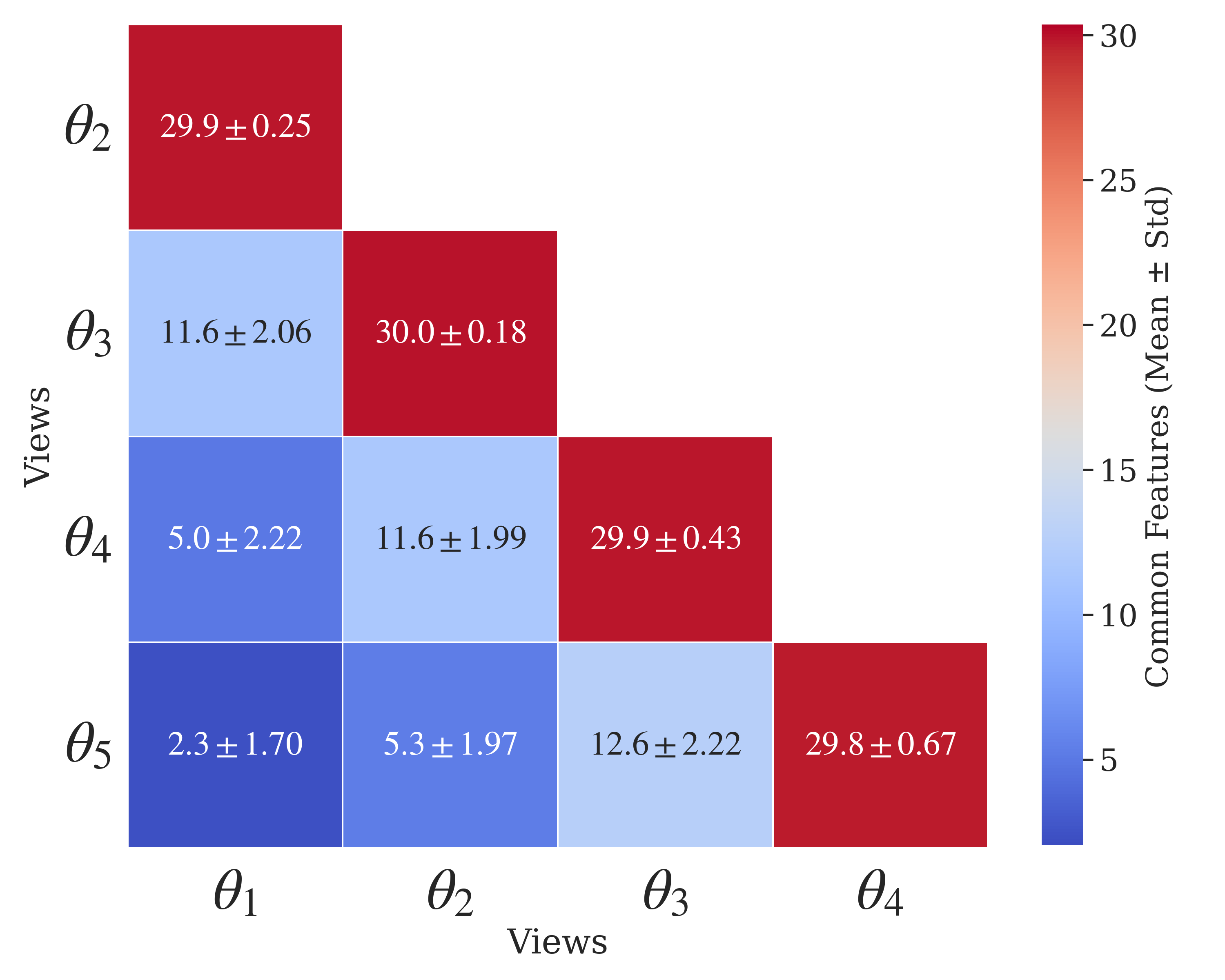}\label{fig:fppd}}

\caption{The number of common features among the artificial views generated by the SPFP algorithm, with parameters \( N_\theta = 5 \), \( N_F = 0.1 \times | F | \) and \( r = 0.6 \).}

\label{fig:sppcommon}
\end{figure*}

\begin{figure*}[!t] 
\centering

\subfloat[APSF]{\includegraphics[width=0.24\textwidth]{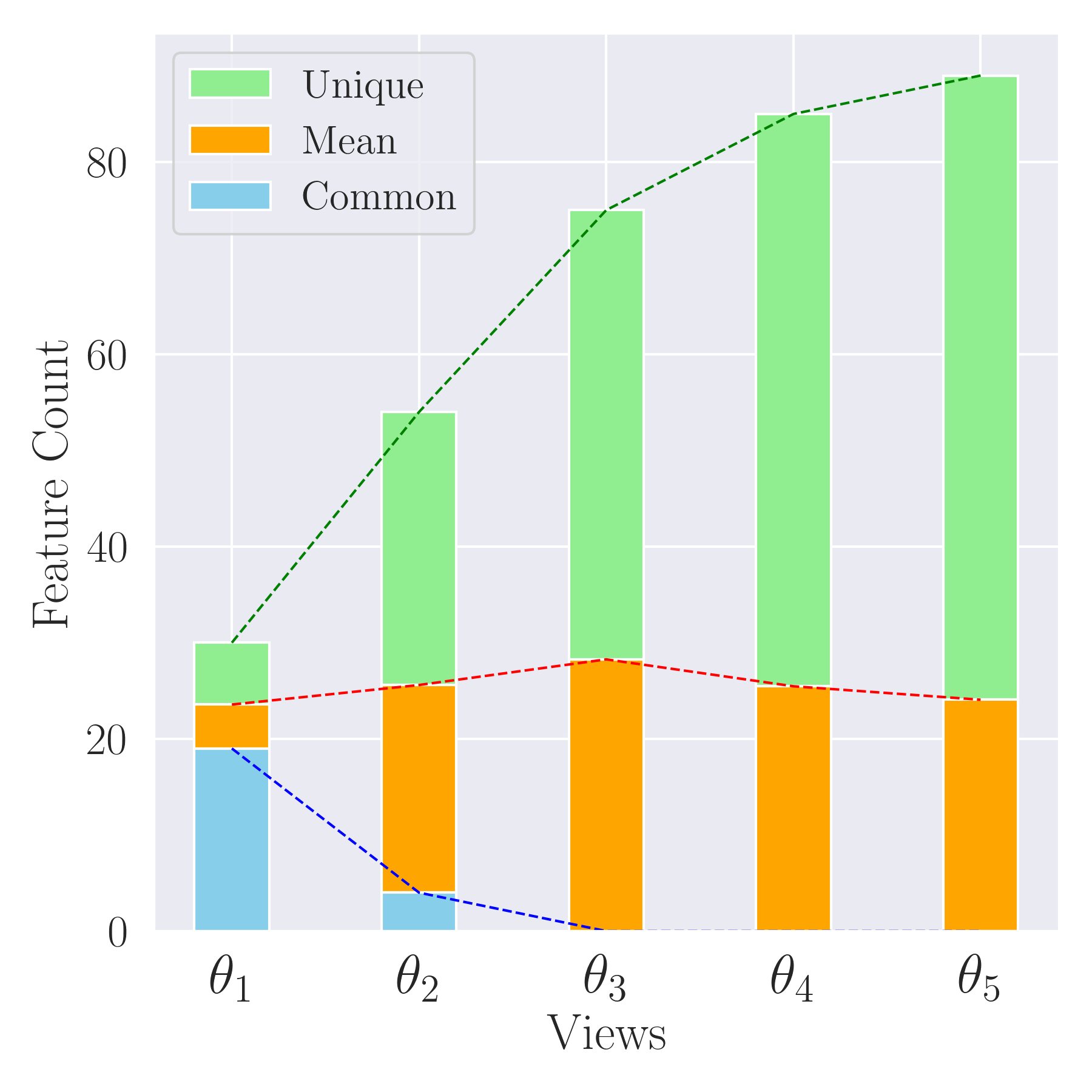}\label{fig:fpafps}}%
\hfill
\subfloat[ARWPM]{\includegraphics[width=0.24\textwidth]{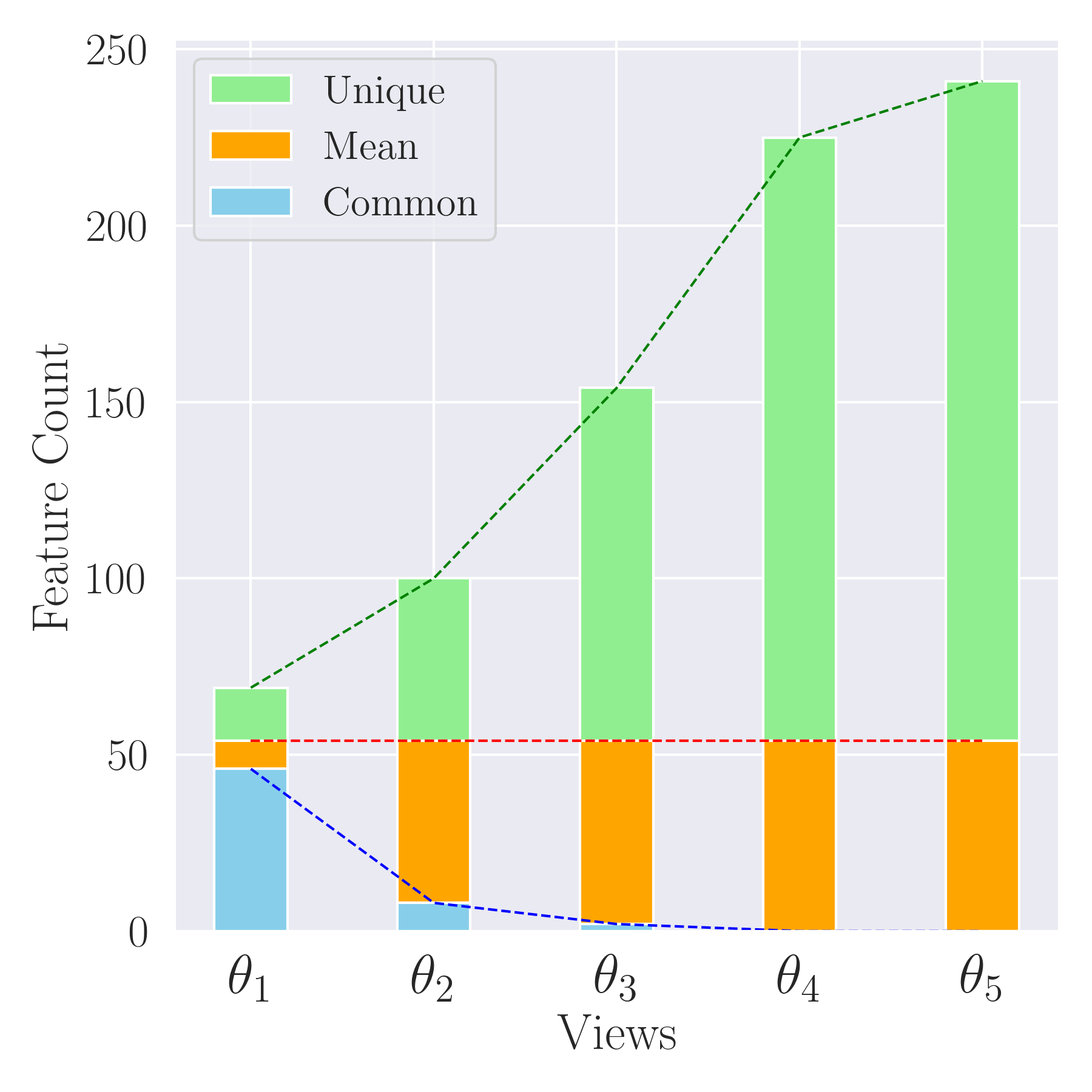}\label{fig:fparwpm}}%
\hfill
\subfloat[GECR]{\includegraphics[width=0.24\textwidth]{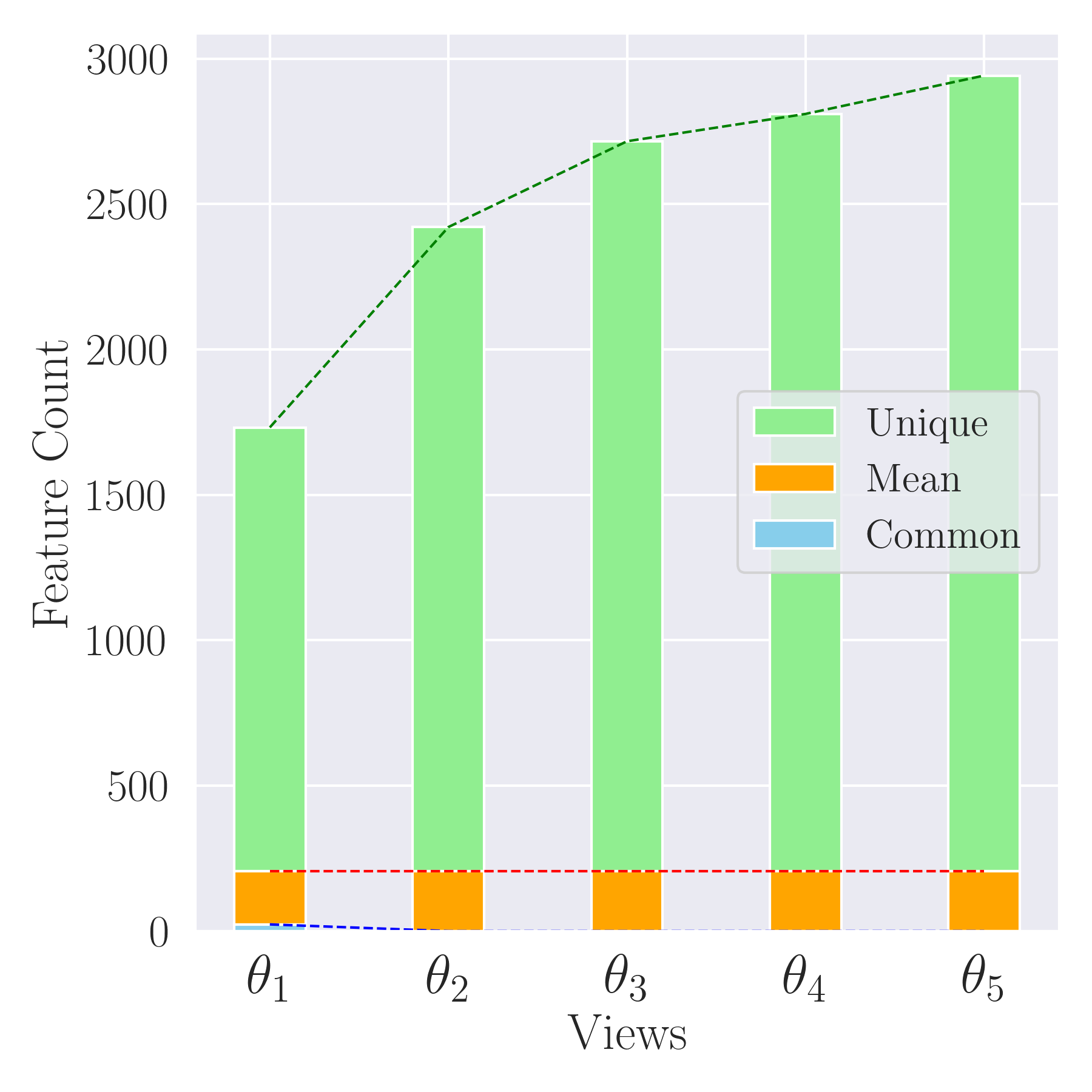}\label{fig:fpgecr}}%
\hfill
\subfloat[GFE]{\includegraphics[width=0.24\textwidth]{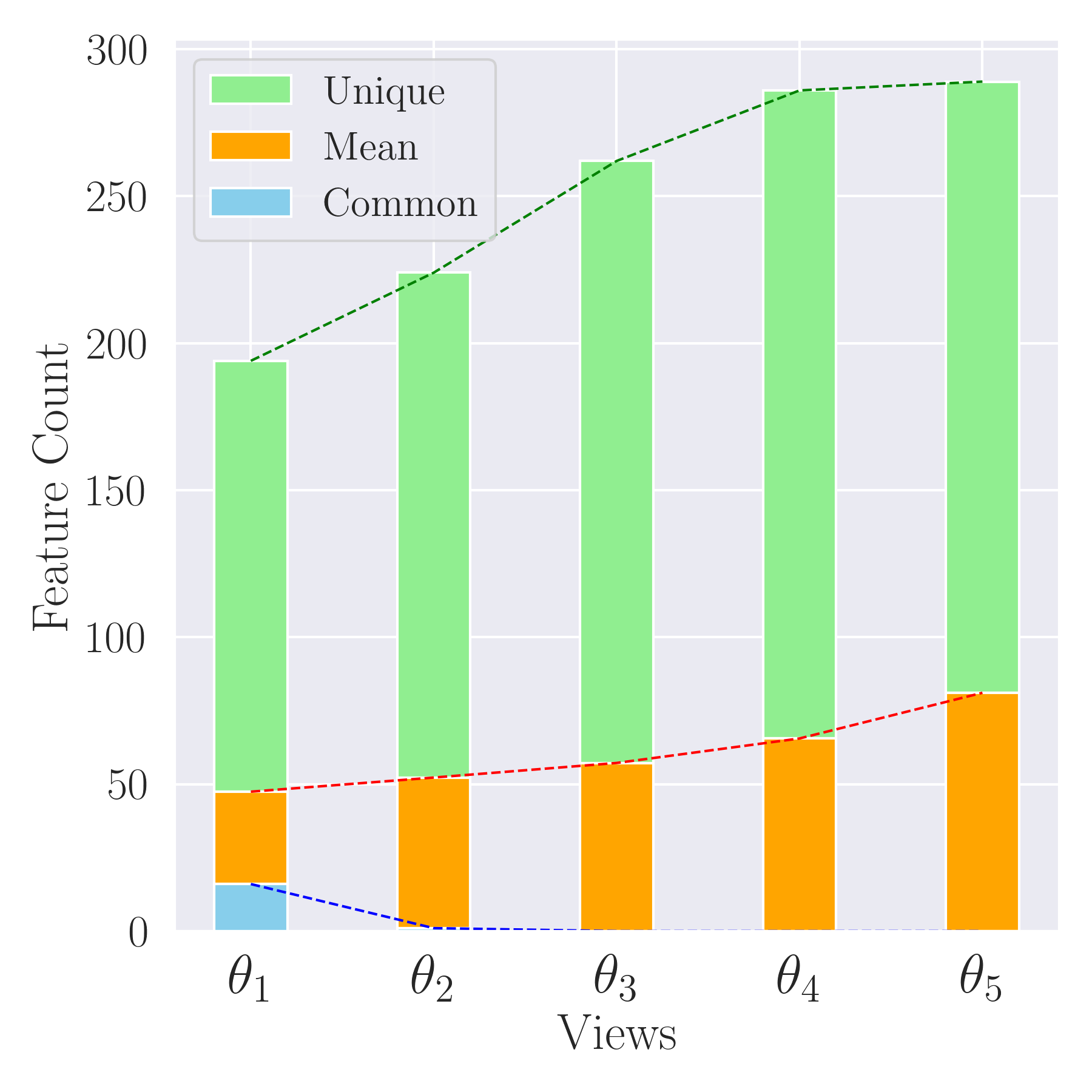}\label{fig:fpgfe}}

\subfloat[GSAD]{\includegraphics[width=0.24\textwidth]{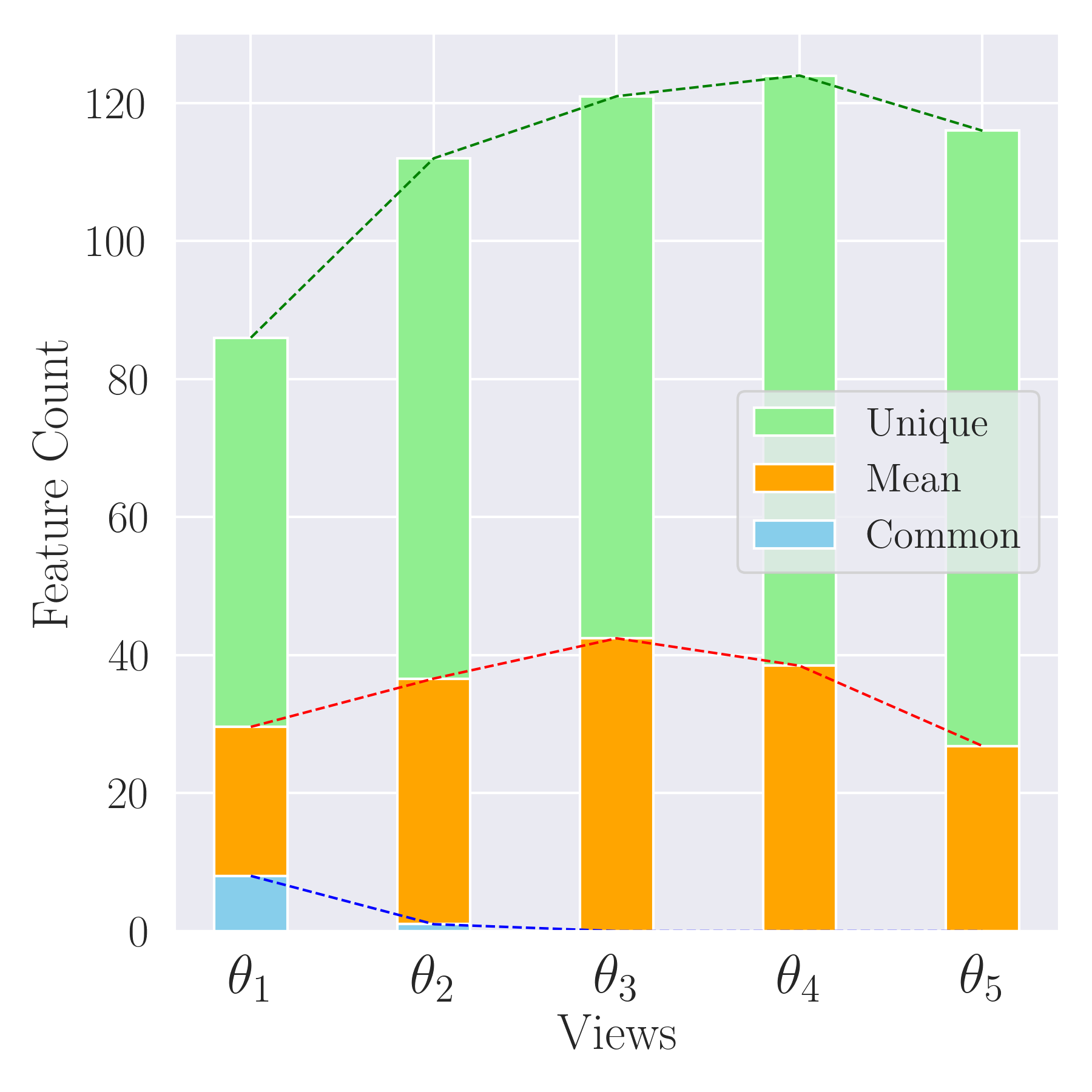}\label{fig:fpgsad}}%
\hfill
\subfloat[HAPT]{\includegraphics[width=0.24\textwidth]{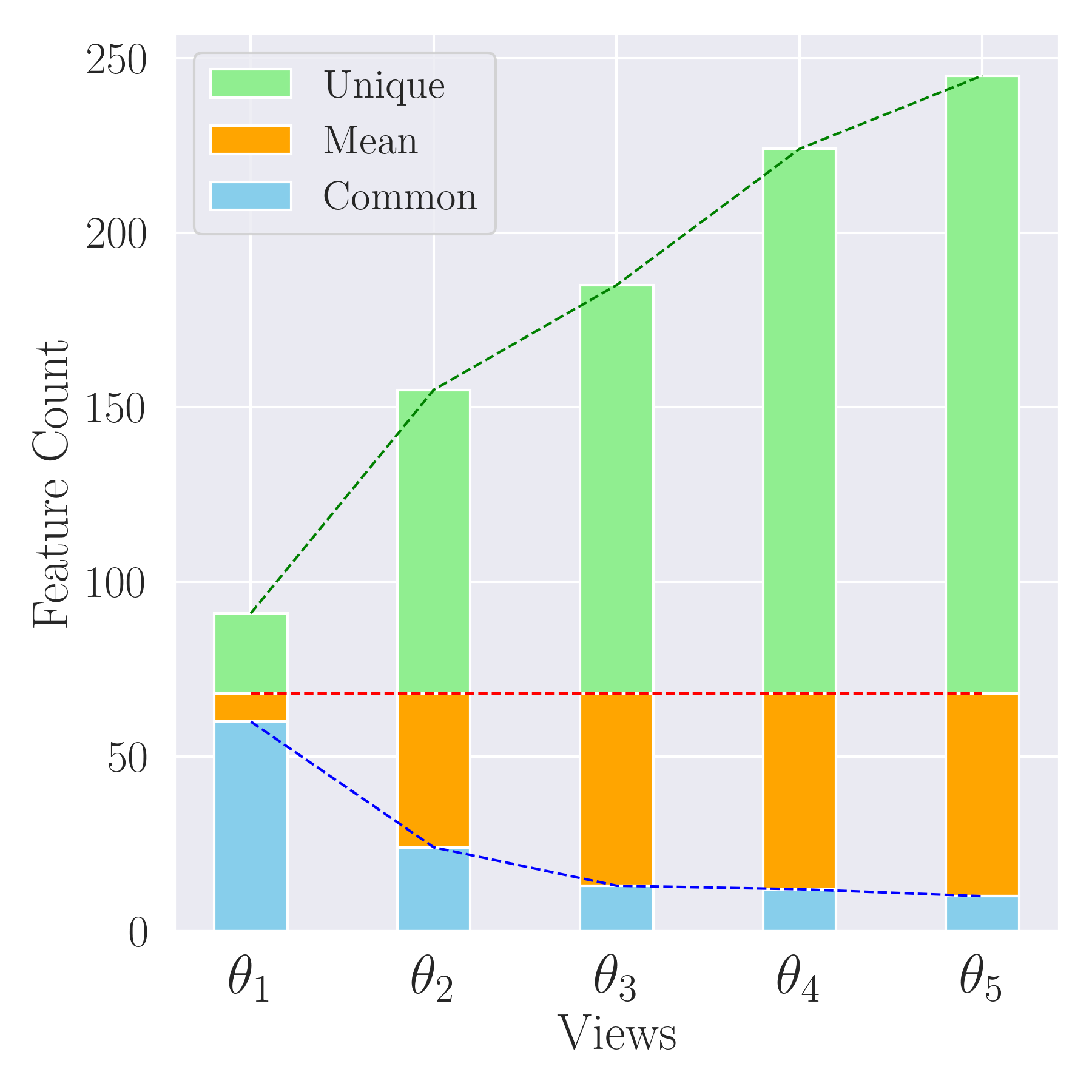}\label{fig:fphapt}}%
\hfill
\subfloat[ISOLET]{\includegraphics[width=0.24\textwidth]{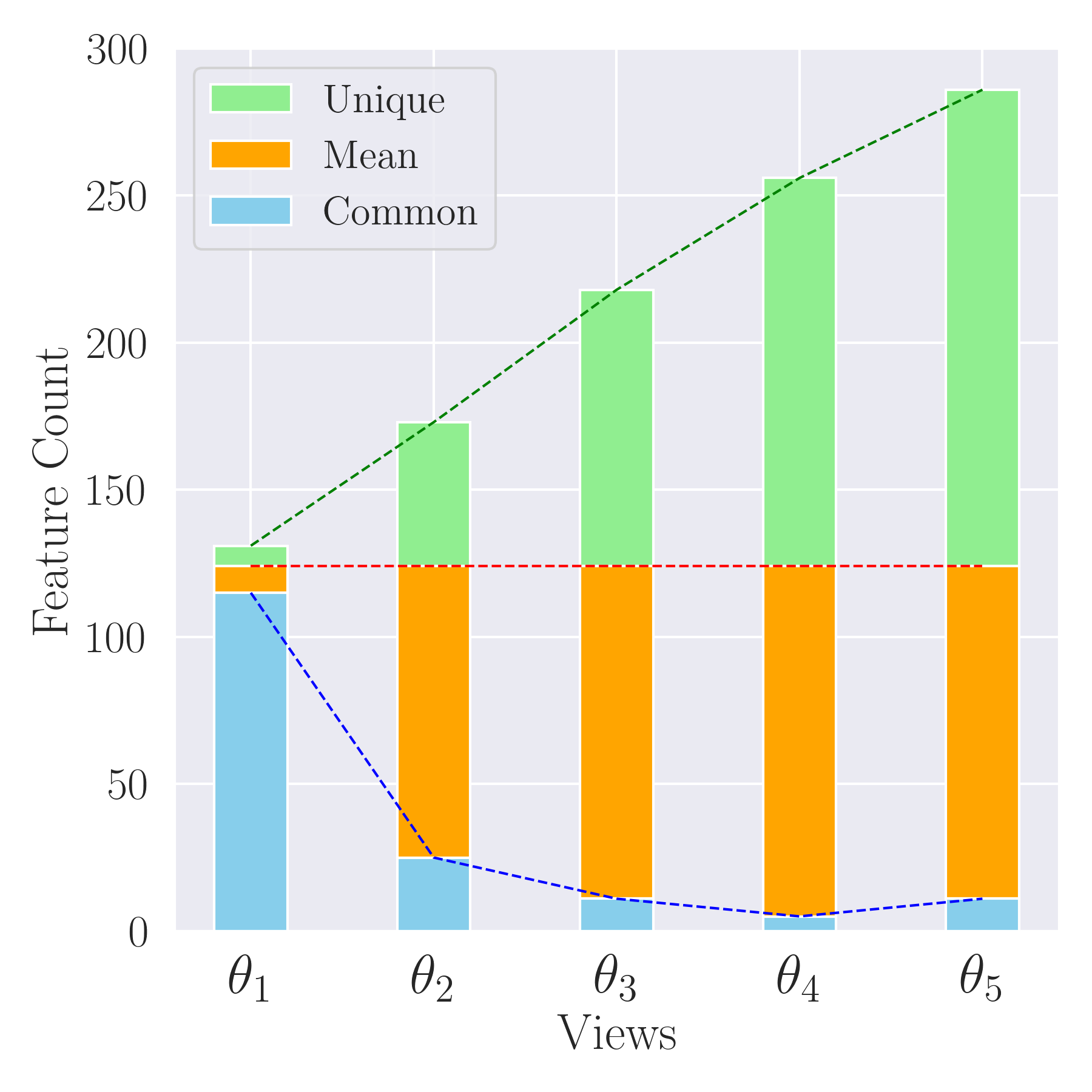}\label{fig:fpisolet}}%
\hfill
\subfloat[PD]{\includegraphics[width=0.24\textwidth]{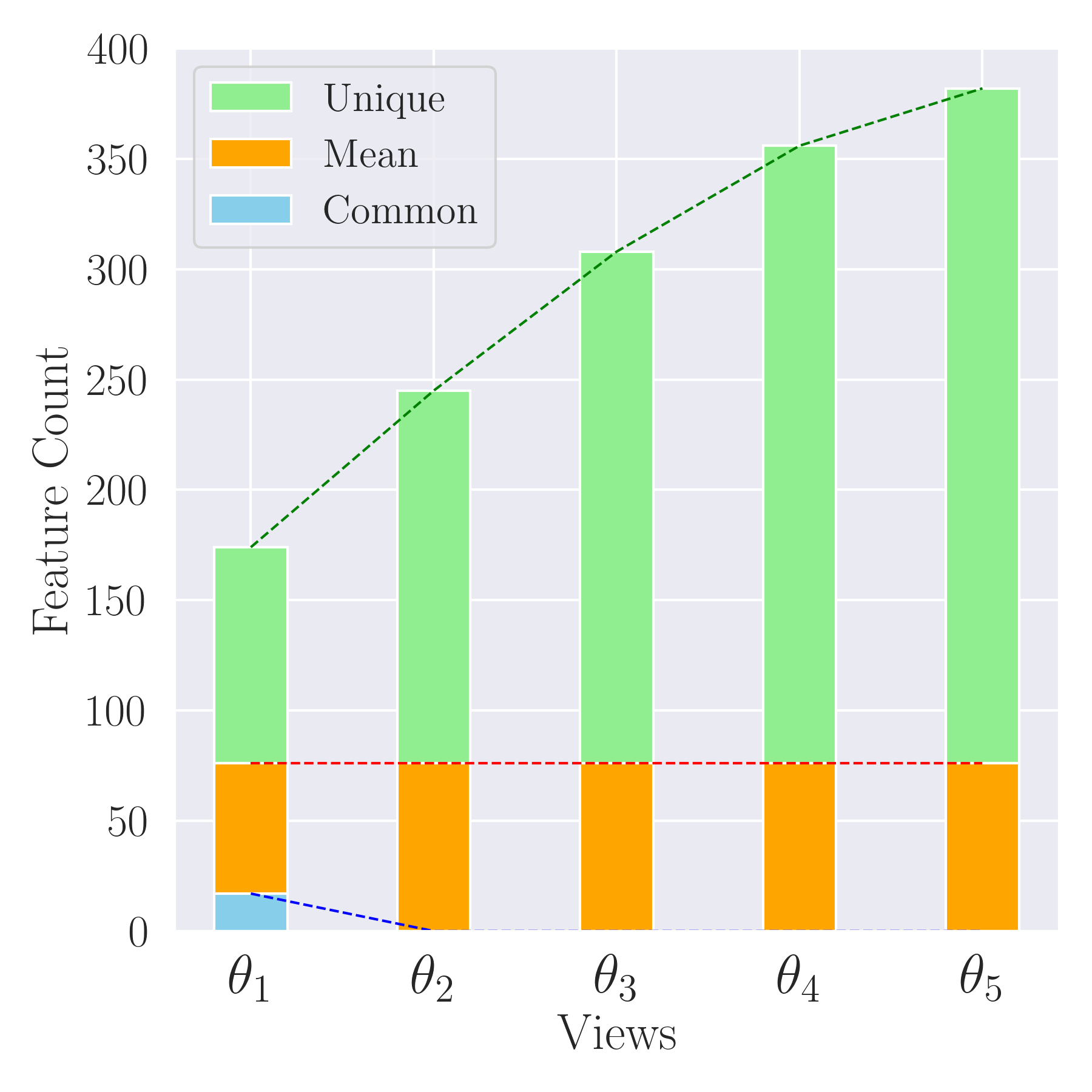}\label{fig:fppd}}

\caption{Overview of feature diversity results using the SPFP algorithm over 30 Runs. The figure displays, for each view, the average (mean) number of features selected per run (orange bars), the total count of unique features selected across all runs (green bars), and the number of common features across every run (blue bars).}

\label{fig:sppall}
\end{figure*}


Figure~\ref{fig:sppcommon} offers an in-depth matrix visualization that elucidates the shared features between various views created by the SPFP algorithm across different datasets.
Each cell within this matrix corresponds to the intersection of features between two distinct views, with the intensity of the cell's color reflecting the volume of common features.
The pattern revealed in Figure~\ref{fig:sppcommon} is particularly noteworthy, illustrating a pronounced sharing of features between successive view pairs, as indicated by the darker cells.
This observation signifies a substantial overlap in features among adjacent views, as opposed to those further apart.
The underlying mechanism, governed by the parameter \( r = 0.6 \), which dictates the exclusion of random features from the feature space before constructing the views, critically affects the distinctiveness and variety of each view.

Except for the GFE dataset, which exhibits approximately a 20\% overlap in features between successive views, the successive pairs of views show over 35\% of overlapping features observed in other datasets. Figure~\ref{fig:sppall} further elucidates the SPFP algorithm's performance by illustrating the mean number of features per view (denoted by orange bars), alongside the quantification of common (blue bars) and unique features (green bars) across 30 executions.
This representation distinctly manifests the cumulative impact of random training set selection and feature elimination (guided by \( r=0.6 \)) on the characteristics of views constructed by the SPFP algorithm.

The mean feature count across different views generally demonstrates uniformity for most datasets, with notable deviations in APFS, GFE, and GSAD.
In the context of GECR, ARWPM, and PD datasets, the algorithm's termination was contingent upon reaching the \( N_F \) threshold.
Pertinently, the third view \( \theta_3 \) in APFS and GSAD, and the fifth view \( \theta_5 \) in GFE, exhibited the highest mean feature count, corroborating the algorithm's efficacy in capturing representative feature subsets that conform to the stipulated criteria \( C_1 \), \( C_2 \), and \( C_3 \).

The disparity in the distribution of unique and common features for each view underscores the algorithm's variable feature selection propensity across different runs and datasets. 
The initial view \( \theta_1 \), predominantly affected by the randomness of training data selection, demonstrates a significant aggregation of common features contrasted with a limited diversity of unique features.
This pattern indicates a tendency of the SPFP algorithm towards a recurrent selection of a core feature subset within this view and its robustness against randomness in instances of the data.
The constancy in feature selection for \( \theta_1 \) is particularly pronounced in datasets such as APSF, ARWPM, HAPT, and ISOLET. Conversely, in datasets with a lower instance count, i.e., GECR and PD, the \( \theta_1 \) view shows a broader diversity of unique features and a lower proportion of common features to the mean feature count.

As the algorithm progresses from \( \theta_2 \) to \( \theta_5 \), the synergistic effect of randomly selecting training data and reducing the feature space (attributable to \( r=0.6 \)) incrementally amplifies the diversity of the selected features for each view, while simultaneously attenuating their commonality. 
For instance, in constructing the \( \theta_5 \) view for the GECR dataset, about 3,000 distinct features were utilized throughout the runs, yet none emerged as a recurrently selected feature.
This phenomenon accentuates the SPFP algorithm's adaptability and robustness in discerning diverse data patterns, ensuring the retention of the intrinsic semantic structures of the original dataset.


\begin{table*}
\centering
\caption{The summary of statistical comparison of results for testing data obtained from 30 XGBoost Runs. W, T, and L denote win, tie, and loss based on Friedman, and Conover adjusted p-values, and Cliff's $\delta$ effect size analysis.}
\label{tabmain:xgbwtl}
\begin{tabular}{cccccccccc} 
\hline
\multicolumn{10}{c}{XGBoost (Win - Tie - Loss)}                                                       \\ 
\hline
Metric   & $\theta_1$ & $\theta_2$ & $\theta_3$ & $\theta_4$ & $\theta_5$ & $E_{1:2}$ & $E_{1:3}$ & $E_{1:4}$ & $E_{1:5}$   \\ 
\hline
$F_1$    & 0 - 0 - 8     & 0 - 0 - 8     & 0 - 0 - 8     & 1 - 0 - 7     & 0 - 1 - 7     & 1 - 1 - 6     & 1 - 1 - 6     & 2 - 2 - 5     & 3 - 3 - 2  \\
AUC      & 0 - 0 - 8     & 0 - 0 - 8     & 0 - 0 - 8     & 1 - 0 - 7     & 0 - 1 - 7     & 1 - 1 - 6     & 1 - 2 - 5     & 4 - 1 - 3     & 6 - 1 - 1  \\
Log-Loss & 1 - 1 - 6     & 1 - 1 - 6     & 1 - 1 - 6     & 1 - 1 - 6     & 0 - 2 - 6     & 1 - 1 - 6     & 2 - 1 - 4     & 3 - 1 - 4     & 5 - 1 - 2  \\
MEC      & 0 - 0 - 8     & 0 - 0 - 8     & 0 - 0 - 8     & 1 - 0 - 7     & 0 - 1 - 7     & 0 - 3 - 5     & 1 - 2 - 5     & 2 - 2 - 4     & 3 - 3 - 2  \\
MEW      & 0 - 0 - 8     & 0 - 0 - 8     & 0 - 0 - 8     & 1 - 0 - 7     & 0 - 1 - 7     & 4 - 1 - 3     & 1 - 2 - 5     & 2 - 2 - 4     & 3 - 3 - 2  \\
Time     & 7 - 1 - 0     & 7 - 1 - 0     & 7 - 1 - 0     & 7 - 1 - 0     & 6 - 2 - 0     & --            & --            & --            & --         \\
\hline
\end{tabular}
\end{table*}

\begin{figure*}[!t] 
\centering

\subfloat[$F_{1}$ Score]{\includegraphics[width=0.33\textwidth]{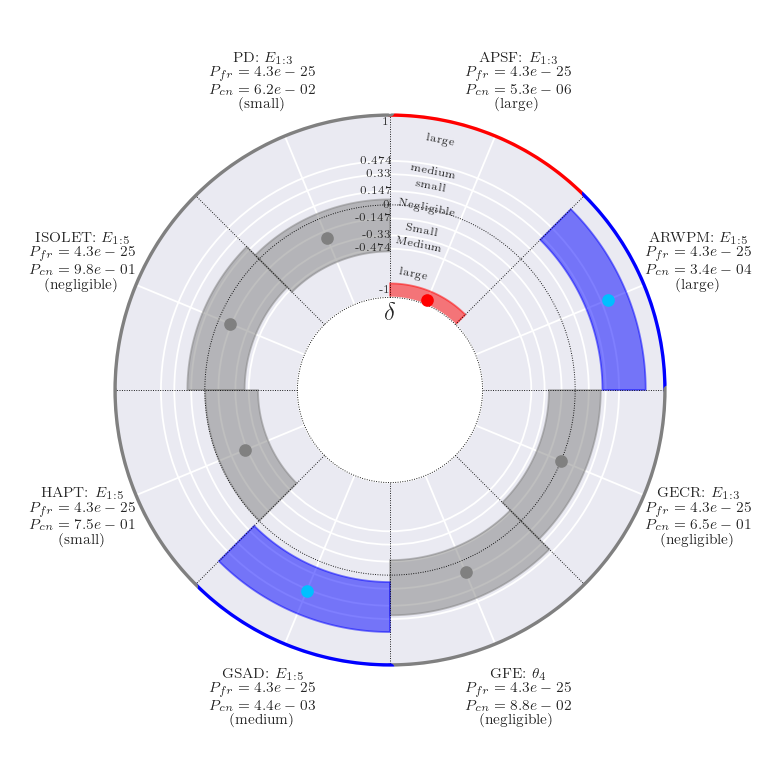}\label{fig:xgbwf1}}%
\hfill
\subfloat[AUC]{\includegraphics[width=0.33\textwidth]{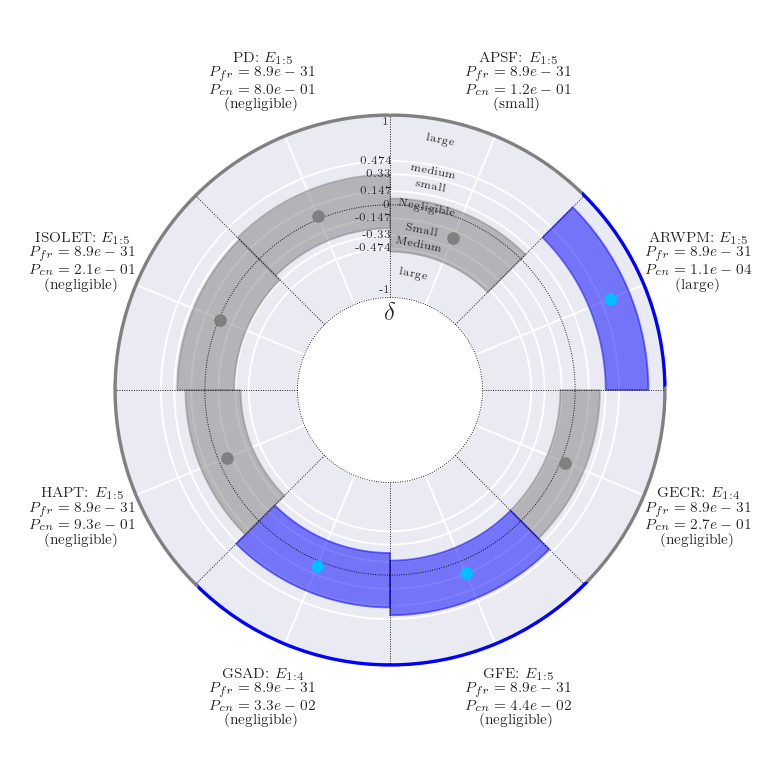}\label{fig:xgbwauc}}%
\hfill
\subfloat[Log-Loss]{\includegraphics[width=0.33\textwidth]{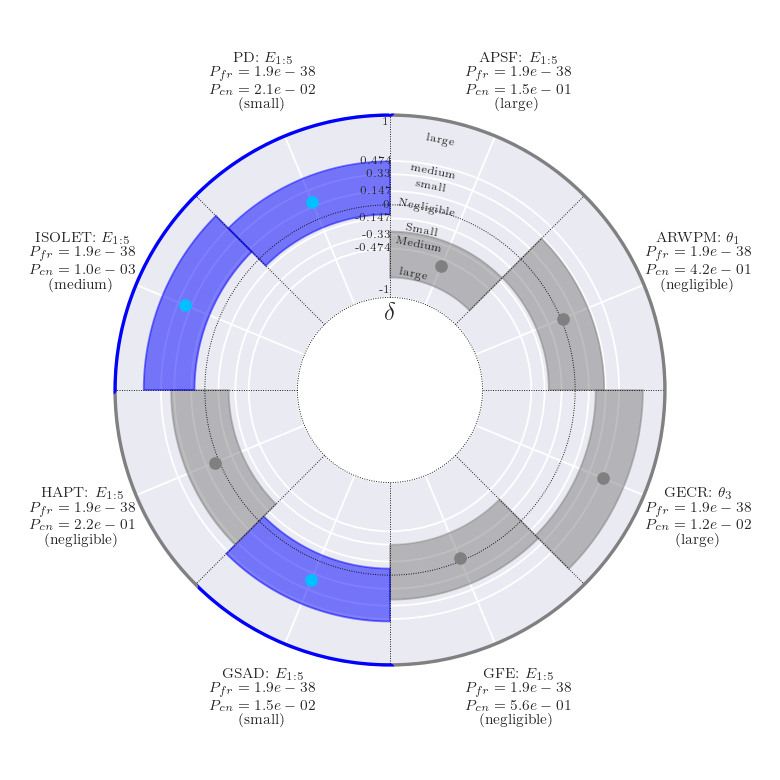}\label{fig:xgbwloss}}%

\subfloat[MEC]{\includegraphics[width=0.33\textwidth]{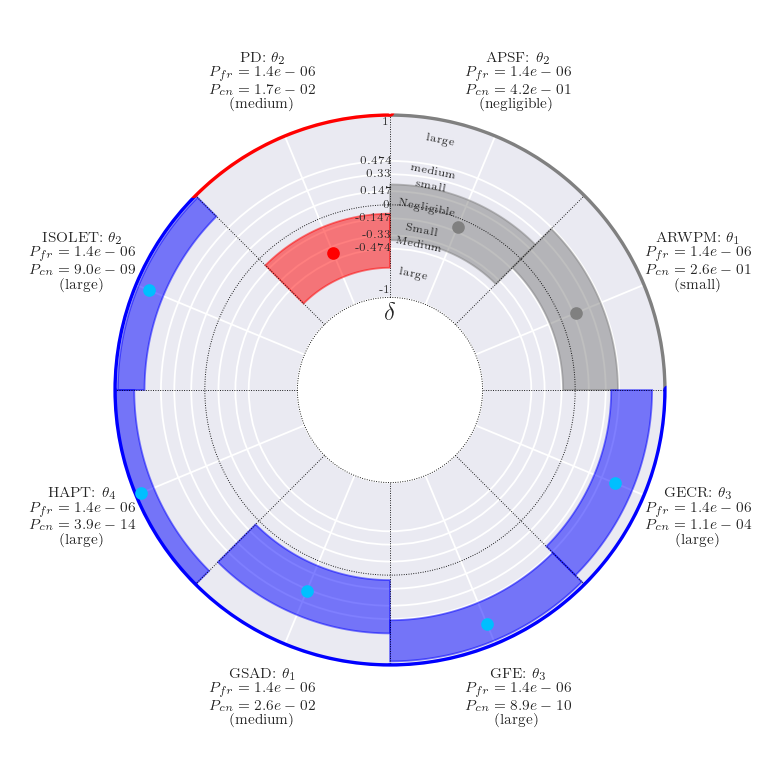}\label{fig:xgbwgsad}}%
\hfill
\subfloat[MEW]{\includegraphics[width=0.33\textwidth]{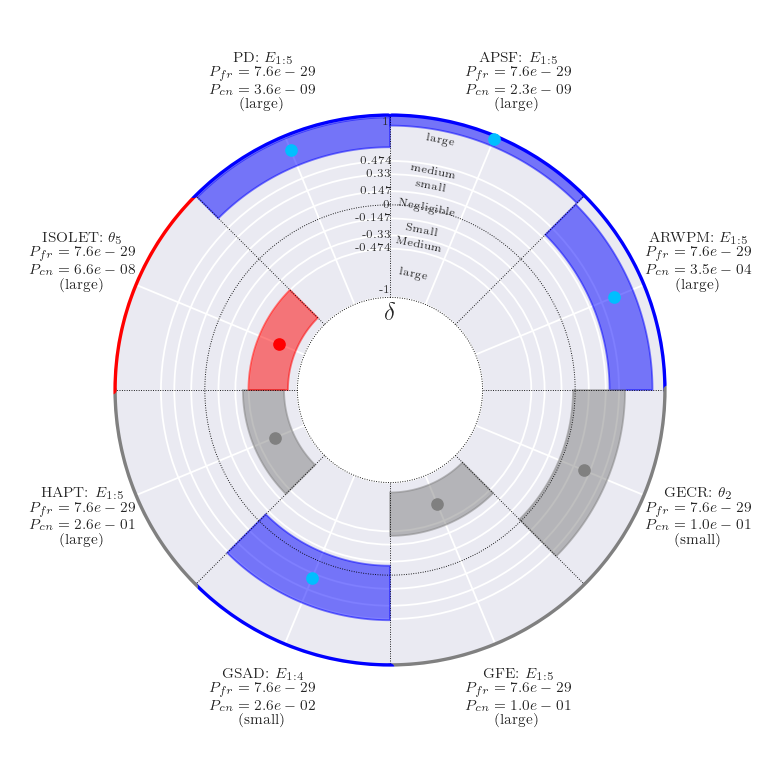}\label{fig:xgbwhapt}}%
\hfill
\subfloat[Time]{\includegraphics[width=0.33\textwidth]{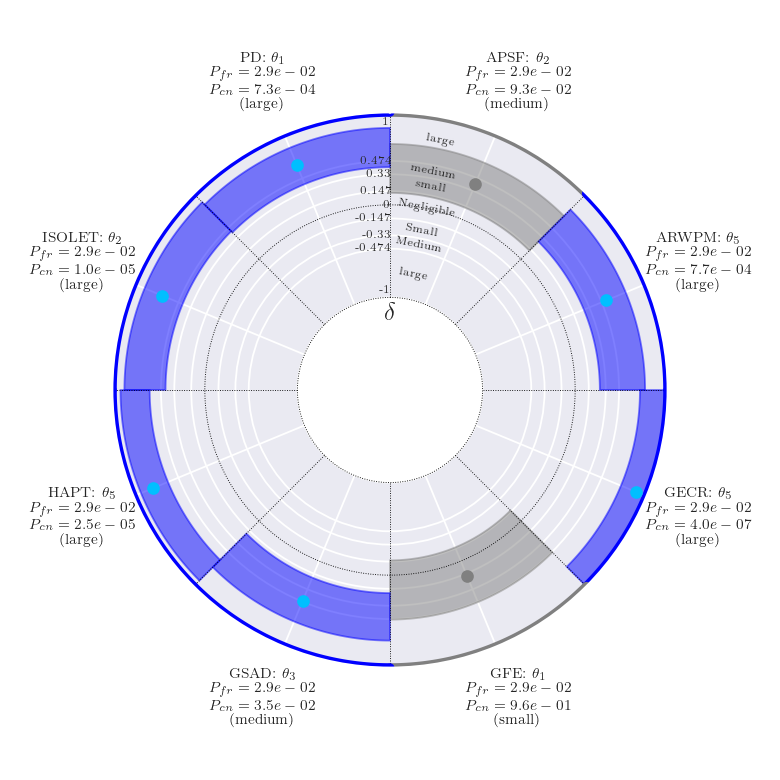}\label{fig:xgbwisolet}}%

\caption{The figure illustrates the best performing XGBoost models in comparison to the benchmark model, with the segments representing the 95\% confidence interval of Cliff's $\delta$ (center point). Grey segments indicate cases where either $P_{fr}>0.05$ or $P_{cn}>0.05$, suggesting no significant difference from the benchmark. Blue segments denote instances where the model outperforms the benchmark ($P_{fr}<0.05$, $P_{cn}<0.05$, and $\delta>0$), while red segments indicate that the benchmark model outperforms the corresponding XGBoost model ($P_{fr}<0.05$, $P_{cn}<0.05$, and $\delta<0$).}

\label{fig:xgbcliffwin}
\end{figure*}

\begin{table*}
\centering
\caption{The summary of statistical comparison of results for testing data obtained from Logistic Regression runs. W, T, and L denote win, tie, and loss based on Friedman and Conover's adjusted p-values, and Cliff's $\delta$ effect size analysis.}
\label{tabmain:lrwtl}
\begin{tabular}{cccccccccc}
\hline
\multicolumn{10}{c}{Logistic Regression (Win - Tie - Loss)}
             \\
\hline
Metric & $\theta_1$ & $\theta_2$ & $\theta_3$ & $\theta_4$ & $\theta_5$ & $E_{1:2}$ & $E_{1:3}$ & $E_{1:4}$ & $E_{1:5}$ \\
\hline
$F_{1}$ Score & 1 - 2 - 5 & 1 - 3 - 4 & 0 - 3 - 5 & 0 - 2 - 6 & 0 - 1 - 7 & 1 - 3 - 4 & 3 - 2 - 3 & 3 - 2 - 3 & 4 - 4 - 0 \\
AUC & 1 - 1 - 6 & 0 - 1 - 7 & 0 - 2 - 6 & 0 - 1 - 7 & 0 - 0 - 8 & 1 - 2 - 5 & 2 - 3 - 3 & 2 - 3 - 3 & 2 - 6 - 0 \\
Loss & 2 - 1 - 5 & 1 - 1 - 6 & 0 - 3 - 5 & 0 - 3 - 5 & 0 - 2 - 6 & 2 - 2 - 4 & 3 - 1 - 4 & 4 - 0 - 4 & 4 - 1 - 3 \\
MEC & 3 - 0 - 5 & 3 - 0 - 5 & 2 - 1 - 5 & 2 - 1 - 5 & 1 - 2 - 5 & 3 - 0 - 5 & 3 - 0 - 5 & 1 - 2 - 5 & 1 - 2 - 5 \\
MEW & 4 - 1 - 3 & 3 - 2 - 3 & 4 - 2 - 2 & 2 - 4 - 2 & 3 - 5 - 0 & 5 - 1 - 2 & 4 - 2 - 2 & 4 - 2 - 2 & 4 - 2 - 2 \\
Time & 8 - 0 - 0 & 8 - 0 - 0 & 8 - 0 - 0 & 8 - 0 - 0 & 8 - 0 - 0 & -- & -- & -- & -- \\
\hline
\end{tabular}
\end{table*}

\begin{figure*}[!t] 
\centering

\subfloat[$F_{1}$ Score]{\includegraphics[width=0.33\textwidth]{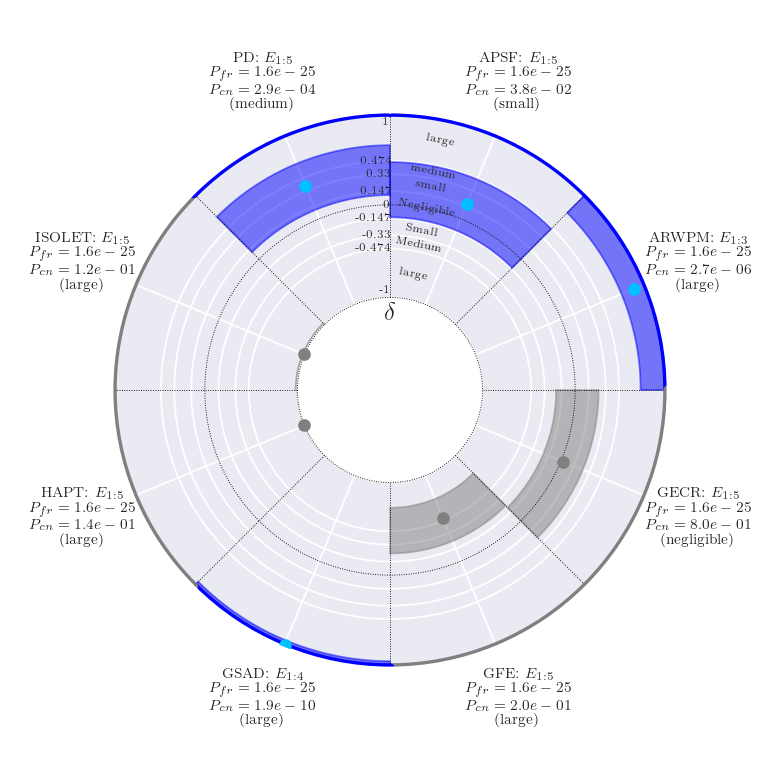}\label{fig:lrwf1}}%
\hfill
\subfloat[AUC]{\includegraphics[width=0.33\textwidth]{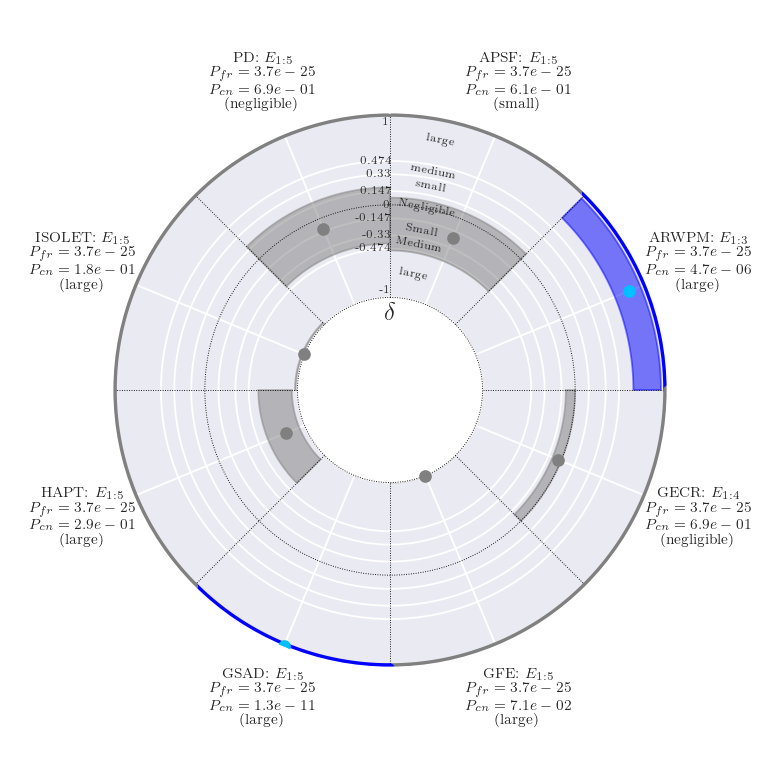}\label{fig:lrwauc}}%
\hfill
\subfloat[Log-Loss]{\includegraphics[width=0.33\textwidth]{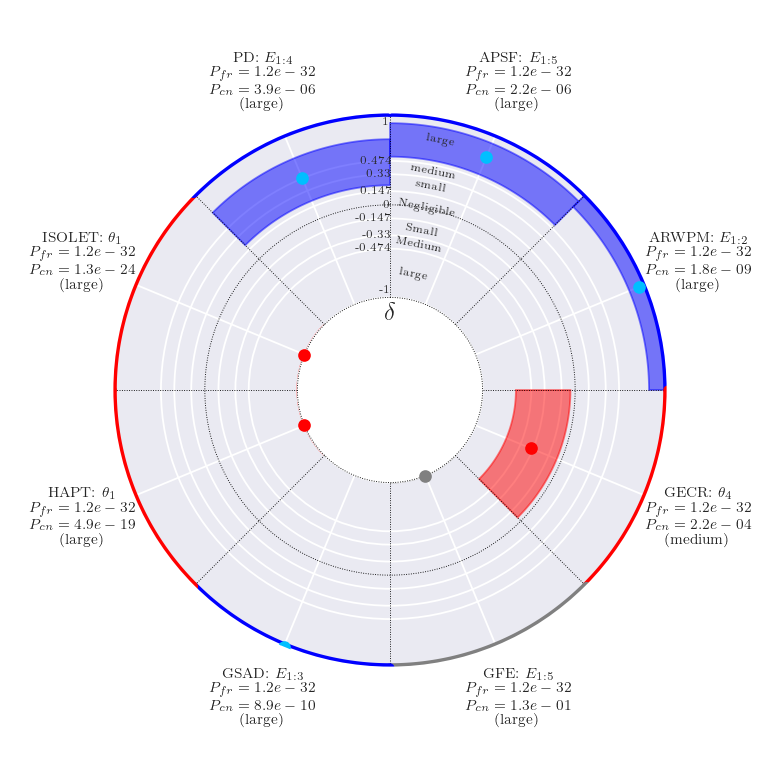}\label{fig:lrwloss}}%

\subfloat[MEC]{\includegraphics[width=0.33\textwidth]{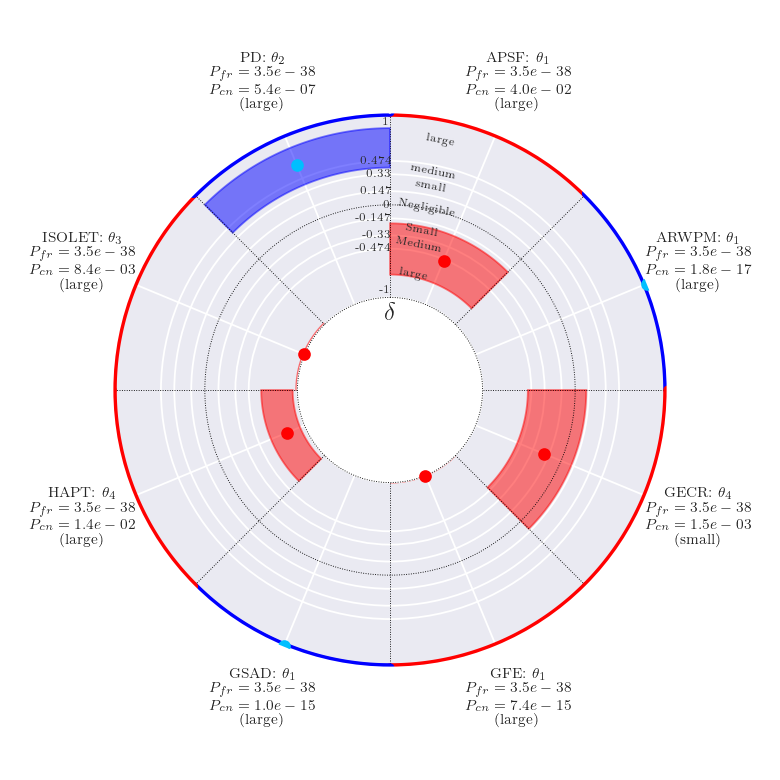}\label{fig:lrwgsad}}%
\hfill
\subfloat[MEW]{\includegraphics[width=0.33\textwidth]{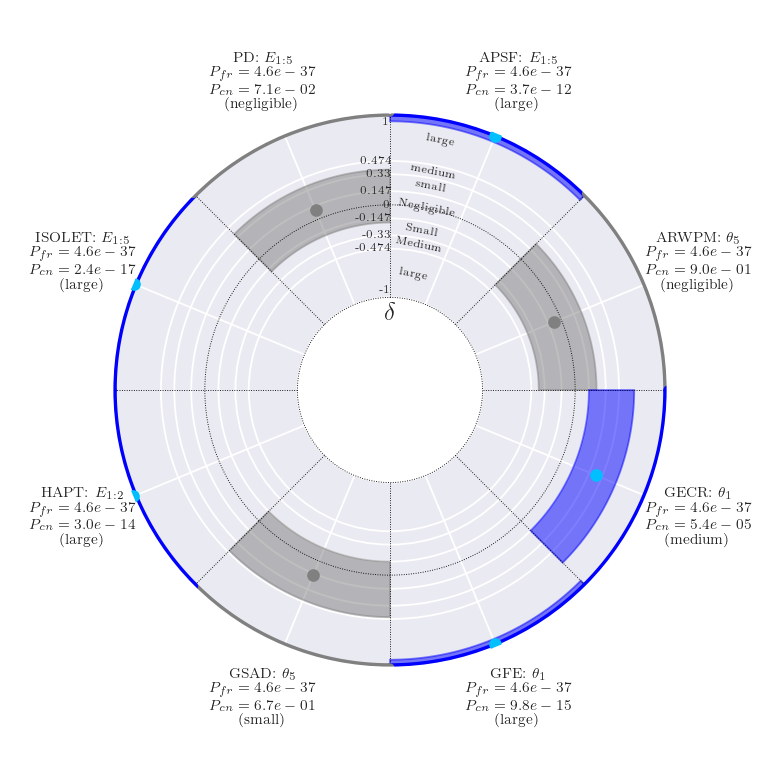}\label{fig:lrwhapt}}%
\hfill
\subfloat[Time]{\includegraphics[width=0.33\textwidth]{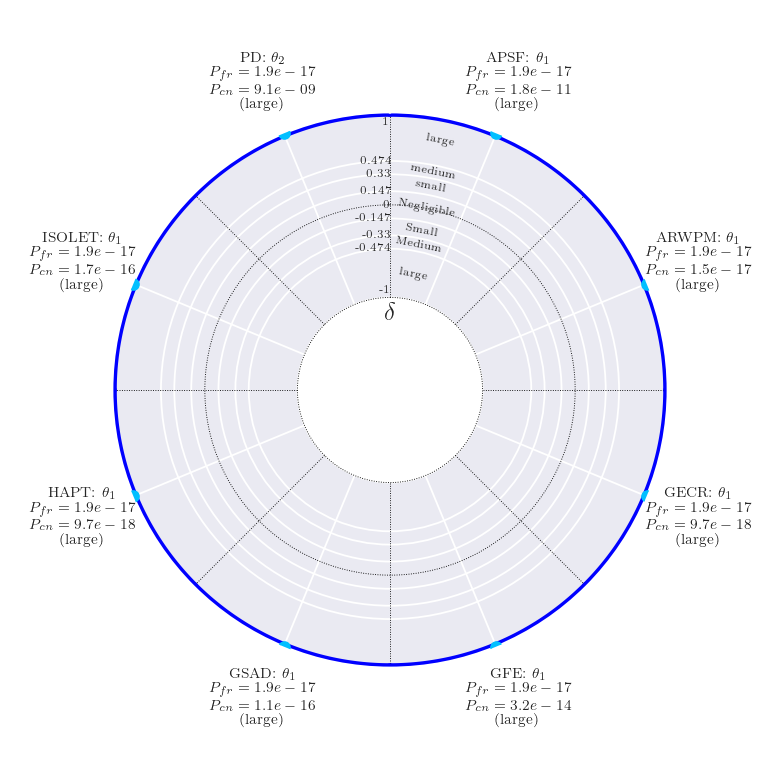}\label{fig:lrwisolet}}%

\caption{The figure illustrates the best performing LR models in comparison to the benchmark model, with the segments representing the 95\% confidence interval of Cliff's $\delta$ (center point). Grey segments indicate cases where either $P_{fr}>0.05$ or $P_{cn}>0.05$, suggesting no significant difference from the benchmark. Blue segments denote instances where the model outperforms the benchmark ($P_{fr}<0.05$, $P_{cn}<0.05$, and $\delta>0$), while red segments indicate that the benchmark model outperforms the corresponding LR model ($P_{fr}<0.05$, $P_{cn}<0.05$, and $\delta<0$).}

\label{fig:lrcliffwin}
\end{figure*}

As stated before, we optimized the hyper-parameters of XGBoost and LR models on the training data of constructed views using 5-fold cross-validation. 
The models were subsequently trained on training data using these optimized hyper-parameters and evaluated on the testing data. 
We then constructed four ensembles of views, denoted as $E_{1:2}$, $E_{1:3}$, $E_{1:4}$, and $E_{1:5}$, based on the predicted class label probabilities from the first two, three, four, and five views, respectively. 
These ensembles were created using a weighted average method, with the view weights determined by their normalized AUC values. 
As a control mechanism in our experiments, both XGBoost and LR models were fine-tuned and evaluated on the testing data using all features of each dataset, which we refer to as the benchmark model and denoted by ‘’All’’ in Figures and Tables. 
This experimental design allows us to compare the results of the views and their ensembles against the benchmark model, serving as an effective measure to evaluate the SPFP algorithm's effectiveness.
For comparative analysis, we employed five accuracy metrics: $F_1$ score, AUC, log-loss, MEC, and MEW, in addition to the running time of the models, which encompassed both hyper-parameter tuning and training phases.

We utilized the Friedman test at a $\alpha=0.05$ significance level to determine whether there were statistically significant differences among the results obtained for each metric.
 To address the family-wise error rate (type I error), the Bonferroni correction method was applied, adjusting the p-values obtained from the Friedman test ($P_{fr}$) across the five accuracy metrics.
Upon rejection of the Friedman test's null hypothesis (indicated by $P_{fr}<0.05$), we conducted the Conover multiple comparison test to identify significant differences between the models.
 We preferred the Conover test over the more commonly used Nemenyi test due to its greater sensitivity; the Nemenyi test, being more conservative, tends to be less sensitive, especially given our large model count (ten).
 This conservatism, while controlling for false discovery (type I error), increases the risk of not detecting a true effect (type II error).
For the p-values obtained from the Conover test ($P_{cn}$), we employed the Benjamini-Hochberg correction, which is designed to control the false discovery proportion in scenarios with numerous tests.
 A result was deemed significantly different if Conover's hypothesis was rejected at the $\alpha=0.05$ level.
 Additionally, we performed Cliff's $\delta$ analysis to ascertain the magnitude of differences between the benchmark model and the views and ensemble models.
 The lower and upper bounds of 95\% confidence intervals for Cliff's $\delta$ were determined using 10,000 bootstrap resampling.
 To maintain consistency in our analysis, negative values of log-loss, MEC, and running time were utilized in the Cliff's $\delta$ calculations. 
Therefore, a model with a positive effect size ($\delta>0$) was declared superior (win) following observations of $P_{fr}<0.05$ and $P_{cn}<0.05$, while a negative effect size ($\delta<0$) indicated inferior performance (loss). 
In cases where $P_{fr}>0.05$ or $P_{cn}>0.05$, no significant difference between the views/ensembles and benchmark models was inferred, leading to a classification of a tie between the models.

The mean and standard deviations of these metrics for all models in 30 XGBoost runs, along with both original and adjusted $P_{fr}$ values, are presented in Table~\ref{tab:xgbres}.
 The XGBoost results indicate that at least one model demonstrated significantly different performance in nearly all cases, with the exception of GECR's Log-Loss and MEW, and GFE's running time.
 A closer examination of the mean and standard deviations in Table~\ref{tab:xgbres}, complemented by a composite visualization of box, violin, and strip plots (referred to as raincloud plots) in Figures ~\ref{fig:xgb_F1},~\ref{fig:xgb_AUC},~\ref{fig:xgb_Loss},~\ref{fig:xgb_MEC},~\ref{fig:xgb_MEW}, and~\ref{fig:xgb_Time}, reveals interesting patterns.

Notably, the finely tuned XGBoost models, encompassing individual views, ensembles, and the benchmark set, generally exhibited high generalization accuracy.
 For instance, the mean $F_1$ scores for all models on APSF, ARWPM, GECR, and GSAD datasets exceeded 0.99, with maximum absolute differences between the models and the benchmark model being quite minimal (below 0.003). 
This trend of high performance was similarly observed for the AUC metric (i.e., AUC$>0.99$) across ARWPM, GECR, GSAD, HAPT, and ISOLET datasets, with the maximum absolute difference remaining below 0.004. 
Comparable findings were noted for Log-Loss, MEC, and MEW metrics.

Despite these minimal differences, statistical analyses employing Friedman (Table~\ref{tab:xgbres}) and Conover p-values (Figures~\ref{fig:xgbnem_F1} to~\ref{fig:xgbnem_MEW}) indicated significant statistical disparities between the models.
 Moreover, Cliff's $\delta$ analysis (Figures~\ref{fig:xgbcliff_F1} to~\ref{fig:xgbnem_MEW}) revealed large effect sizes ($| \delta | >0.474$) in most cases, as illustrated in Tables~\ref{tab:xgbf1} to~\ref{tab:xgbmew}.
It is noteworthy that Figures~\ref{fig:xgb_F1},~\ref{fig:xgb_AUC},~\ref{fig:xgb_Loss},~\ref{fig:xgb_MEC},~\ref{fig:xgb_MEW}, and~\ref{fig:xgb_Time} show raincloud plots, Figures~\ref{fig:xgbnem_F1},~\ref{fig:xgbnem_AUC},~\ref{fig:xgbnem_Loss},~\ref{fig:xgbnem_MEC},~\ref{fig:xgbnem_MEW}, and~\ref{fig:xgbnem_Time} show adjusted p-values of Conover's test, Figures~\ref{fig:xgbcliff_F1},~\ref{fig:xgbcliff_AUC},~\ref{fig:xgbcliff_Loss},~\ref{fig:xgbcliff_MEC},~\ref{fig:xgbcliff_MEW}, and~\ref{fig:xgbcliff_Time} show Cliff's $\delta$ effect size and their 95\% confidence intervals, while Tables~\ref{tab:xgbf1},~\ref{tab:xgbauc},~\ref{tab:xgbloss},~\ref{tab:xgbmec},~\ref{tab:xgbmew}, and~\ref{tab:xgbtime} depicts wins, ties, and losses of XGBoost view models and their ensembles against the benchmark (All) models, respectively.

These observations suggest that while the differences in models' performance are statistically significant and are less likely due to the chance, their practical impact might be limited in scenarios where baseline performance is already high.
 In other words, even minor variations in such high-performing regions can attain statistical significance, yet may not translate into substantial differences in real-world applications.
 Therefore, although the models trained on views $\theta_1$ to $\theta_5$ are statistically distinguishable from those trained on the benchmark set, the practical implications of opting for one model over another may be less pronounced.

Nevertheless, when considering the statistical significance and effect size measures, it's evident that while computational efficiency in fine-tuning and training individual models surpasses that of the original dataset, creating an ensemble of views, particularly $E_{1:5}$, not only enhances prediction accuracy (evidenced by 3 wins, 3 ties, and 2 losses in $F_1$ score), but also significantly improves AUC and reduces overall uncertainty in predictions (Log-Loss), as demonstrated by the results of 6-1-1 and 5-1-2 in wins, ties, and losses, respectively, as shown in Table~\ref{tabmain:xgbwtl}.

Figure~\ref{fig:xgbcliffwin} presents a comparative analysis of the best-performing XGBoost view/ensemble models against the benchmark XGBoost model, covering six different metrics.
 This figure not only visualizes Cliff's $\delta$ values along with their 95\% confidence interval bounds, but it also includes the adjusted $P_{fr}$ and $P_{cn}$ values, as well as the names of the corresponding datasets and models. 
In the case of running time, individual view models generally outperformed the benchmark model.
 For the other metrics, at least one view or ensemble outperformed the benchmark model (indicated by blue segments) or demonstrated comparable performance (grey segments).
 However, there are notable exceptions: all models underperformed compared to the benchmark in the $F_1$ score for the APSF dataset, MEC for the PD dataset, and MEW for the ISOLET dataset, as highlighted by red segments in the figure. 

Furthermore, the ensemble comprising all five views, denoted as $E_{1:5}$, is frequently recognized as the top performer across several metrics. 
Yet, individual views also stood out in certain cases.
 Notably, they were the best performers in the $F_1$ score for the GFE dataset, Log-Loss for both ARWPM and GECR datasets, MEC across all datasets, and MEW for the GECR dataset. 

This observation suggests that the efficacy of view ensembles could potentially be enhanced beyond the weighted average method employed in our study, possibly by adopting a more sophisticated ensemble strategy.

The Cliff's $\delta$ values observed in our study suggest that the effect size is considerable in most instances, whether in scenarios of model superiority or inferiority.
 This finding underscores the effectiveness of the SPFP algorithm, which presents a mathematically robust and systematic approach to constructing multiple views and implementing ensemble learning. 
It particularly enhances the performance of complex models like XGBoost, known for its stochastic feature selection and boosting-based ensemble learning methods.

However, it is imperative to consider the practical significance of these findings within the context of specific applications.
 In fields such as medical diagnosis, fraud detection, and fault detection in engineering, where even minor improvements in precision and recall can be critical, the additional computational load incurred by the SPFP algorithm's ensemble learning approach might not translate into substantial gains in accuracy. 
The paramount concern in these applications is not just the model's uncertainty but its precision and recall capabilities.

Conversely, in high-stakes domains like high-frequency trading and climate-change risk assessment, where rapid changes in the market or environment are frequent yet the underlying structure of feature interactions remains relatively stable over time, the advantages of the SPFP algorithm become more pronounced. 
In these scenarios, the ability to rapidly update or retrain models and the probabilistic nature of their predictions are more critical than mere accuracy metrics.
 The SPFP algorithm's approach to ensemble learning can facilitate more robust risk assessment and management in such dynamic environments, underscoring its significant value.

Therefore, while the SPFP algorithm demonstrably enhances model performance, its utility should be evaluated against the backdrop of the specific requirements and constraints of each application domain. 
The choice of employing this algorithm should be guided by a balance between computational efficiency and the need for precision, recall, and adaptability in decision-making processes.

In addition to our findings with XGBoost, we further complemented our investigation by employing Logistic Regression (LR), a simpler classification model that, unlike XGBoost, does not inherently incorporate feature selection and ensemble learning algorithms.
 The mean and standard deviation for six metrics, obtained over 30 LR runs, are detailed in Table~\ref{tab:lrres}, alongside the Friedman test results and their adjusted p-values ($P_{fr}$).

Significantly, all obtained $P_{fr}$ values fell below the $\alpha=0.05$ significance threshold. 
This indicates that, within the LR model framework, at least one model exhibited performance significantly different from the others.
 Notably, in contrast to the XGBoost results, the LR model showcased more substantial absolute differences between the view/ensemble models and the benchmark models.
 However, exceptions were observed for AUC in the GECR, GSAD, HAPT, and ISOLET datasets, where AUC values exceeded 0.99, and for Log-Loss in these same datasets, where the values were predominantly below 0.3 bit.
Similarly, for the MEC metric in the APDF dataset, the MEC values remained below 0.1 bit.

Comprehensive visualizations of these results are provided through various figures: raincloud plots in Figures ~\ref{fig:lr_F1},~\ref{fig:lr_AUC},~\ref{fig:lr_Loss},~\ref{fig:lr_MEC},~\ref{fig:lr_MEW}, and~\ref{fig:lr_Time} illustrate the distribution of these metrics, while Figures ~\ref{fig:lrnem_F1},~\ref{fig:lrnem_AUC},~\ref{fig:lrnem_Loss},~\ref{fig:lrnem_MEC},~\ref{fig:lrnem_MEW}, and~\ref{fig:lrnem_Time} display the adjusted p-values of Conover's test.
 Additionally, Figures ~\ref{fig:lrcliff_F1},~\ref{fig:lrcliff_AUC},~\ref{fig:lrcliff_Loss},~\ref{fig:lrcliff_MEC},~\ref{fig:lrcliff_MEW}, and~\ref{fig:lrcliff_Time} present Cliff's $\delta$ effect sizes along with their 95\% confidence intervals.
 These visual aids are further complemented by Tables ~\ref{tab:lrf1},~\ref{tab:lrauc},~\ref{tab:lrloss},~\ref{tab:lrmec},~\ref{tab:lrmew}, and~\ref{tab:lrtime}, which enumerate the wins, ties, and losses of the LR view models and their ensembles against the benchmark models.

An analysis comparing the outcomes of the Friedman and Conover tests, as well as the effect sizes determined by Cliff's $\delta$, for 30 runs each of XGBoost and LR models, reveals that the SPFP algorithm's ensemble learning approach has distinct impacts on these models. 
XGBoost, being a more complex model, and LR, a simpler one, respond differently to the algorithm.

Specifically, as detailed in Table~\ref{tabmain:lrwtl}, the ensemble comprising all LR view models ($E_{1:5}$) demonstrated superior performance compared to the benchmark model in several aspects.
 For the $F_1$ score, AUC, Log-Loss, and MEW metrics, $E_{1:5}$ outperformed the benchmark on 4, 2, 4, and 4 datasets, respectively.
 Additionally, it showed comparable performance to the benchmark model on 4, 6, 1, and 2 datasets, respectively, for the same metrics.
 Notably, $E_{1:5}$ improved the $F_1$ score and AUC across all datasets but showed less favorable outcomes in Log-Loss and MEW on 3 and 2 datasets, respectively.

Another significant observation is the running times of the models, which were considerably shorter than those of the benchmark models across all datasets.
 However, the performance in the MEC metric was mixed, with the $E_{1:5}$ ensemble winning, tying, and losing against the benchmark model on 1, 2, and 5 datasets, respectively. 
This suggests that, despite its accuracy, the ensemble provides less confident correct predictions compared to the benchmark.

The statistical significance of the difference between the $E_{1:5}$ ensemble and the benchmark model was consistent across all datasets, as evidenced by large Cliff's $\delta$ effect sizes.
 Yet, the actual differences in the means of the MEC metrics were relatively minor, as previously discussed and illustrated in Figures~\ref{fig:lr_MEC} and~\ref{fig:lrcliff_MEC}.
 This subtlety implies that, while $E_{1:5}$ is statistically more accurate, its confidence in correct predictions is somewhat diminished compared to the benchmark model.

Figure~\ref{fig:lrcliffwin} offers a comprehensive view of the best-performing LR view/ensemble models in comparison with the benchmark LR model across six different metrics.
 This figure not only illustrates Cliff's $\delta$ values, along with their 95\% confidence interval bounds, but also includes the adjusted $P_{fr}$, and $P_{cn}$ values, as well as the names of the corresponding datasets and models.

Similar to the observations made with XGBoost, the ensemble model $E_{1:5}$ is frequently highlighted as the top performer among the LR models.
 For metrics such as the $F_1$ score, AUC, MEW, and running time, there is at least one model in each case that either outperforms (indicated by blue segments) or matches (grey segments) the benchmark model across all datasets.

However, in the contexts of Log-Loss and MEC metrics, particularly where the best-performing model is significantly outdone by the benchmark model (as shown by red segments), it is often an individual view model that stands out.
 This pattern suggests that the overall performance of ensemble models, while robust, still holds potential for further enhancement through more sophisticated ensemble learning techniques.

It's important to note, especially for the datasets where the view/ensemble models underperformed compared to the benchmark models, that the absolute differences in the means of the metrics were relatively small.
 This observation underscores the nuances in performance that may not be immediately apparent from statistical significance alone, but are crucial in understanding the real-world applicability and efficiency of these models.

The findings from our study lead to two pivotal conclusions regarding the efficacy of the SPFP algorithm's ensemble learning approach.
 Firstly, this approach has the capability to significantly enhance the accuracy of models, particularly in scenarios where achieving optimal generalization accuracy is challenging through standard model fine-tuning techniques.
 Secondly, it notably increases the confidence in correct predictions while simultaneously reducing the likelihood of confident incorrect predictions, especially in situations where near-perfect generalization accuracy is attainable with conventional fine-tuning methods.

\section{Conclusion}
\label{sec:conclusion}
This study introduces the SPFP algorithm, a novel approach to constructing artificial views from single-source data for Multi-view Ensemble Learning. Unlike previous algorithms, the SPFP algorithm eschews the random trial and error method for enhancing the accuracy or uncertainty of traditional machine-learning approaches. 
Instead, it employs a robust mathematical concept to generate diverse and complementary views from a single-source dataset.

Our comprehensive analysis of the SPFP algorithm focuses on its effectiveness in improving the performance of machine learning models.
This includes an in-depth examination of its impact on a complex model with built-in feature selection and ensemble-learning algorithms (XGBoost) and a simpler model (Logistic Regression).
The study encompasses a wide array of benchmark datasets with varying characteristics, from high-dimensional with limited instances to high-instance data with lower dimensions, across different real-world domains.
The analysis considers multiple metrics for accuracy, uncertainty, and computational efficiency.
To interpret the experimental results, we conducted non-parametric multiple related sample tests (Friedman test), multiple comparison tests (Conover test) with type I and II error controls, and effect size measurement (Cliff's $\delta$) analysis.

Our findings indicate that the SPFP algorithm significantly enhances the predictive accuracy and robustness of ensemble models.
By partitioning features into distinct views, it not only maintains the semantic integrity of the original dataset but also uncovers varied patterns within the data.
This is particularly evident in the ensemble models' enhanced performance in uncertainty metrics like MEC and MEW, while maintaining accuracy metrics in scenarios where high generalization performance is achievable through conventional fine-tuning of complex models like XGBoost.
Similarly, it preserves uncertainty metrics while improving accuracy metrics like the $F_1$ score and AUC in scenarios where high generalization performance is less attainable with simpler models like Logistic Regression.
This underscores the advantage of integrating multiple perspectives in model training, especially in applications where enhanced accuracy and reduced uncertainty are crucial.

Additionally, the SPFP algorithm effectively balances dimensionality reduction with information retention, a key factor in applications where computational efficiency is as critical as model accuracy.
The individual views, competitive with the complete dataset in isolation, contribute to a more effective ensemble model, demonstrating the SPFP algorithm's efficacy in creating meaningful and efficient feature subsets.

Future research may aim to refine the algorithm further, explore its applications in diverse contexts, and extend its principles to other data forms beyond this study's scope.
While the SPFP algorithm is currently more suited to supervised and semi-supervised learning (subject to data richness limitations), subsequent studies could broaden its application to unsupervised learning by enabling mathematical interactions among distinct views.

\nomenclature{$X$ and $Y$}{Dummy random variables.}
\nomenclature{$x_i$ and $y_j$}{The $i$th and $j$th observation of random variables, $X$ and $Y$.}
\nomenclature{$F$}{The entire feature set.}
\nomenclature{$S$}{The selected feature subset.}
\nomenclature{$S'$}{The disjoint and complementary subset for $S$, i.e., $S \cap S' = \phi$ and $S \cup S' = F$.}
\nomenclature{$\theta_g$}{The $g$th constructed artificial view.}
\nomenclature{$\Theta$}{The set of all constructed artificial view.}
\nomenclature{$U$}{The feature search space.}
\nomenclature{$U_t$}{A temporary feature search space.}
\nomenclature{$f_s$}{A feature within the selected feature subset, $f_s \in S$.}
\nomenclature{$f_{s_i}$}{The $i$th observation of the selected feature, $f_s$.}
\nomenclature{$f_c$}{A candidate feature within the entire feature set, $f_c \in F$, and $f_c \notin S$.}
\nomenclature{$f_{c_j}$}{The $j$th observation of the candidate feature, $f_c$.}
\nomenclature{$p(x_i)$ and $p(y_j)$}{The marginal probabilities of $X$'s $i$th, and $Y$'s $j$th observations.}
\nomenclature{$p(x_i,y_j)$}{The joint probability of $X$'s $i$th and $Y$'s $j$th observations.}
\nomenclature{$p(x_i|y_j)$}{The conditional probability of $X$'s $i$th observation given $Y$'s $j$th observation.}
\nomenclature{$H(X)$ and $H(Y)$}{The entropy of random variable, $X$, and $Y$.}
\nomenclature{$H(X,Y)$}{The joint entropy of $X$ and $Y$.}
\nomenclature{$H(S)$ and $H(F)$}{The joint entropy of selected feature subset, $S$, and the entire feature set, $F$, respectively.}
\nomenclature{$H(S,Y)$ and $H(F,Y)$}{The joint entropy of selected feature subset, $S$ and $Y$, and the entire feature set, $F$ and $Y$, respectively.}
\nomenclature{$H(X|Y)$}{The conditional entropy of $X$ given $Y$.}
\nomenclature{$I(X;Y)$ and $I(f_s;f_c)$}{The mutual information between the random variables, $X$ and $Y$, and a selected feature, $f_s$, and a candidate feature, $f_c$.}
\nomenclature{$I(S;Y)$ and $I(F;Y)$}{The mutual information between the selected feature subset, $S$ and $Y$, and the entire feature set $F$ and $Y$.}
\nomenclature{$I(f_s;f_c;Y)$}{The interaction gain between the selected feature, $f_s$, and the candidate feature, $f_c$, given $Y$.}
\nomenclature{$I(f_s;f_c|Y)$}{The conditional mutual information between the selected feature, $f_s$, and the candidate feature, $f_c$, given $Y$.}
\nomenclature{$J(f_c)$}{The conditional likelihood framework's objective of a candidate feature, $f_c$.}
\nomenclature{$J_{SPFP}(f_c)$}{The SPFP algorithm's objective of a candidate feature, $f_c$.}
\nomenclature{$\alpha$ and $\beta$}{The weights of the redundancy and complementarity criteria within the conditional likelihood framework's objective.}
\nomenclature{$R(f_c,Y)$}{The Pearson correlation between the candidate feature, $f_c$ and the target, $Y$.}
\nomenclature{$N_F$}{The minimum number of features' threshold within a view.}
\nomenclature{$N_\theta$}{The number of views threshold.}
\nomenclature{$r$}{The ratio of the features in view $\theta_g$ to be randomly eliminated from the feature search space of $\theta_{g+1}$.}
\nomenclature{$C_1$}{The first stopping criterion indicates that the number of features within a view should be greater than $N_F$.}
\nomenclature{$C_2$}{The second stopping criterion indicates that the information content of the selected feature subset should be equal to that of the entire feature set.}
\nomenclature{$C_3$}{The third stopping criterion indicates that the information content of the selected feature subset and $Y$ should be equal to that of the entire feature set and $Y$.}
\nomenclature{$E_{1:2}$, $E_{1:3}$, $E_{1:4}$, and $E_{1:5}$}{The ensemble of the models trained on the first two, three, four, and five views used in the experiments.}

\appendix
\section*{Appendix A}
\label{appendix}
\printnomenclature

\bibliography{Bibiliography}

\clearpage
\section*{Supplementary Material}
\input{SupplementaryMaterial}
\end{document}

%% file: SupplementaryMaterial.tex



	
	\renewcommand*{\thesection}{S.\Roman{section}}
	\renewcommand{\thetheorem}{S.\arabic{theorem}}
	\renewcommand*{\thetable}{S.\Roman{table}}
	\renewcommand*{\thefigure}{S.\arabic{figure}}
	\renewcommand*{\theequation}{S.\arabic{equation}}

	\title{Supplementary Document for ``Semantic-Preserving Feature Partitioning for Multi-View Ensemble Learning''}

	\author{Mohammad~Sadegh~Khorshidi\orcidlink{0000-0001-6556-2926}%
		, Navid~Yazdanjue\orcidlink{0000-0001-9670-8422}%
		, Hassan~Gharoun\orcidlink{0000-0001-8298-7512}%
		, Danial~Yazdani\orcidlink{0000-0002-7799-5013},~\IEEEmembership{Member,~IEEE}%
		, Mohammad~Reza~Nikoo\orcidlink{0000-0002-3740-4389}%
		, Fang~Chen\orcidlink{0000-0003-4971-8729}%
		, and Amir~H.~Gandomi\orcidlink{0000-0002-2798-0104},~\IEEEmembership{Senior Member,~IEEE}
		\thanks{Mohammad Sadegh Khorshidi, Navid Yazdanjue, Hassan Gharoun, Danial Yazdani, Fang Chen, and Amir H. Gandomi are with the Faculty of Engineering \& Information Technology, University of Technology Sydney, Ultimo 2007, Australia. e-mails: ms.khorshidi@student.uts.edu.au, navid.yazdanjue@gmail.com, hassan.gharoun@student.uts.edu.au, danial.yazdani@gmail.com, fang.chen@uts.edu.au, gandomi@uts.edu.au}
		\thanks{Mohammad Reza Nikoo is with the Department of Civil and Architectural Engineering, Sultan Qaboos University, Muscat, Oman. (e-mail: m.reza@squ.edu.om}%
		\thanks{Amir H. Gandomi is also with the University Research and Innovation Center (EKIK), Obuda University, Budapest 1034, Hungary.}%
		\thanks{This work was supported by the Australian Government through the Australian Research Council under Project DE210101808.}
		\thanks{Corresponding author: Amir H. Gandomi}}
	
	\maketitle
	
	\listoffigures
	\listoftables
	\begin{sidewaystable*}
		\centering
		\caption[The obtained metrics' values and Friedman test results for 30 XGBoost runs.]{The obtained $F_1$ score, AUC, Log-Loss, MEC, MEW, and running time (sec), for the testing data in 30 XGBoost runs. The obtained p-values of Friedman's p-values and adjusted p-values using Bonferroni method. The bold values indicate at least one sample is significantly different from others.}
		\label{tab:xgbres}
		\resizebox{\linewidth}{!}{%
			\begin{tabular}{c|c|cccccccccccc}
				\hline
				\multicolumn{14}{c}{XGBoost} \\
				\hline
				Dataset & Metric & $\theta_1$ & $\theta_2$ & $\theta_3$ & $\theta_4$ & $\theta_5$ & $E_{1:2}$ & $E_{1:3}$ & $E_{1:4}$ & $E_{1:5}$ & All & Friedman's $P-value$ & Adjusted $P-value$ \\
				\hline
				\multirow{6}{*}{APSF} & $F_1$ & {\cellcolor[rgb]{0.753,0.753,0.753}}$0.994 \pm 0.001$ & {\cellcolor[rgb]{0.753,0.753,0.753}}$0.994 \pm 0.001$ & {\cellcolor[rgb]{0.753,0.753,0.753}}$0.994 \pm 0.001$ & {\cellcolor[rgb]{0.753,0.753,0.753}}$0.994 \pm 0.001$ & {\cellcolor[rgb]{0.753,0.753,0.753}}$0.994 \pm 0.001$ & {\cellcolor[rgb]{0.753,0.753,0.753}}$0.994 \pm 0.001$ & {\cellcolor[rgb]{0.753,0.753,0.753}}$0.994 \pm 0.001$ & {\cellcolor[rgb]{0.753,0.753,0.753}}$0.994 \pm 0.001$ & {\cellcolor[rgb]{0.753,0.753,0.753}}$0.995 \pm 0.001$ & {\cellcolor[rgb]{0.753,0.753,0.753}}$0.996 \pm 0.0$ & {\cellcolor[rgb]{0.753,0.753,0.753}}$8.59e-26$ & {\cellcolor[rgb]{0.753,0.753,0.753}}$\mathbf{4.29e-25}$ \\
				& AUC & $0.976 \pm 0.006$ & $0.976 \pm 0.007$ & $0.978 \pm 0.006$ & $0.978 \pm 0.007$ & $0.976 \pm 0.005$ & $0.978 \pm 0.006$ & $0.98 \pm 0.006$ & $0.981 \pm 0.006$ & $0.982 \pm 0.006$ & $0.984 \pm 0.006$ & $1.78e-31$ & $\mathbf{8.88e-31}$ \\
				& Log-Loss & {\cellcolor[rgb]{0.753,0.753,0.753}}$0.025 \pm 0.004$ & {\cellcolor[rgb]{0.753,0.753,0.753}}$0.024 \pm 0.004$ & {\cellcolor[rgb]{0.753,0.753,0.753}}$0.022 \pm 0.003$ & {\cellcolor[rgb]{0.753,0.753,0.753}}$0.023 \pm 0.004$ & {\cellcolor[rgb]{0.753,0.753,0.753}}$0.023 \pm 0.004$ & {\cellcolor[rgb]{0.753,0.753,0.753}}$0.021 \pm 0.003$ & {\cellcolor[rgb]{0.753,0.753,0.753}}$0.018 \pm 0.002$ & {\cellcolor[rgb]{0.753,0.753,0.753}}$0.018 \pm 0.002$ & {\cellcolor[rgb]{0.753,0.753,0.753}}$0.017 \pm 0.002$ & {\cellcolor[rgb]{0.753,0.753,0.753}}$0.015 \pm 0.002$ & {\cellcolor[rgb]{0.753,0.753,0.753}}$3.86e-39$ & {\cellcolor[rgb]{0.753,0.753,0.753}}$\mathbf{1.93e-38}$ \\
				& MEC & $0.009 \pm 0.004$ & $0.012 \pm 0.017$ & $0.011 \pm 0.007$ & $0.016 \pm 0.024$ & $0.009 \pm 0.004$ & $0.012 \pm 0.01$ & $0.012 \pm 0.004$ & $0.017 \pm 0.013$ & $0.017 \pm 0.012$ & $0.008 \pm 0.004$ & $2.74e-07$ & $\mathbf{1.37e-06}$ \\
				& MEW & {\cellcolor[rgb]{0.753,0.753,0.753}}$0.443 \pm 0.087$ & {\cellcolor[rgb]{0.753,0.753,0.753}}$0.436 \pm 0.092$ & {\cellcolor[rgb]{0.753,0.753,0.753}}$0.454 \pm 0.082$ & {\cellcolor[rgb]{0.753,0.753,0.753}}$0.453 \pm 0.096$ & {\cellcolor[rgb]{0.753,0.753,0.753}}$0.444 \pm 0.075$ & {\cellcolor[rgb]{0.753,0.753,0.753}}$0.521 \pm 0.064$ & {\cellcolor[rgb]{0.753,0.753,0.753}}$0.569 \pm 0.047$ & {\cellcolor[rgb]{0.753,0.753,0.753}}$0.594 \pm 0.041$ & {\cellcolor[rgb]{0.753,0.753,0.753}}$0.598 \pm 0.04$ & {\cellcolor[rgb]{0.753,0.753,0.753}}$0.467 \pm 0.065$ & {\cellcolor[rgb]{0.753,0.753,0.753}}$1.52e-29$ & {\cellcolor[rgb]{0.753,0.753,0.753}}$\mathbf{7.60e-29}$ \\
				& Time (sec) & $326.746 \pm 157.492$ & $340.675 \pm 143.944$ & $409.181 \pm 145.079$ & $389.499 \pm 81.0$ & $427.817 \pm 139.44$ & -- & -- & -- & -- & $475.405 \pm 210.27$ & $\mathbf{0.03}$ & -- \\
				\hline
				\multirow{6}{*}{ARWPM} & $F_1$ & {\cellcolor[rgb]{0.753,0.753,0.753}}$0.993 \pm 0.003$ & {\cellcolor[rgb]{0.753,0.753,0.753}}$0.992 \pm 0.004$ & {\cellcolor[rgb]{0.753,0.753,0.753}}$0.992 \pm 0.005$ & {\cellcolor[rgb]{0.753,0.753,0.753}}$0.991 \pm 0.005$ & {\cellcolor[rgb]{0.753,0.753,0.753}}$0.989 \pm 0.006$ & {\cellcolor[rgb]{0.753,0.753,0.753}}$0.994 \pm 0.003$ & {\cellcolor[rgb]{0.753,0.753,0.753}}$0.996 \pm 0.002$ & {\cellcolor[rgb]{0.753,0.753,0.753}}$0.996 \pm 0.002$ & {\cellcolor[rgb]{0.753,0.753,0.753}}$0.997 \pm 0.002$ & {\cellcolor[rgb]{0.753,0.753,0.753}}$0.994 \pm 0.003$ & {\cellcolor[rgb]{0.753,0.753,0.753}}$5.50e-23$ & {\cellcolor[rgb]{0.753,0.753,0.753}}$\mathbf{2.75e-22}$ \\
				& AUC & $1.0 \pm 0.0$ & $1.0 \pm 0.0$ & $1.0 \pm 0.001$ & $1.0 \pm 0.0$ & $1.0 \pm 0.0$ & $1.0 \pm 0.0$ & $1.0 \pm 0.0$ & $1.0 \pm 0.0$ & $1.0 \pm 0.0$ & $1.0 \pm 0.0$ & $7.78e-25$ & $\mathbf{3.89e-24}$ \\
				& Log-Loss & {\cellcolor[rgb]{0.753,0.753,0.753}}$0.026 \pm 0.007$ & {\cellcolor[rgb]{0.753,0.753,0.753}}$0.034 \pm 0.013$ & {\cellcolor[rgb]{0.753,0.753,0.753}}$0.058 \pm 0.049$ & {\cellcolor[rgb]{0.753,0.753,0.753}}$0.04 \pm 0.016$ & {\cellcolor[rgb]{0.753,0.753,0.753}}$0.052 \pm 0.019$ & {\cellcolor[rgb]{0.753,0.753,0.753}}$0.028 \pm 0.009$ & {\cellcolor[rgb]{0.753,0.753,0.753}}$0.035 \pm 0.015$ & {\cellcolor[rgb]{0.753,0.753,0.753}}$0.033 \pm 0.014$ & {\cellcolor[rgb]{0.753,0.753,0.753}}$0.035 \pm 0.011$ & {\cellcolor[rgb]{0.753,0.753,0.753}}$0.029 \pm 0.014$ & {\cellcolor[rgb]{0.753,0.753,0.753}}$1.14e-15$ & {\cellcolor[rgb]{0.753,0.753,0.753}}$\mathbf{5.68e-15}$ \\
				& MEC & $0.036 \pm 0.021$ & $0.045 \pm 0.039$ & $0.133 \pm 0.192$ & $0.063 \pm 0.057$ & $0.095 \pm 0.117$ & $0.045 \pm 0.036$ & $0.084 \pm 0.078$ & $0.076 \pm 0.072$ & $0.088 \pm 0.089$ & $0.057 \pm 0.048$ & $5.41e-09$ & $\mathbf{2.71e-08}$ \\
				& MEW & {\cellcolor[rgb]{0.753,0.753,0.753}}$0.892 \pm 0.149$ & {\cellcolor[rgb]{0.753,0.753,0.753}}$0.934 \pm 0.138$ & {\cellcolor[rgb]{0.753,0.753,0.753}}$1.126 \pm 0.219$ & {\cellcolor[rgb]{0.753,0.753,0.753}}$0.996 \pm 0.164$ & {\cellcolor[rgb]{0.753,0.753,0.753}}$1.031 \pm 0.126$ & {\cellcolor[rgb]{0.753,0.753,0.753}}$1.02 \pm 0.134$ & {\cellcolor[rgb]{0.753,0.753,0.753}}$1.074 \pm 0.268$ & {\cellcolor[rgb]{0.753,0.753,0.753}}$1.154 \pm 0.154$ & {\cellcolor[rgb]{0.753,0.753,0.753}}$1.189 \pm 0.166$ & {\cellcolor[rgb]{0.753,0.753,0.753}}$1.013 \pm 0.123$ & {\cellcolor[rgb]{0.753,0.753,0.753}}$4.31e-14$ & {\cellcolor[rgb]{0.753,0.753,0.753}}$\mathbf{2.15e-13}$ \\
				& Time (sec) & $452.462 \pm 155.188$ & $442.675 \pm 178.672$ & $378.05 \pm 142.059$ & $457.639 \pm 133.126$ & $369.314 \pm 134.246$ & -- & -- & -- & -- & $808.835 \pm 531.551$ & $\mathbf{8.48e-04}$ & -- \\
				\hline
				\multirow{6}{*}{GECR} & $F_1$ & {\cellcolor[rgb]{0.753,0.753,0.753}}$0.993 \pm 0.007$ & {\cellcolor[rgb]{0.753,0.753,0.753}}$0.993 \pm 0.006$ & {\cellcolor[rgb]{0.753,0.753,0.753}}$0.994 \pm 0.006$ & {\cellcolor[rgb]{0.753,0.753,0.753}}$0.993 \pm 0.006$ & {\cellcolor[rgb]{0.753,0.753,0.753}}$0.994 \pm 0.005$ & {\cellcolor[rgb]{0.753,0.753,0.753}}$0.993 \pm 0.007$ & {\cellcolor[rgb]{0.753,0.753,0.753}}$0.996 \pm 0.004$ & {\cellcolor[rgb]{0.753,0.753,0.753}}$0.996 \pm 0.004$ & {\cellcolor[rgb]{0.753,0.753,0.753}}$0.996 \pm 0.005$ & {\cellcolor[rgb]{0.753,0.753,0.753}}$0.996 \pm 0.004$ & {\cellcolor[rgb]{0.753,0.753,0.753}}$7.23e-04$ & {\cellcolor[rgb]{0.753,0.753,0.753}}$\mathbf{3.62e-03}$ \\
				& AUC & $1.0 \pm 0.0$ & $1.0 \pm 0.0$ & $1.0 \pm 0.0$ & $1.0 \pm 0.0$ & $1.0 \pm 0.0$ & $1.0 \pm 0.0$ & $1.0 \pm 0.0$ & $1.0 \pm 0.0$ & $1.0 \pm 0.0$ & $1.0 \pm 0.0$ & $9.19e-07$ & $\mathbf{4.60e-06}$ \\
				& Log-Loss & {\cellcolor[rgb]{0.753,0.753,0.753}}$0.095 \pm 0.071$ & {\cellcolor[rgb]{0.753,0.753,0.753}}$0.214 \pm 0.314$ & {\cellcolor[rgb]{0.753,0.753,0.753}}$0.084 \pm 0.069$ & {\cellcolor[rgb]{0.753,0.753,0.753}}$0.09 \pm 0.069$ & {\cellcolor[rgb]{0.753,0.753,0.753}}$0.103 \pm 0.096$ & {\cellcolor[rgb]{0.753,0.753,0.753}}$0.136 \pm 0.151$ & {\cellcolor[rgb]{0.753,0.753,0.753}}$0.118 \pm 0.102$ & {\cellcolor[rgb]{0.753,0.753,0.753}}$0.115 \pm 0.075$ & {\cellcolor[rgb]{0.753,0.753,0.753}}$0.109 \pm 0.072$ & {\cellcolor[rgb]{0.753,0.753,0.753}}$0.165 \pm 0.149$ & {\cellcolor[rgb]{0.753,0.753,0.753}}$0.41$ & {\cellcolor[rgb]{0.753,0.753,0.753}}$1.00$ \\
				& MEC & $0.388 \pm 0.284$ & $0.612 \pm 0.73$ & $0.331 \pm 0.279$ & $0.346 \pm 0.272$ & $0.393 \pm 0.348$ & $0.535 \pm 0.489$ & $0.52 \pm 0.385$ & $0.523 \pm 0.303$ & $0.497 \pm 0.297$ & $0.73 \pm 0.43$ & $3.14e-03$ & $\mathbf{0.02}$ \\
				& MEW & {\cellcolor[rgb]{0.753,0.753,0.753}}$1.52 \pm 0.801$ & {\cellcolor[rgb]{0.753,0.753,0.753}}$1.656 \pm 0.802$ & {\cellcolor[rgb]{0.753,0.753,0.753}}$1.285 \pm 0.928$ & {\cellcolor[rgb]{0.753,0.753,0.753}}$1.388 \pm 0.832$ & {\cellcolor[rgb]{0.753,0.753,0.753}}$1.229 \pm 0.906$ & {\cellcolor[rgb]{0.753,0.753,0.753}}$1.525 \pm 0.874$ & {\cellcolor[rgb]{0.753,0.753,0.753}}$1.203 \pm 0.996$ & {\cellcolor[rgb]{0.753,0.753,0.753}}$1.274 \pm 0.982$ & {\cellcolor[rgb]{0.753,0.753,0.753}}$1.244 \pm 0.961$ & {\cellcolor[rgb]{0.753,0.753,0.753}}$1.299 \pm 0.955$ & {\cellcolor[rgb]{0.753,0.753,0.753}}$0.21$ & {\cellcolor[rgb]{0.753,0.753,0.753}}$1.00$ \\
				& Time (sec) & $55.301 \pm 32.75$ & $47.969 \pm 34.83$ & $57.918 \pm 42.172$ & $42.47 \pm 26.359$ & $44.776 \pm 28.427$ & -- & -- & -- & -- & $147.825 \pm 78.647$ & $\mathbf{1.93e-09}$ & -- \\
				\hline
				\multirow{6}{*}{GFE} & $F_1$ & {\cellcolor[rgb]{0.753,0.753,0.753}}$0.934 \pm 0.006$ & {\cellcolor[rgb]{0.753,0.753,0.753}}$0.939 \pm 0.005$ & {\cellcolor[rgb]{0.753,0.753,0.753}}$0.944 \pm 0.004$ & {\cellcolor[rgb]{0.753,0.753,0.753}}$0.947 \pm 0.003$ & {\cellcolor[rgb]{0.753,0.753,0.753}}$0.946 \pm 0.002$ & {\cellcolor[rgb]{0.753,0.753,0.753}}$0.937 \pm 0.005$ & {\cellcolor[rgb]{0.753,0.753,0.753}}$0.941 \pm 0.005$ & {\cellcolor[rgb]{0.753,0.753,0.753}}$0.944 \pm 0.005$ & {\cellcolor[rgb]{0.753,0.753,0.753}}$0.945 \pm 0.005$ & {\cellcolor[rgb]{0.753,0.753,0.753}}$0.946 \pm 0.006$ & {\cellcolor[rgb]{0.753,0.753,0.753}}$1.83e-28$ & {\cellcolor[rgb]{0.753,0.753,0.753}}$\mathbf{9.14e-28}$ \\
				& AUC & $0.977 \pm 0.003$ & $0.979 \pm 0.002$ & $0.982 \pm 0.002$ & $0.983 \pm 0.002$ & $0.983 \pm 0.002$ & $0.98 \pm 0.002$ & $0.982 \pm 0.002$ & $0.983 \pm 0.002$ & $0.983 \pm 0.002$ & $0.983 \pm 0.003$ & $2.63e-30$ & $\mathbf{1.32e-29}$ \\
				& Log-Loss & {\cellcolor[rgb]{0.753,0.753,0.753}}$0.216 \pm 0.018$ & {\cellcolor[rgb]{0.753,0.753,0.753}}$0.191 \pm 0.02$ & {\cellcolor[rgb]{0.753,0.753,0.753}}$0.183 \pm 0.02$ & {\cellcolor[rgb]{0.753,0.753,0.753}}$0.167 \pm 0.013$ & {\cellcolor[rgb]{0.753,0.753,0.753}}$0.165 \pm 0.013$ & {\cellcolor[rgb]{0.753,0.753,0.753}}$0.19 \pm 0.022$ & {\cellcolor[rgb]{0.753,0.753,0.753}}$0.175 \pm 0.018$ & {\cellcolor[rgb]{0.753,0.753,0.753}}$0.166 \pm 0.015$ & {\cellcolor[rgb]{0.753,0.753,0.753}}$0.159 \pm 0.014$ & {\cellcolor[rgb]{0.753,0.753,0.753}}$0.157 \pm 0.015$ & {\cellcolor[rgb]{0.753,0.753,0.753}}$1.05e-31$ & {\cellcolor[rgb]{0.753,0.753,0.753}}$\mathbf{5.26e-31}$ \\
				& MEC & $0.077 \pm 0.03$ & $0.095 \pm 0.037$ & $0.073 \pm 0.035$ & $0.092 \pm 0.041$ & $0.109 \pm 0.05$ & $0.086 \pm 0.035$ & $0.085 \pm 0.032$ & $0.083 \pm 0.028$ & $0.096 \pm 0.032$ & $0.199 \pm 0.089$ & $1.38e-09$ & $\mathbf{6.88e-09}$ \\
				& MEW & {\cellcolor[rgb]{0.753,0.753,0.753}}$0.492 \pm 0.051$ & {\cellcolor[rgb]{0.753,0.753,0.753}}$0.531 \pm 0.06$ & {\cellcolor[rgb]{0.753,0.753,0.753}}$0.501 \pm 0.06$ & {\cellcolor[rgb]{0.753,0.753,0.753}}$0.536 \pm 0.062$ & {\cellcolor[rgb]{0.753,0.753,0.753}}$0.562 \pm 0.068$ & {\cellcolor[rgb]{0.753,0.753,0.753}}$0.557 \pm 0.059$ & {\cellcolor[rgb]{0.753,0.753,0.753}}$0.577 \pm 0.046$ & {\cellcolor[rgb]{0.753,0.753,0.753}}$0.584 \pm 0.034$ & {\cellcolor[rgb]{0.753,0.753,0.753}}$0.601 \pm 0.035$ & {\cellcolor[rgb]{0.753,0.753,0.753}}$0.67 \pm 0.086$ & {\cellcolor[rgb]{0.753,0.753,0.753}}$1.73e-23$ & {\cellcolor[rgb]{0.753,0.753,0.753}}$\mathbf{8.63e-23}$ \\
				& Time (sec) & $43.825 \pm 30.227$ & $40.604 \pm 32.278$ & $41.051 \pm 27.423$ & $43.043 \pm 30.507$ & $41.728 \pm 26.758$ & -- & -- & -- & -- & $65.734 \pm 53.252$ & $0.74$ & -- \\
				\hline
				\multirow{6}{*}{GSAD} & $F_1$ & {\cellcolor[rgb]{0.753,0.753,0.753}}$0.992 \pm 0.002$ & {\cellcolor[rgb]{0.753,0.753,0.753}}$0.993 \pm 0.002$ & {\cellcolor[rgb]{0.753,0.753,0.753}}$0.993 \pm 0.002$ & {\cellcolor[rgb]{0.753,0.753,0.753}}$0.993 \pm 0.002$ & {\cellcolor[rgb]{0.753,0.753,0.753}}$0.992 \pm 0.002$ & {\cellcolor[rgb]{0.753,0.753,0.753}}$0.993 \pm 0.002$ & {\cellcolor[rgb]{0.753,0.753,0.753}}$0.994 \pm 0.001$ & {\cellcolor[rgb]{0.753,0.753,0.753}}$0.994 \pm 0.001$ & {\cellcolor[rgb]{0.753,0.753,0.753}}$0.994 \pm 0.001$ & {\cellcolor[rgb]{0.753,0.753,0.753}}$0.994 \pm 0.001$ & {\cellcolor[rgb]{0.753,0.753,0.753}}$2.36e-20$ & {\cellcolor[rgb]{0.753,0.753,0.753}}$\mathbf{1.18e-19}$ \\
				& AUC & $0.999 \pm 0.0$ & $0.999 \pm 0.0$ & $0.999 \pm 0.0$ & $0.999 \pm 0.0$ & $0.999 \pm 0.0$ & $0.999 \pm 0.0$ & $0.999 \pm 0.0$ & $1.0 \pm 0.0$ & $0.999 \pm 0.0$ & $0.999 \pm 0.0$ & $4.93e-13$ & $\mathbf{2.47e-12}$ \\
				& Log-Loss & {\cellcolor[rgb]{0.753,0.753,0.753}}$0.038 \pm 0.009$ & {\cellcolor[rgb]{0.753,0.753,0.753}}$0.035 \pm 0.008$ & {\cellcolor[rgb]{0.753,0.753,0.753}}$0.034 \pm 0.008$ & {\cellcolor[rgb]{0.753,0.753,0.753}}$0.034 \pm 0.006$ & {\cellcolor[rgb]{0.753,0.753,0.753}}$0.036 \pm 0.006$ & {\cellcolor[rgb]{0.753,0.753,0.753}}$0.033 \pm 0.007$ & {\cellcolor[rgb]{0.753,0.753,0.753}}$0.031 \pm 0.006$ & {\cellcolor[rgb]{0.753,0.753,0.753}}$0.029 \pm 0.005$ & {\cellcolor[rgb]{0.753,0.753,0.753}}$0.029 \pm 0.005$ & {\cellcolor[rgb]{0.753,0.753,0.753}}$0.032 \pm 0.006$ & {\cellcolor[rgb]{0.753,0.753,0.753}}$7.72e-26$ & {\cellcolor[rgb]{0.753,0.753,0.753}}$\mathbf{3.86e-25}$ \\
				& MEC & $0.017 \pm 0.008$ & $0.02 \pm 0.012$ & $0.019 \pm 0.013$ & $0.02 \pm 0.012$ & $0.023 \pm 0.013$ & $0.02 \pm 0.009$ & $0.021 \pm 0.011$ & $0.022 \pm 0.012$ & $0.023 \pm 0.012$ & $0.036 \pm 0.031$ & $4.90e-04$ & $\mathbf{2.45e-03}$ \\
				& MEW & {\cellcolor[rgb]{0.753,0.753,0.753}}$0.696 \pm 0.109$ & {\cellcolor[rgb]{0.753,0.753,0.753}}$0.725 \pm 0.122$ & {\cellcolor[rgb]{0.753,0.753,0.753}}$0.689 \pm 0.139$ & {\cellcolor[rgb]{0.753,0.753,0.753}}$0.706 \pm 0.103$ & {\cellcolor[rgb]{0.753,0.753,0.753}}$0.743 \pm 0.139$ & {\cellcolor[rgb]{0.753,0.753,0.753}}$0.79 \pm 0.129$ & {\cellcolor[rgb]{0.753,0.753,0.753}}$0.791 \pm 0.152$ & {\cellcolor[rgb]{0.753,0.753,0.753}}$0.821 \pm 0.147$ & {\cellcolor[rgb]{0.753,0.753,0.753}}$0.813 \pm 0.156$ & {\cellcolor[rgb]{0.753,0.753,0.753}}$0.768 \pm 0.183$ & {\cellcolor[rgb]{0.753,0.753,0.753}}$1.10e-07$ & {\cellcolor[rgb]{0.753,0.753,0.753}}$\mathbf{5.49e-07}$ \\
				& Time (sec) & $376.691 \pm 159.509$ & $402.544 \pm 181.878$ & $358.6 \pm 151.871$ & $372.989 \pm 132.698$ & $357.346 \pm 135.543$ & -- & -- & -- & -- & $597.07 \pm 350.754$ & $\mathbf{0.03}$ & -- \\
				\hline
				\multirow{6}{*}{HAPT} & $F_1$ & {\cellcolor[rgb]{0.753,0.753,0.753}}$0.952 \pm 0.004$ & {\cellcolor[rgb]{0.753,0.753,0.753}}$0.951 \pm 0.005$ & {\cellcolor[rgb]{0.753,0.753,0.753}}$0.959 \pm 0.004$ & {\cellcolor[rgb]{0.753,0.753,0.753}}$0.968 \pm 0.004$ & {\cellcolor[rgb]{0.753,0.753,0.753}}$0.964 \pm 0.01$ & {\cellcolor[rgb]{0.753,0.753,0.753}}$0.955 \pm 0.004$ & {\cellcolor[rgb]{0.753,0.753,0.753}}$0.962 \pm 0.005$ & {\cellcolor[rgb]{0.753,0.753,0.753}}$0.967 \pm 0.004$ & {\cellcolor[rgb]{0.753,0.753,0.753}}$0.971 \pm 0.004$ & {\cellcolor[rgb]{0.753,0.753,0.753}}$0.973 \pm 0.005$ & {\cellcolor[rgb]{0.753,0.753,0.753}}$4.20e-41$ & {\cellcolor[rgb]{0.753,0.753,0.753}}$\mathbf{2.10e-40}$ \\
				& AUC & $0.997 \pm 0.001$ & $0.996 \pm 0.001$ & $0.997 \pm 0.001$ & $0.998 \pm 0.001$ & $0.998 \pm 0.001$ & $0.997 \pm 0.001$ & $0.997 \pm 0.001$ & $0.998 \pm 0.001$ & $0.998 \pm 0.001$ & $0.998 \pm 0.001$ & $3.21e-38$ & $\mathbf{1.60e-37}$ \\
				& Log-Loss & {\cellcolor[rgb]{0.753,0.753,0.753}}$0.139 \pm 0.014$ & {\cellcolor[rgb]{0.753,0.753,0.753}}$0.147 \pm 0.016$ & {\cellcolor[rgb]{0.753,0.753,0.753}}$0.122 \pm 0.012$ & {\cellcolor[rgb]{0.753,0.753,0.753}}$0.093 \pm 0.009$ & {\cellcolor[rgb]{0.753,0.753,0.753}}$0.104 \pm 0.029$ & {\cellcolor[rgb]{0.753,0.753,0.753}}$0.123 \pm 0.011$ & {\cellcolor[rgb]{0.753,0.753,0.753}}$0.105 \pm 0.01$ & {\cellcolor[rgb]{0.753,0.753,0.753}}$0.091 \pm 0.009$ & {\cellcolor[rgb]{0.753,0.753,0.753}}$0.085 \pm 0.01$ & {\cellcolor[rgb]{0.753,0.753,0.753}}$0.089 \pm 0.021$ & {\cellcolor[rgb]{0.753,0.753,0.753}}$3.80e-40$ & {\cellcolor[rgb]{0.753,0.753,0.753}}$\mathbf{1.90e-39}$ \\
				& MEC & $0.052 \pm 0.031$ & $0.049 \pm 0.033$ & $0.046 \pm 0.037$ & $0.04 \pm 0.032$ & $0.05 \pm 0.04$ & $0.057 \pm 0.035$ & $0.06 \pm 0.042$ & $0.06 \pm 0.045$ & $0.067 \pm 0.052$ & $0.187 \pm 0.099$ & $3.51e-17$ & $\mathbf{1.75e-16}$ \\
				& MEW & {\cellcolor[rgb]{0.753,0.753,0.753}}$0.876 \pm 0.077$ & {\cellcolor[rgb]{0.753,0.753,0.753}}$0.916 \pm 0.108$ & {\cellcolor[rgb]{0.753,0.753,0.753}}$0.944 \pm 0.101$ & {\cellcolor[rgb]{0.753,0.753,0.753}}$0.941 \pm 0.098$ & {\cellcolor[rgb]{0.753,0.753,0.753}}$0.918 \pm 0.114$ & {\cellcolor[rgb]{0.753,0.753,0.753}}$1.04 \pm 0.094$ & {\cellcolor[rgb]{0.753,0.753,0.753}}$1.134 \pm 0.067$ & {\cellcolor[rgb]{0.753,0.753,0.753}}$1.195 \pm 0.072$ & {\cellcolor[rgb]{0.753,0.753,0.753}}$1.231 \pm 0.073$ & {\cellcolor[rgb]{0.753,0.753,0.753}}$1.379 \pm 0.148$ & {\cellcolor[rgb]{0.753,0.753,0.753}}$8.20e-40$ & {\cellcolor[rgb]{0.753,0.753,0.753}}$\mathbf{4.10e-39}$ \\
				& Time (sec) & $26.23 \pm 11.782$ & $22.54 \pm 8.365$ & $22.191 \pm 6.395$ & $25.304 \pm 11.507$ & $21.655 \pm 10.108$ & -- & -- & -- & -- & $52.055 \pm 27.878$ & $\mathbf{5.99e-06}$ & -- \\
				\hline
				\multirow{6}{*}{ISOLET} & $F_1$ & {\cellcolor[rgb]{0.753,0.753,0.753}}$0.865 \pm 0.01$ & {\cellcolor[rgb]{0.753,0.753,0.753}}$0.918 \pm 0.006$ & {\cellcolor[rgb]{0.753,0.753,0.753}}$0.928 \pm 0.007$ & {\cellcolor[rgb]{0.753,0.753,0.753}}$0.927 \pm 0.006$ & {\cellcolor[rgb]{0.753,0.753,0.753}}$0.924 \pm 0.008$ & {\cellcolor[rgb]{0.753,0.753,0.753}}$0.902 \pm 0.019$ & {\cellcolor[rgb]{0.753,0.753,0.753}}$0.924 \pm 0.012$ & {\cellcolor[rgb]{0.753,0.753,0.753}}$0.933 \pm 0.01$ & {\cellcolor[rgb]{0.753,0.753,0.753}}$0.939 \pm 0.005$ & {\cellcolor[rgb]{0.753,0.753,0.753}}$0.94 \pm 0.008$ & {\cellcolor[rgb]{0.753,0.753,0.753}}$5.02e-41$ & {\cellcolor[rgb]{0.753,0.753,0.753}}$\mathbf{2.51e-40}$ \\
				& AUC & $0.995 \pm 0.001$ & $0.998 \pm 0.0$ & $0.998 \pm 0.0$ & $0.998 \pm 0.0$ & $0.998 \pm 0.0$ & $0.997 \pm 0.001$ & $0.998 \pm 0.0$ & $0.999 \pm 0.0$ & $0.999 \pm 0.0$ & $0.999 \pm 0.0$ & $2.11e-38$ & $\mathbf{1.06e-37}$ \\
				& Log-Loss & {\cellcolor[rgb]{0.753,0.753,0.753}}$0.417 \pm 0.037$ & {\cellcolor[rgb]{0.753,0.753,0.753}}$0.27 \pm 0.029$ & {\cellcolor[rgb]{0.753,0.753,0.753}}$0.243 \pm 0.032$ & {\cellcolor[rgb]{0.753,0.753,0.753}}$0.25 \pm 0.033$ & {\cellcolor[rgb]{0.753,0.753,0.753}}$0.269 \pm 0.055$ & {\cellcolor[rgb]{0.753,0.753,0.753}}$0.302 \pm 0.035$ & {\cellcolor[rgb]{0.753,0.753,0.753}}$0.258 \pm 0.026$ & {\cellcolor[rgb]{0.753,0.753,0.753}}$0.242 \pm 0.024$ & {\cellcolor[rgb]{0.753,0.753,0.753}}$0.229 \pm 0.021$ & {\cellcolor[rgb]{0.753,0.753,0.753}}$0.275 \pm 0.076$ & {\cellcolor[rgb]{0.753,0.753,0.753}}$4.43e-24$ & {\cellcolor[rgb]{0.753,0.753,0.753}}$\mathbf{2.21e-23}$ \\
				& MEC & $0.345 \pm 0.266$ & $0.281 \pm 0.17$ & $0.262 \pm 0.195$ & $0.301 \pm 0.168$ & $0.357 \pm 0.241$ & $0.317 \pm 0.163$ & $0.329 \pm 0.18$ & $0.351 \pm 0.162$ & $0.345 \pm 0.153$ & $0.735 \pm 0.343$ & $5.82e-10$ & $\mathbf{2.91e-09}$ \\
				& MEW & {\cellcolor[rgb]{0.753,0.753,0.753}}$1.414 \pm 0.325$ & {\cellcolor[rgb]{0.753,0.753,0.753}}$1.528 \pm 0.304$ & {\cellcolor[rgb]{0.753,0.753,0.753}}$1.585 \pm 0.326$ & {\cellcolor[rgb]{0.753,0.753,0.753}}$1.658 \pm 0.335$ & {\cellcolor[rgb]{0.753,0.753,0.753}}$1.741 \pm 0.487$ & {\cellcolor[rgb]{0.753,0.753,0.753}}$1.547 \pm 0.198$ & {\cellcolor[rgb]{0.753,0.753,0.753}}$1.684 \pm 0.199$ & {\cellcolor[rgb]{0.753,0.753,0.753}}$1.758 \pm 0.157$ & {\cellcolor[rgb]{0.753,0.753,0.753}}$1.787 \pm 0.205$ & {\cellcolor[rgb]{0.753,0.753,0.753}}$2.365 \pm 0.353$ & {\cellcolor[rgb]{0.753,0.753,0.753}}$2.43e-19$ & {\cellcolor[rgb]{0.753,0.753,0.753}}$\mathbf{1.22e-18}$ \\
				& Time (sec) & $17.568 \pm 3.756$ & $15.62 \pm 3.466$ & $17.703 \pm 4.436$ & $15.957 \pm 3.206$ & $18.786 \pm 4.499$ & -- & -- & -- & -- & $44.976 \pm 26.659$ & $\mathbf{1.29e-06}$ & -- \\
				\hline
				\multirow{6}{*}{PD} & $F_1$ & {\cellcolor[rgb]{0.753,0.753,0.753}}$0.877 \pm 0.02$ & {\cellcolor[rgb]{0.753,0.753,0.753}}$0.875 \pm 0.018$ & {\cellcolor[rgb]{0.753,0.753,0.753}}$0.872 \pm 0.023$ & {\cellcolor[rgb]{0.753,0.753,0.753}}$0.867 \pm 0.02$ & {\cellcolor[rgb]{0.753,0.753,0.753}}$0.864 \pm 0.024$ & {\cellcolor[rgb]{0.753,0.753,0.753}}$0.879 \pm 0.019$ & {\cellcolor[rgb]{0.753,0.753,0.753}}$0.879 \pm 0.019$ & {\cellcolor[rgb]{0.753,0.753,0.753}}$0.877 \pm 0.023$ & {\cellcolor[rgb]{0.753,0.753,0.753}}$0.871 \pm 0.019$ & {\cellcolor[rgb]{0.753,0.753,0.753}}$0.886 \pm 0.022$ & {\cellcolor[rgb]{0.753,0.753,0.753}}$8.74e-08$ & {\cellcolor[rgb]{0.753,0.753,0.753}}$\mathbf{4.37e-07}$ \\
				& AUC & $0.926 \pm 0.018$ & $0.924 \pm 0.018$ & $0.922 \pm 0.023$ & $0.922 \pm 0.018$ & $0.917 \pm 0.022$ & $0.933 \pm 0.015$ & $0.937 \pm 0.016$ & $0.94 \pm 0.015$ & $0.939 \pm 0.015$ & $0.939 \pm 0.016$ & $1.00e-22$ & $\mathbf{5.01e-22}$ \\
				& Log-Loss & {\cellcolor[rgb]{0.753,0.753,0.753}}$0.354 \pm 0.07$ & {\cellcolor[rgb]{0.753,0.753,0.753}}$0.373 \pm 0.072$ & {\cellcolor[rgb]{0.753,0.753,0.753}}$0.361 \pm 0.07$ & {\cellcolor[rgb]{0.753,0.753,0.753}}$0.374 \pm 0.067$ & {\cellcolor[rgb]{0.753,0.753,0.753}}$0.376 \pm 0.076$ & {\cellcolor[rgb]{0.753,0.753,0.753}}$0.329 \pm 0.058$ & {\cellcolor[rgb]{0.753,0.753,0.753}}$0.309 \pm 0.053$ & {\cellcolor[rgb]{0.753,0.753,0.753}}$0.298 \pm 0.046$ & {\cellcolor[rgb]{0.753,0.753,0.753}}$0.297 \pm 0.038$ & {\cellcolor[rgb]{0.753,0.753,0.753}}$0.319 \pm 0.064$ & {\cellcolor[rgb]{0.753,0.753,0.753}}$1.84e-25$ & {\cellcolor[rgb]{0.753,0.753,0.753}}$\mathbf{9.19e-25}$ \\
				& MEC & $0.172 \pm 0.06$ & $0.155 \pm 0.062$ & $0.171 \pm 0.063$ & $0.157 \pm 0.058$ & $0.172 \pm 0.055$ & $0.161 \pm 0.07$ & $0.174 \pm 0.065$ & $0.175 \pm 0.058$ & $0.186 \pm 0.057$ & $0.103 \pm 0.087$ & $2.63e-06$ & $\mathbf{1.31e-05}$ \\
				& MEW & {\cellcolor[rgb]{0.753,0.753,0.753}}$0.55 \pm 0.081$ & {\cellcolor[rgb]{0.753,0.753,0.753}}$0.534 \pm 0.097$ & {\cellcolor[rgb]{0.753,0.753,0.753}}$0.564 \pm 0.07$ & {\cellcolor[rgb]{0.753,0.753,0.753}}$0.561 \pm 0.074$ & {\cellcolor[rgb]{0.753,0.753,0.753}}$0.565 \pm 0.102$ & {\cellcolor[rgb]{0.753,0.753,0.753}}$0.595 \pm 0.084$ & {\cellcolor[rgb]{0.753,0.753,0.753}}$0.64 \pm 0.075$ & {\cellcolor[rgb]{0.753,0.753,0.753}}$0.669 \pm 0.057$ & {\cellcolor[rgb]{0.753,0.753,0.753}}$0.695 \pm 0.053$ & {\cellcolor[rgb]{0.753,0.753,0.753}}$0.56 \pm 0.09$ & {\cellcolor[rgb]{0.753,0.753,0.753}}$3.65e-20$ & {\cellcolor[rgb]{0.753,0.753,0.753}}$\mathbf{1.83e-19}$ \\
				& Time (sec) & $9757.384 \pm 2264.88$ & $10590.484 \pm 2341.929$ & $9631.233 \pm 2521.523$ & $9658.711 \pm 2321.093$ & $9747.392 \pm 2474.261$ & -- & -- & -- & -- & $12490.251 \pm 2070.225$ & $\mathbf{3.83e-05}$ & -- \\
				\hline
			\end{tabular}
		}
	\end{sidewaystable*}
	\FloatBarrier
	%
	\begin{figure*}[ht] 
		\centering
		\subfloat[APSF]{\includegraphics[width=0.24\textwidth]{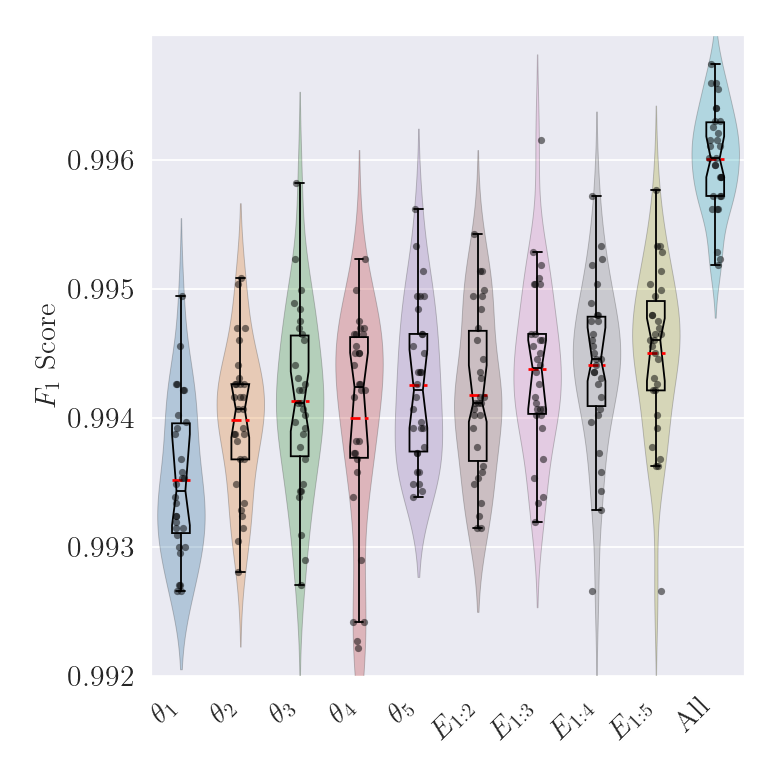}\label{fig:xgbapsf_F1}}%
		\hfill
		\subfloat[ARWPM]{\includegraphics[width=0.24\textwidth]{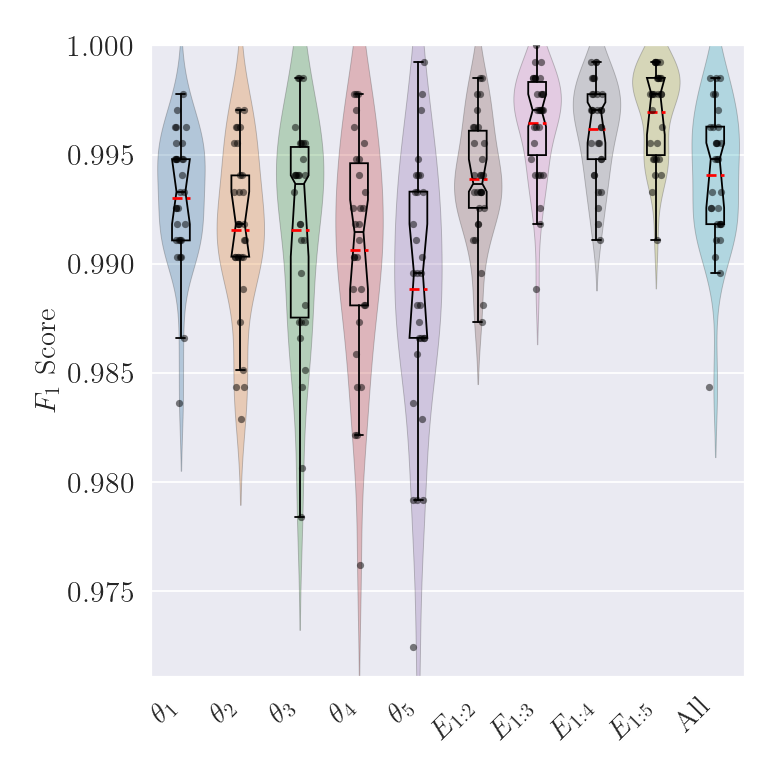}\label{fig:xgbarwpm_F1}}%
		\hfill
		\subfloat[GECR]{\includegraphics[width=0.24\textwidth]{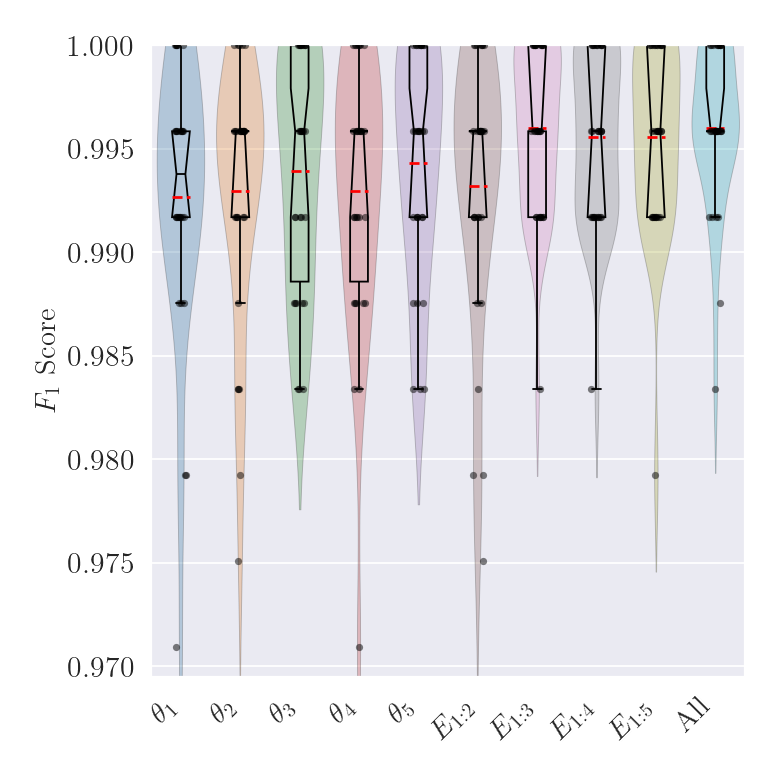}\label{fig:xgbgecr_F1}}%
		\hfill
		\subfloat[GFE]{\includegraphics[width=0.24\textwidth]{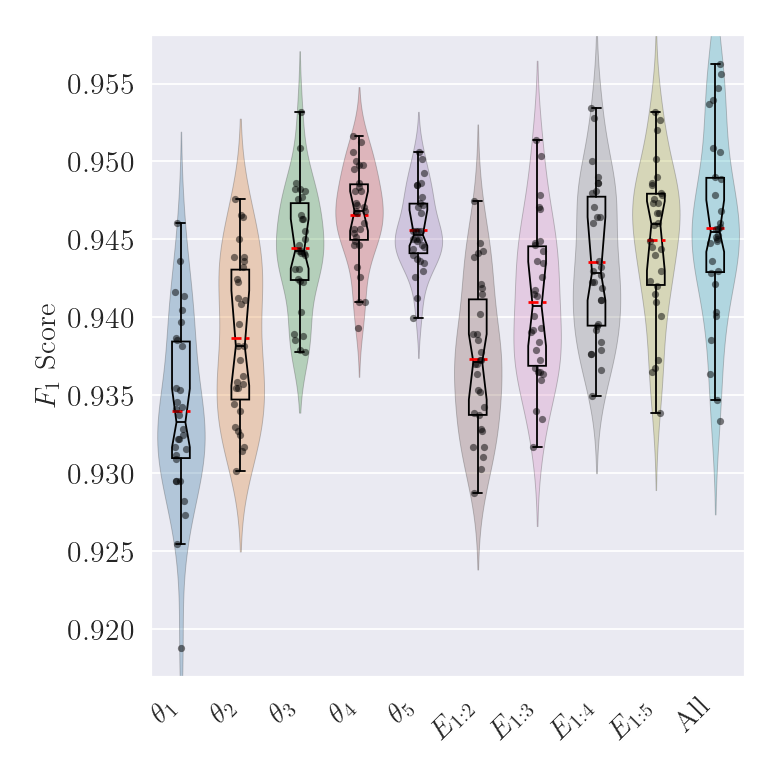}\label{fig:xgbgfe_F1}}
		
		\subfloat[GSAD]{\includegraphics[width=0.24\textwidth]{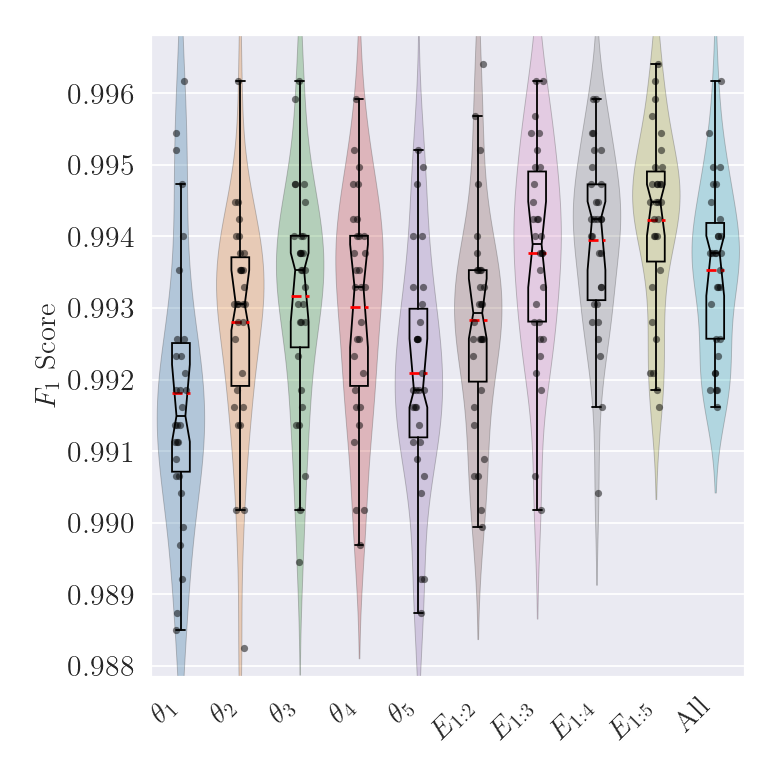}\label{fig:fpgsad_F1}}%
		\hfill
		\subfloat[HAPT]{\includegraphics[width=0.24\textwidth]{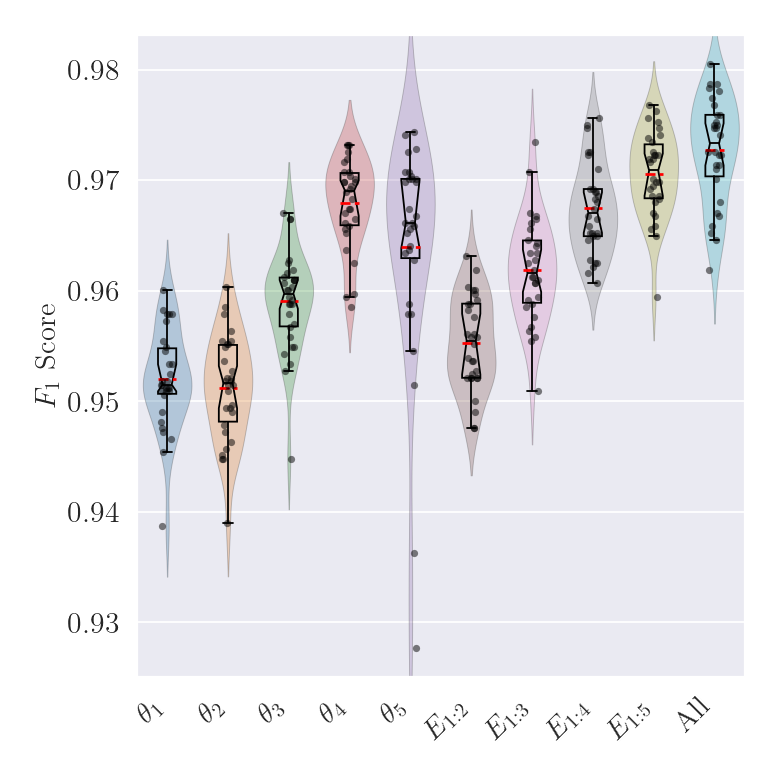}\label{fig:xgbhapt_F1}}%
		\hfill
		\subfloat[ISOLET]{\includegraphics[width=0.24\textwidth]{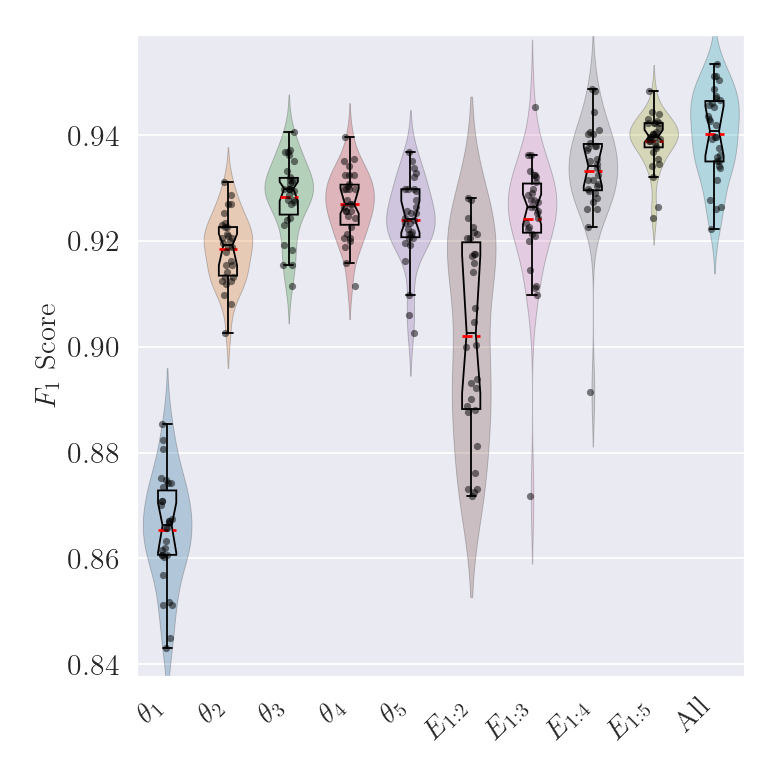}\label{fig:xgbisolet_F1}}%
		\hfill
		\subfloat[PD]{\includegraphics[width=0.24\textwidth]{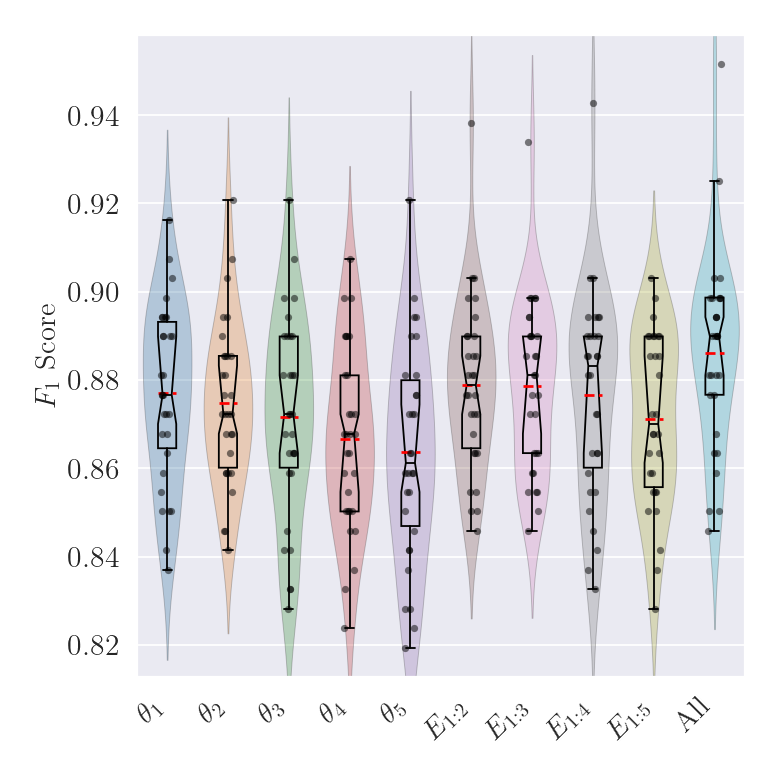}\label{fig:xgbpd_F1}}
		\caption[The distribution of the obtained $F_1$ score for 30 XGBoost runs.]{The raincloud plot of $F_1$ score results obtained from 30 XGBoost runs.}
		
		\label{fig:xgb_F1}
	\end{figure*}
	
	\begin{figure*}[t] 
		\centering
		\subfloat[APSF]{\includegraphics[width=0.24\textwidth]{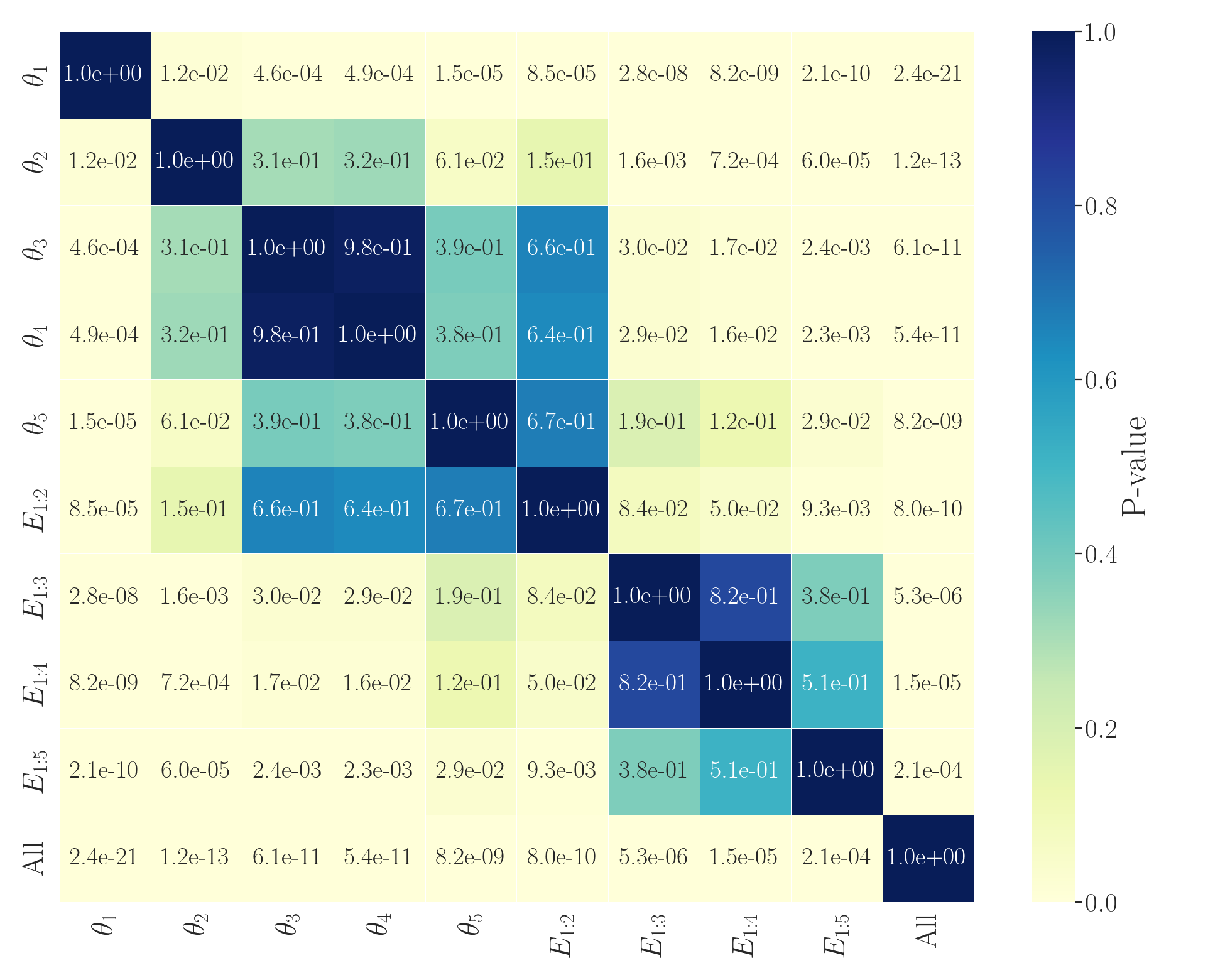}\label{fig:xgbnemapsf_F1}}%
		\hfill
		\subfloat[ARWPM]{\includegraphics[width=0.24\textwidth]{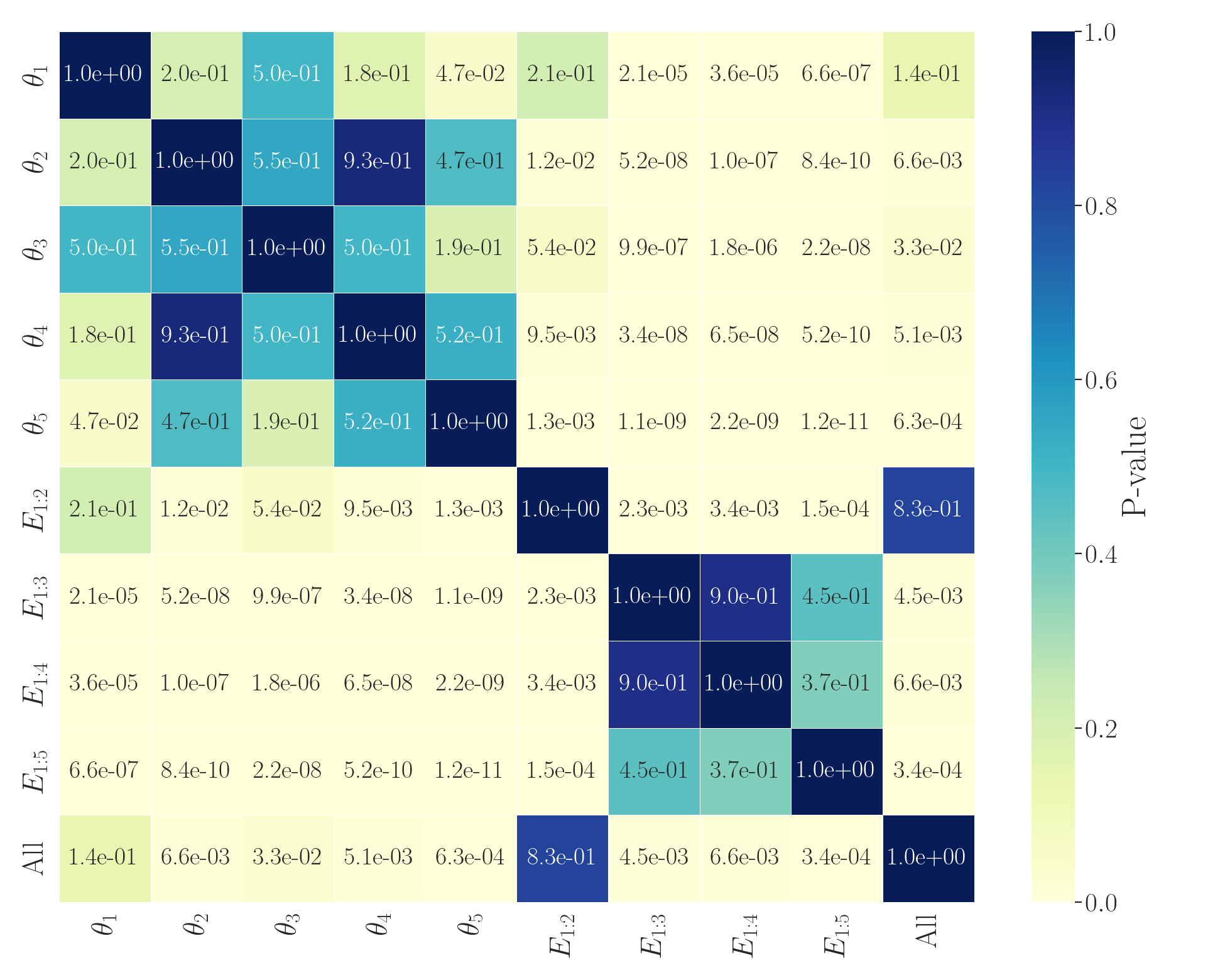}\label{fig:xgbnemarwpm_F1}}%
		\hfill
		\subfloat[GECR]{\includegraphics[width=0.24\textwidth]{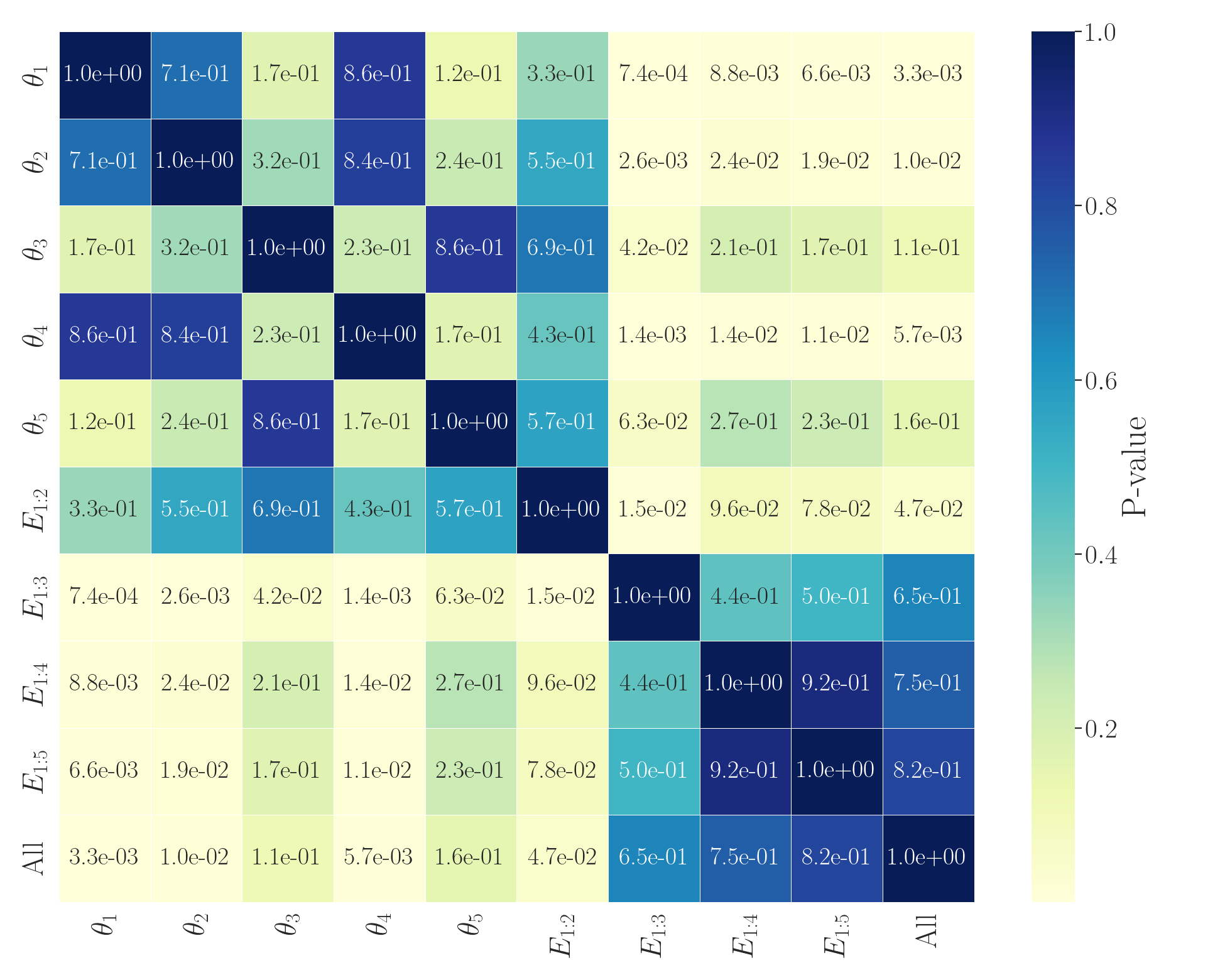}\label{fig:xgbnemgecr_F1}}%
		\hfill
		\subfloat[GFE]{\includegraphics[width=0.24\textwidth]{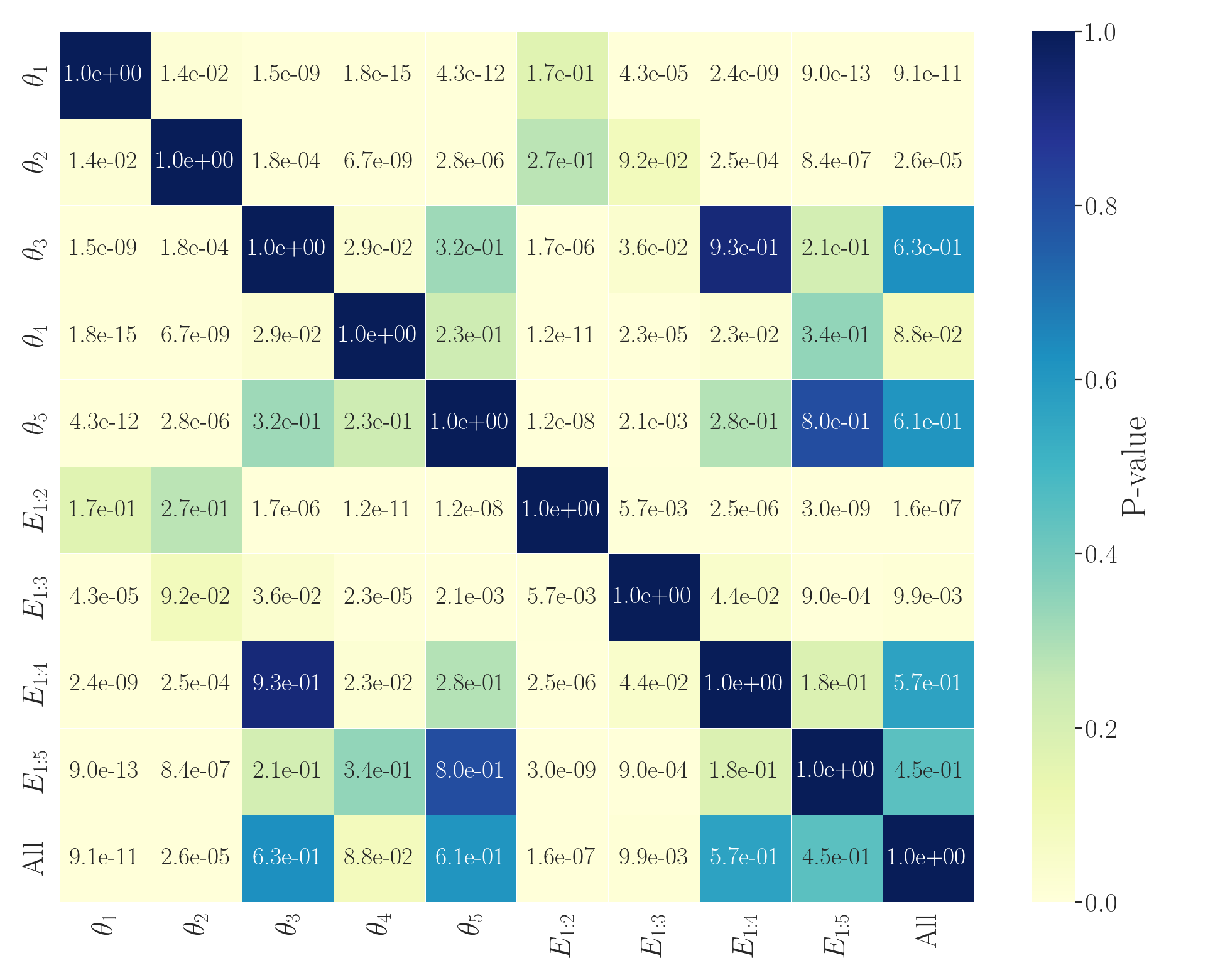}\label{fig:xgbnemgfe_F1}}
		
		\subfloat[GSAD]{\includegraphics[width=0.24\textwidth]{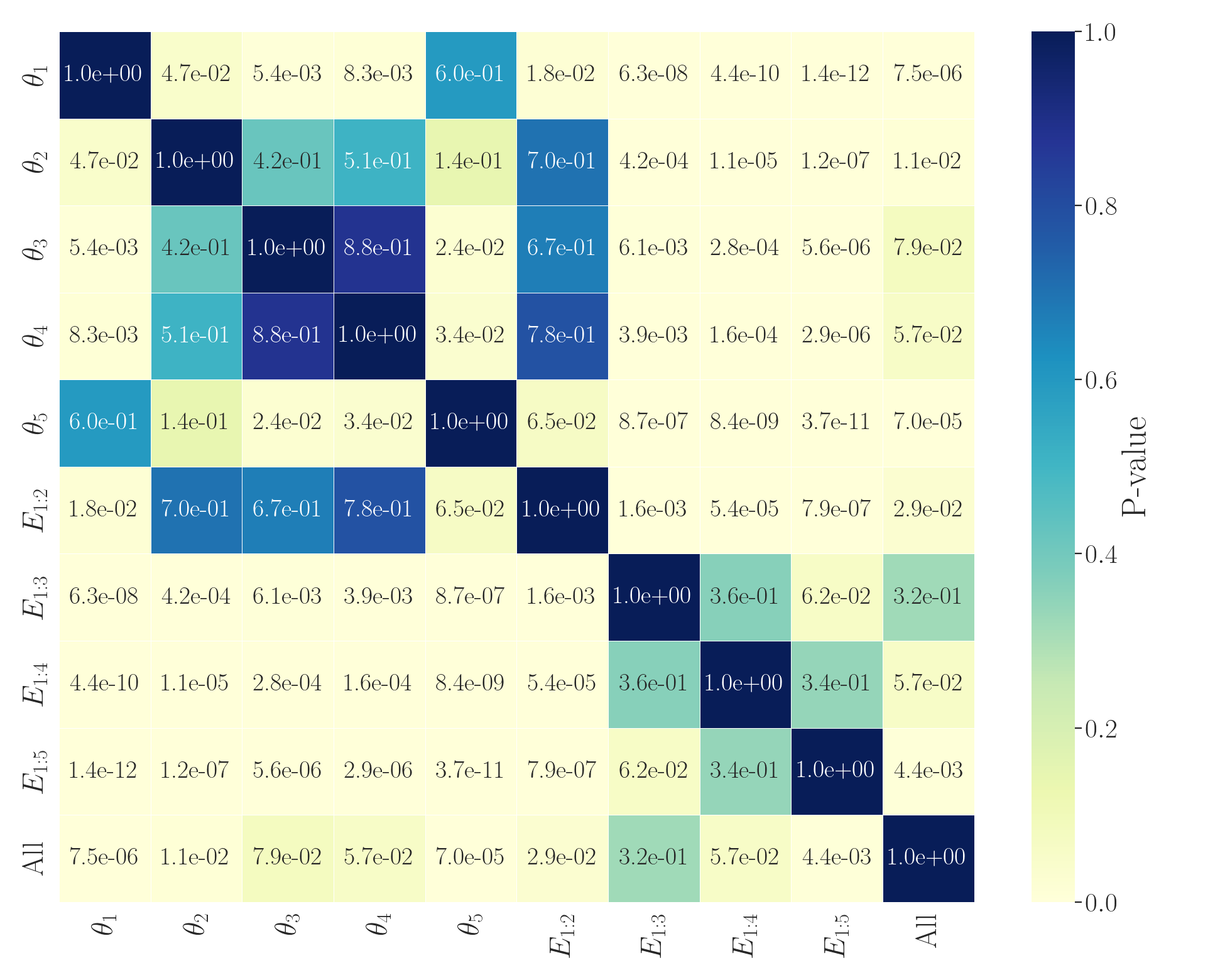}\label{fig:xgbnemgsad_F1}}%
		\hfill
		\subfloat[HAPT]{\includegraphics[width=0.24\textwidth]{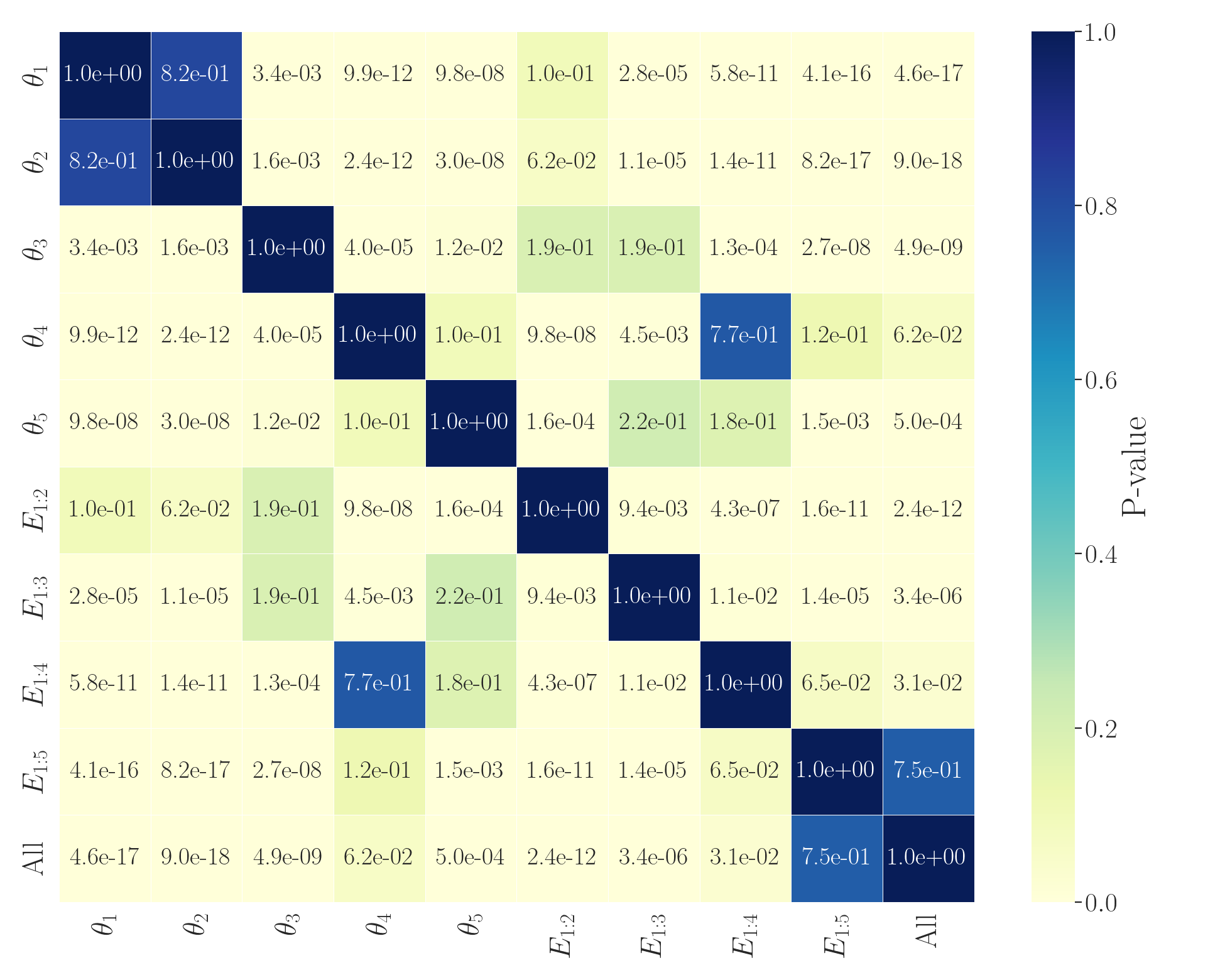}\label{fig:xgbnemhapt_F1}}%
		\hfill
		\subfloat[ISOLET]{\includegraphics[width=0.24\textwidth]{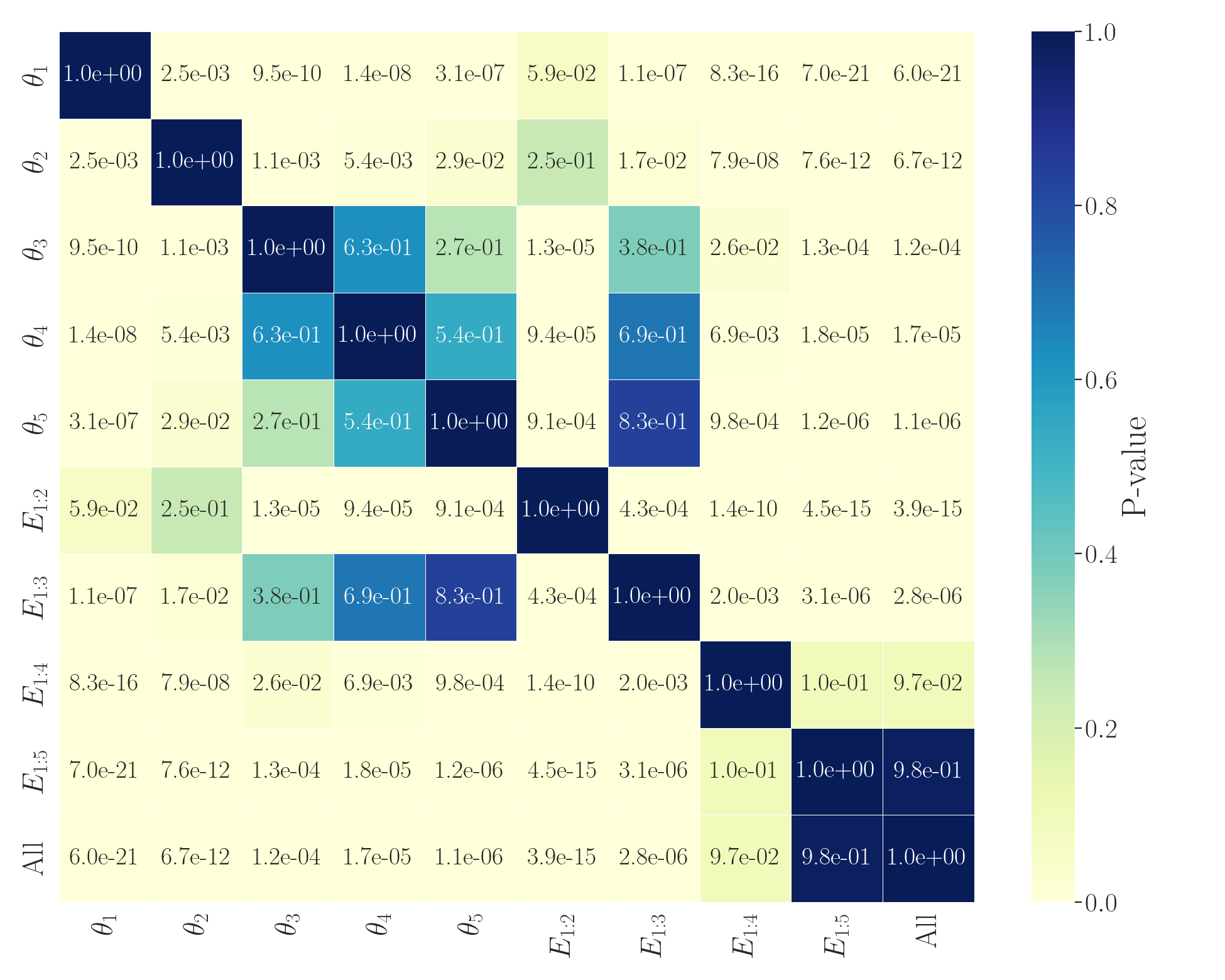}\label{fig:xgbnemisolet_F1}}%
		\hfill
		\subfloat[PD]{\includegraphics[width=0.24\textwidth]{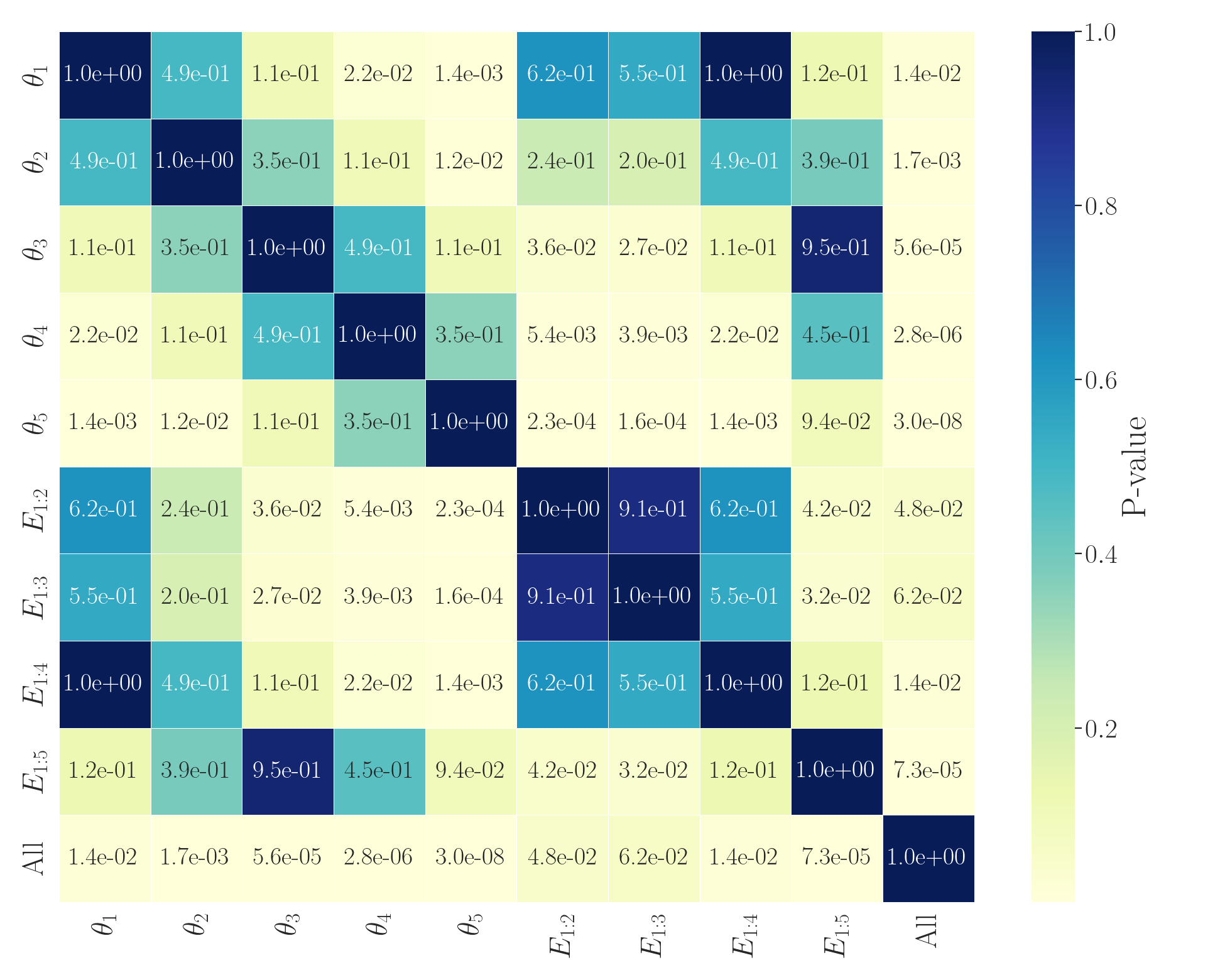}\label{fig:xgbnempd_F1}}
		\caption[The adjusted Conover's P-values for the obtained $F_1$ score from 30 XGBoost runs.]{The results of the Conover post-hoc test on testing data’s $F_1$ score obtained from 30 XGBoost runs.}
		
		\label{fig:xgbnem_F1}
	\end{figure*}
	\FloatBarrier
	
	\begin{figure*}[htbp]
		\centering
		\subfloat[APSF]{\includegraphics[width=0.24\textwidth]{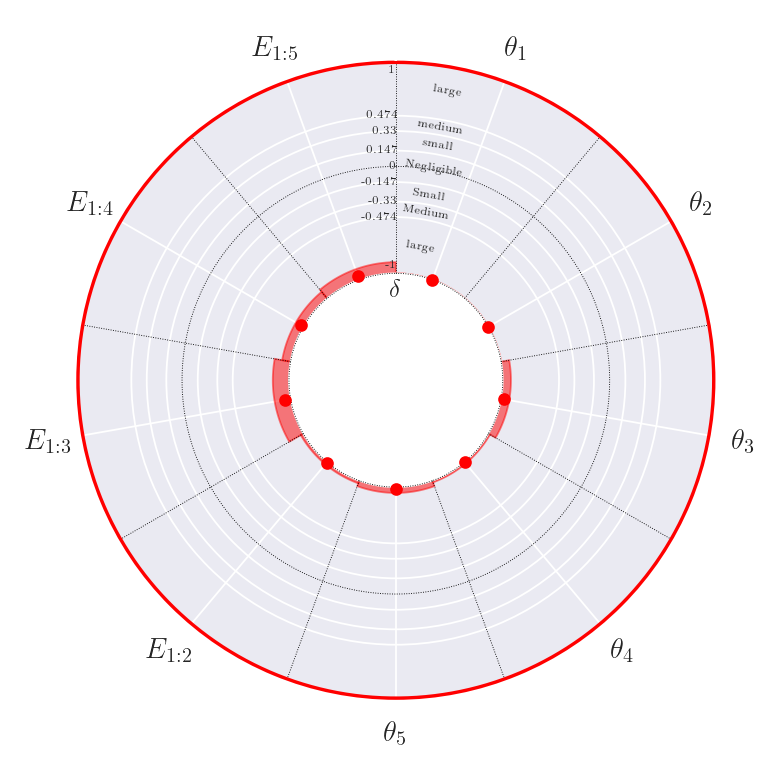}\label{fig:xgbcliffapsf_F1}}%
		\hfill
		\subfloat[ARWPM]{\includegraphics[width=0.24\textwidth]{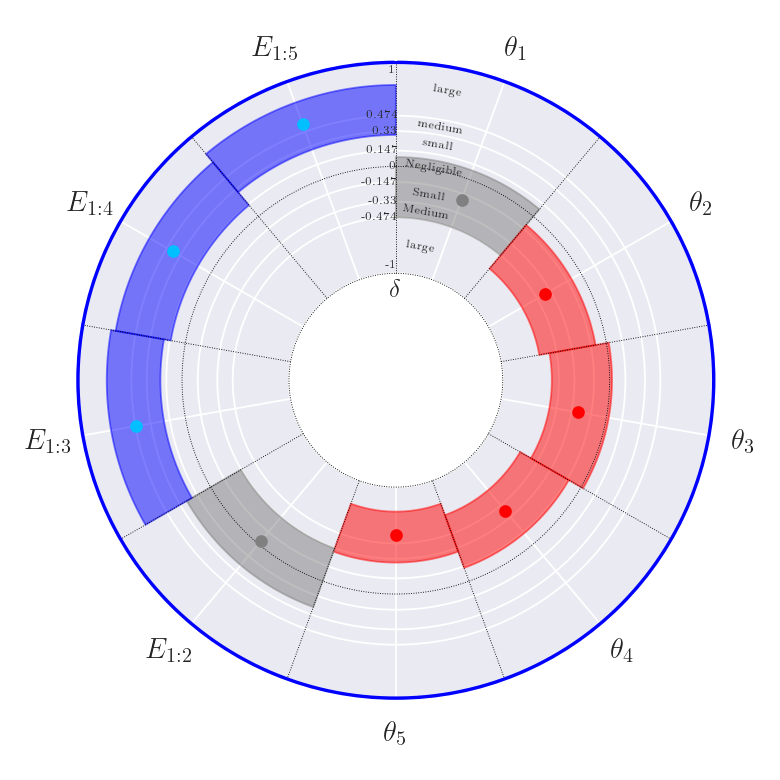}\label{fig:xgbcliffarwpm_F1}}%
		\hfill
		\subfloat[GECR]{\includegraphics[width=0.24\textwidth]{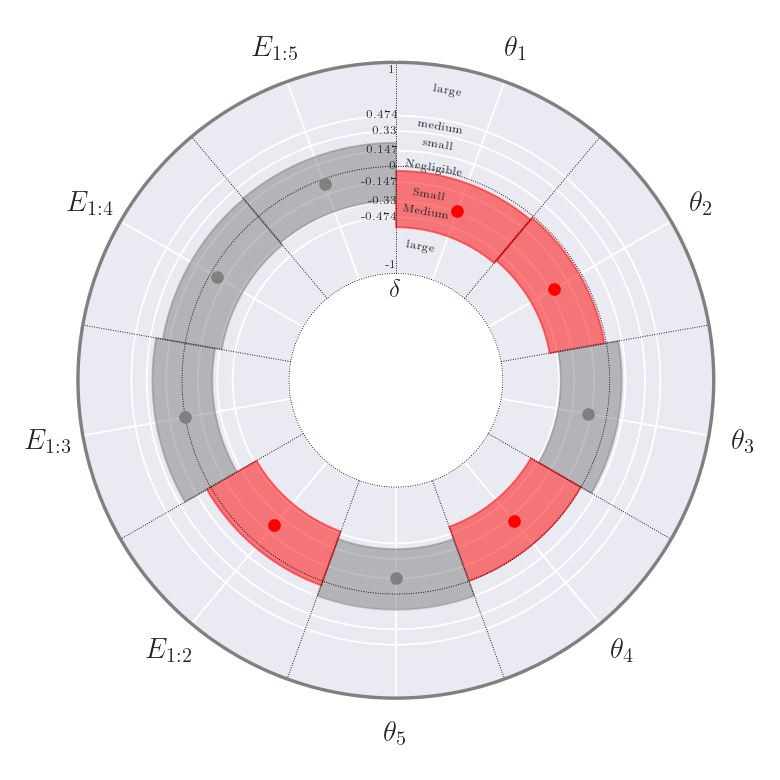}\label{fig:xgbcliffgecr_F1}}%
		\hfill
		\subfloat[GFE]{\includegraphics[width=0.24\textwidth]{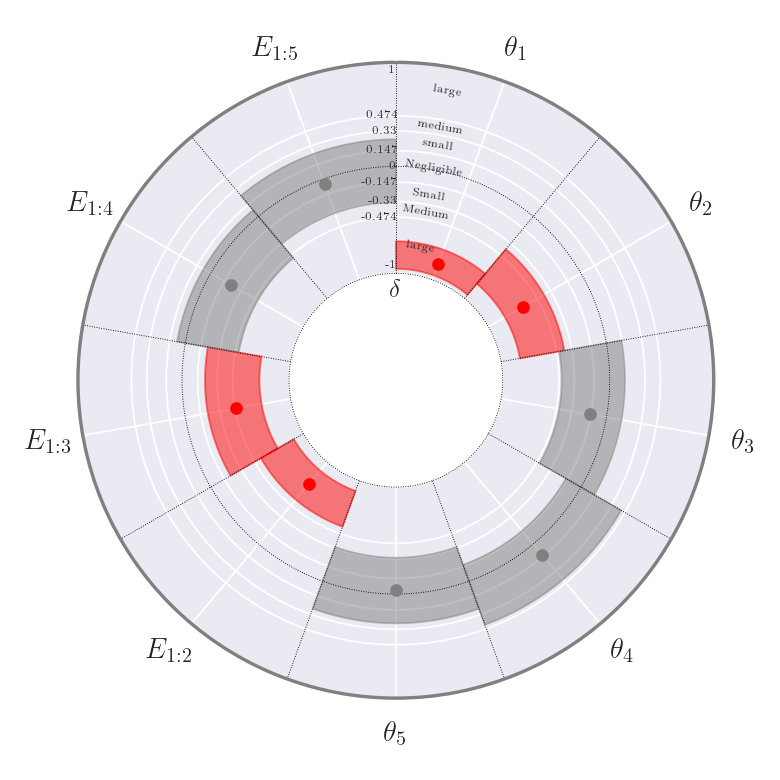}\label{fig:xgbcliffgfe_F1}}
		
		\subfloat[GSAD]{\includegraphics[width=0.24\textwidth]{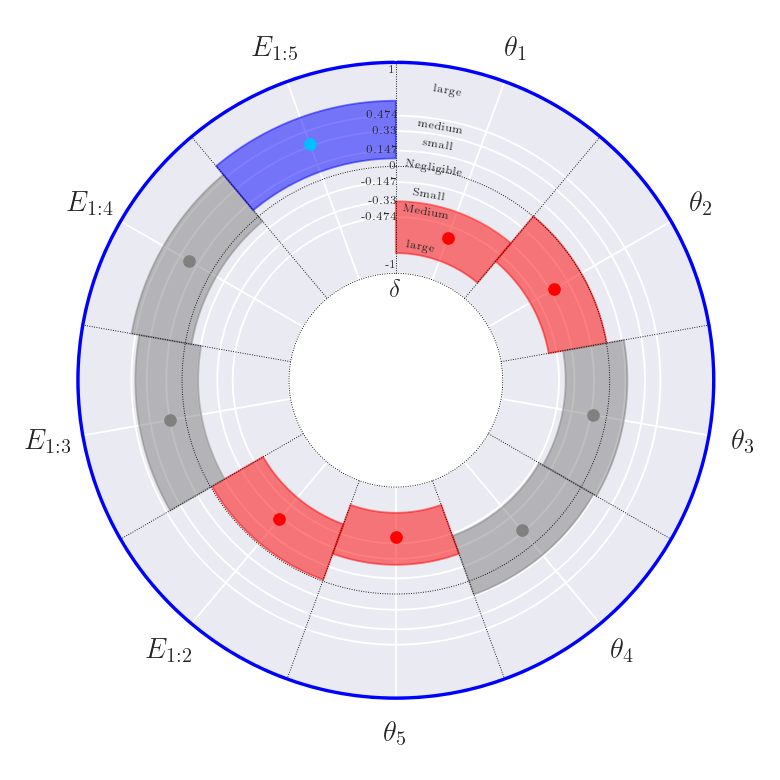}\label{fig:xgbcliffgsad_F1}}%
		\hfill
		\subfloat[HAPT]{\includegraphics[width=0.24\textwidth]{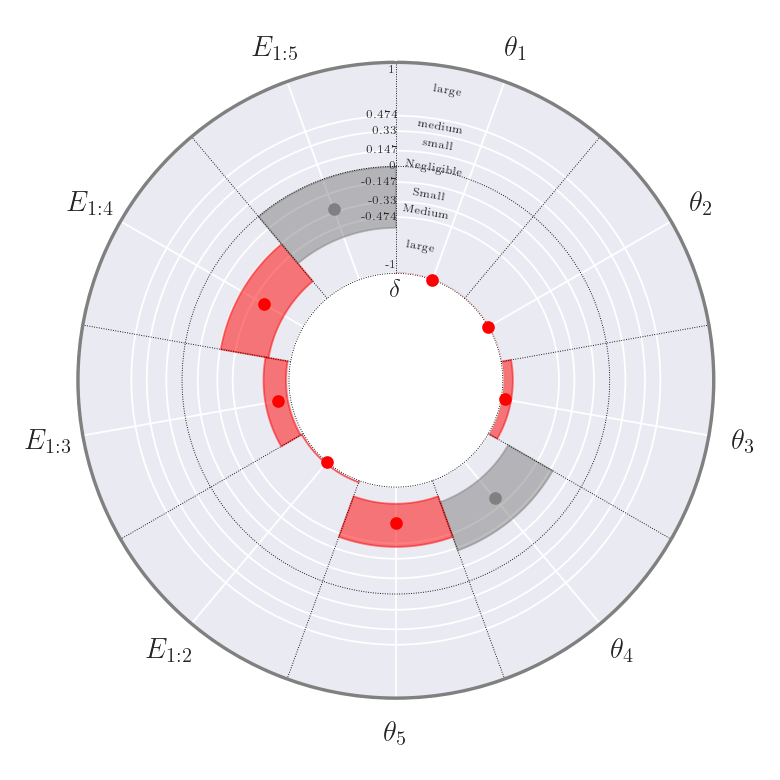}\label{fig:xgbcliffhapt_F1}}%
		\hfill
		\subfloat[ISOLET]{\includegraphics[width=0.24\textwidth]{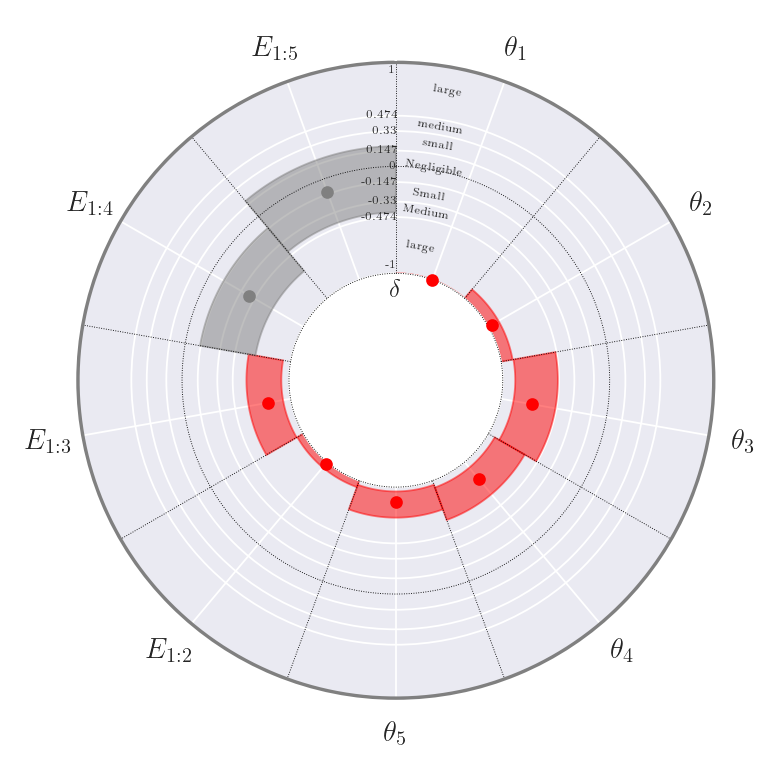}\label{fig:xgbcliffisolet_F1}}%
		\hfill
		\subfloat[PD]{\includegraphics[width=0.24\textwidth]{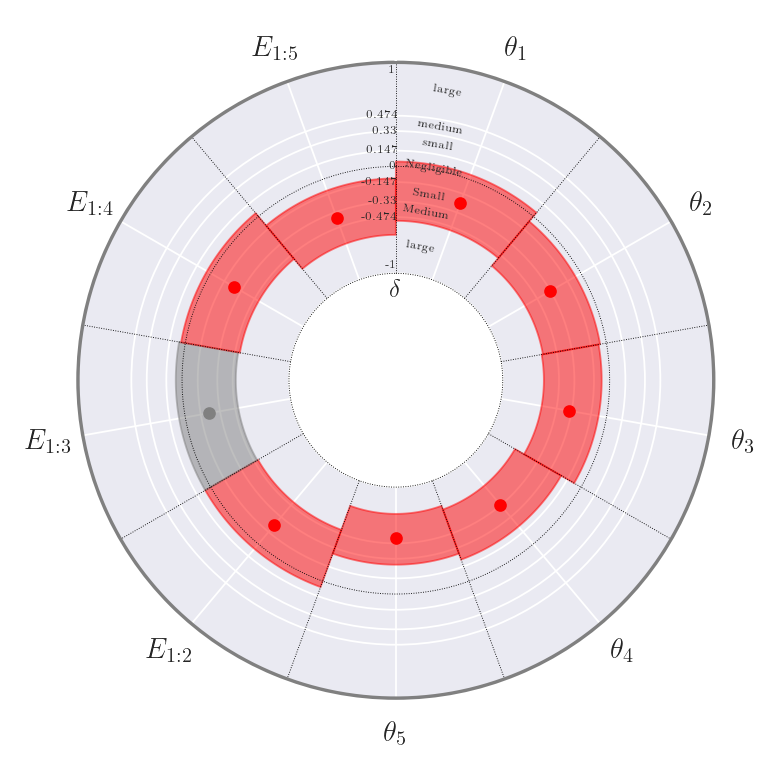}\label{fig:xgbcliffpd_F1}}
		\caption[The Cliff's $\delta$ effect size measure and its 95\% confidence intervals for the $F_1$ score obtained from 30 XGBoost runs.]{Effect size analysis of test data $F_1$ scores across 30 XGBoost runs using Cliff's $\delta$. Each point represents the actual value obtained, with segments denoting 95\% confidence intervals based on 10,000 bootstrap resamplings. The outer ring color visualizes the statistical significance: grey illustrates no significant difference (adjusted Friedman's P-value$>0.05$), while color indicates significant differences; blue indicates at least one view and/or ensemble outperforms the benchmark (adjusted Conover's p-value$ < 0.05$, Cliff's $\delta > 0$), and red signifies all views and ensembles underperform relative to the benchmark (adjusted Conover's p-value$ < 0.05$, Cliff's $\delta < 0$). Segment colors show performance difference against the benchmark: grey for no significant difference (adjusted Conover's p-value$  > 0.05$), blue for better performance (Cliff's $\delta > 0$), and red for worse performance (Cliff's $\delta < 0$).}
		\label{fig:xgbcliff_F1}
	\end{figure*}
	
	\begin{table*}[htbp]
		\centering
		\caption[The results of Friedman and Conover tests and Cliff's $\delta$ analysis for the $F_1$ score obtained from 30 XGBoost runs.]{Statistical comparison of $F_1$ score results for testing data obtained from XGBoost runs. W, T, and L denote win, tie, and loss based on adjusted Friedman and Conover's p-values. Effect sizes are calculated using Cliff's Delta method and are categorized as negligible, small, medium, or large.}
		\label{tab:xgbf1}
		\resizebox{\linewidth}{!}{%
			\begin{tabular}{c|ccccccccc}
				\hline
				\multicolumn{10}{c}{XGBoost's $F_1$ Score}\\
				\hline
				Dataset & $\theta_1$ & $\theta_2$ & $\theta_3$ & $\theta_4$ & $\theta_5$ & $E_{1:2}$ & $E_{1:3}$ & $E_{1:4}$ & $E_{1:5}$ \\
				\hline
				APSF  & L (large) & L (large) & L (large) & L (large) & L (large) & L (large) & L (large) & L (large) & L (large) \\
				ARWPM  & T (small) & L (medium) & L (small) & L (medium) & L (large) & T (negligible) & W (medium) & W (medium) & W (large) \\
				GECR  & L (small) & L (small) & T (small) & L (small) & T (negligible) & L (small) & T (negligible) & T (negligible) & T (negligible) \\
				GFE  & L (large) & L (large) & T (small) & T (negligible) & T (negligible) & L (large) & L (large) & T (small) & T (negligible) \\
				GSAD  & L (large) & L (small) & T (negligible) & T (small) & L (large) & L (small) & T (negligible) & T (small) & W (medium) \\
				HAPT  & L (large) & L (large) & L (large) & T (large) & L (large) & L (large) & L (large) & L (large) & T (small) \\
				ISOLET  & L (large) & L (large) & L (large) & L (large) & L (large) & L (large) & L (large) & T (medium) & T (negligible) \\
				PD  & L (small) & L (medium) & L (medium) & L (large) & L (large) & L (small) & T (small) & L (small) & L (medium) \\
				\hline
				W - T - L  & 0 - 1 - 7 & 0 - 0 - 8 & 0 - 3 - 5 & 0 - 3 - 5 & 0 - 2 - 6 & 0 - 1 - 7 & 1 - 3 - 4 & 1 - 4 - 3 & 2 - 4 - 2 \\
				\hline
			\end{tabular}
		}
	\end{table*}
	\FloatBarrier
	
	\begin{figure*}[t] 
		\centering
		\subfloat[APSF]{\includegraphics[width=0.24\textwidth]{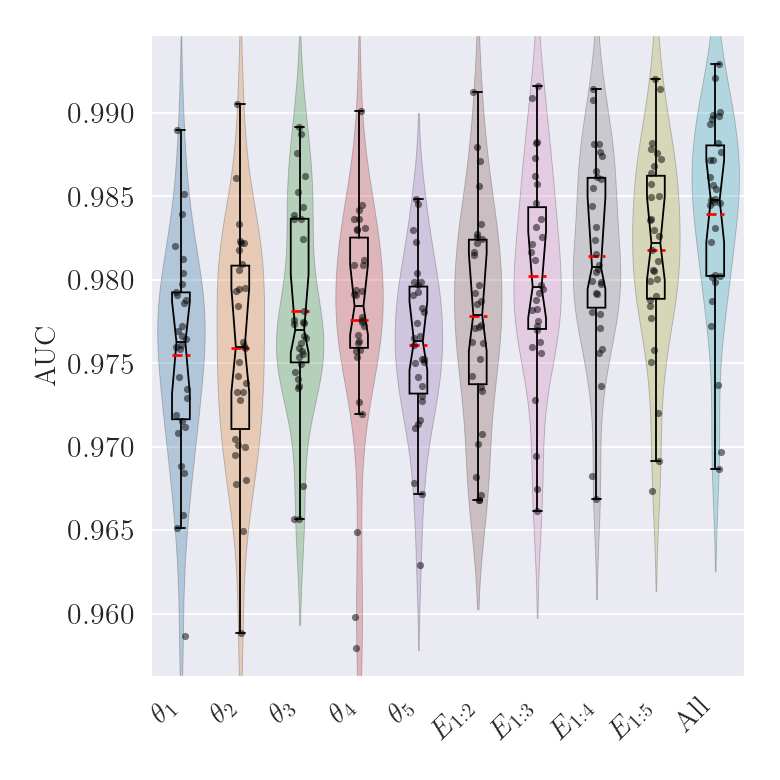}\label{fig:xgbapsf_AUC}}%
		\hfill
		\subfloat[ARWPM]{\includegraphics[width=0.24\textwidth]{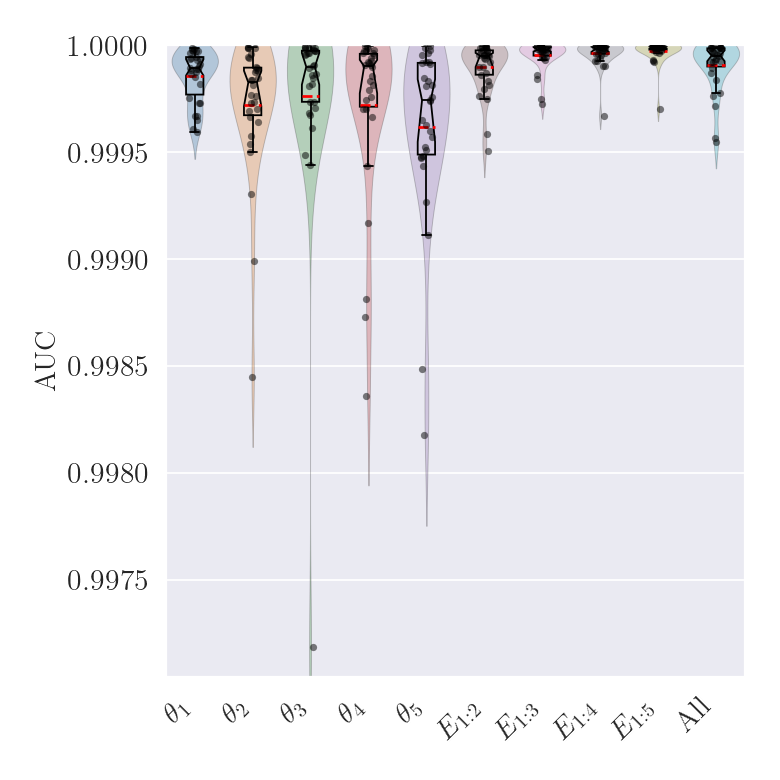}\label{fig:xgbarwpm_AUC}}%
		\hfill
		\subfloat[GECR]{\includegraphics[width=0.24\textwidth]{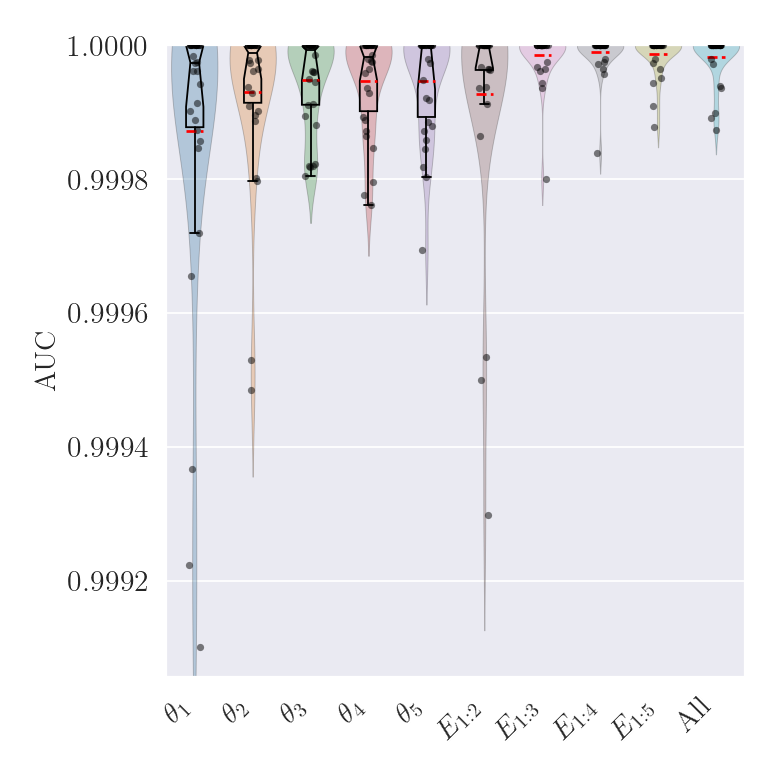}\label{fig:xgbgecr_AUC}}%
		\hfill
		\subfloat[GFE]{\includegraphics[width=0.24\textwidth]{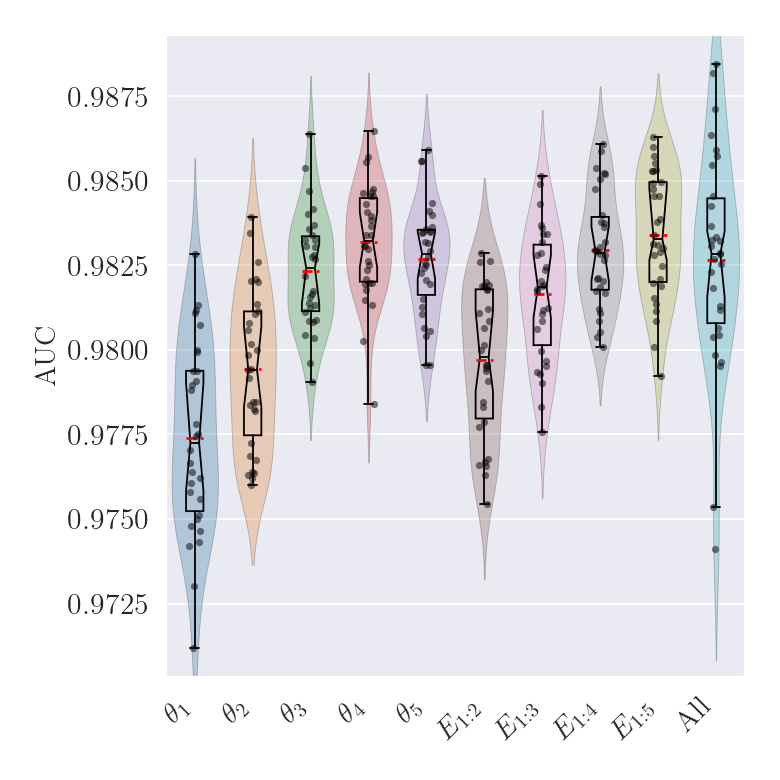}\label{fig:xgbgfe_AUC}}
		
		\subfloat[GSAD]{\includegraphics[width=0.24\textwidth]{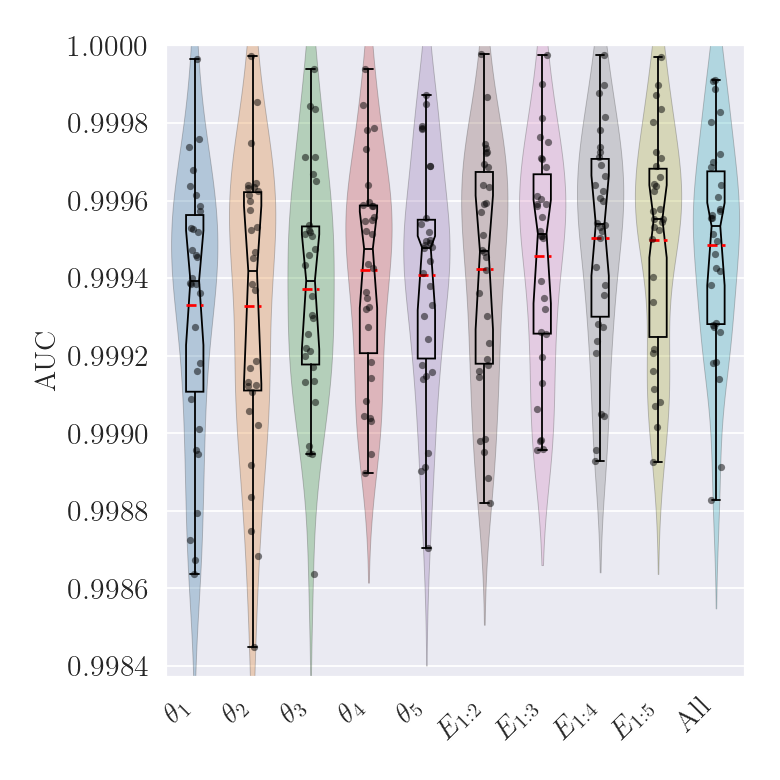}\label{fig:fpgsad_AUC}}%
		\hfill
		\subfloat[HAPT]{\includegraphics[width=0.24\textwidth]{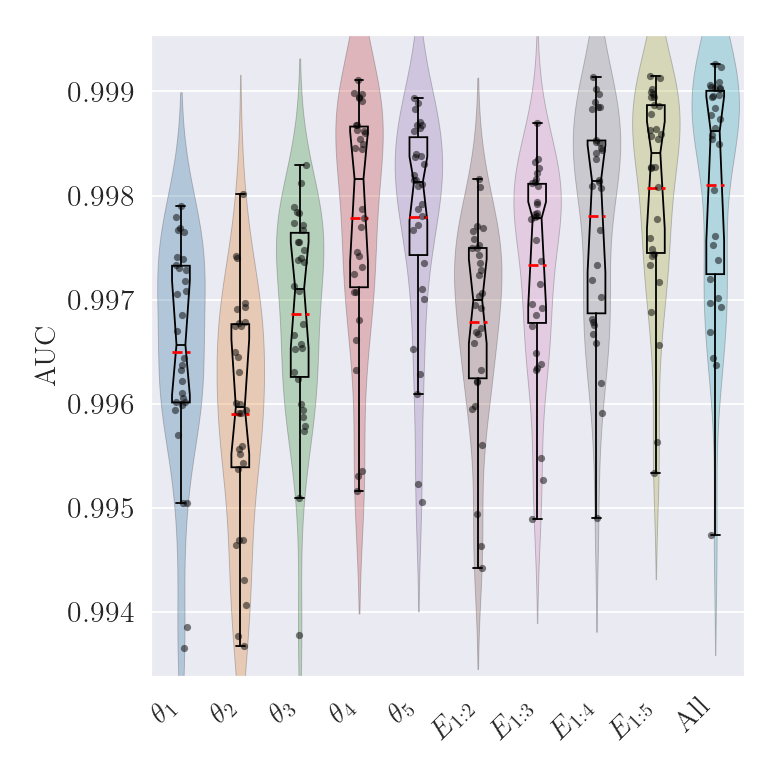}\label{fig:xgbhapt_AUC}}%
		\hfill
		\subfloat[ISOLET]{\includegraphics[width=0.24\textwidth]{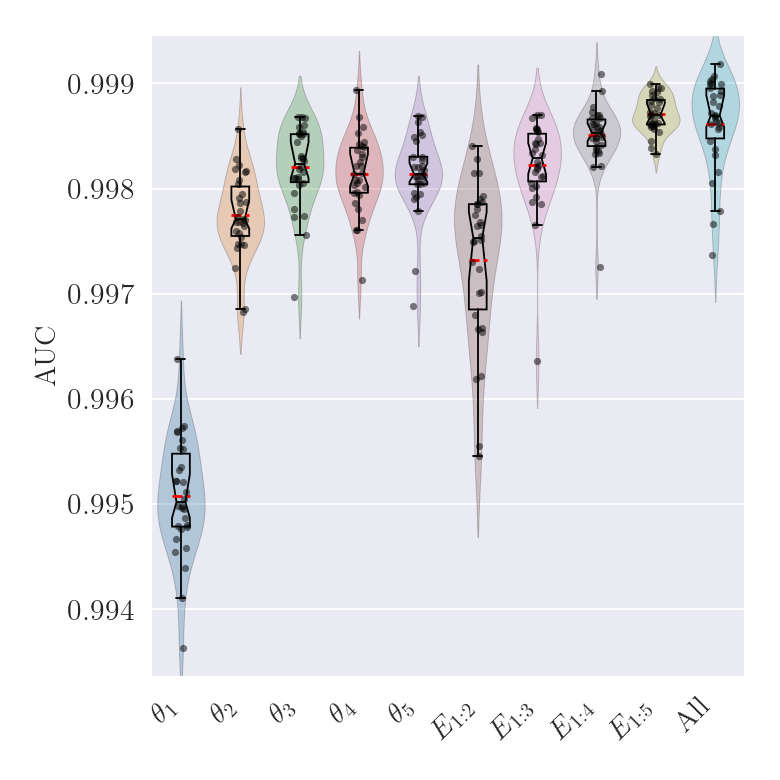}\label{fig:xgbisolet_AUC}}%
		\hfill
		\subfloat[PD]{\includegraphics[width=0.24\textwidth]{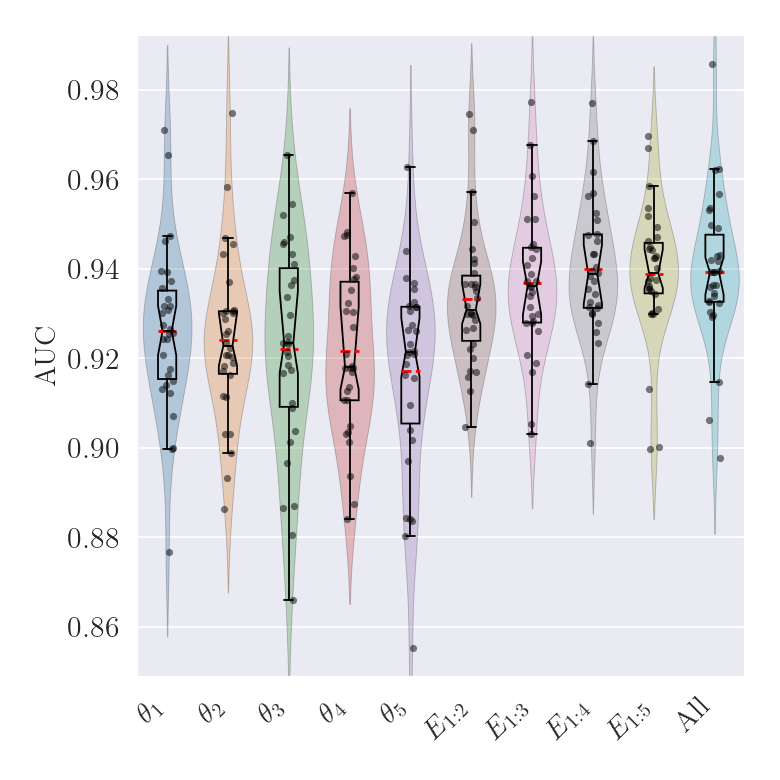}\label{fig:xgbpd_AUC}}
		\caption[The distribution of the obtained AUC values for 30 XGBoost runs.]{The raincloud plot of AUC results obtained from 30 XGBoost runs.}
		
		\label{fig:xgb_AUC}
	\end{figure*}
	
	\begin{figure*}[t] 
		\centering
		\subfloat[APSF]{\includegraphics[width=0.24\textwidth]{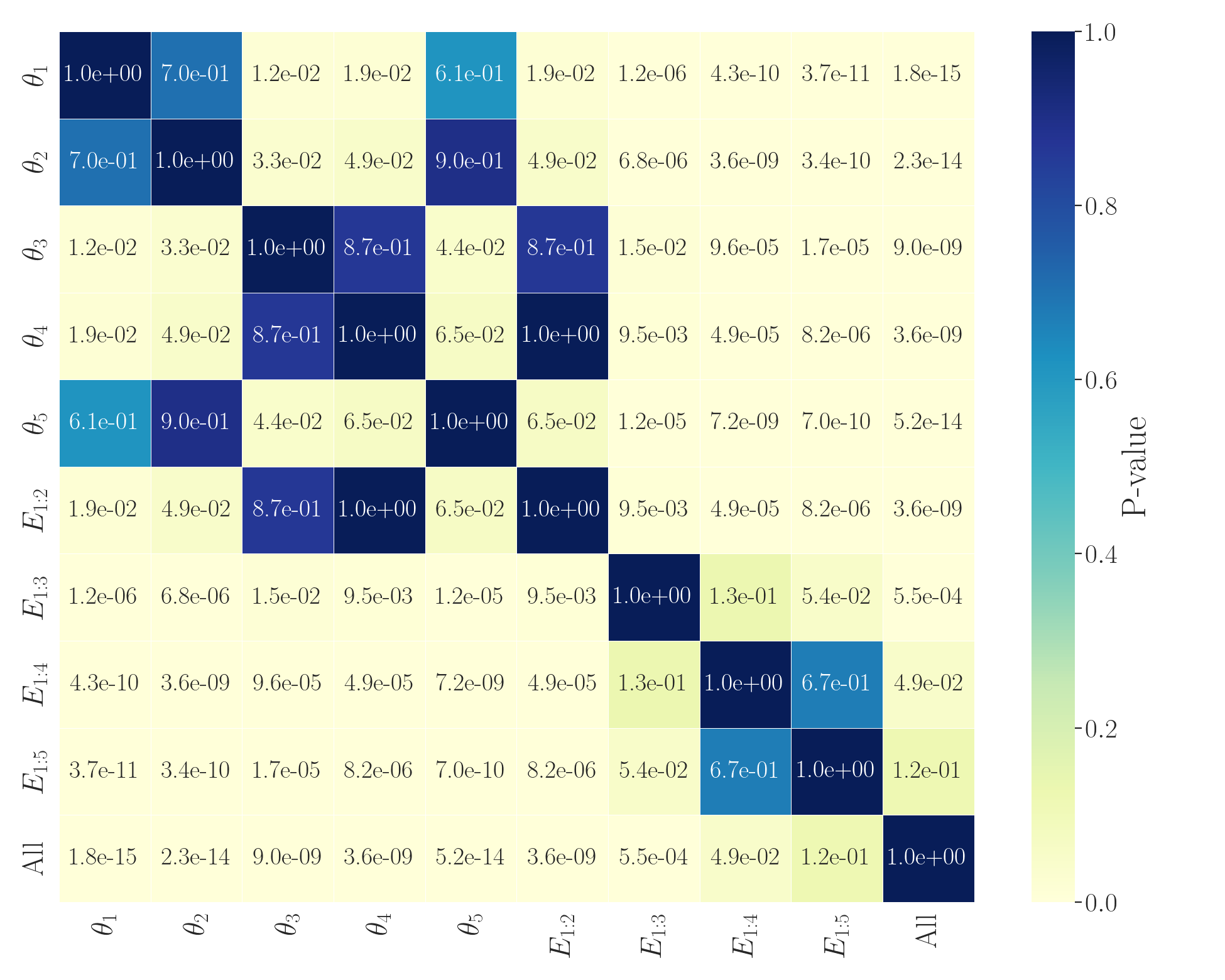}\label{fig:xgbnemapsf_AUC}}%
		\hfill
		\subfloat[ARWPM]{\includegraphics[width=0.24\textwidth]{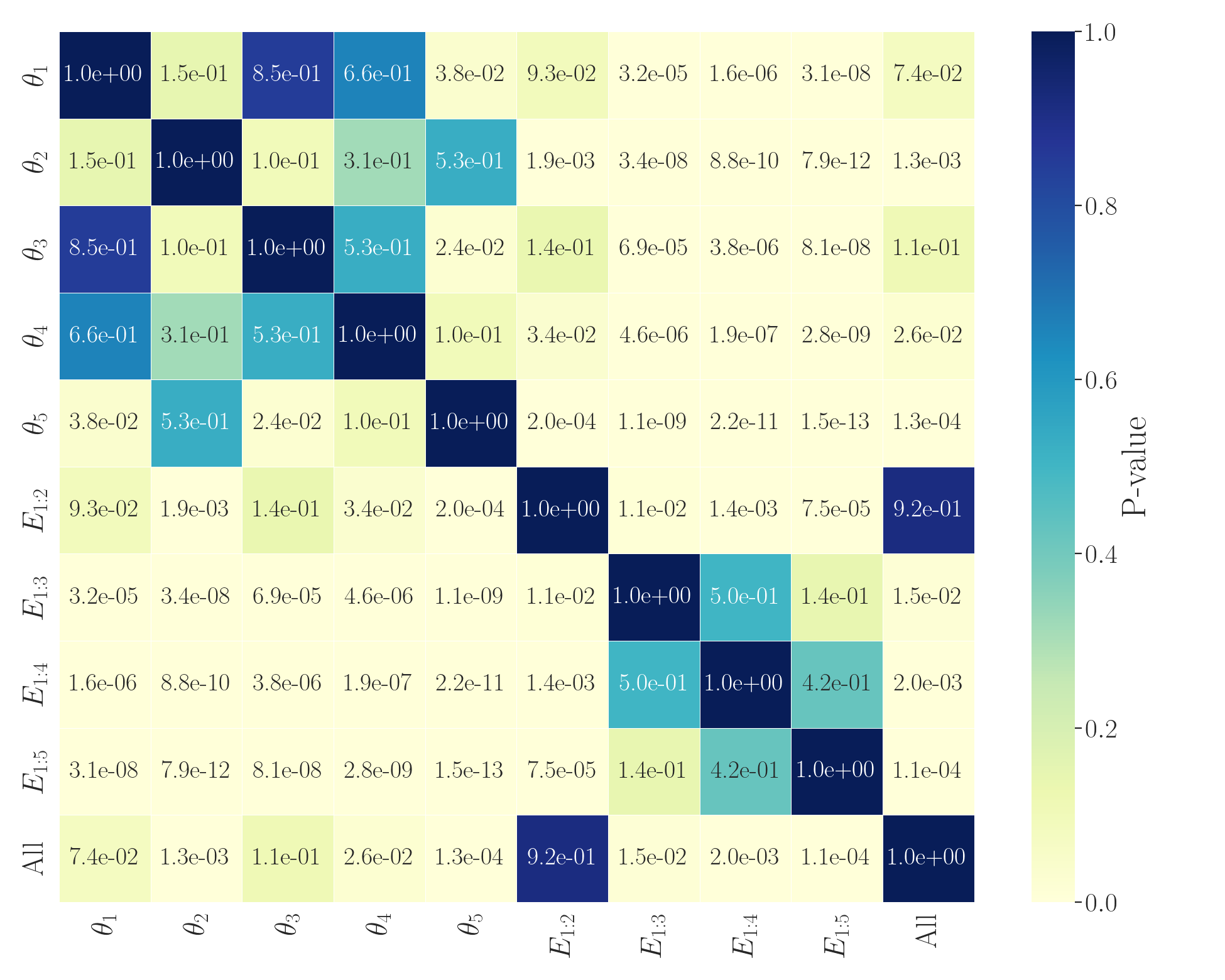}\label{fig:xgbnemarwpm_AUC}}%
		\hfill
		\subfloat[GECR]{\includegraphics[width=0.24\textwidth]{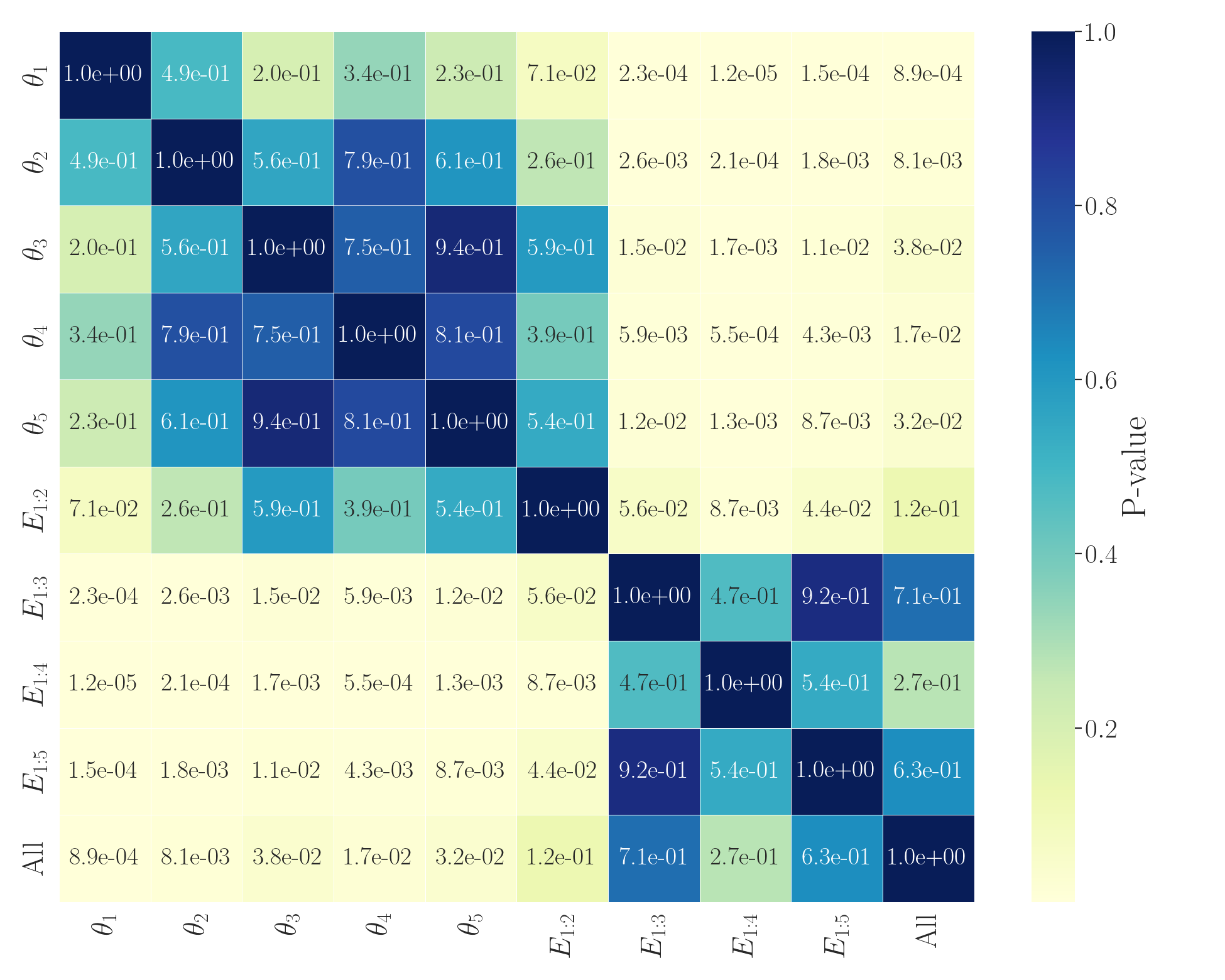}\label{fig:xgbnemgecr_AUC}}%
		\hfill
		\subfloat[GFE]{\includegraphics[width=0.24\textwidth]{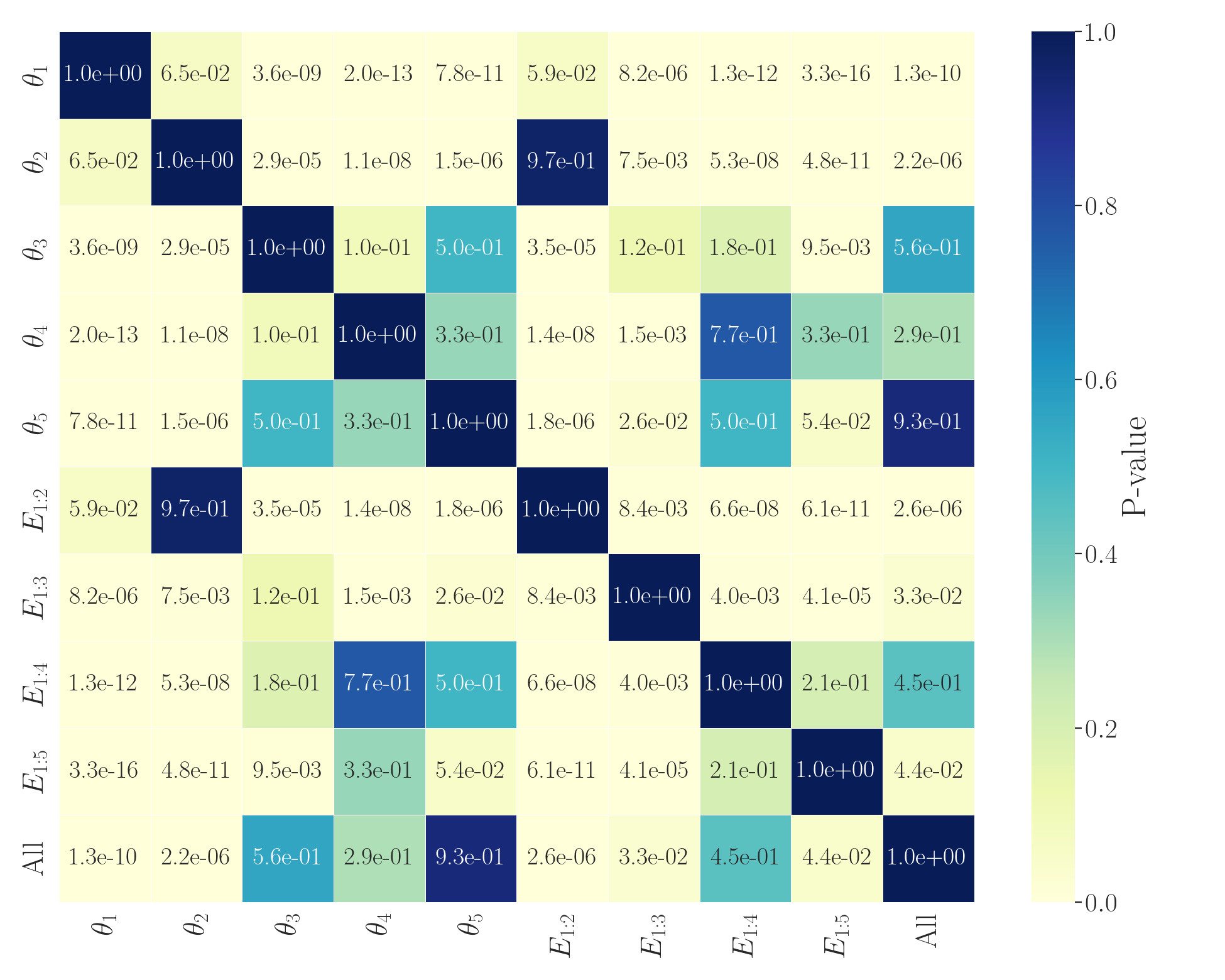}\label{fig:xgbnemgfe_AUC}}
		
		\subfloat[GSAD]{\includegraphics[width=0.24\textwidth]{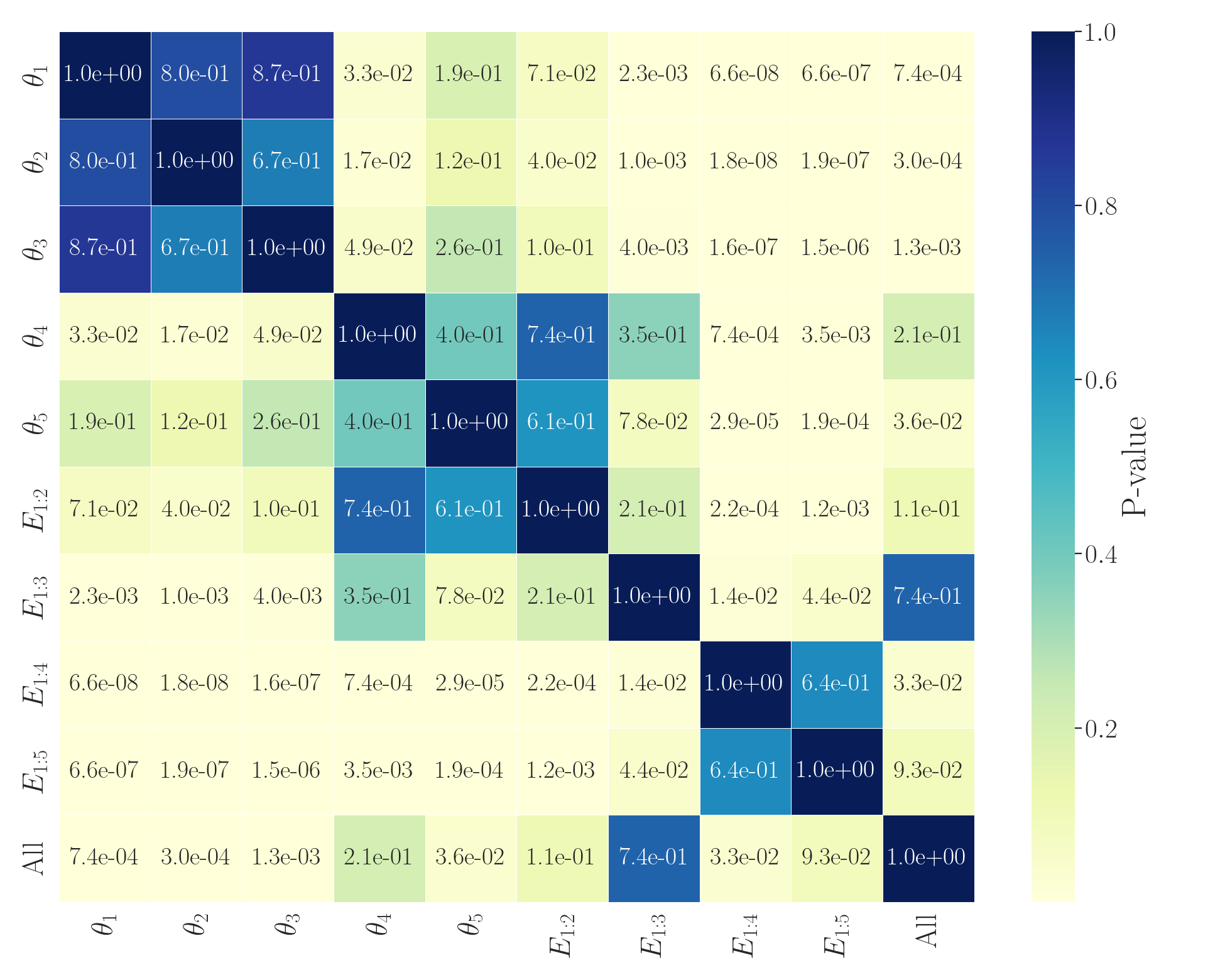}\label{fig:xgbnemgsad_AUC}}%
		\hfill
		\subfloat[HAPT]{\includegraphics[width=0.24\textwidth]{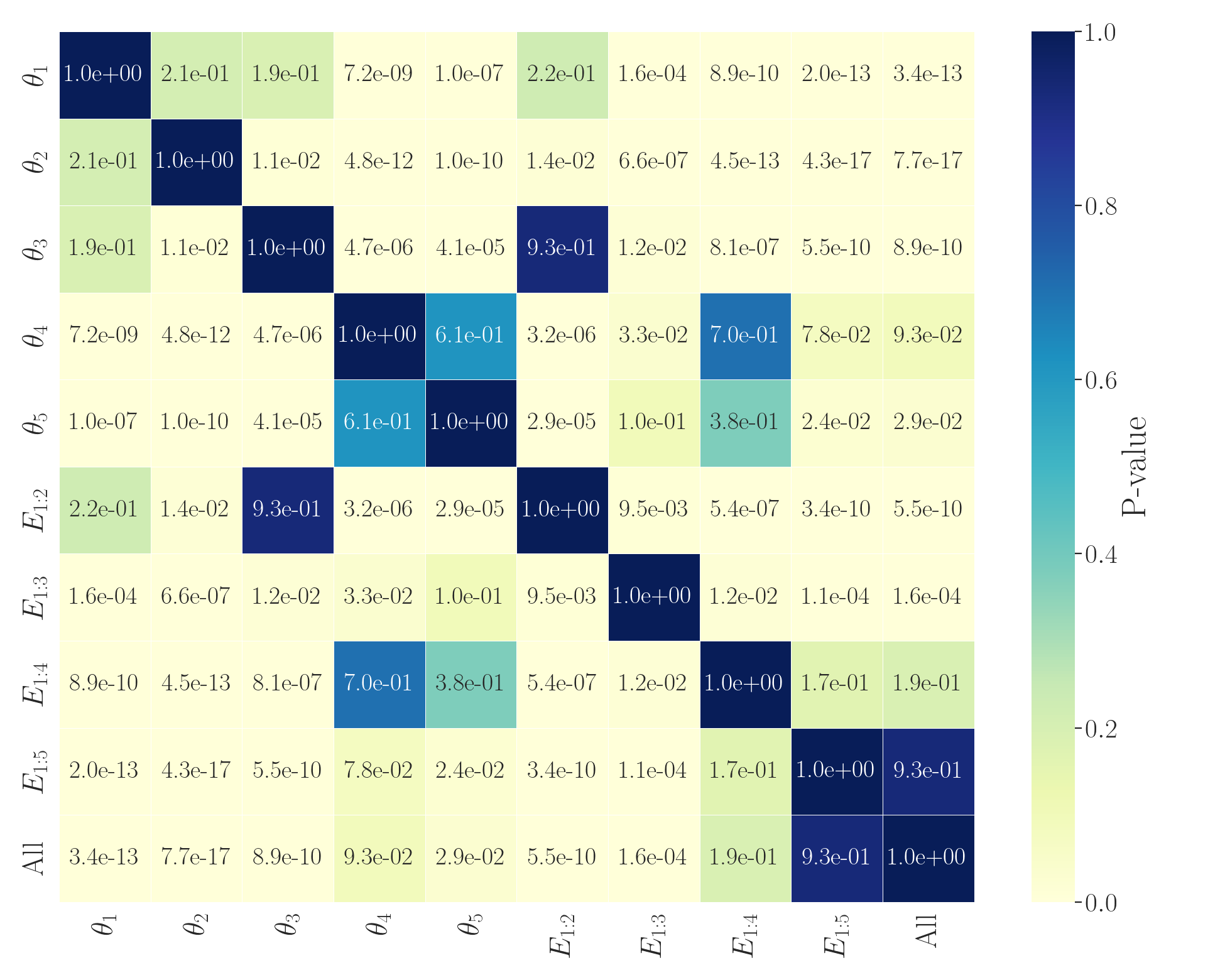}\label{fig:xgbnemhapt_AUC}}%
		\hfill
		\subfloat[ISOLET]{\includegraphics[width=0.24\textwidth]{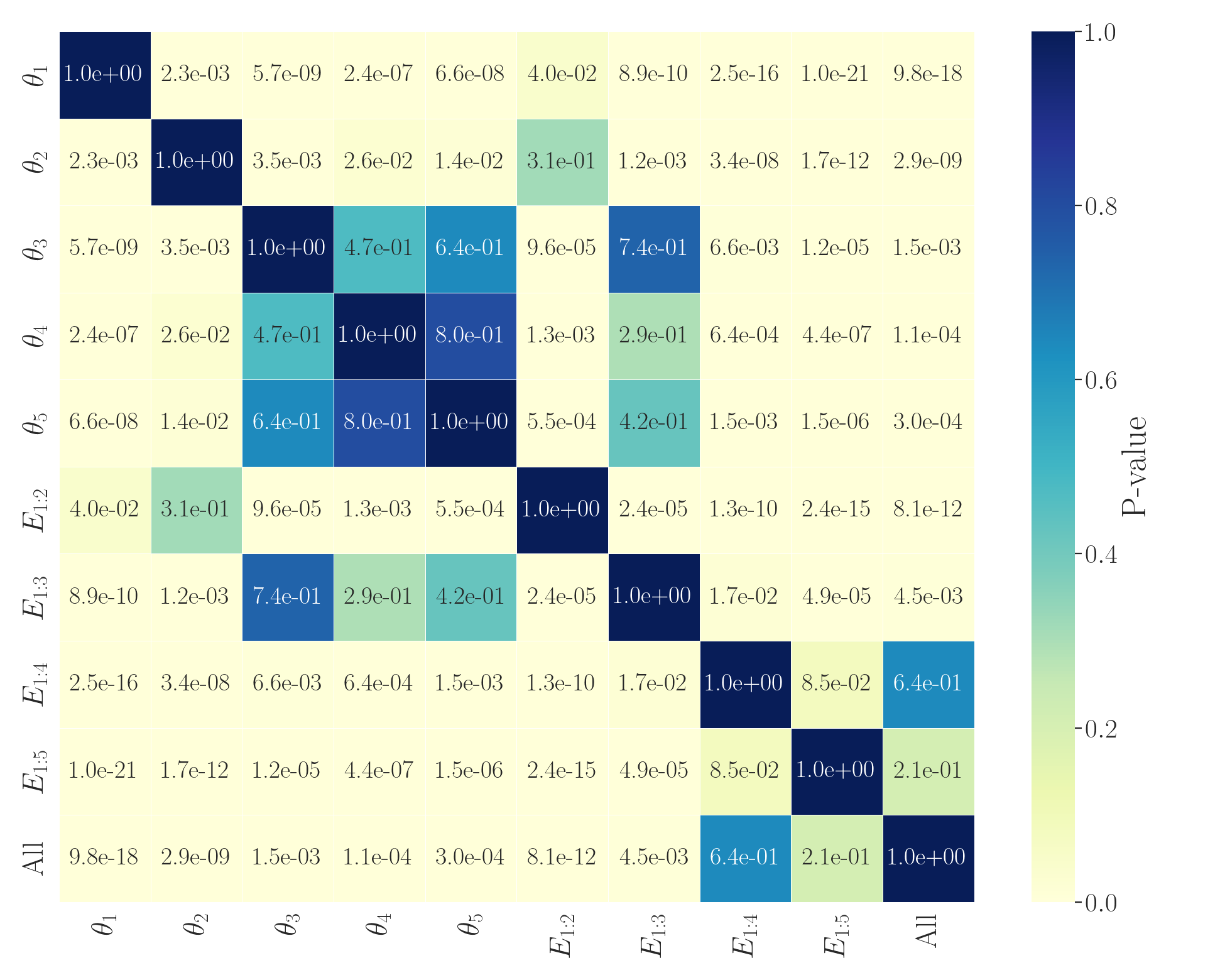}\label{fig:xgbnemisolet_AUC}}%
		\hfill
		\subfloat[PD]{\includegraphics[width=0.24\textwidth]{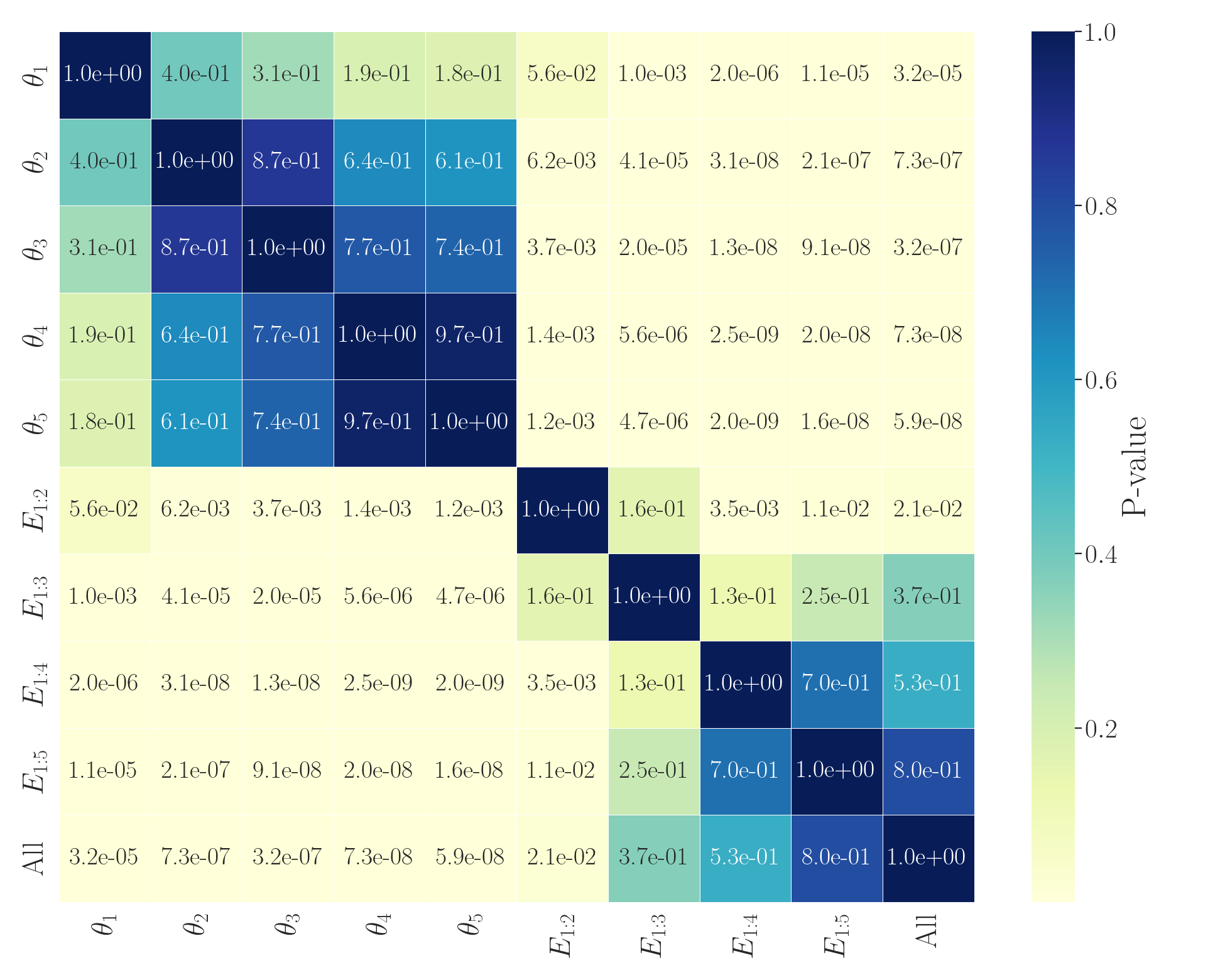}\label{fig:xgbnempd_AUC}}
		\caption[The adjusted Conover's P-values for the obtained AUC values from 30 XGBoost runs.]{The results of the Conover post-hoc test on testing data’s AUC obtained from 30 XGBoost runs.}
		
		\label{fig:xgbnem_AUC}
	\end{figure*}
	\FloatBarrier

	\begin{figure*}[htbp] 
		\centering
		\subfloat[APSF]{\includegraphics[width=0.24\textwidth]{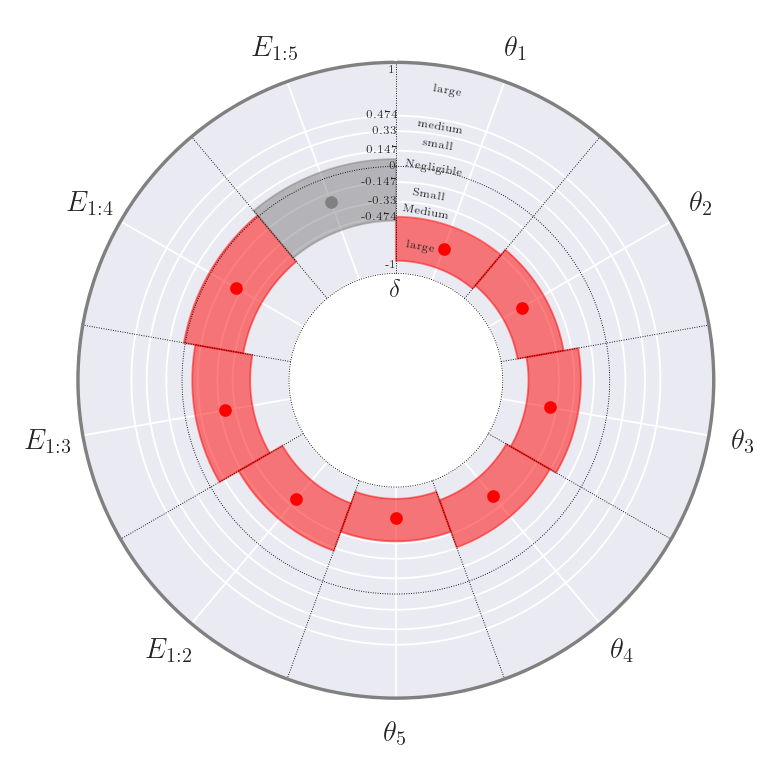}\label{fig:xgbcliffapsf_AUC}}%
		\hfill
		\subfloat[ARWPM]{\includegraphics[width=0.24\textwidth]{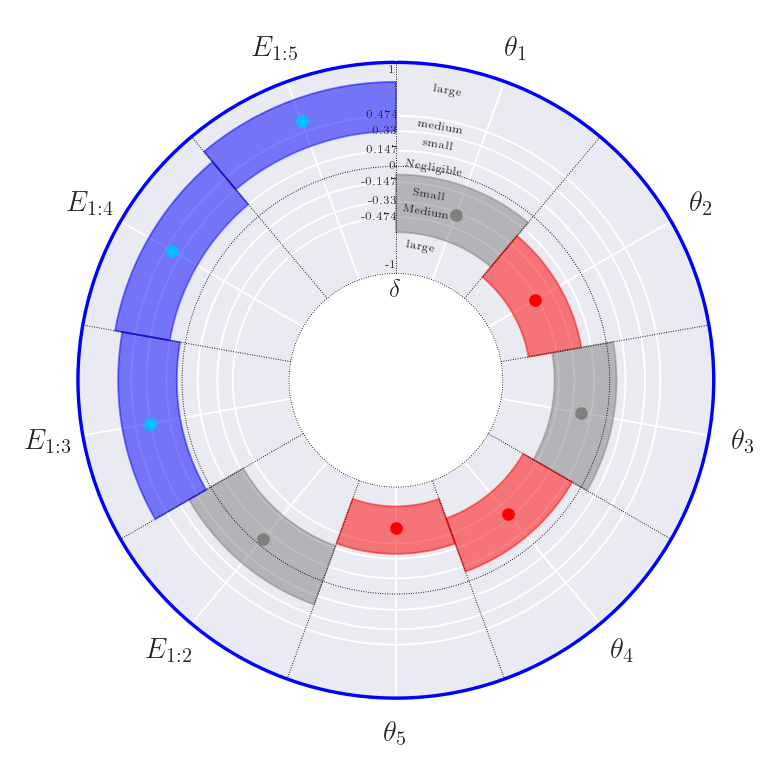}\label{fig:xgbcliffarwpm_AUC}}%
		\hfill
		\subfloat[GECR]{\includegraphics[width=0.24\textwidth]{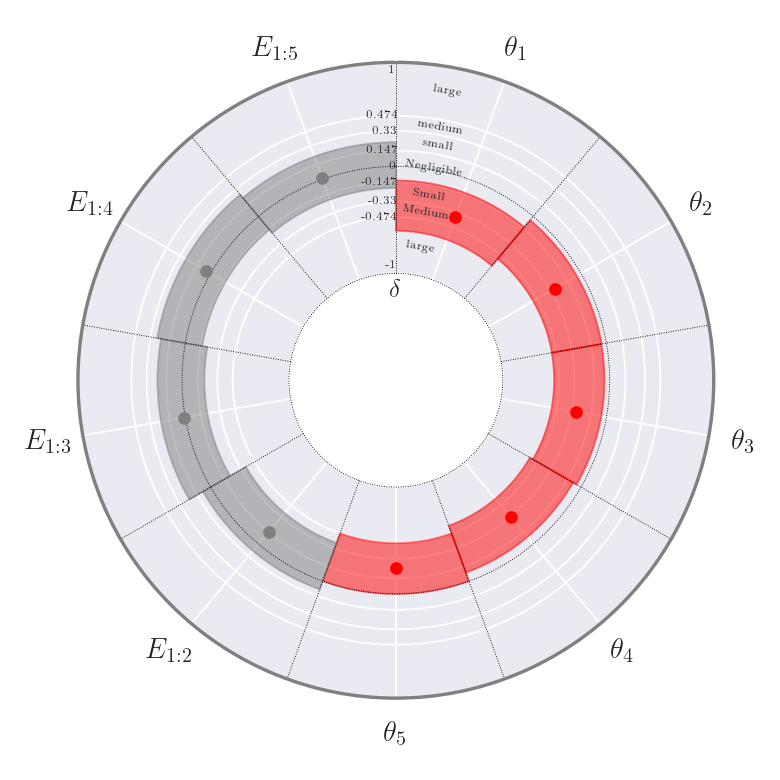}\label{fig:xgbcliffgecr_AUC}}%
		\hfill
		\subfloat[GFE]{\includegraphics[width=0.24\textwidth]{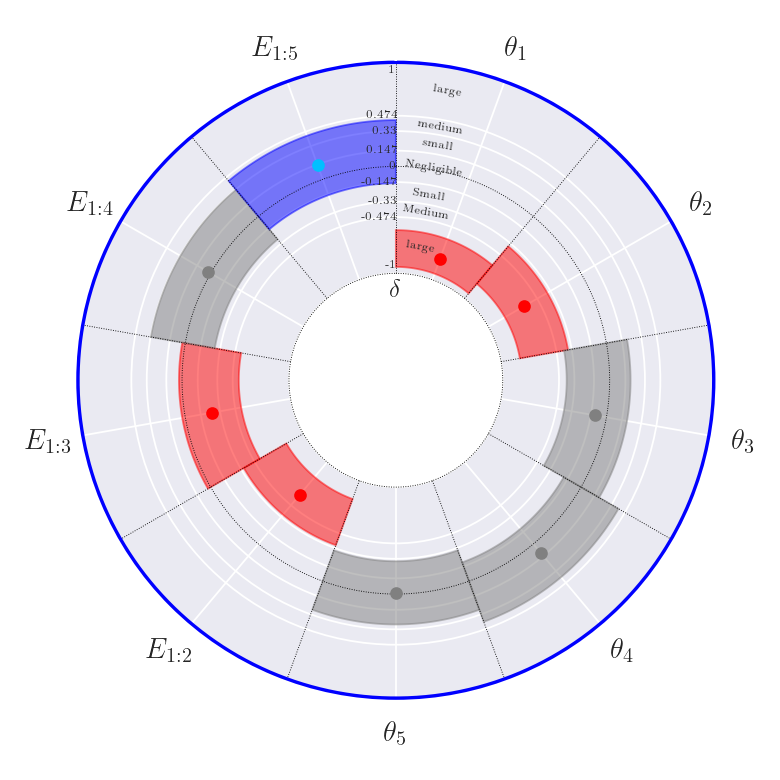}\label{fig:xgbcliffgfe_AUC}}
		
		\subfloat[GSAD]{\includegraphics[width=0.24\textwidth]{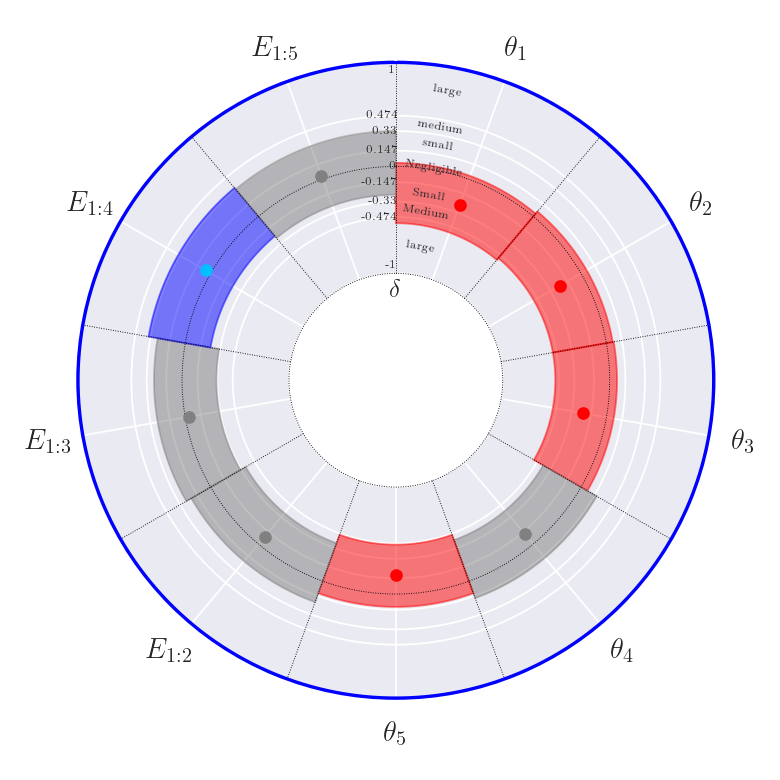}\label{fig:xgbcliffgsad_AUC}}%
		\hfill
		\subfloat[HAPT]{\includegraphics[width=0.24\textwidth]{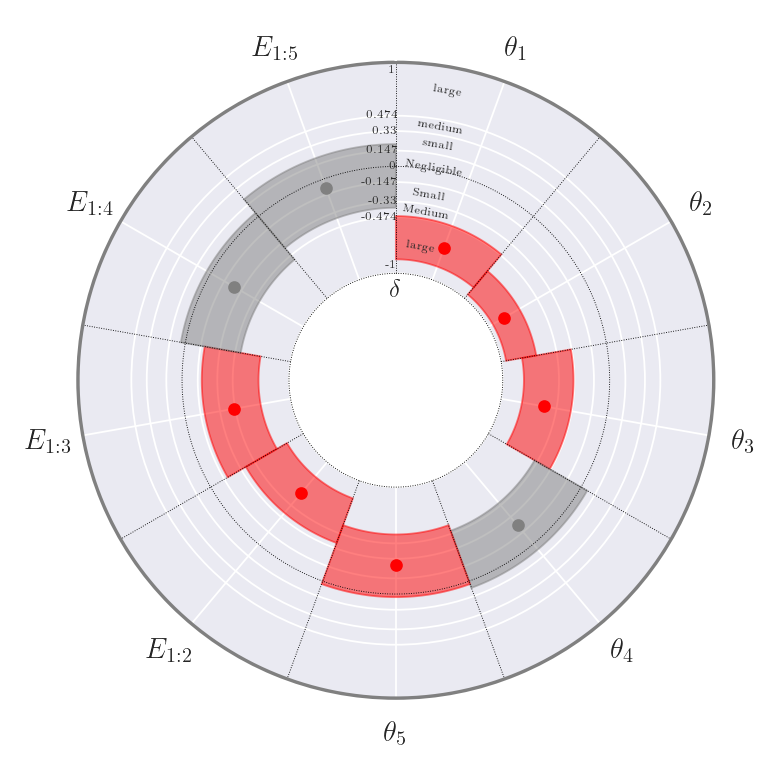}\label{fig:xgbcliffhapt_AUC}}%
		\hfill
		\subfloat[ISOLET]{\includegraphics[width=0.24\textwidth]{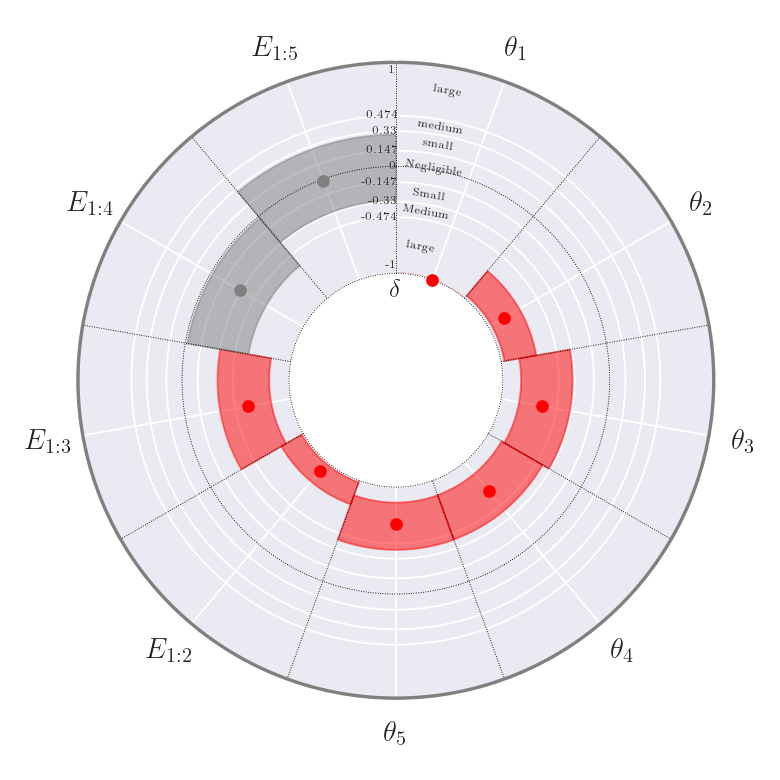}\label{fig:xgbcliffisolet_AUC}}%
		\hfill
		\subfloat[PD]{\includegraphics[width=0.24\textwidth]{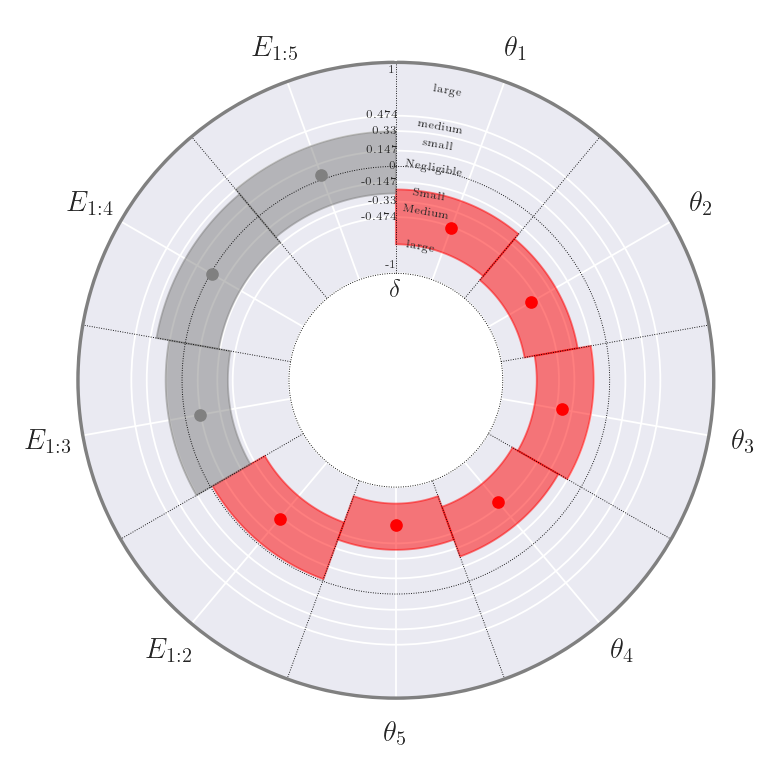}\label{fig:xgbcliffpd_AUC}}
		\caption[The Cliff's $\delta$ effect size measure and its 95\% confidence intervals for the AUC values obtained from 30 XGBoost runs.]{Effect size analysis of test data AUC across 30 XGBoost runs using Cliff's $\delta$. Each point represents the actual value obtained, with segments denoting 95\% confidence intervals based on 10,000 bootstrap resamplings. The outer ring color visualizes the statistical significance: grey illustrates no significant difference (adjusted Friedman's P-value$>0.05$), while color indicates significant differences; blue indicates at least one view and/or ensemble outperforms the benchmark (adjusted Conover's p-value$ < 0.05$, Cliff's $\delta > 0$), and red signifies all views and ensembles underperform relative to the benchmark (adjusted Conover's p-value$ < 0.05$, Cliff's $\delta < 0$). Segment colors show performance difference against the benchmark: grey for no significant difference (adjusted Conover's p-value$  > 0.05$), blue for better performance (Cliff's $\delta > 0$), and red for worse performance (Cliff's $\delta < 0$).}
		
		\label{fig:xgbcliff_AUC}
	\end{figure*}
	
	\begin{table*}[htbp]
		\centering
		\caption[The results of Friedman and Conover tests and Cliff's $\delta$ analysis for the AUC values obtained from 30 XGBoost runs.]{Statistical comparison of AUC results for testing data obtained from XGBoost runs. W, T, and L denote win, tie, and loss based on adjusted Friedman and Conover's p-values. Effect sizes are calculated using Cliff's Delta method and are categorized as negligible, small, medium, or large.}
		\label{tab:xgbauc}
		\resizebox{\linewidth}{!}{%
			\begin{tabular}{c|ccccccccc}
				\hline
				\multicolumn{10}{c}{XGBoost's AUC}\\
				\hline
				Dataset & $\theta_1$ & $\theta_2$ & $\theta_3$ & $\theta_4$ & $\theta_5$ & $E_{1:2}$ & $E_{1:3}$ & $E_{1:4}$ & $E_{1:5}$ \\
				\hline
				APSF  & L (large) & L (large) & L (large) & L (large) & L (large) & L (large) & L (medium) & L (small) & T (small) \\
				ARWPM  & T (medium) & L (large) & T (small) & L (medium) & L (large) & T (negligible) & W (medium) & W (medium) & W (large) \\
				GECR  & L (medium) & L (small) & L (small) & L (small) & L (small) & T (small) & T (negligible) & T (negligible) & T (negligible) \\
				GFE  & L (large) & L (large) & T (negligible) & T (negligible) & T (negligible) & L (large) & L (small) & T (negligible) & W (negligible) \\
				GSAD  & L (small) & L (small) & L (small) & T (negligible) & L (small) & T (negligible) & T (negligible) & W (negligible) & T (negligible) \\
				HAPT  & L (large) & L (large) & L (large) & T (small) & L (small) & L (large) & L (medium) & T (small) & T (negligible) \\
				ISOLET  & L (large) & L (large) & L (large) & L (large) & L (large) & L (large) & L (large) & T (small) & T (negligible) \\
				PD  & L (large) & L (large) & L (medium) & L (large) & L (large) & L (small) & T (negligible) & T (negligible) & T (negligible) \\
				\hline
				W - T - L  & 0 - 1 - 7 & 0 - 0 - 8 & 0 - 2 - 6 & 0 - 3 - 5 & 0 - 1 - 7 & 0 - 3 - 5 & 1 - 3 - 4 & 2 - 5 - 1 & 2 - 6 - 0 \\
				\hline
			\end{tabular}
		}
	\end{table*}
	\FloatBarrier
	
	\begin{figure*}[t] 
		\centering
		\subfloat[APSF]{\includegraphics[width=0.24\textwidth]{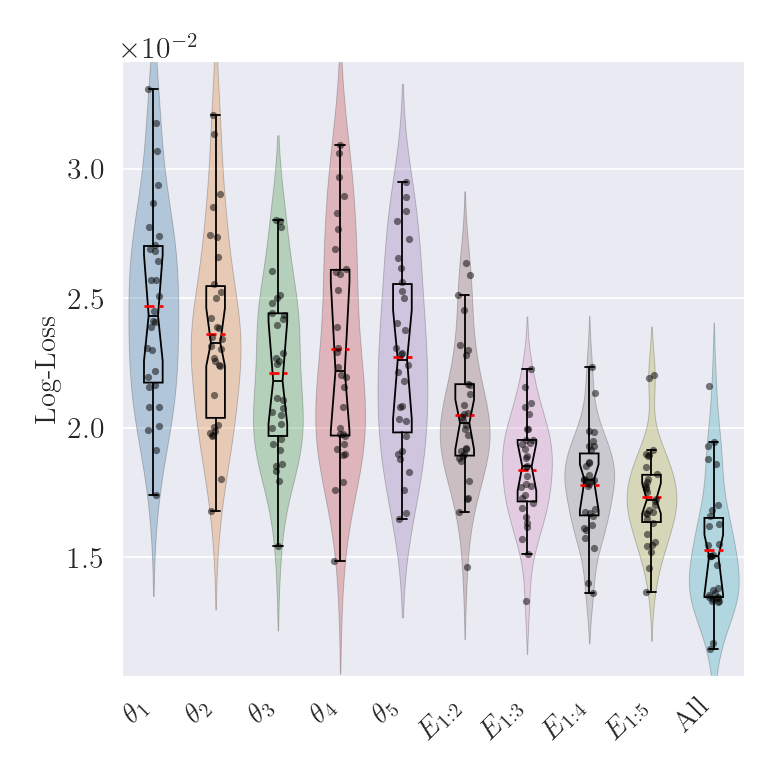}\label{fig:xgbapsf_Loss}}%
		\hfill
		\subfloat[ARWPM]{\includegraphics[width=0.24\textwidth]{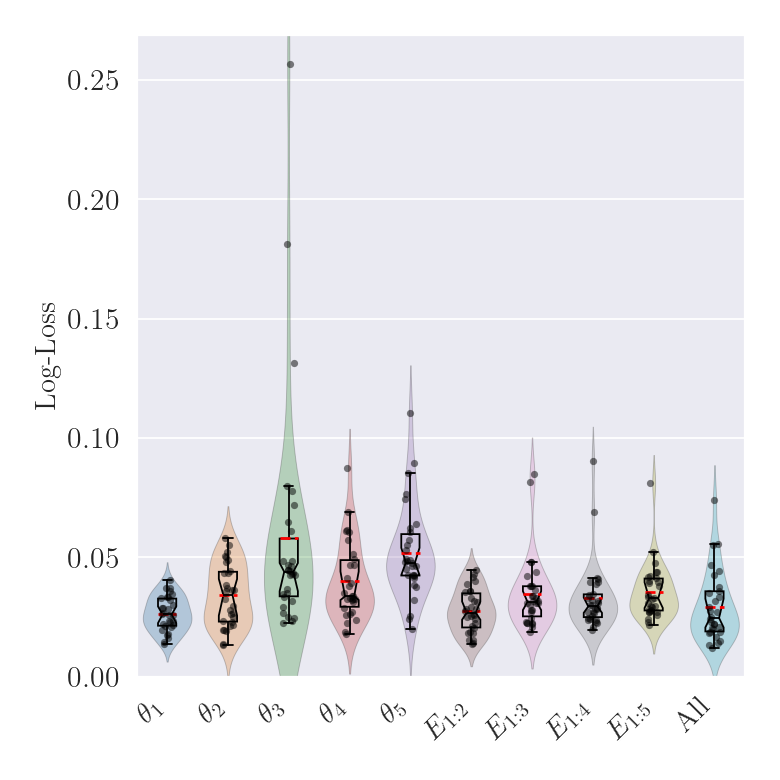}\label{fig:xgbarwpm_Loss}}%
		\hfill
		\subfloat[GECR]{\includegraphics[width=0.24\textwidth]{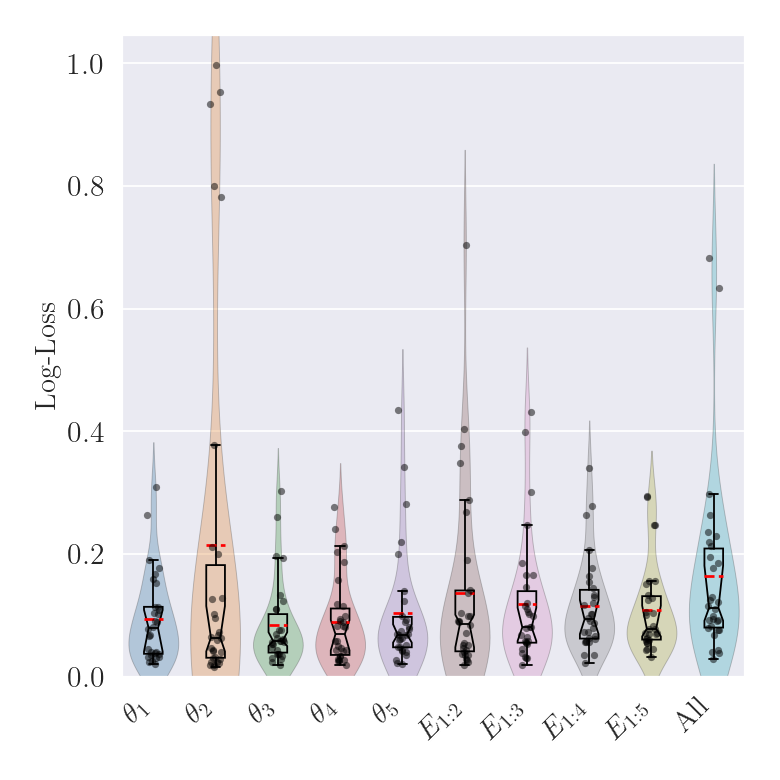}\label{fig:xgbgecr_Loss}}%
		\hfill
		\subfloat[GFE]{\includegraphics[width=0.24\textwidth]{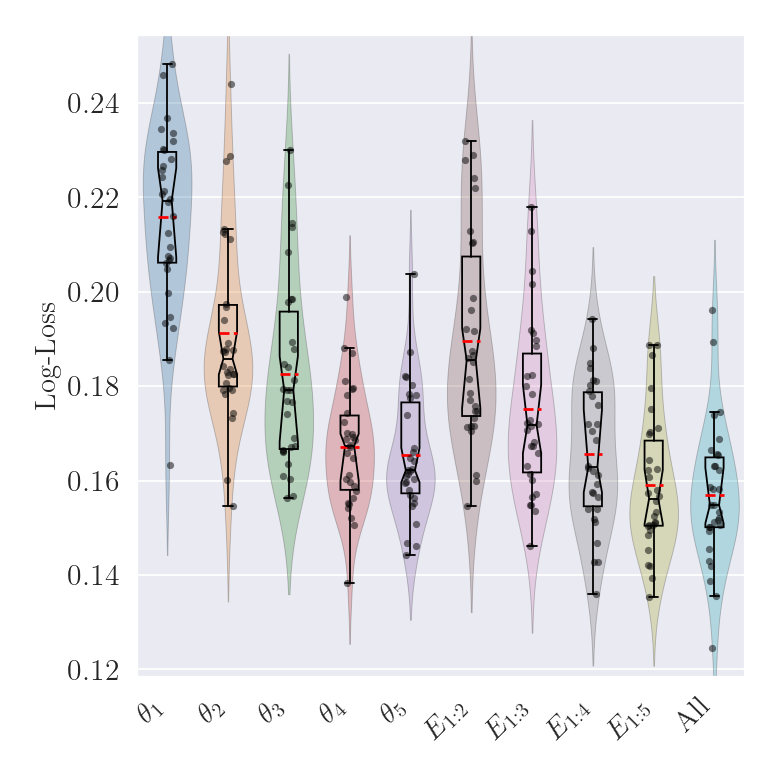}\label{fig:xgbgfe_Loss}}
		
		\subfloat[GSAD]{\includegraphics[width=0.24\textwidth]{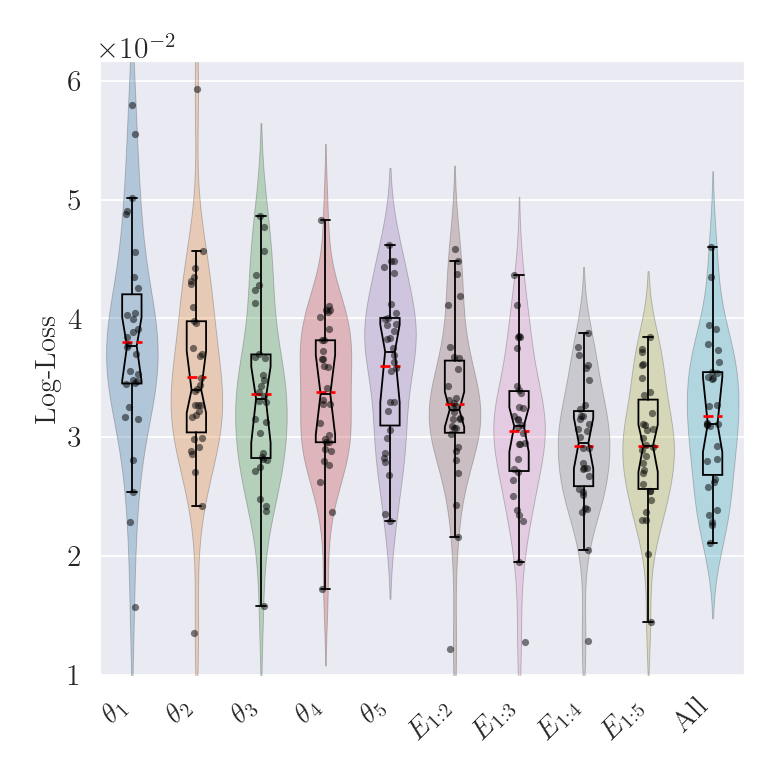}\label{fig:fpgsad_Loss}}%
		\hfill
		\subfloat[HAPT]{\includegraphics[width=0.24\textwidth]{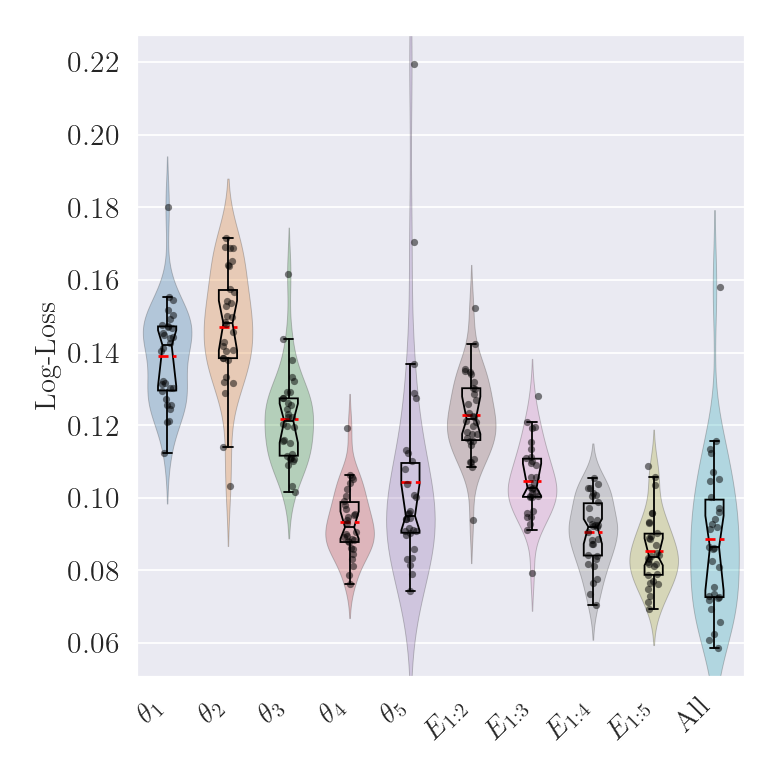}\label{fig:xgbhapt_Loss}}%
		\hfill
		\subfloat[ISOLET]{\includegraphics[width=0.24\textwidth]{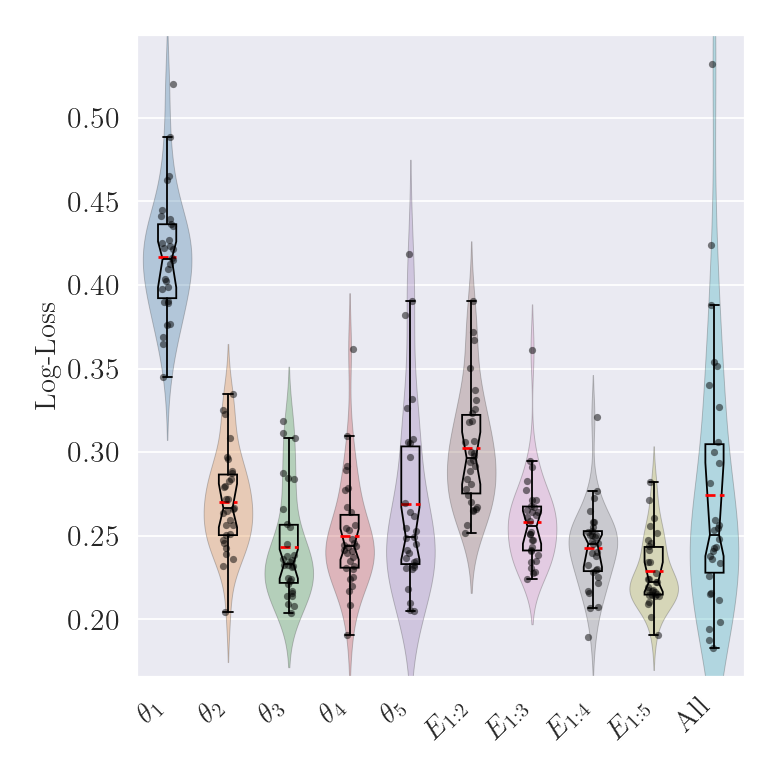}\label{fig:xgbisolet_Loss}}%
		\hfill
		\subfloat[PD]{\includegraphics[width=0.24\textwidth]{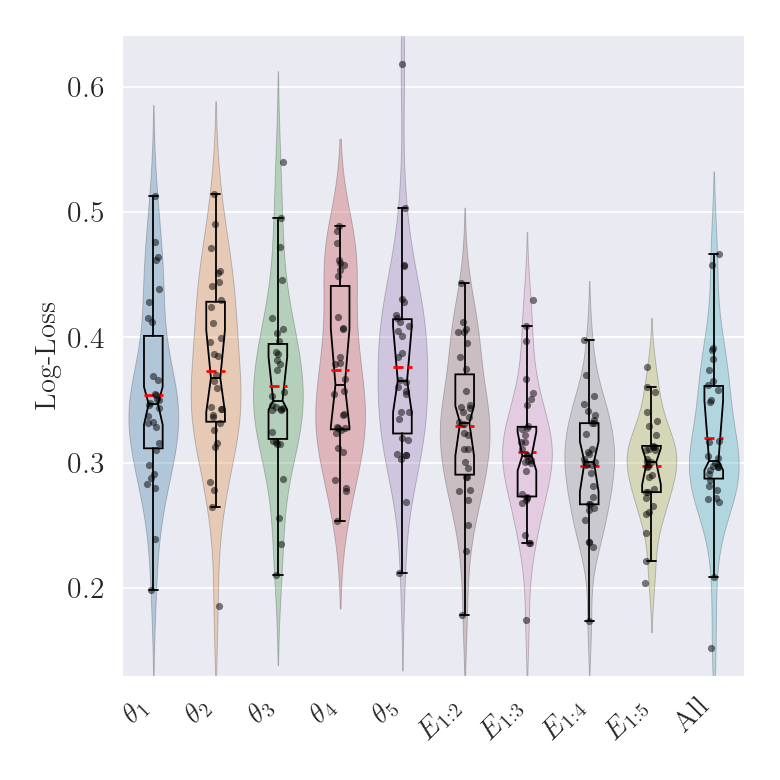}\label{fig:xgbpd_Loss}}
		\caption[The distribution of the obtained Log-Loss values for 30 XGBoost runs.]{The raincloud plot of Log-Loss results obtained from 30 XGBoost runs.}
		
		\label{fig:xgb_Loss}
	\end{figure*}
	
	\begin{figure*}[t] 
		\centering
		\subfloat[APSF]{\includegraphics[width=0.24\textwidth]{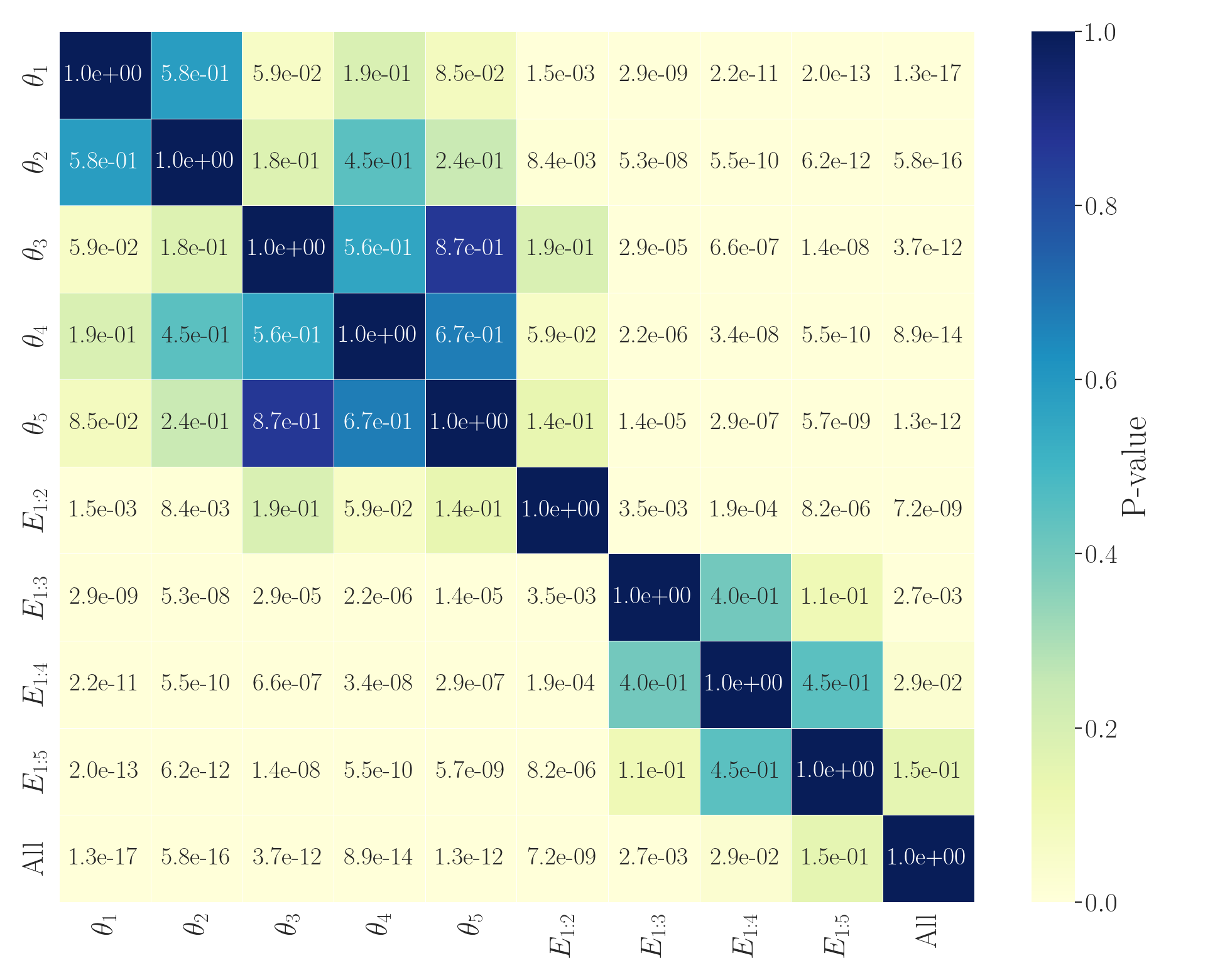}\label{fig:xgbnemapsf_Loss}}%
		\hfill
		\subfloat[ARWPM]{\includegraphics[width=0.24\textwidth]{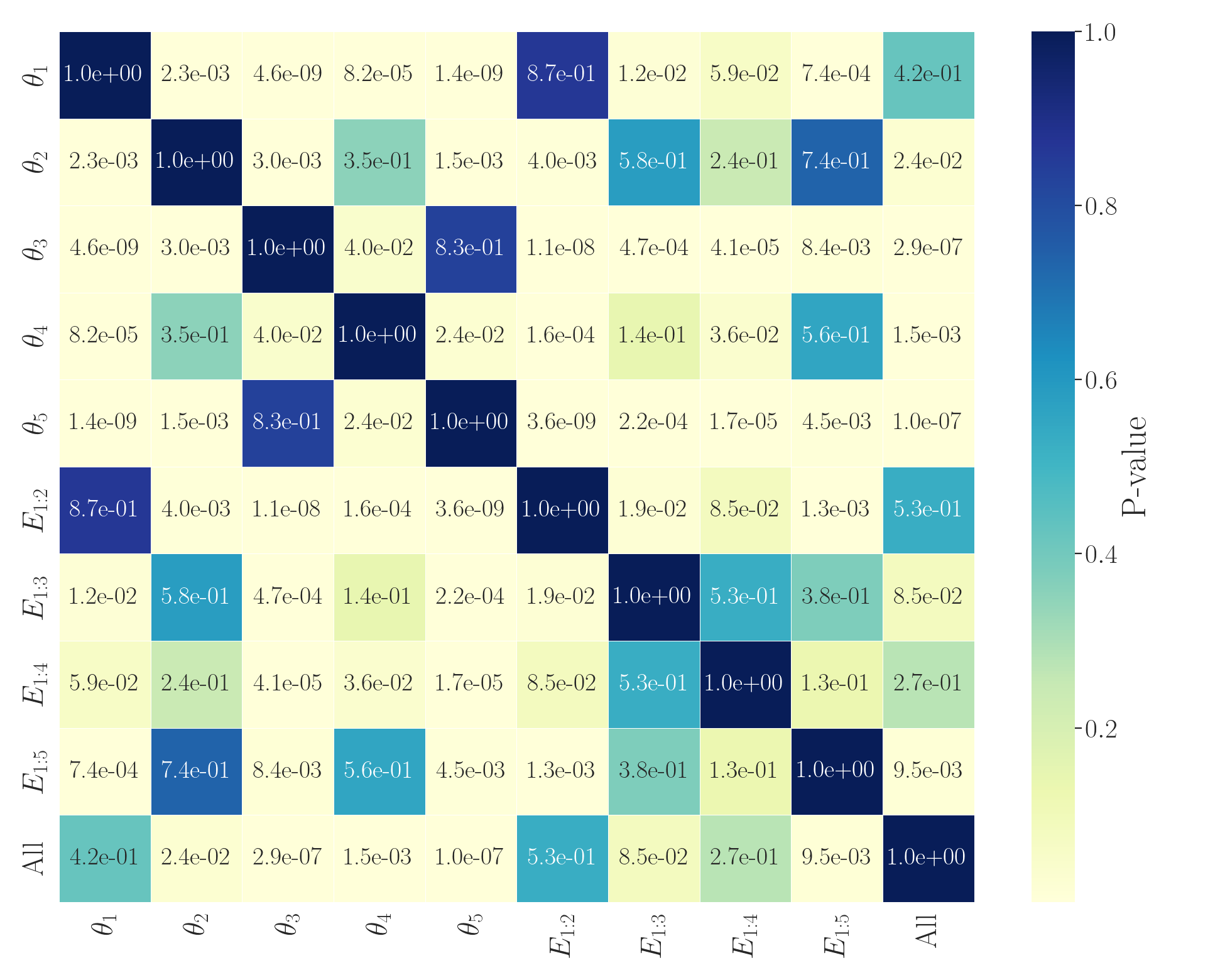}\label{fig:xgbnemarwpm_Loss}}%
		\hfill
		\subfloat[GECR]{\includegraphics[width=0.24\textwidth]{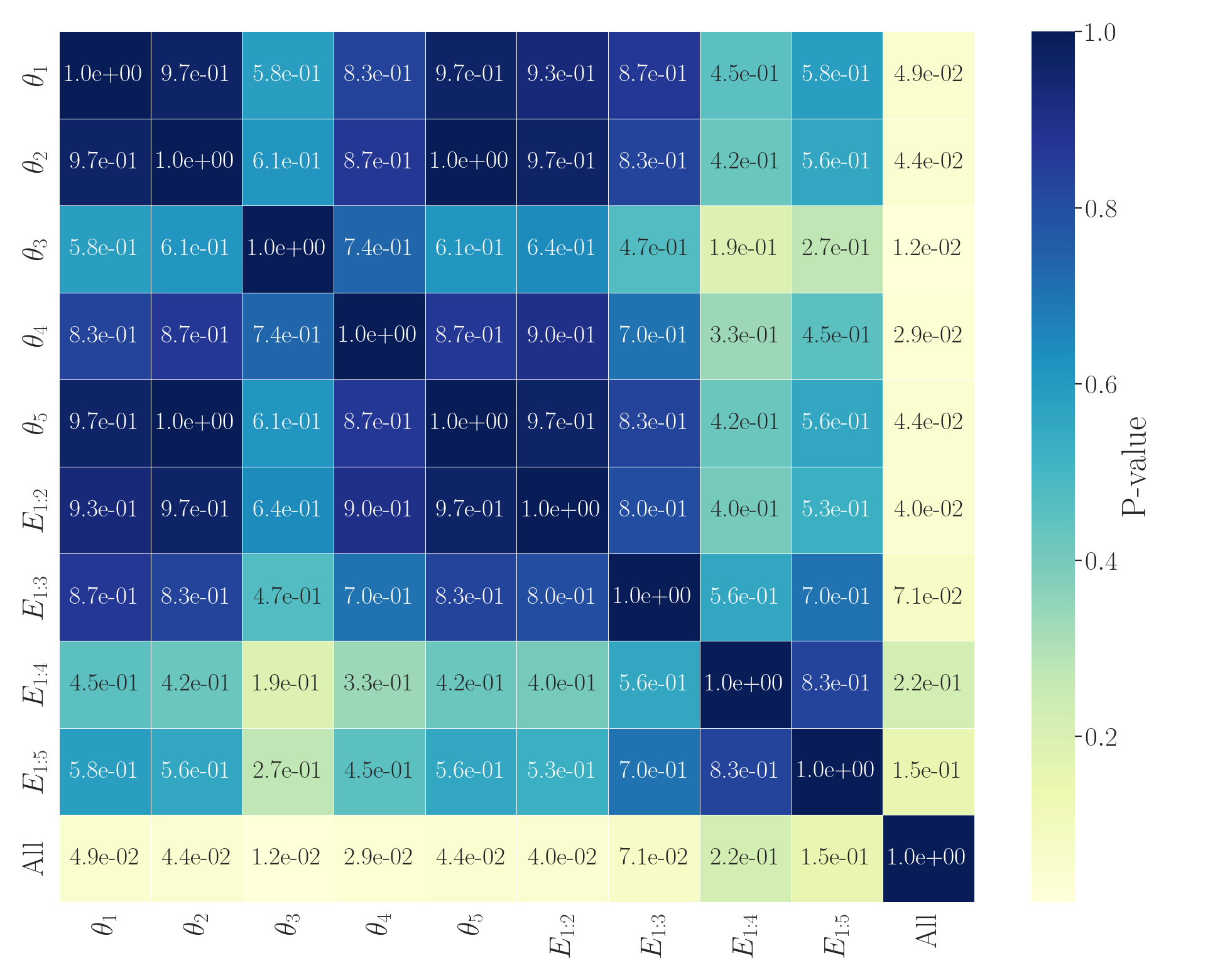}\label{fig:xgbnemgecr_Loss}}%
		\hfill
		\subfloat[GFE]{\includegraphics[width=0.24\textwidth]{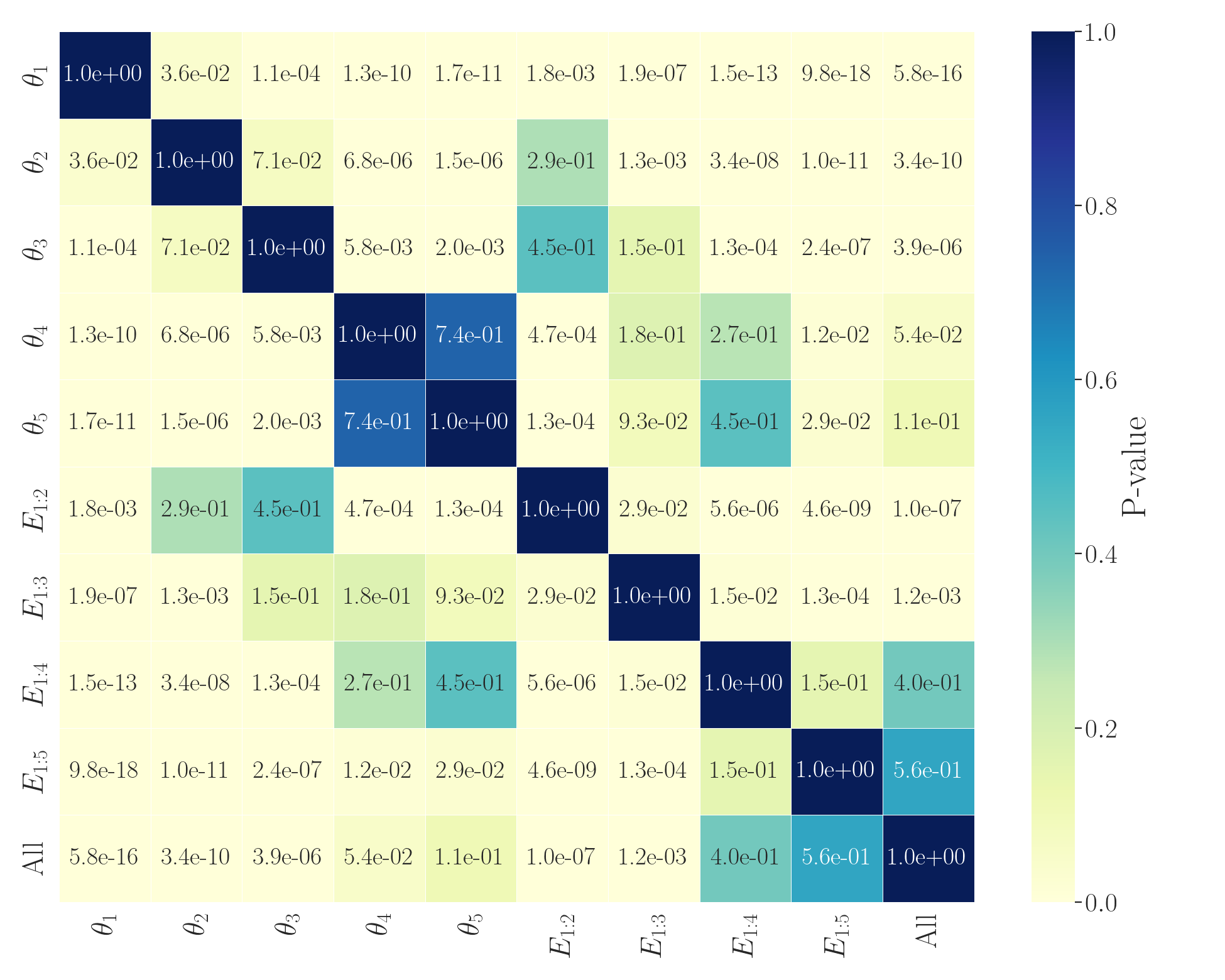}\label{fig:xgbnemgfe_Loss}}
		
		\subfloat[GSAD]{\includegraphics[width=0.24\textwidth]{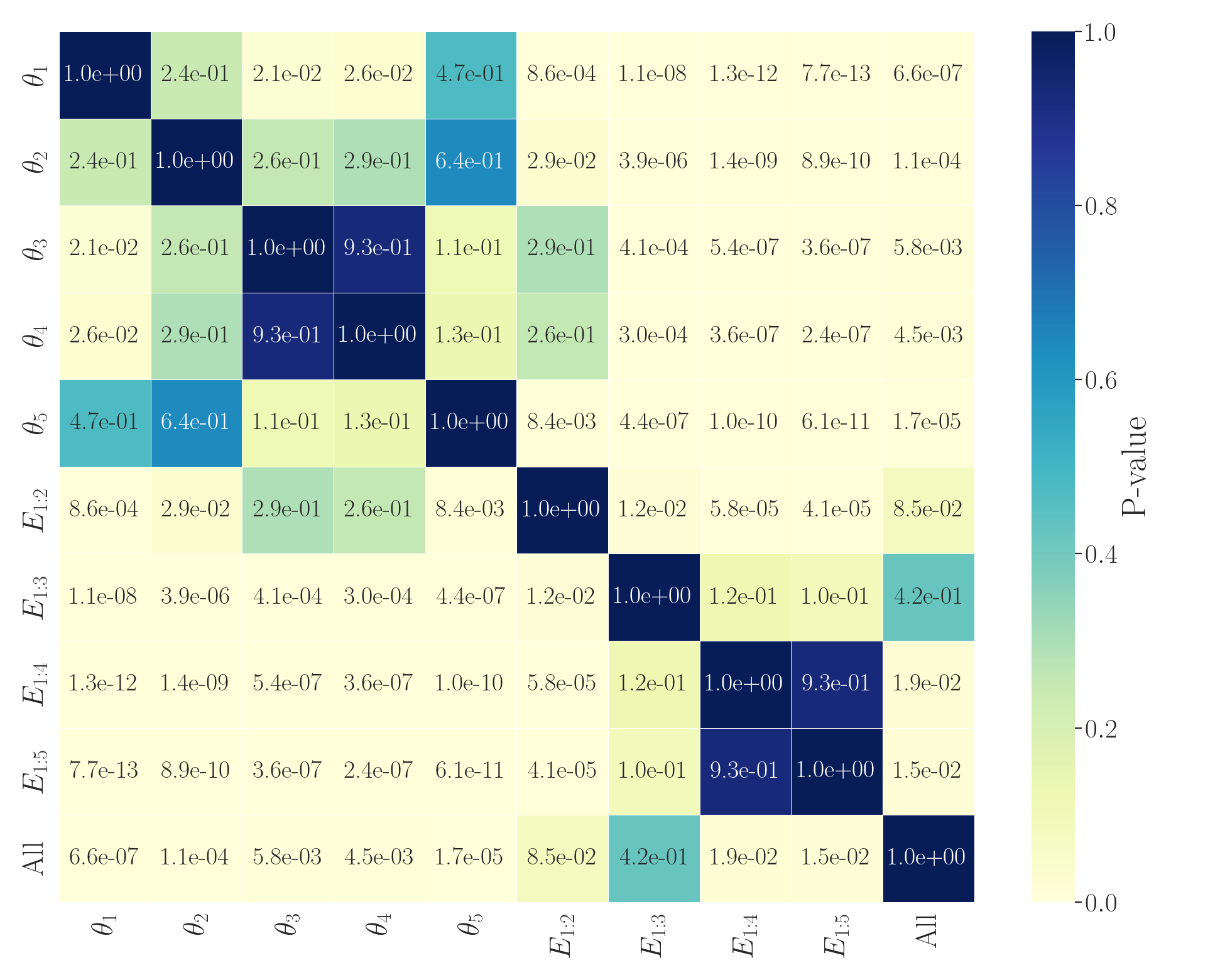}\label{fig:xgbnemgsad_Loss}}%
		\hfill
		\subfloat[HAPT]{\includegraphics[width=0.24\textwidth]{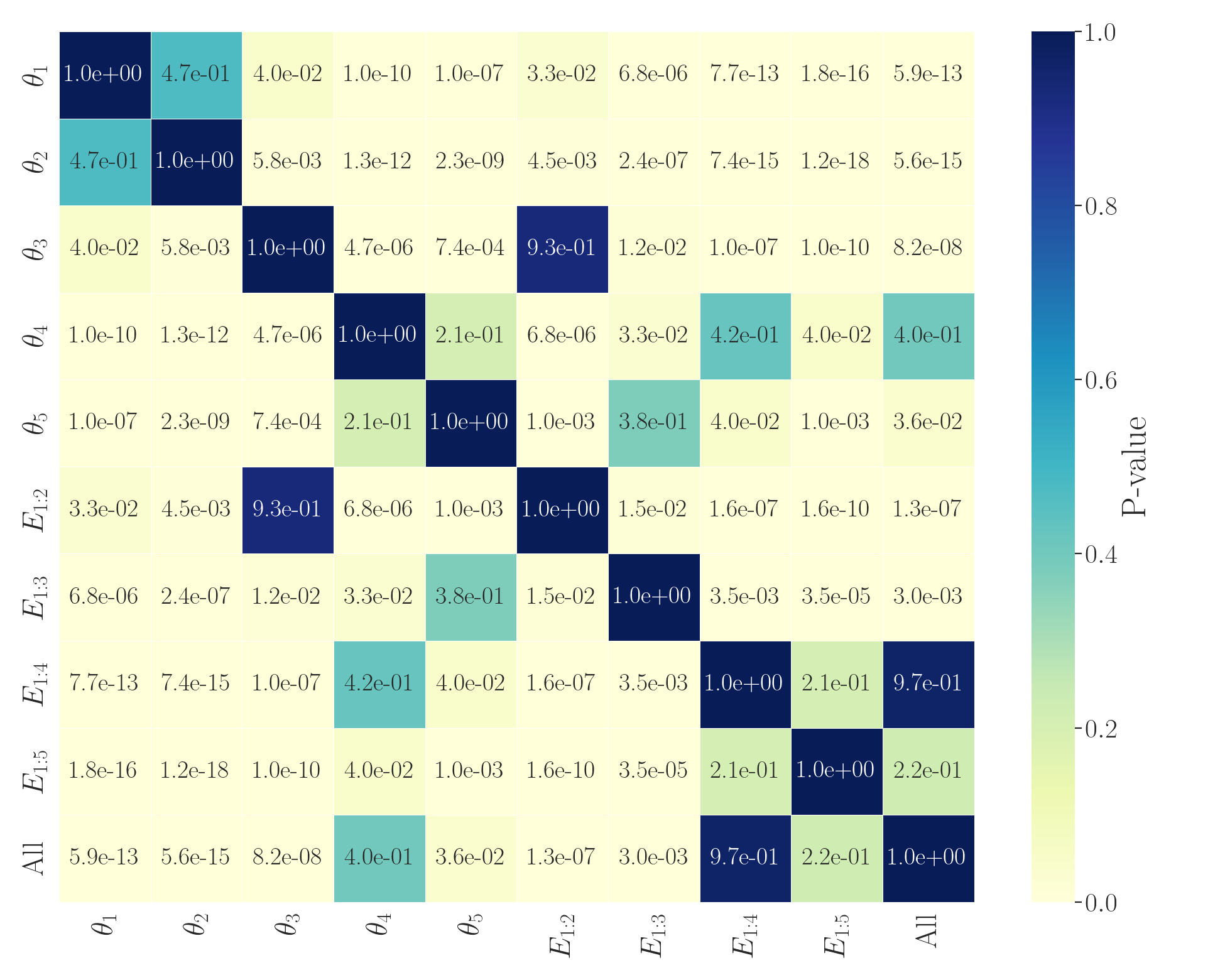}\label{fig:xgbnemhapt_Loss}}%
		\hfill
		\subfloat[ISOLET]{\includegraphics[width=0.24\textwidth]{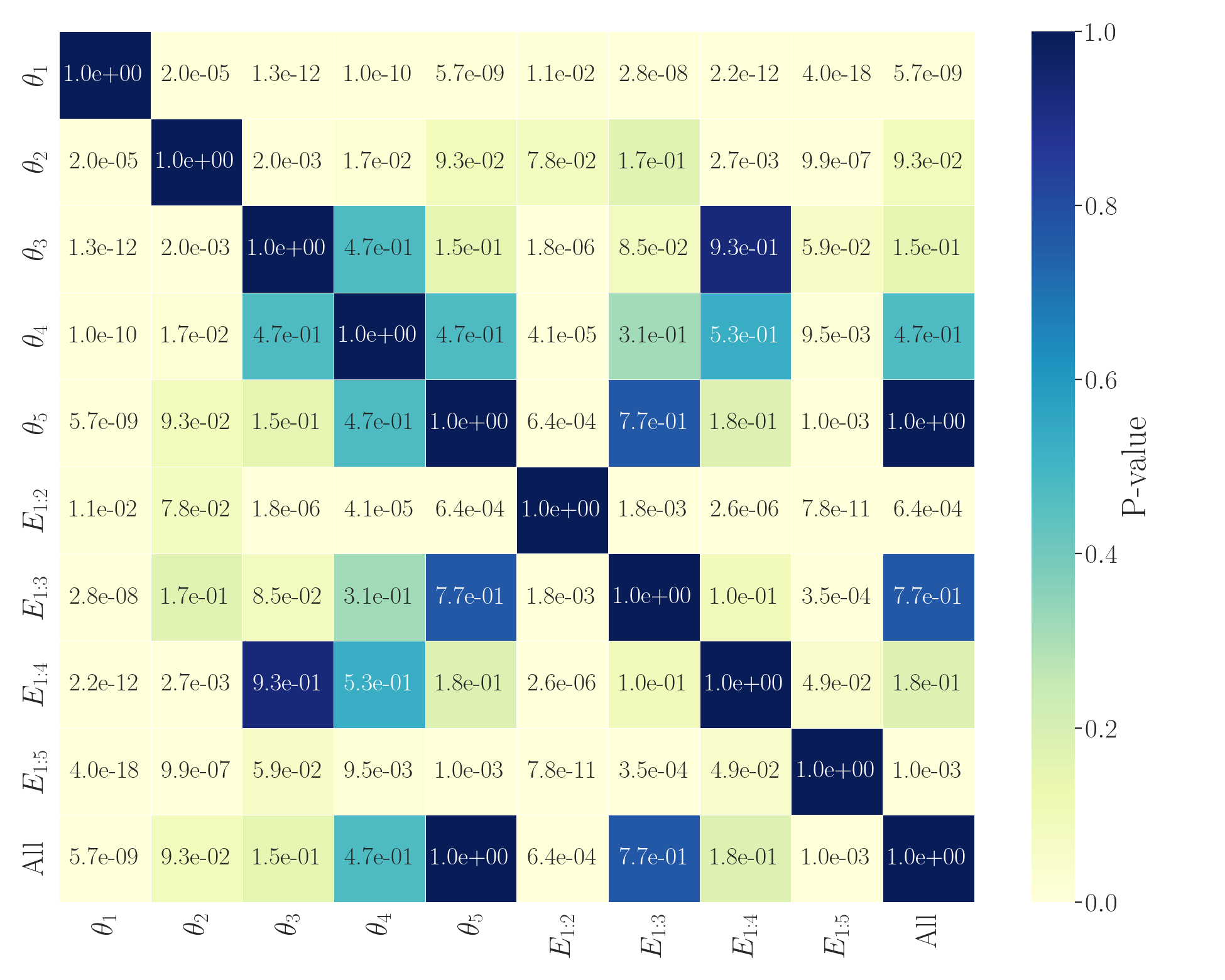}\label{fig:xgbnemisolet_Loss}}%
		\hfill
		\subfloat[PD]{\includegraphics[width=0.24\textwidth]{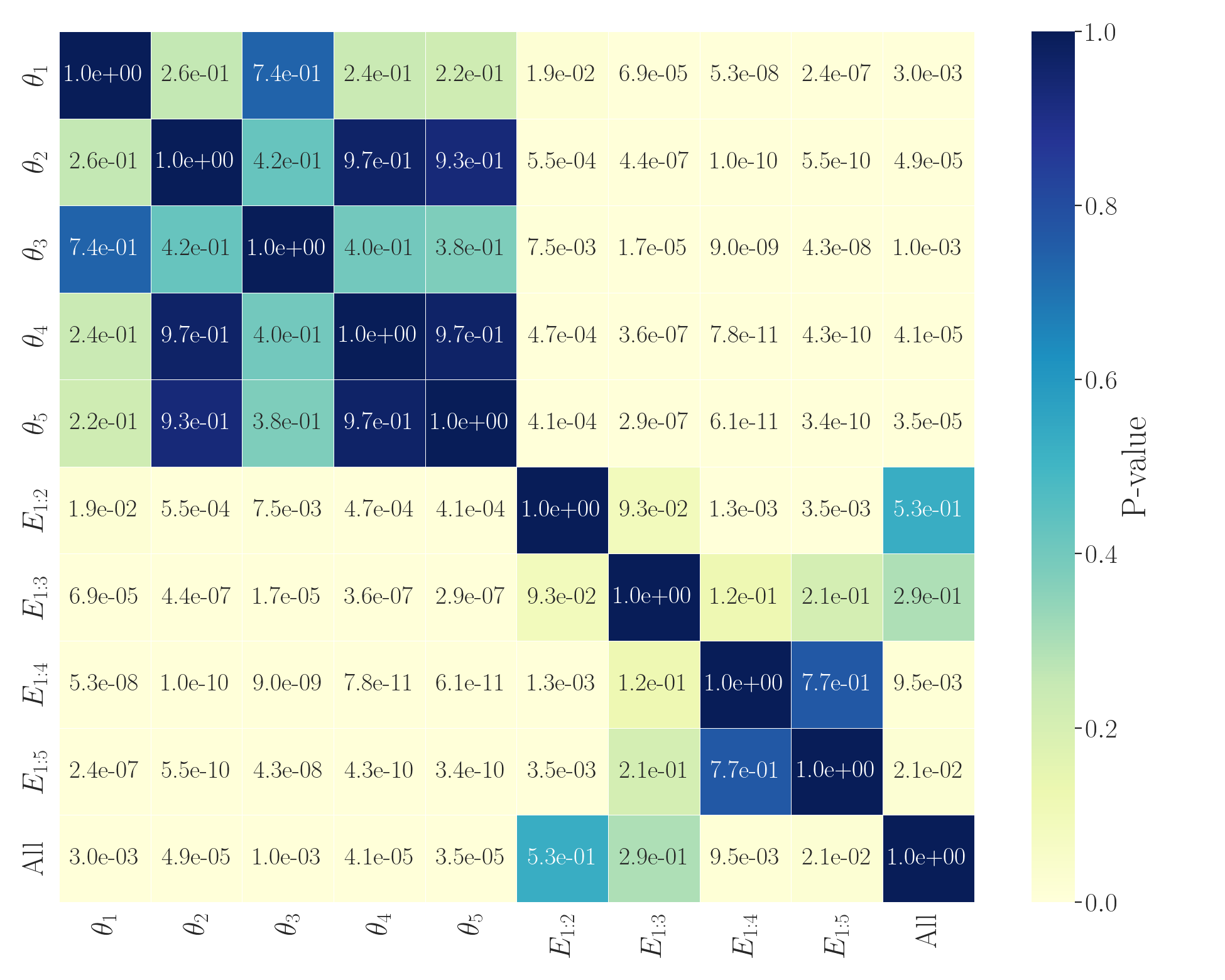}\label{fig:xgbnempd_Loss}}
		\caption[The adjusted Conover's P-values for the obtained Log-Loss values from 30 XGBoost runs.]{The results of the Conover post-hoc test on testing data’s Log-Loss obtained from 30 XGBoost runs.}
		
		\label{fig:xgbnem_Loss}
	\end{figure*}
	\FloatBarrier
	
	\begin{figure*}[htbp] 
		\centering
		\subfloat[APSF]{\includegraphics[width=0.24\textwidth]{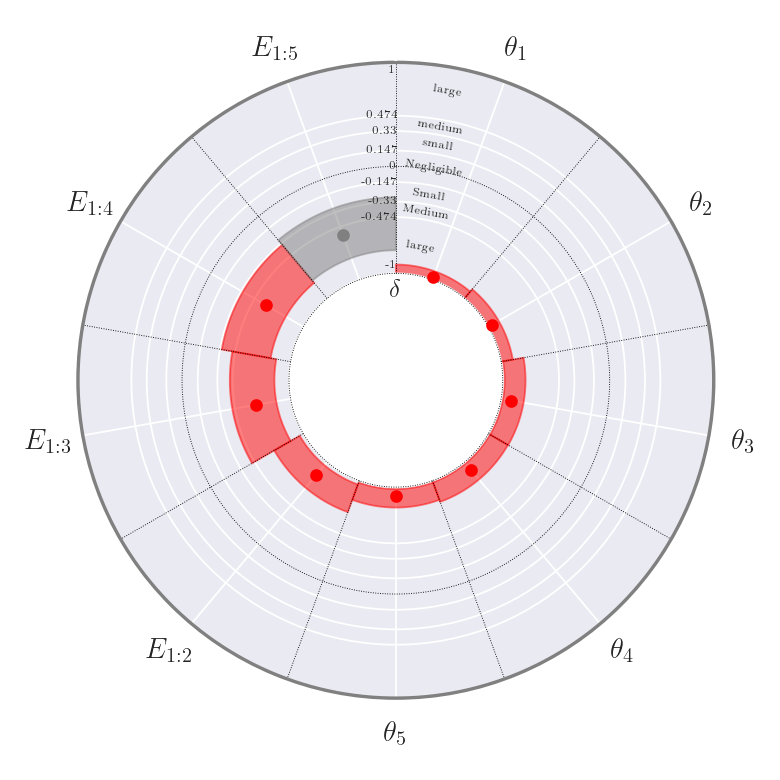}\label{fig:xgbcliffapsf_Loss}}%
		\hfill
		\subfloat[ARWPM]{\includegraphics[width=0.24\textwidth]{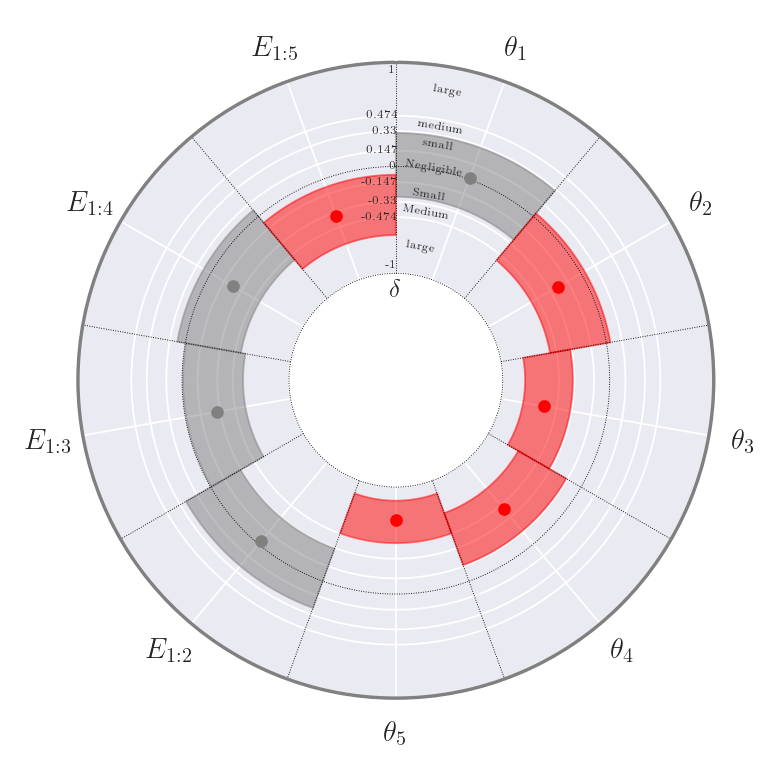}\label{fig:xgbcliffarwpm_Loss}}%
		\hfill
		\subfloat[GECR]{\includegraphics[width=0.24\textwidth]{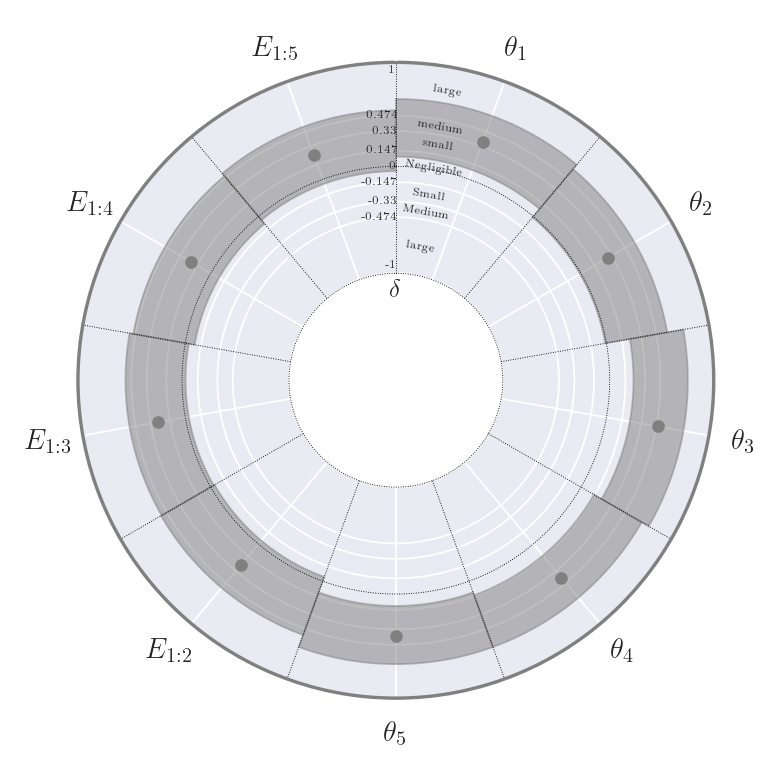}\label{fig:xgbcliffgecr_Loss}}%
		\hfill
		\subfloat[GFE]{\includegraphics[width=0.24\textwidth]{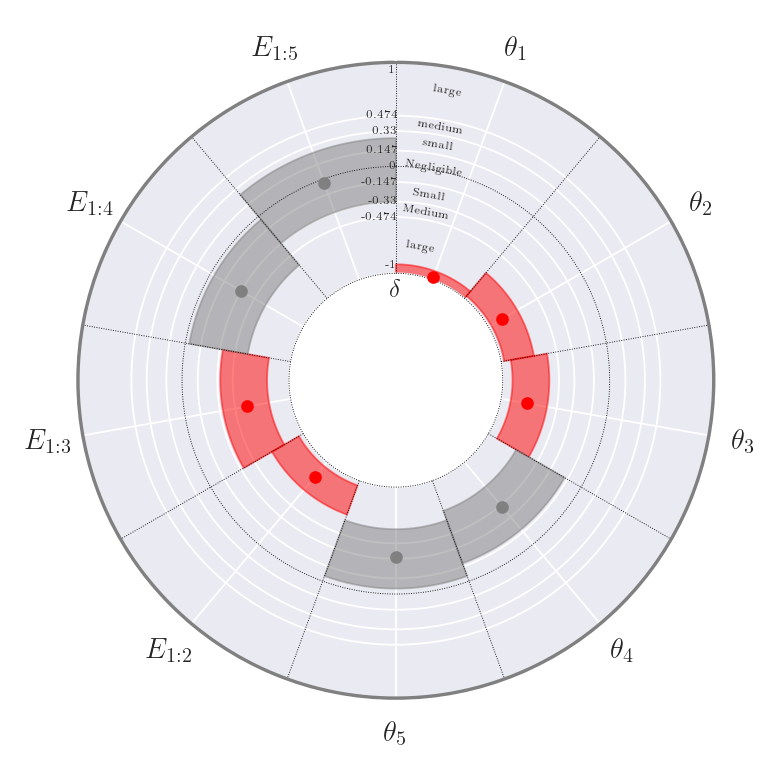}\label{fig:xgbcliffgfe_Loss}}
		
		\subfloat[GSAD]{\includegraphics[width=0.24\textwidth]{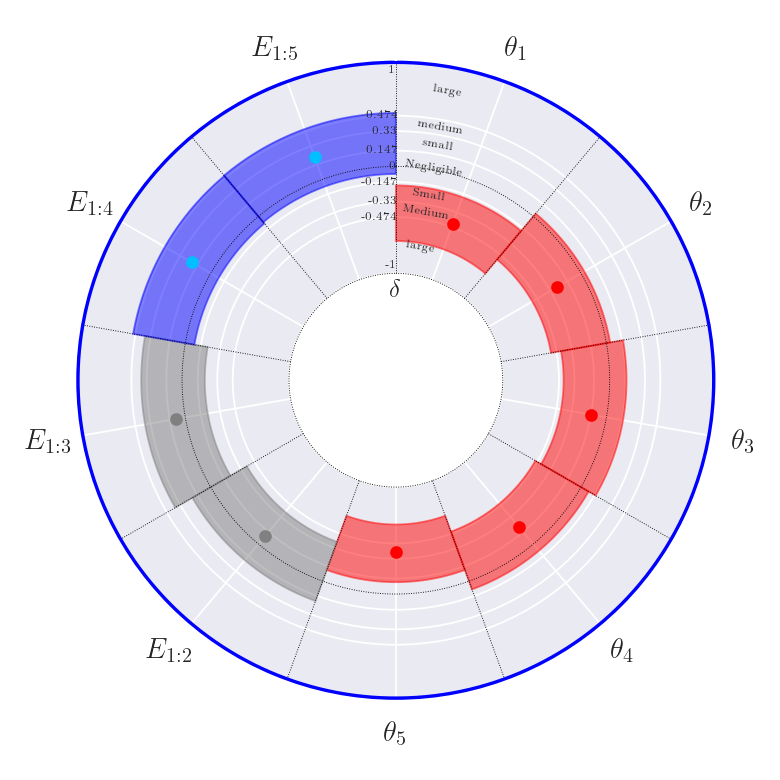}\label{fig:xgbcliffgsad_Loss}}%
		\hfill
		\subfloat[HAPT]{\includegraphics[width=0.24\textwidth]{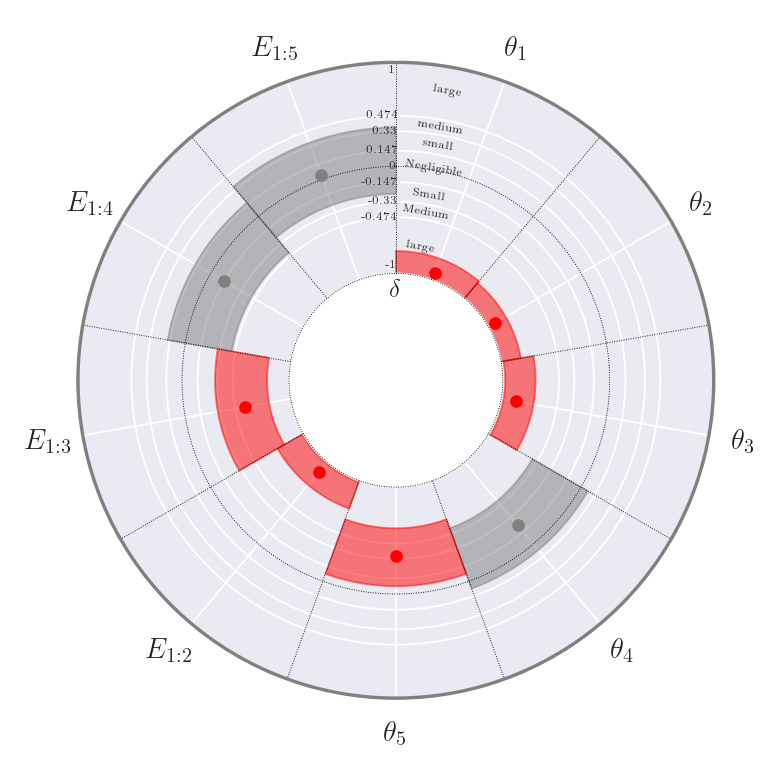}\label{fig:xgbcliffhapt_Loss}}%
		\hfill
		\subfloat[ISOLET]{\includegraphics[width=0.24\textwidth]{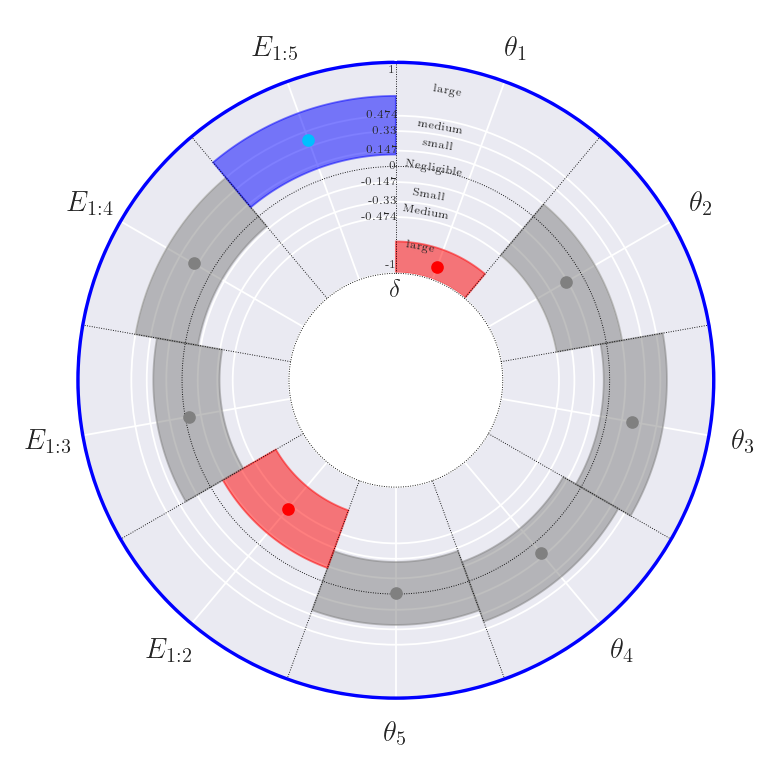}\label{fig:xgbcliffisolet_Loss}}%
		\hfill
		\subfloat[PD]{\includegraphics[width=0.24\textwidth]{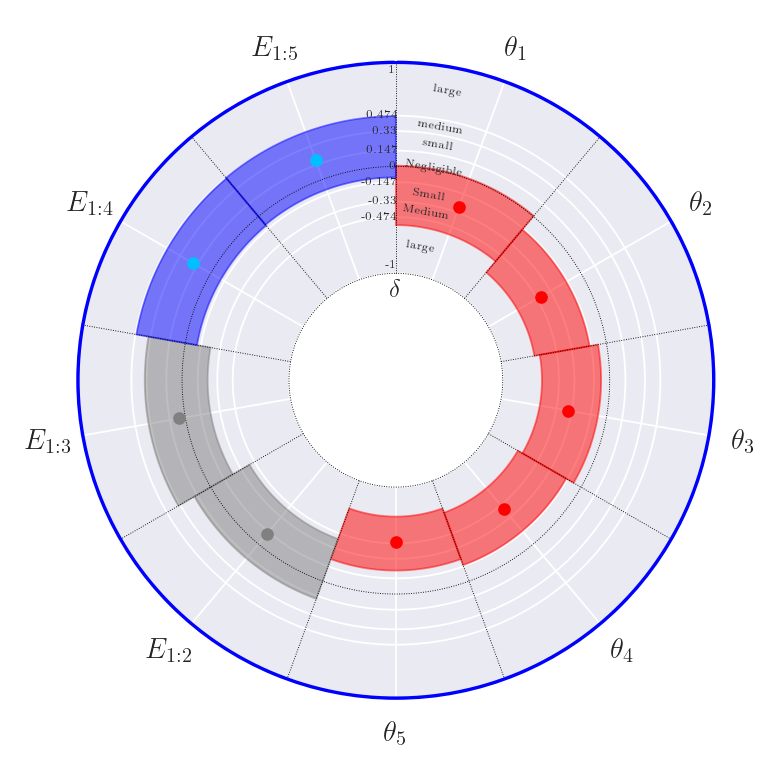}\label{fig:xgbcliffpd_Loss}}
		\caption[The Cliff's $\delta$ effect size measure and its 95\% confidence intervals for the Log-Loss values obtained from 30 XGBoost runs.]{Effect size analysis of test data Log-Loss across 30 XGBoost runs using Cliff's $\delta$. Each point represents the actual value obtained, with segments denoting 95\% confidence intervals based on 10,000 bootstrap resamplings. The outer ring color visualizes the statistical significance: grey illustrates no significant difference (adjusted Friedman's P-value$>0.05$), while color indicates significant differences; blue indicates at least one view and/or ensemble outperforms the benchmark (adjusted Conover's p-value$ < 0.05$, Cliff's $\delta > 0$), and red signifies all views and ensembles underperform relative to the benchmark (adjusted Conover's p-value$ < 0.05$, Cliff's $\delta < 0$). Segment colors show performance difference against the benchmark: grey for no significant difference (adjusted Conover's p-value$  > 0.05$), blue for better performance (Cliff's $\delta > 0$), and red for worse performance (Cliff's $\delta < 0$).}
		
		\label{fig:xgbcliff_Loss}
	\end{figure*}
	
	\begin{table*}[htbp]
		\centering
		\caption[The results of Friedman and Conover tests and Cliff's $\delta$ analysis for the Log-Loss values obtained from 30 XGBoost runs.]{Statistical comparison of Log-Loss results for testing data obtained from XGBoost runs. W, T, and L denote win, tie, and loss based on adjusted Friedman and Conover's p-values. Effect sizes are calculated using Cliff's Delta method and are categorized as negligible, small, medium, or large.}
		\label{tab:xgbloss}
		\resizebox{\linewidth}{!}{%
			\begin{tabular}{c|ccccccccc}
				\hline
				\multicolumn{10}{c}{XGBoost's Log-Loss}\\
				\hline
				Dataset & $\theta_1$ & $\theta_2$ & $\theta_3$ & $\theta_4$ & $\theta_5$ & $E_{1:2}$ & $E_{1:3}$ & $E_{1:4}$ & $E_{1:5}$ \\
				\hline
				APSF  & L (large) & L (large) & L (large) & L (large) & L (large) & L (large) & L (large) & L (large) & T (large) \\
				ARWPM  & T (negligible) & L (small) & L (large) & L (medium) & L (large) & T (negligible) & T (small) & T (small) & L (medium) \\
				GECR  & T (medium) & T (small) & T (large) & T (medium) & T (medium) & T (small) & T (small) & T (small) & T (small) \\
				GFE  & L (large) & L (large) & L (large) & T (medium) & T (medium) & L (large) & L (large) & T (small) & T (negligible) \\
				GSAD  & L (medium) & L (small) & L (negligible) & L (small) & L (medium) & T (negligible) & T (negligible) & W (small) & W (small) \\
				HAPT  & L (large) & L (large) & L (large) & T (small) & L (medium) & L (large) & L (large) & T (negligible) & T (negligible) \\
				ISOLET  & L (large) & T (small) & T (small) & T (negligible) & T (negligible) & L (medium) & T (negligible) & T (small) & W (medium) \\
				PD  & L (small) & L (medium) & L (medium) & L (medium) & L (large) & T (negligible) & T (negligible) & W (small) & W (small) \\
				\hline
				W - T - L  & 0 - 2 - 6 & 0 - 2 - 6 & 0 - 2 - 6 & 0 - 4 - 4 & 0 - 3 - 5 & 0 - 4 - 4 & 0 - 5 - 3 & 2 - 5 - 1 & 3 - 4 - 1 \\
				\hline
			\end{tabular}
		}
	\end{table*}
	\FloatBarrier

	\begin{figure*}[t] 
		\centering
		\subfloat[APSF]{\includegraphics[width=0.24\textwidth]{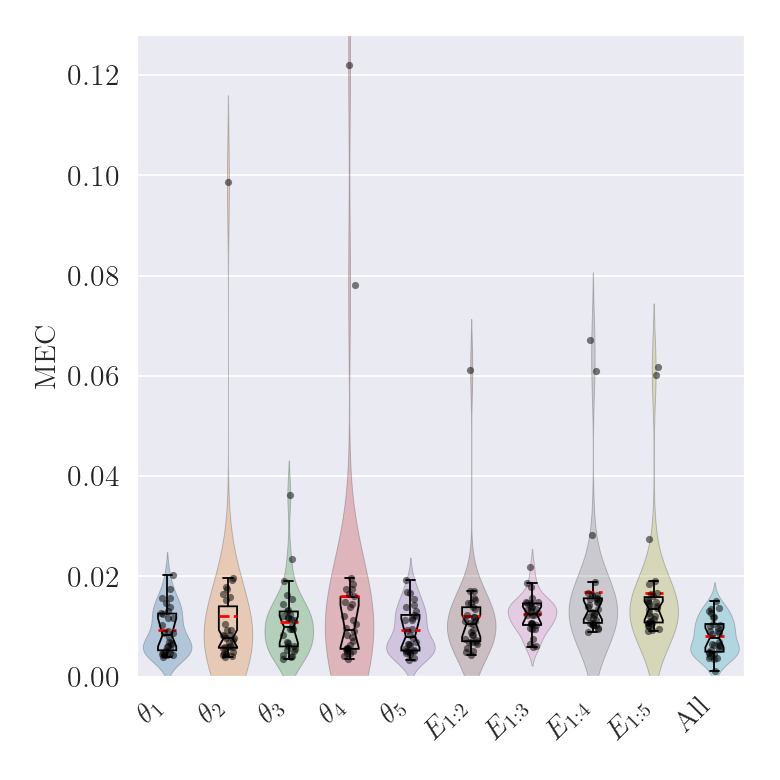}\label{fig:xgbapsf_MEC}}%
		\hfill
		\subfloat[ARWPM]{\includegraphics[width=0.24\textwidth]{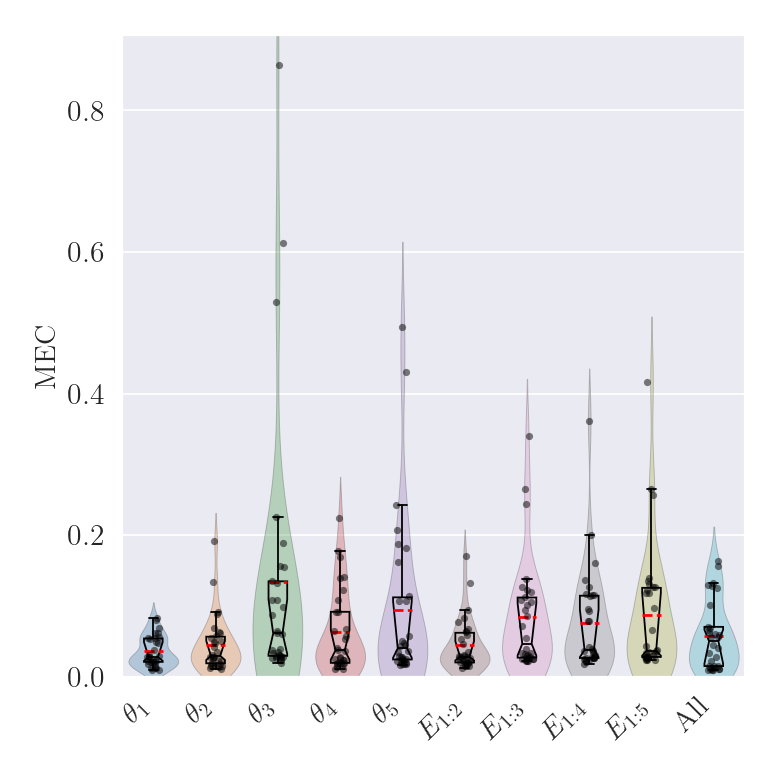}\label{fig:xgbarwpm_MEC}}%
		\hfill
		\subfloat[GECR]{\includegraphics[width=0.24\textwidth]{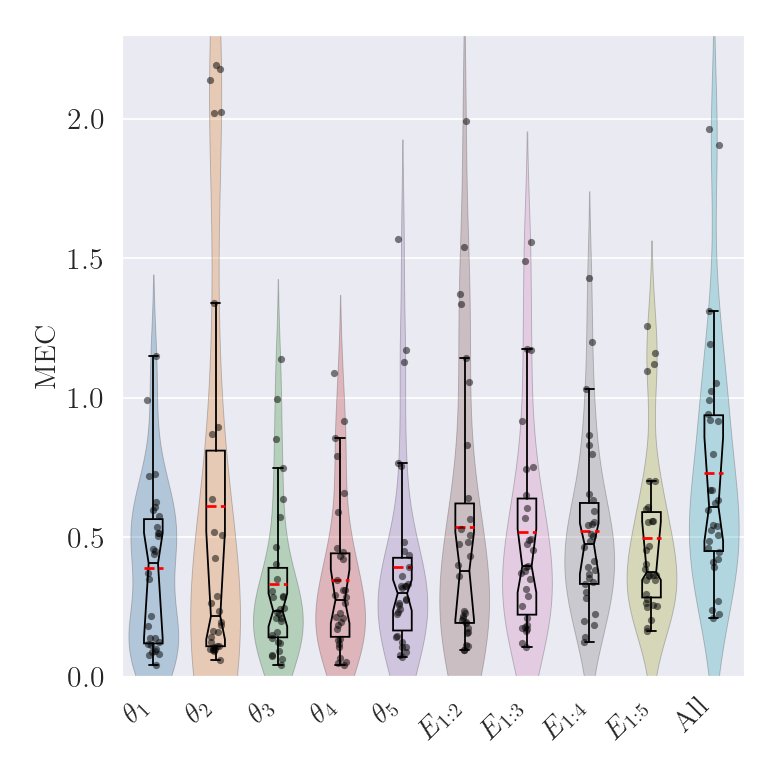}\label{fig:xgbgecr_MEC}}%
		\hfill
		\subfloat[GFE]{\includegraphics[width=0.24\textwidth]{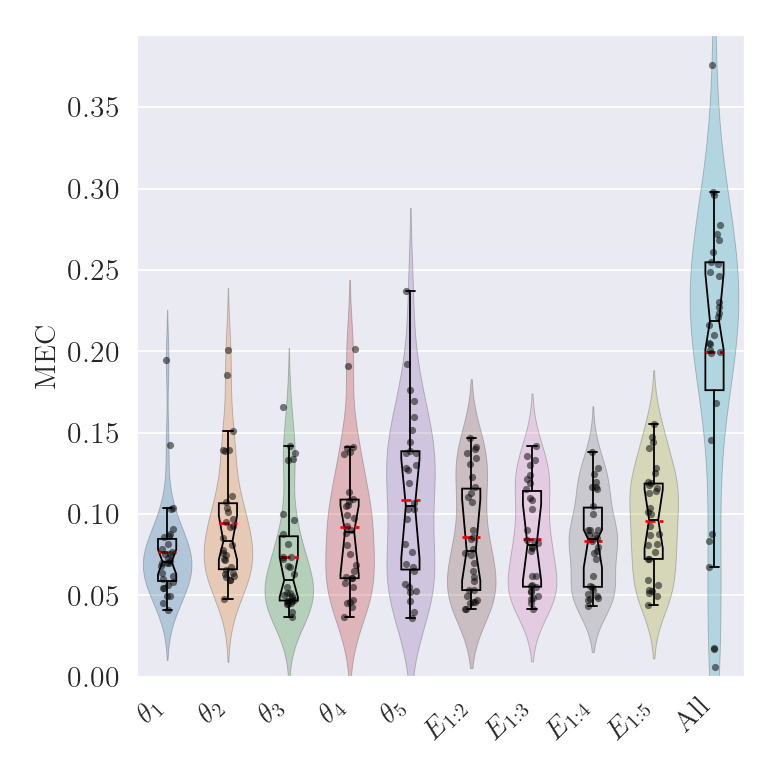}\label{fig:xgbgfe_MEC}}
		
		\subfloat[GSAD]{\includegraphics[width=0.24\textwidth]{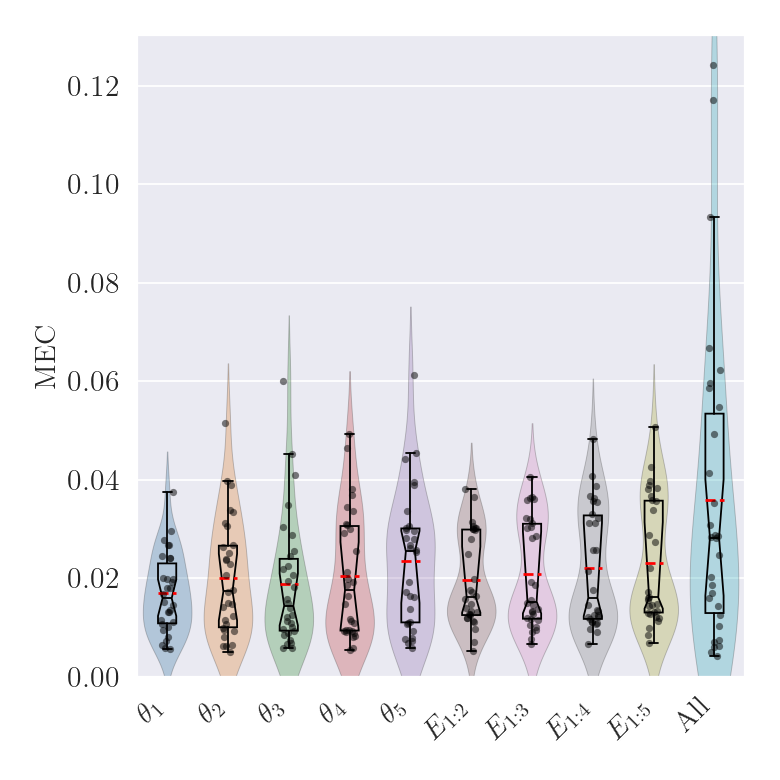}\label{fig:fpgsad_MEC}}%
		\hfill
		\subfloat[HAPT]{\includegraphics[width=0.24\textwidth]{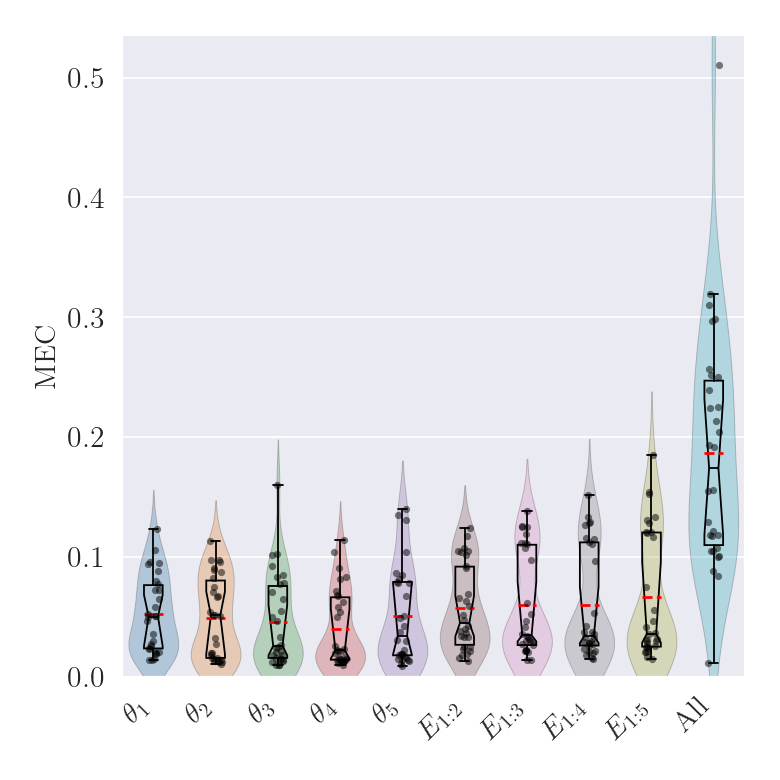}\label{fig:xgbhapt_MEC}}%
		\hfill
		\subfloat[ISOLET]{\includegraphics[width=0.24\textwidth]{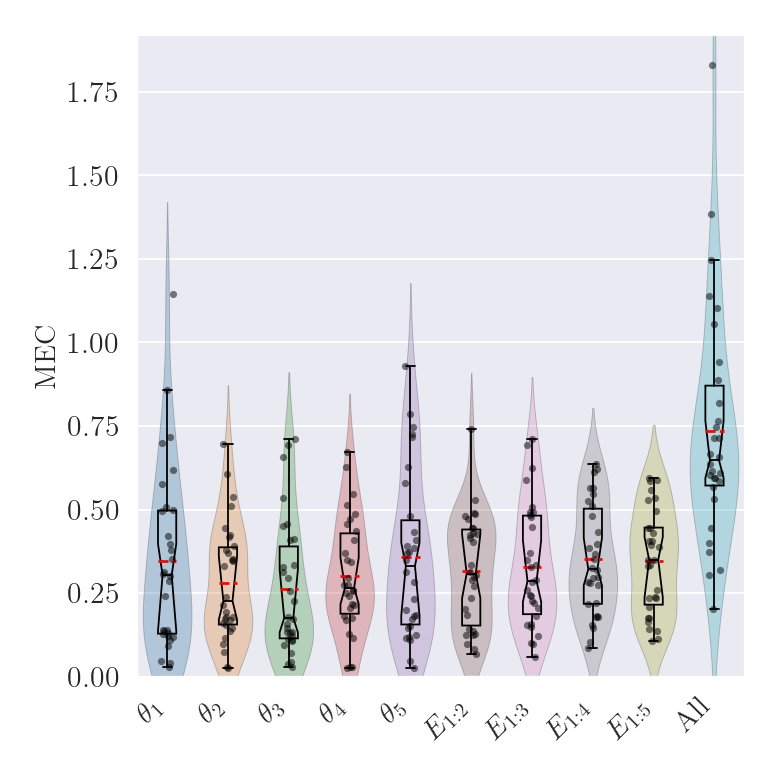}\label{fig:xgbisolet_MEC}}%
		\hfill
		\subfloat[PD]{\includegraphics[width=0.24\textwidth]{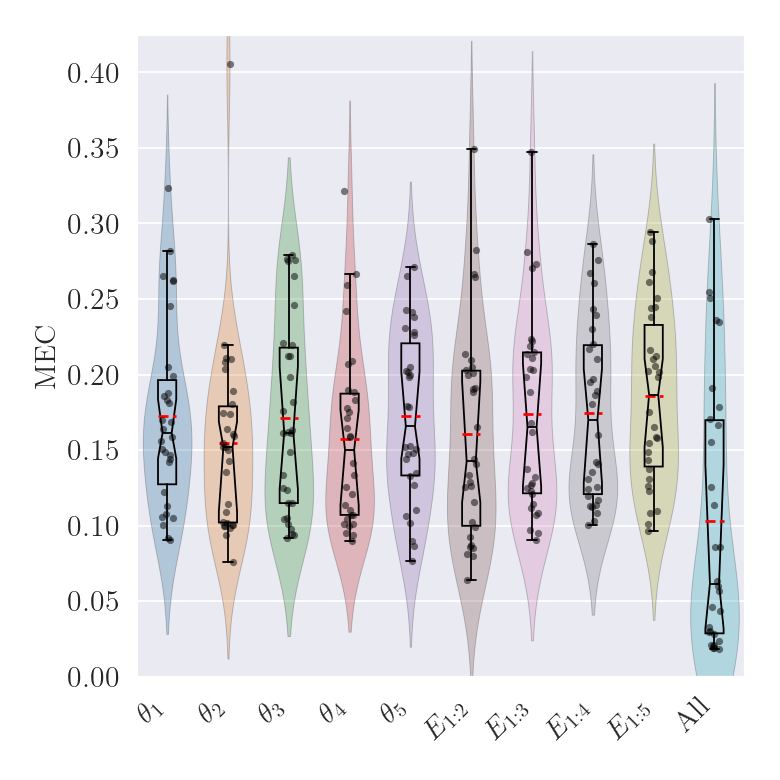}\label{fig:xgbpd_MEC}}
		\caption[The distribution of the obtained MEC values for 30 XGBoost runs.]{The raincloud plot of MEC results obtained from 30 XGBoost runs.}
		
		\label{fig:xgb_MEC}
	\end{figure*}
	
	\begin{figure*}[t] 
		\centering
		\subfloat[APSF]{\includegraphics[width=0.24\textwidth]{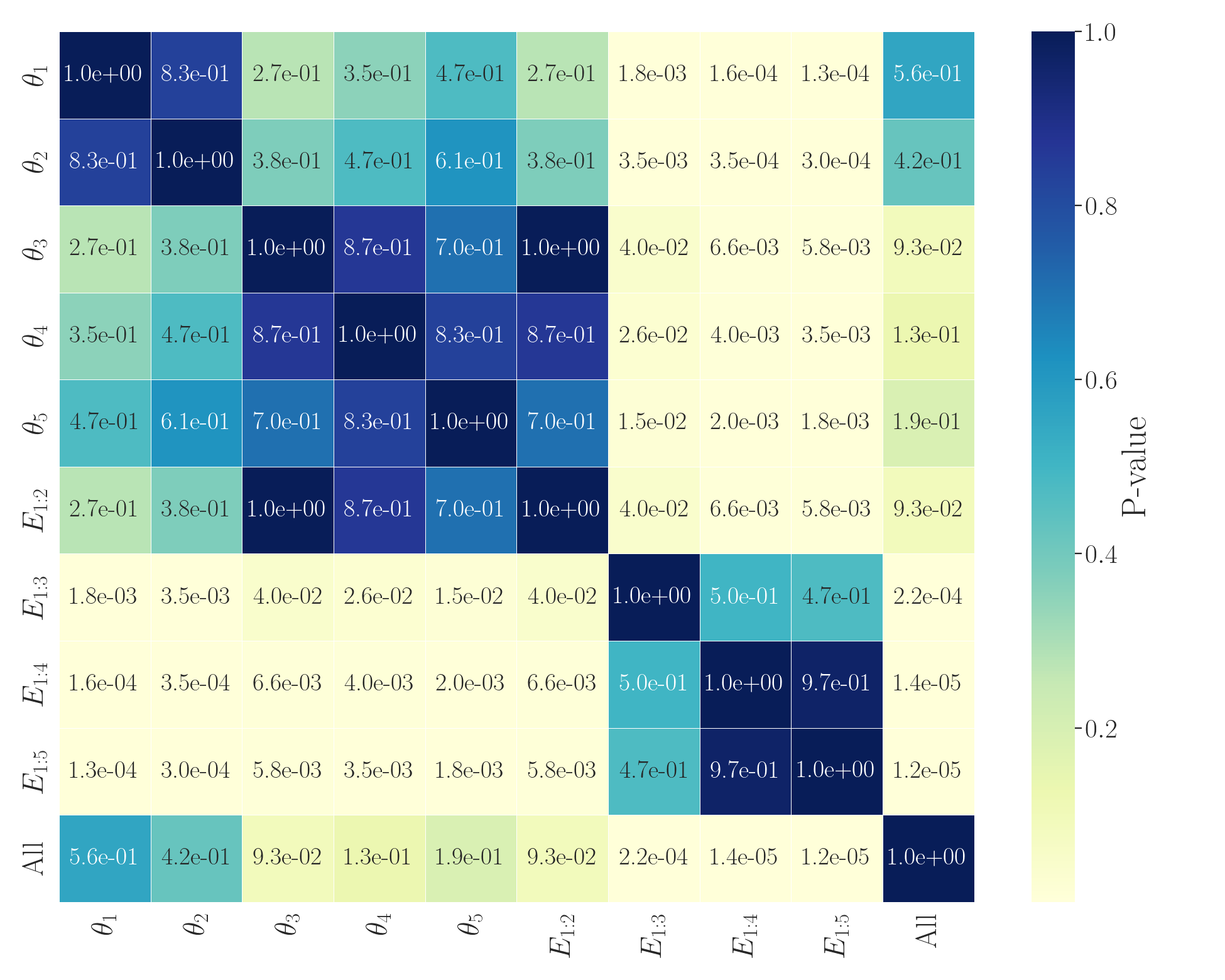}\label{fig:xgbnemapsf_MEC}}%
		\hfill
		\subfloat[ARWPM]{\includegraphics[width=0.24\textwidth]{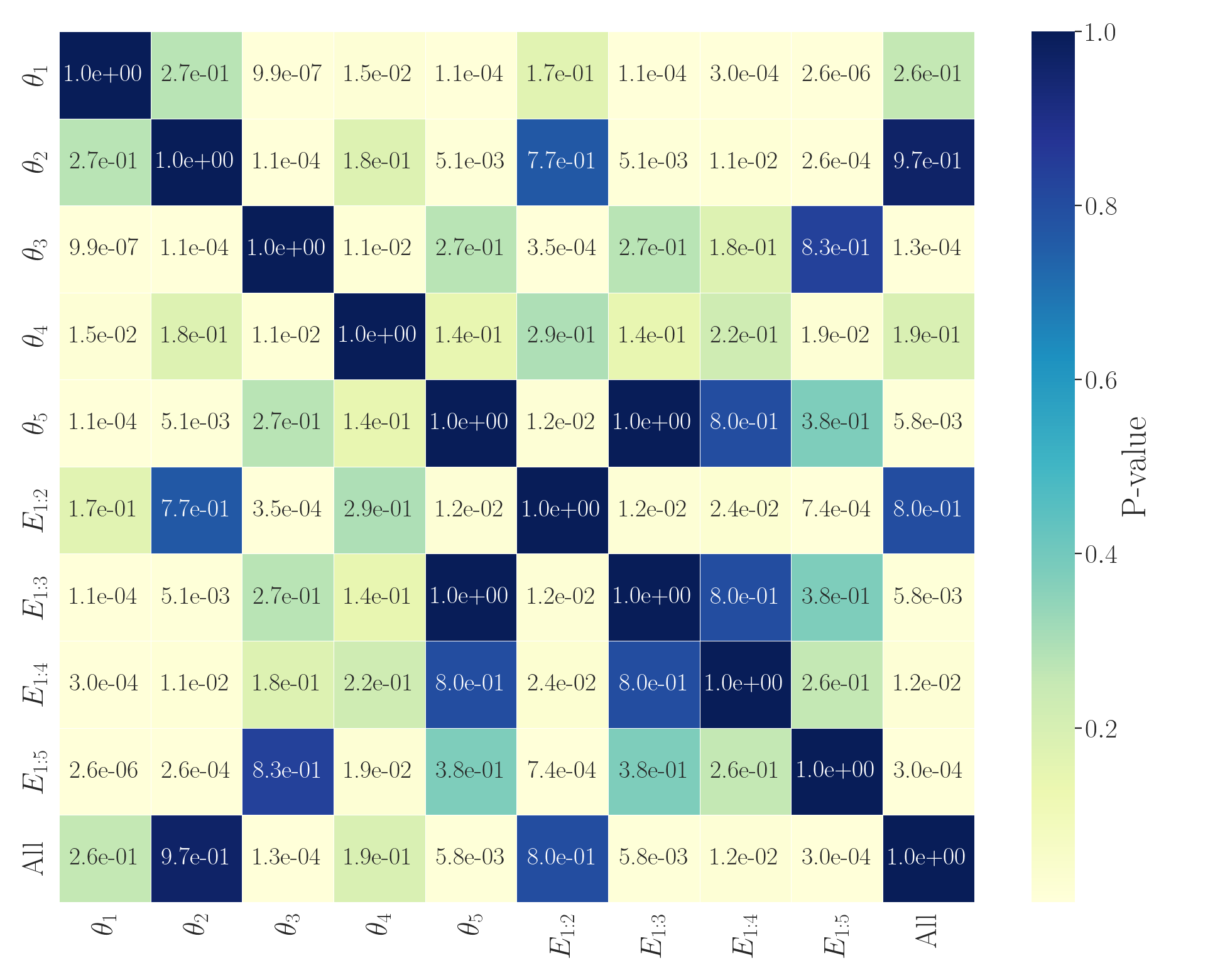}\label{fig:xgbnemarwpm_MEC}}%
		\hfill
		\subfloat[GECR]{\includegraphics[width=0.24\textwidth]{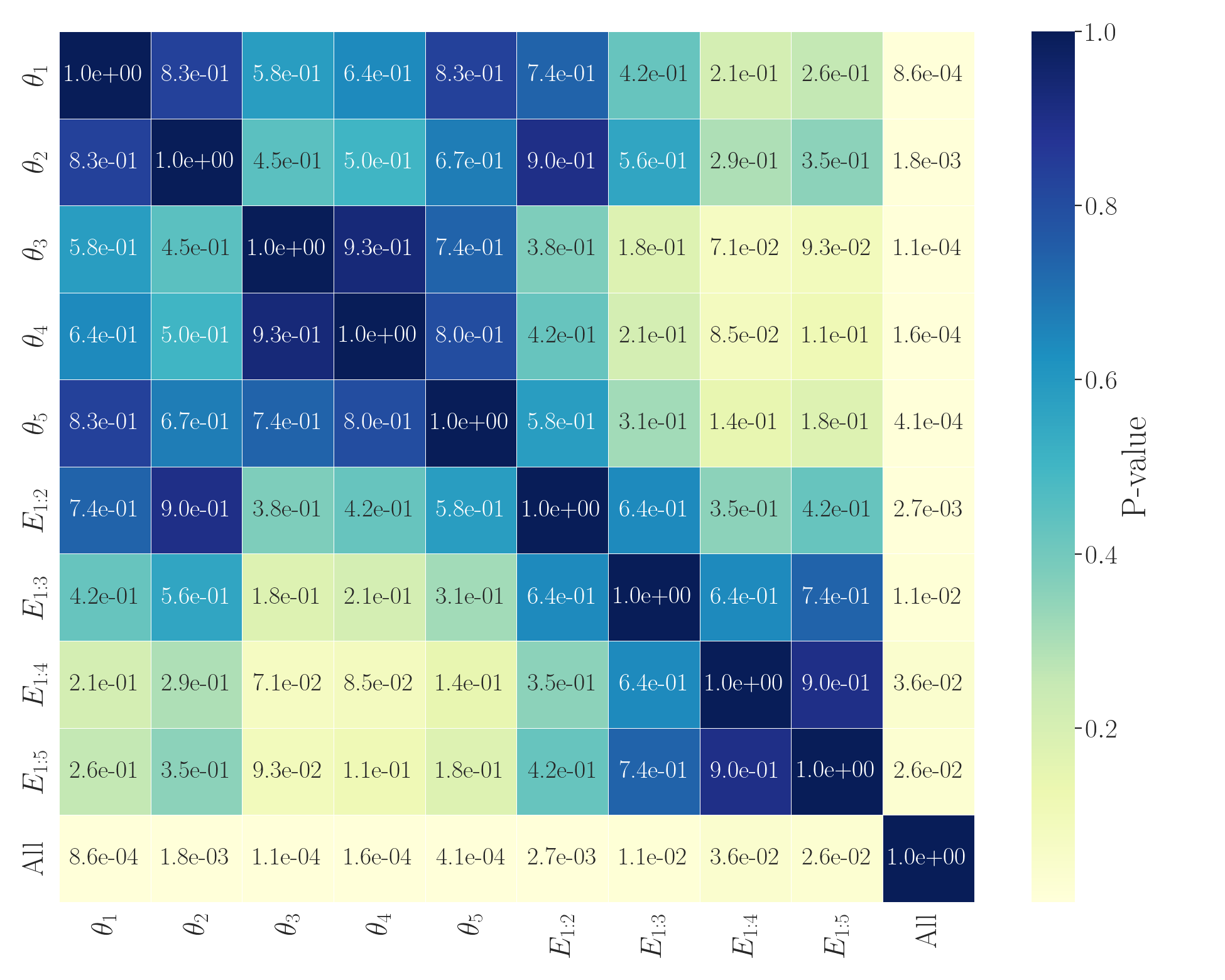}\label{fig:xgbnemgecr_MEC}}%
		\hfill
		\subfloat[GFE]{\includegraphics[width=0.24\textwidth]{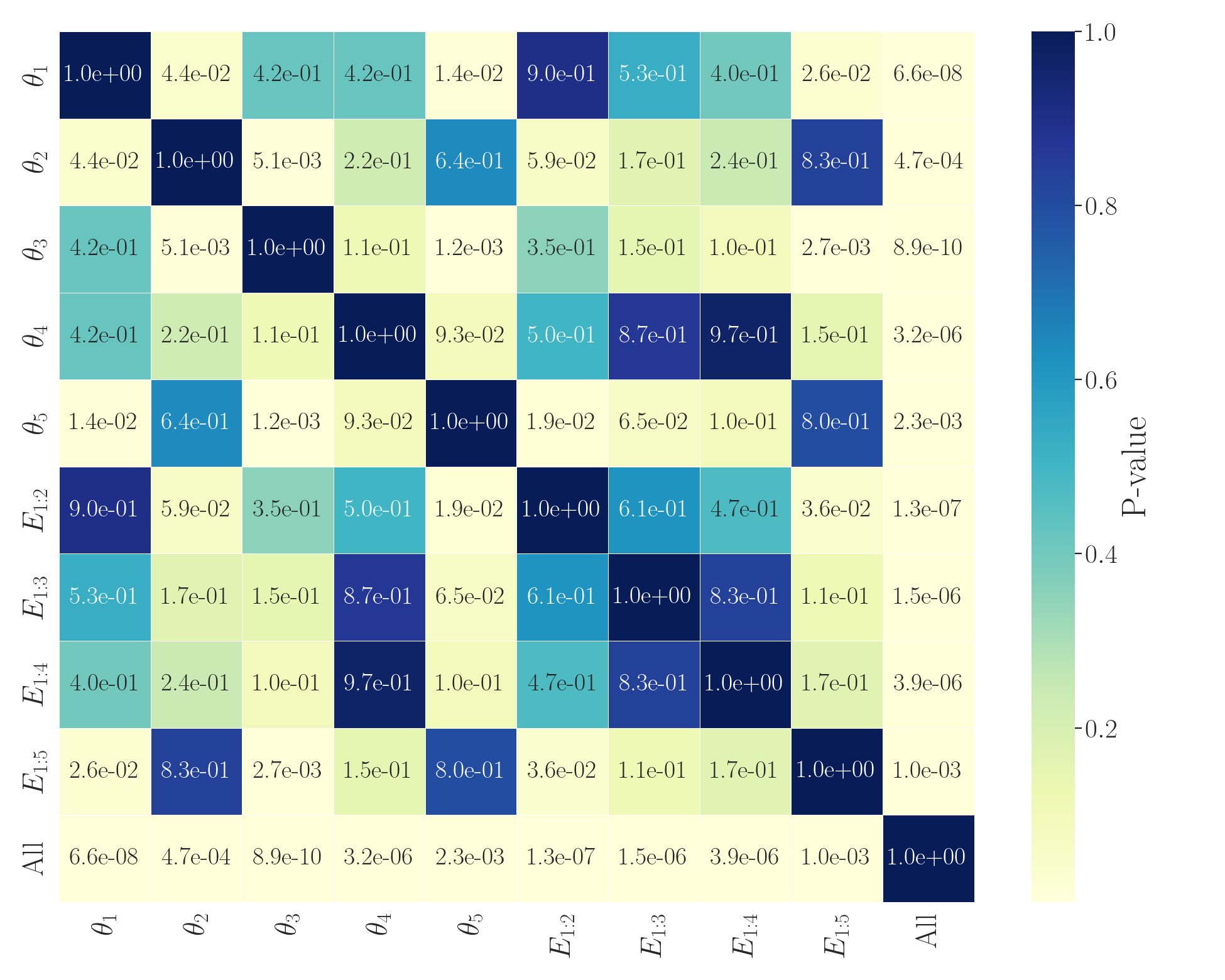}\label{fig:xgbnemgfe_MEC}}
		
		\subfloat[GSAD]{\includegraphics[width=0.24\textwidth]{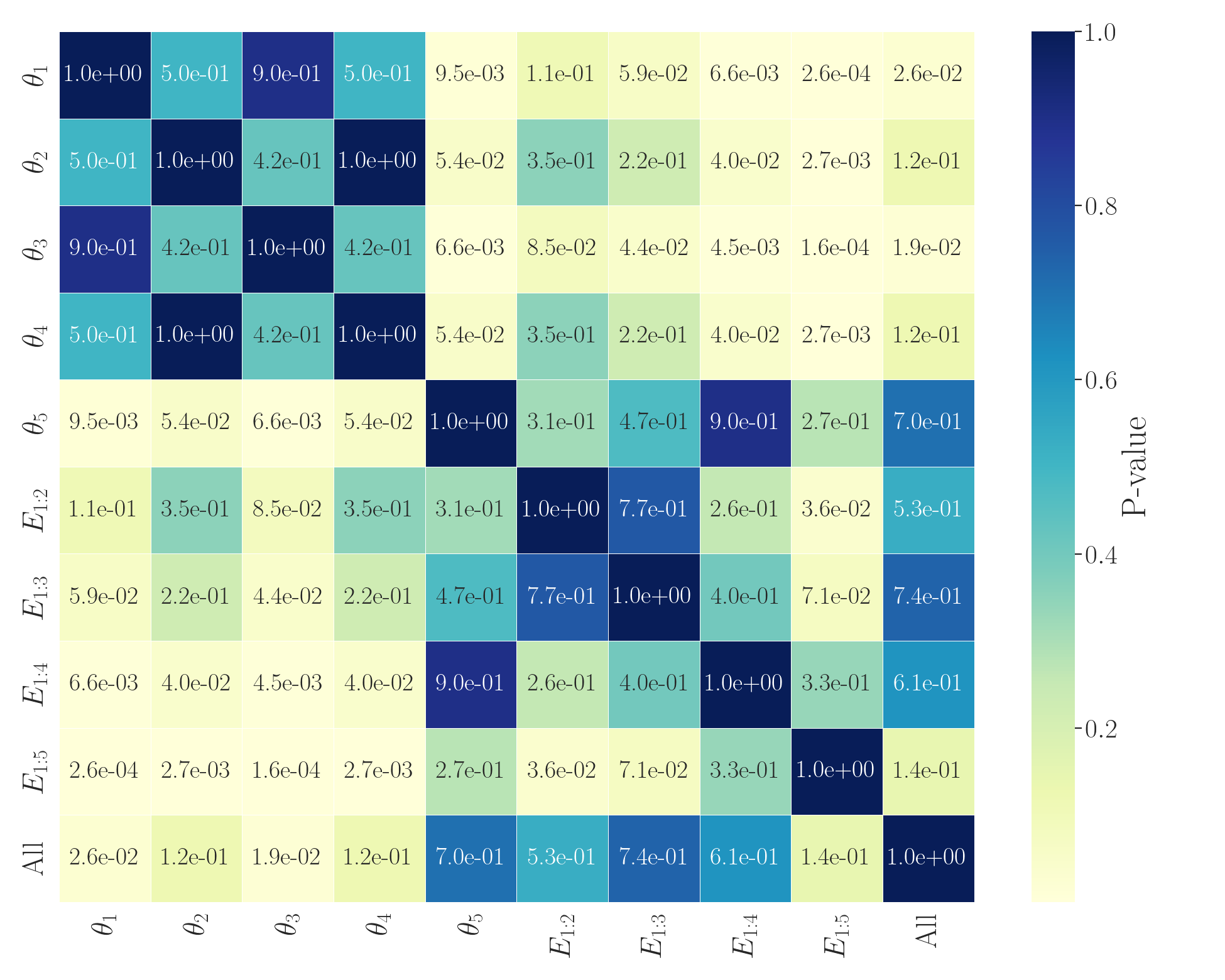}\label{fig:xgbnemgsad_MEC}}%
		\hfill
		\subfloat[HAPT]{\includegraphics[width=0.24\textwidth]{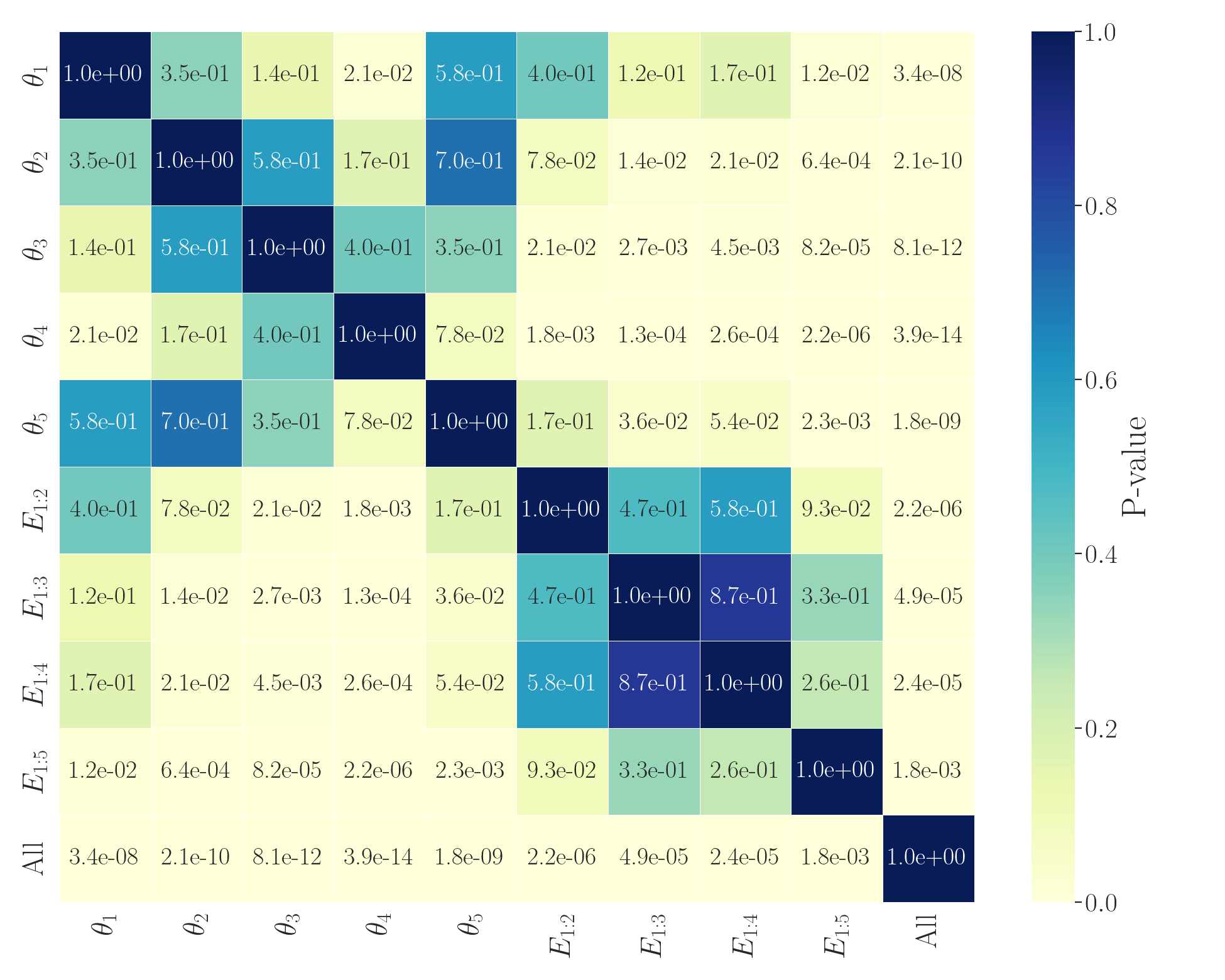}\label{fig:xgbnemhapt_MEC}}%
		\hfill
		\subfloat[ISOLET]{\includegraphics[width=0.24\textwidth]{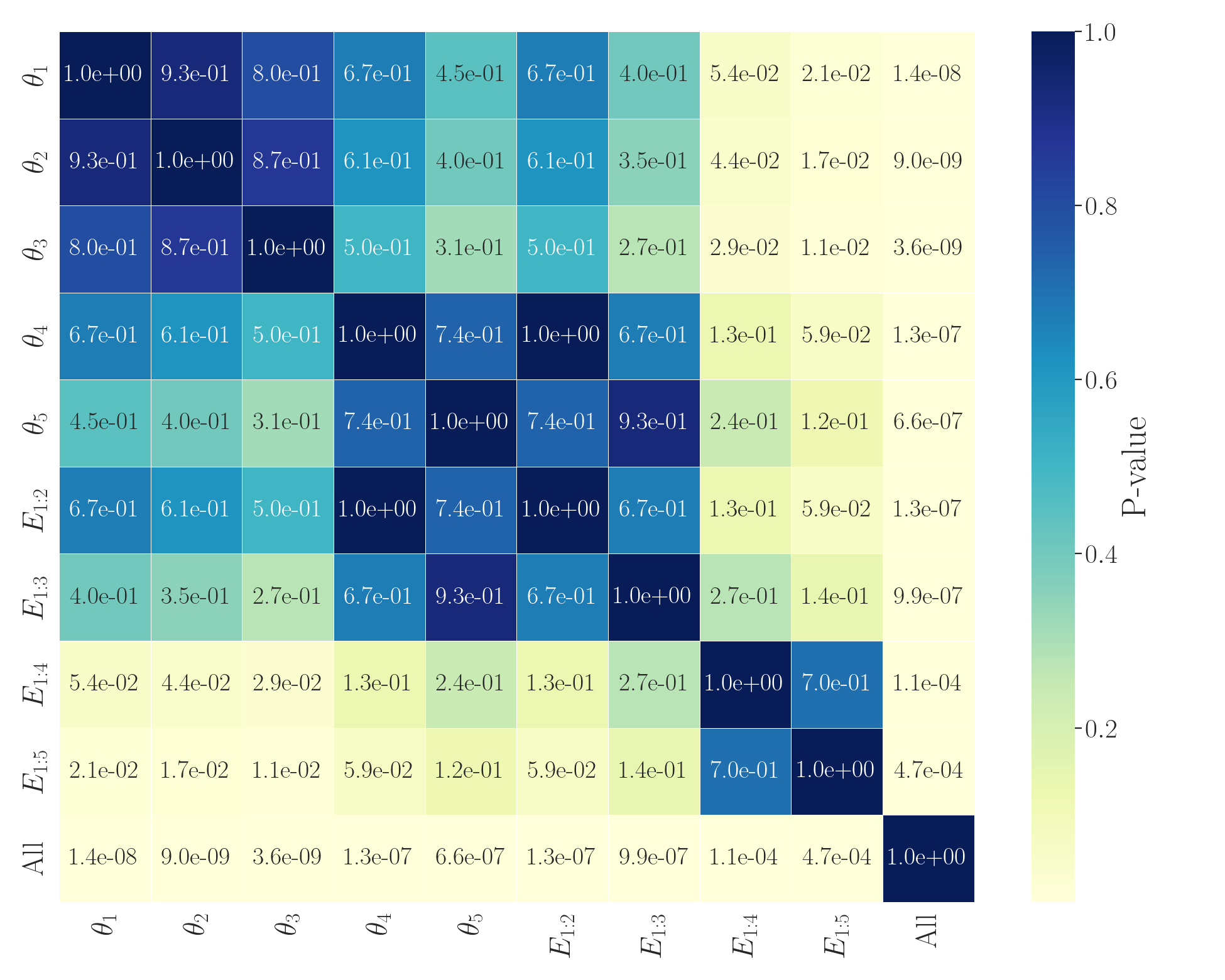}\label{fig:xgbnemisolet_MEC}}%
		\hfill
		\subfloat[PD]{\includegraphics[width=0.24\textwidth]{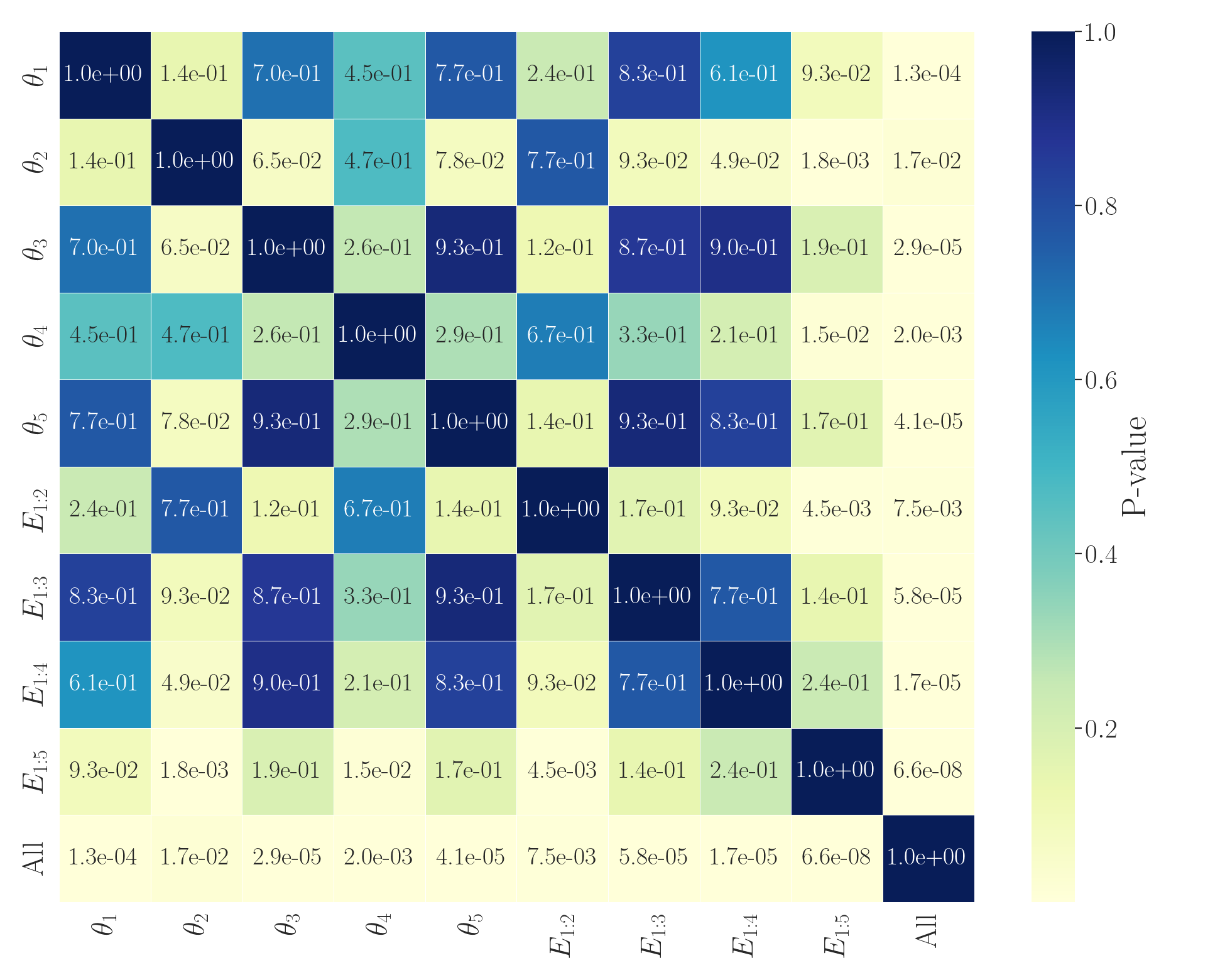}\label{fig:xgbnempd_MEC}}
		\caption[The adjusted Conover's P-values for the obtained MEC values from 30 XGBoost runs.]{The results of the Conover post-hoc test on testing data’s MEC obtained from 30 XGBoost runs.}
		
		\label{fig:xgbnem_MEC}
	\end{figure*}
	\FloatBarrier
	
	\begin{figure*}[htbp] 
		\centering
		\subfloat[APSF]{\includegraphics[width=0.24\textwidth]{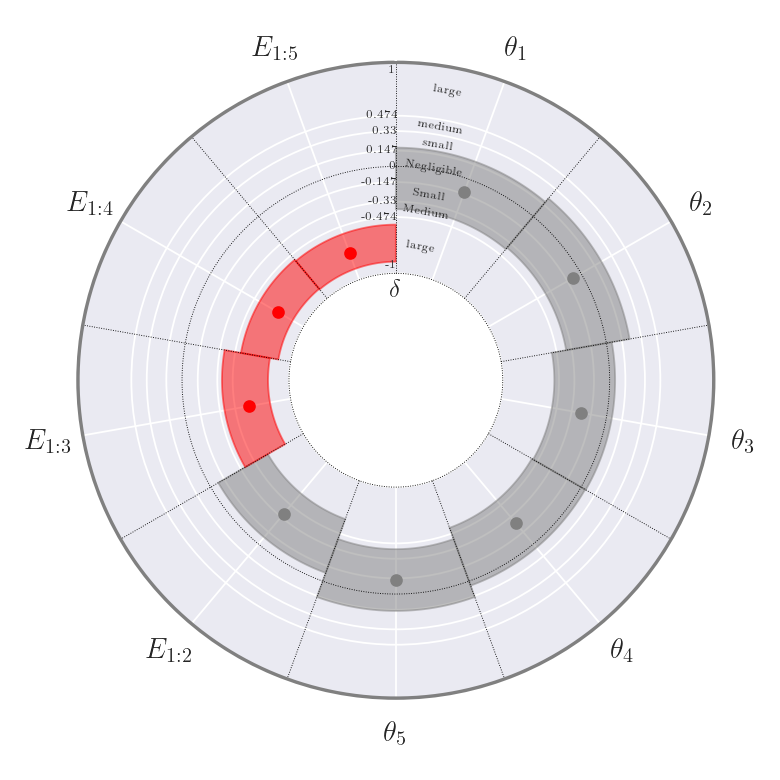}\label{fig:xgbcliffapsf_MEC}}%
		\hfill
		\subfloat[ARWPM]{\includegraphics[width=0.24\textwidth]{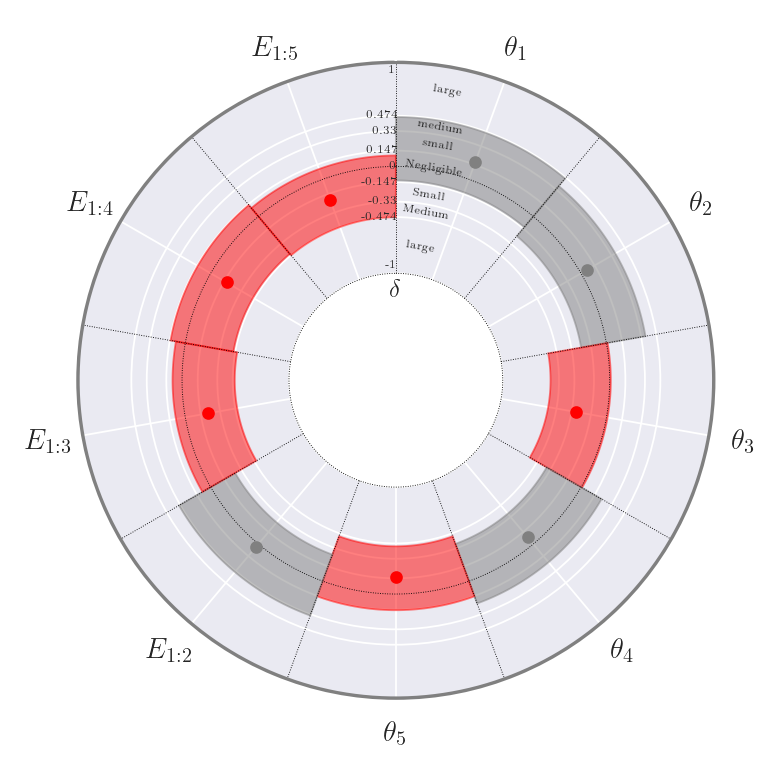}\label{fig:xgbcliffarwpm_MEC}}%
		\hfill
		\subfloat[GECR]{\includegraphics[width=0.24\textwidth]{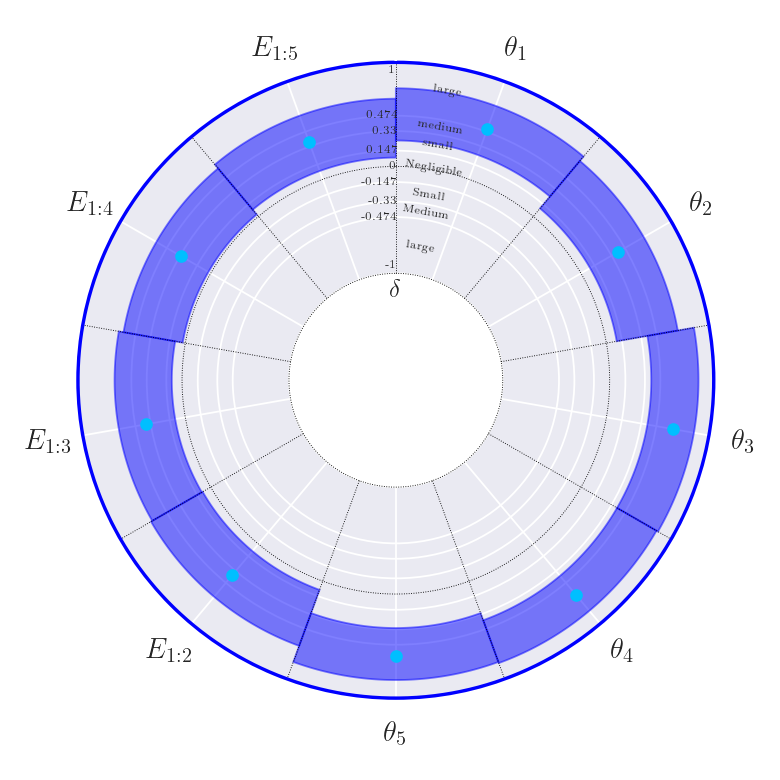}\label{fig:xgbcliffgecr_MEC}}%
		\hfill
		\subfloat[GFE]{\includegraphics[width=0.24\textwidth]{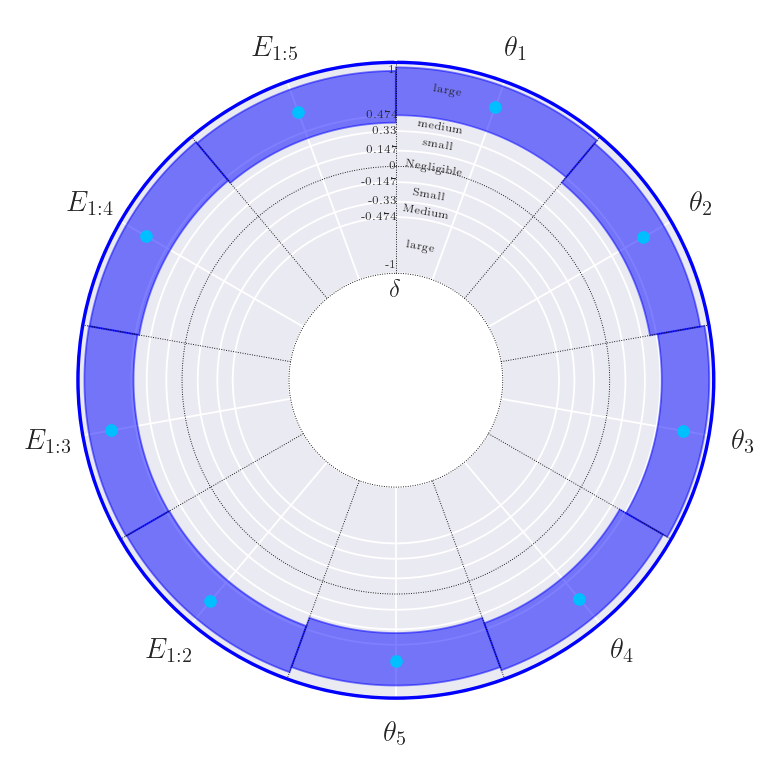}\label{fig:xgbcliffgfe_MEC}}
		
		\subfloat[GSAD]{\includegraphics[width=0.24\textwidth]{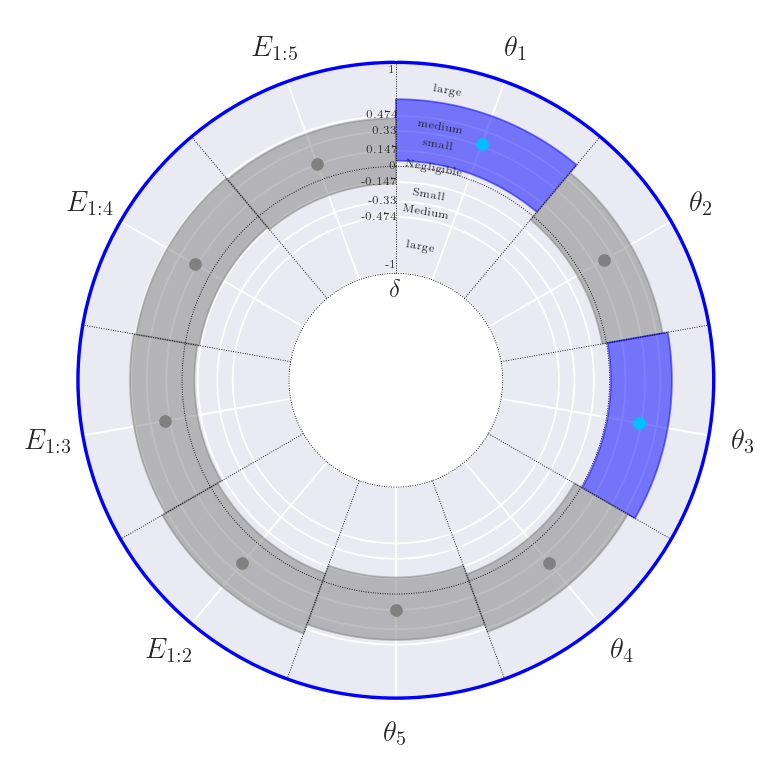}\label{fig:xgbcliffgsad_MEC}}%
		\hfill
		\subfloat[HAPT]{\includegraphics[width=0.24\textwidth]{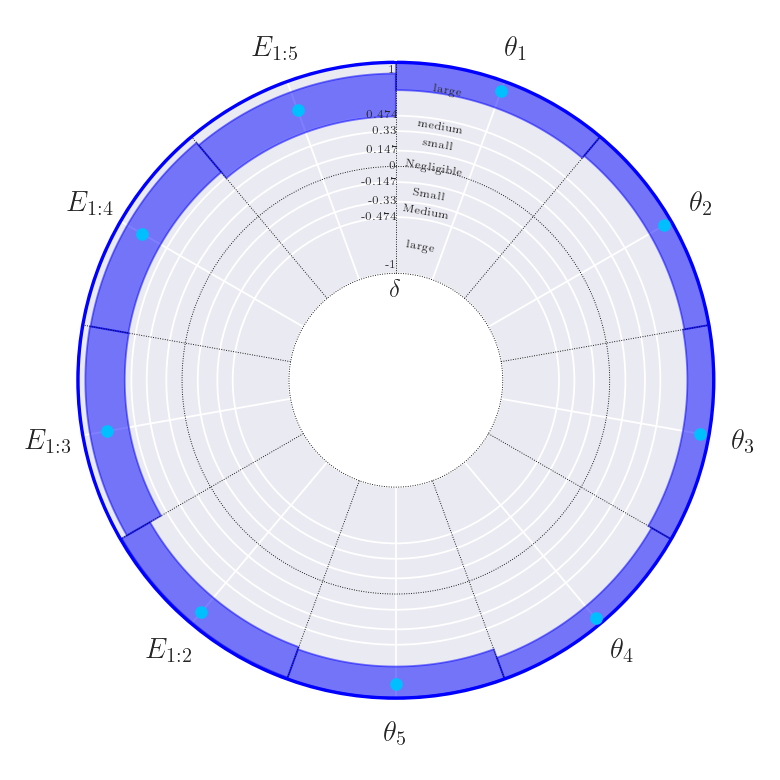}\label{fig:xgbcliffhapt_MEC}}%
		\hfill
		\subfloat[ISOLET]{\includegraphics[width=0.24\textwidth]{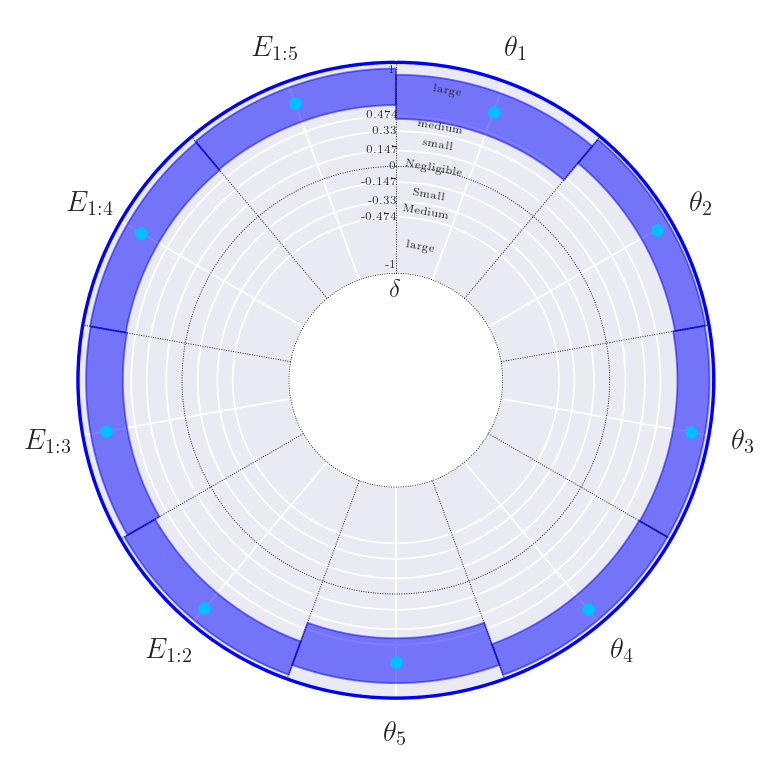}\label{fig:xgbcliffisolet_MEC}}%
		\hfill
		\subfloat[PD]{\includegraphics[width=0.24\textwidth]{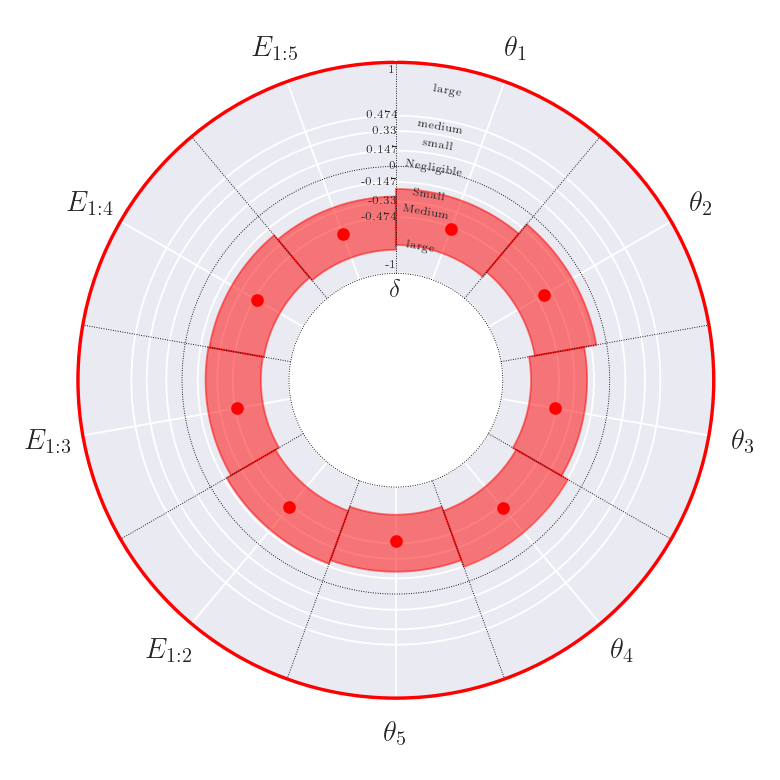}\label{fig:xgbcliffpd_MEC}}
		\caption[The Cliff's $\delta$ effect size measure and its 95\% confidence intervals for the MEC values obtained from 30 XGBoost runs.]{Effect size analysis of test data MEC across 30 XGBoost runs using Cliff's $\delta$. Each point represents the actual value obtained, with segments denoting 95\% confidence intervals based on 10,000 bootstrap resamplings. The outer ring color visualizes the statistical significance: grey illustrates no significant difference (adjusted Friedman's P-value$>0.05$), while color indicates significant differences; blue indicates at least one view and/or ensemble outperforms the benchmark (adjusted Conover's p-value$ < 0.05$, Cliff's $\delta > 0$), and red signifies all views and ensembles underperform relative to the benchmark (adjusted Conover's p-value$ < 0.05$, Cliff's $\delta < 0$). Segment colors show performance difference against the benchmark: grey for no significant difference (adjusted Conover's p-value$  > 0.05$), blue for better performance (Cliff's $\delta > 0$), and red for worse performance (Cliff's $\delta < 0$).}
		
		\label{fig:xgbcliff_MEC}
	\end{figure*}
	
	\begin{table*}[htbp]
		\centering
		\caption[The results of Friedman and Conover tests and Cliff's $\delta$ analysis for the MEC values obtained from 30 XGBoost runs.]{Statistical comparison of MEC results for testing data obtained from XGBoost runs. W, T, and L denote win, tie, and loss based on adjusted Friedman and Conover's p-values. Effect sizes are calculated using Cliff's Delta method and are categorized as negligible, small, medium, or large.}
		\label{tab:xgbmec}
		\resizebox{\linewidth}{!}{%
			\begin{tabular}{c|ccccccccc}
				\hline
				\multicolumn{10}{c}{XGBoost's MEC}\\
				\hline
				Dataset & $\theta_1$ & $\theta_2$ & $\theta_3$ & $\theta_4$ & $\theta_5$ & $E_{1:2}$ & $E_{1:3}$ & $E_{1:4}$ & $E_{1:5}$ \\
				\hline
				APSF  & T (negligible) & T (negligible) & T (small) & T (small) & T (negligible) & T (medium) & L (large) & L (large) & L (large) \\
				ARWPM  & T (small) & T (negligible) & L (small) & T (negligible) & L (small) & T (negligible) & L (small) & L (small) & L (small) \\
				GECR  & W (large) & W (medium) & W (large) & W (large) & W (large) & W (medium) & W (medium) & W (small) & W (medium) \\
				GFE  & W (large) & W (large) & W (large) & W (large) & W (large) & W (large) & W (large) & W (large) & W (large) \\
				GSAD  & W (medium) & T (small) & W (small) & T (small) & T (negligible) & T (small) & T (small) & T (small) & T (small) \\
				HAPT  & W (large) & W (large) & W (large) & W (large) & W (large) & W (large) & W (large) & W (large) & W (large) \\
				ISOLET  & W (large) & W (large) & W (large) & W (large) & W (large) & W (large) & W (large) & W (large) & W (large) \\
				PD  & L (large) & L (medium) & L (large) & L (medium) & L (large) & L (medium) & L (large) & L (large) & L (large) \\
				\hline
				W - T - L  & 5 - 2 - 1 & 4 - 3 - 1 & 5 - 1 - 2 & 4 - 3 - 1 & 4 - 2 - 2 & 4 - 3 - 1 & 4 - 1 - 3 & 4 - 1 - 3 & 4 - 1 - 3 \\
				\hline
			\end{tabular}
		}
	\end{table*}
	\FloatBarrier

	\begin{figure*}[t] 
		\centering
		\subfloat[APSF]{\includegraphics[width=0.24\textwidth]{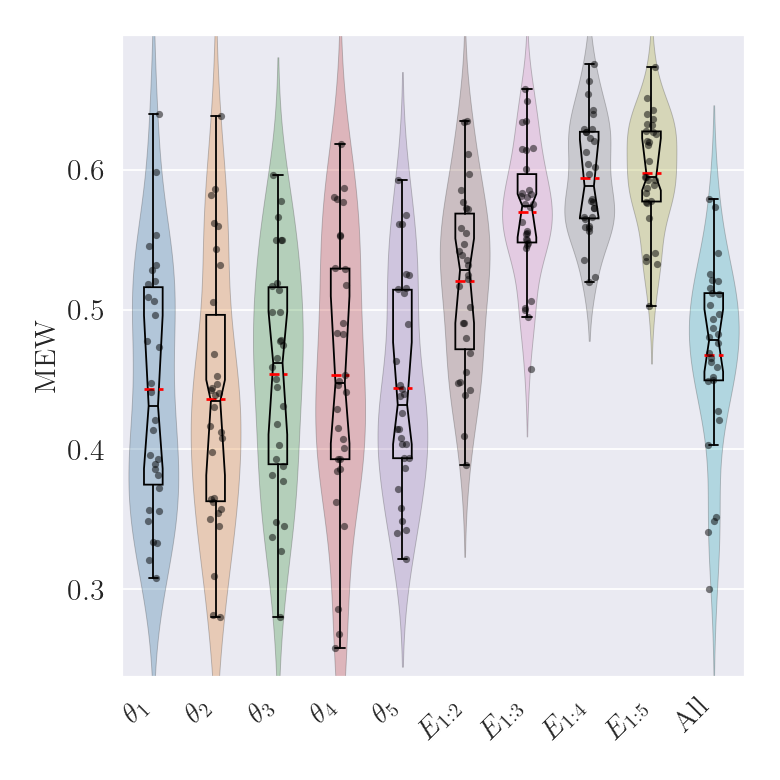}\label{fig:xgbapsf_MEW}}%
		\hfill
		\subfloat[ARWPM]{\includegraphics[width=0.24\textwidth]{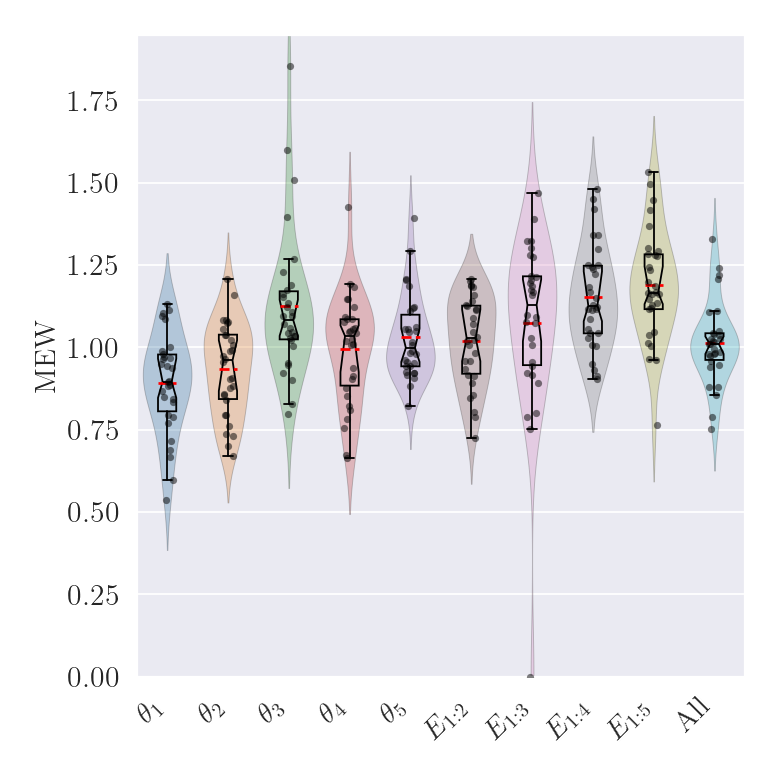}\label{fig:xgbarwpm_MEW}}%
		\hfill
		\subfloat[GECR]{\includegraphics[width=0.24\textwidth]{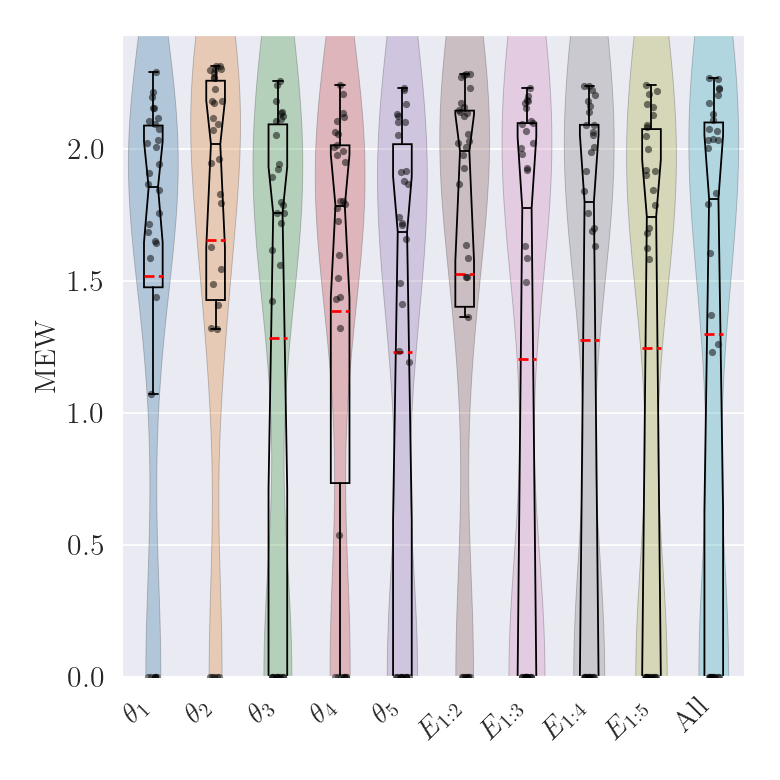}\label{fig:xgbgecr_MEW}}%
		\hfill
		\subfloat[GFE]{\includegraphics[width=0.24\textwidth]{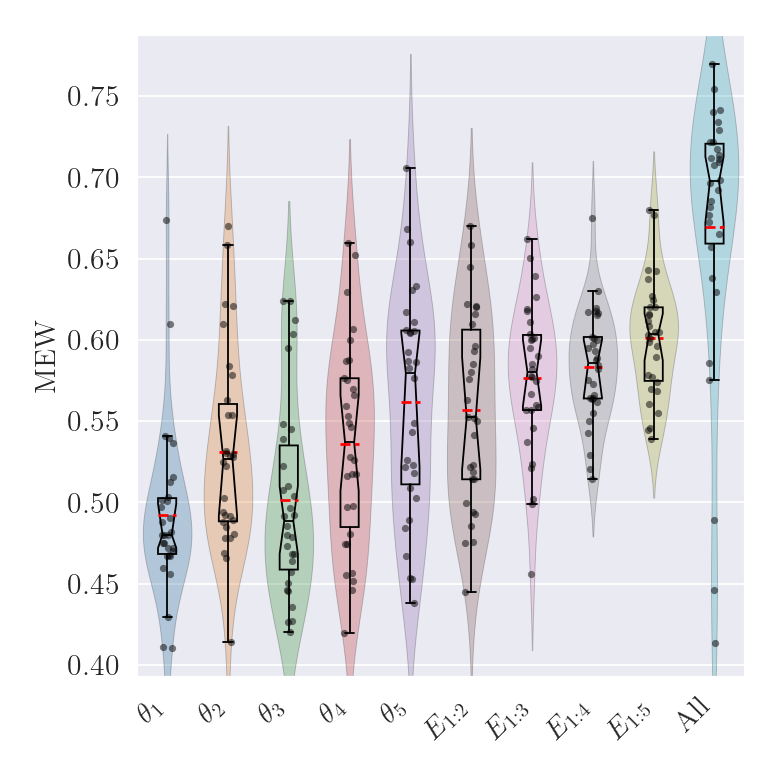}\label{fig:xgbgfe_MEW}}
		
		\subfloat[GSAD]{\includegraphics[width=0.24\textwidth]{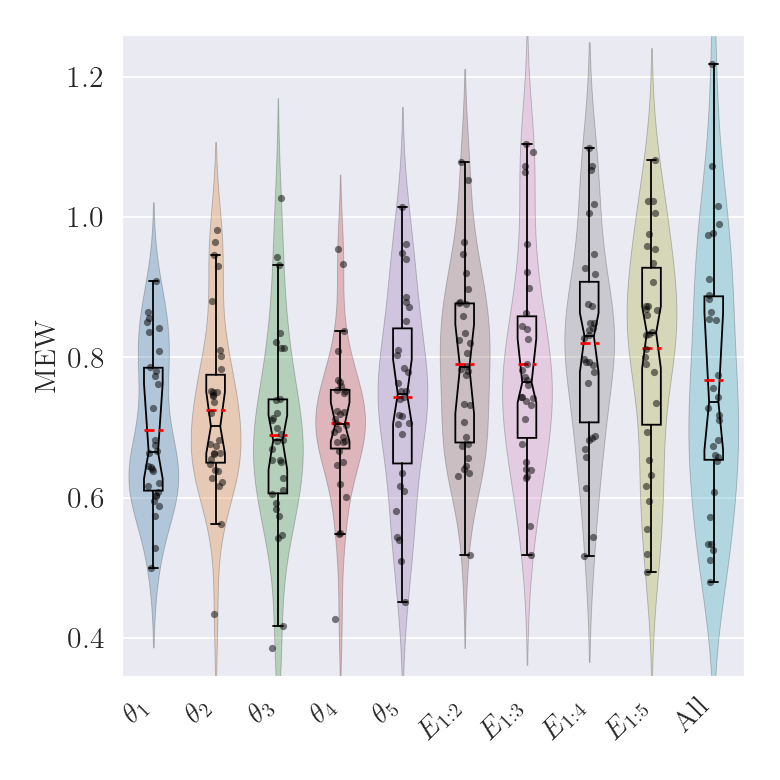}\label{fig:fpgsad_MEW}}%
		\hfill
		\subfloat[HAPT]{\includegraphics[width=0.24\textwidth]{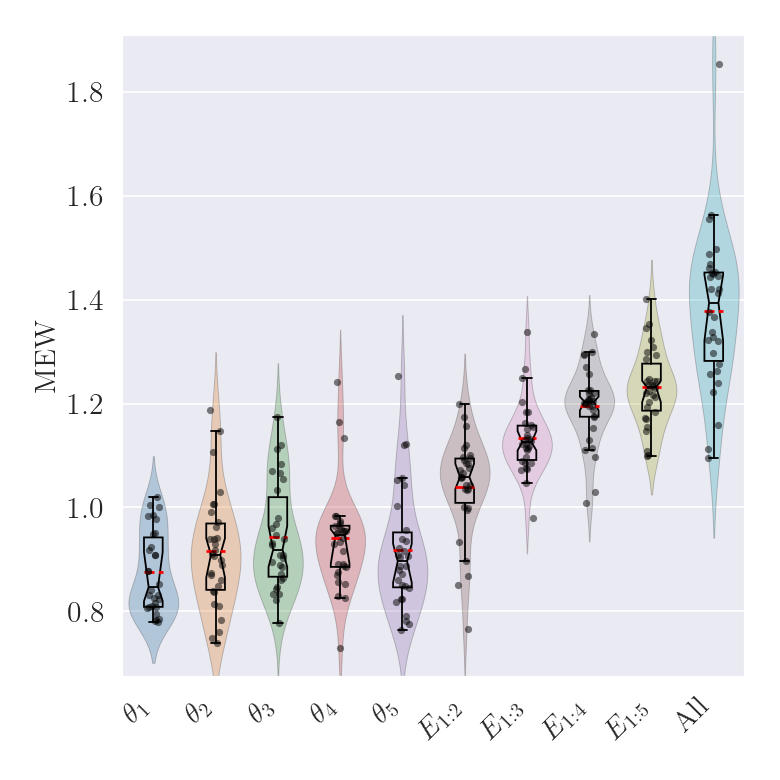}\label{fig:xgbhapt_MEW}}%
		\hfill
		\subfloat[ISOLET]{\includegraphics[width=0.24\textwidth]{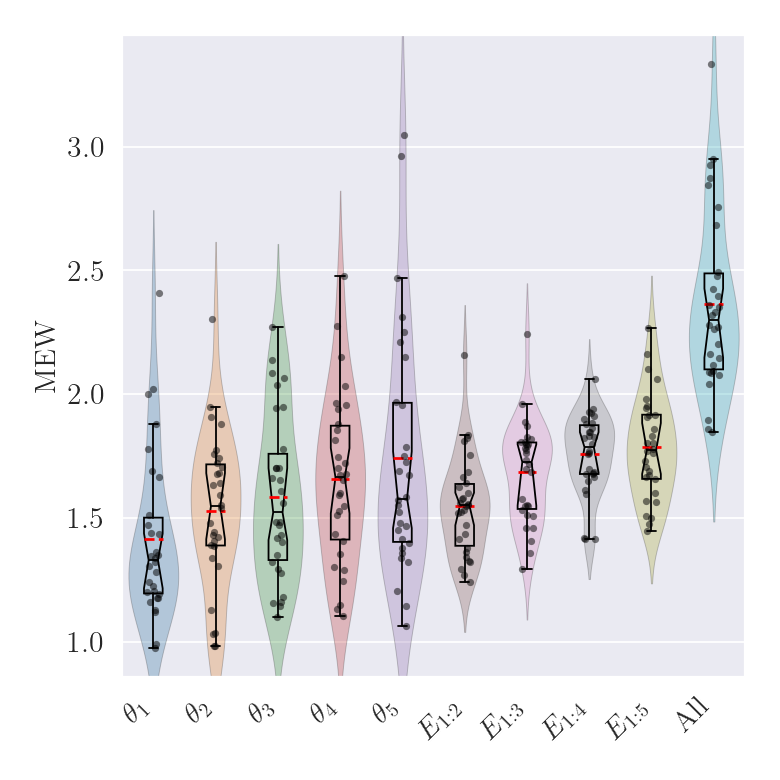}\label{fig:xgbisolet_MEW}}%
		\hfill
		\subfloat[PD]{\includegraphics[width=0.24\textwidth]{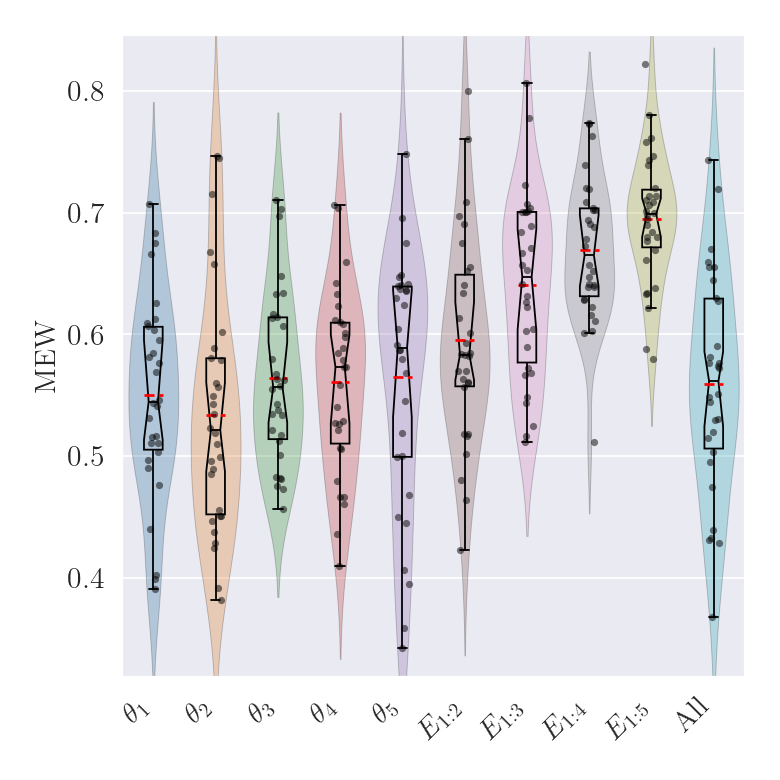}\label{fig:xgbpd_MEW}}
		\caption[The distribution of the obtained MEW values for 30 XGBoost runs.]{The raincloud plot of MEW results obtained from 30 XGBoost runs.}
		
		\label{fig:xgb_MEW}
	\end{figure*}
	
	\begin{figure*}[t] 
		\centering
		\subfloat[APSF]{\includegraphics[width=0.24\textwidth]{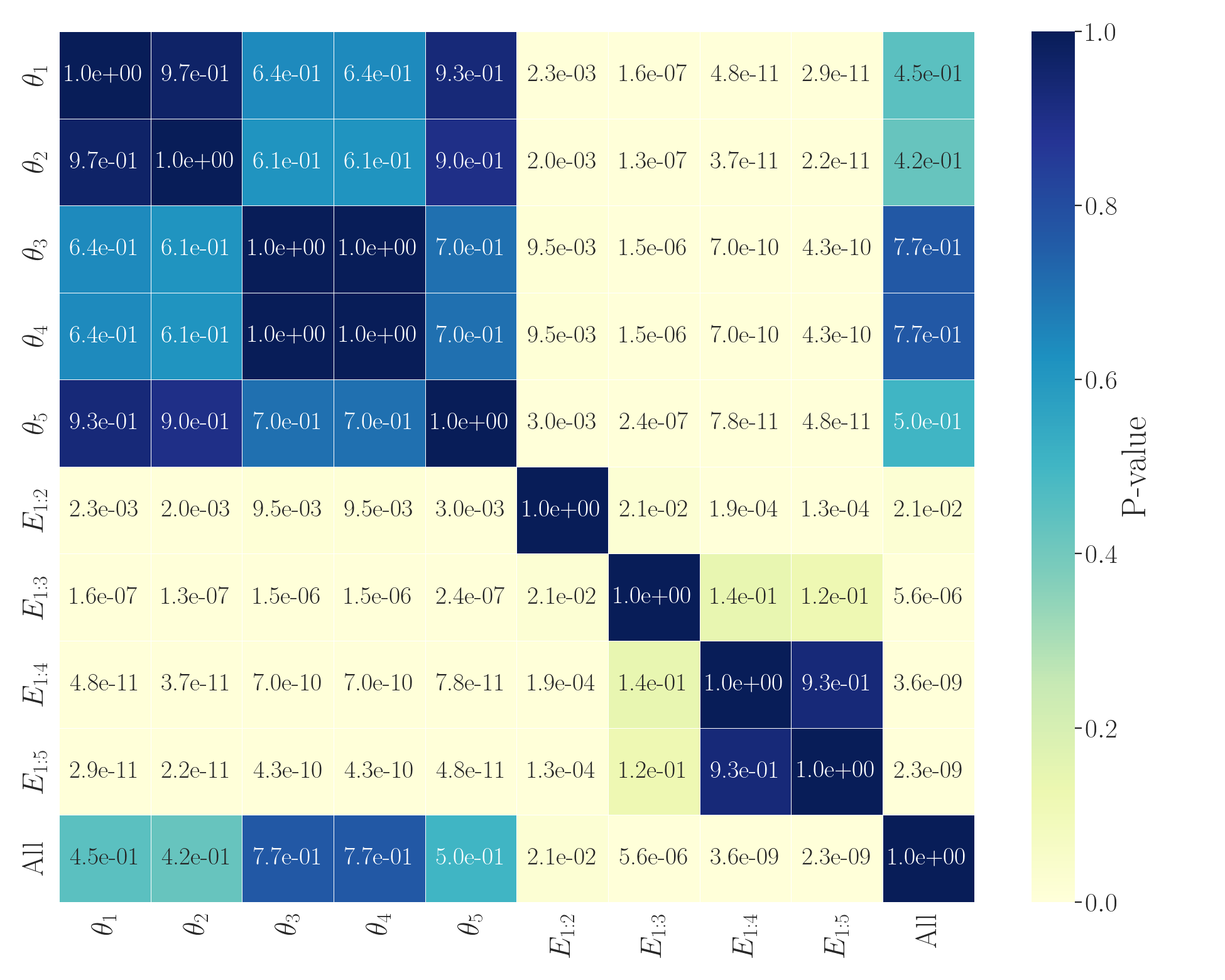}\label{fig:xgbnemapsf_MEW}}%
		\hfill
		\subfloat[ARWPM]{\includegraphics[width=0.24\textwidth]{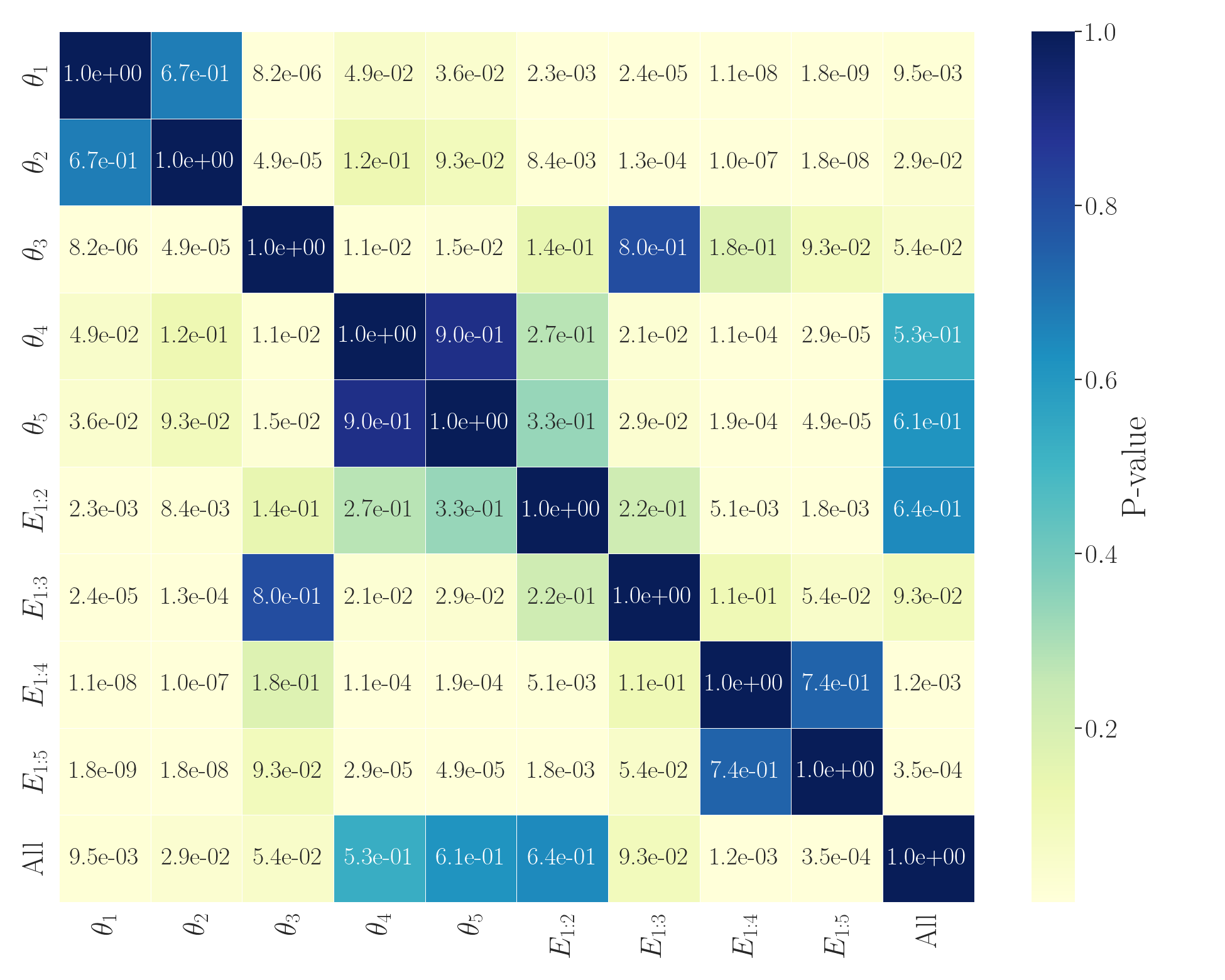}\label{fig:xgbnemarwpm_MEW}}%
		\hfill
		\subfloat[GECR]{\includegraphics[width=0.24\textwidth]{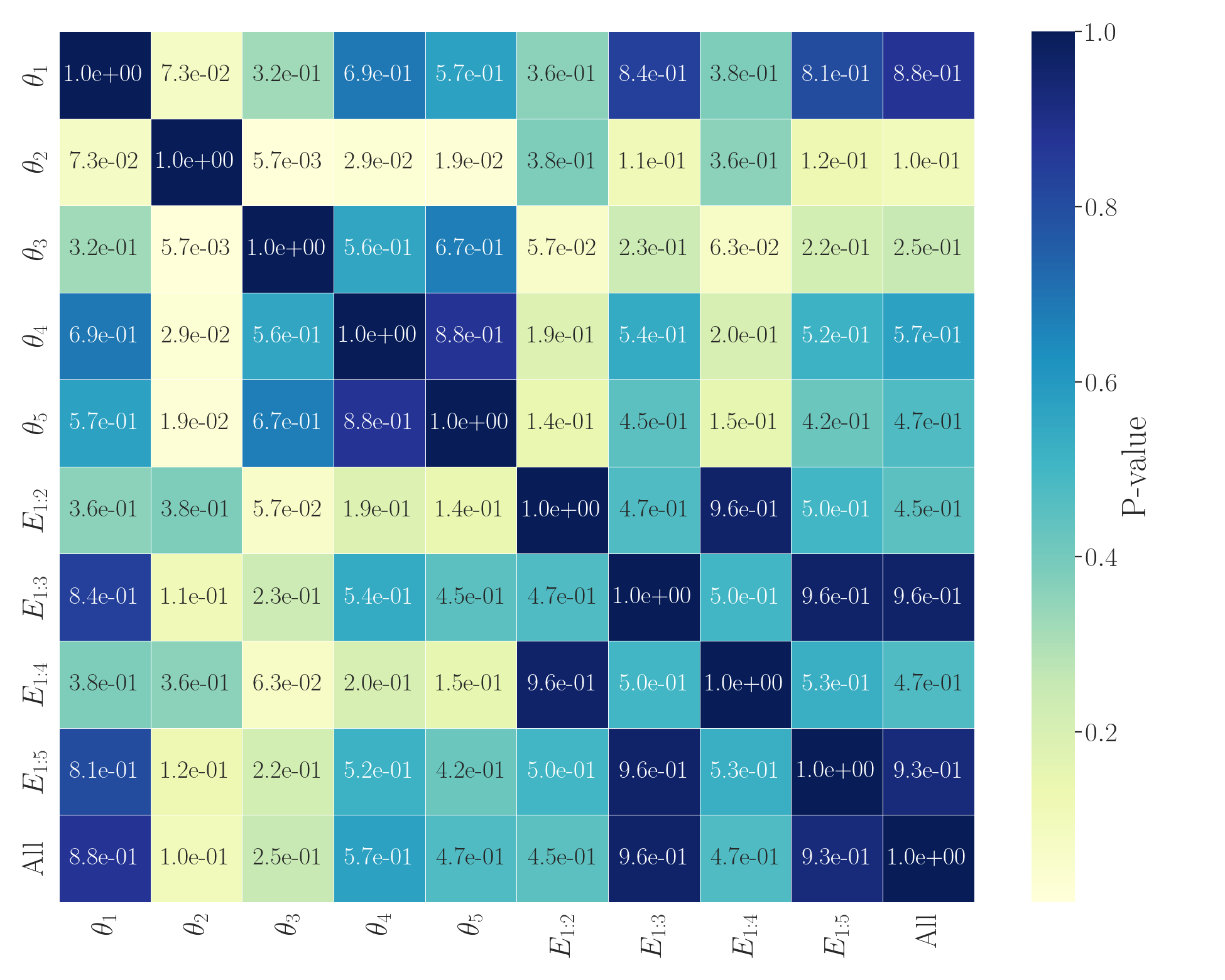}\label{fig:xgbnemgecr_MEW}}%
		\hfill
		\subfloat[GFE]{\includegraphics[width=0.24\textwidth]{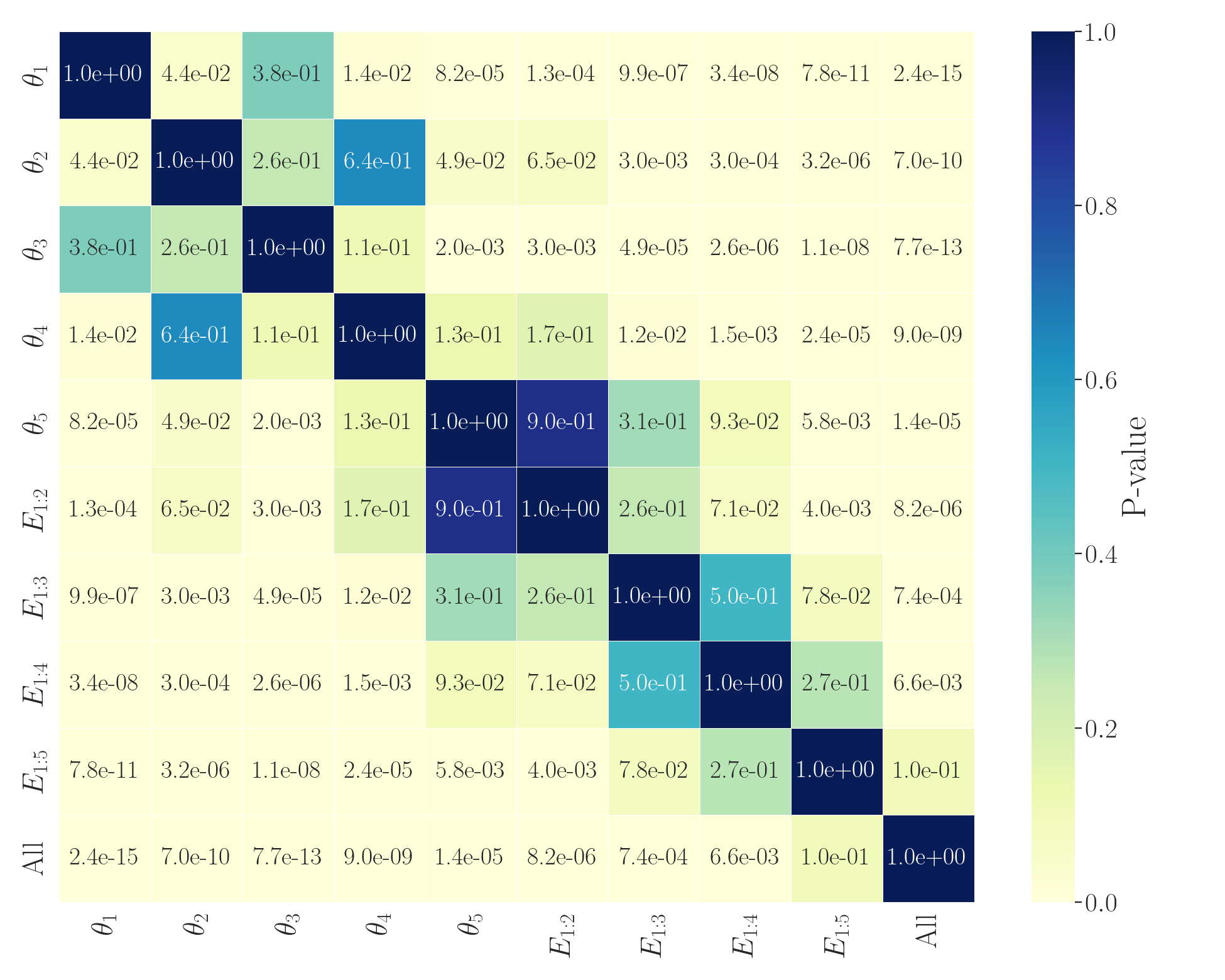}\label{fig:xgbnemgfe_MEW}}
		
		\subfloat[GSAD]{\includegraphics[width=0.24\textwidth]{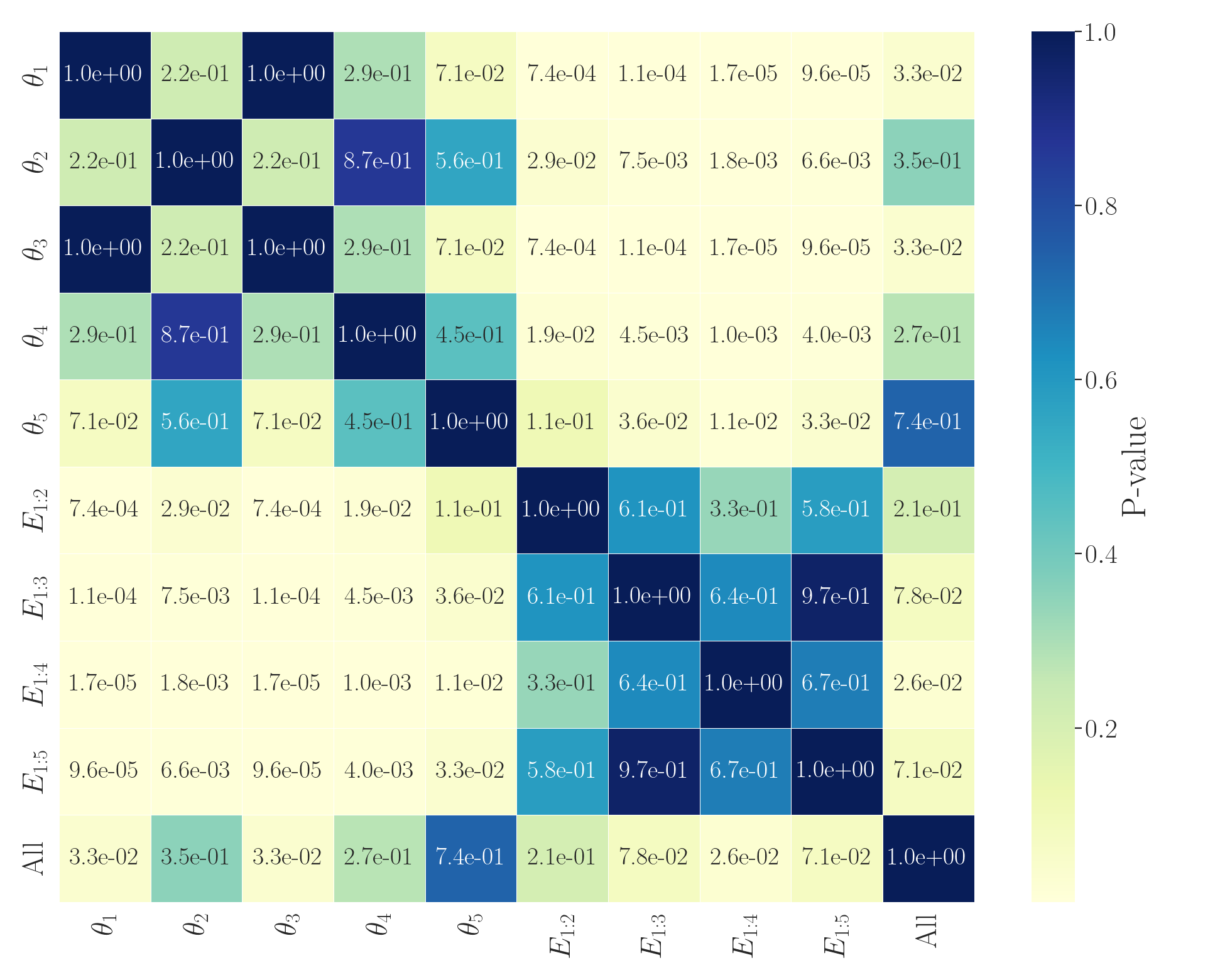}\label{fig:xgbnemgsad_MEW}}%
		\hfill
		\subfloat[HAPT]{\includegraphics[width=0.24\textwidth]{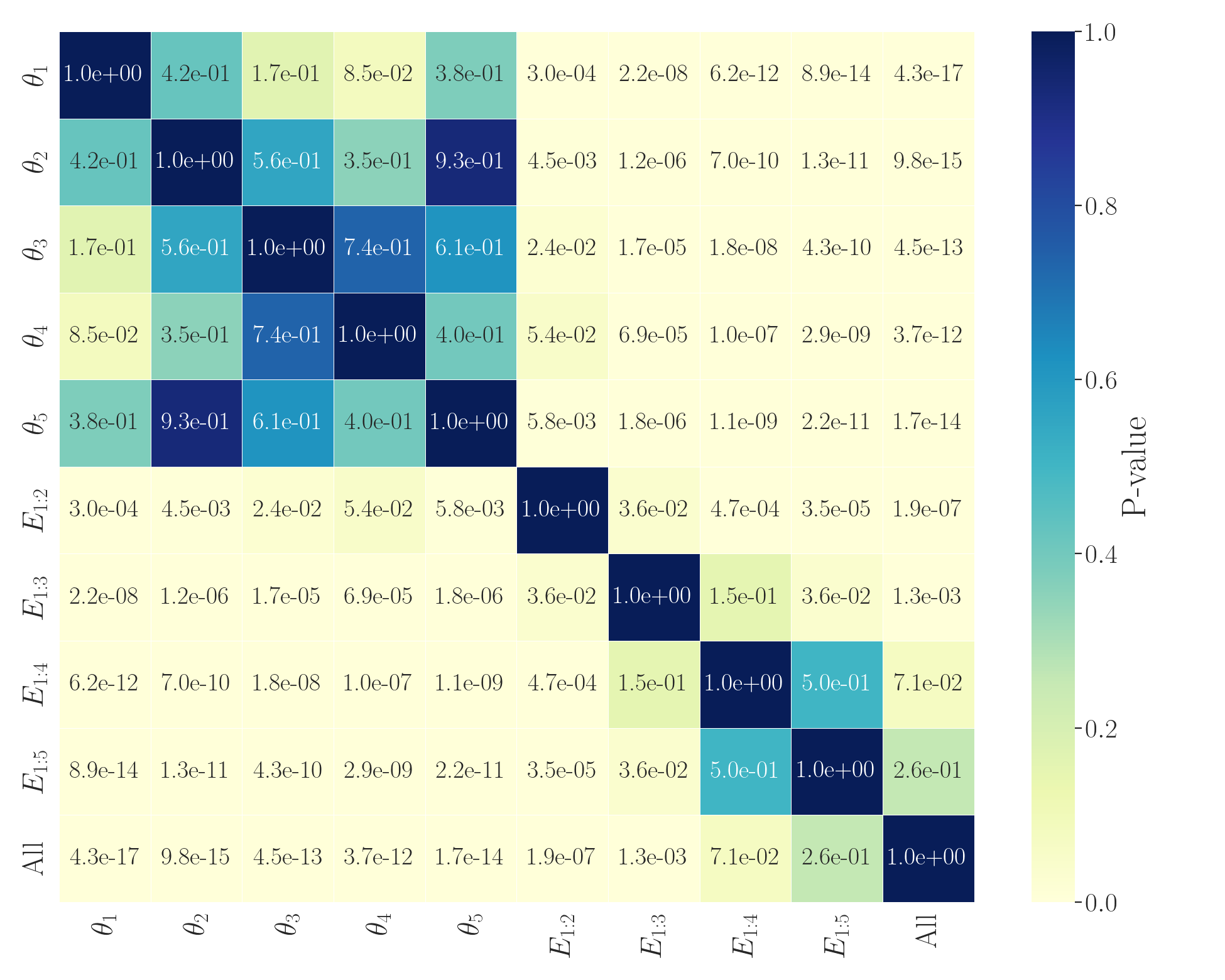}\label{fig:xgbnemhapt_MEW}}%
		\hfill
		\subfloat[ISOLET]{\includegraphics[width=0.24\textwidth]{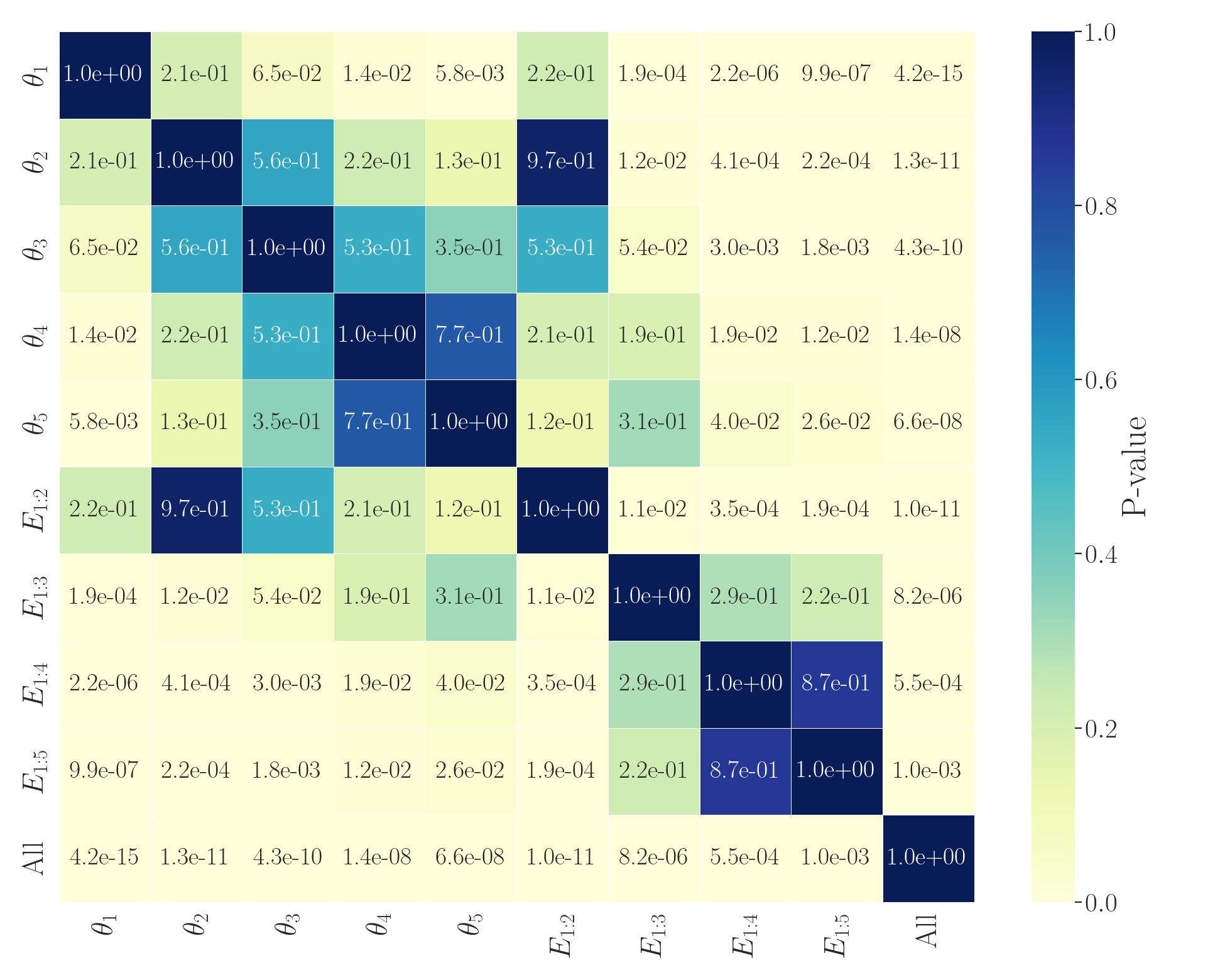}\label{fig:xgbnemisolet_MEW}}%
		\hfill
		\subfloat[PD]{\includegraphics[width=0.24\textwidth]{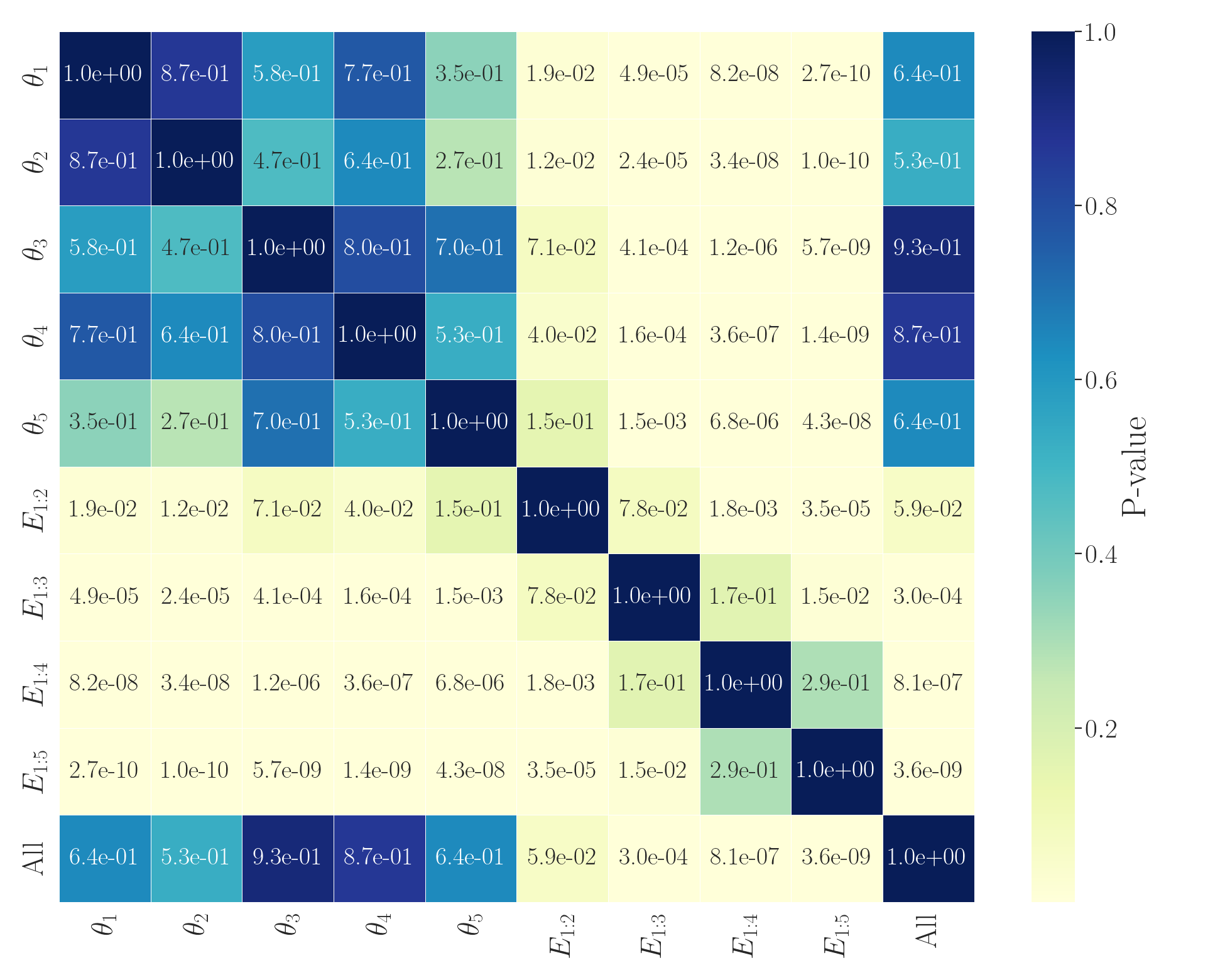}\label{fig:xgbnempd_MEW}}
		\caption[The adjusted Conover's P-values for the obtained MEW values from 30 XGBoost runs.]{The results of the Conover post-hoc test on testing data’s MEW obtained from 30 XGBoost runs.}
		
		\label{fig:xgbnem_MEW}
	\end{figure*}
	\FloatBarrier
	
	\begin{figure*}[htbp] 
		\centering
		\subfloat[APSF]{\includegraphics[width=0.24\textwidth]{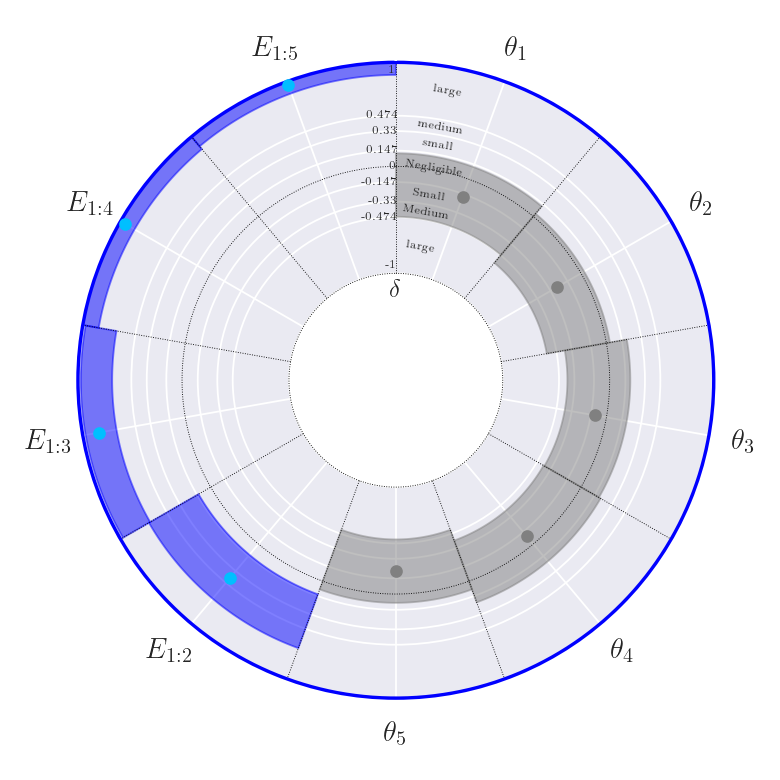}\label{fig:xgbcliffapsf_MEW}}%
		\hfill
		\subfloat[ARWPM]{\includegraphics[width=0.24\textwidth]{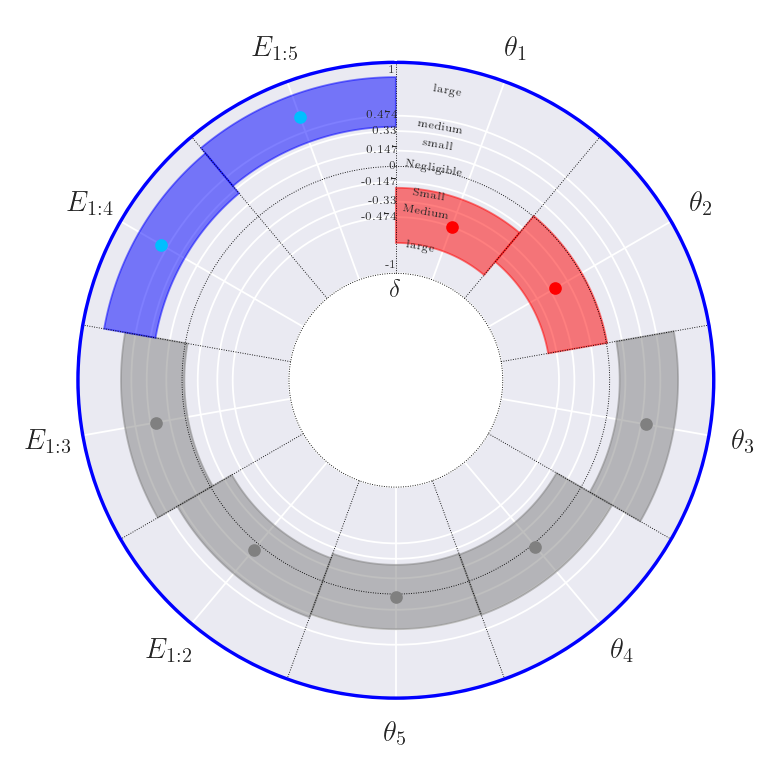}\label{fig:xgbcliffarwpm_MEW}}%
		\hfill
		\subfloat[GECR]{\includegraphics[width=0.24\textwidth]{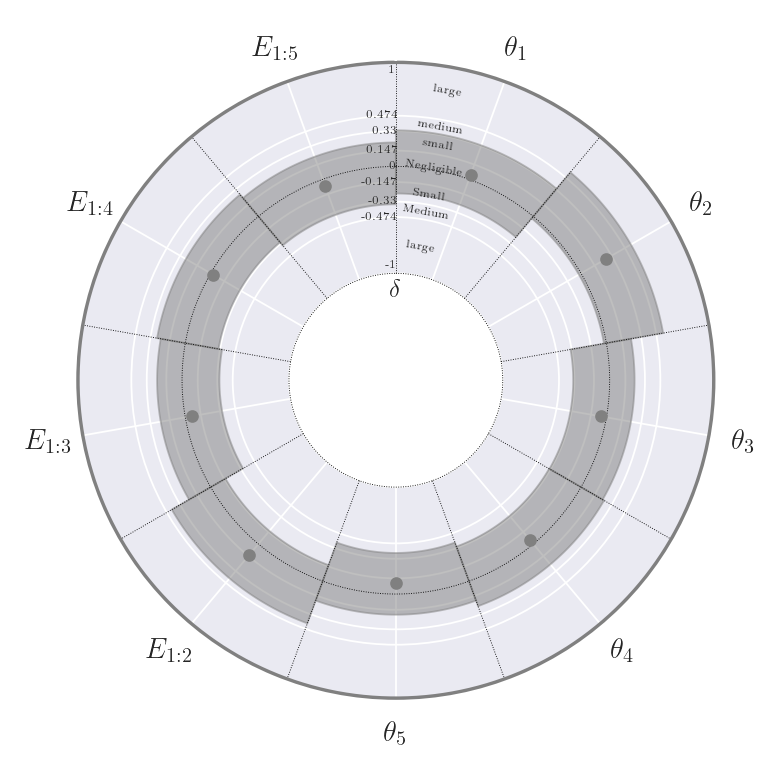}\label{fig:xgbcliffgecr_MEW}}%
		\hfill
		\subfloat[GFE]{\includegraphics[width=0.24\textwidth]{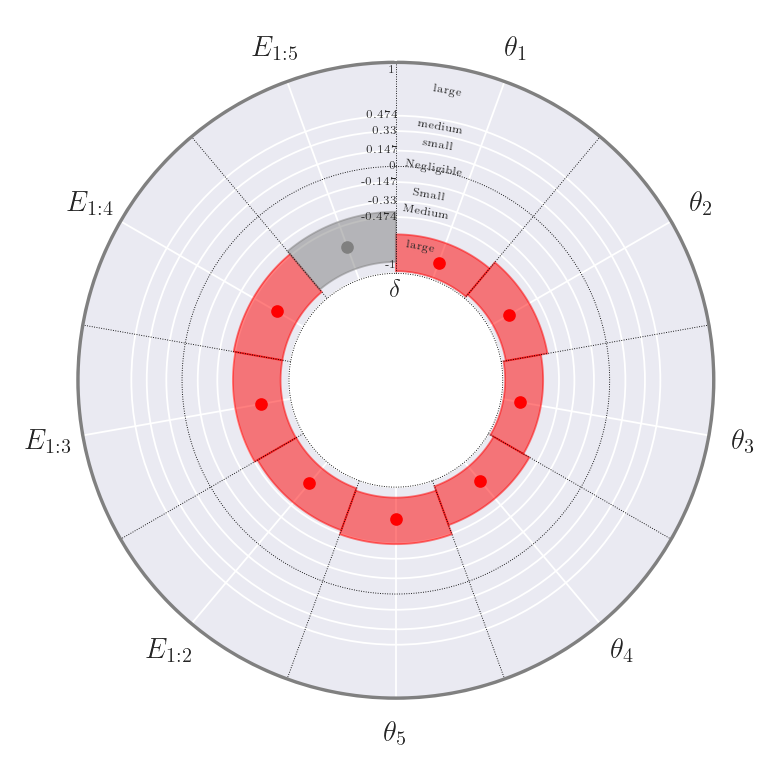}\label{fig:xgbcliffgfe_MEW}}
		
		\subfloat[GSAD]{\includegraphics[width=0.24\textwidth]{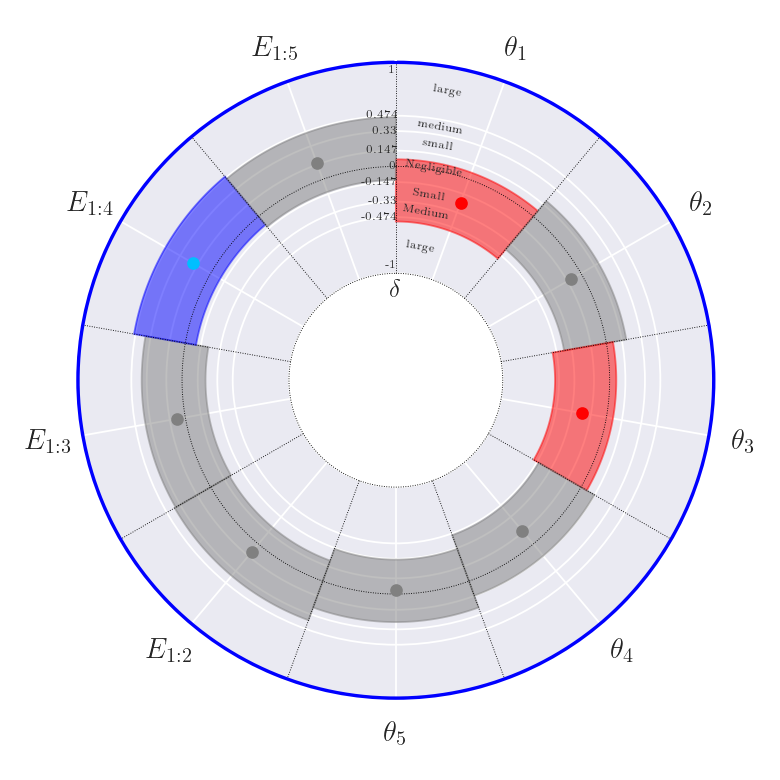}\label{fig:xgbcliffgsad_MEW}}%
		\hfill
		\subfloat[HAPT]{\includegraphics[width=0.24\textwidth]{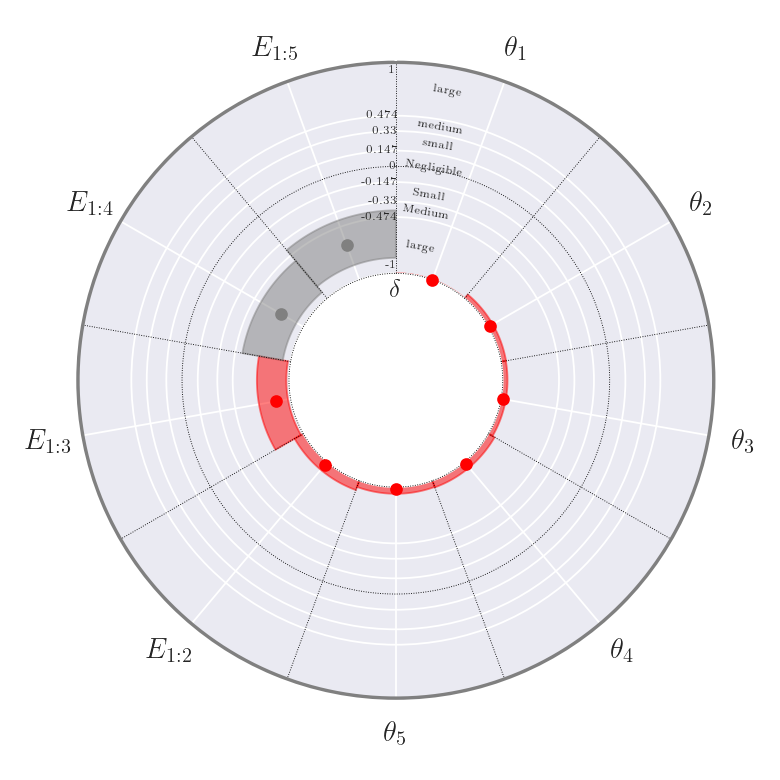}\label{fig:xgbcliffhapt_MEW}}%
		\hfill
		\subfloat[ISOLET]{\includegraphics[width=0.24\textwidth]{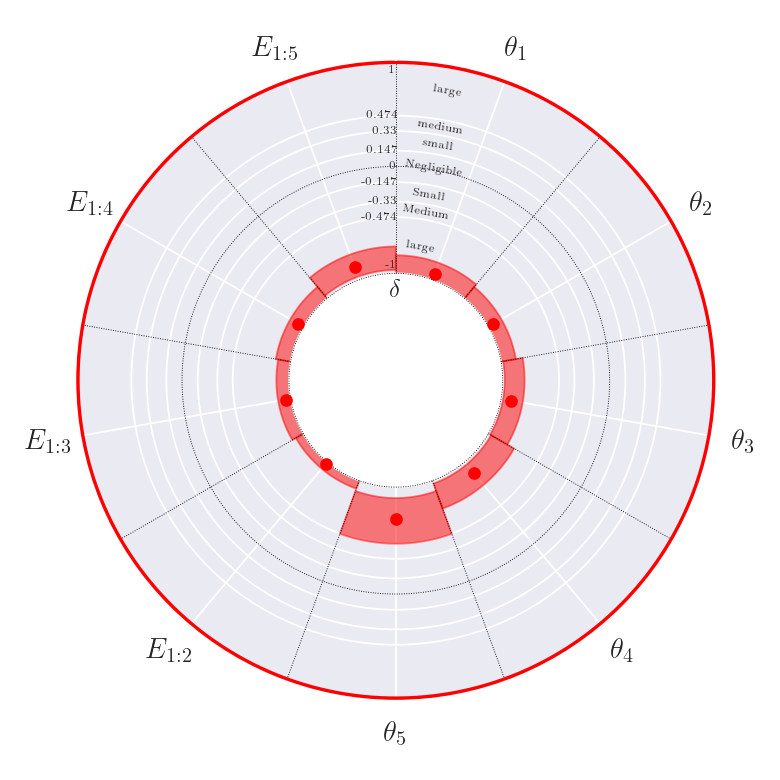}\label{fig:xgbcliffisolet_MEW}}%
		\hfill
		\subfloat[PD]{\includegraphics[width=0.24\textwidth]{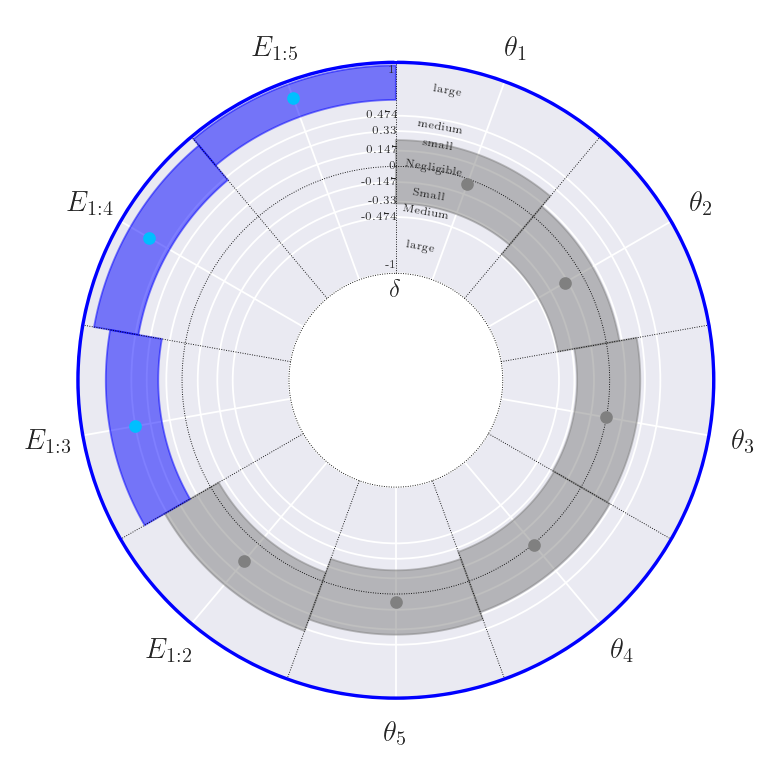}\label{fig:xgbcliffpd_MEW}}
		\caption[The Cliff's $\delta$ effect size measure and its 95\% confidence intervals for the MEW values obtained from 30 XGBoost runs.]{Effect size analysis of test data MEW across 30 XGBoost runs using Cliff's $\delta$. Each point represents the actual value obtained, with segments denoting 95\% confidence intervals based on 10,000 bootstrap resamplings. The outer ring color visualizes the statistical significance: grey illustrates no significant difference (adjusted Friedman's P-value$>0.05$), while color indicates significant differences; blue indicates at least one view and/or ensemble outperforms the benchmark (adjusted Conover's p-value$ < 0.05$, Cliff's $\delta > 0$), and red signifies all views and ensembles underperform relative to the benchmark (adjusted Conover's p-value$ < 0.05$, Cliff's $\delta < 0$). Segment colors show performance difference against the benchmark: grey for no significant difference (adjusted Conover's p-value$  > 0.05$), blue for better performance (Cliff's $\delta > 0$), and red for worse performance (Cliff's $\delta < 0$).}
		
		\label{fig:xgbcliff_MEW}
	\end{figure*}
	
	\begin{table*}[htbp]
		\centering
		\caption[The results of Friedman and Conover tests and Cliff's $\delta$ analysis for the MEW values obtained from 30 XGBoost runs.]{Statistical comparison of MEW results for testing data obtained from XGBoost runs. W, T, and L denote win, tie, and loss based on adjusted Friedman and Conover's p-values. Effect sizes are calculated using Cliff's Delta method and are categorized as negligible, small, medium, or large.}
		\label{tab:xgbmew}
		\resizebox{\linewidth}{!}{%
			\begin{tabular}{c|ccccccccc}
				\hline
				\multicolumn{10}{c}{XGBoost's MEW}\\
				\hline
				Dataset & $\theta_1$ & $\theta_2$ & $\theta_3$ & $\theta_4$ & $\theta_5$ & $E_{1:2}$ & $E_{1:3}$ & $E_{1:4}$ & $E_{1:5}$ \\
				\hline
				APSF  & T (small) & T (small) & T (negligible) & T (negligible) & T (small) & W (medium) & W (large) & W (large) & W (large) \\
				ARWPM  & L (medium) & L (small) & T (medium) & T (negligible) & T (negligible) & T (negligible) & T (small) & W (large) & W (large) \\
				GECR  & T (negligible) & T (small) & T (negligible) & T (negligible) & T (negligible) & T (negligible) & T (negligible) & T (negligible) & T (negligible) \\
				GFE  & L (large) & L (large) & L (large) & L (large) & L (large) & L (large) & L (large) & L (large) & T (large) \\
				GSAD  & L (small) & T (negligible) & L (small) & T (small) & T (negligible) & T (negligible) & T (negligible) & W (small) & T (small) \\
				HAPT  & L (large) & L (large) & L (large) & L (large) & L (large) & L (large) & L (large) & T (large) & T (large) \\
				ISOLET  & L (large) & L (large) & L (large) & L (large) & L (large) & L (large) & L (large) & L (large) & L (large) \\
				PD  & T (negligible) & T (small) & T (negligible) & T (negligible) & T (negligible) & T (small) & W (large) & W (large) & W (large) \\
				\hline
				W - T - L  & 0 - 3 - 5 & 0 - 4 - 4 & 0 - 4 - 4 & 0 - 5 - 3 & 0 - 5 - 3 & 1 - 4 - 3 & 2 - 3 - 3 & 4 - 2 - 2 & 3 - 4 - 1 \\
				\hline
			\end{tabular}
		}
	\end{table*}
	\FloatBarrier

	\begin{figure*}[t] 
		\centering
		\subfloat[APSF]{\includegraphics[width=0.24\textwidth]{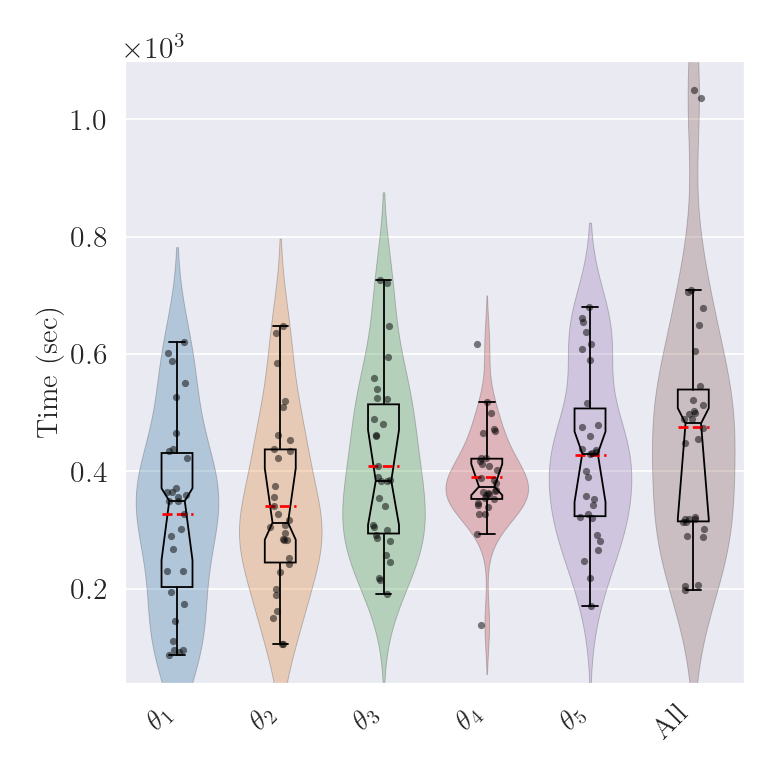}\label{fig:xgbapsf_Time}}%
		\hfill
		\subfloat[ARWPM]{\includegraphics[width=0.24\textwidth]{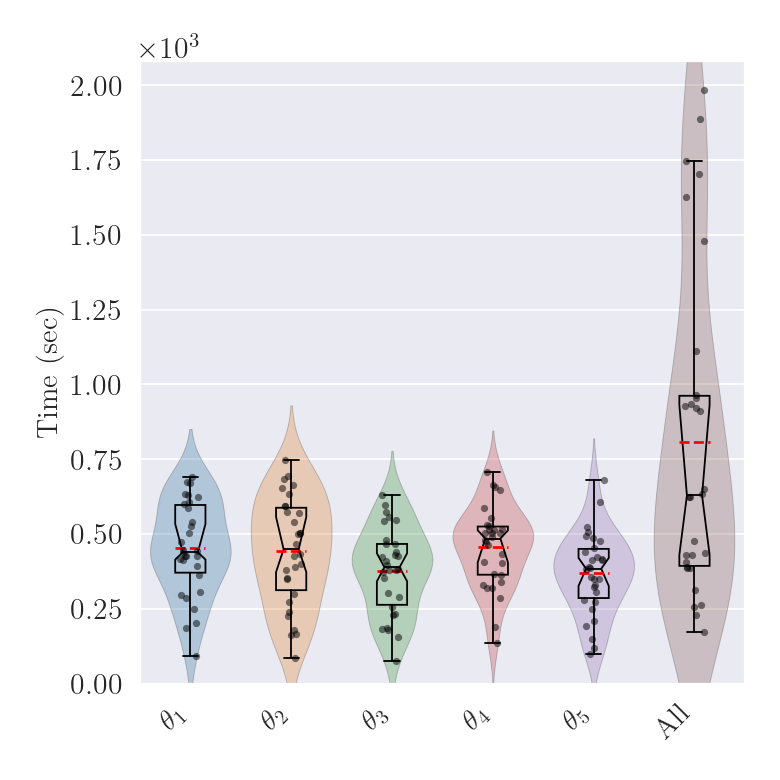}\label{fig:xgbarwpm_Time}}%
		\hfill
		\subfloat[GECR]{\includegraphics[width=0.24\textwidth]{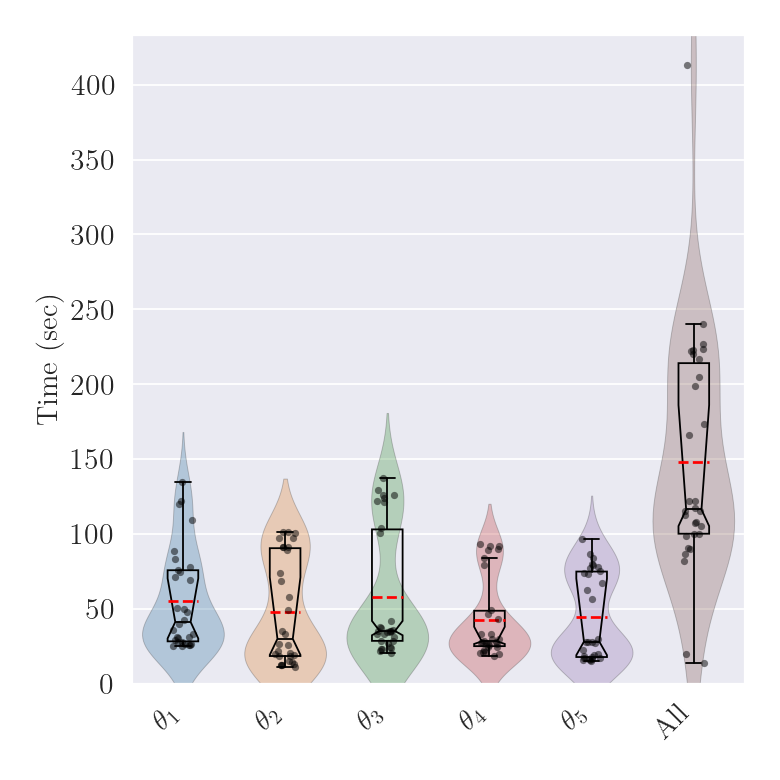}\label{fig:xgbgecr_Time}}%
		\hfill
		\subfloat[GFE]{\includegraphics[width=0.24\textwidth]{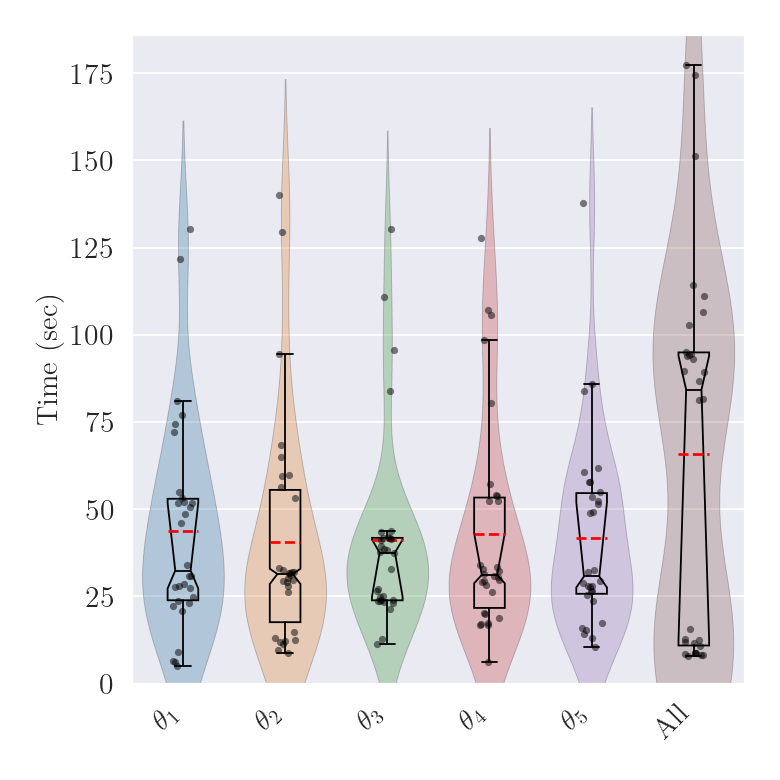}\label{fig:xgbgfe_Time}}
		
		\subfloat[GSAD]{\includegraphics[width=0.24\textwidth]{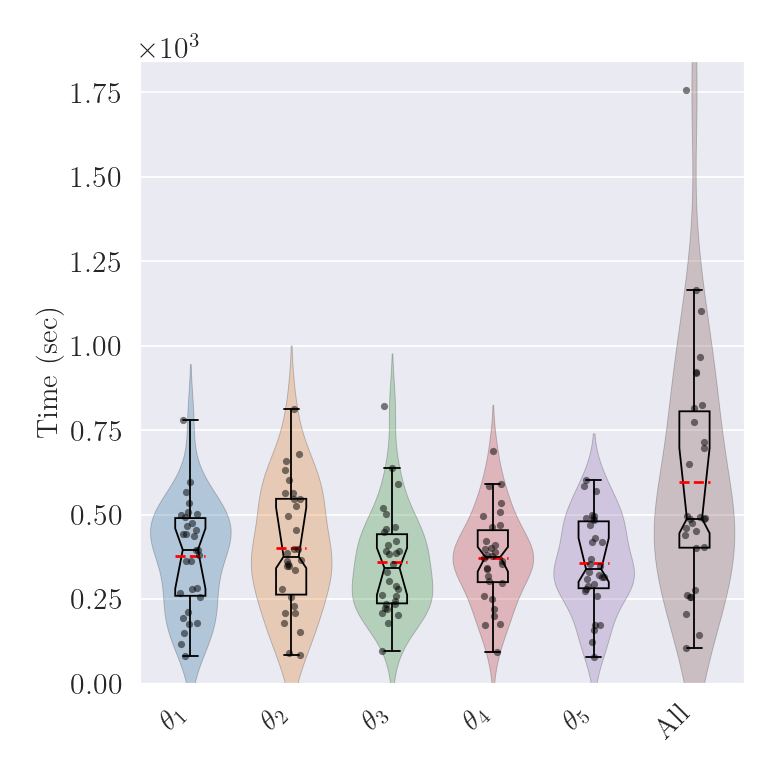}\label{fig:fpgsad_Time}}%
		\hfill
		\subfloat[HAPT]{\includegraphics[width=0.24\textwidth]{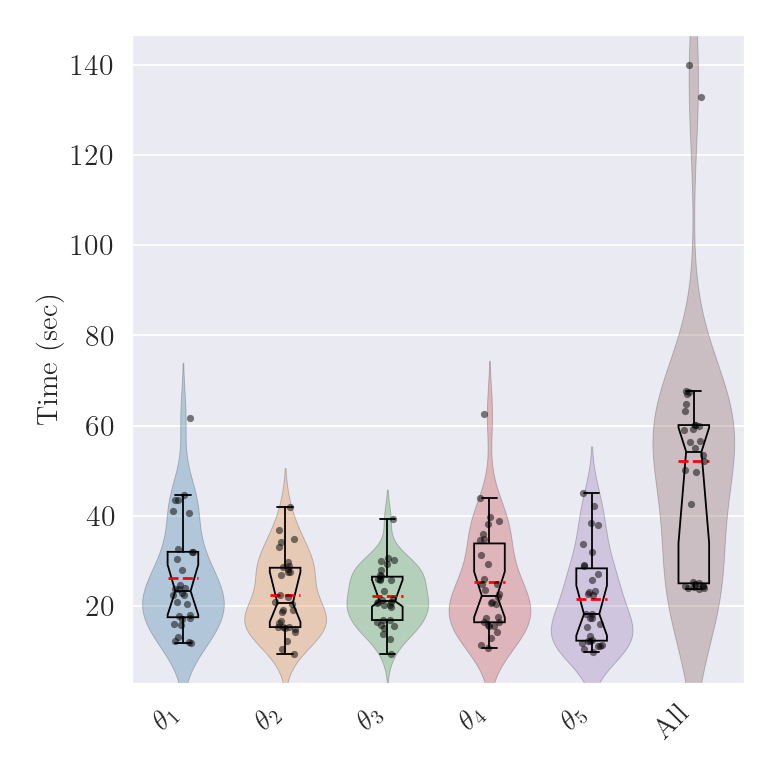}\label{fig:xgbhapt_Time}}%
		\hfill
		\subfloat[ISOLET]{\includegraphics[width=0.24\textwidth]{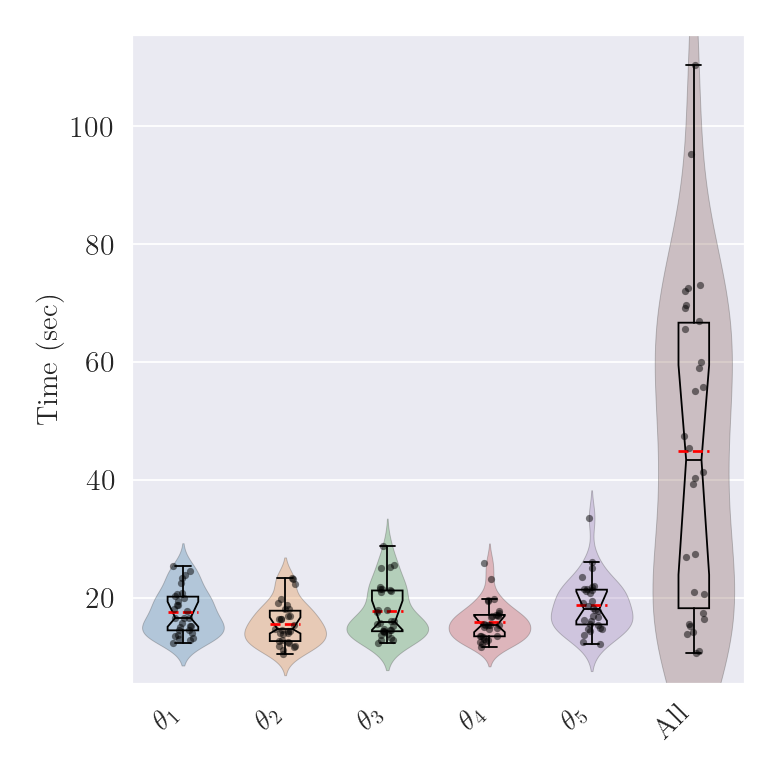}\label{fig:xgbisolet_Time}}%
		\hfill
		\subfloat[PD]{\includegraphics[width=0.24\textwidth]{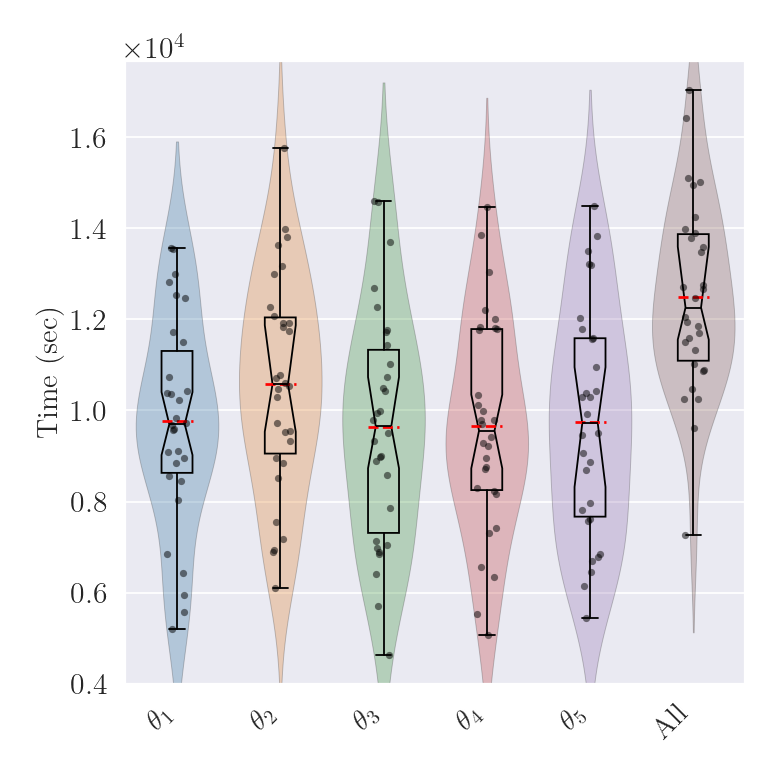}\label{fig:xgbpd_Time}}
		\caption[The distribution of the obtained running time for 30 XGBoost runs.]{The raincloud plot of running time results obtained from 30 XGBoost runs.}
		
		\label{fig:xgb_Time}
	\end{figure*}
	
	\begin{figure*}[t] 
		\centering
		\subfloat[APSF]{\includegraphics[width=0.24\textwidth]{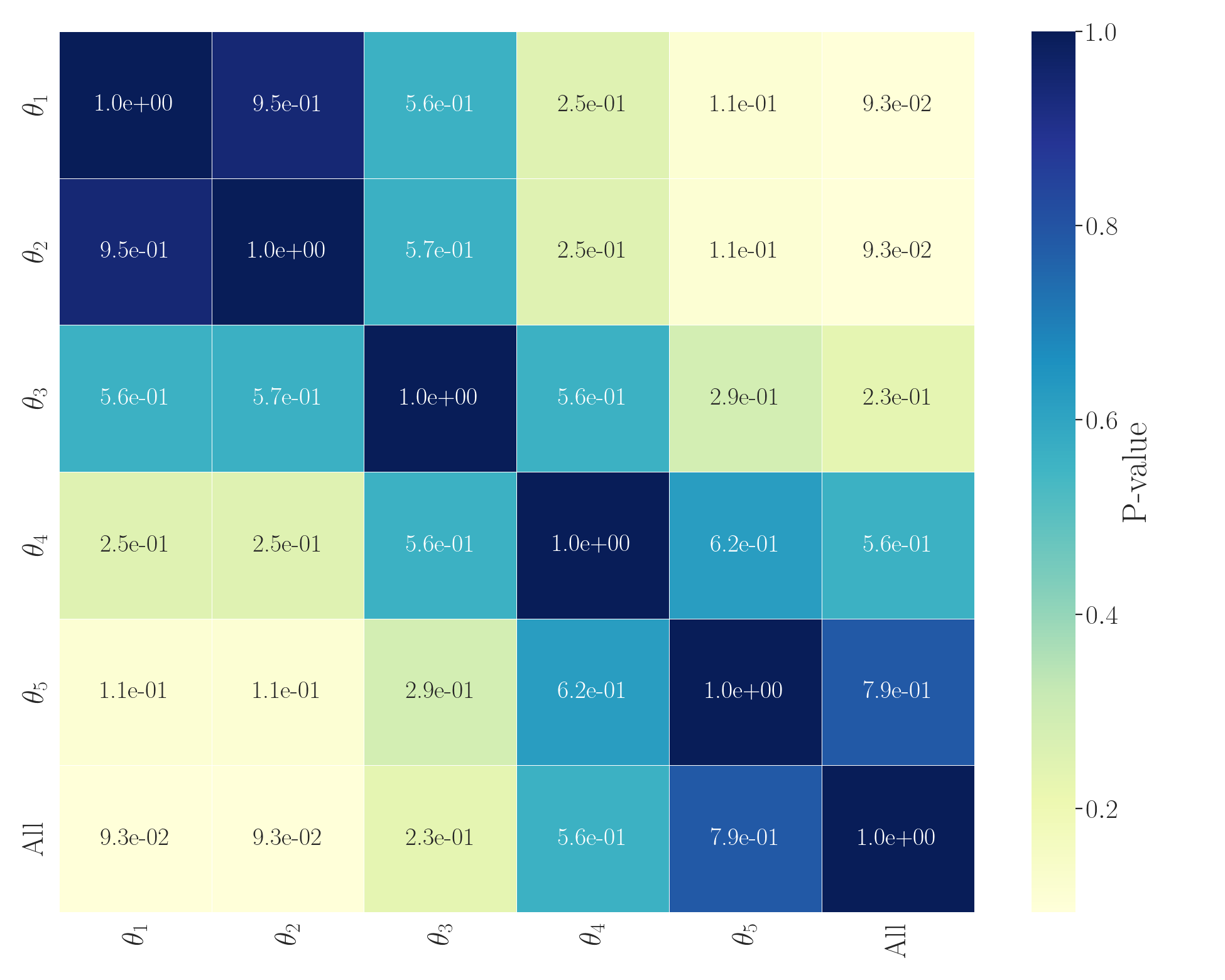}\label{fig:xgbnemapsf_Time}}%
		\hfill
		\subfloat[ARWPM]{\includegraphics[width=0.24\textwidth]{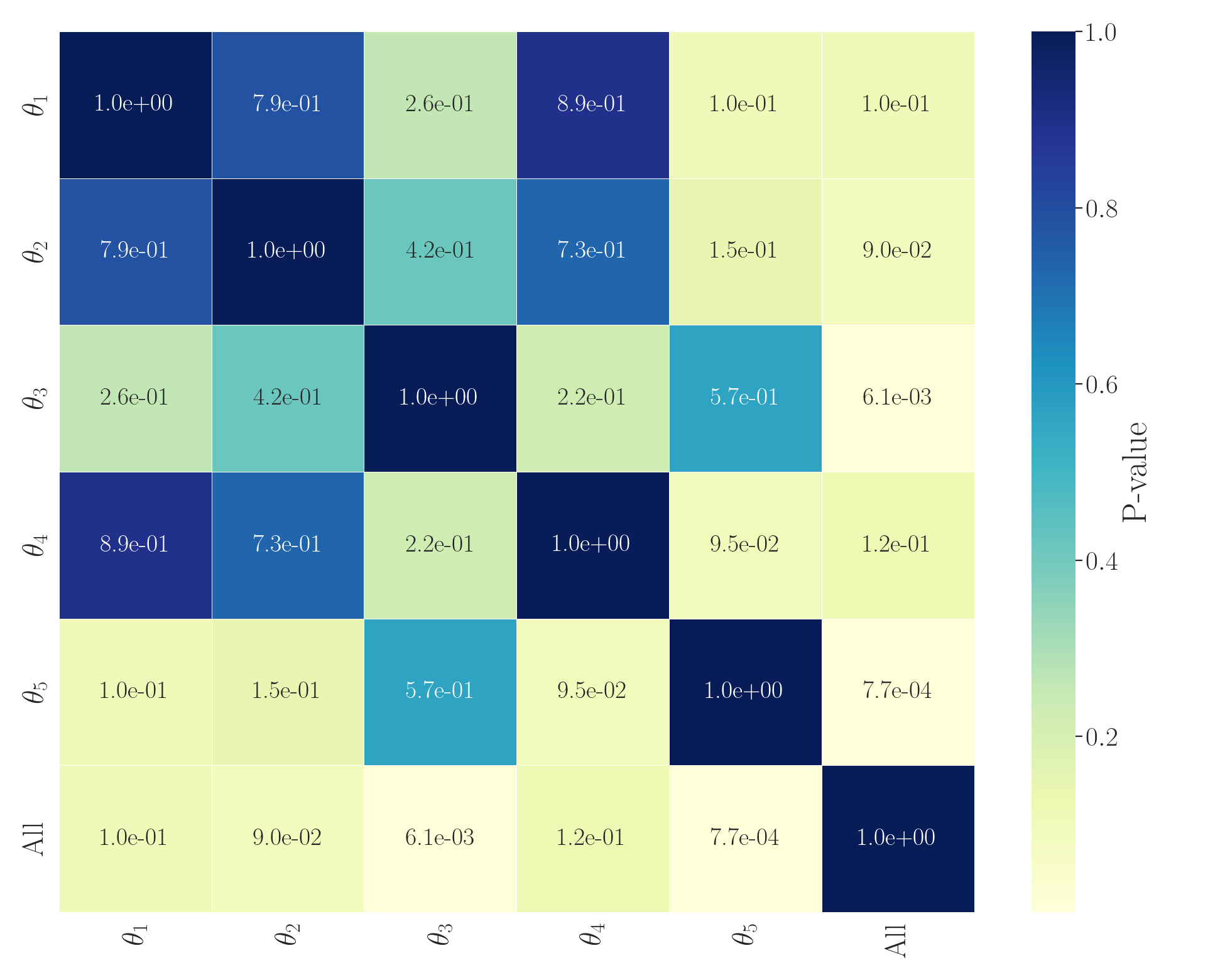}\label{fig:xgbnemarwpm_Time}}%
		\hfill
		\subfloat[GECR]{\includegraphics[width=0.24\textwidth]{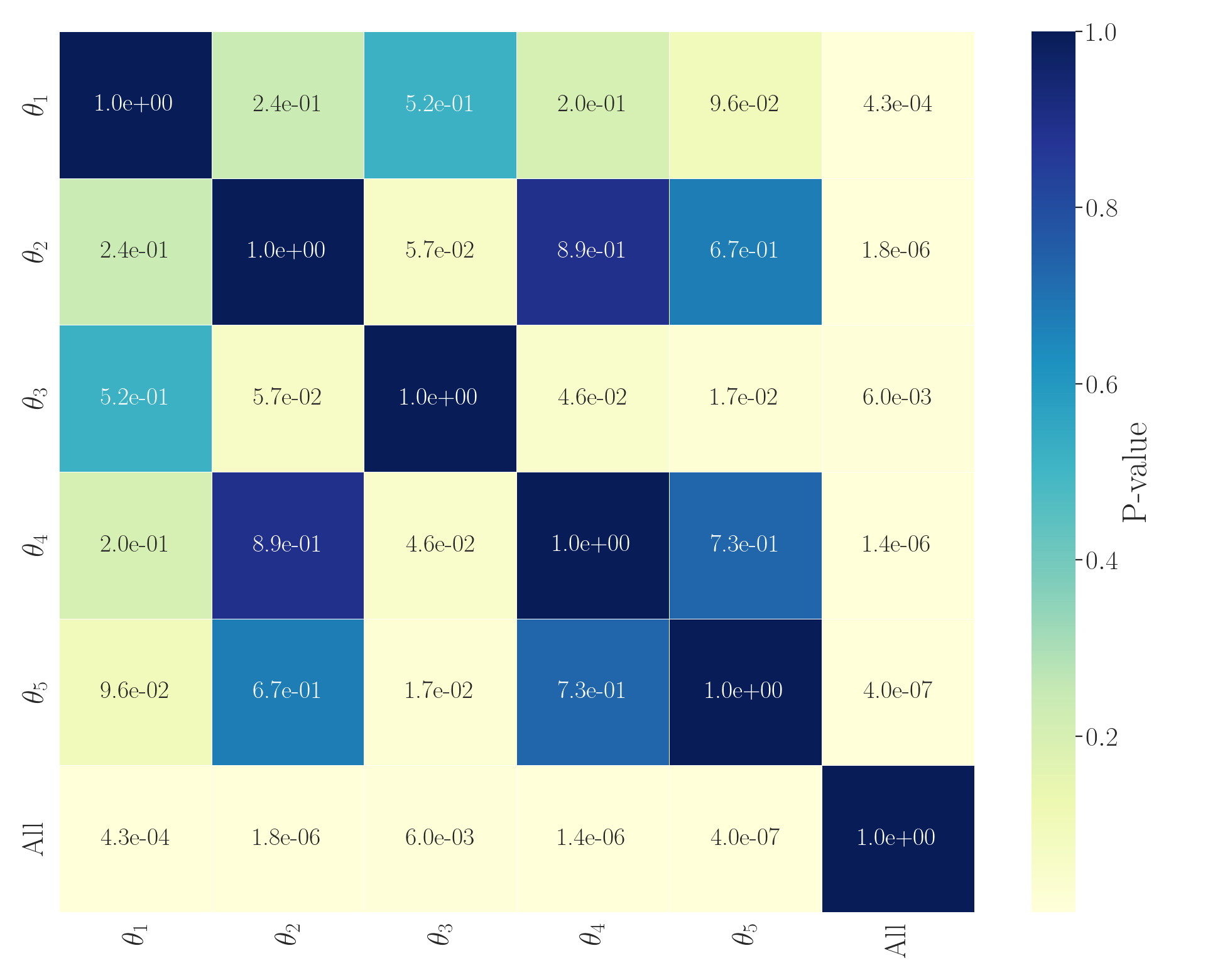}\label{fig:xgbnemgecr_Time}}%
		\hfill
		\subfloat[GFE]{\includegraphics[width=0.24\textwidth]{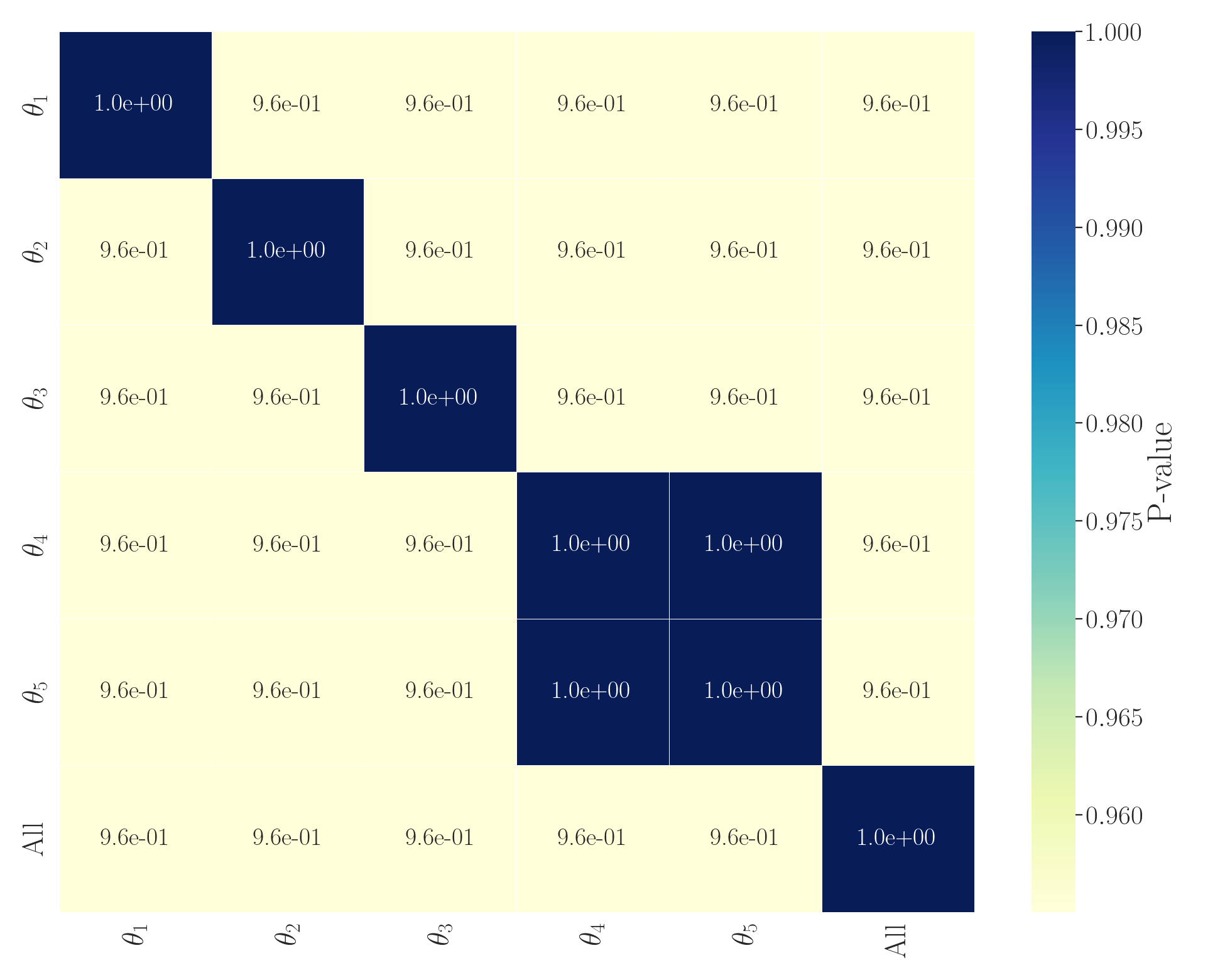}\label{fig:xgbnemgfe_Time}}
		
		\subfloat[GSAD]{\includegraphics[width=0.24\textwidth]{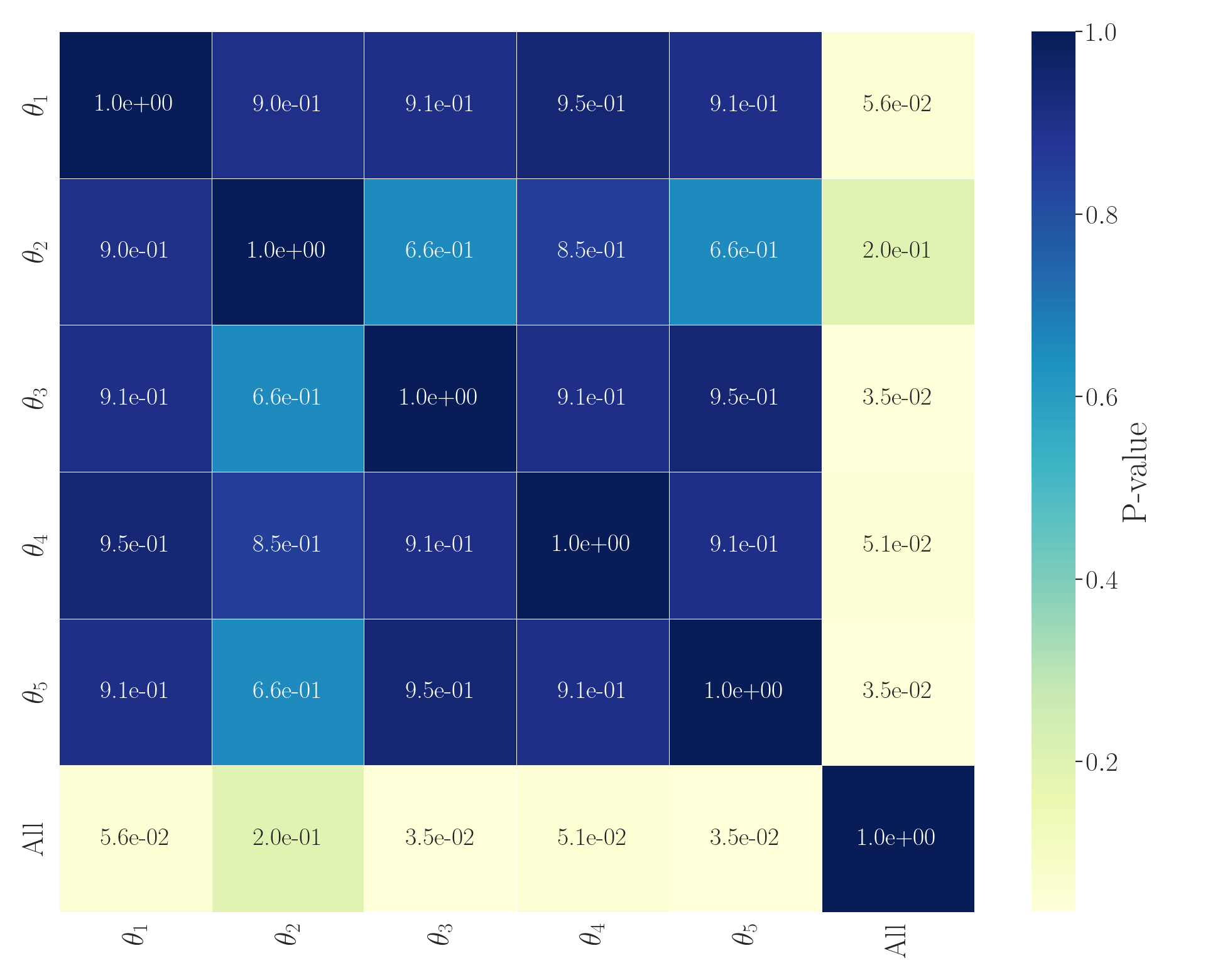}\label{fig:xgbnemgsad_Time}}%
		\hfill
		\subfloat[HAPT]{\includegraphics[width=0.24\textwidth]{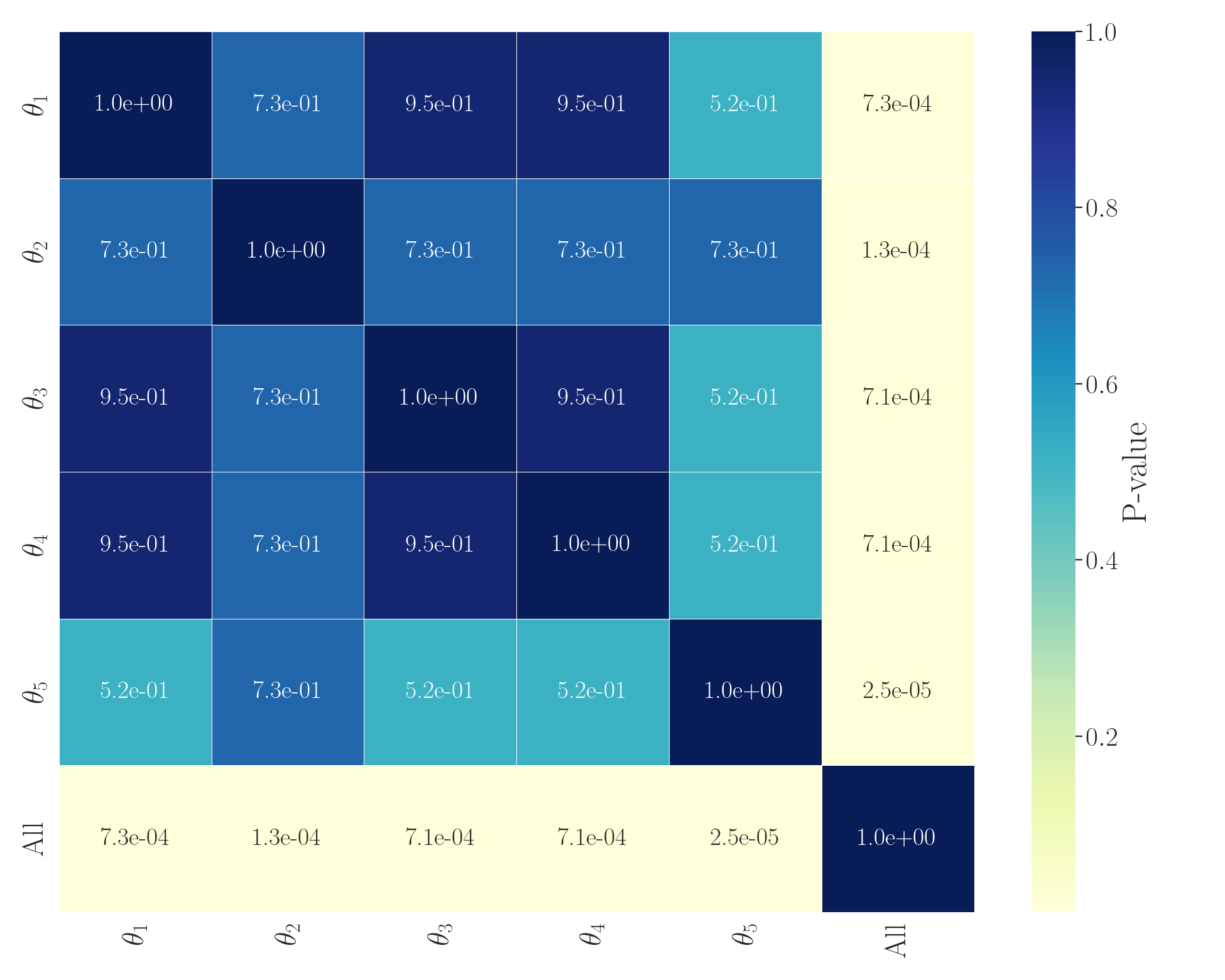}\label{fig:xgbnemhapt_Time}}%
		\hfill
		\subfloat[ISOLET]{\includegraphics[width=0.24\textwidth]{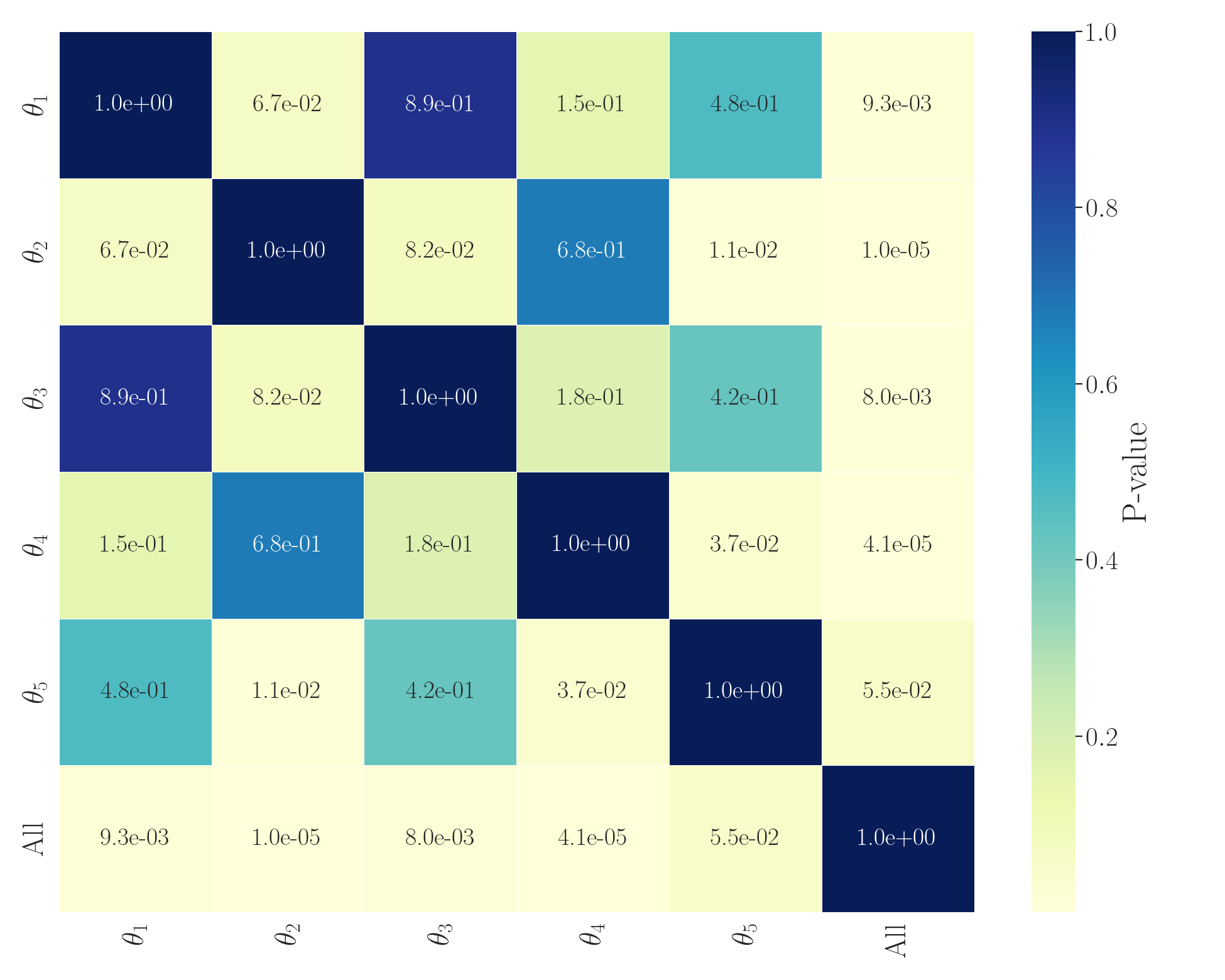}\label{fig:xgbnemisolet_Time}}%
		\hfill
		\subfloat[PD]{\includegraphics[width=0.24\textwidth]{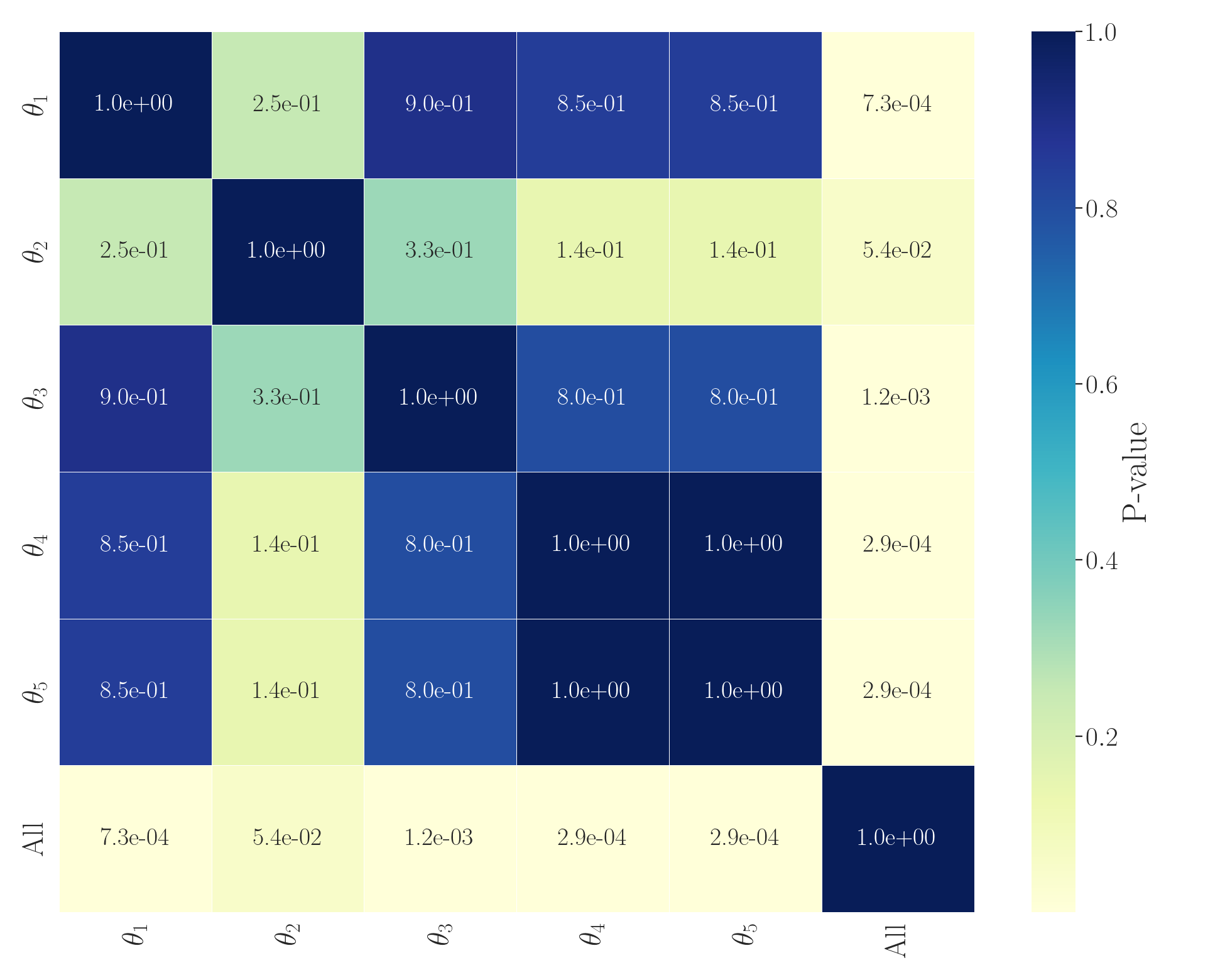}\label{fig:xgbnempd_Time}}
		\caption[The adjusted Conover's P-values for the obtained running time of 30 XGBoost runs.]{The results of the Conover post-hoc test on testing data’s running time obtained from 30 XGBoost runs.}
		
		\label{fig:xgbnem_Time}
	\end{figure*}
	\FloatBarrier
	
	\begin{figure*}[htbp] 
		\centering
		\subfloat[APSF]{\includegraphics[width=0.24\textwidth]{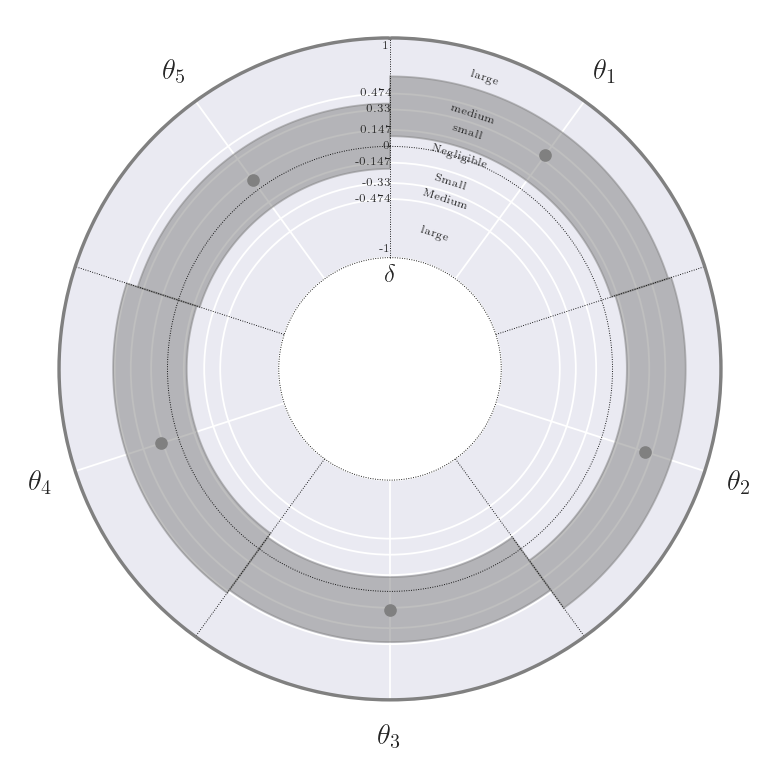}\label{fig:xgbcliffapsf_Time}}%
		\hfill
		\subfloat[ARWPM]{\includegraphics[width=0.24\textwidth]{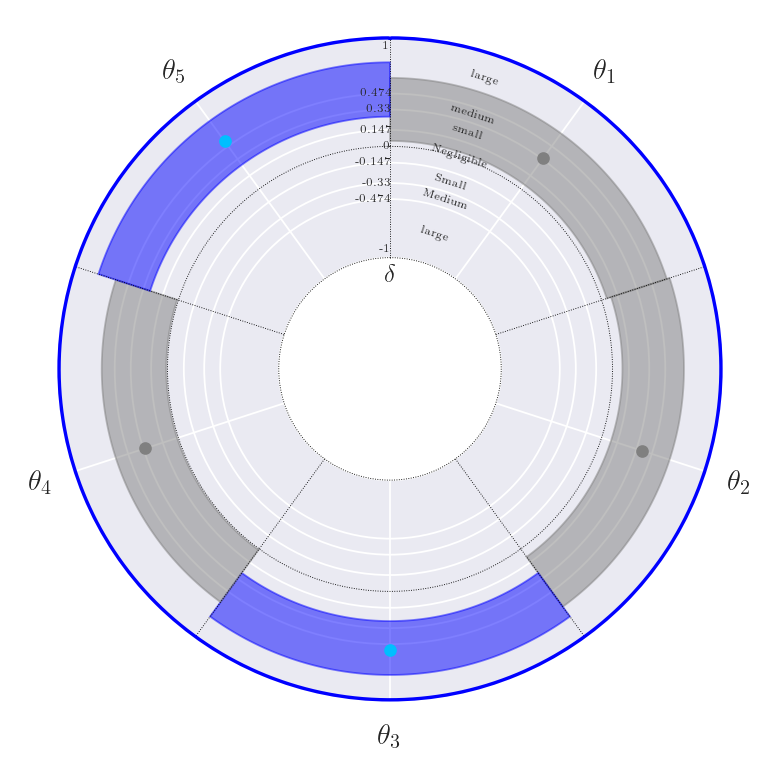}\label{fig:xgbcliffarwpm_Time}}%
		\hfill
		\subfloat[GECR]{\includegraphics[width=0.24\textwidth]{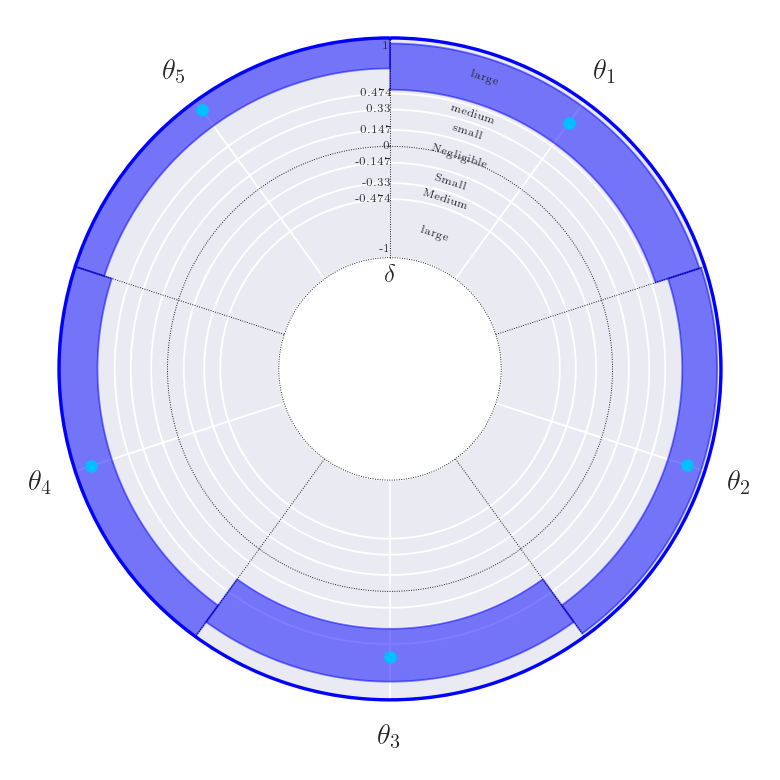}\label{fig:xgbcliffgecr_Time}}%
		\hfill
		\subfloat[GFE]{\includegraphics[width=0.24\textwidth]{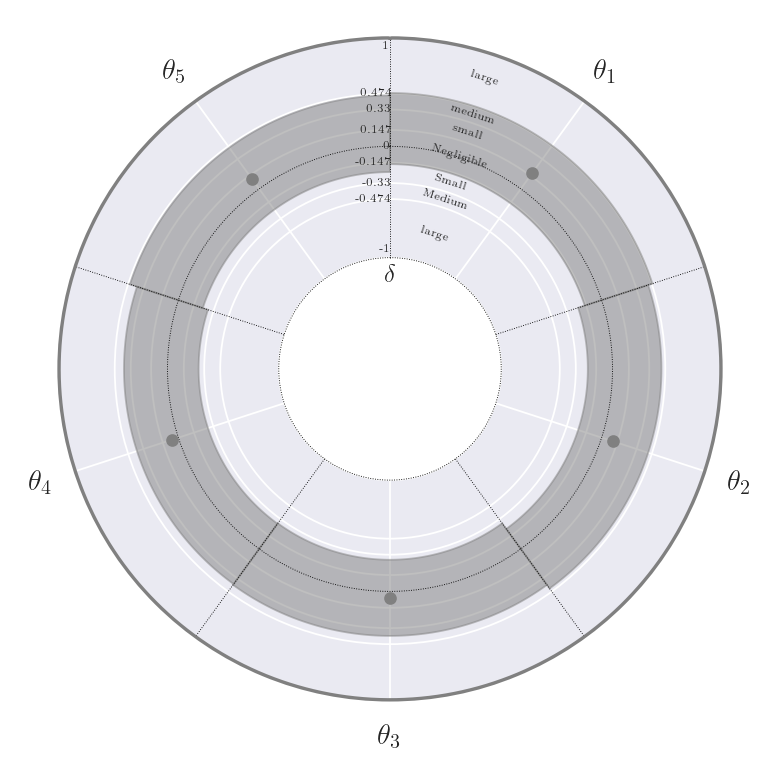}\label{fig:xgbcliffgfe_Time}}
		
		\subfloat[GSAD]{\includegraphics[width=0.24\textwidth]{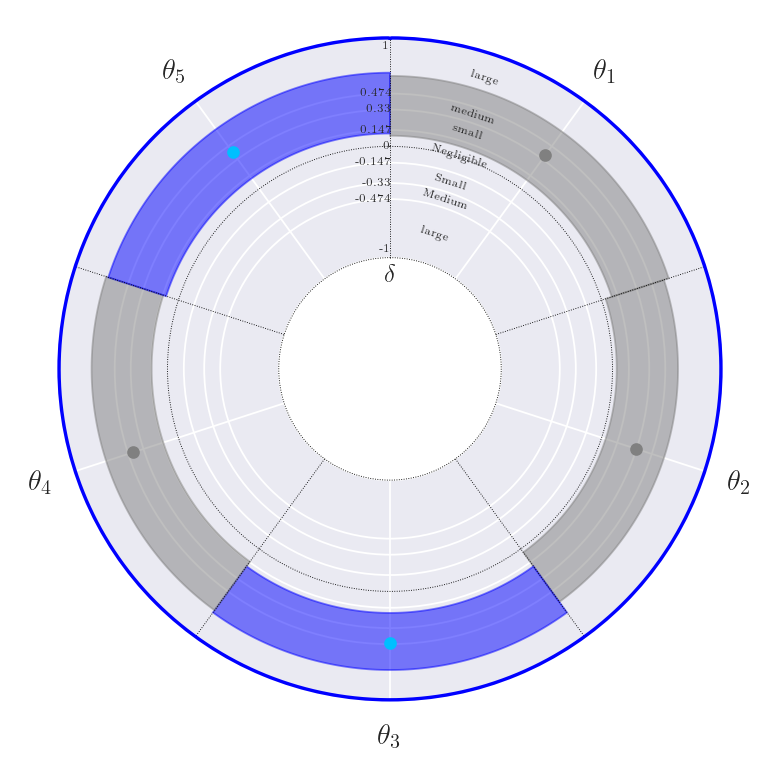}\label{fig:xgbcliffgsad_Time}}%
		\hfill
		\subfloat[HAPT]{\includegraphics[width=0.24\textwidth]{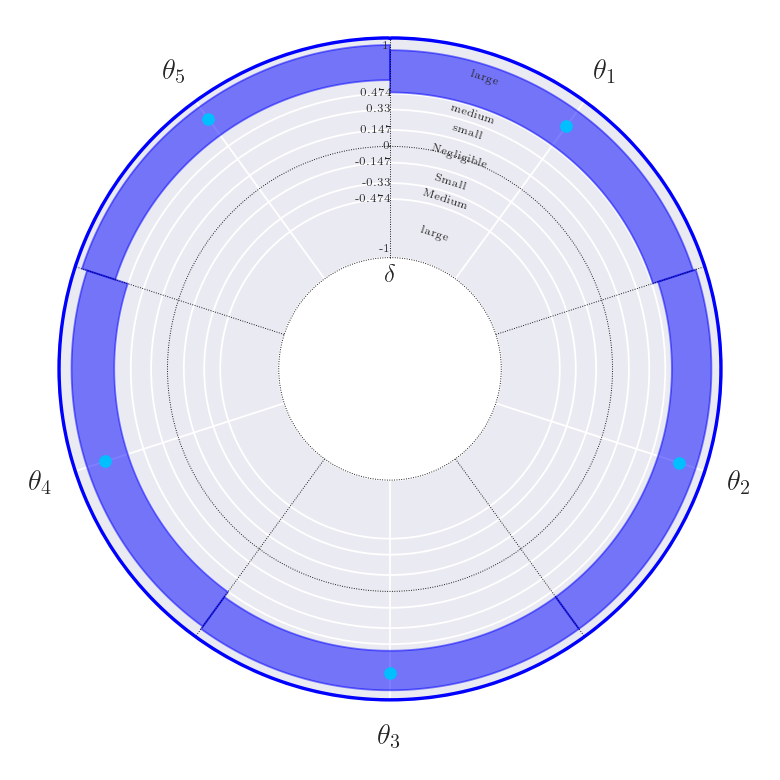}\label{fig:xgbcliffhapt_Time}}%
		\hfill
		\subfloat[ISOLET]{\includegraphics[width=0.24\textwidth]{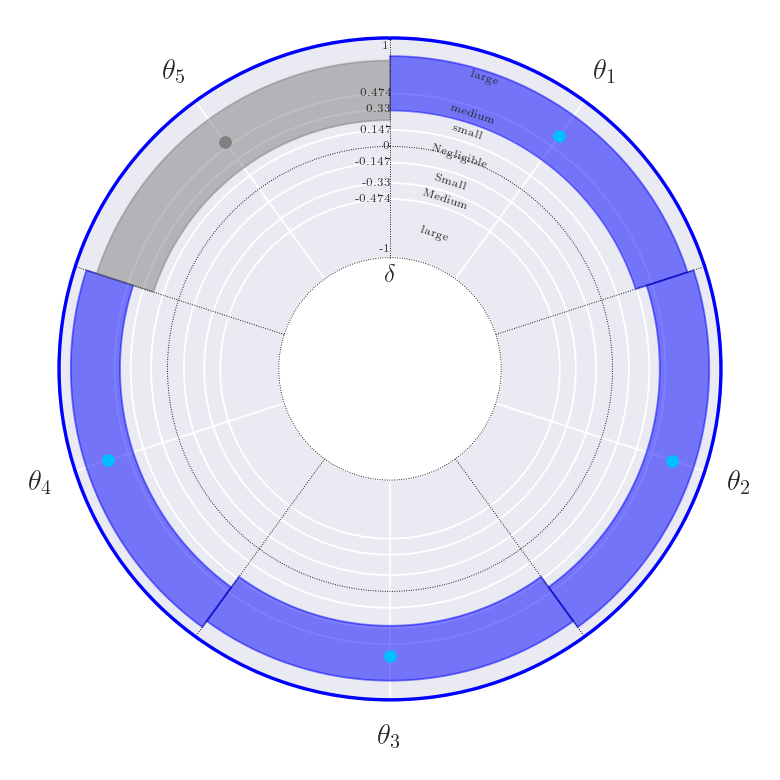}\label{fig:xgbcliffisolet_Time}}%
		\hfill
		\subfloat[PD]{\includegraphics[width=0.24\textwidth]{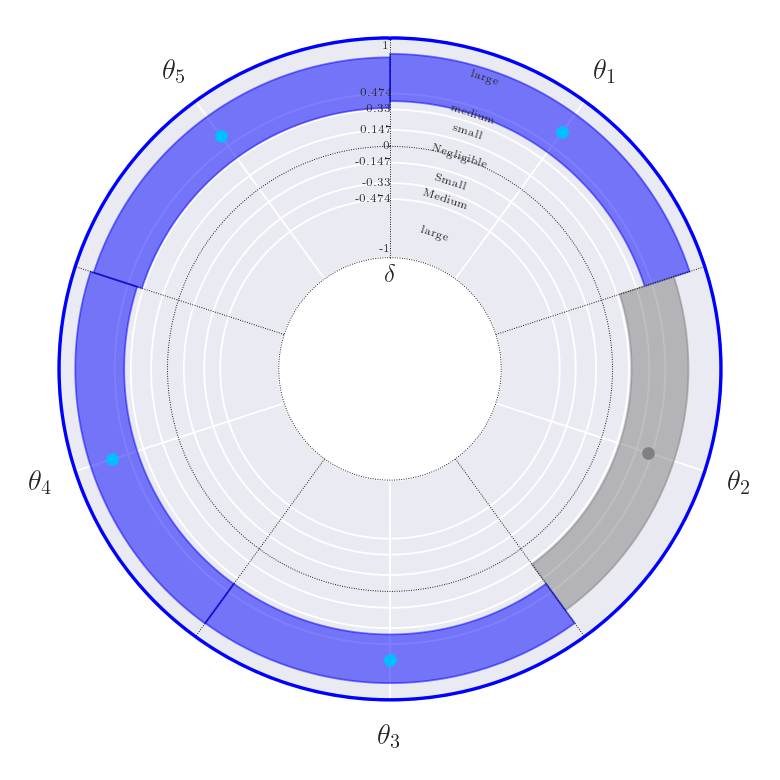}\label{fig:xgbcliffpd_Time}}
		\caption[The Cliff's $\delta$ effect size measure and its 95\% confidence intervals for the running time of 30 XGBoost runs.]{Effect size analysis of running time across 30 XGBoost runs using Cliff's $\delta$. Each point represents the actual value obtained, with segments denoting 95\% confidence intervals based on 10,000 bootstrap resamplings. The outer ring color visualizes the statistical significance: grey illustrates no significant difference (adjusted Friedman's P-value$>0.05$), while color indicates significant differences; blue indicates at least one view outperforms the benchmark (adjusted Conover's p-value$ < 0.05$, Cliff's $\delta > 0$), and red signifies all views underperform relative to the benchmark (adjusted Conover's p-value$ < 0.05$, Cliff's $\delta < 0$). Segment colors show performance difference against the benchmark: grey for no significant difference (adjusted Conover's p-value$  > 0.05$), blue for better performance (Cliff's $\delta > 0$), and red for worse performance (Cliff's $\delta < 0$).}
		
		\label{fig:xgbcliff_Time}
	\end{figure*}
	
	\begin{table*}
		\centering
		\caption[The results of Friedman and Conover tests and Cliff's $\delta$ analysis for the running time of 30 XGBoost runs.]{Statistical comparison of Running Time (seconds) for testing data obtained from XGBoost runs. W, T, and L denote win, tie, and loss based on Friedman and Conover's p-values. Effect sizes are calculated using Cliff's Delta method and are categorized as negligible, small, medium, or large.}
		\label{tab:xgbtime}
			\begin{tabular}{c|ccccccccc}
				\hline
				\multicolumn{10}{c}{XGBoost's Running Time (seconds)}\\
				\hline
				Dataset & $\theta_1$ & $\theta_2$ & $\theta_3$ & $\theta_4$ & $\theta_5$ & $E_{1:2}$ & $E_{1:3}$ & $E_{1:4}$ & $E_{1:5}$ \\
				\hline
				APSF  & T (medium) & T (medium) & T (small) & T (small) & T (negligible) & --  & --  & --  & --  \\
				ARWPM  & T (medium) & T (medium) & W (large) & T (small) & W (large) & --  & --  & --  & --  \\
				GECR  & W (large) & W (large) & W (large) & W (large) & W (large) & --  & --  & --  & --  \\
				GFE  & T (small) & T (negligible) & T (negligible) & T (negligible) & T (negligible) & --  & --  & --  & --  \\
				GSAD  & T (medium) & T (small) & W (medium) & T (medium) & W (medium) & --  & --  & --  & --  \\
				HAPT  & W (large) & W (large) & W (large) & W (large) & W (large) & --  & --  & --  & --  \\
				ISOLET  & W (large) & W (large) & W (large) & W (large) & T (large) & --  & --  & --  & --  \\
				PD  & W (large) & T (medium) & W (large) & W (large) & W (large) & --  & --  & --  & --  \\
				\hline
				W - T - L  & 4 - 4 - 0 & 3 - 5 - 0 & 6 - 2 - 0 & 4 - 4 - 0 & 5 - 3 - 0 & -- & -- & -- & -- \\
				\hline
			\end{tabular}
	\end{table*}
	
	\begin{table*}[b]
		\centering
		\caption[The summary of statistical comparison of results obtained from 30 XGBoost runs.]{The summary of statistical comparison of results for testing data obtained from XGBoost runs. W, T, and L denote win, tie, and loss based on Friedman and Conover's adjusted p-values, and Cliff's $\delta$ effect size analysis.}
		\label{tab:xgbwtl}
			\begin{tabular}{cccccccccc}
				\hline
				\multicolumn{10}{c}{XGBoost (Win - Tie - Loss)}
				\\
				\hline
				Metric & $\theta_1$ & $\theta_2$ & $\theta_3$ & $\theta_4$ & $\theta_5$ & $E_{1:2}$ & $E_{1:3}$ & $E_{1:4}$ & $E_{1:5}$ \\
				\hline
				$F_{1}$ Score & 0 - 1 - 7 & 0 - 0 - 8 & 0 - 3 - 5 & 0 - 3 - 5 & 0 - 2 - 6 & 0 - 1 - 7 & 1 - 3 - 4 & 1 - 4 - 3 & 2 - 4 - 2 \\
				AUC & 0 - 1 - 7 & 0 - 0 - 8 & 0 - 2 - 6 & 0 - 3 - 5 & 0 - 1 - 7 & 0 - 3 - 5 & 1 - 3 - 4 & 2 - 5 - 1 & 2 - 6 - 0 \\
				Loss & 0 - 2 - 6 & 0 - 2 - 6 & 0 - 2 - 6 & 0 - 4 - 4 & 0 - 3 - 5 & 0 - 4 - 4 & 0 - 5 - 3 & 2 - 5 - 1 & 3 - 4 - 1 \\
				MEC & 5 - 2 - 1 & 4 - 3 - 1 & 5 - 1 - 2 & 4 - 3 - 1 & 4 - 2 - 2 & 4 - 3 - 1 & 4 - 1 - 3 & 4 - 1 - 3 & 4 - 1 - 3 \\
				MEW & 0 - 3 - 5 & 0 - 4 - 4 & 0 - 4 - 4 & 0 - 5 - 3 & 0 - 5 - 3 & 1 - 4 - 3 & 2 - 3 - 3 & 4 - 2 - 2 & 3 - 4 - 1 \\
				Time & 4 - 4 - 0 & 3 - 5 - 0 & 6 - 2 - 0 & 4 - 4 - 0 & 5 - 3 - 0 & -- & -- & -- & -- \\
				\hline
			\end{tabular}
	\end{table*}
	\FloatBarrier
	
	\begin{sidewaystable*}
		\centering
		\caption[The obtained metrics' values and Friedman test results for 30 Logistic Regression runs.]{The obtained $F_1$ score, AUC, Log-Loss, MEC, MEW, and running time (sec), for the testing data in 30 Logistic Regression runs. The obtained p-values of Friedman's p-values and adjusted p-values using Bonferroni method. The bold values indicate at least one sample is significantly different from others.}
		\label{tab:lrres}
		\resizebox{\linewidth}{!}{%
			\begin{tabular}{c|c|cccccccccccc}
				\hline
				\multicolumn{14}{c}{Logit} \\
				\hline
				Dataset & Metric & $\theta_1$ & $\theta_2$ & $\theta_3$ & $\theta_4$ & $\theta_5$ & $E_{1:2}$ & $E_{1:3}$ & $E_{1:4}$ & $E_{1:5}$ & All & Friedman's $P-value$ & Adjusted $P-value$ \\
				\hline
				\multirow{6}{*}{APSF} & $F_1$ & {\cellcolor[rgb]{0.753,0.753,0.753}}$0.887 \pm 0.055$ & {\cellcolor[rgb]{0.753,0.753,0.753}}$0.886 \pm 0.046$ & {\cellcolor[rgb]{0.753,0.753,0.753}}$0.877 \pm 0.031$ & {\cellcolor[rgb]{0.753,0.753,0.753}}$0.876 \pm 0.035$ & {\cellcolor[rgb]{0.753,0.753,0.753}}$0.892 \pm 0.039$ & {\cellcolor[rgb]{0.753,0.753,0.753}}$0.914 \pm 0.041$ & {\cellcolor[rgb]{0.753,0.753,0.753}}$0.924 \pm 0.036$ & {\cellcolor[rgb]{0.753,0.753,0.753}}$0.928 \pm 0.037$ & {\cellcolor[rgb]{0.753,0.753,0.753}}$0.94 \pm 0.033$ & {\cellcolor[rgb]{0.753,0.753,0.753}}$0.927 \pm 0.028$ & {\cellcolor[rgb]{0.753,0.753,0.753}}$3.30e-26$ & {\cellcolor[rgb]{0.753,0.753,0.753}}$\mathbf{1.65e-25}$ \\
				& AUC & $0.923 \pm 0.078$ & $0.921 \pm 0.101$ & $0.954 \pm 0.016$ & $0.945 \pm 0.043$ & $0.931 \pm 0.063$ & $0.957 \pm 0.016$ & $0.964 \pm 0.008$ & $0.965 \pm 0.007$ & $0.965 \pm 0.007$ & $0.952 \pm 0.069$ & $7.30e-26$ & $\mathbf{3.65e-25}$ \\
				& Log-Loss & {\cellcolor[rgb]{0.753,0.753,0.753}}$2.295 \pm 1.488$ & {\cellcolor[rgb]{0.753,0.753,0.753}}$2.307 \pm 1.159$ & {\cellcolor[rgb]{0.753,0.753,0.753}}$2.564 \pm 0.786$ & {\cellcolor[rgb]{0.753,0.753,0.753}}$2.485 \pm 0.894$ & {\cellcolor[rgb]{0.753,0.753,0.753}}$1.946 \pm 0.927$ & {\cellcolor[rgb]{0.753,0.753,0.753}}$1.329 \pm 0.838$ & {\cellcolor[rgb]{0.753,0.753,0.753}}$1.014 \pm 0.66$ & {\cellcolor[rgb]{0.753,0.753,0.753}}$0.877 \pm 0.606$ & {\cellcolor[rgb]{0.753,0.753,0.753}}$0.631 \pm 0.445$ & {\cellcolor[rgb]{0.753,0.753,0.753}}$1.515 \pm 0.746$ & {\cellcolor[rgb]{0.753,0.753,0.753}}$2.34e-33$ & {\cellcolor[rgb]{0.753,0.753,0.753}}$\mathbf{1.17e-32}$ \\
				& MEC & $0.013 \pm 0.008$ & $0.013 \pm 0.006$ & $0.013 \pm 0.004$ & $0.014 \pm 0.006$ & $0.015 \pm 0.006$ & $0.067 \pm 0.048$ & $0.082 \pm 0.036$ & $0.092 \pm 0.032$ & $0.096 \pm 0.044$ & $0.009 \pm 0.007$ & $6.99e-39$ & $\mathbf{3.50e-38}$ \\
				& MEW & {\cellcolor[rgb]{0.753,0.753,0.753}}$0.09 \pm 0.035$ & {\cellcolor[rgb]{0.753,0.753,0.753}}$0.089 \pm 0.031$ & {\cellcolor[rgb]{0.753,0.753,0.753}}$0.082 \pm 0.022$ & {\cellcolor[rgb]{0.753,0.753,0.753}}$0.09 \pm 0.03$ & {\cellcolor[rgb]{0.753,0.753,0.753}}$0.106 \pm 0.031$ & {\cellcolor[rgb]{0.753,0.753,0.753}}$0.307 \pm 0.193$ & {\cellcolor[rgb]{0.753,0.753,0.753}}$0.388 \pm 0.182$ & {\cellcolor[rgb]{0.753,0.753,0.753}}$0.44 \pm 0.162$ & {\cellcolor[rgb]{0.753,0.753,0.753}}$0.468 \pm 0.159$ & {\cellcolor[rgb]{0.753,0.753,0.753}}$0.096 \pm 0.065$ & {\cellcolor[rgb]{0.753,0.753,0.753}}$9.26e-38$ & {\cellcolor[rgb]{0.753,0.753,0.753}}$\mathbf{4.63e-37}$ \\
				& Time (sec) & $0.755 \pm 0.092$ & $0.775 \pm 0.088$ & $0.822 \pm 0.109$ & $0.765 \pm 0.092$ & $0.738 \pm 0.076$ & -- & -- & -- & -- & $2.447 \pm 0.277$ & $\mathbf{1.90e-17}$ & -- \\
				\hline
				\multirow{6}{*}{ARWPM} & $F_1$ & {\cellcolor[rgb]{0.753,0.753,0.753}}$0.795 \pm 0.014$ & {\cellcolor[rgb]{0.753,0.753,0.753}}$0.782 \pm 0.012$ & {\cellcolor[rgb]{0.753,0.753,0.753}}$0.795 \pm 0.018$ & {\cellcolor[rgb]{0.753,0.753,0.753}}$0.761 \pm 0.03$ & {\cellcolor[rgb]{0.753,0.753,0.753}}$0.721 \pm 0.043$ & {\cellcolor[rgb]{0.753,0.753,0.753}}$0.8 \pm 0.011$ & {\cellcolor[rgb]{0.753,0.753,0.753}}$0.814 \pm 0.011$ & {\cellcolor[rgb]{0.753,0.753,0.753}}$0.815 \pm 0.011$ & {\cellcolor[rgb]{0.753,0.753,0.753}}$0.816 \pm 0.012$ & {\cellcolor[rgb]{0.753,0.753,0.753}}$0.793 \pm 0.01$ & {\cellcolor[rgb]{0.753,0.753,0.753}}$4.78e-41$ & {\cellcolor[rgb]{0.753,0.753,0.753}}$\mathbf{2.39e-40}$ \\
				& AUC & $0.946 \pm 0.005$ & $0.94 \pm 0.005$ & $0.944 \pm 0.008$ & $0.929 \pm 0.015$ & $0.906 \pm 0.025$ & $0.948 \pm 0.004$ & $0.953 \pm 0.004$ & $0.954 \pm 0.004$ & $0.954 \pm 0.004$ & $0.946 \pm 0.005$ & $1.27e-42$ & $\mathbf{6.37e-42}$ \\
				& Log-Loss & {\cellcolor[rgb]{0.753,0.753,0.753}}$0.432 \pm 0.023$ & {\cellcolor[rgb]{0.753,0.753,0.753}}$0.471 \pm 0.03$ & {\cellcolor[rgb]{0.753,0.753,0.753}}$0.475 \pm 0.039$ & {\cellcolor[rgb]{0.753,0.753,0.753}}$0.545 \pm 0.065$ & {\cellcolor[rgb]{0.753,0.753,0.753}}$0.642 \pm 0.1$ & {\cellcolor[rgb]{0.753,0.753,0.753}}$0.432 \pm 0.02$ & {\cellcolor[rgb]{0.753,0.753,0.753}}$0.43 \pm 0.023$ & {\cellcolor[rgb]{0.753,0.753,0.753}}$0.441 \pm 0.022$ & {\cellcolor[rgb]{0.753,0.753,0.753}}$0.454 \pm 0.03$ & {\cellcolor[rgb]{0.753,0.753,0.753}}$0.485 \pm 0.025$ & {\cellcolor[rgb]{0.753,0.753,0.753}}$1.32e-41$ & {\cellcolor[rgb]{0.753,0.753,0.753}}$\mathbf{6.62e-41}$ \\
				& MEC & $0.523 \pm 0.021$ & $0.576 \pm 0.023$ & $0.598 \pm 0.036$ & $0.674 \pm 0.07$ & $0.825 \pm 0.14$ & $0.564 \pm 0.035$ & $0.616 \pm 0.033$ & $0.664 \pm 0.041$ & $0.726 \pm 0.08$ & $0.726 \pm 0.044$ & $2.86e-45$ & $\mathbf{1.43e-44}$ \\
				& MEW & {\cellcolor[rgb]{0.753,0.753,0.753}}$1.168 \pm 0.018$ & {\cellcolor[rgb]{0.753,0.753,0.753}}$1.253 \pm 0.041$ & {\cellcolor[rgb]{0.753,0.753,0.753}}$1.281 \pm 0.042$ & {\cellcolor[rgb]{0.753,0.753,0.753}}$1.35 \pm 0.068$ & {\cellcolor[rgb]{0.753,0.753,0.753}}$1.464 \pm 0.096$ & {\cellcolor[rgb]{0.753,0.753,0.753}}$1.214 \pm 0.036$ & {\cellcolor[rgb]{0.753,0.753,0.753}}$1.292 \pm 0.04$ & {\cellcolor[rgb]{0.753,0.753,0.753}}$1.336 \pm 0.04$ & {\cellcolor[rgb]{0.753,0.753,0.753}}$1.375 \pm 0.058$ & {\cellcolor[rgb]{0.753,0.753,0.753}}$1.456 \pm 0.035$ & {\cellcolor[rgb]{0.753,0.753,0.753}}$9.05e-46$ & {\cellcolor[rgb]{0.753,0.753,0.753}}$\mathbf{4.53e-45}$ \\
				& Time (sec) & $0.089 \pm 0.047$ & $2.416 \pm 0.416$ & $2.356 \pm 0.264$ & $2.323 \pm 0.2$ & $2.369 \pm 0.231$ & -- & -- & -- & -- & $3.33 \pm 0.169$ & $\mathbf{2.59e-21}$ & -- \\
				\hline
				\multirow{6}{*}{GECR} & $F_1$ & {\cellcolor[rgb]{0.753,0.753,0.753}}$0.996 \pm 0.004$ & {\cellcolor[rgb]{0.753,0.753,0.753}}$0.998 \pm 0.004$ & {\cellcolor[rgb]{0.753,0.753,0.753}}$0.996 \pm 0.005$ & {\cellcolor[rgb]{0.753,0.753,0.753}}$0.996 \pm 0.004$ & {\cellcolor[rgb]{0.753,0.753,0.753}}$0.998 \pm 0.003$ & {\cellcolor[rgb]{0.753,0.753,0.753}}$0.997 \pm 0.004$ & {\cellcolor[rgb]{0.753,0.753,0.753}}$0.998 \pm 0.003$ & {\cellcolor[rgb]{0.753,0.753,0.753}}$0.998 \pm 0.003$ & {\cellcolor[rgb]{0.753,0.753,0.753}}$0.999 \pm 0.002$ & {\cellcolor[rgb]{0.753,0.753,0.753}}$0.999 \pm 0.002$ & {\cellcolor[rgb]{0.753,0.753,0.753}}$4.01e-06$ & {\cellcolor[rgb]{0.753,0.753,0.753}}$\mathbf{2.00e-05}$ \\
				& AUC & $1.0 \pm 0.0$ & $1.0 \pm 0.0$ & $1.0 \pm 0.0$ & $1.0 \pm 0.0$ & $1.0 \pm 0.0$ & $1.0 \pm 0.0$ & $1.0 \pm 0.0$ & $1.0 \pm 0.0$ & $1.0 \pm 0.0$ & $1.0 \pm 0.0$ & $9.39e-05$ & $\mathbf{4.70e-04}$ \\
				& Log-Loss & {\cellcolor[rgb]{0.753,0.753,0.753}}$0.131 \pm 0.124$ & {\cellcolor[rgb]{0.753,0.753,0.753}}$0.132 \pm 0.117$ & {\cellcolor[rgb]{0.753,0.753,0.753}}$0.098 \pm 0.104$ & {\cellcolor[rgb]{0.753,0.753,0.753}}$0.093 \pm 0.11$ & {\cellcolor[rgb]{0.753,0.753,0.753}}$0.09 \pm 0.079$ & {\cellcolor[rgb]{0.753,0.753,0.753}}$0.122 \pm 0.095$ & {\cellcolor[rgb]{0.753,0.753,0.753}}$0.119 \pm 0.094$ & {\cellcolor[rgb]{0.753,0.753,0.753}}$0.102 \pm 0.073$ & {\cellcolor[rgb]{0.753,0.753,0.753}}$0.101 \pm 0.057$ & {\cellcolor[rgb]{0.753,0.753,0.753}}$0.026 \pm 0.016$ & {\cellcolor[rgb]{0.753,0.753,0.753}}$2.27e-07$ & {\cellcolor[rgb]{0.753,0.753,0.753}}$\mathbf{1.14e-06}$ \\
				& MEC & $0.552 \pm 0.44$ & $0.567 \pm 0.412$ & $0.418 \pm 0.376$ & $0.386 \pm 0.399$ & $0.416 \pm 0.298$ & $0.568 \pm 0.331$ & $0.556 \pm 0.337$ & $0.504 \pm 0.276$ & $0.519 \pm 0.226$ & $0.135 \pm 0.089$ & $2.27e-07$ & $\mathbf{1.14e-06}$ \\
				& MEW & {\cellcolor[rgb]{0.753,0.753,0.753}}$1.146 \pm 1.007$ & {\cellcolor[rgb]{0.753,0.753,0.753}}$0.649 \pm 0.925$ & {\cellcolor[rgb]{0.753,0.753,0.753}}$1.042 \pm 0.93$ & {\cellcolor[rgb]{0.753,0.753,0.753}}$1.157 \pm 0.866$ & {\cellcolor[rgb]{0.753,0.753,0.753}}$0.833 \pm 0.902$ & {\cellcolor[rgb]{0.753,0.753,0.753}}$0.73 \pm 0.961$ & {\cellcolor[rgb]{0.753,0.753,0.753}}$0.707 \pm 0.933$ & {\cellcolor[rgb]{0.753,0.753,0.753}}$0.587 \pm 0.9$ & {\cellcolor[rgb]{0.753,0.753,0.753}}$0.513 \pm 0.853$ & {\cellcolor[rgb]{0.753,0.753,0.753}}$0.518 \pm 0.794$ & {\cellcolor[rgb]{0.753,0.753,0.753}}$1.05e-03$ & {\cellcolor[rgb]{0.753,0.753,0.753}}$\mathbf{5.24e-03}$ \\
				& Time (sec) & $0.1 \pm 0.049$ & $3.687 \pm 0.429$ & $3.179 \pm 0.456$ & $3.11 \pm 0.392$ & $3.129 \pm 0.339$ & -- & -- & -- & -- & $9.276 \pm 0.347$ & $\mathbf{1.26e-23}$ & -- \\
				\hline
				\multirow{6}{*}{GFE} & $F_1$ & {\cellcolor[rgb]{0.753,0.753,0.753}}$0.739 \pm 0.01$ & {\cellcolor[rgb]{0.753,0.753,0.753}}$0.75 \pm 0.009$ & {\cellcolor[rgb]{0.753,0.753,0.753}}$0.762 \pm 0.007$ & {\cellcolor[rgb]{0.753,0.753,0.753}}$0.77 \pm 0.008$ & {\cellcolor[rgb]{0.753,0.753,0.753}}$0.773 \pm 0.006$ & {\cellcolor[rgb]{0.753,0.753,0.753}}$0.752 \pm 0.008$ & {\cellcolor[rgb]{0.753,0.753,0.753}}$0.764 \pm 0.006$ & {\cellcolor[rgb]{0.753,0.753,0.753}}$0.772 \pm 0.006$ & {\cellcolor[rgb]{0.753,0.753,0.753}}$0.776 \pm 0.005$ & {\cellcolor[rgb]{0.753,0.753,0.753}}$0.781 \pm 0.004$ & {\cellcolor[rgb]{0.753,0.753,0.753}}$1.65e-45$ & {\cellcolor[rgb]{0.753,0.753,0.753}}$\mathbf{8.27e-45}$ \\
				& AUC & $0.789 \pm 0.01$ & $0.793 \pm 0.009$ & $0.802 \pm 0.007$ & $0.808 \pm 0.007$ & $0.812 \pm 0.006$ & $0.797 \pm 0.008$ & $0.805 \pm 0.006$ & $0.81 \pm 0.006$ & $0.814 \pm 0.005$ & $0.835 \pm 0.004$ & $1.76e-45$ & $\mathbf{8.79e-45}$ \\
				& Log-Loss & {\cellcolor[rgb]{0.753,0.753,0.753}}$0.55 \pm 0.011$ & {\cellcolor[rgb]{0.753,0.753,0.753}}$0.545 \pm 0.009$ & {\cellcolor[rgb]{0.753,0.753,0.753}}$0.535 \pm 0.008$ & {\cellcolor[rgb]{0.753,0.753,0.753}}$0.529 \pm 0.007$ & {\cellcolor[rgb]{0.753,0.753,0.753}}$0.529 \pm 0.008$ & {\cellcolor[rgb]{0.753,0.753,0.753}}$0.539 \pm 0.007$ & {\cellcolor[rgb]{0.753,0.753,0.753}}$0.531 \pm 0.006$ & {\cellcolor[rgb]{0.753,0.753,0.753}}$0.524 \pm 0.006$ & {\cellcolor[rgb]{0.753,0.753,0.753}}$0.521 \pm 0.005$ & {\cellcolor[rgb]{0.753,0.753,0.753}}$0.494 \pm 0.005$ & {\cellcolor[rgb]{0.753,0.753,0.753}}$1.72e-44$ & {\cellcolor[rgb]{0.753,0.753,0.753}}$\mathbf{8.58e-44}$ \\
				& MEC & $0.793 \pm 0.014$ & $0.791 \pm 0.013$ & $0.782 \pm 0.01$ & $0.776 \pm 0.01$ & $0.776 \pm 0.013$ & $0.797 \pm 0.012$ & $0.79 \pm 0.01$ & $0.786 \pm 0.009$ & $0.787 \pm 0.01$ & $0.72 \pm 0.004$ & $3.27e-26$ & $\mathbf{1.64e-25}$ \\
				& MEW & {\cellcolor[rgb]{0.753,0.753,0.753}}$0.879 \pm 0.008$ & {\cellcolor[rgb]{0.753,0.753,0.753}}$0.874 \pm 0.008$ & {\cellcolor[rgb]{0.753,0.753,0.753}}$0.866 \pm 0.006$ & {\cellcolor[rgb]{0.753,0.753,0.753}}$0.863 \pm 0.007$ & {\cellcolor[rgb]{0.753,0.753,0.753}}$0.865 \pm 0.008$ & {\cellcolor[rgb]{0.753,0.753,0.753}}$0.88 \pm 0.008$ & {\cellcolor[rgb]{0.753,0.753,0.753}}$0.871 \pm 0.007$ & {\cellcolor[rgb]{0.753,0.753,0.753}}$0.869 \pm 0.006$ & {\cellcolor[rgb]{0.753,0.753,0.753}}$0.871 \pm 0.007$ & {\cellcolor[rgb]{0.753,0.753,0.753}}$0.859 \pm 0.004$ & {\cellcolor[rgb]{0.753,0.753,0.753}}$8.86e-29$ & {\cellcolor[rgb]{0.753,0.753,0.753}}$\mathbf{4.43e-28}$ \\
				& Time (sec) & $0.193 \pm 0.05$ & $0.203 \pm 0.027$ & $0.227 \pm 0.034$ & $0.251 \pm 0.034$ & $0.287 \pm 0.032$ & -- & -- & -- & -- & $1.029 \pm 0.037$ & $\mathbf{2.97e-23}$ & -- \\
				\hline
				\multirow{6}{*}{GSAD} & $F_1$ & {\cellcolor[rgb]{0.753,0.753,0.753}}$0.946 \pm 0.014$ & {\cellcolor[rgb]{0.753,0.753,0.753}}$0.931 \pm 0.014$ & {\cellcolor[rgb]{0.753,0.753,0.753}}$0.9 \pm 0.018$ & {\cellcolor[rgb]{0.753,0.753,0.753}}$0.886 \pm 0.02$ & {\cellcolor[rgb]{0.753,0.753,0.753}}$0.844 \pm 0.042$ & {\cellcolor[rgb]{0.753,0.753,0.753}}$0.951 \pm 0.011$ & {\cellcolor[rgb]{0.753,0.753,0.753}}$0.951 \pm 0.012$ & {\cellcolor[rgb]{0.753,0.753,0.753}}$0.951 \pm 0.012$ & {\cellcolor[rgb]{0.753,0.753,0.753}}$0.951 \pm 0.012$ & {\cellcolor[rgb]{0.753,0.753,0.753}}$0.902 \pm 0.021$ & {\cellcolor[rgb]{0.753,0.753,0.753}}$5.34e-44$ & {\cellcolor[rgb]{0.753,0.753,0.753}}$\mathbf{2.67e-43}$ \\
				& AUC & $0.993 \pm 0.002$ & $0.991 \pm 0.002$ & $0.988 \pm 0.003$ & $0.985 \pm 0.004$ & $0.976 \pm 0.01$ & $0.994 \pm 0.001$ & $0.994 \pm 0.001$ & $0.994 \pm 0.001$ & $0.994 \pm 0.001$ & $0.989 \pm 0.002$ & $7.61e-48$ & $\mathbf{3.81e-47}$ \\
				& Log-Loss & {\cellcolor[rgb]{0.753,0.753,0.753}}$0.295 \pm 0.046$ & {\cellcolor[rgb]{0.753,0.753,0.753}}$0.339 \pm 0.037$ & {\cellcolor[rgb]{0.753,0.753,0.753}}$0.411 \pm 0.045$ & {\cellcolor[rgb]{0.753,0.753,0.753}}$0.453 \pm 0.051$ & {\cellcolor[rgb]{0.753,0.753,0.753}}$0.56 \pm 0.098$ & {\cellcolor[rgb]{0.753,0.753,0.753}}$0.282 \pm 0.039$ & {\cellcolor[rgb]{0.753,0.753,0.753}}$0.279 \pm 0.039$ & {\cellcolor[rgb]{0.753,0.753,0.753}}$0.278 \pm 0.041$ & {\cellcolor[rgb]{0.753,0.753,0.753}}$0.276 \pm 0.04$ & {\cellcolor[rgb]{0.753,0.753,0.753}}$0.432 \pm 0.047$ & {\cellcolor[rgb]{0.753,0.753,0.753}}$1.49e-43$ & {\cellcolor[rgb]{0.753,0.753,0.753}}$\mathbf{7.45e-43}$ \\
				& MEC & $0.575 \pm 0.082$ & $0.676 \pm 0.091$ & $0.796 \pm 0.094$ & $0.868 \pm 0.095$ & $1.016 \pm 0.153$ & $0.597 \pm 0.089$ & $0.605 \pm 0.085$ & $0.618 \pm 0.087$ & $0.627 \pm 0.082$ & $0.872 \pm 0.092$ & $3.94e-41$ & $\mathbf{1.97e-40}$ \\
				& MEW & {\cellcolor[rgb]{0.753,0.753,0.753}}$1.557 \pm 0.143$ & {\cellcolor[rgb]{0.753,0.753,0.753}}$1.691 \pm 0.119$ & {\cellcolor[rgb]{0.753,0.753,0.753}}$1.841 \pm 0.106$ & {\cellcolor[rgb]{0.753,0.753,0.753}}$1.899 \pm 0.088$ & {\cellcolor[rgb]{0.753,0.753,0.753}}$1.982 \pm 0.129$ & {\cellcolor[rgb]{0.753,0.753,0.753}}$1.596 \pm 0.137$ & {\cellcolor[rgb]{0.753,0.753,0.753}}$1.613 \pm 0.123$ & {\cellcolor[rgb]{0.753,0.753,0.753}}$1.628 \pm 0.124$ & {\cellcolor[rgb]{0.753,0.753,0.753}}$1.641 \pm 0.111$ & {\cellcolor[rgb]{0.753,0.753,0.753}}$1.95 \pm 0.12$ & {\cellcolor[rgb]{0.753,0.753,0.753}}$1.85e-41$ & {\cellcolor[rgb]{0.753,0.753,0.753}}$\mathbf{9.25e-41}$ \\
				& Time (sec) & $0.156 \pm 0.066$ & $3.082 \pm 0.392$ & $3.041 \pm 0.445$ & $2.993 \pm 0.3$ & $2.993 \pm 0.433$ & -- & -- & -- & -- & $3.721 \pm 0.37$ & $\mathbf{3.58e-20}$ & -- \\
				\hline
				\multirow{6}{*}{HAPT} & $F_1$ & {\cellcolor[rgb]{0.753,0.753,0.753}}$0.888 \pm 0.007$ & {\cellcolor[rgb]{0.753,0.753,0.753}}$0.889 \pm 0.01$ & {\cellcolor[rgb]{0.753,0.753,0.753}}$0.904 \pm 0.01$ & {\cellcolor[rgb]{0.753,0.753,0.753}}$0.908 \pm 0.009$ & {\cellcolor[rgb]{0.753,0.753,0.753}}$0.886 \pm 0.028$ & {\cellcolor[rgb]{0.753,0.753,0.753}}$0.897 \pm 0.006$ & {\cellcolor[rgb]{0.753,0.753,0.753}}$0.91 \pm 0.006$ & {\cellcolor[rgb]{0.753,0.753,0.753}}$0.918 \pm 0.007$ & {\cellcolor[rgb]{0.753,0.753,0.753}}$0.922 \pm 0.006$ & {\cellcolor[rgb]{0.753,0.753,0.753}}$0.95 \pm 0.003$ & {\cellcolor[rgb]{0.753,0.753,0.753}}$6.11e-44$ & {\cellcolor[rgb]{0.753,0.753,0.753}}$\mathbf{3.06e-43}$ \\
				& AUC & $0.992 \pm 0.001$ & $0.992 \pm 0.001$ & $0.994 \pm 0.001$ & $0.995 \pm 0.001$ & $0.991 \pm 0.004$ & $0.993 \pm 0.001$ & $0.995 \pm 0.001$ & $0.996 \pm 0.001$ & $0.996 \pm 0.0$ & $0.997 \pm 0.001$ & $6.48e-45$ & $\mathbf{3.24e-44}$ \\
				& Log-Loss & {\cellcolor[rgb]{0.753,0.753,0.753}}$0.32 \pm 0.034$ & {\cellcolor[rgb]{0.753,0.753,0.753}}$0.322 \pm 0.023$ & {\cellcolor[rgb]{0.753,0.753,0.753}}$0.277 \pm 0.025$ & {\cellcolor[rgb]{0.753,0.753,0.753}}$0.262 \pm 0.021$ & {\cellcolor[rgb]{0.753,0.753,0.753}}$0.335 \pm 0.072$ & {\cellcolor[rgb]{0.753,0.753,0.753}}$0.306 \pm 0.016$ & {\cellcolor[rgb]{0.753,0.753,0.753}}$0.273 \pm 0.015$ & {\cellcolor[rgb]{0.753,0.753,0.753}}$0.257 \pm 0.014$ & {\cellcolor[rgb]{0.753,0.753,0.753}}$0.266 \pm 0.018$ & {\cellcolor[rgb]{0.753,0.753,0.753}}$0.174 \pm 0.011$ & {\cellcolor[rgb]{0.753,0.753,0.753}}$1.77e-38$ & {\cellcolor[rgb]{0.753,0.753,0.753}}$\mathbf{8.87e-38}$ \\
				& MEC & $0.543 \pm 0.044$ & $0.529 \pm 0.04$ & $0.479 \pm 0.044$ & $0.476 \pm 0.045$ & $0.575 \pm 0.083$ & $0.563 \pm 0.04$ & $0.547 \pm 0.043$ & $0.539 \pm 0.035$ & $0.604 \pm 0.058$ & $0.399 \pm 0.039$ & $4.01e-34$ & $\mathbf{2.00e-33}$ \\
				& MEW & {\cellcolor[rgb]{0.753,0.753,0.753}}$1.195 \pm 0.029$ & {\cellcolor[rgb]{0.753,0.753,0.753}}$1.233 \pm 0.047$ & {\cellcolor[rgb]{0.753,0.753,0.753}}$1.164 \pm 0.046$ & {\cellcolor[rgb]{0.753,0.753,0.753}}$1.082 \pm 0.051$ & {\cellcolor[rgb]{0.753,0.753,0.753}}$1.114 \pm 0.057$ & {\cellcolor[rgb]{0.753,0.753,0.753}}$1.266 \pm 0.046$ & {\cellcolor[rgb]{0.753,0.753,0.753}}$1.272 \pm 0.046$ & {\cellcolor[rgb]{0.753,0.753,0.753}}$1.229 \pm 0.05$ & {\cellcolor[rgb]{0.753,0.753,0.753}}$1.267 \pm 0.046$ & {\cellcolor[rgb]{0.753,0.753,0.753}}$1.099 \pm 0.045$ & {\cellcolor[rgb]{0.753,0.753,0.753}}$1.58e-41$ & {\cellcolor[rgb]{0.753,0.753,0.753}}$\mathbf{7.92e-41}$ \\
				& Time (sec) & $0.503 \pm 0.059$ & $2.041 \pm 0.222$ & $2.177 \pm 0.197$ & $2.339 \pm 0.29$ & $2.247 \pm 0.248$ & -- & -- & -- & -- & $7.15 \pm 0.497$ & $\mathbf{3.71e-23}$ & -- \\
				\hline
				\multirow{6}{*}{ISOLET} & $F_1$ & {\cellcolor[rgb]{0.753,0.753,0.753}}$0.86 \pm 0.009$ & {\cellcolor[rgb]{0.753,0.753,0.753}}$0.915 \pm 0.007$ & {\cellcolor[rgb]{0.753,0.753,0.753}}$0.927 \pm 0.006$ & {\cellcolor[rgb]{0.753,0.753,0.753}}$0.924 \pm 0.006$ & {\cellcolor[rgb]{0.753,0.753,0.753}}$0.92 \pm 0.007$ & {\cellcolor[rgb]{0.753,0.753,0.753}}$0.918 \pm 0.006$ & {\cellcolor[rgb]{0.753,0.753,0.753}}$0.932 \pm 0.005$ & {\cellcolor[rgb]{0.753,0.753,0.753}}$0.936 \pm 0.005$ & {\cellcolor[rgb]{0.753,0.753,0.753}}$0.939 \pm 0.005$ & {\cellcolor[rgb]{0.753,0.753,0.753}}$0.957 \pm 0.004$ & {\cellcolor[rgb]{0.753,0.753,0.753}}$4.81e-48$ & {\cellcolor[rgb]{0.753,0.753,0.753}}$\mathbf{2.40e-47}$ \\
				& AUC & $0.994 \pm 0.001$ & $0.997 \pm 0.0$ & $0.998 \pm 0.0$ & $0.998 \pm 0.0$ & $0.998 \pm 0.0$ & $0.997 \pm 0.0$ & $0.998 \pm 0.0$ & $0.998 \pm 0.0$ & $0.999 \pm 0.0$ & $0.999 \pm 0.0$ & $3.72e-50$ & $\mathbf{1.86e-49}$ \\
				& Log-Loss & {\cellcolor[rgb]{0.753,0.753,0.753}}$0.44 \pm 0.03$ & {\cellcolor[rgb]{0.753,0.753,0.753}}$0.297 \pm 0.018$ & {\cellcolor[rgb]{0.753,0.753,0.753}}$0.256 \pm 0.017$ & {\cellcolor[rgb]{0.753,0.753,0.753}}$0.261 \pm 0.016$ & {\cellcolor[rgb]{0.753,0.753,0.753}}$0.275 \pm 0.021$ & {\cellcolor[rgb]{0.753,0.753,0.753}}$0.309 \pm 0.023$ & {\cellcolor[rgb]{0.753,0.753,0.753}}$0.265 \pm 0.019$ & {\cellcolor[rgb]{0.753,0.753,0.753}}$0.254 \pm 0.015$ & {\cellcolor[rgb]{0.753,0.753,0.753}}$0.252 \pm 0.017$ & {\cellcolor[rgb]{0.753,0.753,0.753}}$0.152 \pm 0.008$ & {\cellcolor[rgb]{0.753,0.753,0.753}}$1.39e-43$ & {\cellcolor[rgb]{0.753,0.753,0.753}}$\mathbf{6.94e-43}$ \\
				& MEC & $0.69 \pm 0.123$ & $0.564 \pm 0.063$ & $0.49 \pm 0.084$ & $0.503 \pm 0.056$ & $0.526 \pm 0.064$ & $0.641 \pm 0.081$ & $0.59 \pm 0.086$ & $0.586 \pm 0.073$ & $0.604 \pm 0.068$ & $0.258 \pm 0.028$ & $2.19e-33$ & $\mathbf{1.10e-32}$ \\
				& MEW & {\cellcolor[rgb]{0.753,0.753,0.753}}$1.547 \pm 0.147$ & {\cellcolor[rgb]{0.753,0.753,0.753}}$1.521 \pm 0.087$ & {\cellcolor[rgb]{0.753,0.753,0.753}}$1.461 \pm 0.122$ & {\cellcolor[rgb]{0.753,0.753,0.753}}$1.488 \pm 0.089$ & {\cellcolor[rgb]{0.753,0.753,0.753}}$1.515 \pm 0.1$ & {\cellcolor[rgb]{0.753,0.753,0.753}}$1.617 \pm 0.098$ & {\cellcolor[rgb]{0.753,0.753,0.753}}$1.594 \pm 0.102$ & {\cellcolor[rgb]{0.753,0.753,0.753}}$1.621 \pm 0.098$ & {\cellcolor[rgb]{0.753,0.753,0.753}}$1.667 \pm 0.092$ & {\cellcolor[rgb]{0.753,0.753,0.753}}$1.327 \pm 0.066$ & {\cellcolor[rgb]{0.753,0.753,0.753}}$7.16e-27$ & {\cellcolor[rgb]{0.753,0.753,0.753}}$\mathbf{3.58e-26}$ \\
				& Time (sec) & $0.78 \pm 0.12$ & $1.614 \pm 0.653$ & $1.5 \pm 0.6$ & $1.914 \pm 0.536$ & $1.6 \pm 0.599$ & -- & -- & -- & -- & $5.51 \pm 0.728$ & $\mathbf{2.72e-20}$ & -- \\
				\hline
				\multirow{6}{*}{PD} & $F_1$ & {\cellcolor[rgb]{0.753,0.753,0.753}}$0.781 \pm 0.022$ & {\cellcolor[rgb]{0.753,0.753,0.753}}$0.785 \pm 0.02$ & {\cellcolor[rgb]{0.753,0.753,0.753}}$0.793 \pm 0.027$ & {\cellcolor[rgb]{0.753,0.753,0.753}}$0.79 \pm 0.033$ & {\cellcolor[rgb]{0.753,0.753,0.753}}$0.781 \pm 0.029$ & {\cellcolor[rgb]{0.753,0.753,0.753}}$0.789 \pm 0.022$ & {\cellcolor[rgb]{0.753,0.753,0.753}}$0.803 \pm 0.023$ & {\cellcolor[rgb]{0.753,0.753,0.753}}$0.805 \pm 0.021$ & {\cellcolor[rgb]{0.753,0.753,0.753}}$0.809 \pm 0.023$ & {\cellcolor[rgb]{0.753,0.753,0.753}}$0.791 \pm 0.026$ & {\cellcolor[rgb]{0.753,0.753,0.753}}$5.86e-11$ & {\cellcolor[rgb]{0.753,0.753,0.753}}$\mathbf{2.93e-10}$ \\
				& AUC & $0.862 \pm 0.022$ & $0.861 \pm 0.019$ & $0.863 \pm 0.026$ & $0.863 \pm 0.025$ & $0.853 \pm 0.028$ & $0.867 \pm 0.02$ & $0.874 \pm 0.02$ & $0.879 \pm 0.02$ & $0.88 \pm 0.019$ & $0.883 \pm 0.024$ & $6.61e-19$ & $\mathbf{3.30e-18}$ \\
				& Log-Loss & {\cellcolor[rgb]{0.753,0.753,0.753}}$0.45 \pm 0.031$ & {\cellcolor[rgb]{0.753,0.753,0.753}}$0.455 \pm 0.031$ & {\cellcolor[rgb]{0.753,0.753,0.753}}$0.447 \pm 0.038$ & {\cellcolor[rgb]{0.753,0.753,0.753}}$0.447 \pm 0.038$ & {\cellcolor[rgb]{0.753,0.753,0.753}}$0.469 \pm 0.065$ & {\cellcolor[rgb]{0.753,0.753,0.753}}$0.444 \pm 0.028$ & {\cellcolor[rgb]{0.753,0.753,0.753}}$0.427 \pm 0.032$ & {\cellcolor[rgb]{0.753,0.753,0.753}}$0.417 \pm 0.028$ & {\cellcolor[rgb]{0.753,0.753,0.753}}$0.419 \pm 0.027$ & {\cellcolor[rgb]{0.753,0.753,0.753}}$0.451 \pm 0.045$ & {\cellcolor[rgb]{0.753,0.753,0.753}}$1.41e-17$ & {\cellcolor[rgb]{0.753,0.753,0.753}}$\mathbf{7.06e-17}$ \\
				& MEC & $0.545 \pm 0.045$ & $0.536 \pm 0.042$ & $0.546 \pm 0.067$ & $0.559 \pm 0.07$ & $0.586 \pm 0.123$ & $0.547 \pm 0.038$ & $0.56 \pm 0.036$ & $0.568 \pm 0.029$ & $0.593 \pm 0.067$ & $0.59 \pm 0.043$ & $9.06e-12$ & $\mathbf{4.53e-11}$ \\
				& MEW & {\cellcolor[rgb]{0.753,0.753,0.753}}$0.821 \pm 0.045$ & {\cellcolor[rgb]{0.753,0.753,0.753}}$0.808 \pm 0.049$ & {\cellcolor[rgb]{0.753,0.753,0.753}}$0.806 \pm 0.05$ & {\cellcolor[rgb]{0.753,0.753,0.753}}$0.818 \pm 0.054$ & {\cellcolor[rgb]{0.753,0.753,0.753}}$0.832 \pm 0.056$ & {\cellcolor[rgb]{0.753,0.753,0.753}}$0.82 \pm 0.045$ & {\cellcolor[rgb]{0.753,0.753,0.753}}$0.827 \pm 0.047$ & {\cellcolor[rgb]{0.753,0.753,0.753}}$0.84 \pm 0.039$ & {\cellcolor[rgb]{0.753,0.753,0.753}}$0.851 \pm 0.041$ & {\cellcolor[rgb]{0.753,0.753,0.753}}$0.841 \pm 0.044$ & {\cellcolor[rgb]{0.753,0.753,0.753}}$1.96e-10$ & {\cellcolor[rgb]{0.753,0.753,0.753}}$\mathbf{9.81e-10}$ \\
				& Time (sec) & $0.049 \pm 0.04$ & $0.041 \pm 0.003$ & $0.04 \pm 0.002$ & $0.04 \pm 0.002$ & $0.04 \pm 0.002$ & -- & -- & -- & -- & $0.097 \pm 0.006$ & $\mathbf{3.64e-13}$ & -- \\
				\hline
			\end{tabular}
		}
	\end{sidewaystable*}
	\FloatBarrier
	
	\begin{figure*}[t] 
		\centering
		\subfloat[APSF]{\includegraphics[width=0.24\textwidth]{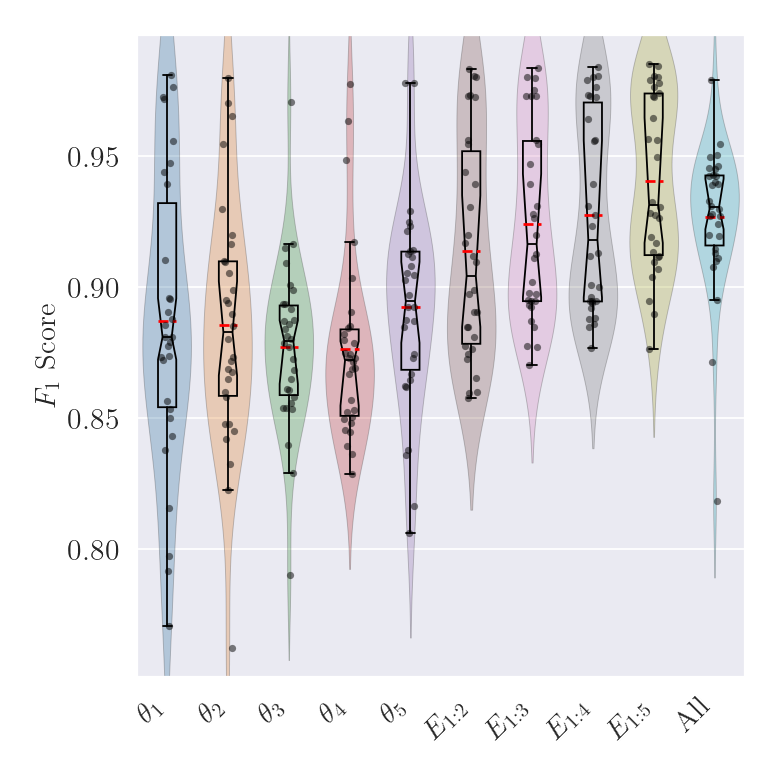}\label{fig:lrapsf_F1}}%
		\hfill
		\subfloat[ARWPM]{\includegraphics[width=0.24\textwidth]{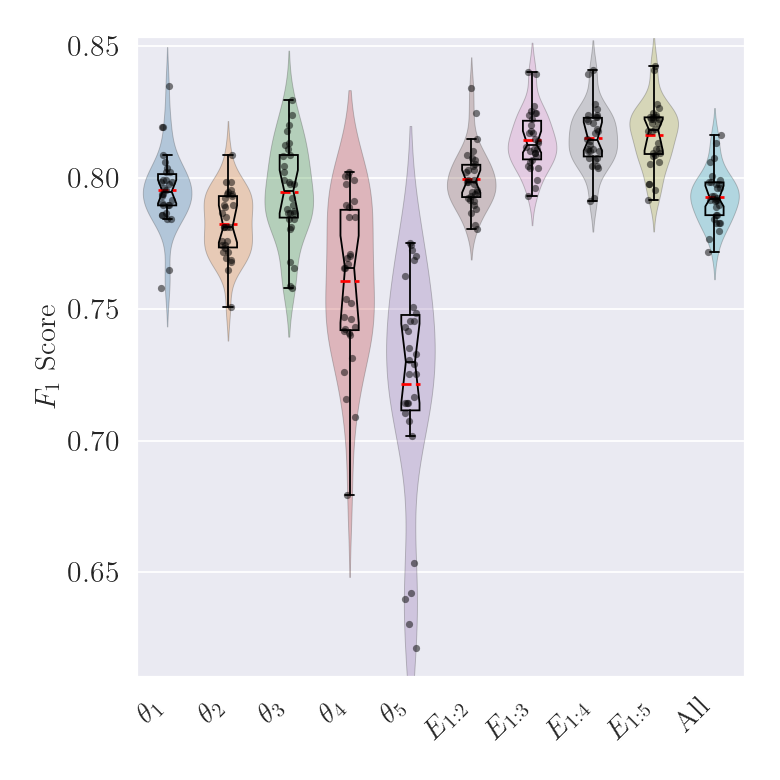}\label{fig:lrarwpm_F1}}%
		\hfill
		\subfloat[GECR]{\includegraphics[width=0.24\textwidth]{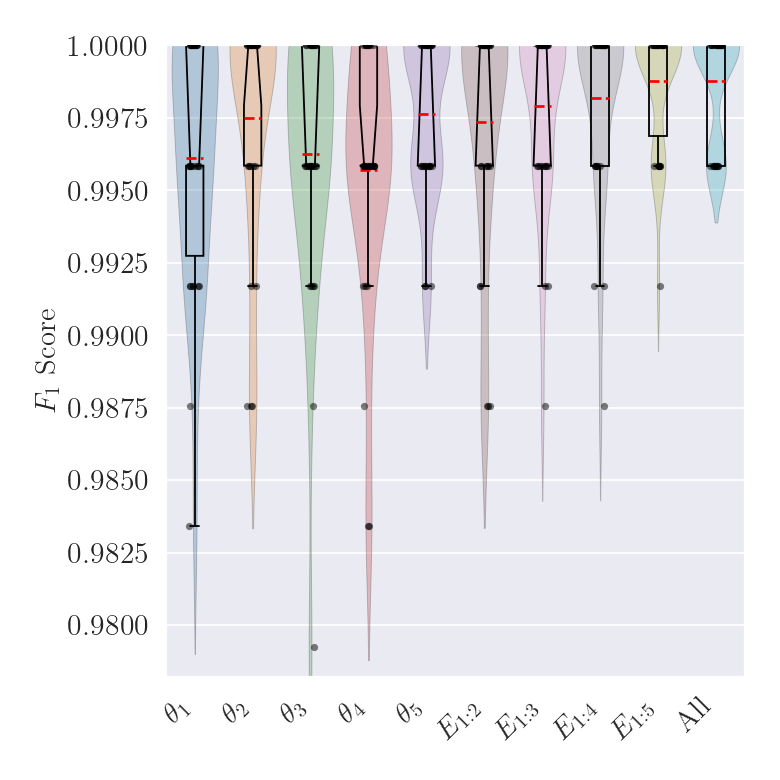}\label{fig:lrgecr_F1}}%
		\hfill
		\subfloat[GFE]{\includegraphics[width=0.24\textwidth]{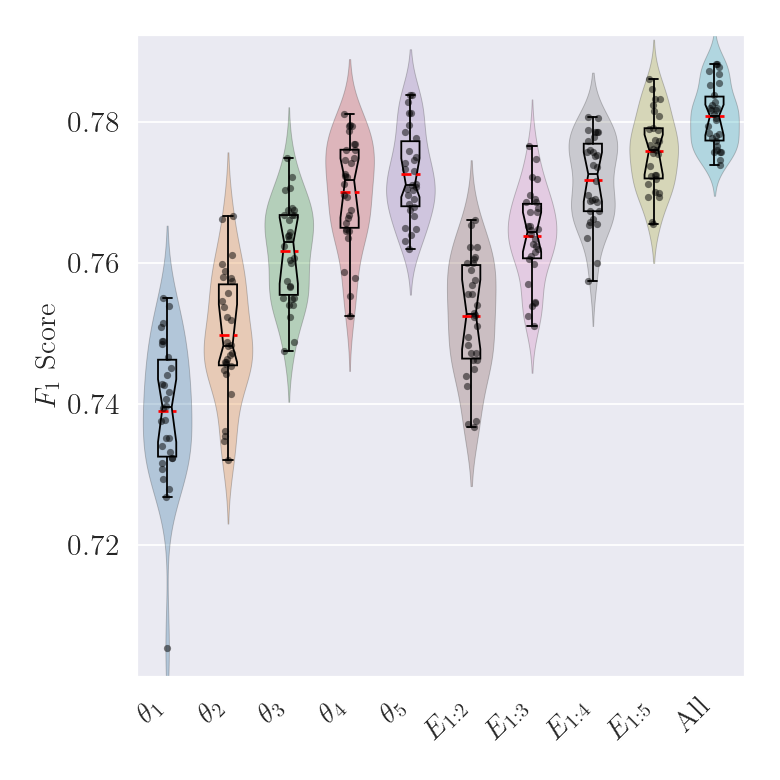}\label{fig:lrgfe_F1}}
		
		\subfloat[GSAD]{\includegraphics[width=0.24\textwidth]{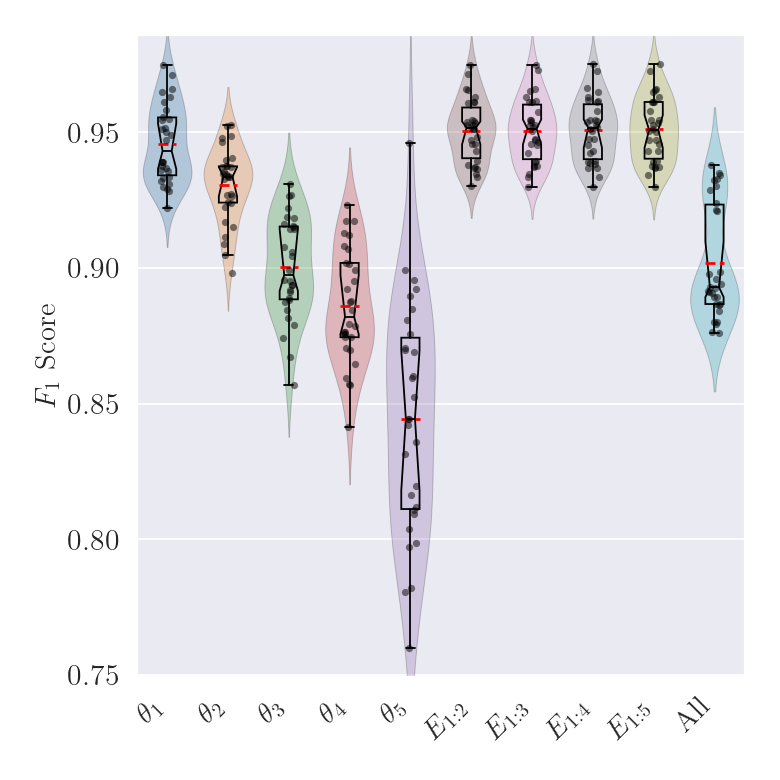}\label{fig:fpgsad_F1}}%
		\hfill
		\subfloat[HAPT]{\includegraphics[width=0.24\textwidth]{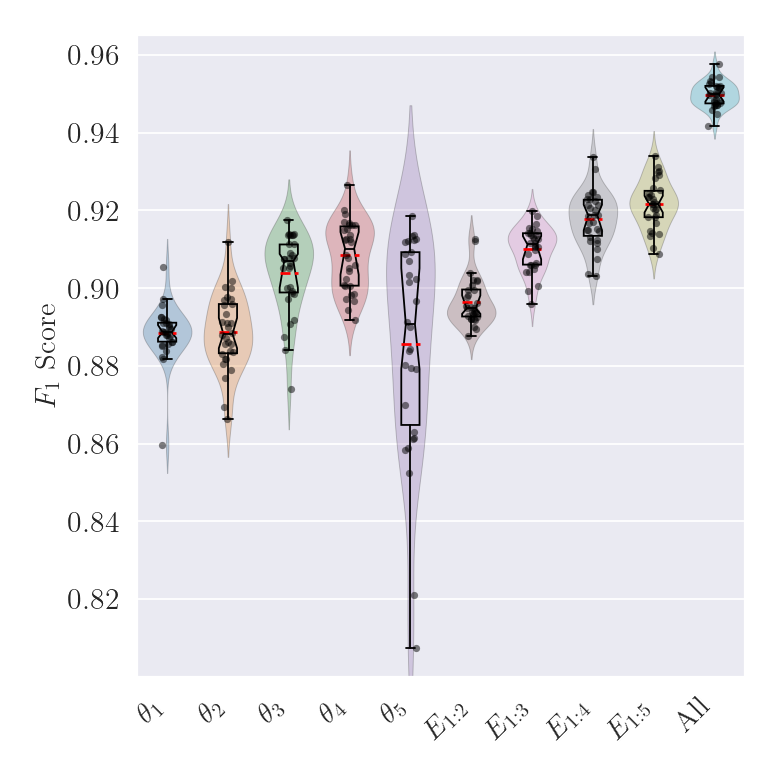}\label{fig:lrhapt_F1}}%
		\hfill
		\subfloat[ISOLET]{\includegraphics[width=0.24\textwidth]{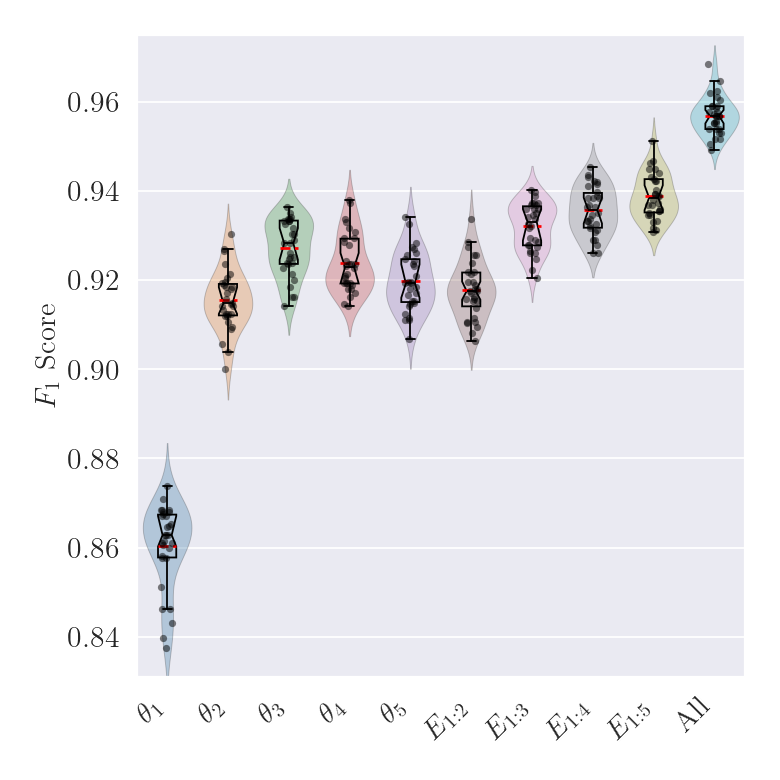}\label{fig:lrisolet_F1}}%
		\hfill
		\subfloat[PD]{\includegraphics[width=0.24\textwidth]{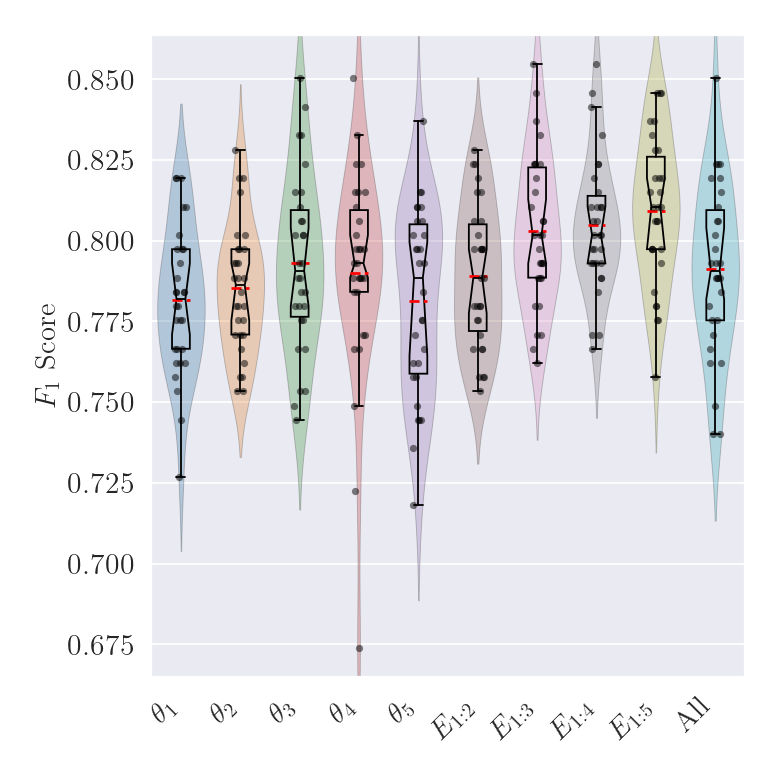}\label{fig:lrpd_F1}}
		\caption[The distribution of the obtained $F_1$ score for 30 Logistic Regression runs.]{The raincloud plot of $F_1$ score results obtained from 30 Logistic Regression runs.}
		
		\label{fig:lr_F1}
	\end{figure*}
	
	\begin{figure*}[t] 
		\centering
		\subfloat[APSF]{\includegraphics[width=0.24\textwidth]{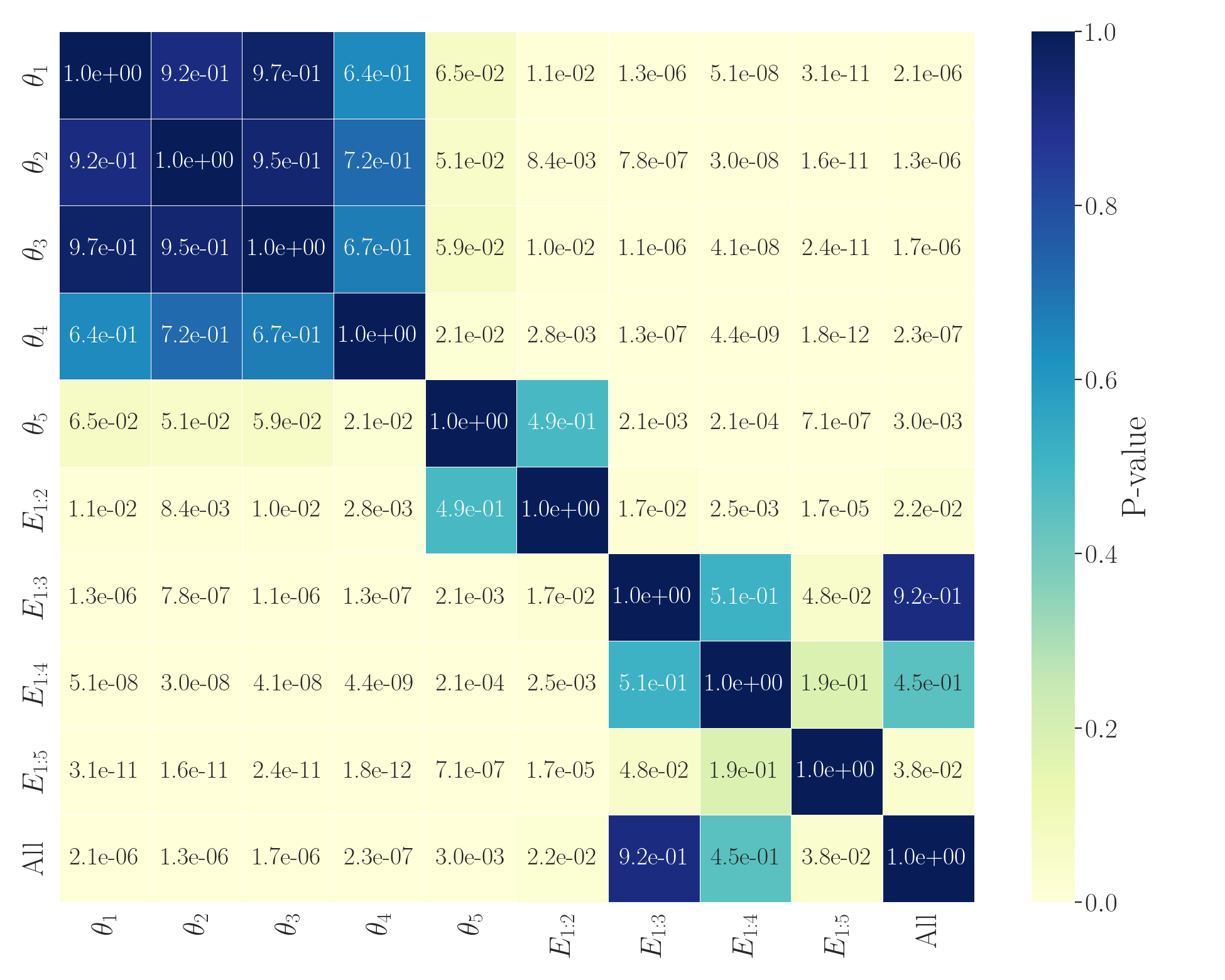}\label{fig:lrnemapsf_F1}}%
		\hfill
		\subfloat[ARWPM]{\includegraphics[width=0.24\textwidth]{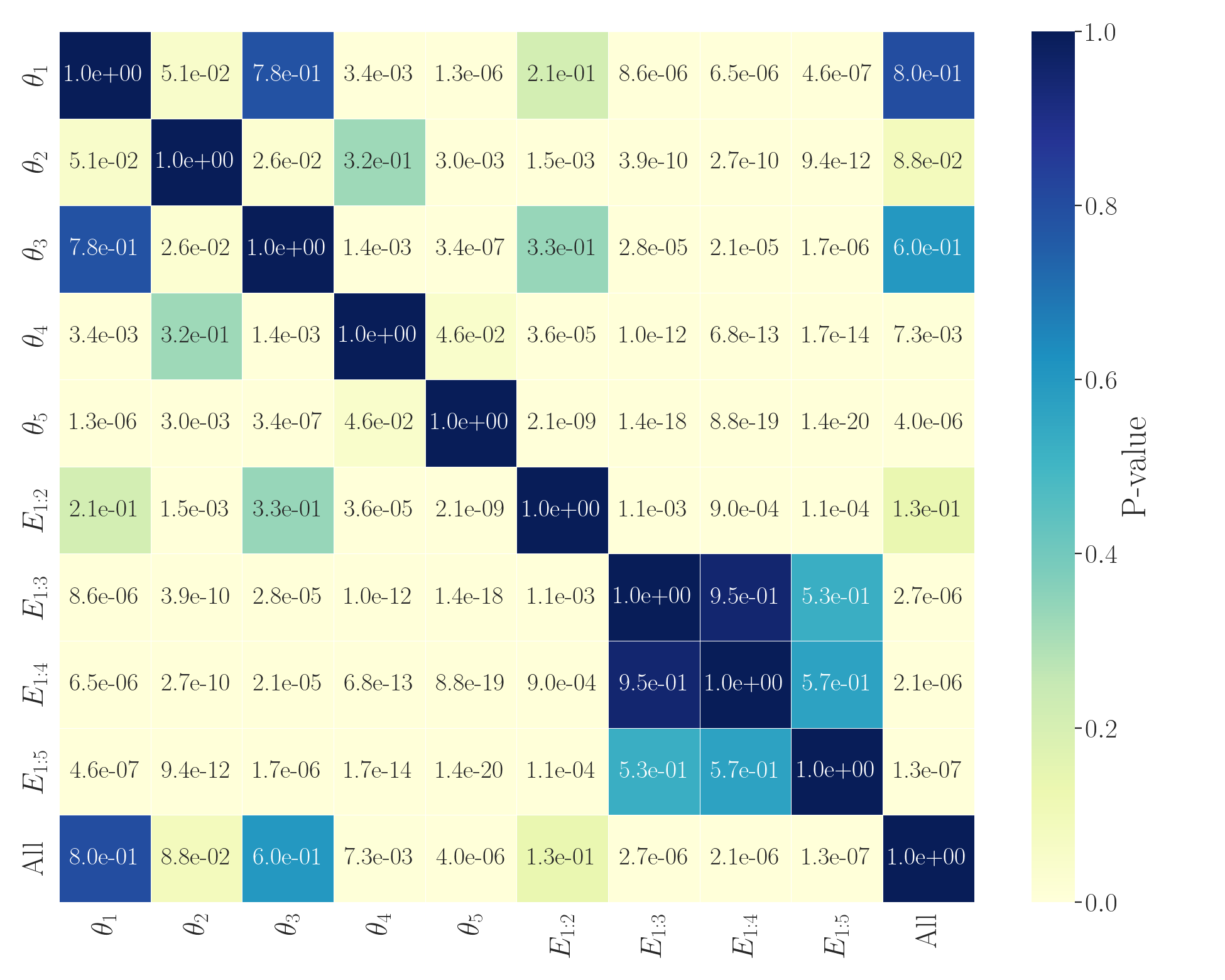}\label{fig:lrnemarwpm_F1}}%
		\hfill
		\subfloat[GECR]{\includegraphics[width=0.24\textwidth]{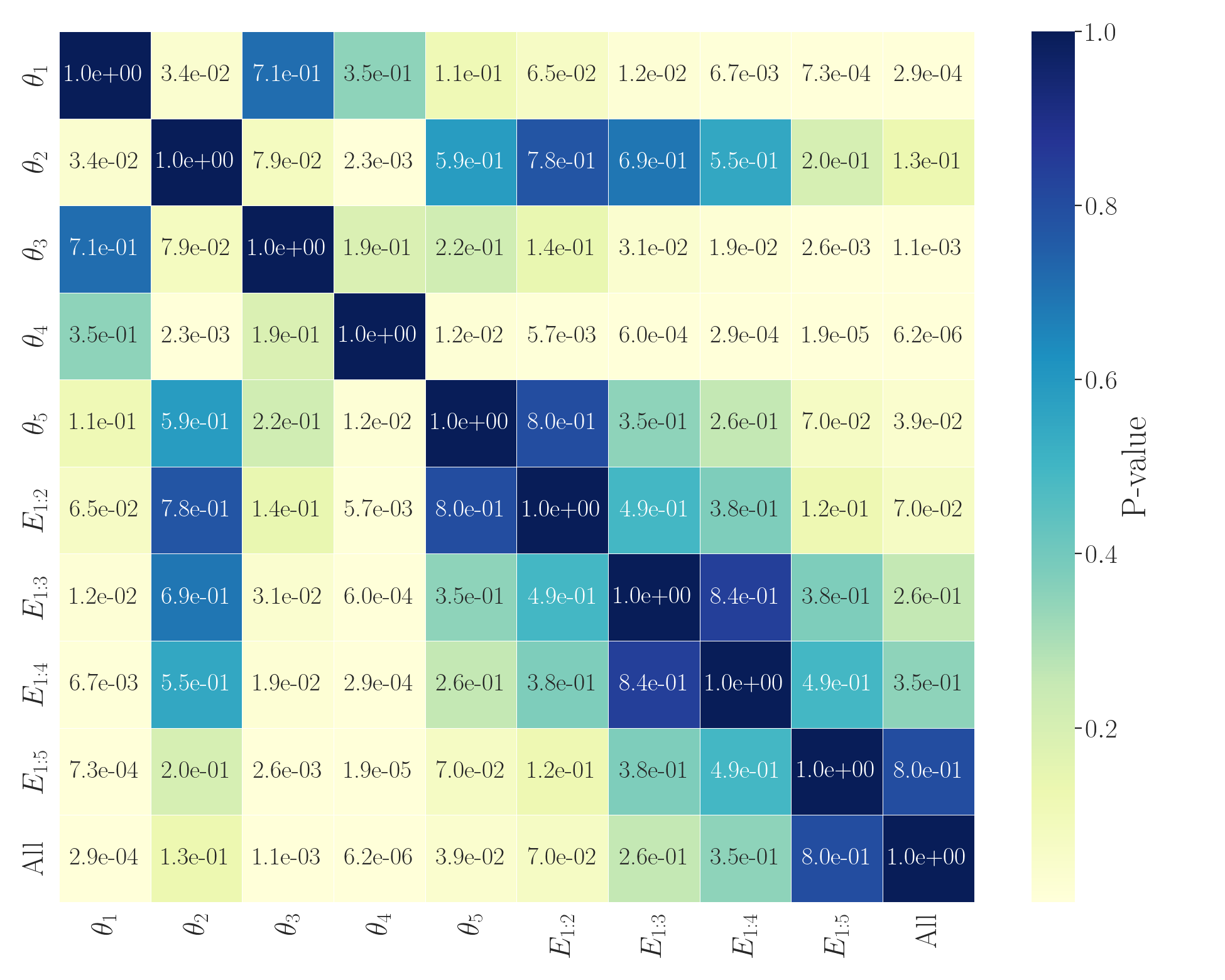}\label{fig:lrnemgecr_F1}}%
		\hfill
		\subfloat[GFE]{\includegraphics[width=0.24\textwidth]{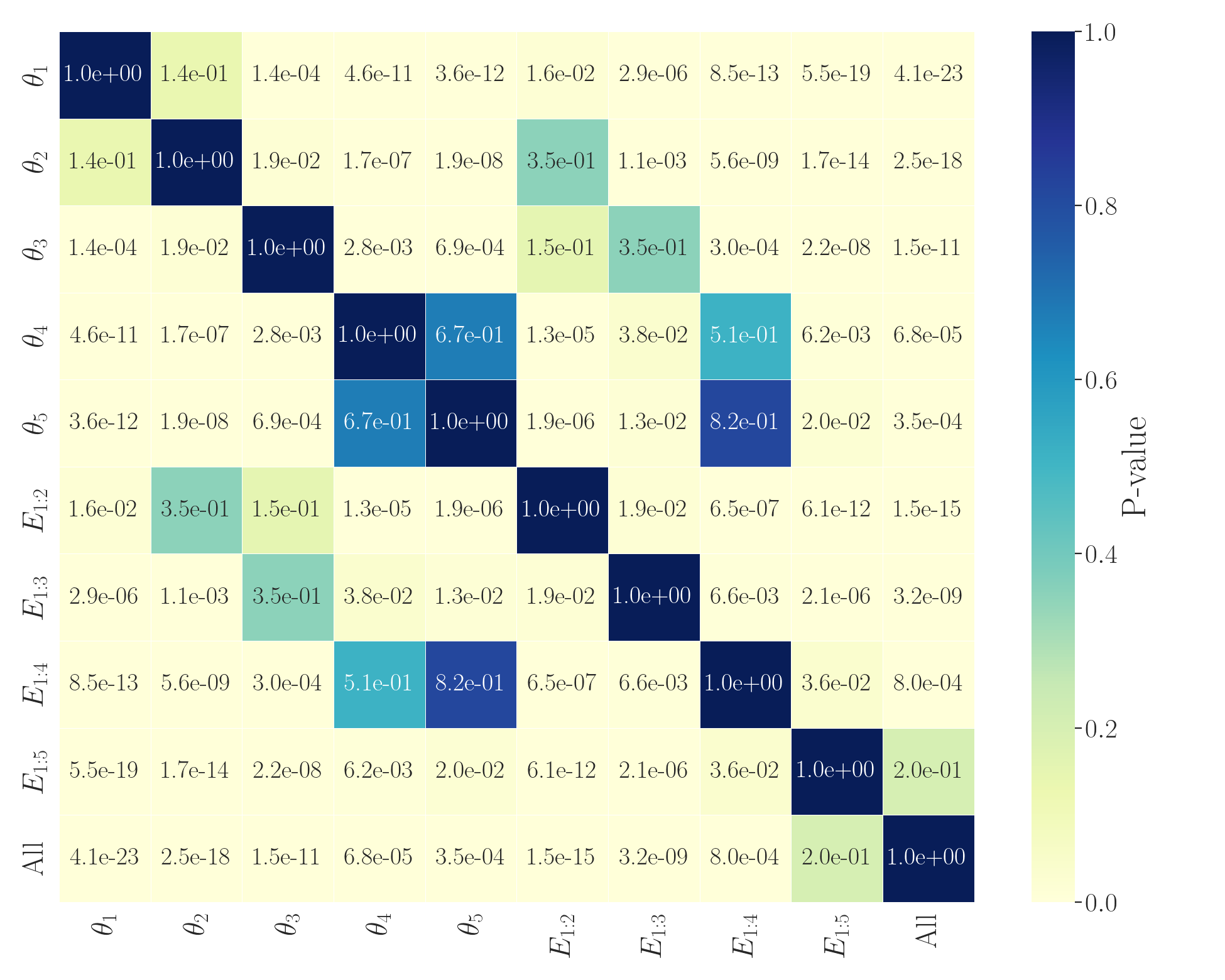}\label{fig:lrnemgfe_F1}}
		
		\subfloat[GSAD]{\includegraphics[width=0.24\textwidth]{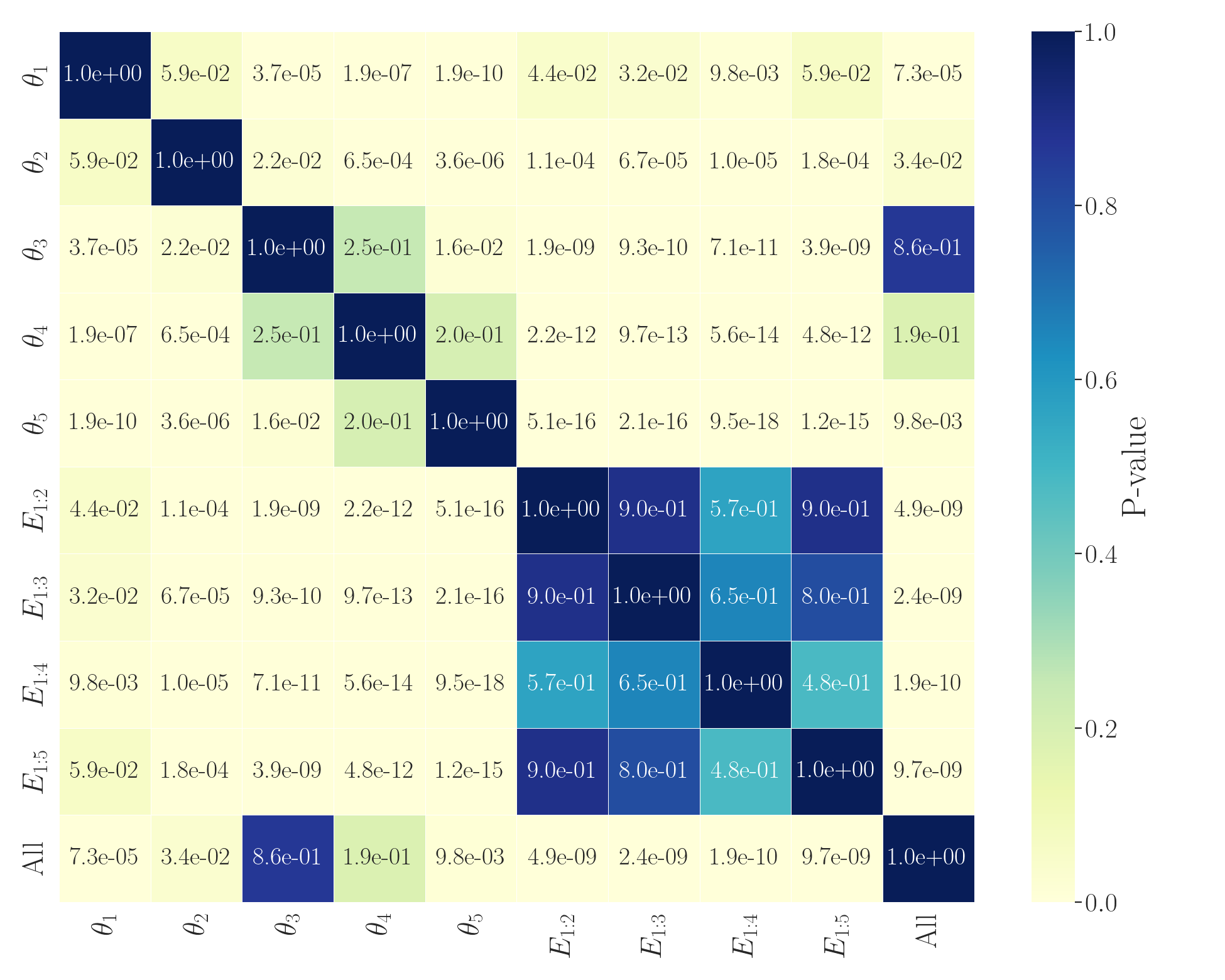}\label{fig:lrnemgsad_F1}}%
		\hfill
		\subfloat[HAPT]{\includegraphics[width=0.24\textwidth]{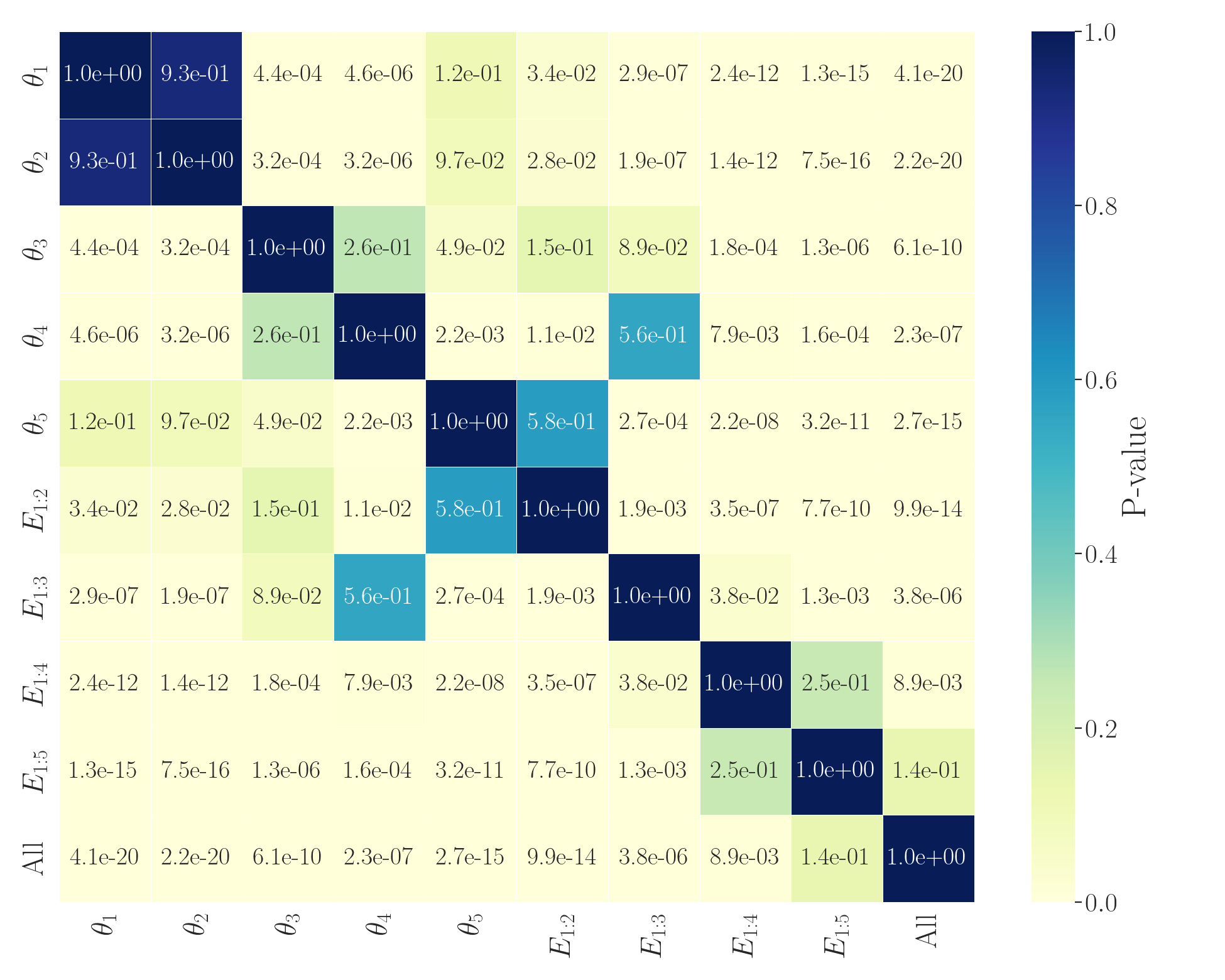}\label{fig:lrnemhapt_F1}}%
		\hfill
		\subfloat[ISOLET]{\includegraphics[width=0.24\textwidth]{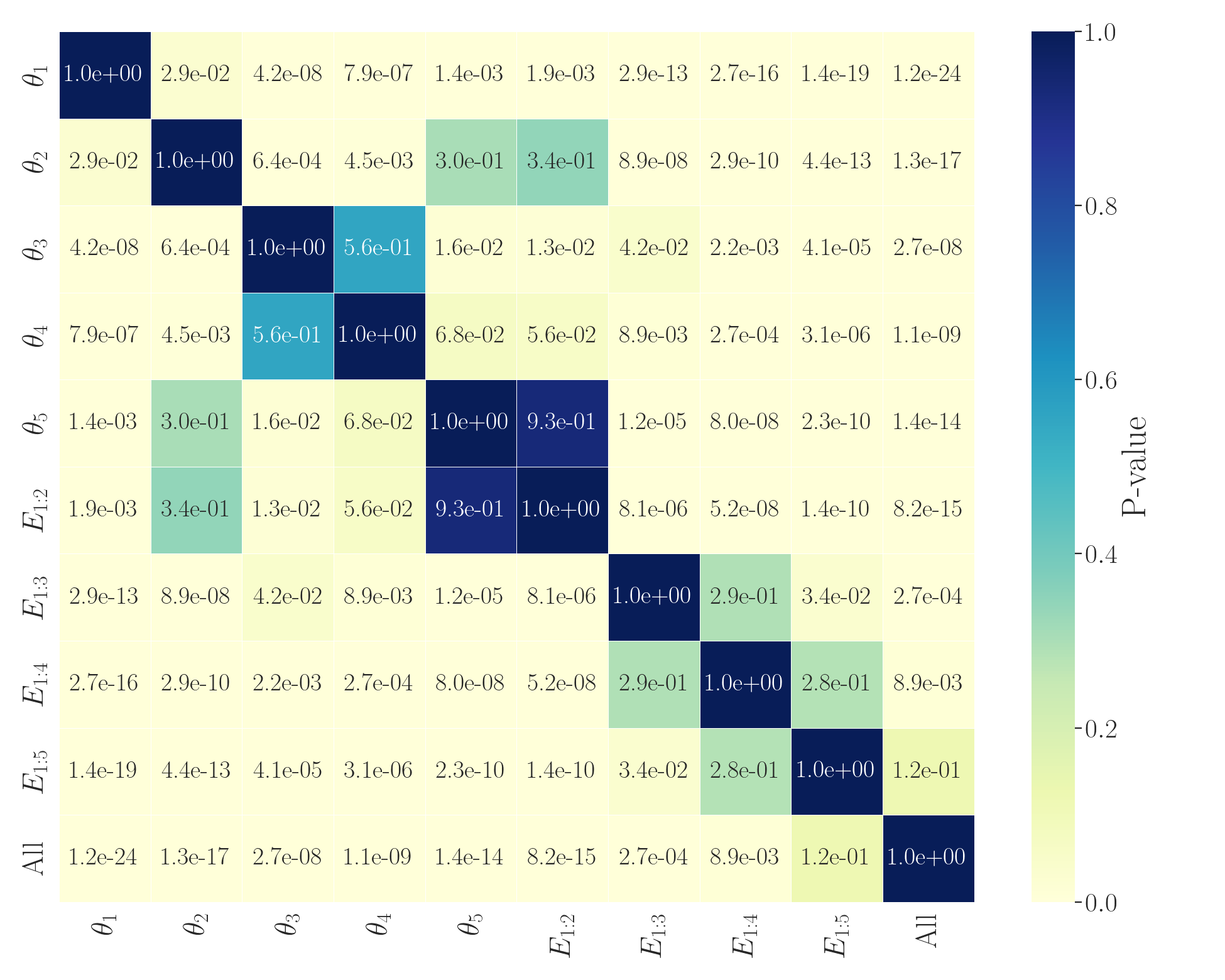}\label{fig:lrnemisolet_F1}}%
		\hfill
		\subfloat[PD]{\includegraphics[width=0.24\textwidth]{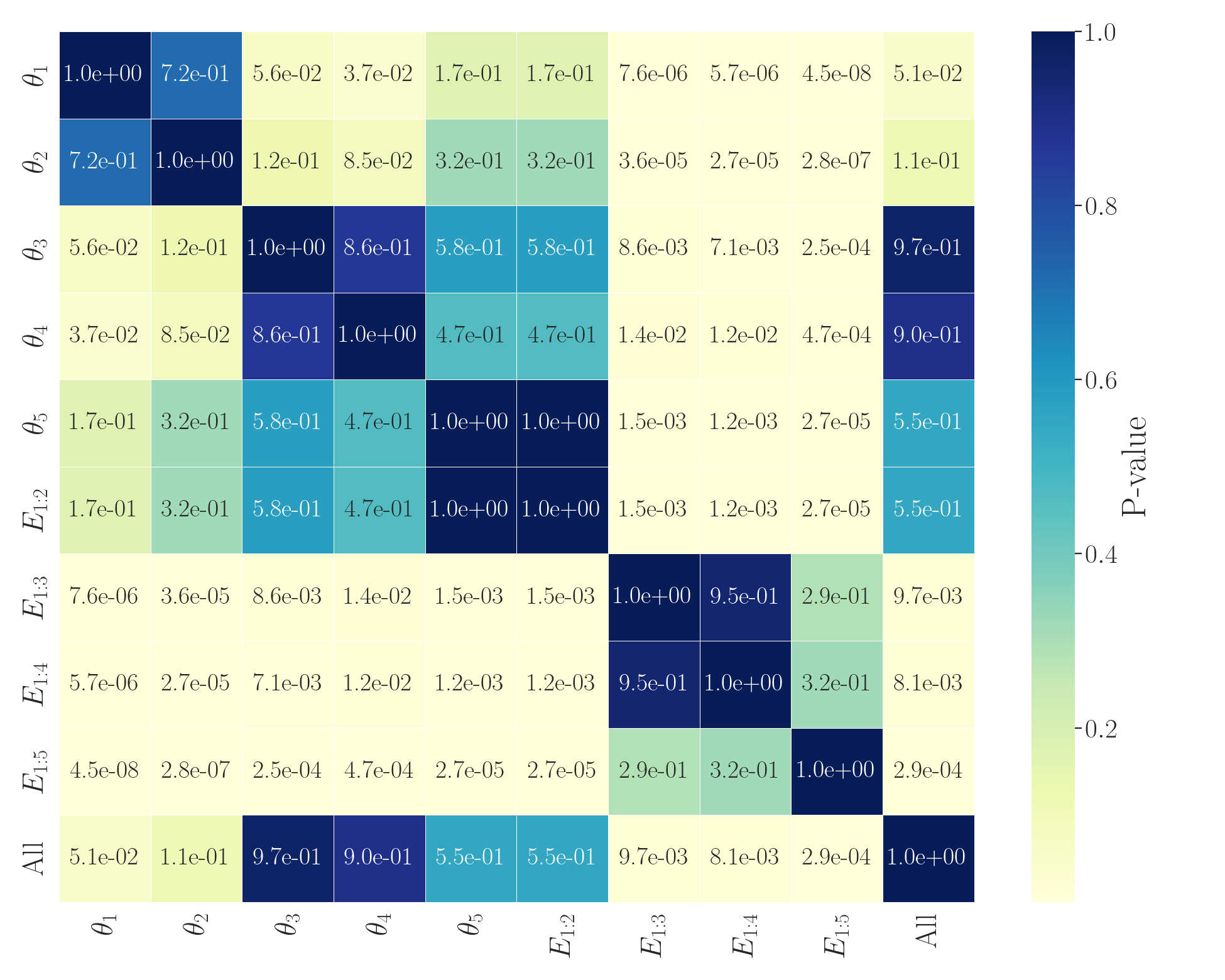}\label{fig:lrnempd_F1}}
		\caption[The adjusted Conover's P-values for the obtained $F_1$ score in 30 Logistic Regression runs.]{The results of the Conover post-hoc test on testing data’s $F_1$ score obtained from 30 Logistic Regression runs.}
		
		\label{fig:lrnem_F1}
	\end{figure*}
	\FloatBarrier
	
	\begin{figure*}[htbp] 
		\centering
		\subfloat[APSF]{\includegraphics[width=0.24\textwidth]{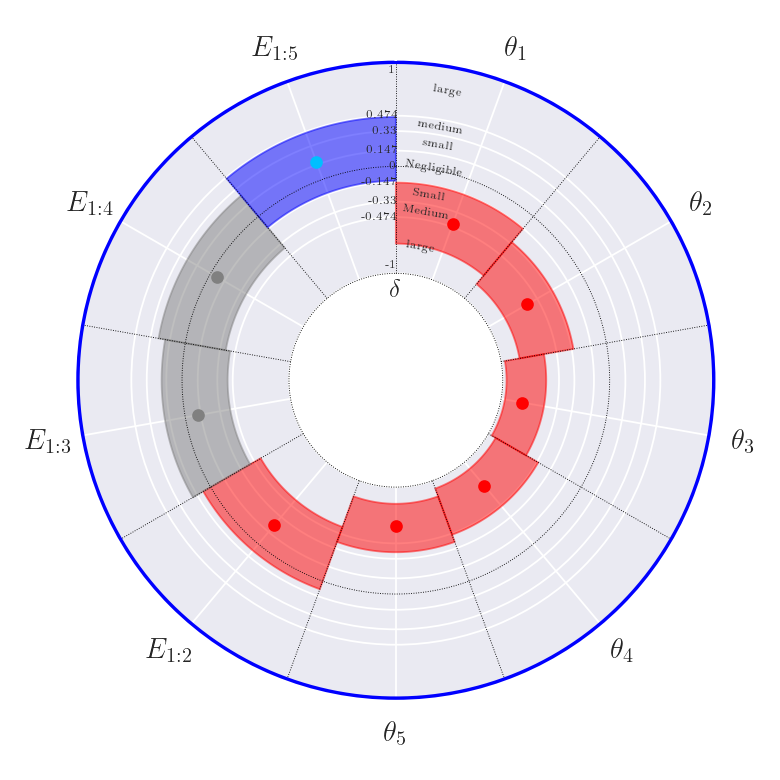}\label{fig:lrcliffapsf_F1}}%
		\hfill
		\subfloat[ARWPM]{\includegraphics[width=0.24\textwidth]{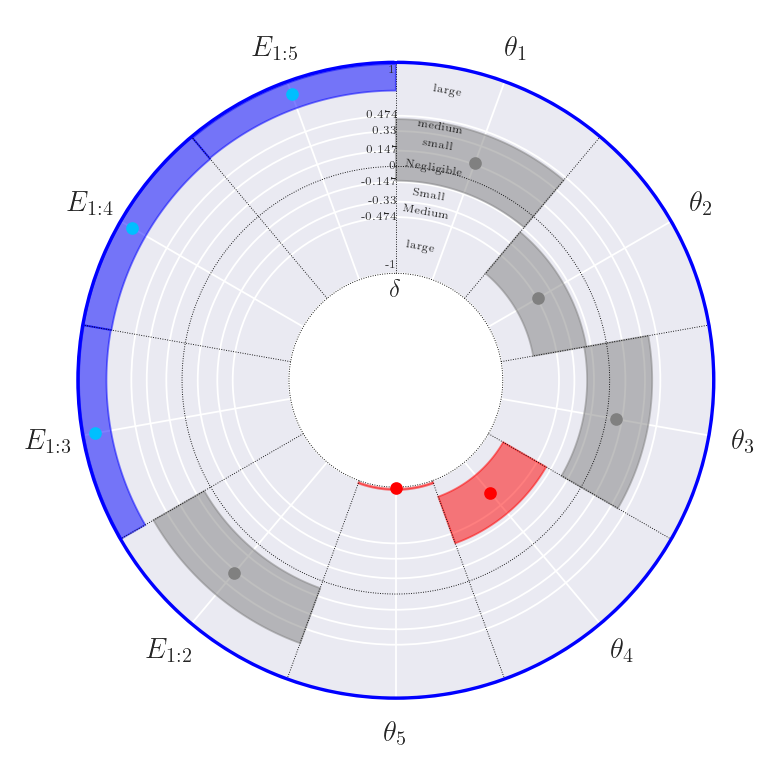}\label{fig:lrcliffarwpm_F1}}%
		\hfill
		\subfloat[GECR]{\includegraphics[width=0.24\textwidth]{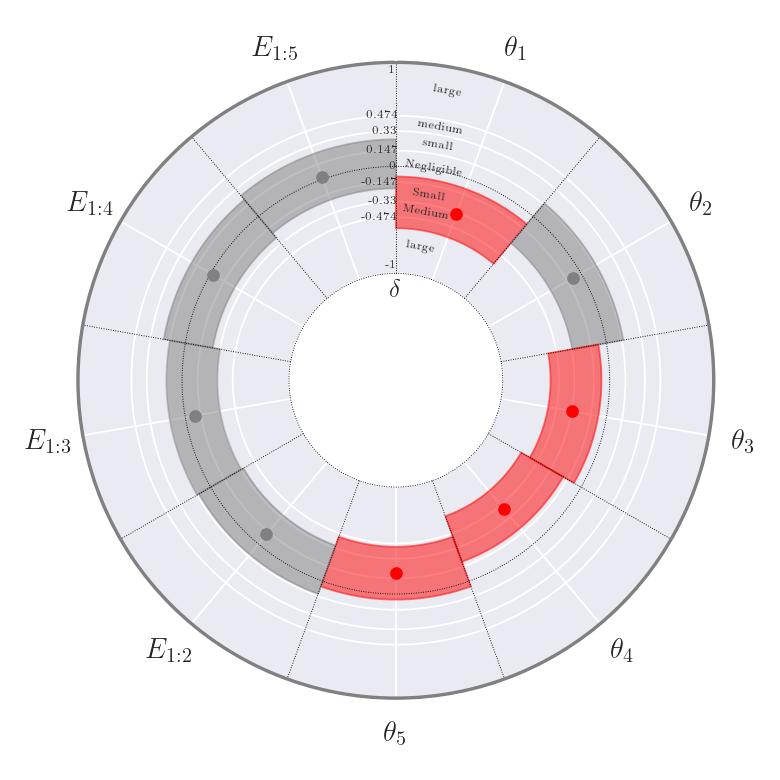}\label{fig:lrcliffgecr_F1}}%
		\hfill
		\subfloat[GFE]{\includegraphics[width=0.24\textwidth]{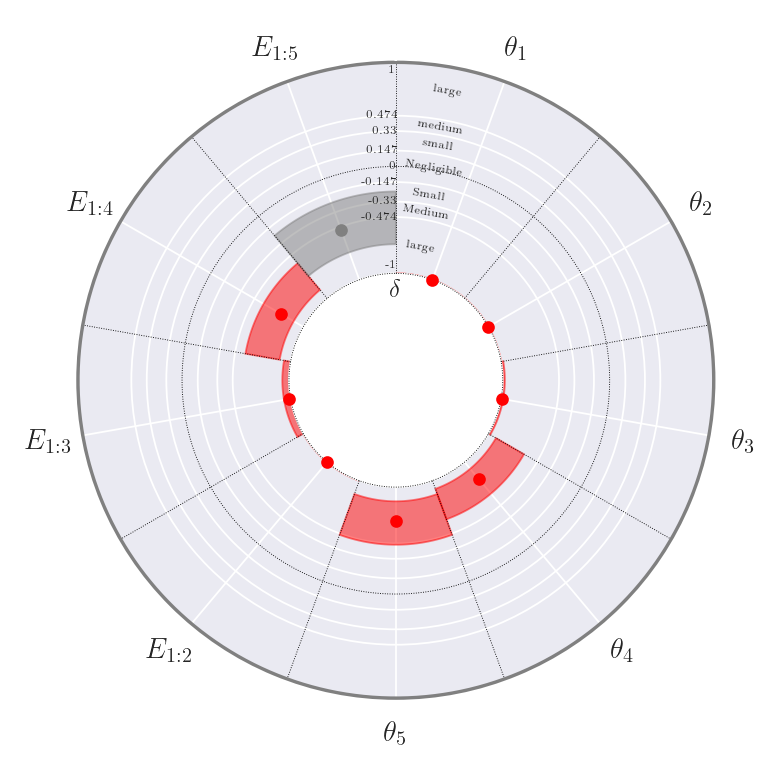}\label{fig:lrcliffgfe_F1}}
		
		\subfloat[GSAD]{\includegraphics[width=0.24\textwidth]{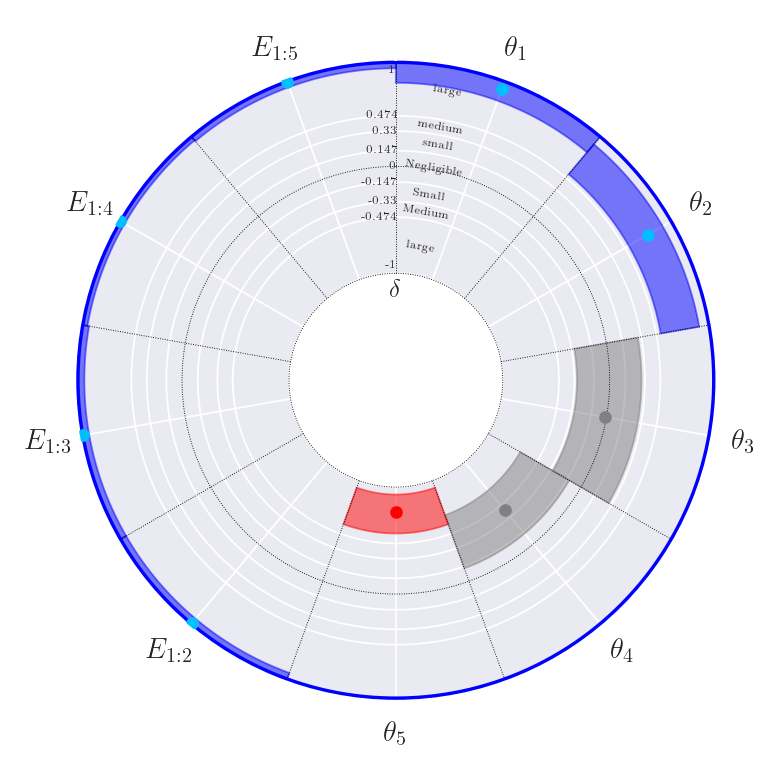}\label{fig:lrcliffgsad_F1}}%
		\hfill
		\subfloat[HAPT]{\includegraphics[width=0.24\textwidth]{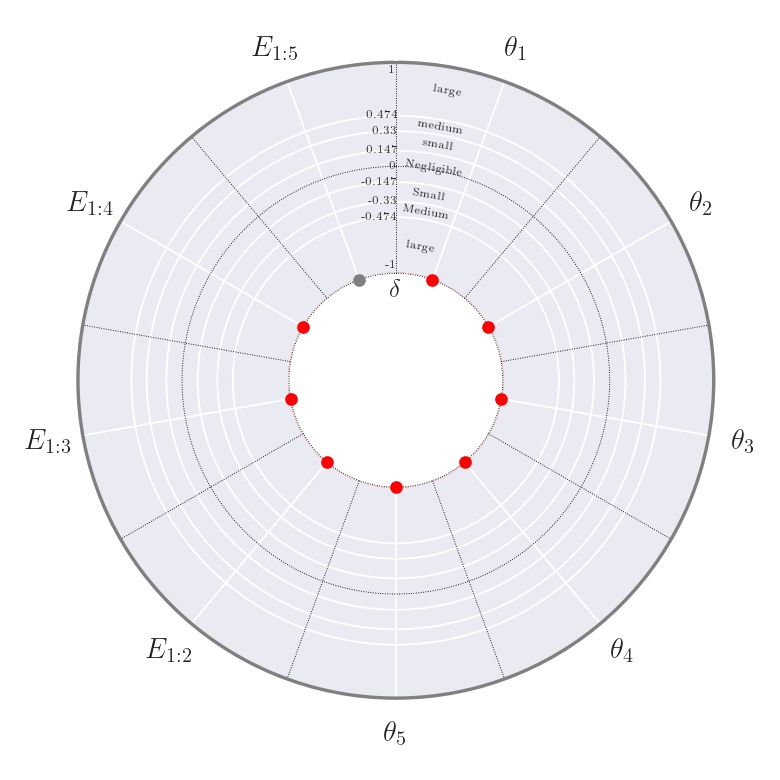}\label{fig:lrcliffhapt_F1}}%
		\hfill
		\subfloat[ISOLET]{\includegraphics[width=0.24\textwidth]{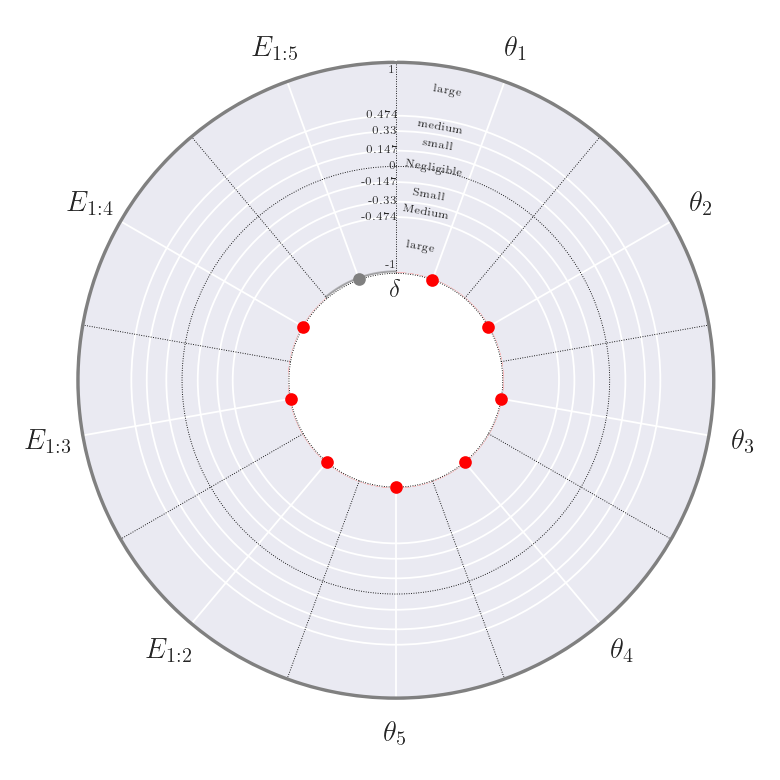}\label{fig:lrcliffisolet_F1}}%
		\hfill
		\subfloat[PD]{\includegraphics[width=0.24\textwidth]{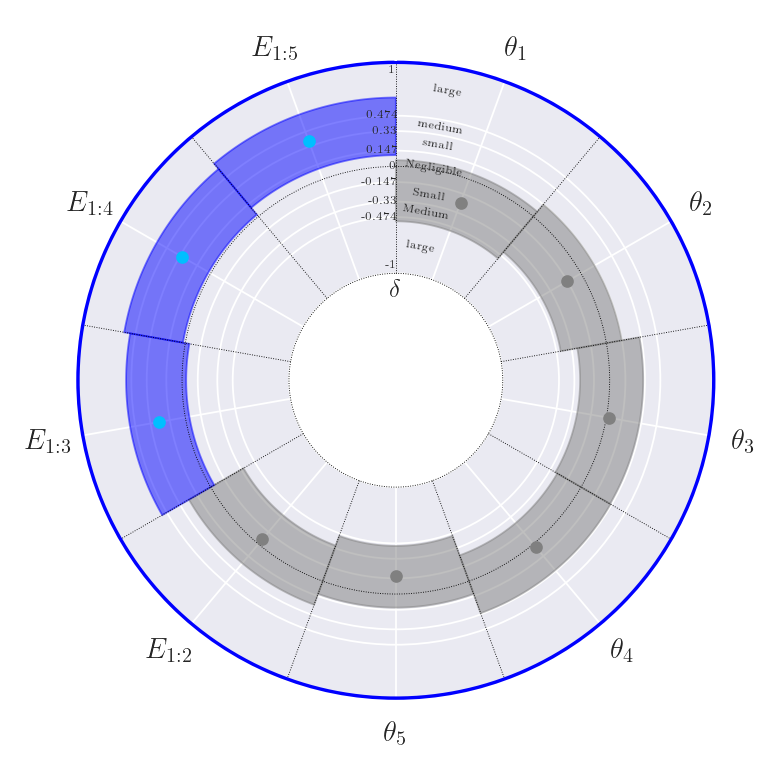}\label{fig:lrcliffpd_F1}}
		\caption[The Cliff's $\delta$ effect size measure and its 95\% confidence intervals for the $F_1$ score obtained from 30 Logistic Regression runs.]{Effect size analysis of test data $F_1$ scores across 30 Logistic Regression runs using Cliff's $\delta$. Each point represents the actual value obtained, with segments denoting 95\% confidence intervals based on 10,000 bootstrap resamplings. The outer ring color visualizes the statistical significance: grey illustrates no significant difference (adjusted Friedman's P-value$>0.05$), while color indicates significant differences; blue indicates at least one view and/or ensemble outperforms the benchmark (adjusted Conover's p-value$ < 0.05$, Cliff's $\delta > 0$), and red signifies all views and ensembles underperform relative to the benchmark (adjusted Conover's p-value$ < 0.05$, Cliff's $\delta < 0$). Segment colors show performance difference against the benchmark: grey for no significant difference (adjusted Conover's p-value$  > 0.05$), blue for better performance (Cliff's $\delta > 0$), and red for worse performance (Cliff's $\delta < 0$).}
		
		\label{fig:lrcliff_F1}
	\end{figure*}
	
	\begin{table*}[htbp]
		\centering
		\caption[The results of Friedman and Conover tests and Cliff's $\delta$ analysis for the $F_1$ score obtained from 30 Logistic Regression runs.]{Statistical comparison of $F_1$ score results for testing data obtained from Logistic Regression runs. W, T, and L denote win, tie, and loss based on adjusted Friedman and Conover's p-values. Effect sizes are calculated using Cliff's Delta method and are categorized as negligible, small, medium, or large.}
		\label{tab:lrf1}
		\resizebox{\linewidth}{!}{%
			\begin{tabular}{c|ccccccccc}
				\hline
				\multicolumn{10}{c}{Logistic Regression's $F_1$ Score}\\
				\hline
				Dataset & $\theta_1$ & $\theta_2$ & $\theta_3$ & $\theta_4$ & $\theta_5$ & $E_{1:2}$ & $E_{1:3}$ & $E_{1:4}$ & $E_{1:5}$ \\
				\hline
				APSF  & L (medium) & L (large) & L (large) & L (large) & L (large) & L (small) & T (negligible) & T (negligible) & W (small) \\
				ARWPM  & T (small) & T (medium) & T (negligible) & L (large) & L (large) & T (medium) & W (large) & W (large) & W (large) \\
				GECR  & L (medium) & T (negligible) & L (small) & L (medium) & L (small) & T (negligible) & T (negligible) & T (negligible) & T (negligible) \\
				GFE  & L (large) & L (large) & L (large) & L (large) & L (large) & L (large) & L (large) & L (large) & T (large) \\
				GSAD  & W (large) & W (large) & T (negligible) & T (medium) & L (large) & W (large) & W (large) & W (large) & W (large) \\
				HAPT  & L (large) & L (large) & L (large) & L (large) & L (large) & L (large) & L (large) & L (large) & T (large) \\
				ISOLET  & L (large) & L (large) & L (large) & L (large) & L (large) & L (large) & L (large) & L (large) & T (large) \\
				PD  & T (small) & T (small) & T (negligible) & T (negligible) & T (small) & T (negligible) & W (small) & W (small) & W (medium) \\
				\hline
				W - T - L  & 1 - 2 - 5 & 1 - 3 - 4 & 0 - 3 - 5 & 0 - 2 - 6 & 0 - 1 - 7 & 1 - 3 - 4 & 3 - 2 - 3 & 3 - 2 - 3 & 4 - 4 - 0 \\
				\hline
			\end{tabular}
		}
	\end{table*}
	\FloatBarrier
	
	\begin{figure*}[t] 
		\centering
		\subfloat[APSF]{\includegraphics[width=0.24\textwidth]{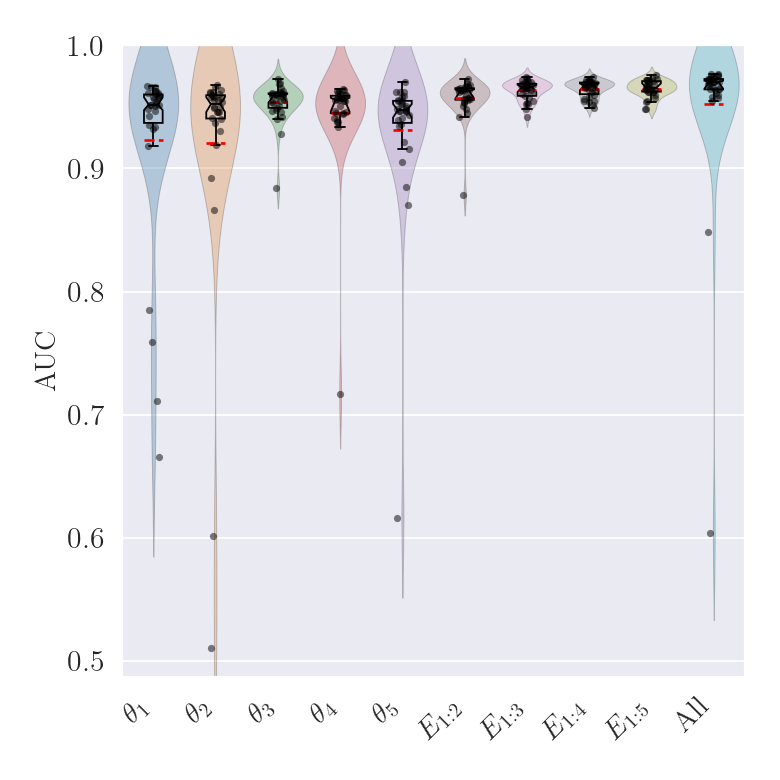}\label{fig:lrapsf_AUC}}%
		\hfill
		\subfloat[ARWPM]{\includegraphics[width=0.24\textwidth]{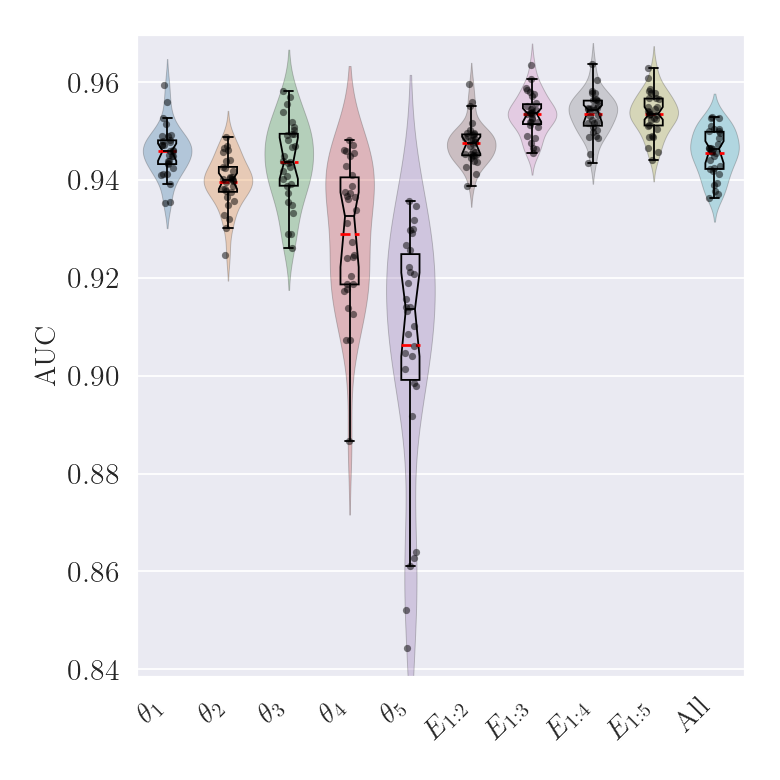}\label{fig:lrarwpm_AUC}}%
		\hfill
		\subfloat[GECR]{\includegraphics[width=0.24\textwidth]{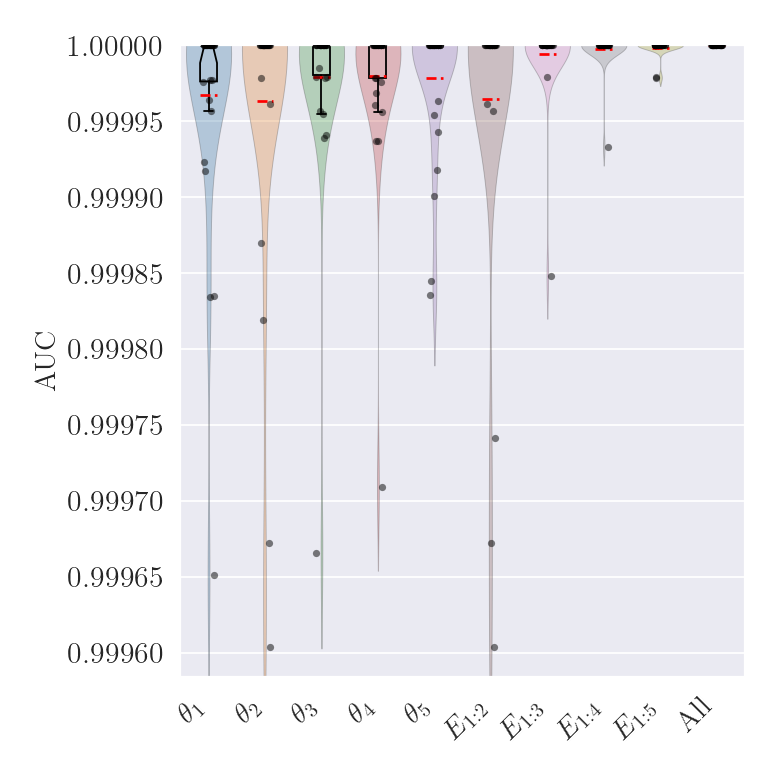}\label{fig:lrgecr_AUC}}%
		\hfill
		\subfloat[GFE]{\includegraphics[width=0.24\textwidth]{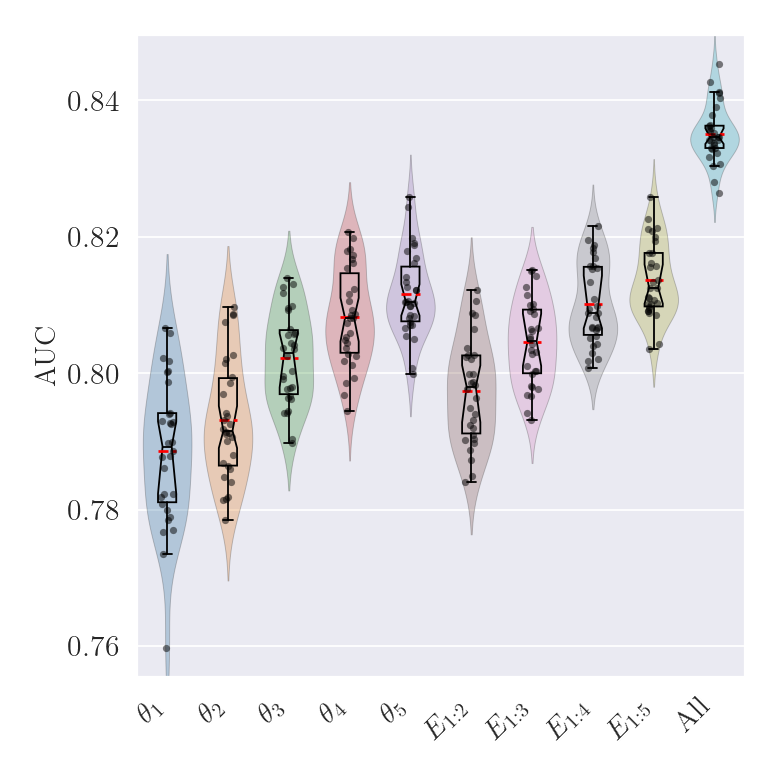}\label{fig:lrgfe_AUC}}
		
		\subfloat[GSAD]{\includegraphics[width=0.24\textwidth]{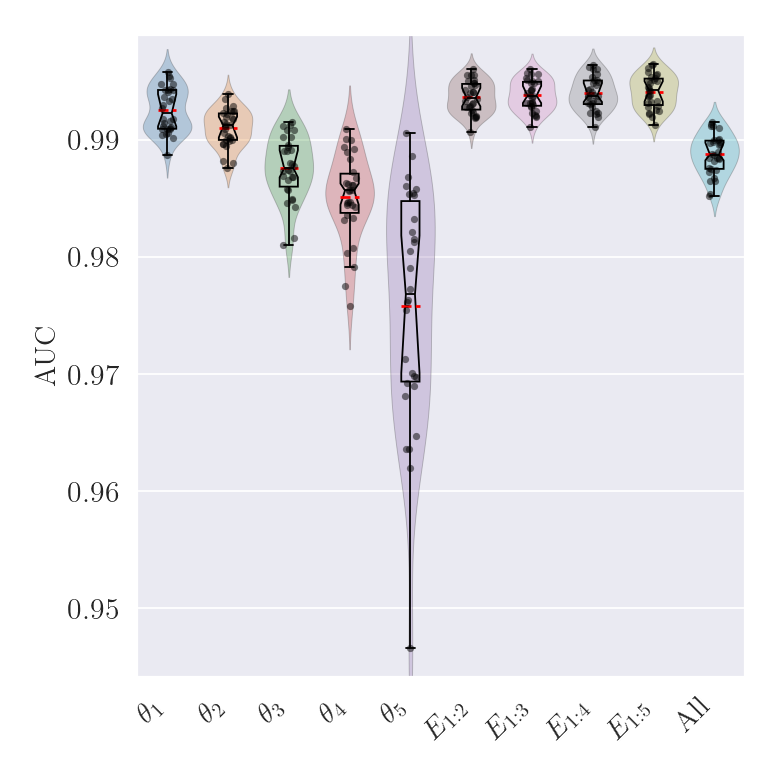}\label{fig:fpgsad_AUC}}%
		\hfill
		\subfloat[HAPT]{\includegraphics[width=0.24\textwidth]{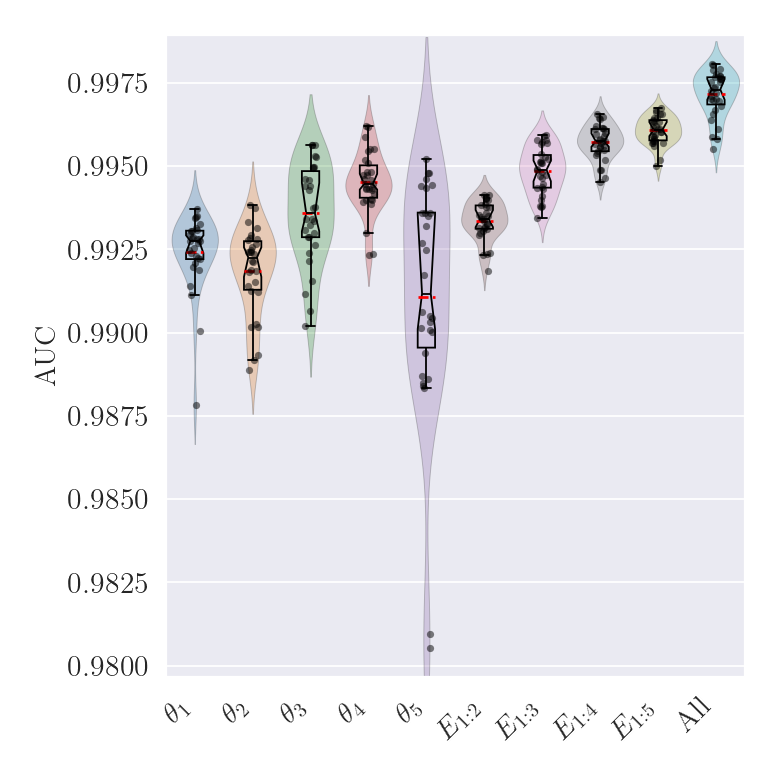}\label{fig:lrhapt_AUC}}%
		\hfill
		\subfloat[ISOLET]{\includegraphics[width=0.24\textwidth]{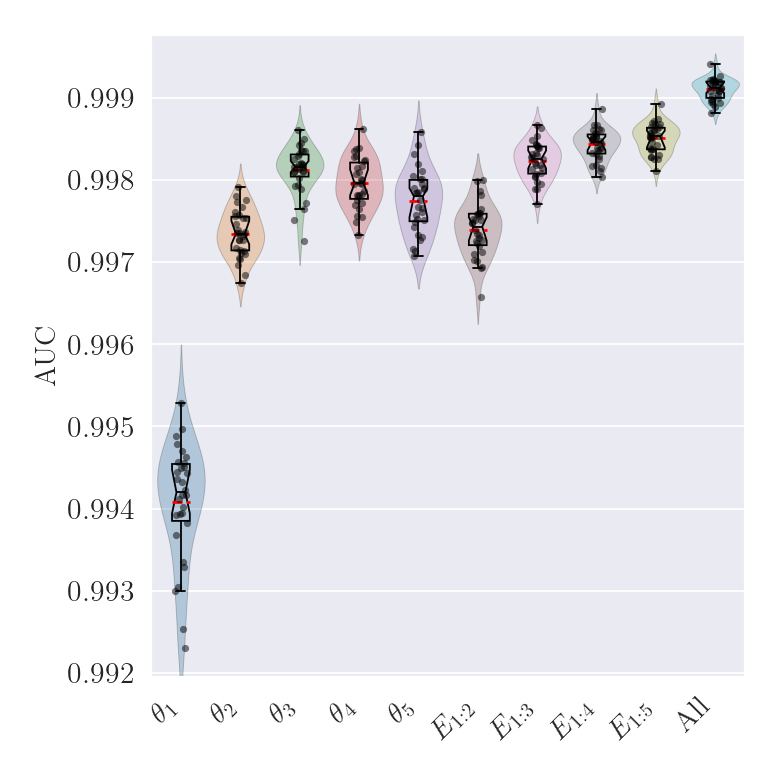}\label{fig:lrisolet_AUC}}%
		\hfill
		\subfloat[PD]{\includegraphics[width=0.24\textwidth]{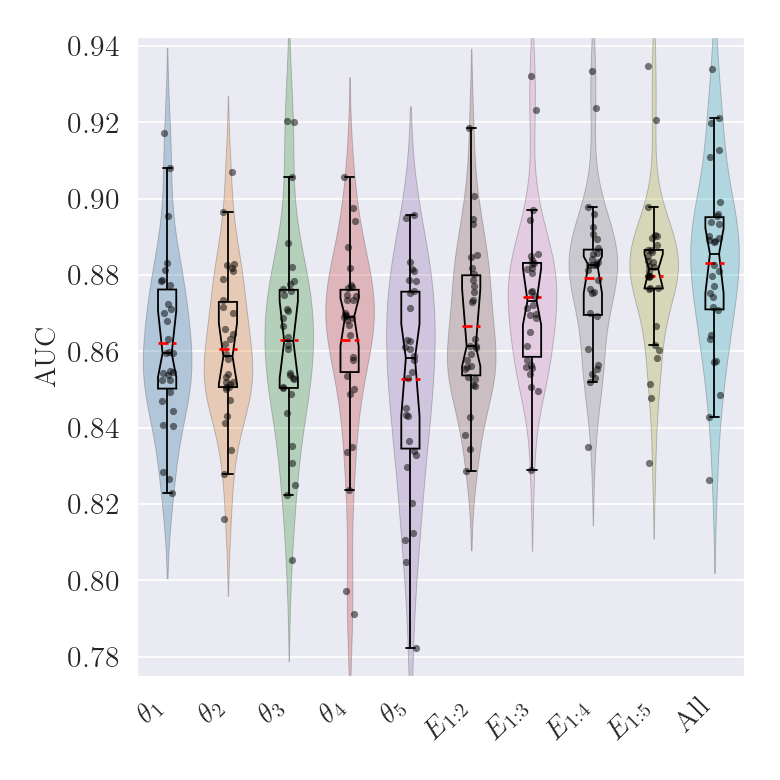}\label{fig:lrpd_AUC}}
		\caption[The distribution of the obtained AUC values for 30 Logistic Regression runs.]{The raincloud plot of AUC results obtained from 30 Logistic Regression runs.}
		
		\label{fig:lr_AUC}
	\end{figure*}
	
	\begin{figure*}[t] 
		\centering
		\subfloat[APSF]{\includegraphics[width=0.24\textwidth]{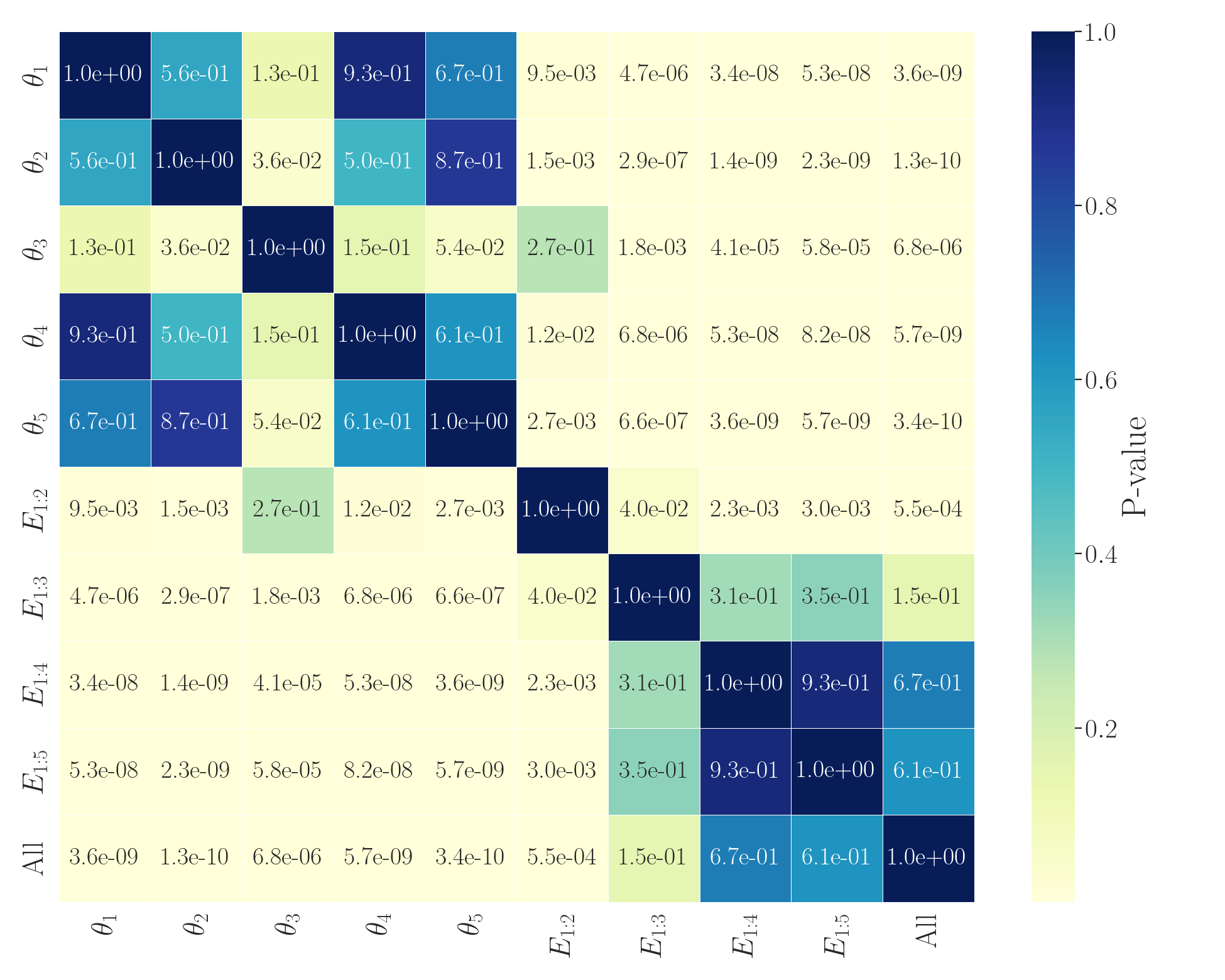}\label{fig:lrnemapsf_AUC}}%
		\hfill
		\subfloat[ARWPM]{\includegraphics[width=0.24\textwidth]{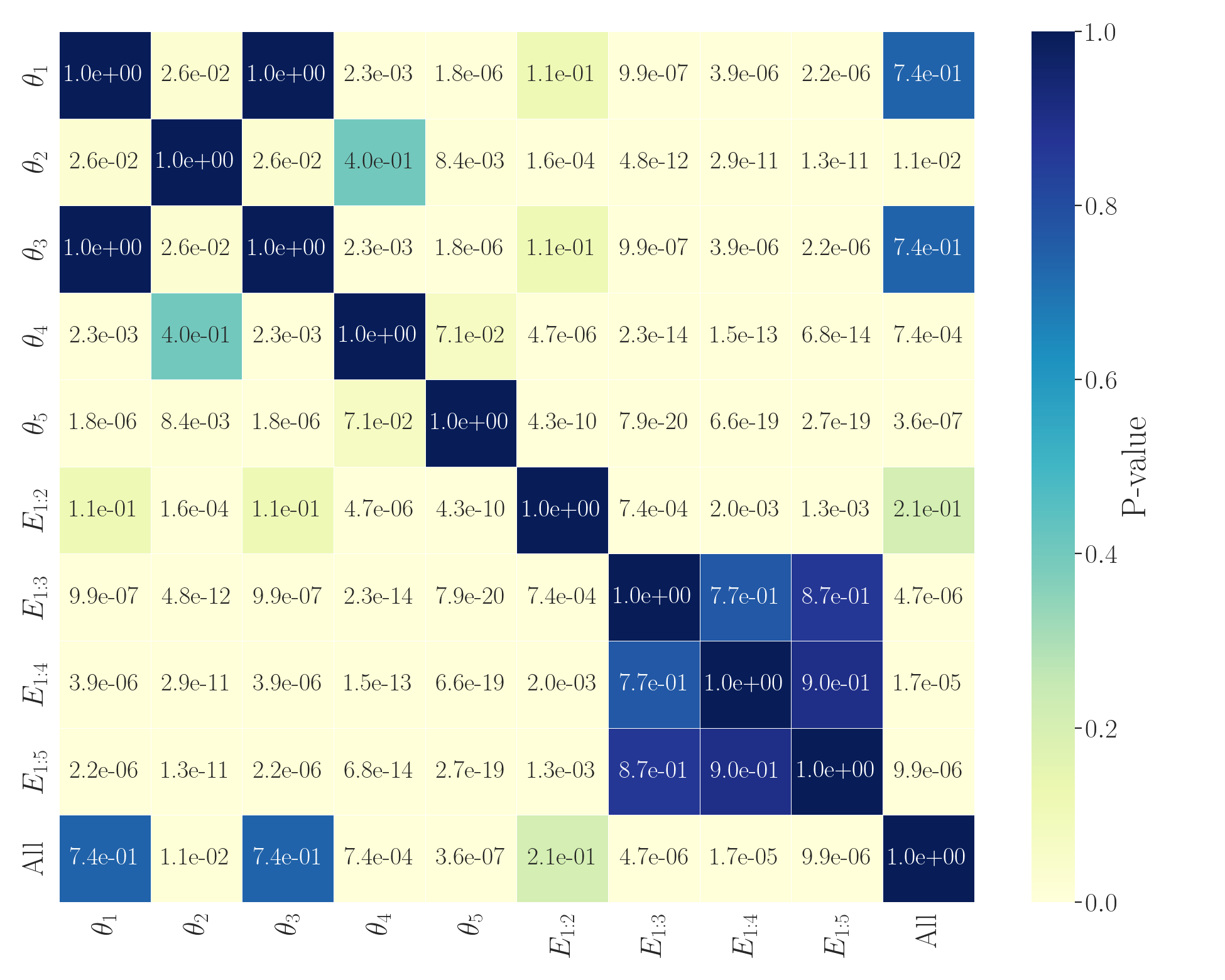}\label{fig:lrnemarwpm_AUC}}%
		\hfill
		\subfloat[GECR]{\includegraphics[width=0.24\textwidth]{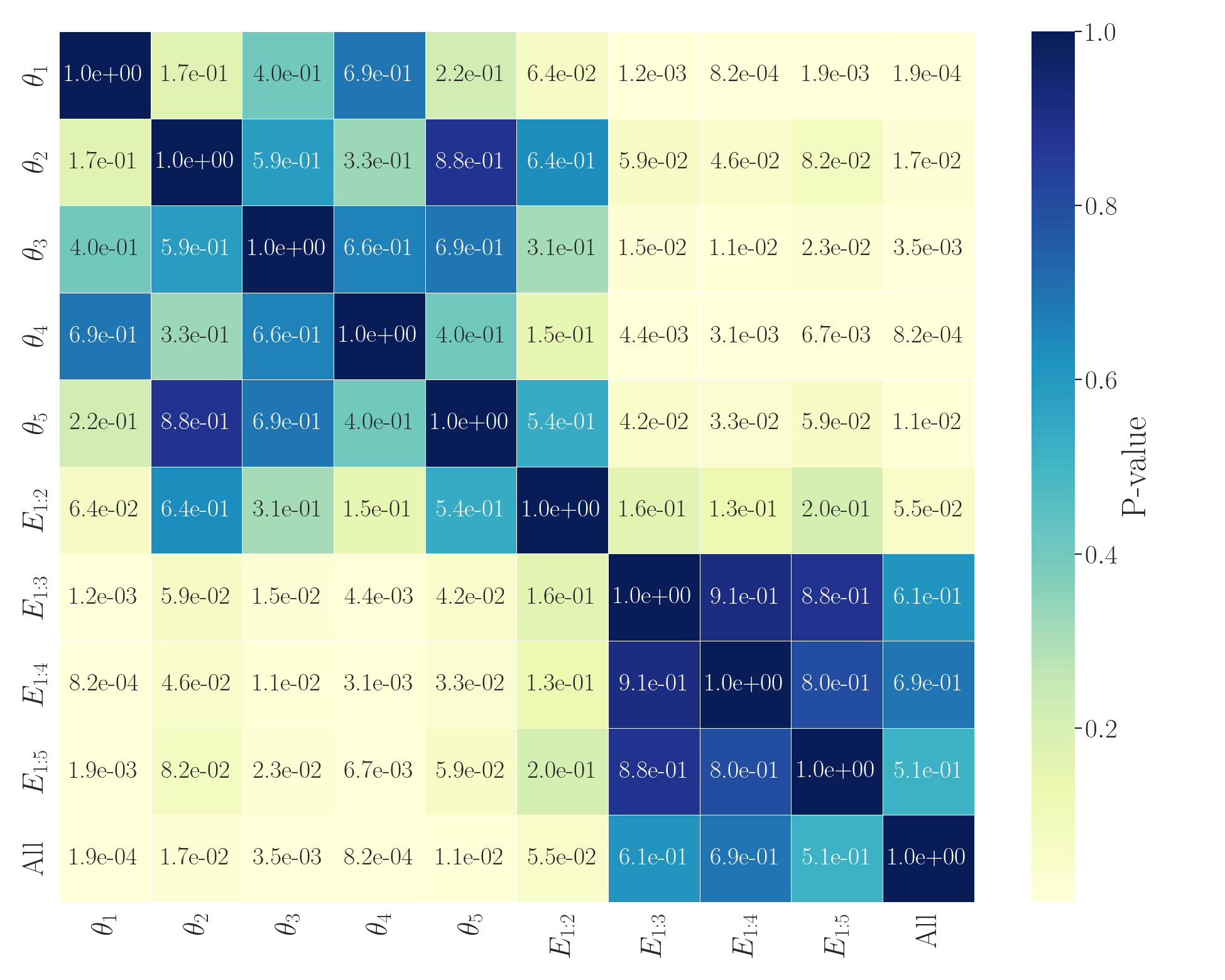}\label{fig:lrnemgecr_AUC}}%
		\hfill
		\subfloat[GFE]{\includegraphics[width=0.24\textwidth]{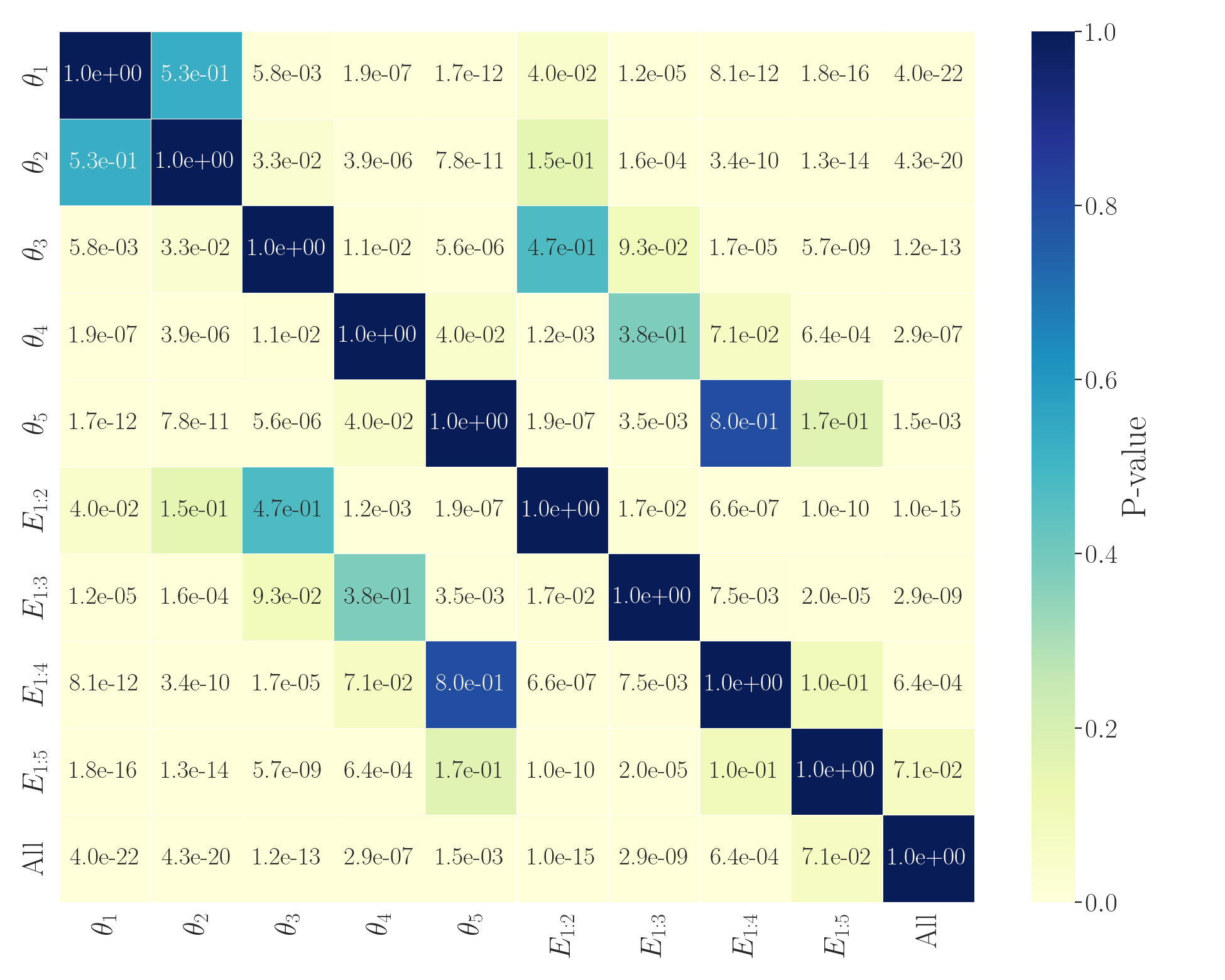}\label{fig:lrnemgfe_AUC}}
		
		\subfloat[GSAD]{\includegraphics[width=0.24\textwidth]{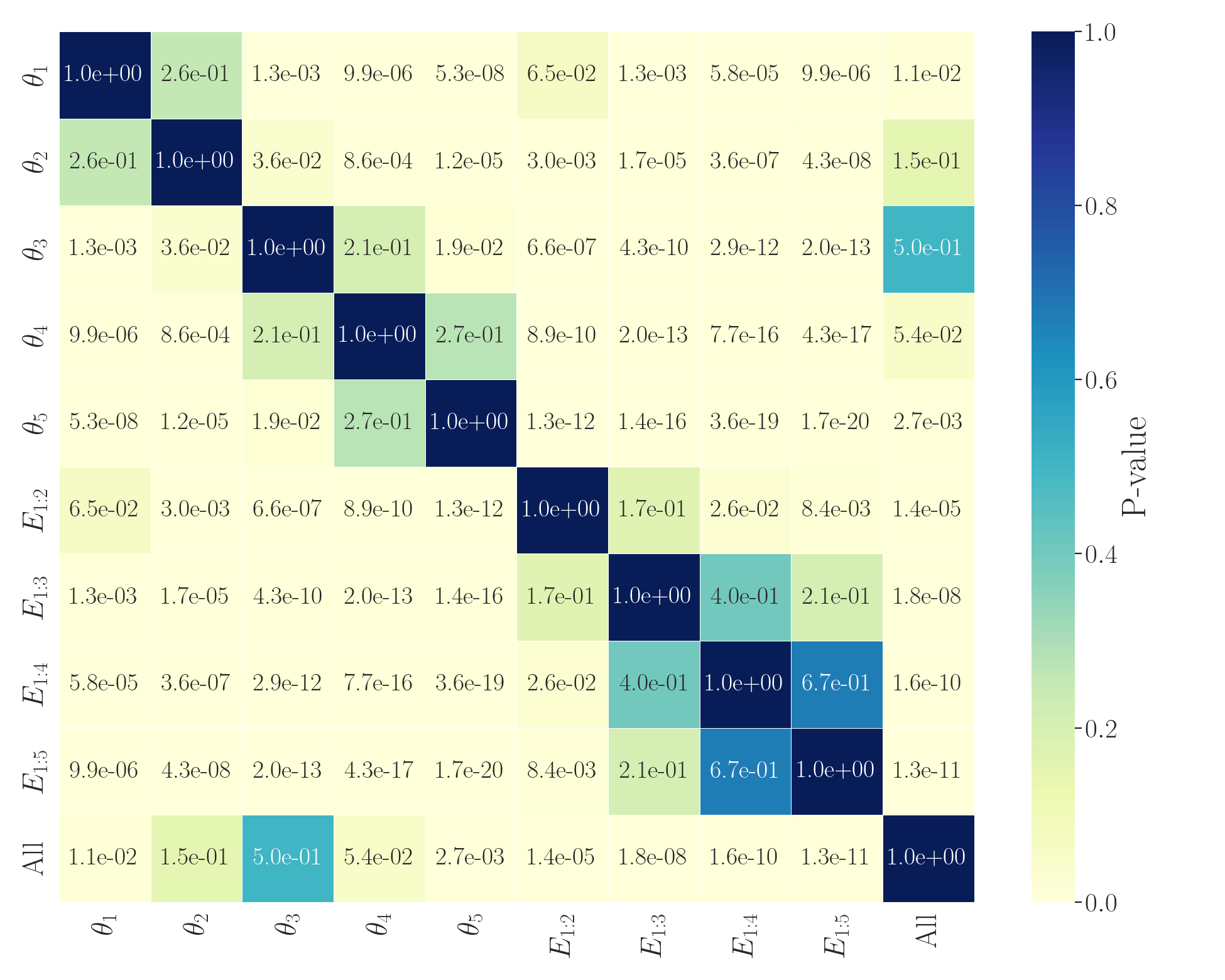}\label{fig:lrnemgsad_AUC}}%
		\hfill
		\subfloat[HAPT]{\includegraphics[width=0.24\textwidth]{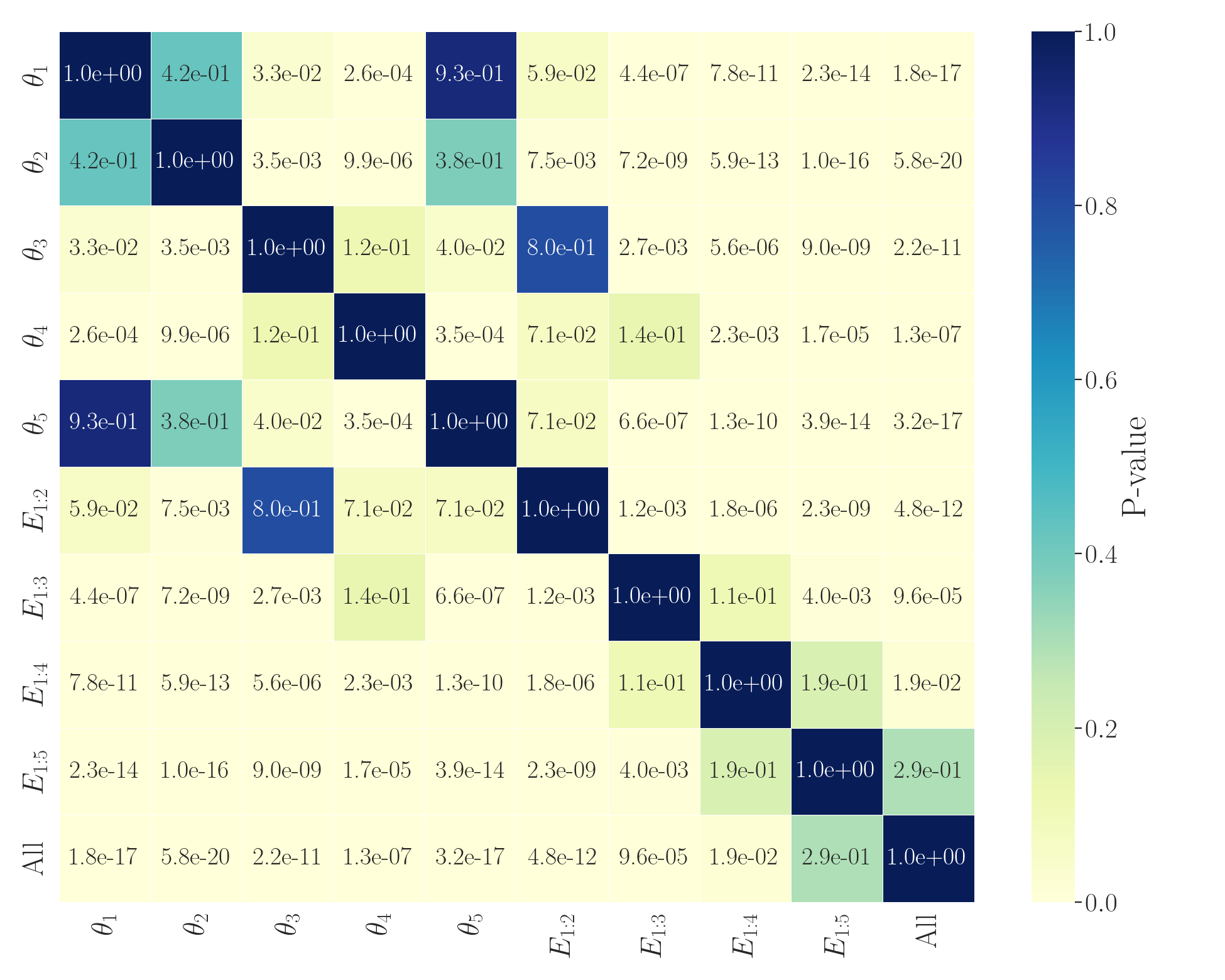}\label{fig:lrnemhapt_AUC}}%
		\hfill
		\subfloat[ISOLET]{\includegraphics[width=0.24\textwidth]{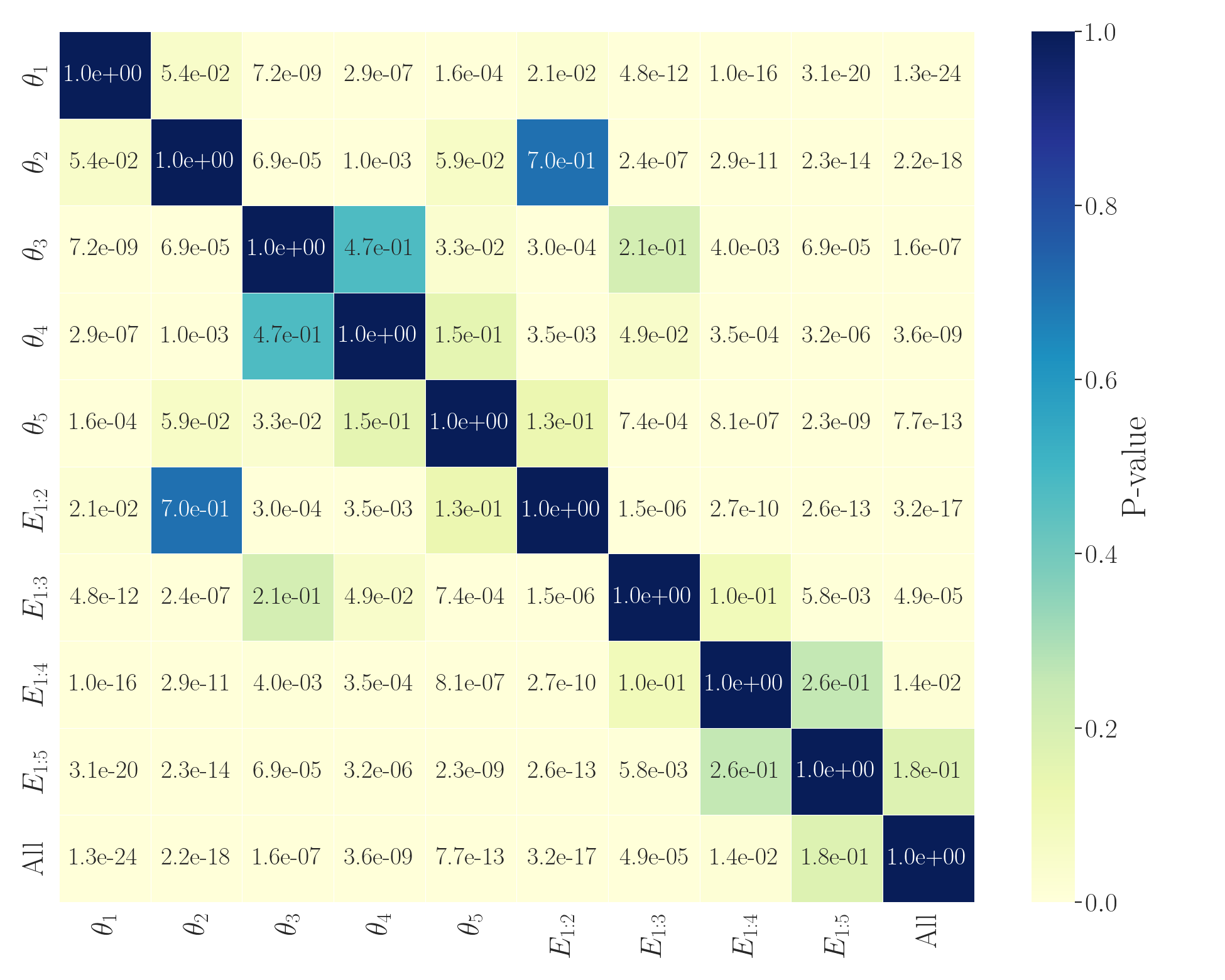}\label{fig:lrnemisolet_AUC}}%
		\hfill
		\subfloat[PD]{\includegraphics[width=0.24\textwidth]{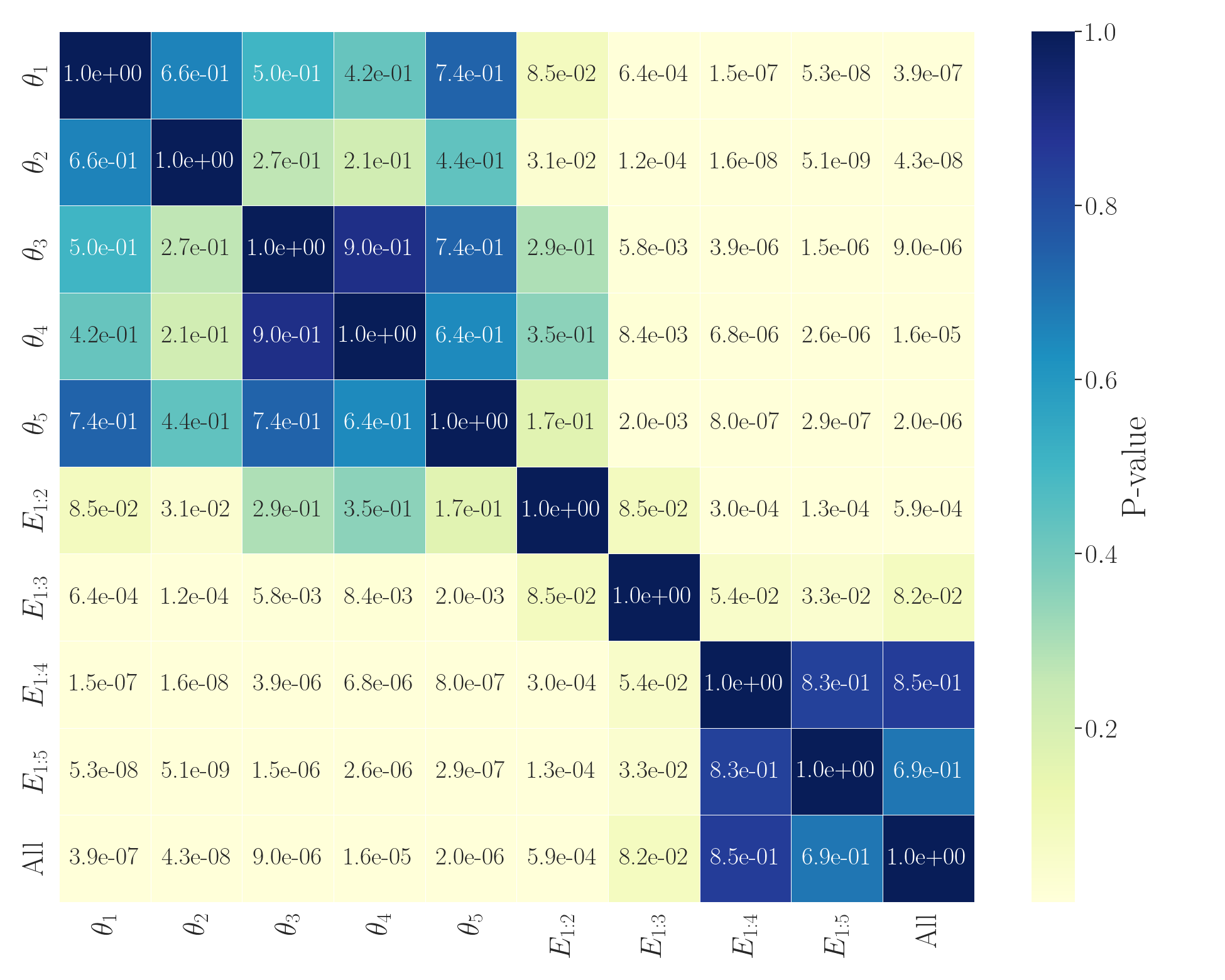}\label{fig:lrnempd_AUC}}
		\caption[The adjusted Conover's P-values for the obtained AUC values in 30 Logistic Regression runs.]{The results of the Conover post-hoc test on testing data’s AUC obtained from 30 Logistic Regression runs.}
		
		\label{fig:lrnem_AUC}
	\end{figure*}
	\FloatBarrier
	
	\begin{figure*}[htbp] 
		\centering
		\subfloat[APSF]{\includegraphics[width=0.24\textwidth]{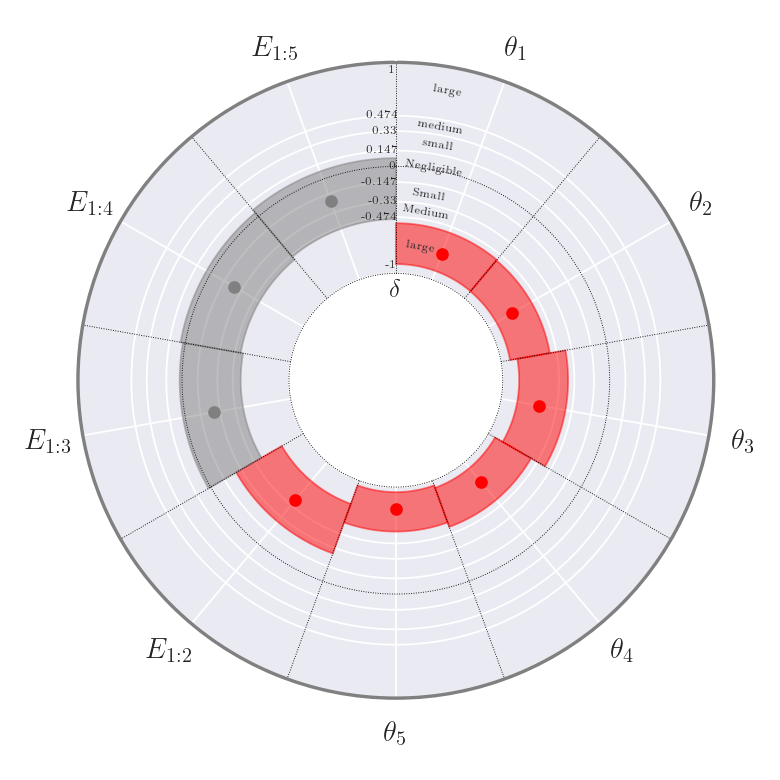}\label{fig:lrcliffapsf_AUC}}%
		\hfill
		\subfloat[ARWPM]{\includegraphics[width=0.24\textwidth]{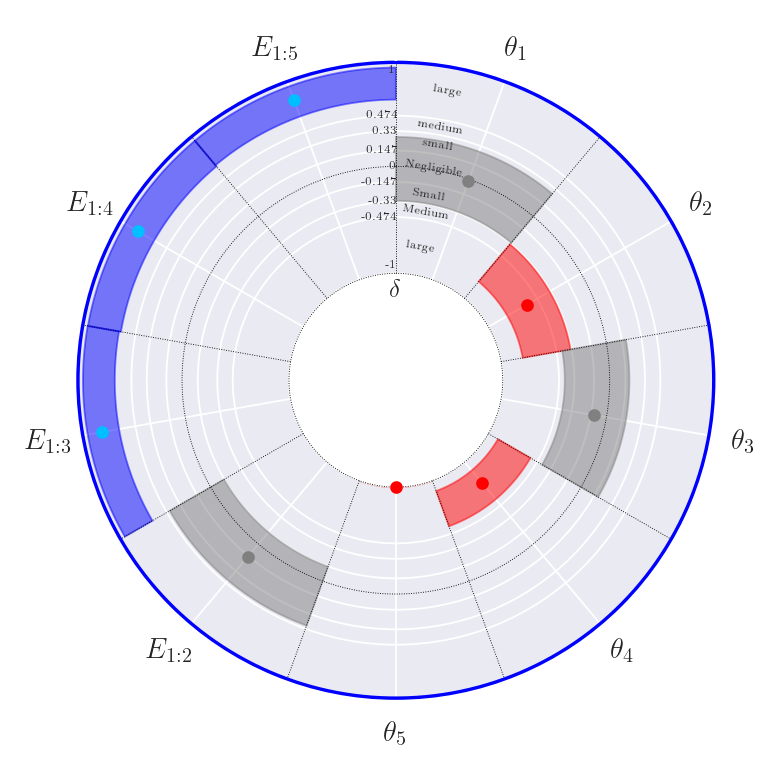}\label{fig:lrcliffarwpm_AUC}}%
		\hfill
		\subfloat[GECR]{\includegraphics[width=0.24\textwidth]{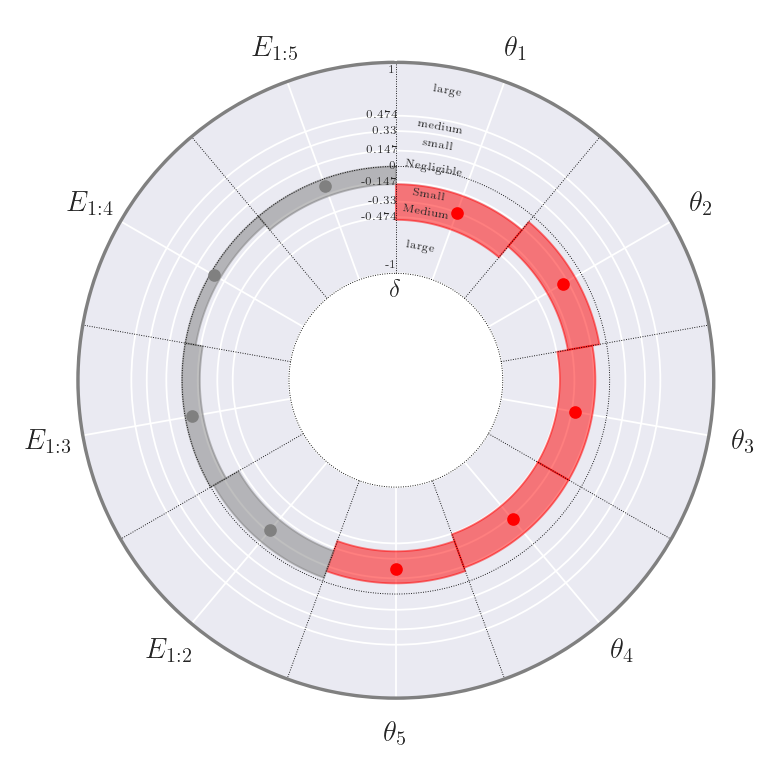}\label{fig:lrcliffgecr_AUC}}%
		\hfill
		\subfloat[GFE]{\includegraphics[width=0.24\textwidth]{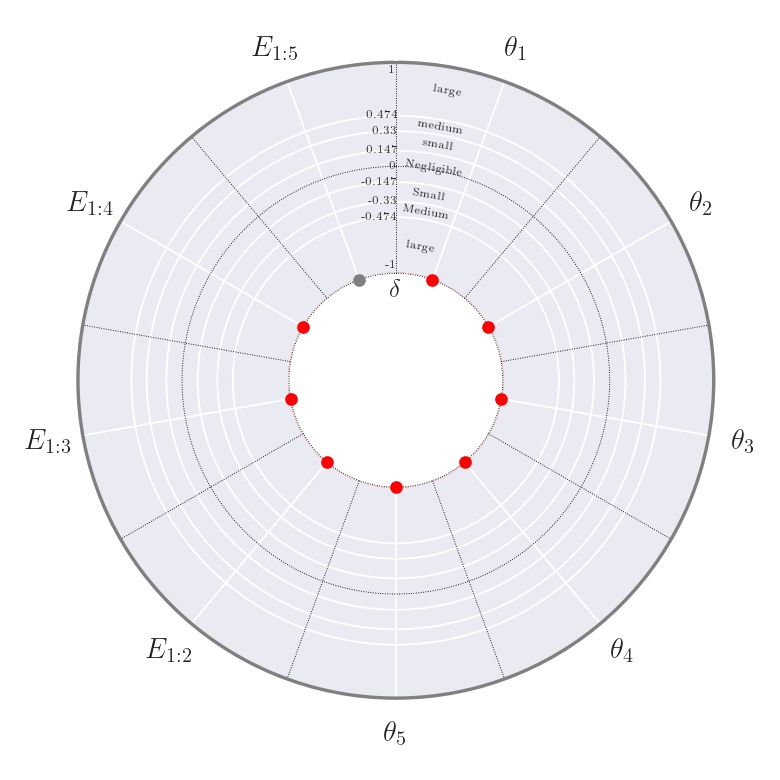}\label{fig:lrcliffgfe_AUC}}
		
		\subfloat[GSAD]{\includegraphics[width=0.24\textwidth]{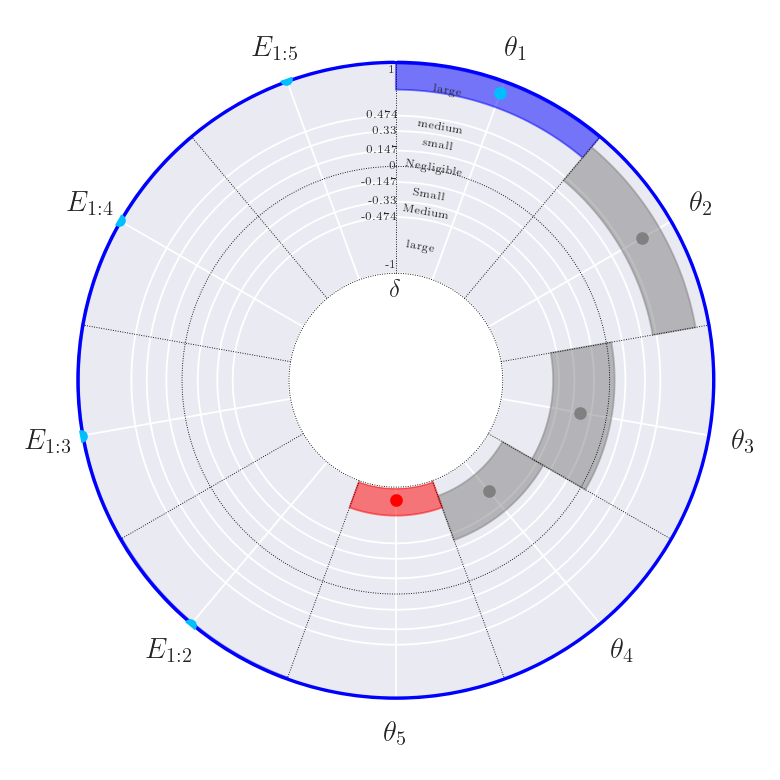}\label{fig:lrcliffgsad_AUC}}%
		\hfill
		\subfloat[HAPT]{\includegraphics[width=0.24\textwidth]{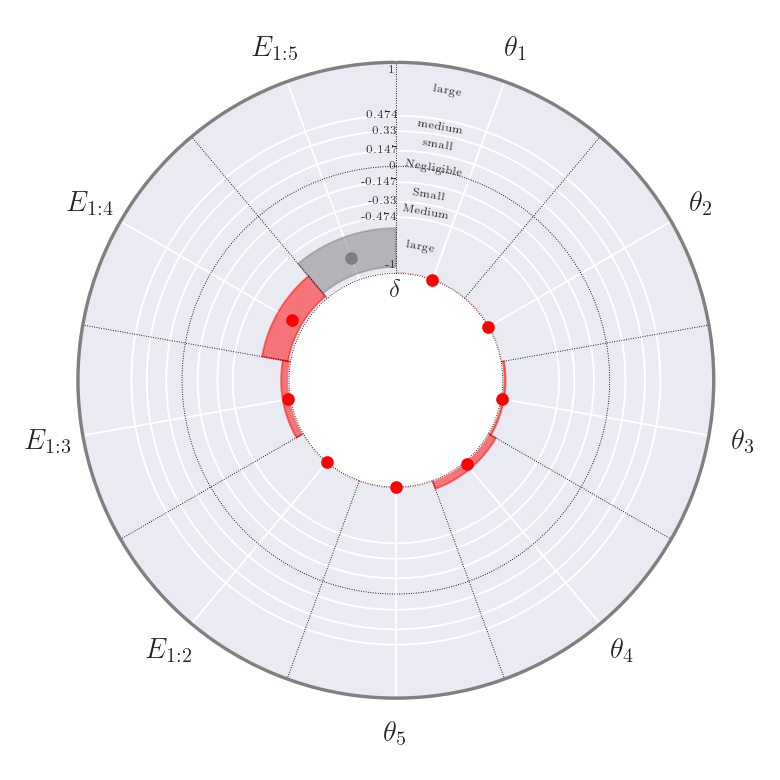}\label{fig:lrcliffhapt_AUC}}%
		\hfill
		\subfloat[ISOLET]{\includegraphics[width=0.24\textwidth]{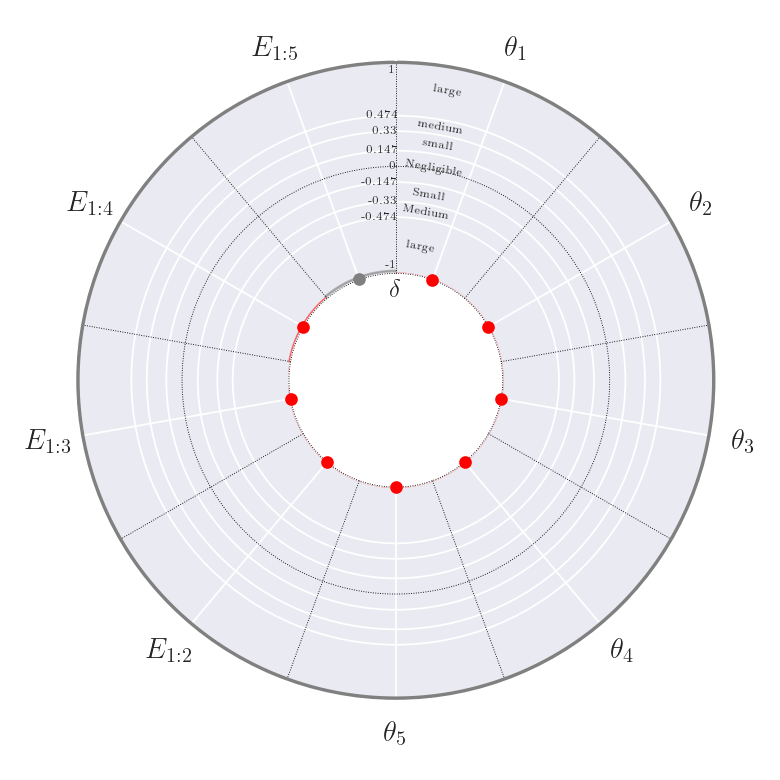}\label{fig:lrcliffisolet_AUC}}%
		\hfill
		\subfloat[PD]{\includegraphics[width=0.24\textwidth]{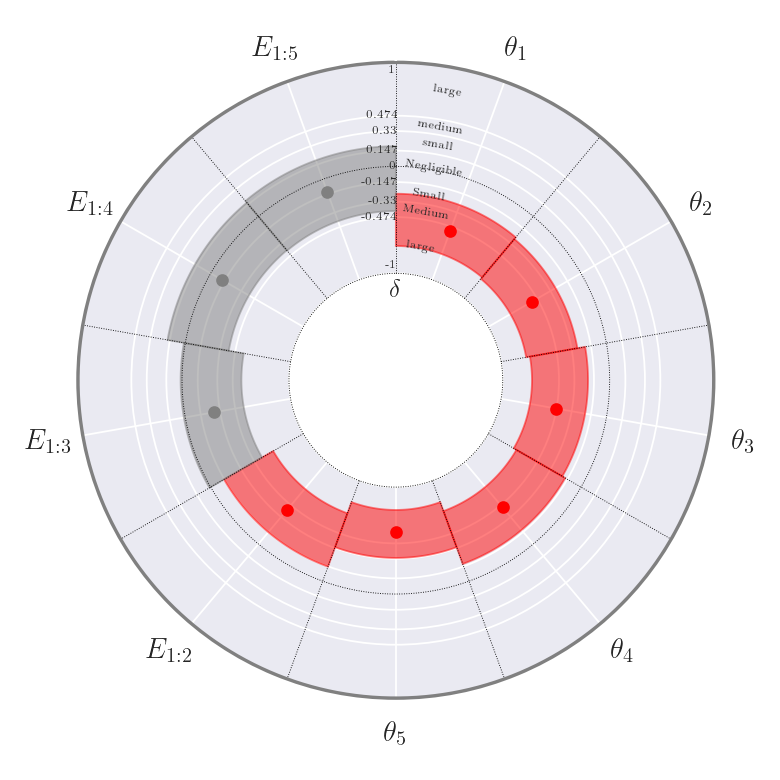}\label{fig:lrcliffpd_AUC}}
		\caption[The Cliff's $\delta$ effect size measure and its 95\% confidence intervals for the AUC values obtained from 30 Logistic Regression runs.]{Effect size analysis of test data AUC across 30 Logistic Regression runs using Cliff's $\delta$. Each point represents the actual value obtained, with segments denoting 95\% confidence intervals based on 10,000 bootstrap resamplings. The outer ring color visualizes the statistical significance: grey illustrates no significant difference (adjusted Friedman's P-value$>0.05$), while color indicates significant differences; blue indicates at least one view and/or ensemble outperforms the benchmark (adjusted Conover's p-value$ < 0.05$, Cliff's $\delta > 0$), and red signifies all views and ensembles underperform relative to the benchmark (adjusted Conover's p-value$ < 0.05$, Cliff's $\delta < 0$). Segment colors show performance difference against the benchmark: grey for no significant difference (adjusted Conover's p-value$  > 0.05$), blue for better performance (Cliff's $\delta > 0$), and red for worse performance (Cliff's $\delta < 0$).}
		
		\label{fig:lrcliff_AUC}
	\end{figure*}
	
	\begin{table*}[htbp]
		\centering
		\caption[The results of Friedman and Conover tests and Cliff's $\delta$ analysis for the AUC values obtained from 30 Logistic Regression runs.]{Statistical comparison of AUC results for testing data obtained from Logistic Regression runs. W, T, and L denote win, tie, and loss based on adjusted Friedman and Conover's p-values. Effect sizes are calculated using Cliff's Delta method and are categorized as negligible, small, medium, or large.}
		\label{tab:lrauc}
		\resizebox{\linewidth}{!}{%
			\begin{tabular}{c|ccccccccc}
				\hline
				\multicolumn{10}{c}{Logistic Regression's AUC}\\
				\hline
				Dataset & $\theta_1$ & $\theta_2$ & $\theta_3$ & $\theta_4$ & $\theta_5$ & $E_{1:2}$ & $E_{1:3}$ & $E_{1:4}$ & $E_{1:5}$ \\
				\hline
				APSF  & L (large) & L (large) & L (large) & L (large) & L (large) & L (large) & T (small) & T (small) & T (small) \\
				ARWPM  & T (negligible) & L (large) & T (negligible) & L (large) & L (large) & T (small) & W (large) & W (large) & W (large) \\
				GECR  & L (medium) & L (small) & L (small) & L (small) & L (small) & T (small) & T (negligible) & T (negligible) & T (negligible) \\
				GFE  & L (large) & L (large) & L (large) & L (large) & L (large) & L (large) & L (large) & L (large) & T (large) \\
				GSAD  & W (large) & T (large) & T (small) & T (large) & L (large) & W (large) & W (large) & W (large) & W (large) \\
				HAPT  & L (large) & L (large) & L (large) & L (large) & L (large) & L (large) & L (large) & L (large) & T (large) \\
				ISOLET  & L (large) & L (large) & L (large) & L (large) & L (large) & L (large) & L (large) & L (large) & T (large) \\
				PD  & L (large) & L (large) & L (large) & L (medium) & L (large) & L (medium) & T (small) & T (negligible) & T (negligible) \\
				\hline
				W - T - L  & 1 - 1 - 6 & 0 - 1 - 7 & 0 - 2 - 6 & 0 - 1 - 7 & 0 - 0 - 8 & 1 - 2 - 5 & 2 - 3 - 3 & 2 - 3 - 3 & 2 - 6 - 0 \\
				\hline
			\end{tabular}
		}
	\end{table*}
	\FloatBarrier

	\begin{figure*}[t] 
		\centering
		\subfloat[APSF]{\includegraphics[width=0.24\textwidth]{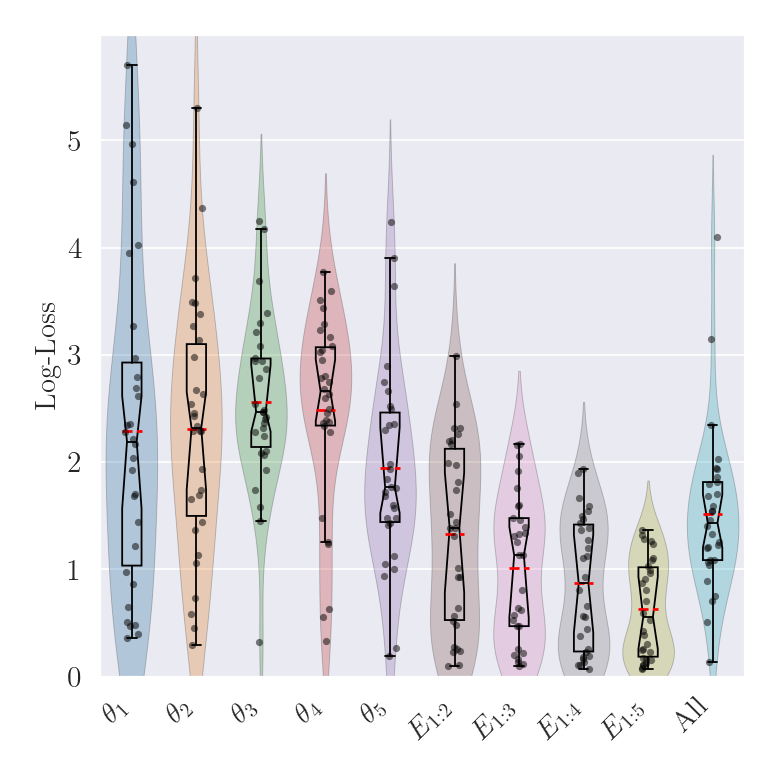}\label{fig:lrapsf_Loss}}%
		\hfill
		\subfloat[ARWPM]{\includegraphics[width=0.24\textwidth]{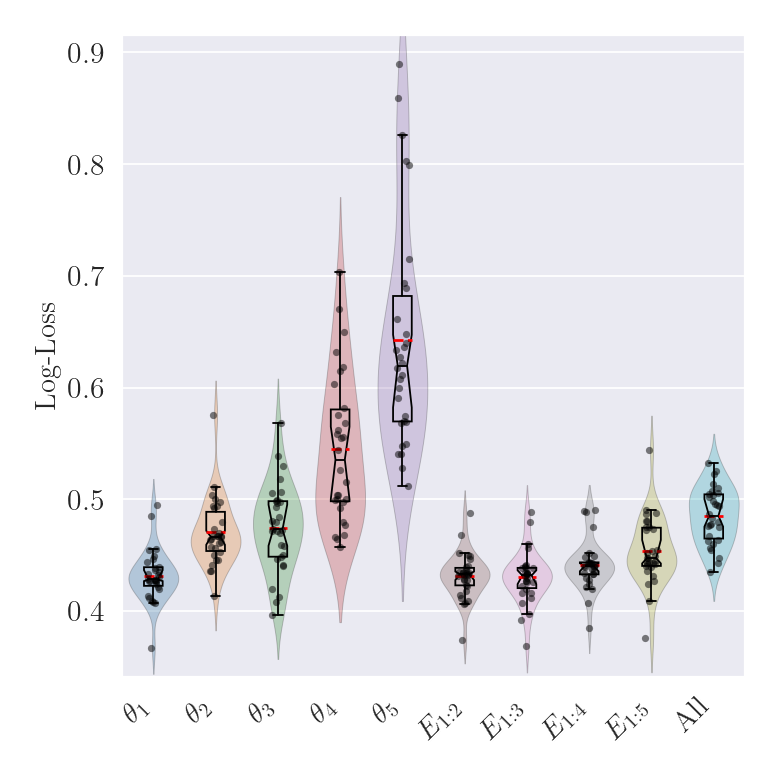}\label{fig:lrarwpm_Loss}}%
		\hfill
		\subfloat[GECR]{\includegraphics[width=0.24\textwidth]{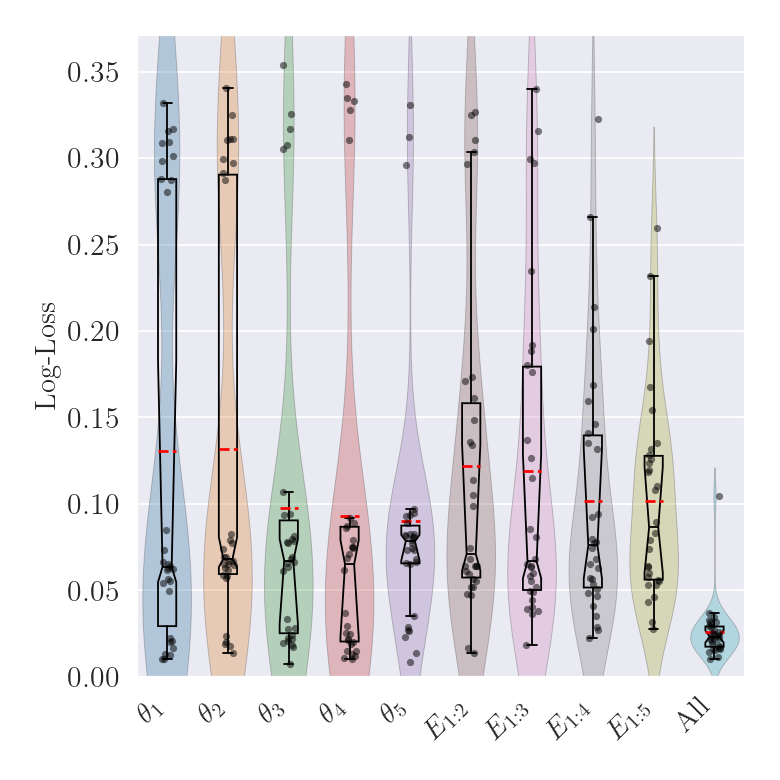}\label{fig:lrgecr_Loss}}%
		\hfill
		\subfloat[GFE]{\includegraphics[width=0.24\textwidth]{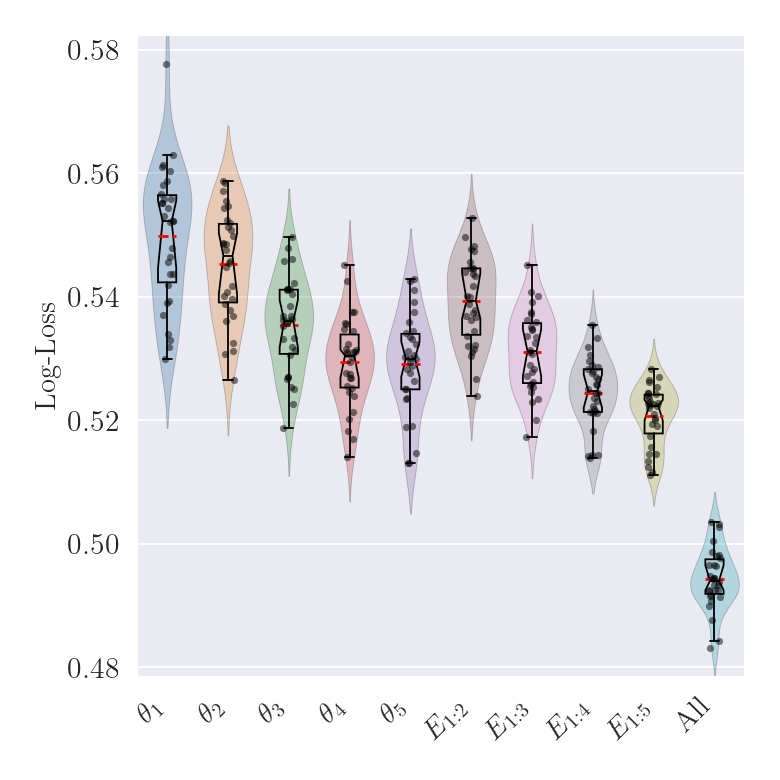}\label{fig:lrgfe_Loss}}
		
		\subfloat[GSAD]{\includegraphics[width=0.24\textwidth]{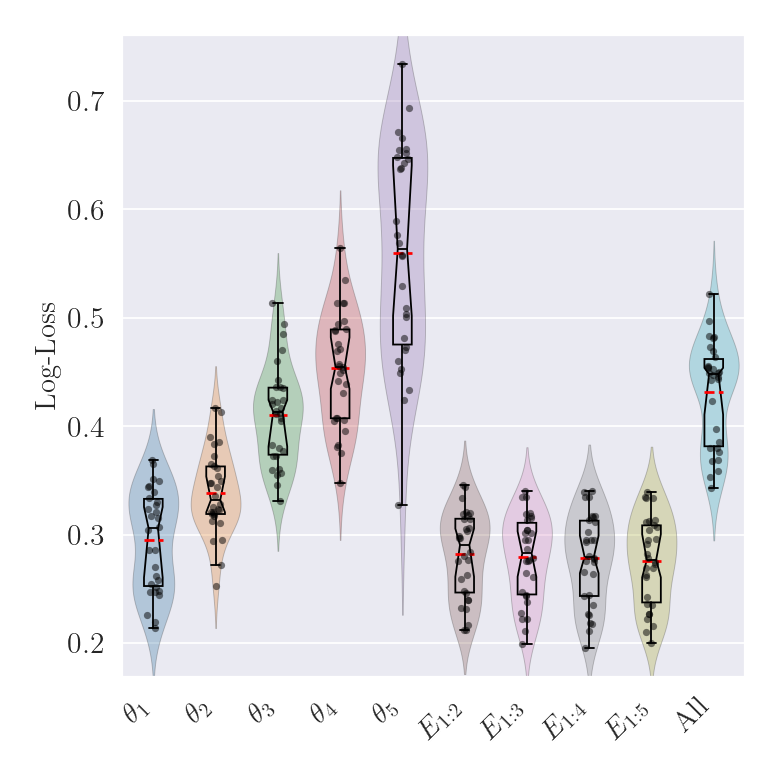}\label{fig:fpgsad_Loss}}%
		\hfill
		\subfloat[HAPT]{\includegraphics[width=0.24\textwidth]{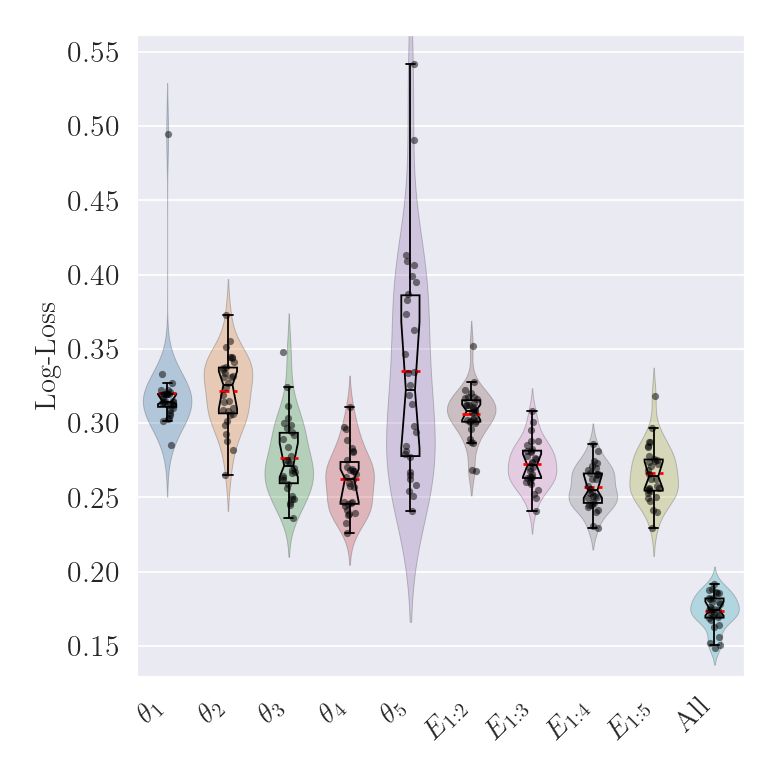}\label{fig:lrhapt_Loss}}%
		\hfill
		\subfloat[ISOLET]{\includegraphics[width=0.24\textwidth]{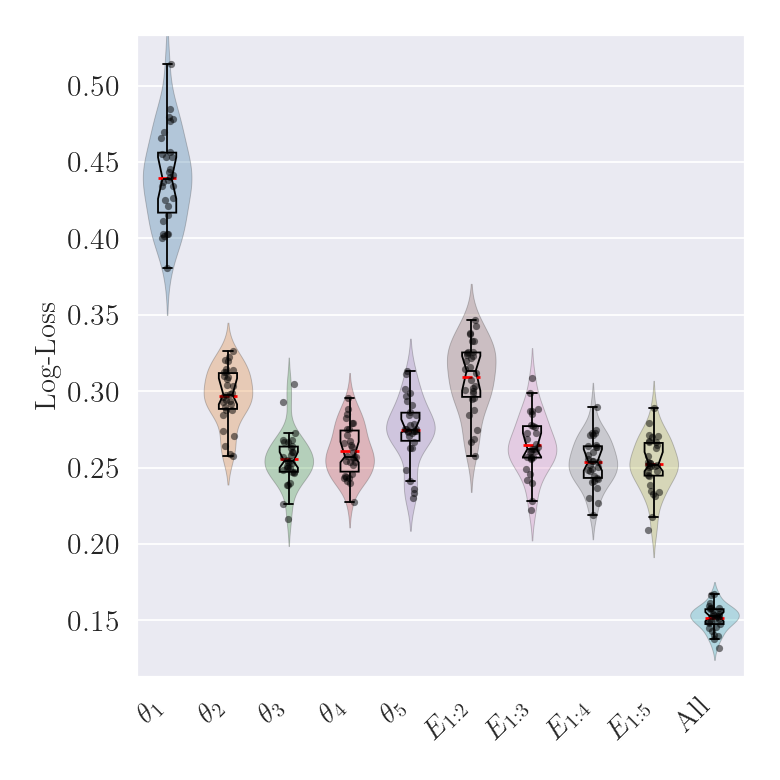}\label{fig:lrisolet_Loss}}%
		\hfill
		\subfloat[PD]{\includegraphics[width=0.24\textwidth]{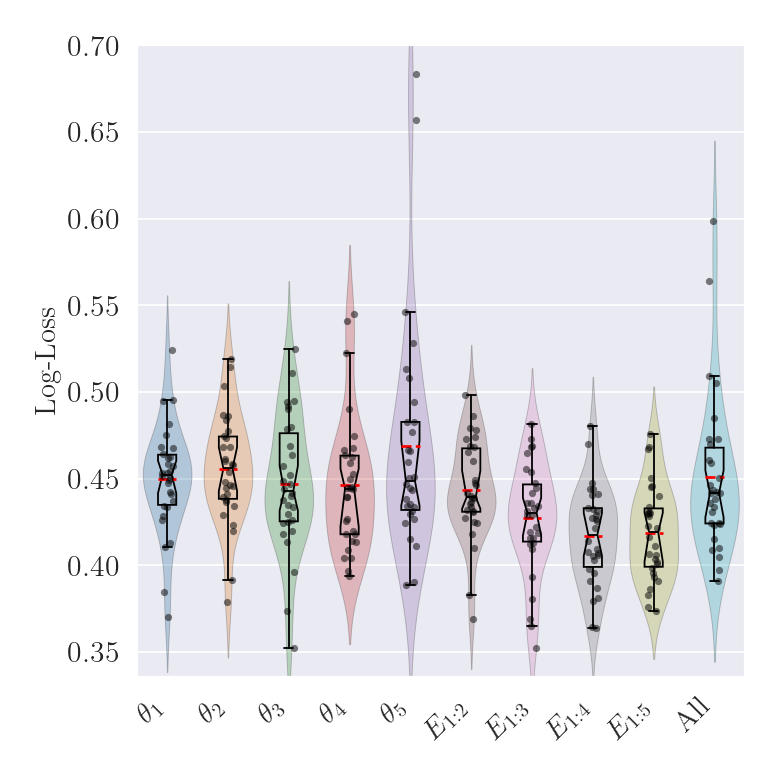}\label{fig:lrpd_Loss}}
		\caption[The distribution of the obtained Log-Loss values for 30 Logistic Regression runs.]{The raining cloud plot of Log-Loss results obtained from 30 Logistic Regression runs.}
		
		\label{fig:lr_Loss}
	\end{figure*}
	
	\begin{figure*}[t] 
		\centering
		\subfloat[APSF]{\includegraphics[width=0.24\textwidth]{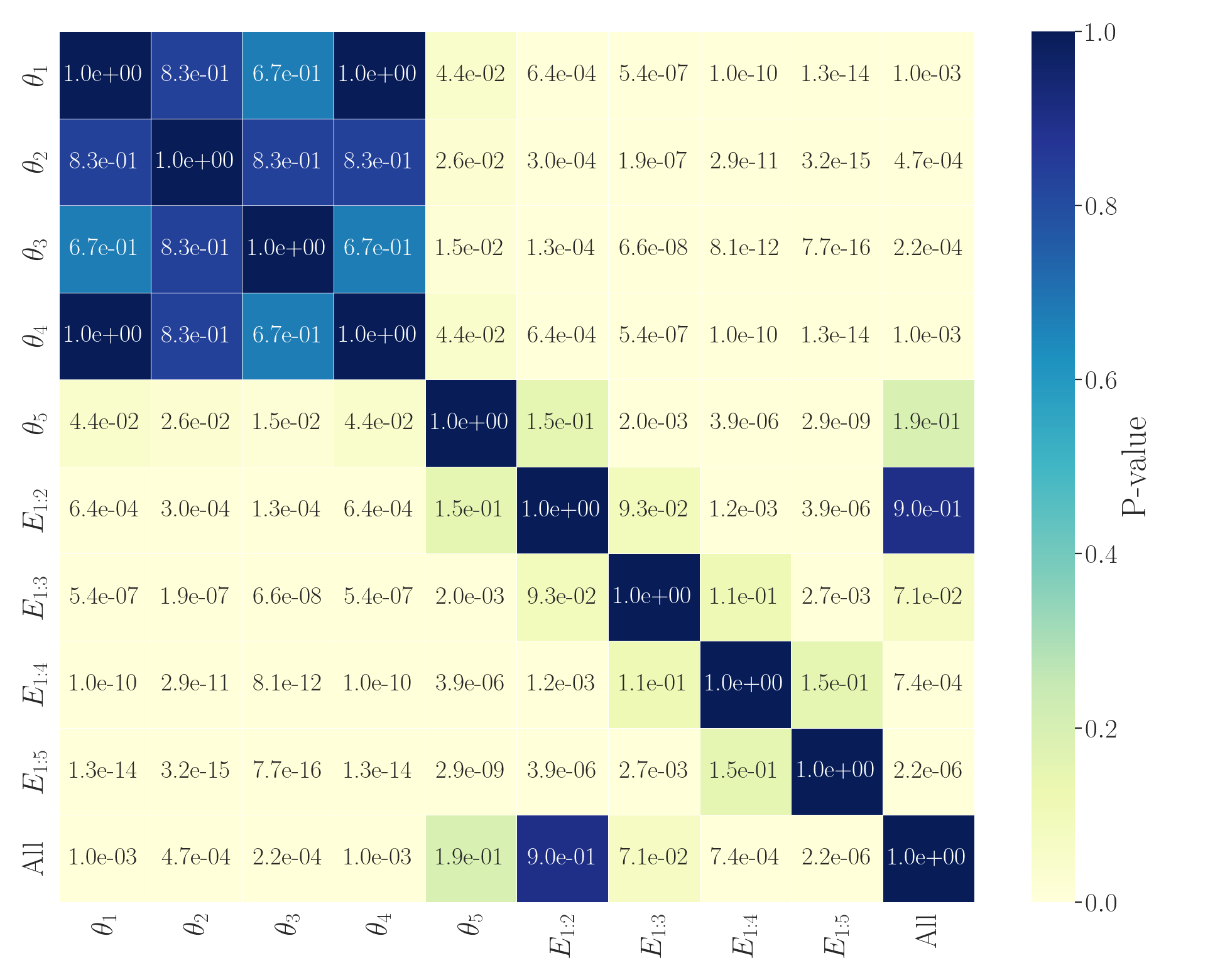}\label{fig:lrnemapsf_Loss}}%
		\hfill
		\subfloat[ARWPM]{\includegraphics[width=0.24\textwidth]{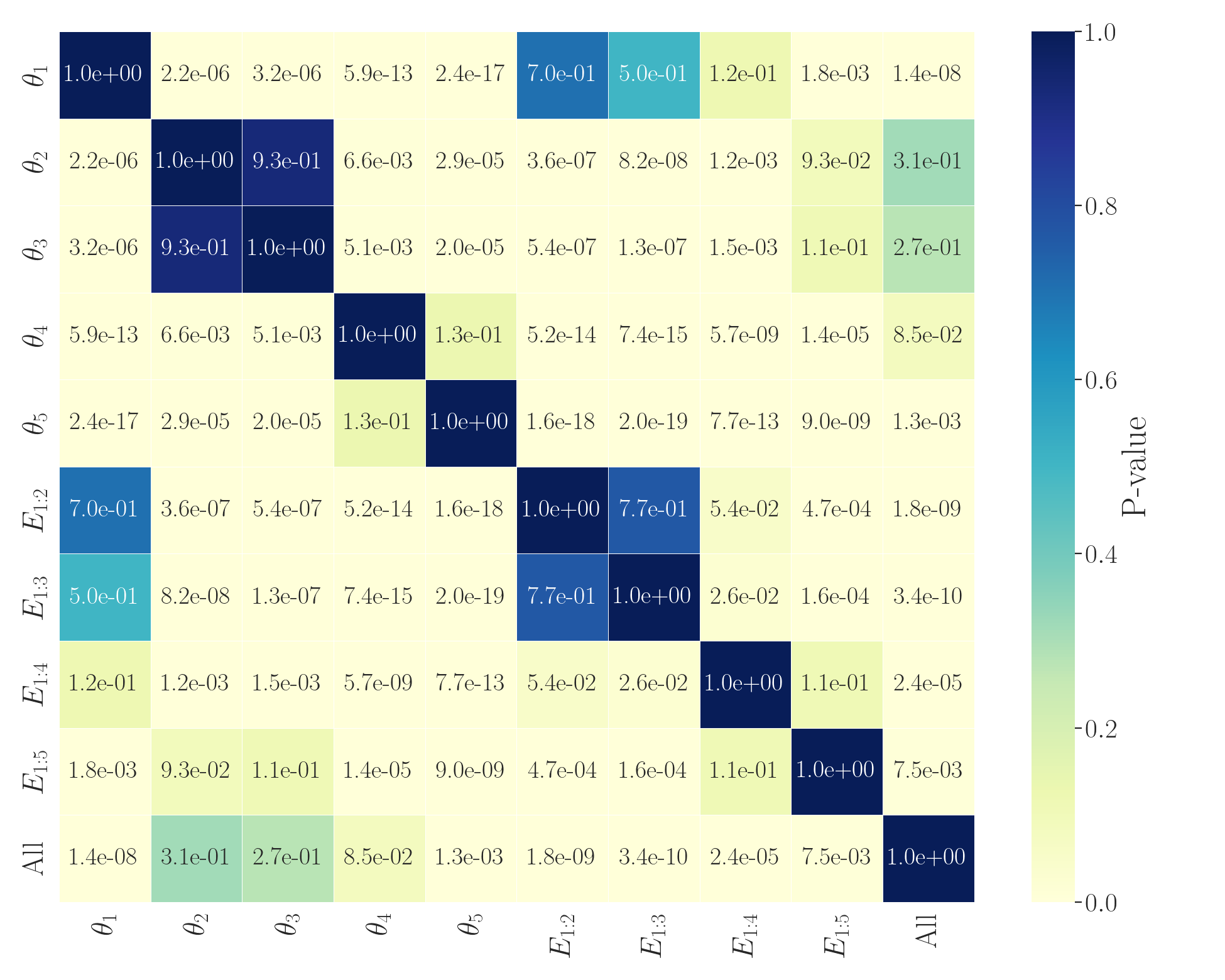}\label{fig:lrnemarwpm_Loss}}%
		\hfill
		\subfloat[GECR]{\includegraphics[width=0.24\textwidth]{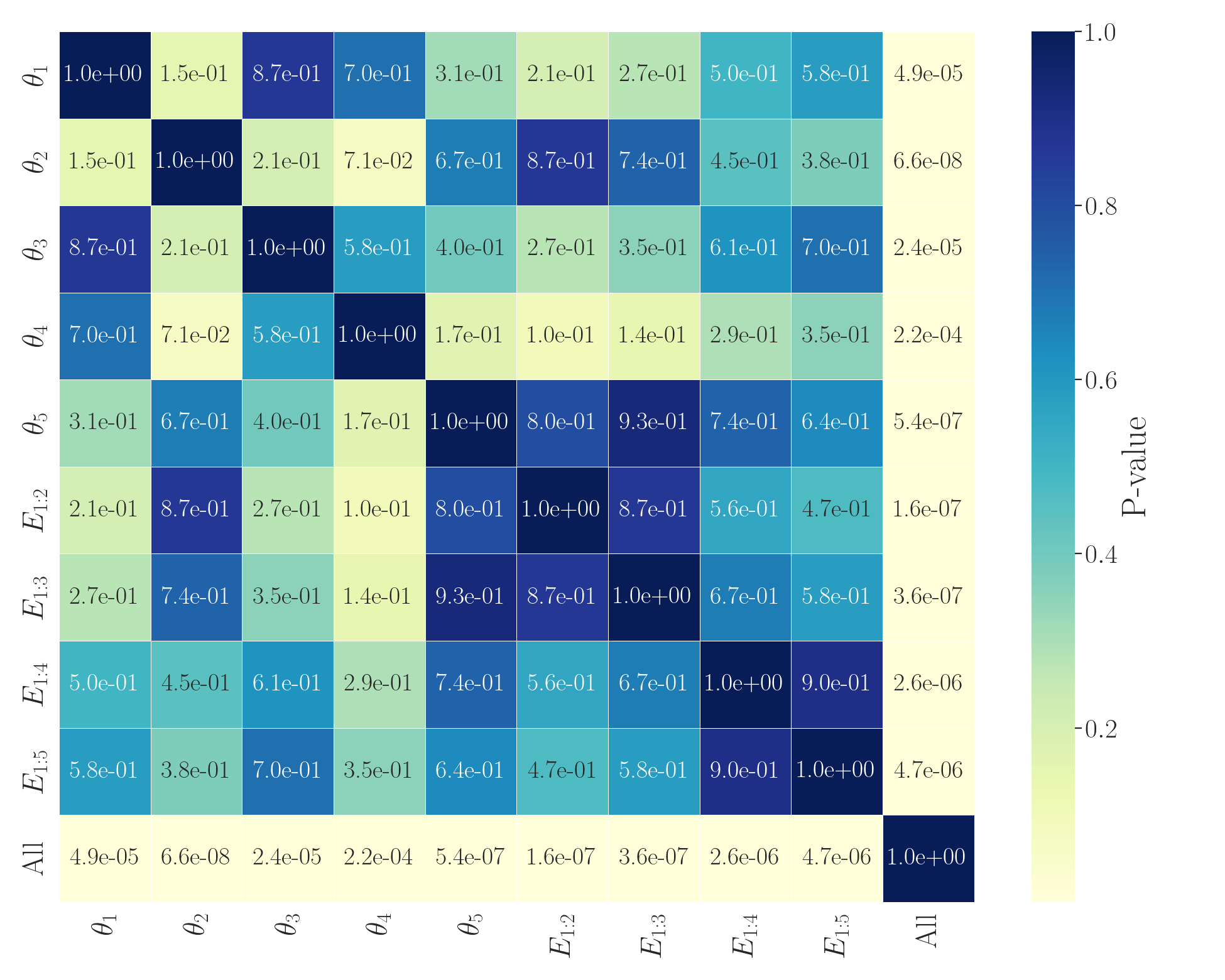}\label{fig:lrnemgecr_Loss}}%
		\hfill
		\subfloat[GFE]{\includegraphics[width=0.24\textwidth]{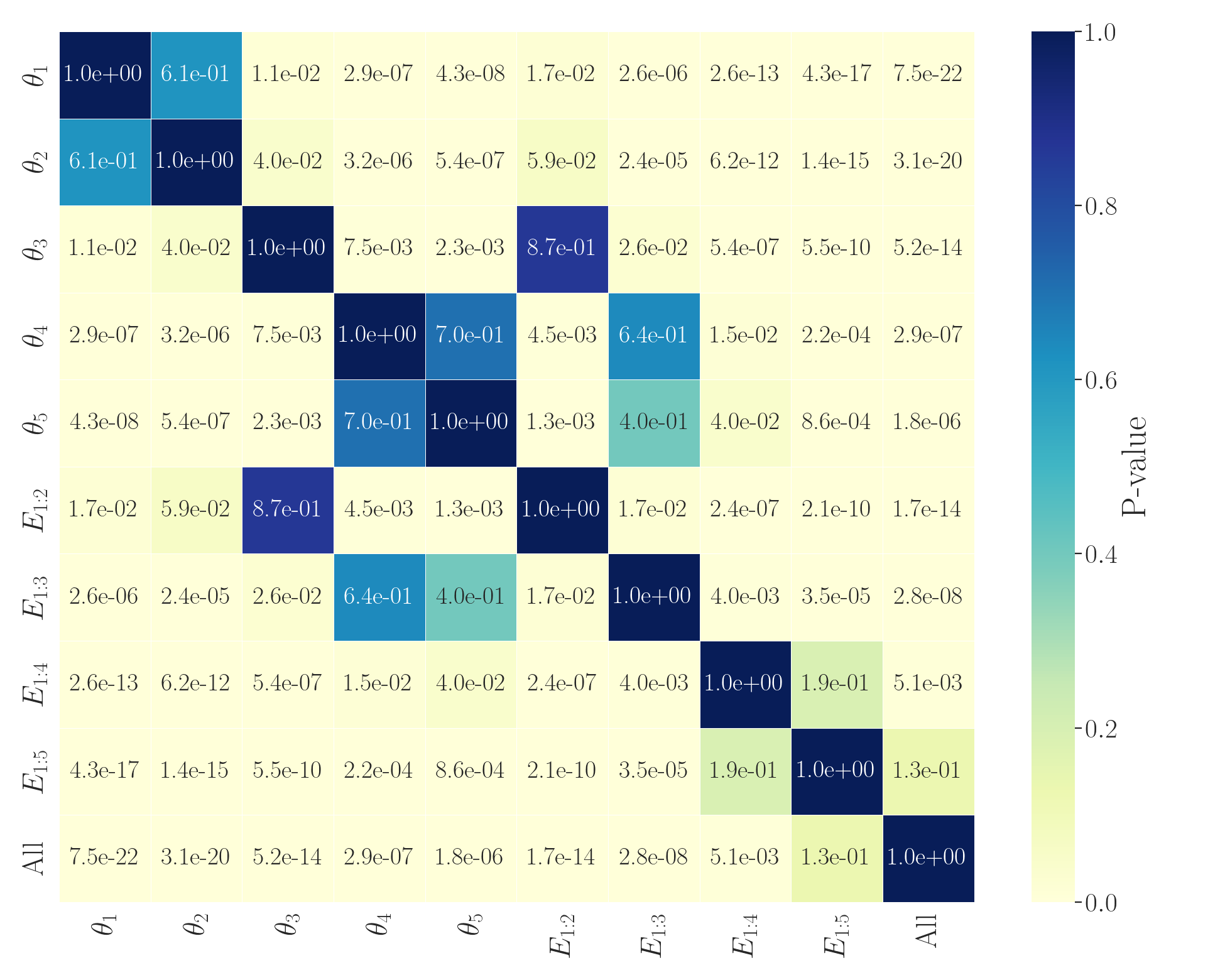}\label{fig:lrnemgfe_Loss}}
		
		\subfloat[GSAD]{\includegraphics[width=0.24\textwidth]{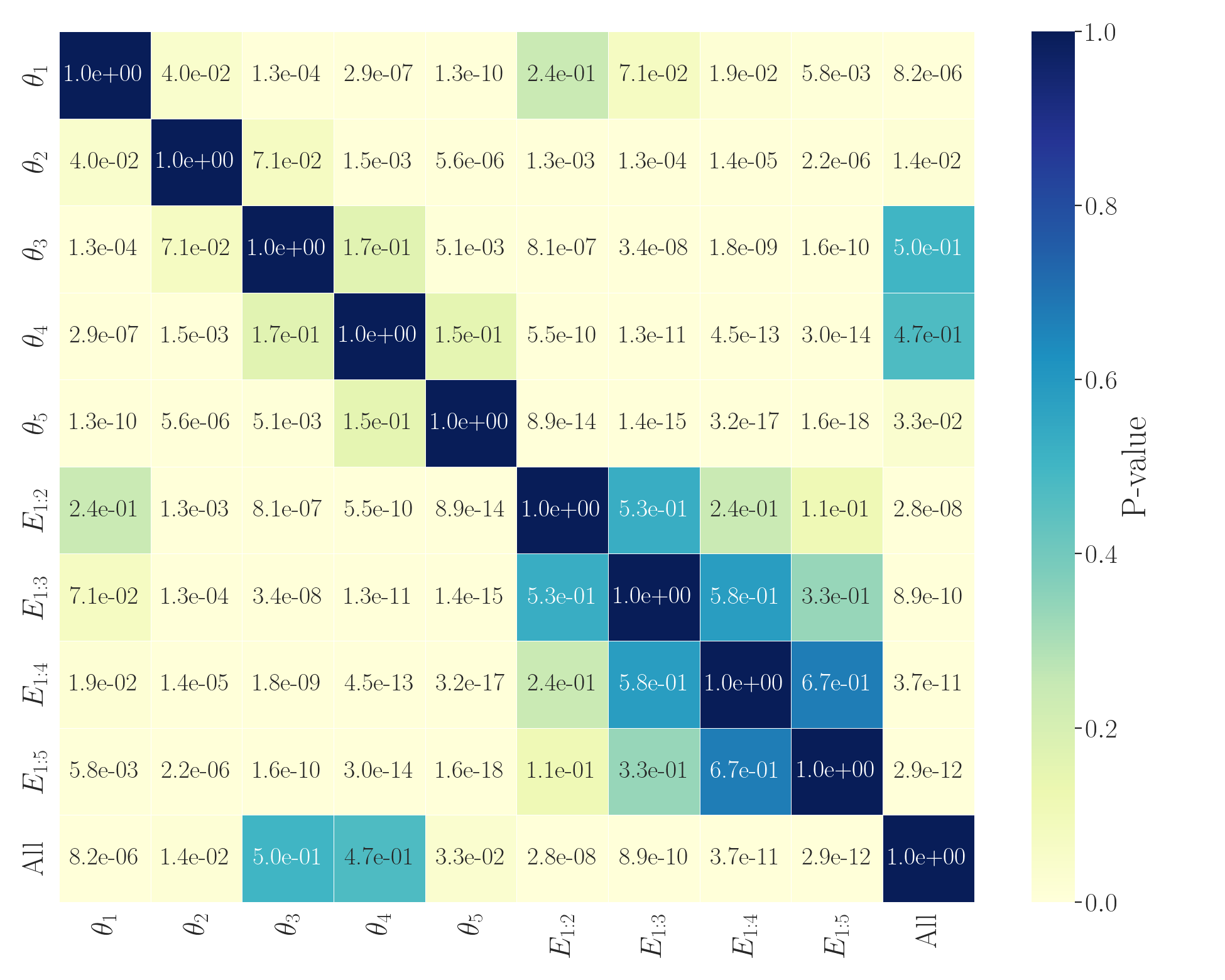}\label{fig:lrnemgsad_Loss}}%
		\hfill
		\subfloat[HAPT]{\includegraphics[width=0.24\textwidth]{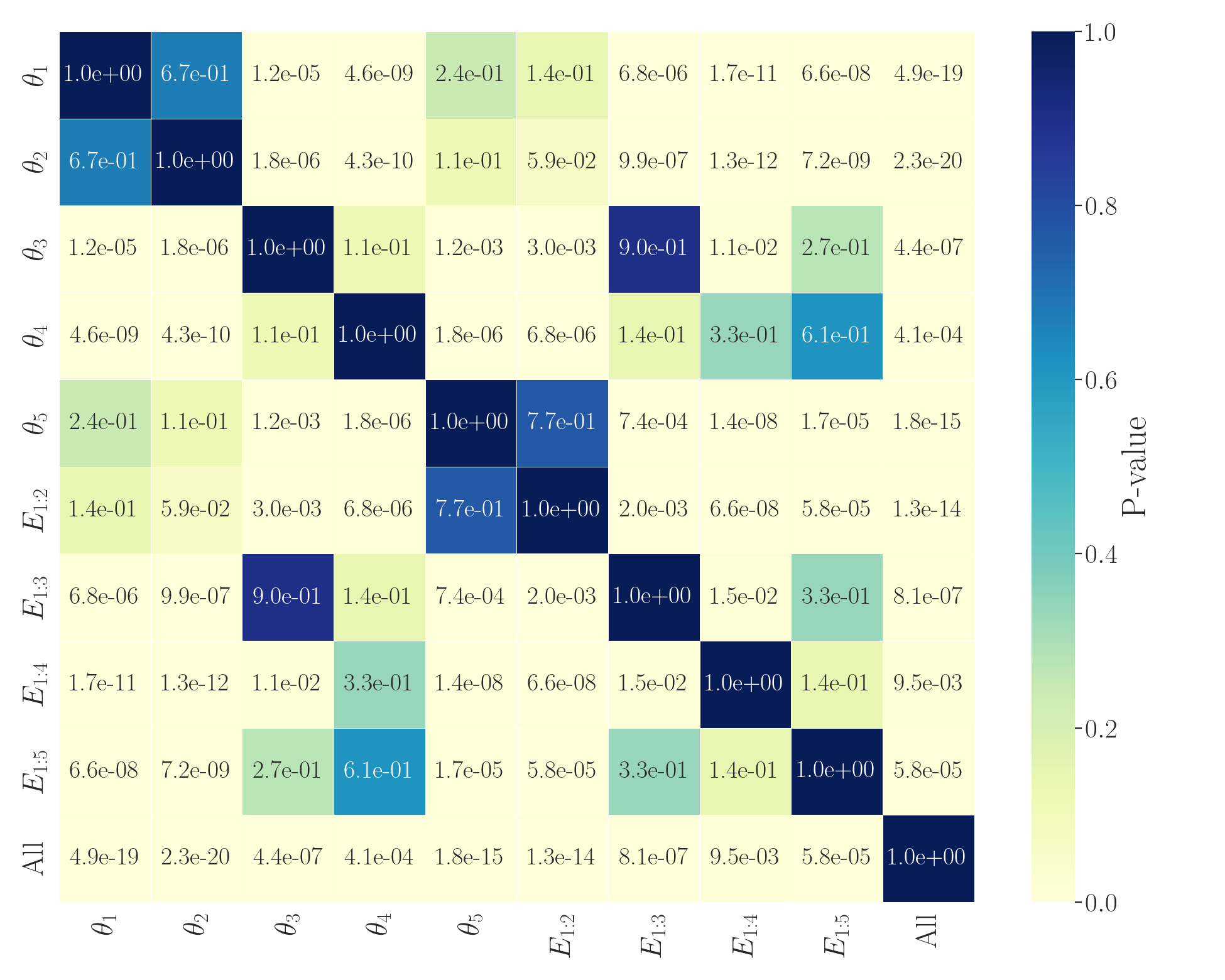}\label{fig:lrnemhapt_Loss}}%
		\hfill
		\subfloat[ISOLET]{\includegraphics[width=0.24\textwidth]{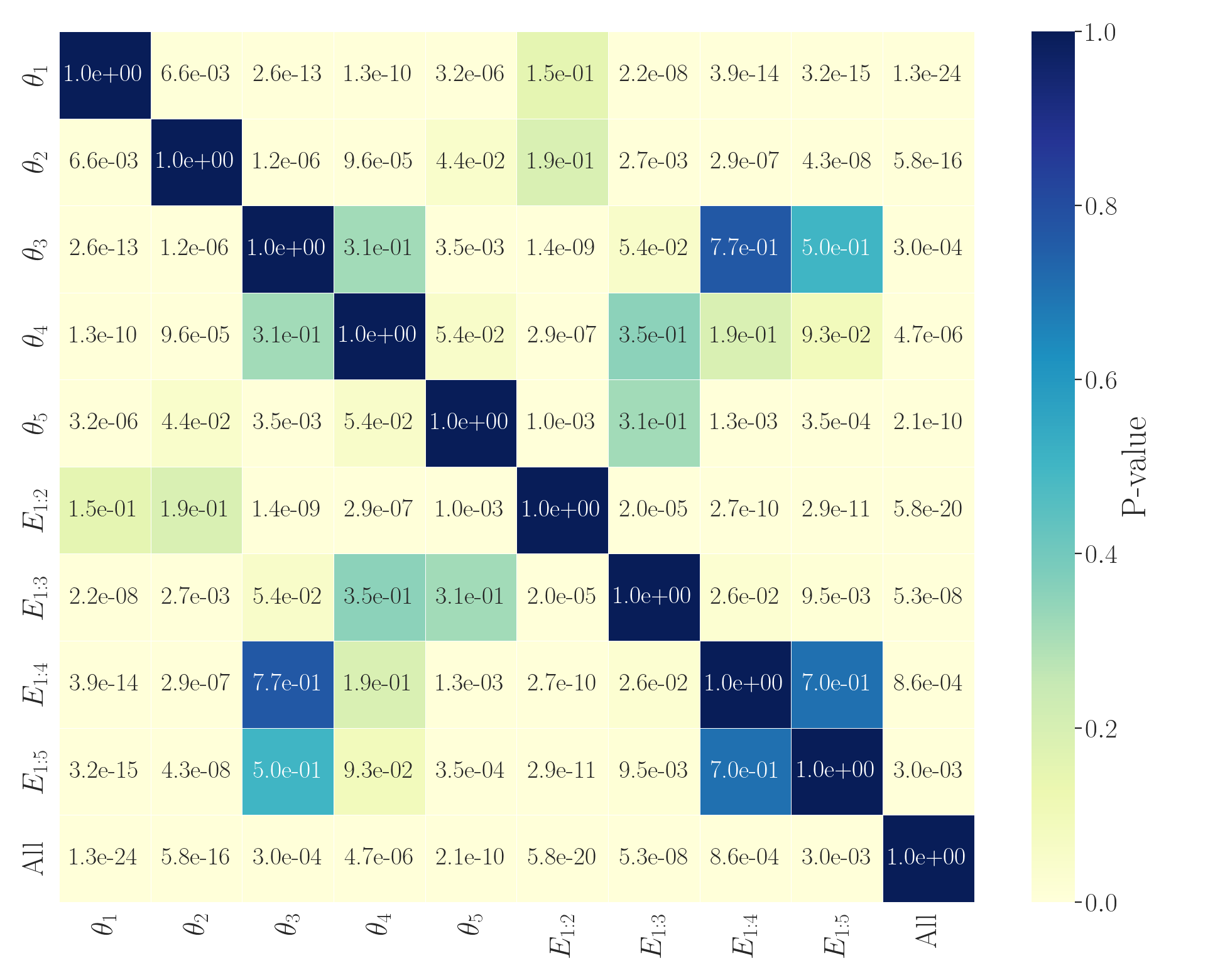}\label{fig:lrnemisolet_Loss}}%
		\hfill
		\subfloat[PD]{\includegraphics[width=0.24\textwidth]{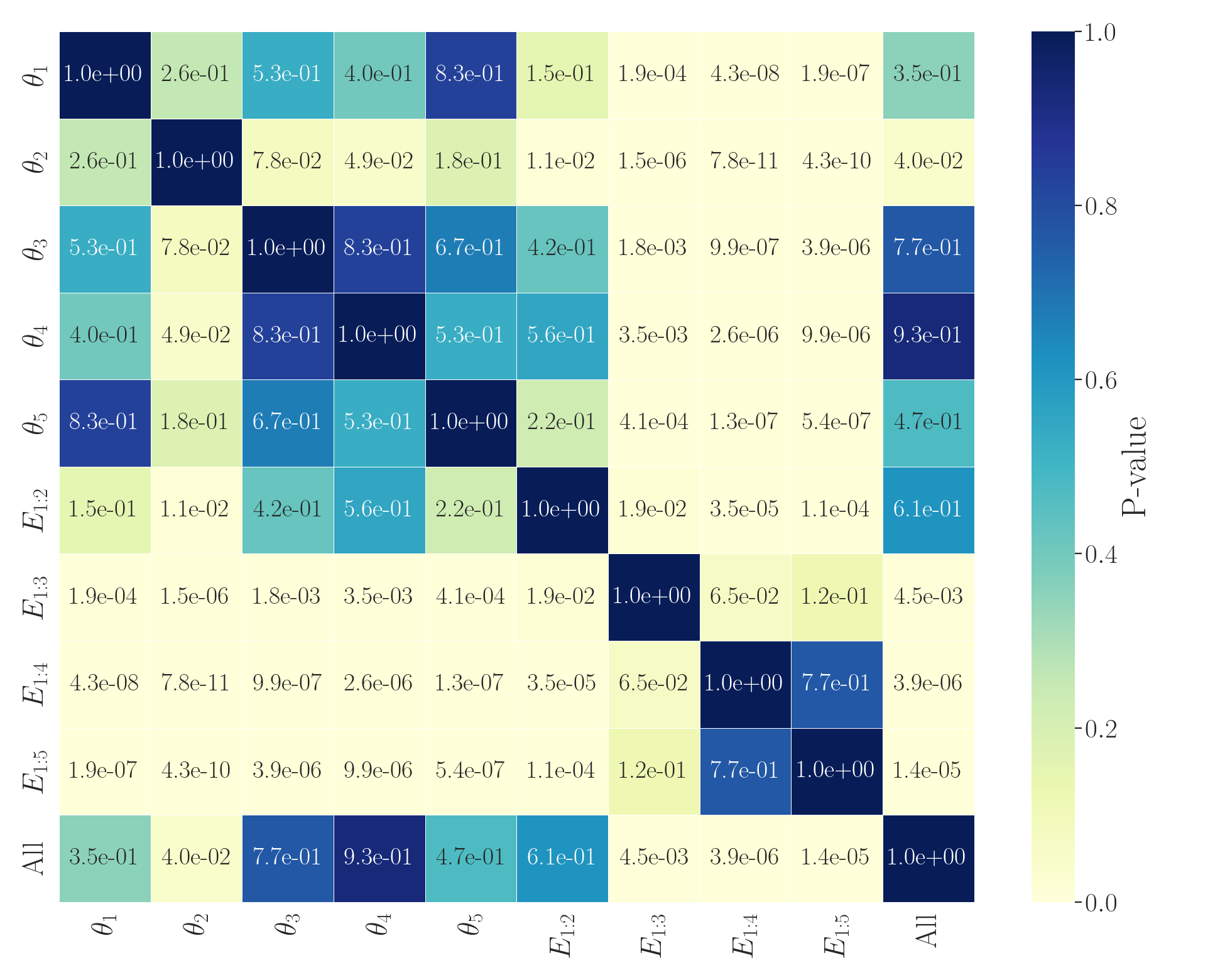}\label{fig:lrnempd_Loss}}
		\caption[The adjusted Conover's P-values for the obtained Log-Loss values in 30 Logistic Regression runs.]{The results of the Conover post-hoc test on testing data’s Log-Loss obtained from 30 Logistic Regression runs.}
		
		\label{fig:lrnem_Loss}
	\end{figure*}
	\FloatBarrier
	
	\begin{figure*}[htbp] 
		\centering
		\subfloat[APSF]{\includegraphics[width=0.24\textwidth]{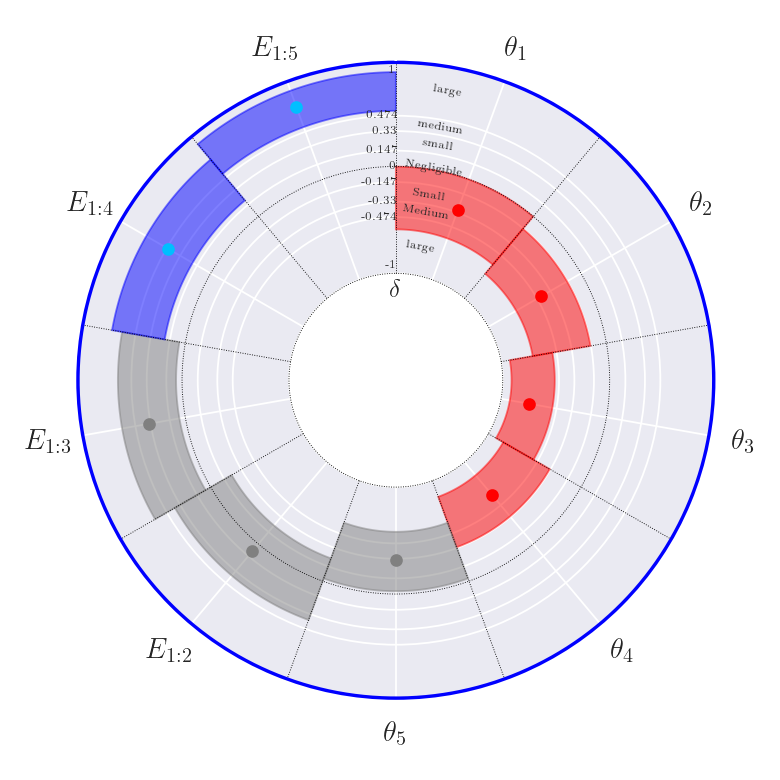}\label{fig:lrcliffapsf_Loss}}%
		\hfill
		\subfloat[ARWPM]{\includegraphics[width=0.24\textwidth]{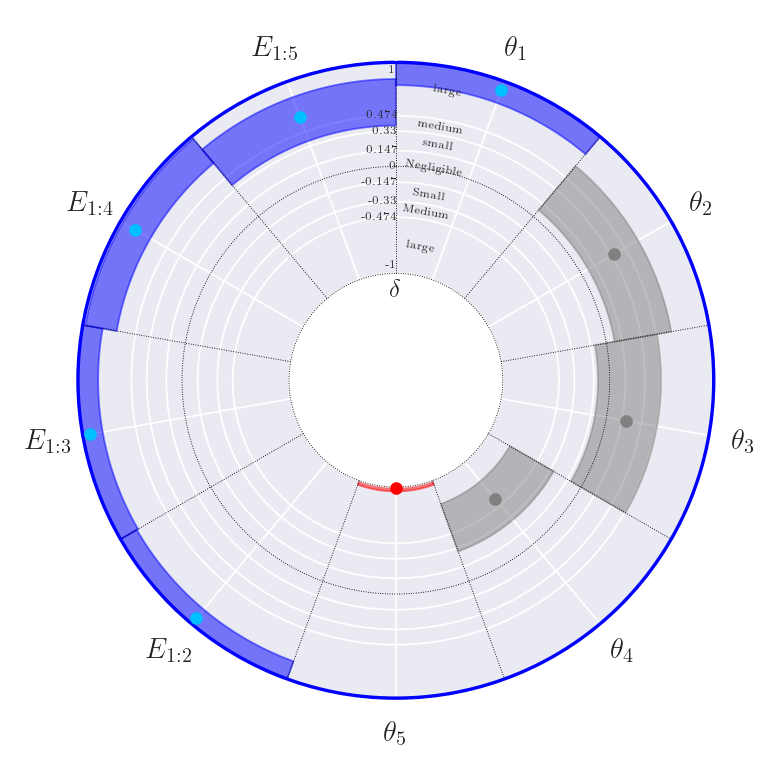}\label{fig:lrcliffarwpm_Loss}}%
		\hfill
		\subfloat[GECR]{\includegraphics[width=0.24\textwidth]{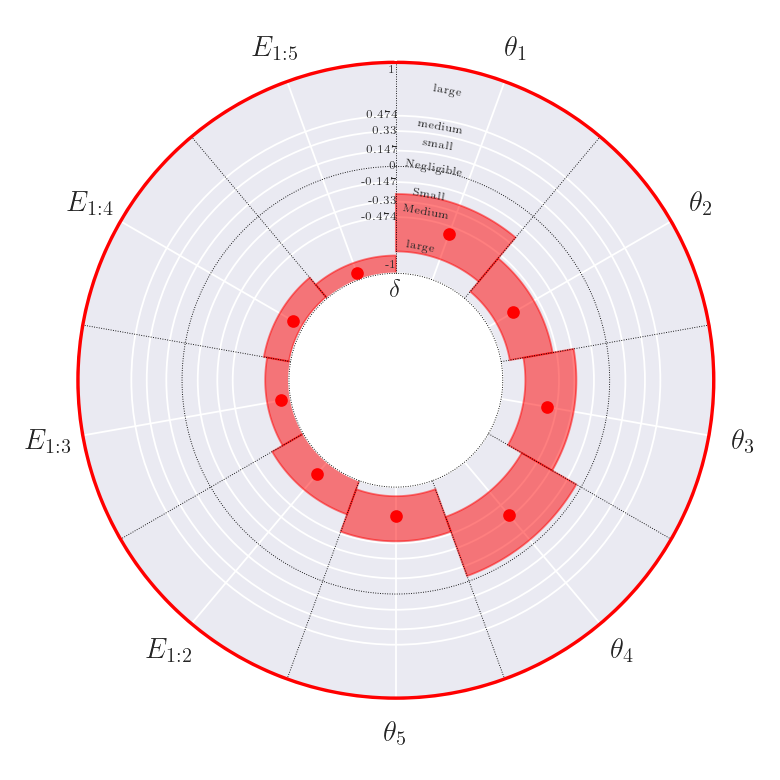}\label{fig:lrcliffgecr_Loss}}%
		\hfill
		\subfloat[GFE]{\includegraphics[width=0.24\textwidth]{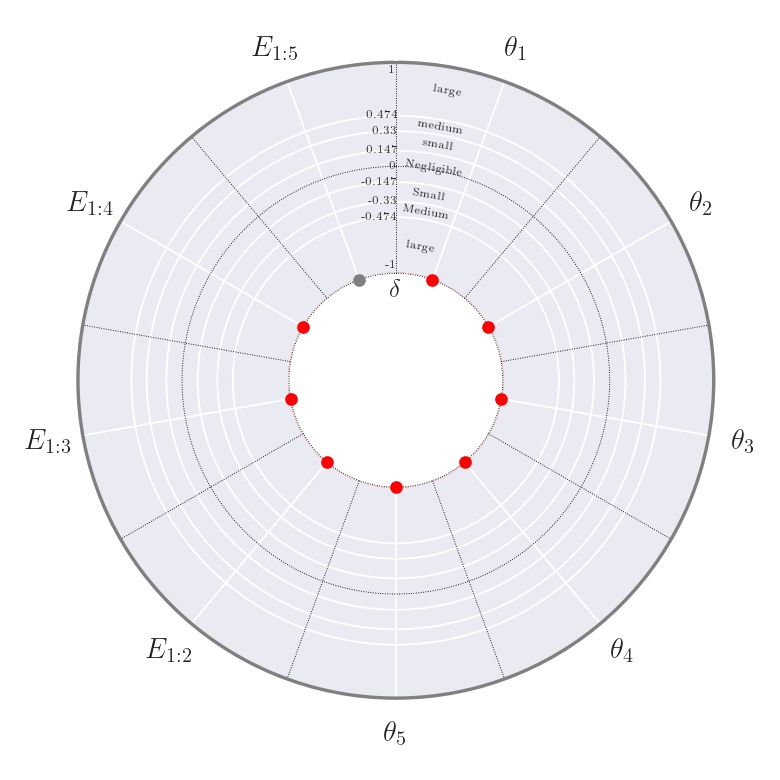}\label{fig:lrcliffgfe_Loss}}
		
		\subfloat[GSAD]{\includegraphics[width=0.24\textwidth]{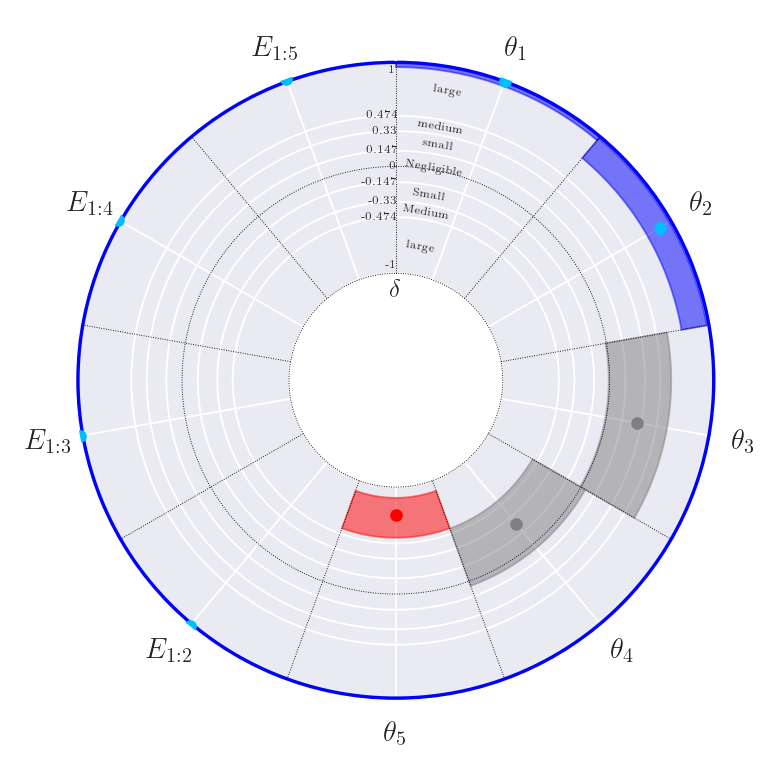}\label{fig:lrcliffgsad_Loss}}%
		\hfill
		\subfloat[HAPT]{\includegraphics[width=0.24\textwidth]{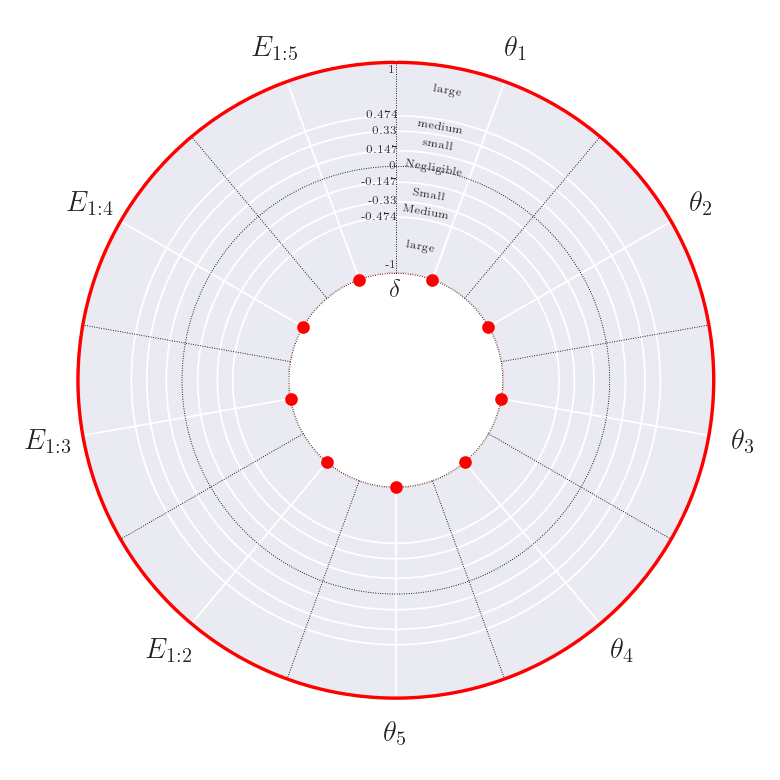}\label{fig:lrcliffhapt_Loss}}%
		\hfill
		\subfloat[ISOLET]{\includegraphics[width=0.24\textwidth]{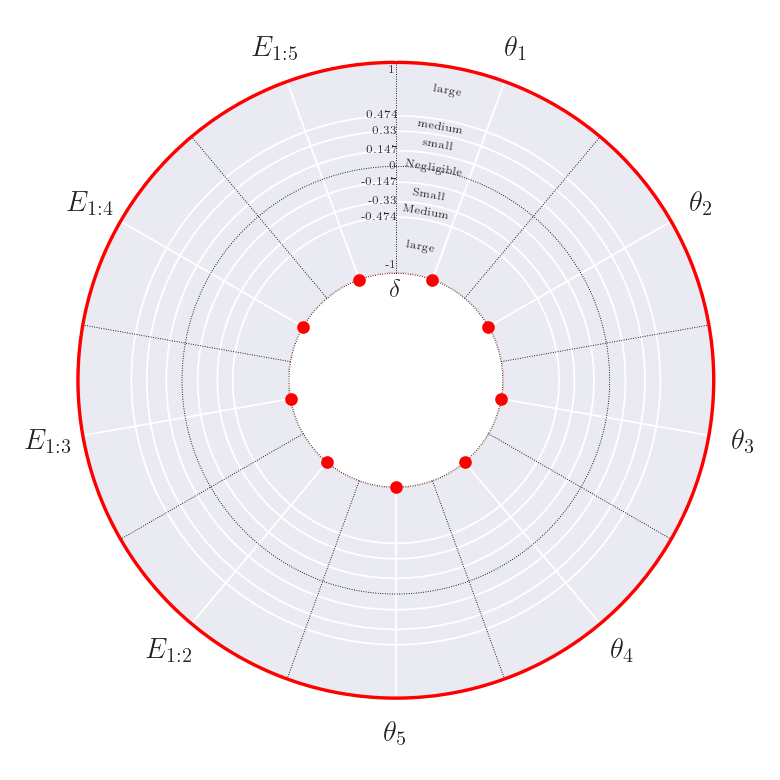}\label{fig:lrcliffisolet_Loss}}%
		\hfill
		\subfloat[PD]{\includegraphics[width=0.24\textwidth]{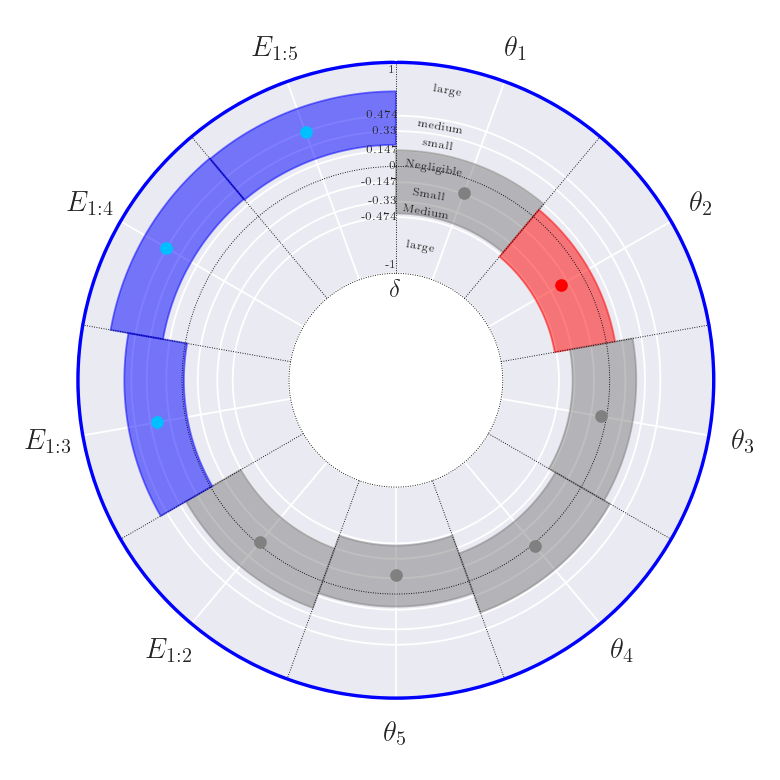}\label{fig:lrcliffpd_Loss}}
		\caption[The Cliff's $\delta$ effect size measure and its 95\% confidence intervals for the Log-Loss values obtained from 30 Logistic Regression runs.]{Effect size analysis of test data Log-Loss across 30 Logistic Regression runs using Cliff's $\delta$. Each point represents the actual value obtained, with segments denoting 95\% confidence intervals based on 10,000 bootstrap resamplings. The outer ring color visualizes the statistical significance: grey illustrates no significant difference (adjusted Friedman's P-value$>0.05$), while color indicates significant differences; blue indicates at least one view and/or ensemble outperforms the benchmark (adjusted Conover's p-value$ < 0.05$, Cliff's $\delta > 0$), and red signifies all views and ensembles underperform relative to the benchmark (adjusted Conover's p-value$ < 0.05$, Cliff's $\delta < 0$). Segment colors show performance difference against the benchmark: grey for no significant difference (adjusted Conover's p-value$  > 0.05$), blue for better performance (Cliff's $\delta > 0$), and red for worse performance (Cliff's $\delta < 0$).}
		
		\label{fig:lrcliff_Loss}
	\end{figure*}
	
	\begin{table*}[htbp]
		\centering
		\caption[The results of Friedman and Conover tests and Cliff's $\delta$ analysis for the Log-Loss values obtained from 30 Logistic Regression runs.]{Statistical comparison of Log-Loss results for testing data obtained from Logistic Regression runs. W, T, and L denote win, tie, and loss based on adjusted Friedman and Conover's p-values. Effect sizes are calculated using Cliff's Delta method and are categorized as negligible, small, medium, or large.}
		\label{tab:lrloss}
		\resizebox{\linewidth}{!}{%
			\begin{tabular}{c|ccccccccc}
				\hline
				\multicolumn{10}{c}{Logistic Regression's Log-Loss}\\
				\hline
				Dataset & $\theta_1$ & $\theta_2$ & $\theta_3$ & $\theta_4$ & $\theta_5$ & $E_{1:2}$ & $E_{1:3}$ & $E_{1:4}$ & $E_{1:5}$ \\
				\hline
				APSF  & L (small) & L (medium) & L (large) & L (large) & T (small) & T (negligible) & T (medium) & W (medium) & W (large) \\
				ARWPM  & W (large) & T (medium) & T (small) & T (large) & L (large) & W (large) & W (large) & W (large) & W (large) \\
				GECR  & L (large) & L (large) & L (large) & L (medium) & L (large) & L (large) & L (large) & L (large) & L (large) \\
				GFE  & L (large) & L (large) & L (large) & L (large) & L (large) & L (large) & L (large) & L (large) & T (large) \\
				GSAD  & W (large) & W (large) & T (small) & T (small) & L (large) & W (large) & W (large) & W (large) & W (large) \\
				HAPT  & L (large) & L (large) & L (large) & L (large) & L (large) & L (large) & L (large) & L (large) & L (large) \\
				ISOLET  & L (large) & L (large) & L (large) & L (large) & L (large) & L (large) & L (large) & L (large) & L (large) \\
				PD  & T (negligible) & L (small) & T (negligible) & T (negligible) & T (small) & T (negligible) & W (small) & W (large) & W (medium) \\
				\hline
				W - T - L  & 2 - 1 - 5 & 1 - 1 - 6 & 0 - 3 - 5 & 0 - 3 - 5 & 0 - 2 - 6 & 2 - 2 - 4 & 3 - 1 - 4 & 4 - 0 - 4 & 4 - 1 - 3 \\
				\hline
			\end{tabular}
		}
	\end{table*}
	\FloatBarrier
	
	\begin{figure*}[t] 
		\centering
		\subfloat[APSF]{\includegraphics[width=0.24\textwidth]{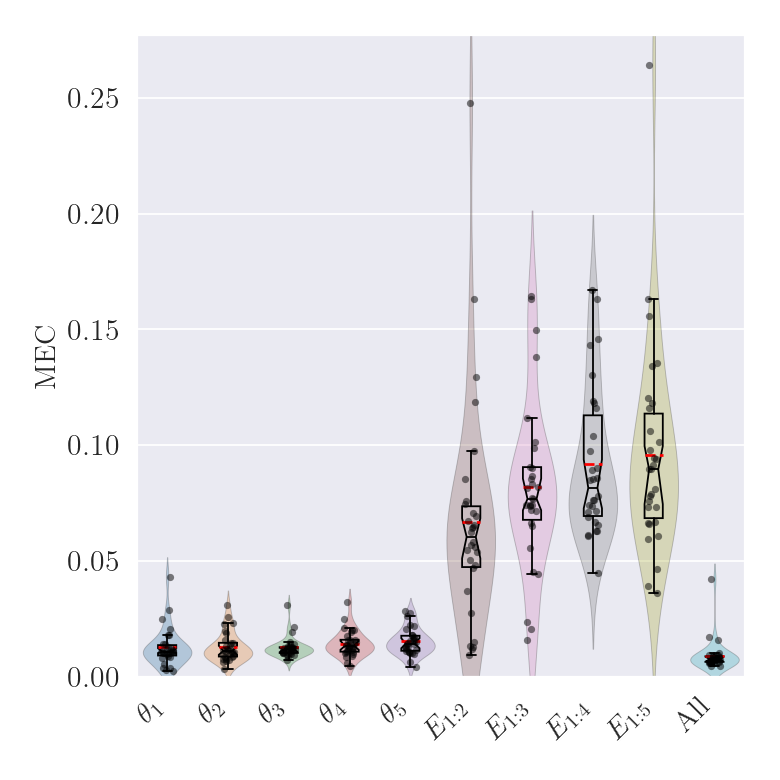}\label{fig:lrapsf_MEC}}%
		\hfill
		\subfloat[ARWPM]{\includegraphics[width=0.24\textwidth]{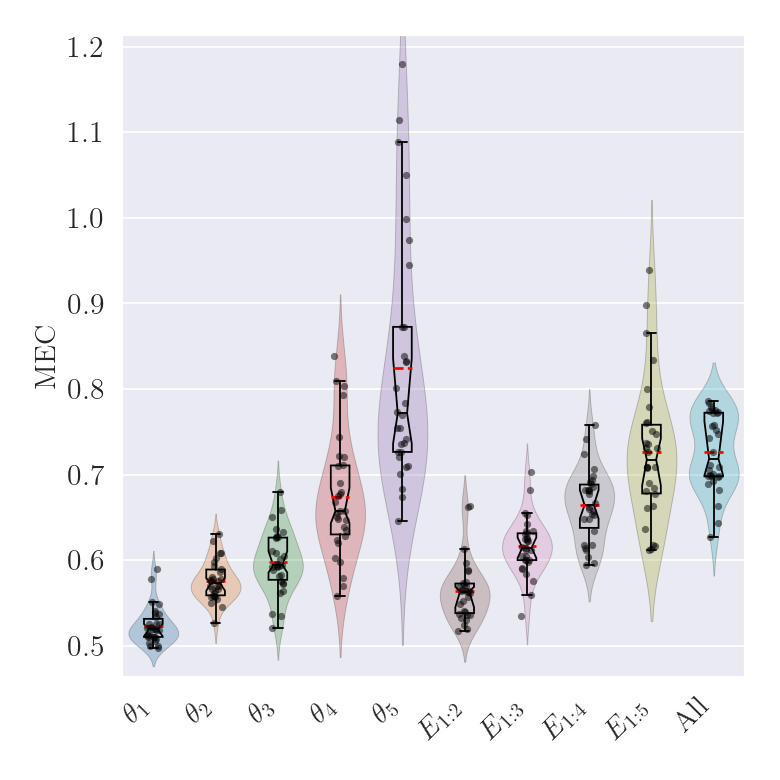}\label{fig:lrarwpm_MEC}}%
		\hfill
		\subfloat[GECR]{\includegraphics[width=0.24\textwidth]{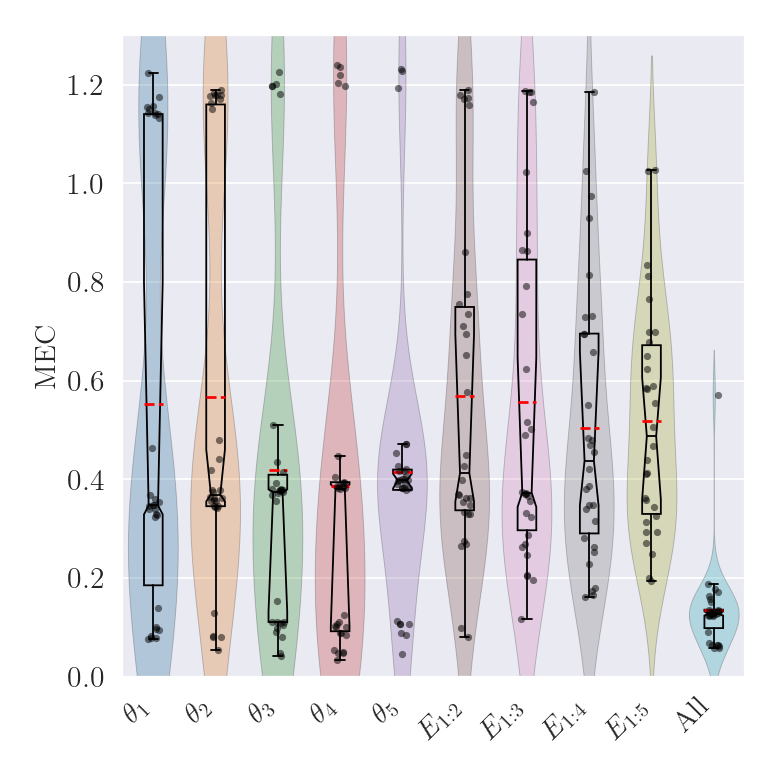}\label{fig:lrgecr_MEC}}%
		\hfill
		\subfloat[GFE]{\includegraphics[width=0.24\textwidth]{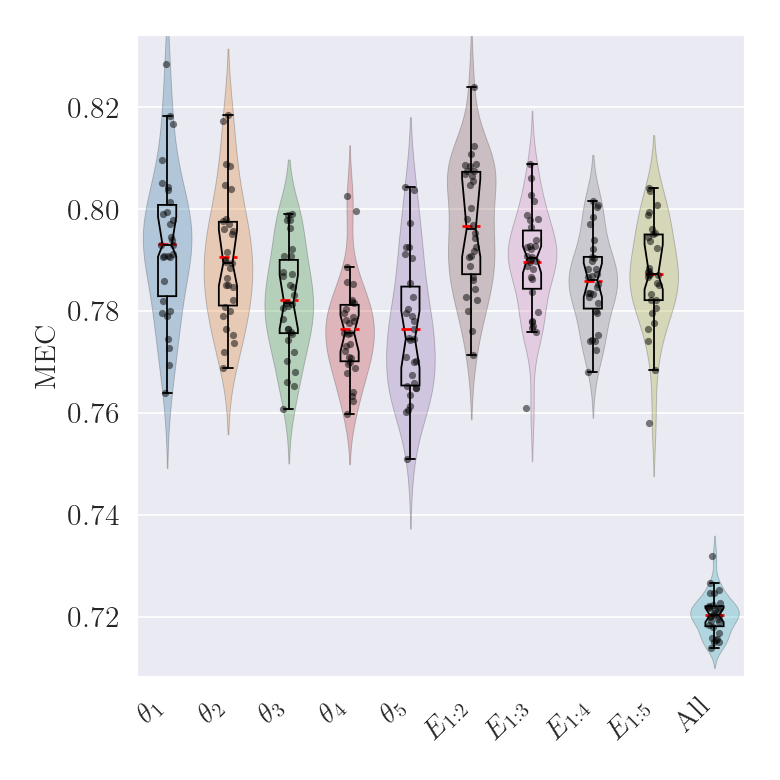}\label{fig:lrgfe_MEC}}
		
		\subfloat[GSAD]{\includegraphics[width=0.24\textwidth]{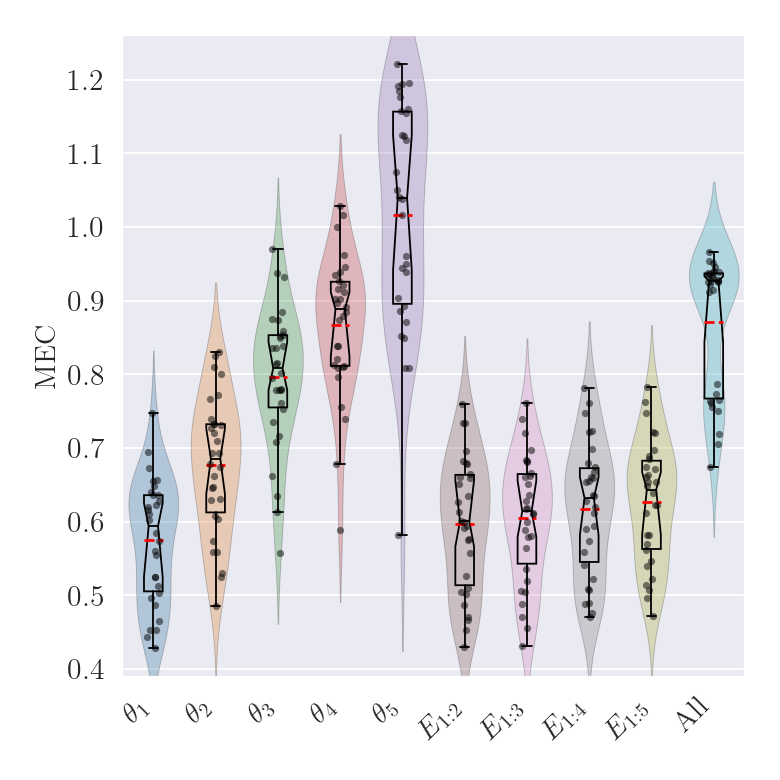}\label{fig:fpgsad_MEC}}%
		\hfill
		\subfloat[HAPT]{\includegraphics[width=0.24\textwidth]{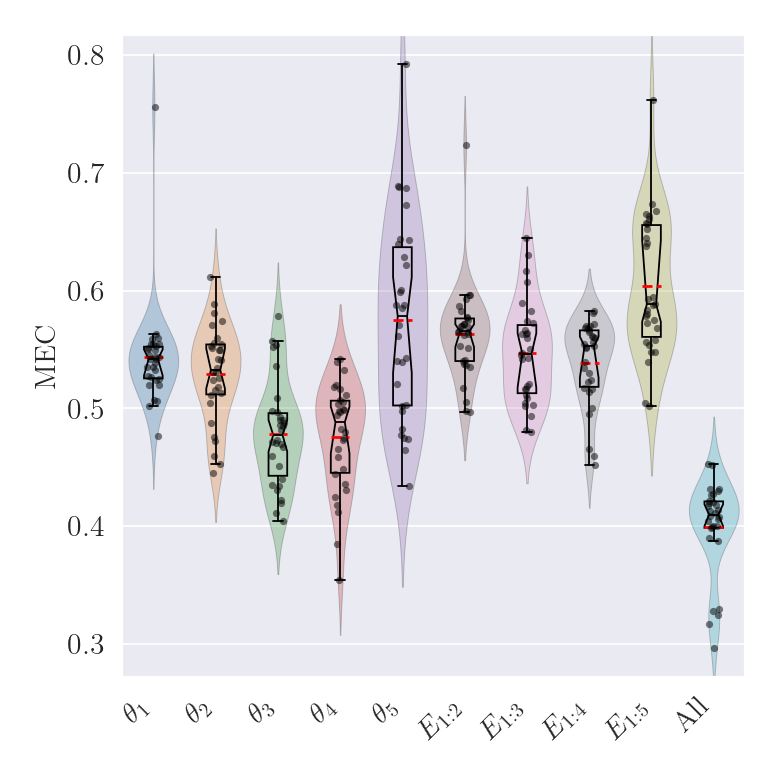}\label{fig:lrhapt_MEC}}%
		\hfill
		\subfloat[ISOLET]{\includegraphics[width=0.24\textwidth]{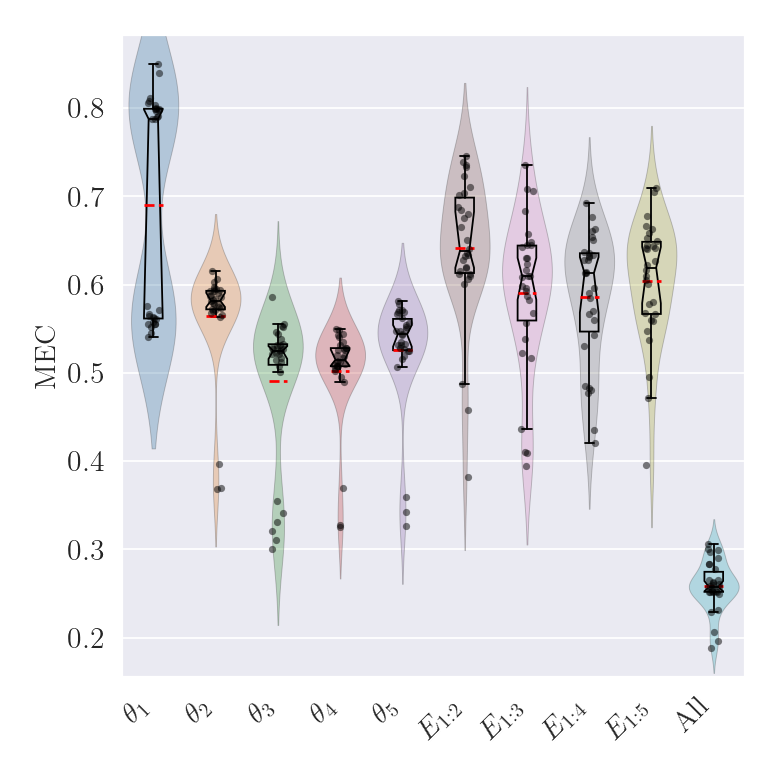}\label{fig:lrisolet_MEC}}%
		\hfill
		\subfloat[PD]{\includegraphics[width=0.24\textwidth]{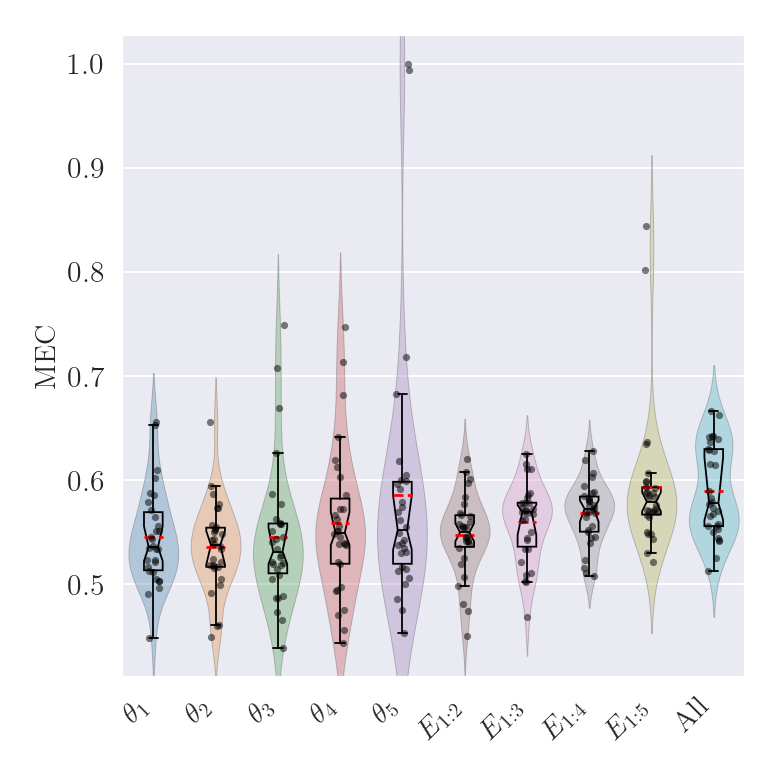}\label{fig:lrpd_MEC}}
		\caption[The distribution of the obtained MEC values for 30 Logistic Regression runs.]{The raincloud plot of MEC results obtained from 30 Logistic Regression runs.}
		
		\label{fig:lr_MEC}
	\end{figure*}
	
	\begin{figure*}[t] 
		\centering
		\subfloat[APSF]{\includegraphics[width=0.24\textwidth]{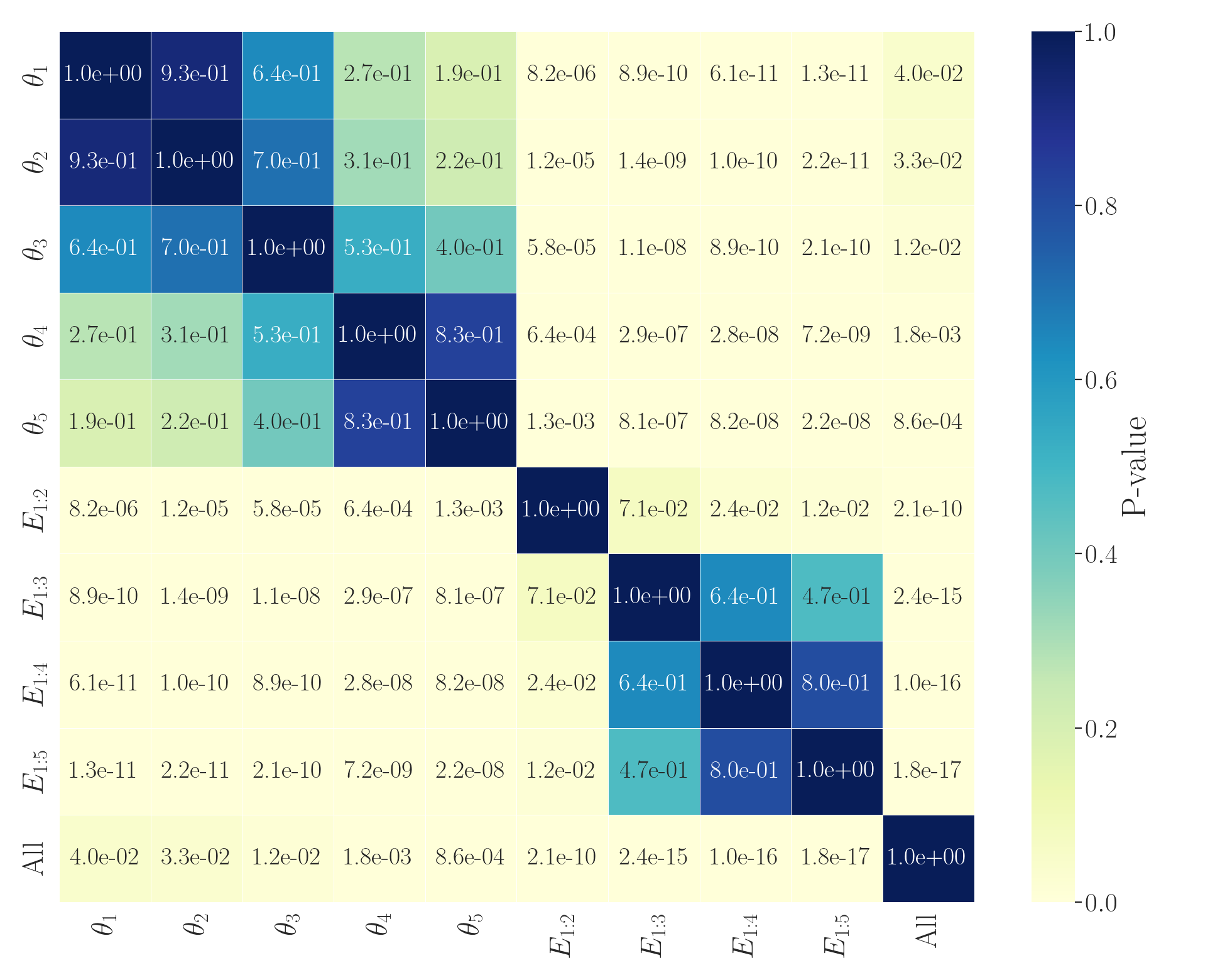}\label{fig:lrnemapsf_MEC}}%
		\hfill
		\subfloat[ARWPM]{\includegraphics[width=0.24\textwidth]{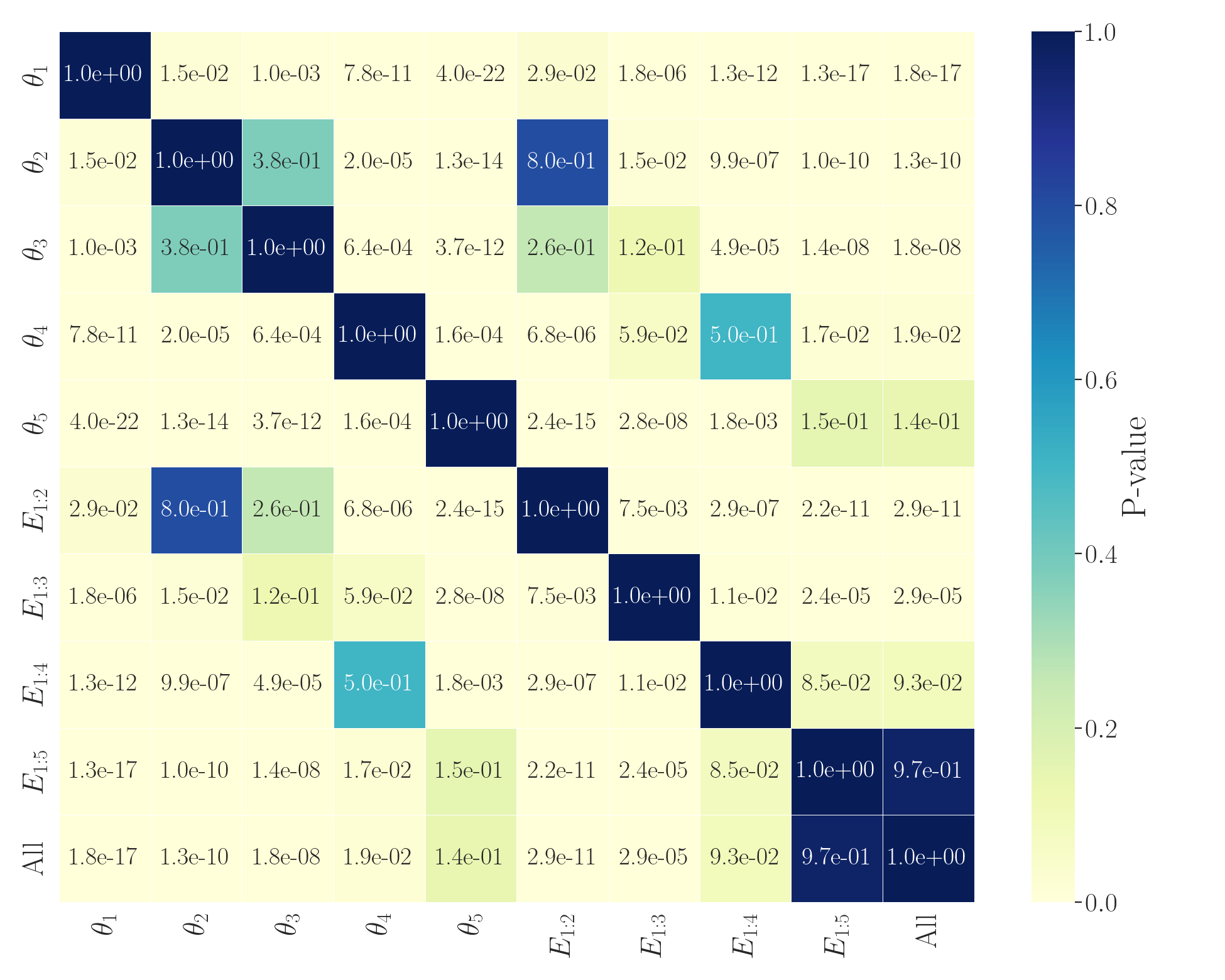}\label{fig:lrnemarwpm_MEC}}%
		\hfill
		\subfloat[GECR]{\includegraphics[width=0.24\textwidth]{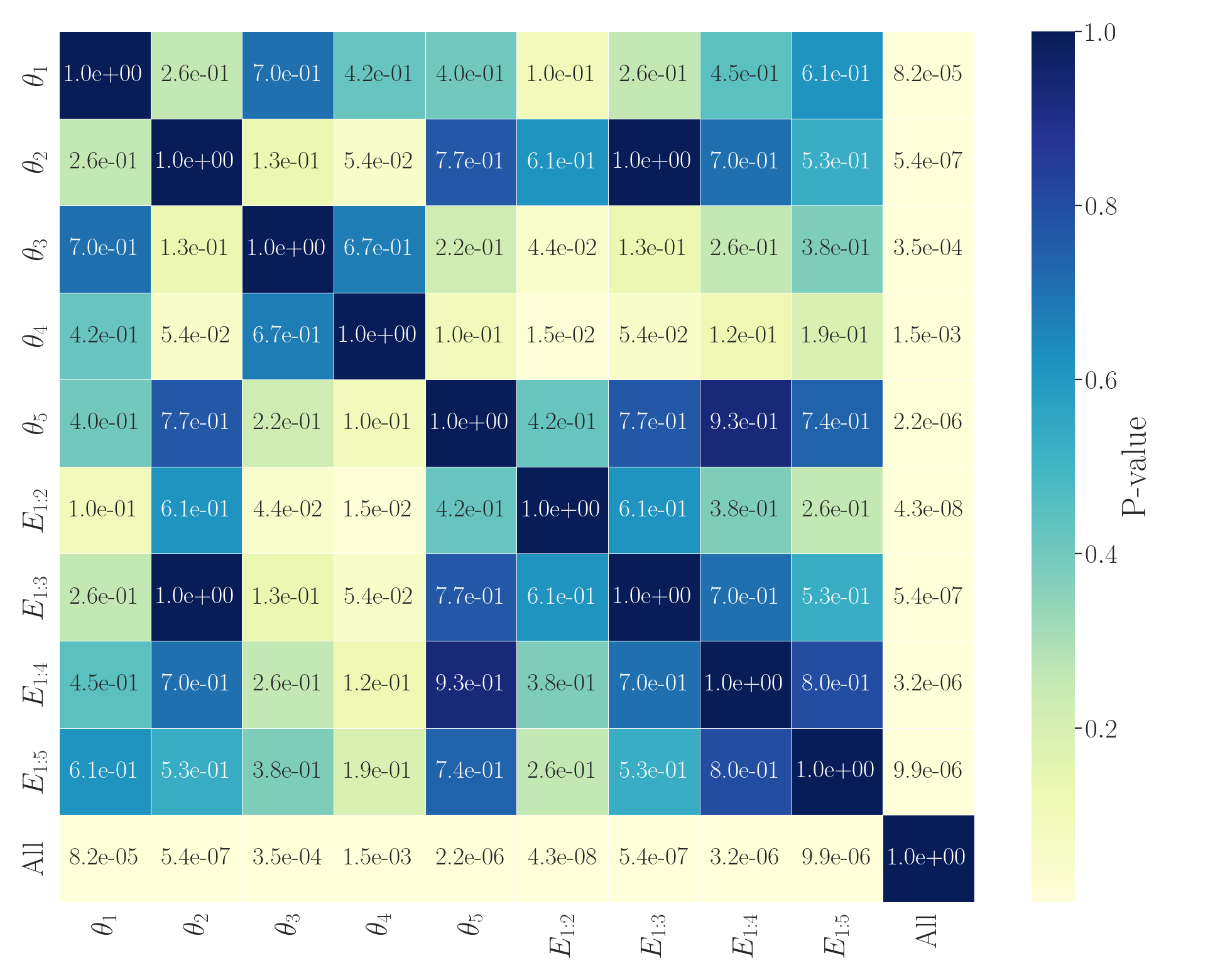}\label{fig:lrnemgecr_MEC}}%
		\hfill
		\subfloat[GFE]{\includegraphics[width=0.24\textwidth]{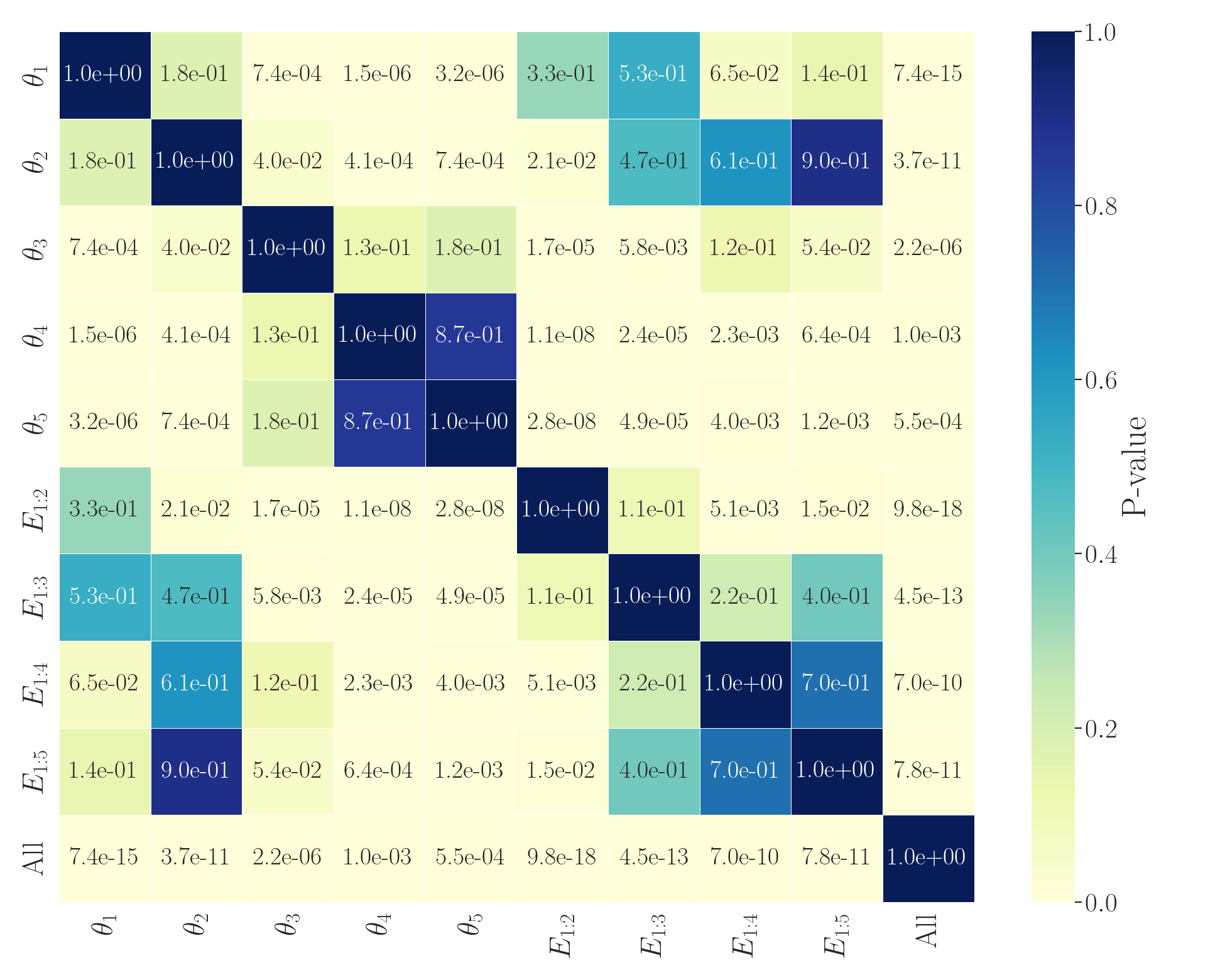}\label{fig:lrnemgfe_MEC}}
		
		\subfloat[GSAD]{\includegraphics[width=0.24\textwidth]{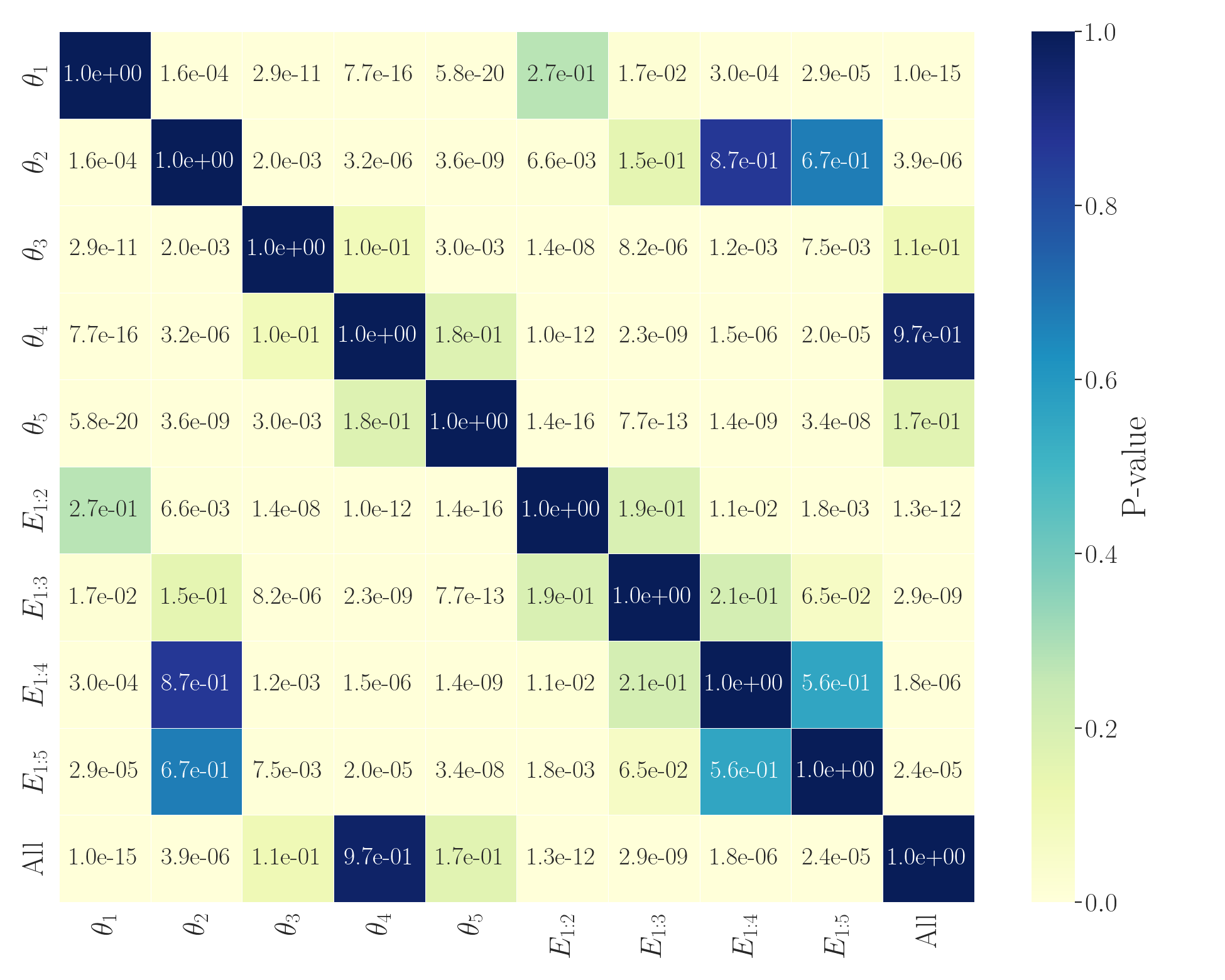}\label{fig:lrnemgsad_MEC}}%
		\hfill
		\subfloat[HAPT]{\includegraphics[width=0.24\textwidth]{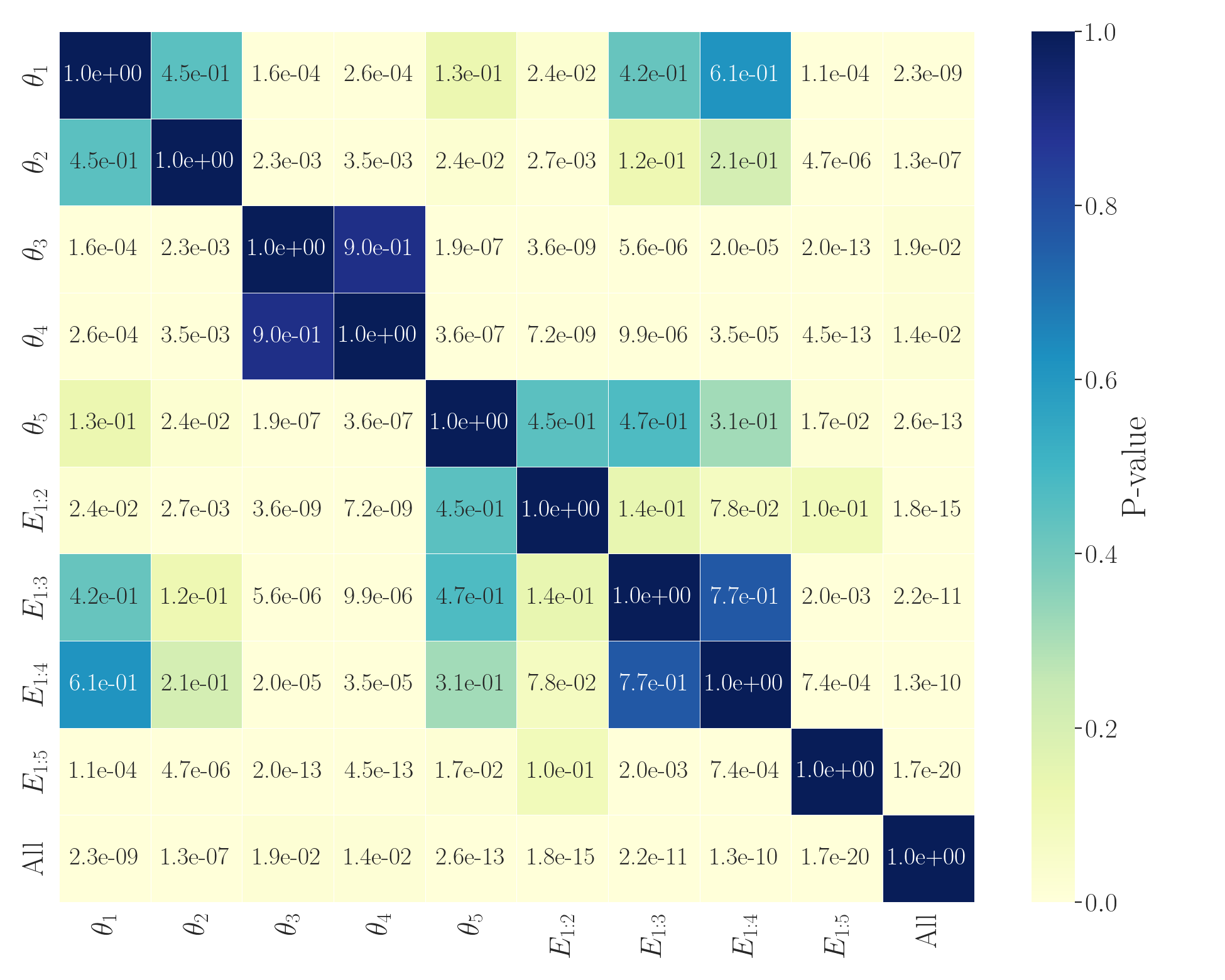}\label{fig:lrnemhapt_MEC}}%
		\hfill
		\subfloat[ISOLET]{\includegraphics[width=0.24\textwidth]{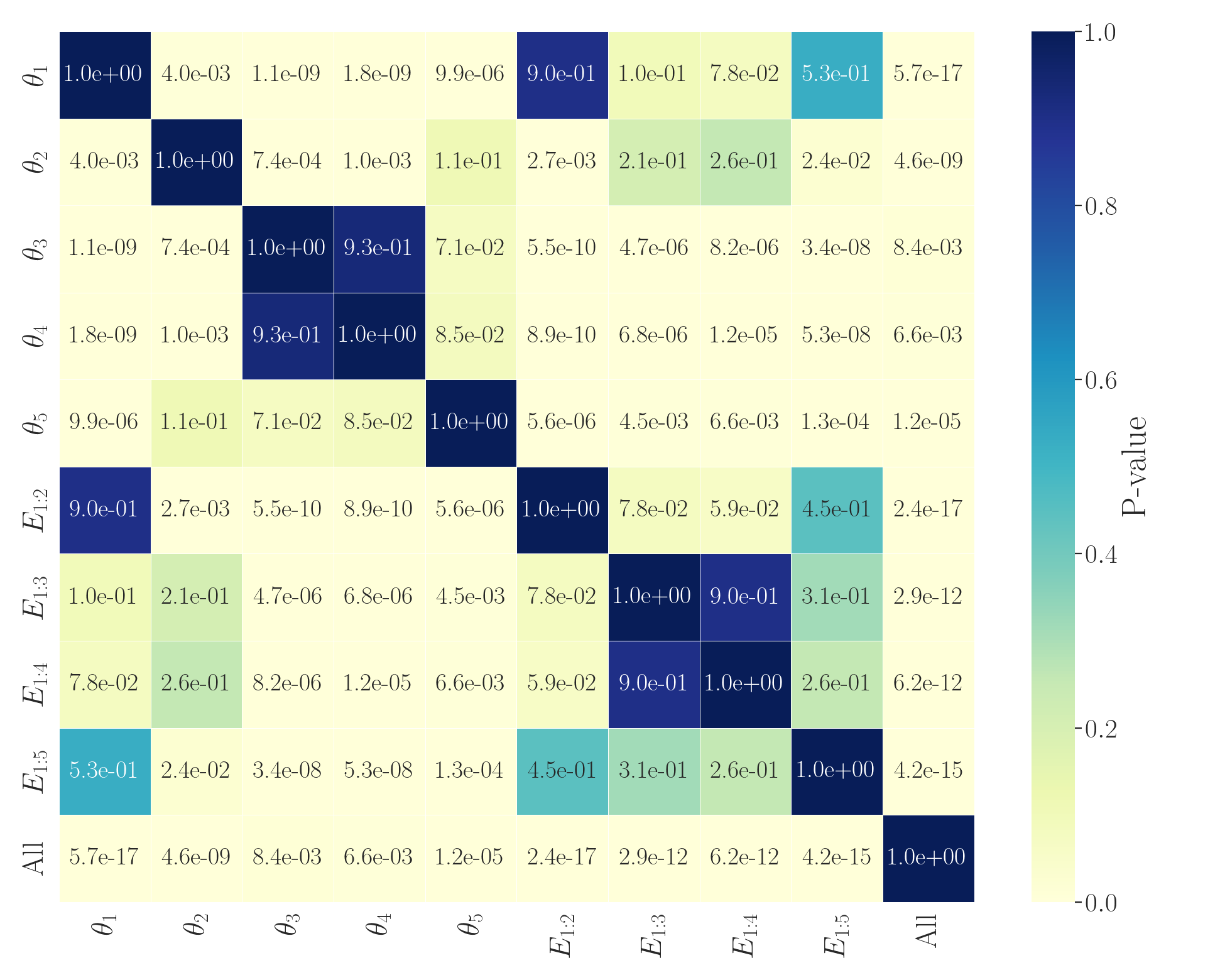}\label{fig:lrnemisolet_MEC}}%
		\hfill
		\subfloat[PD]{\includegraphics[width=0.24\textwidth]{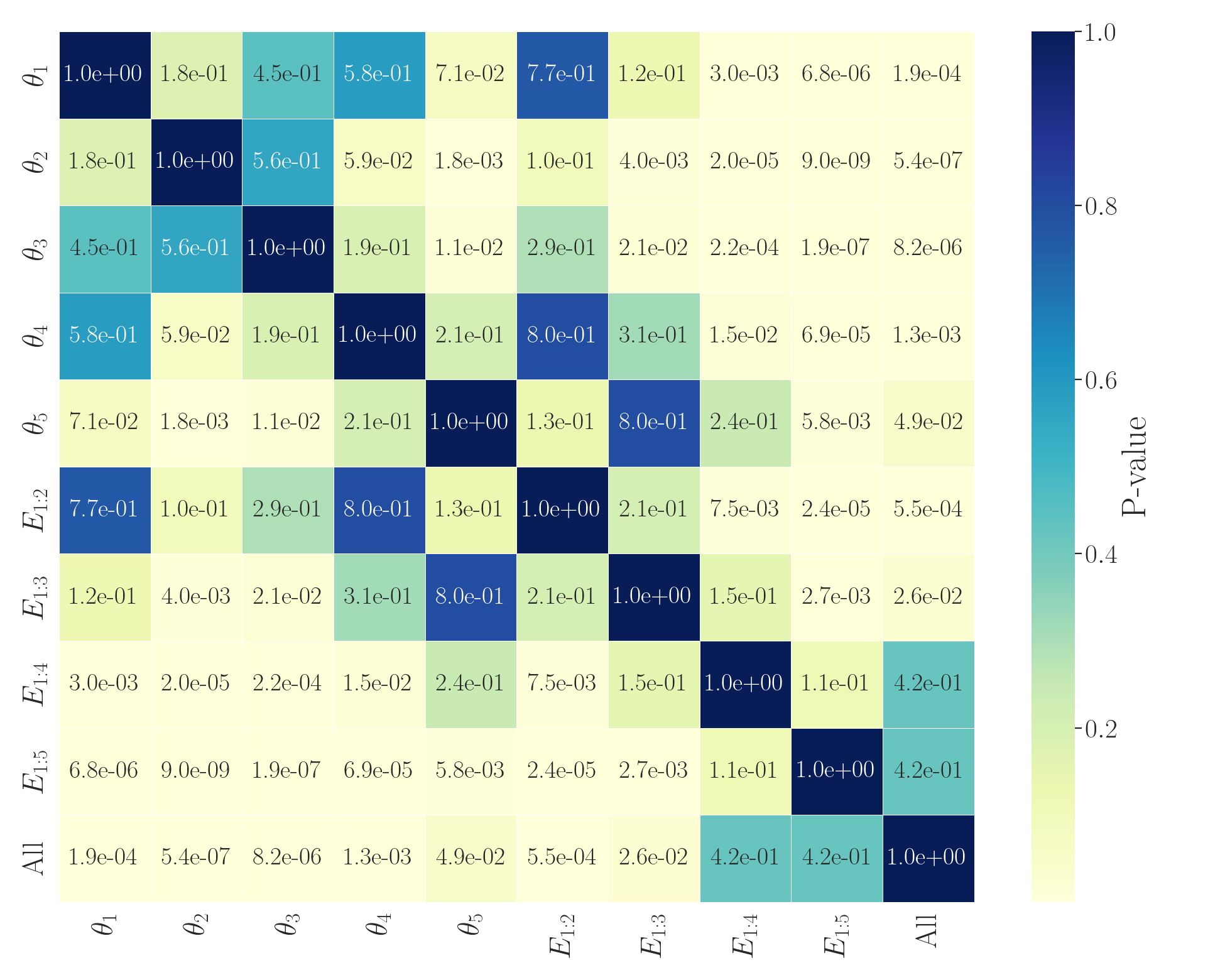}\label{fig:lrnempd_MEC}}
		\caption[The adjusted Conover's P-values for the obtained MEC values in 30 Logistic Regression runs.]{The results of the Conover post-hoc test on testing data’s MEC obtained from 30 Logistic Regression runs.}
		
		\label{fig:lrnem_MEC}
	\end{figure*}
	\FloatBarrier
	
	\begin{figure*}[htbp] 
		\centering
		\subfloat[APSF]{\includegraphics[width=0.24\textwidth]{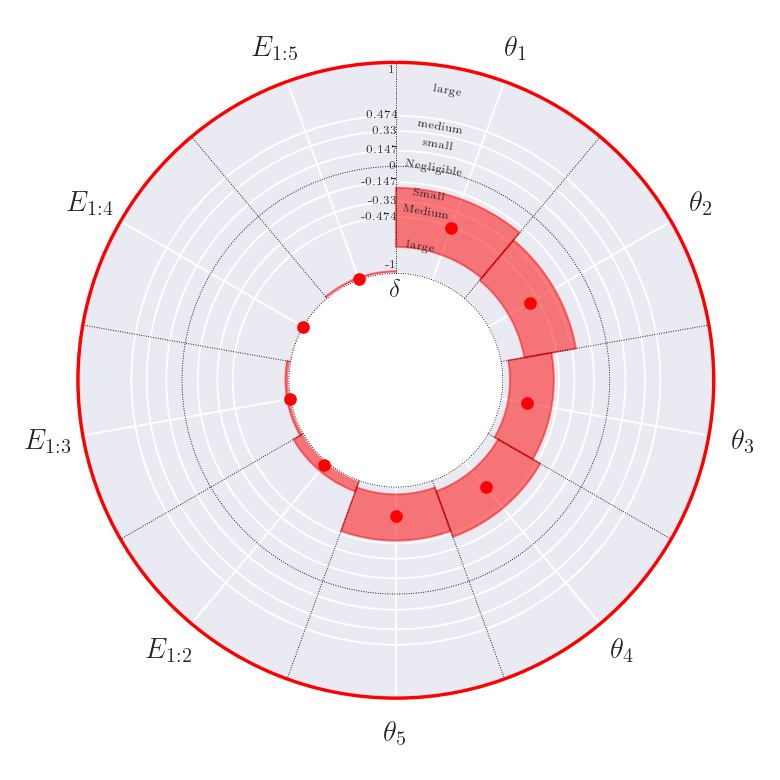}\label{fig:lrcliffapsf_MEC}}%
		\hfill
		\subfloat[ARWPM]{\includegraphics[width=0.24\textwidth]{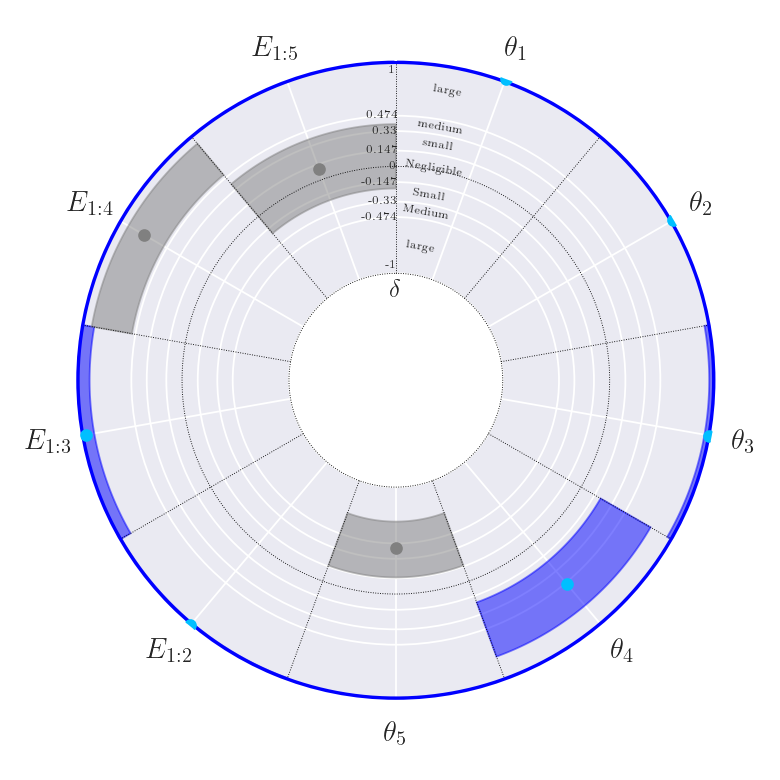}\label{fig:lrcliffarwpm_MEC}}%
		\hfill
		\subfloat[GECR]{\includegraphics[width=0.24\textwidth]{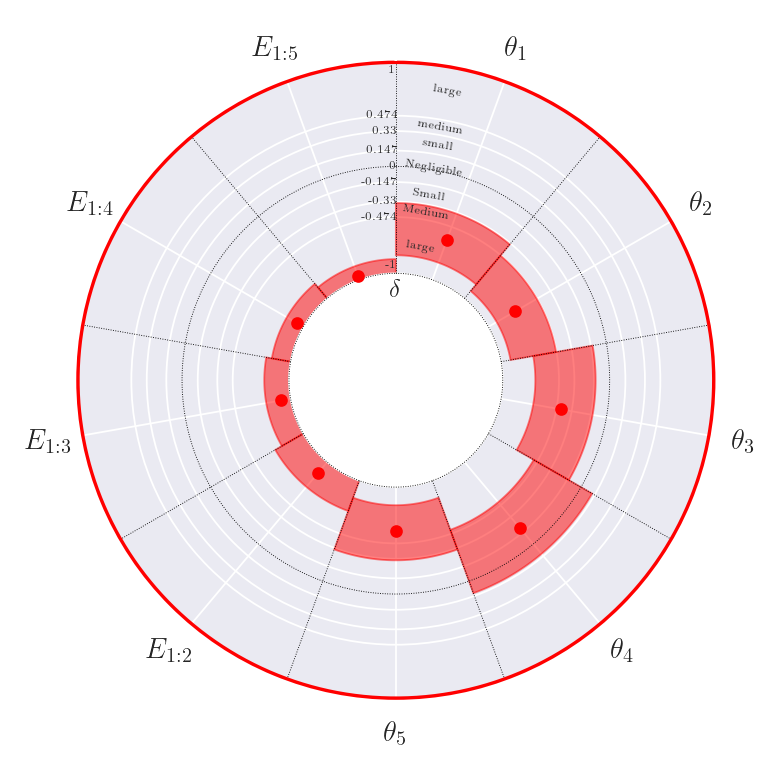}\label{fig:lrcliffgecr_MEC}}%
		\hfill
		\subfloat[GFE]{\includegraphics[width=0.24\textwidth]{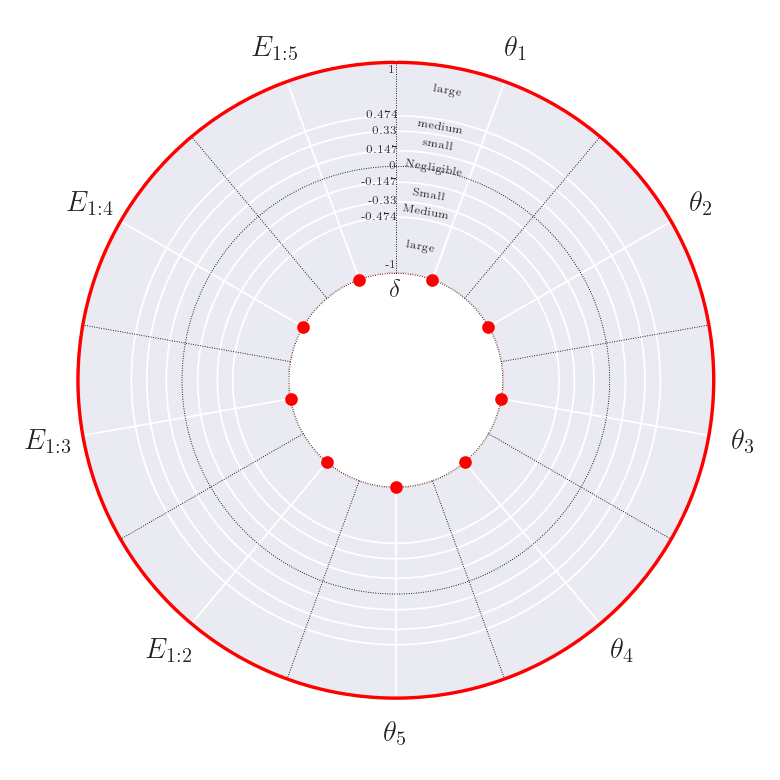}\label{fig:lrcliffgfe_MEC}}
		
		\subfloat[GSAD]{\includegraphics[width=0.24\textwidth]{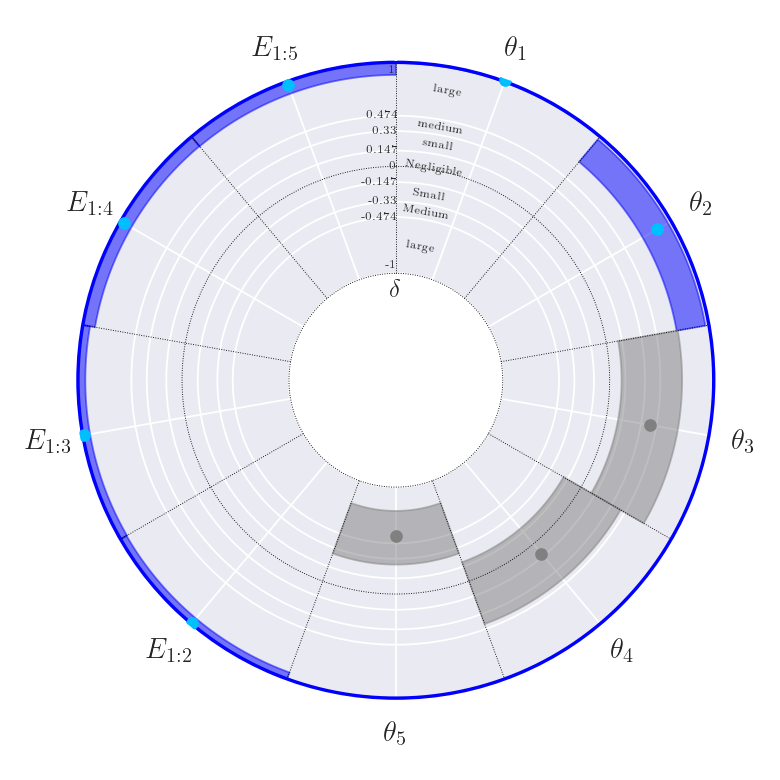}\label{fig:lrcliffgsad_MEC}}%
		\hfill
		\subfloat[HAPT]{\includegraphics[width=0.24\textwidth]{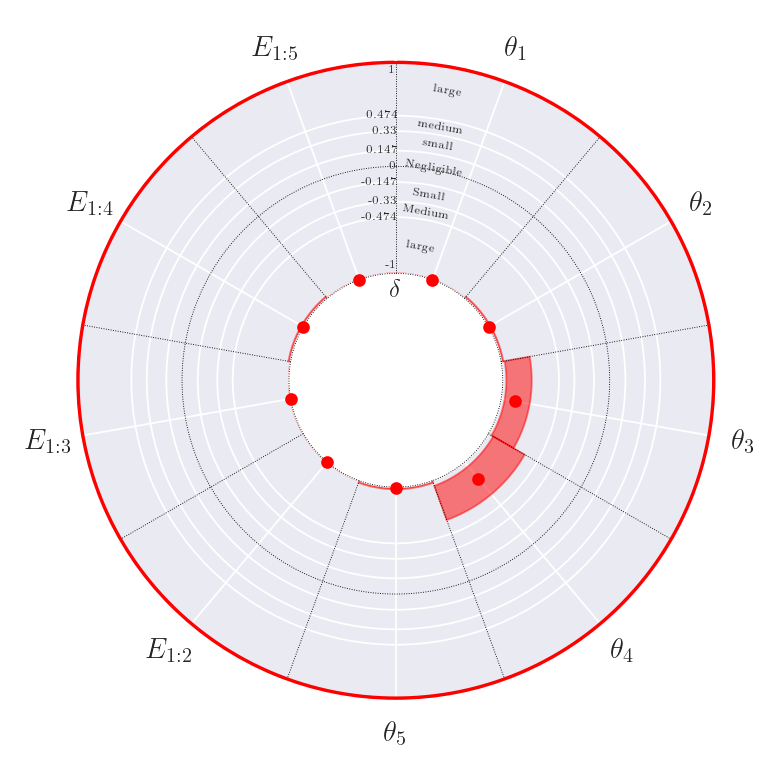}\label{fig:lrcliffhapt_MEC}}%
		\hfill
		\subfloat[ISOLET]{\includegraphics[width=0.24\textwidth]{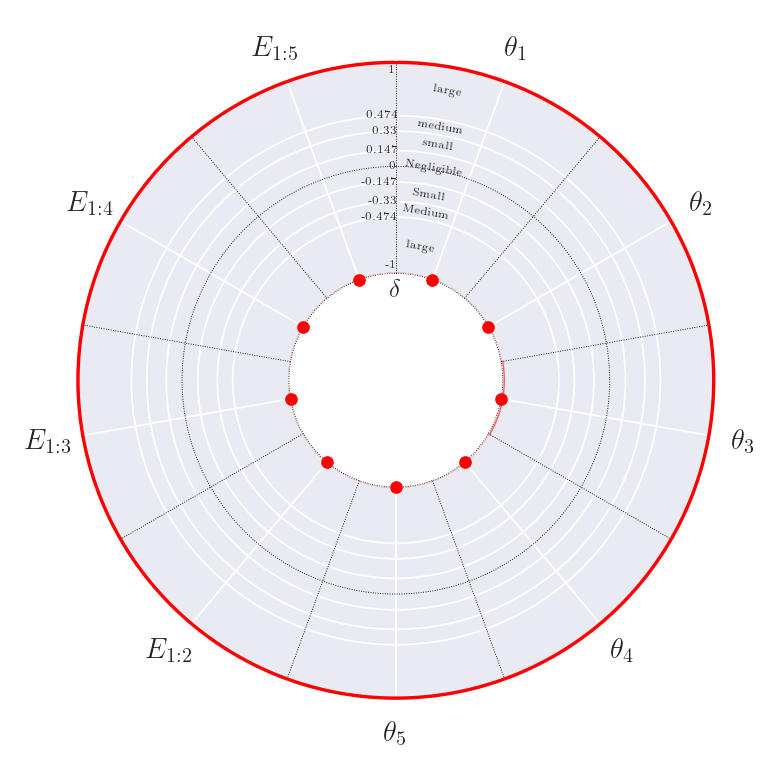}\label{fig:lrcliffisolet_MEC}}%
		\hfill
		\subfloat[PD]{\includegraphics[width=0.24\textwidth]{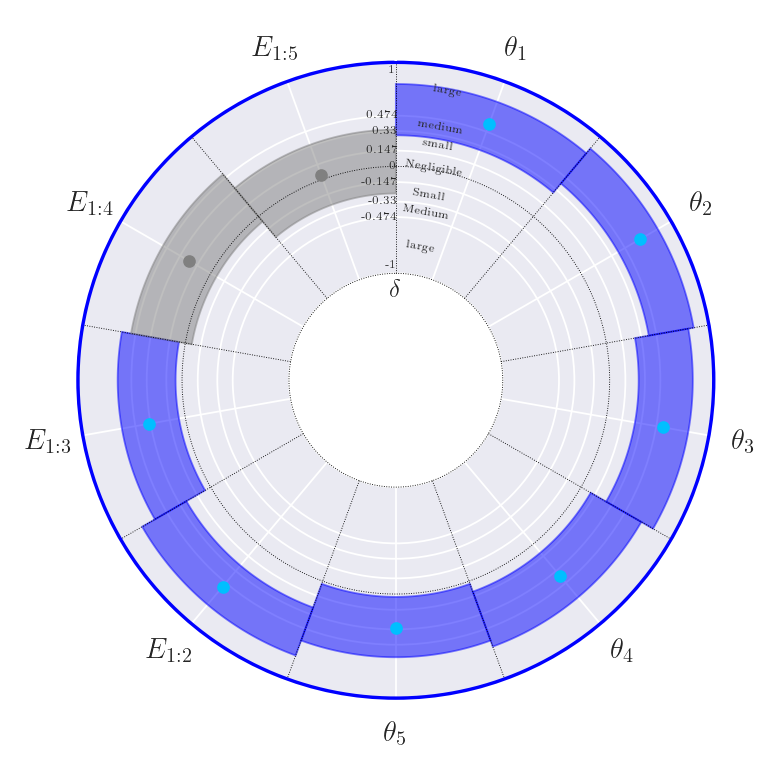}\label{fig:lrcliffpd_MEC}}
		\caption[The Cliff's $\delta$ effect size measure and its 95\% confidence intervals for the MEC values obtained from 30 Logistic Regression runs.]{Effect size analysis of test data MEC across 30 Logistic Regression runs using Cliff's $\delta$. Each point represents the actual value obtained, with segments denoting 95\% confidence intervals based on 10,000 bootstrap resamplings. The outer ring color visualizes the statistical significance: grey illustrates no significant difference (adjusted Friedman's P-value$>0.05$), while color indicates significant differences; blue indicates at least one view and/or ensemble outperforms the benchmark (adjusted Conover's p-value$ < 0.05$, Cliff's $\delta > 0$), and red signifies all views and ensembles underperform relative to the benchmark (adjusted Conover's p-value$ < 0.05$, Cliff's $\delta < 0$). Segment colors show performance difference against the benchmark: grey for no significant difference (adjusted Conover's p-value$  > 0.05$), blue for better performance (Cliff's $\delta > 0$), and red for worse performance (Cliff's $\delta < 0$).}
		
		\label{fig:lrcliff_MEC}
	\end{figure*}
	
	\begin{table*}
		\centering
		\caption[The results of Friedman and Conover tests and Cliff's $\delta$ analysis for the obtained MEC values from 30 Logistic Regression runs.]{Statistical comparison of MEC results for testing data obtained from Logistic Regression runs. W, T, and L denote win, tie, and loss based on adjusted Friedman and Conover's p-values. Effect sizes are calculated using Cliff's Delta method and are categorized as negligible, small, medium, or large.}
		\label{tab:lrmec}
		\resizebox{\linewidth}{!}{%
			\begin{tabular}{c|ccccccccc}
				\hline
				\multicolumn{10}{c}{Logistic Regression's MEC}\\
				\hline
				Dataset & $\theta_1$ & $\theta_2$ & $\theta_3$ & $\theta_4$ & $\theta_5$ & $E_{1:2}$ & $E_{1:3}$ & $E_{1:4}$ & $E_{1:5}$ \\
				\hline
				APSF  & L (large) & L (large) & L (large) & L (large) & L (large) & L (large) & L (large) & L (large) & L (large) \\
				ARWPM  & W (large) & W (large) & W (large) & W (large) & T (medium) & W (large) & W (large) & T (large) & T (negligible) \\
				GECR  & L (large) & L (large) & L (medium) & L (small) & L (large) & L (large) & L (large) & L (large) & L (large) \\
				GFE  & L (large) & L (large) & L (large) & L (large) & L (large) & L (large) & L (large) & L (large) & L (large) \\
				GSAD  & W (large) & W (large) & T (medium) & T (negligible) & T (large) & W (large) & W (large) & W (large) & W (large) \\
				HAPT  & L (large) & L (large) & L (large) & L (large) & L (large) & L (large) & L (large) & L (large) & L (large) \\
				ISOLET  & L (large) & L (large) & L (large) & L (large) & L (large) & L (large) & L (large) & L (large) & L (large) \\
				PD  & W (large) & W (large) & W (large) & W (medium) & W (small) & W (large) & W (medium) & T (small) & T (negligible) \\
				\hline
				W - T - L  & 3 - 0 - 5 & 3 - 0 - 5 & 2 - 1 - 5 & 2 - 1 - 5 & 1 - 2 - 5 & 3 - 0 - 5 & 3 - 0 - 5 & 1 - 2 - 5 & 1 - 2 - 5 \\
				\hline
			\end{tabular}
		}
	\end{table*}
	\FloatBarrier
	
	\begin{figure*}[t] 
		\centering
		\subfloat[APSF]{\includegraphics[width=0.24\textwidth]{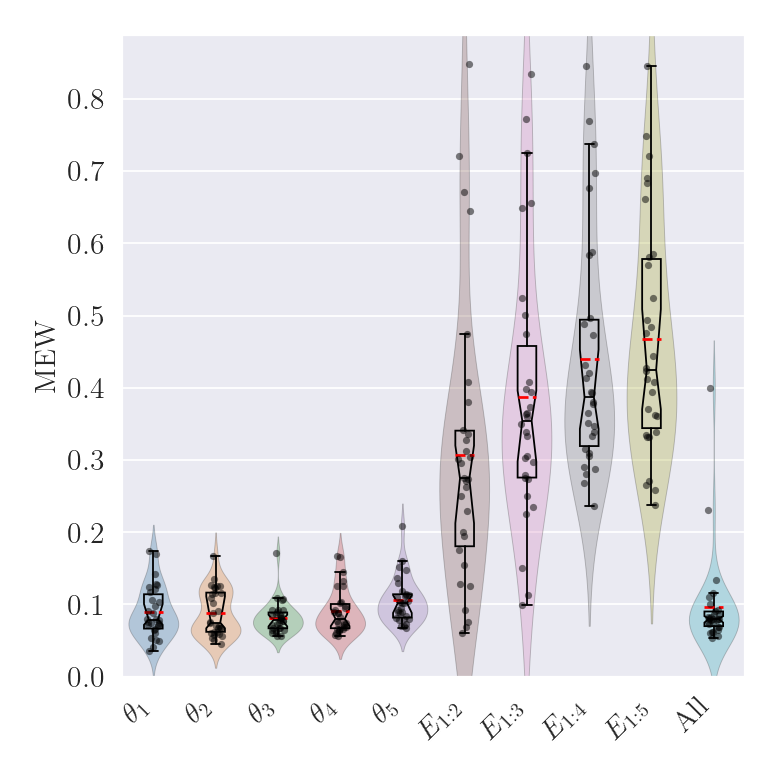}\label{fig:lrapsf_MEW}}%
		\hfill
		\subfloat[ARWPM]{\includegraphics[width=0.24\textwidth]{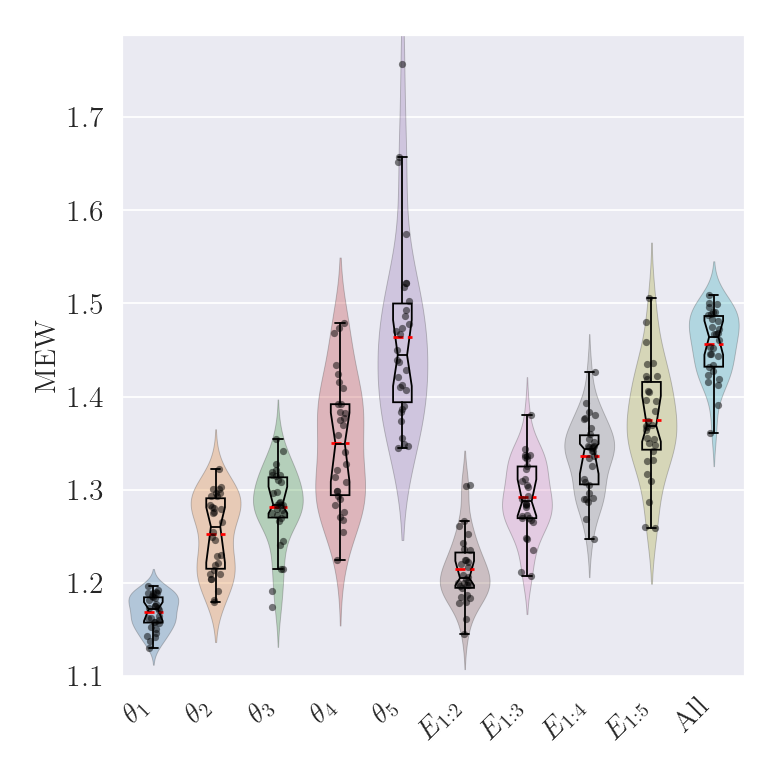}\label{fig:lrarwpm_MEW}}%
		\hfill
		\subfloat[GECR]{\includegraphics[width=0.24\textwidth]{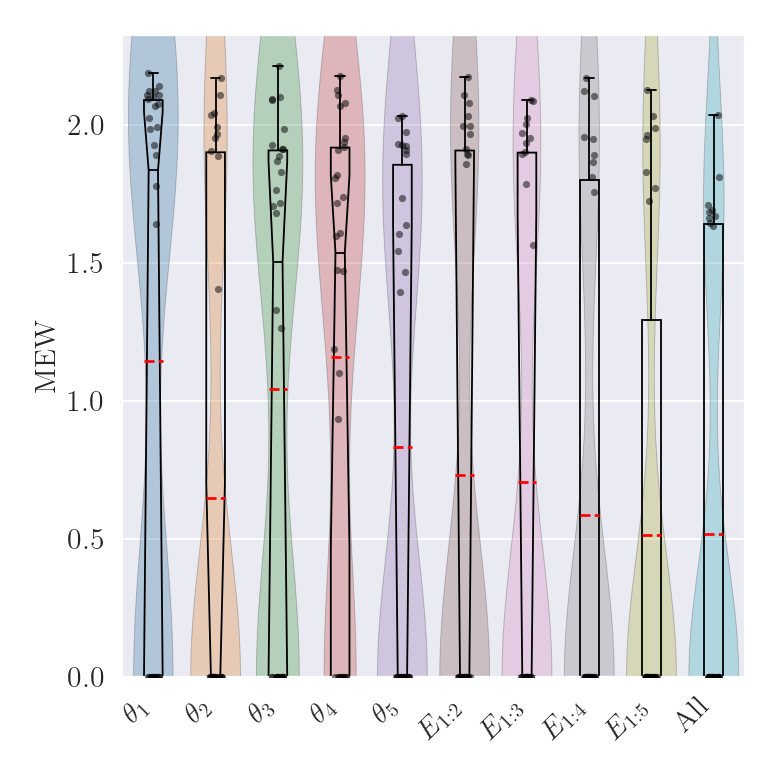}\label{fig:lrgecr_MEW}}%
		\hfill
		\subfloat[GFE]{\includegraphics[width=0.24\textwidth]{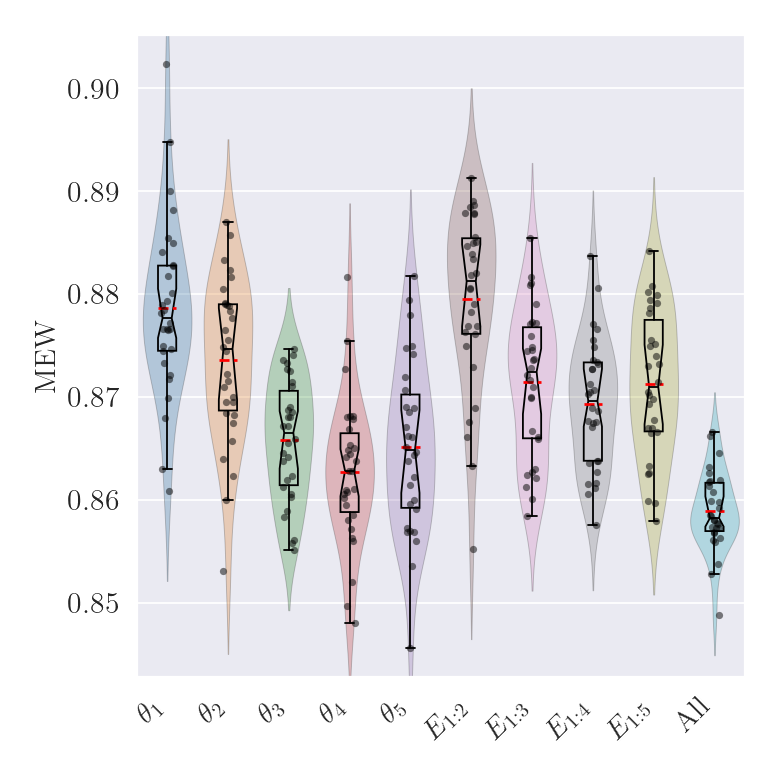}\label{fig:lrgfe_MEW}}
		
		\subfloat[GSAD]{\includegraphics[width=0.24\textwidth]{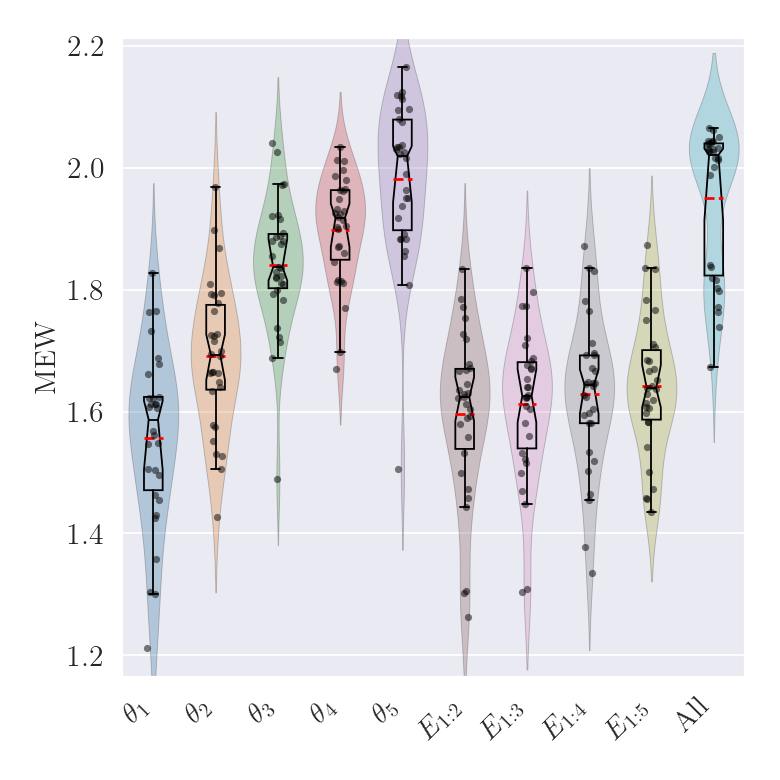}\label{fig:fpgsad_MEW}}%
		\hfill
		\subfloat[HAPT]{\includegraphics[width=0.24\textwidth]{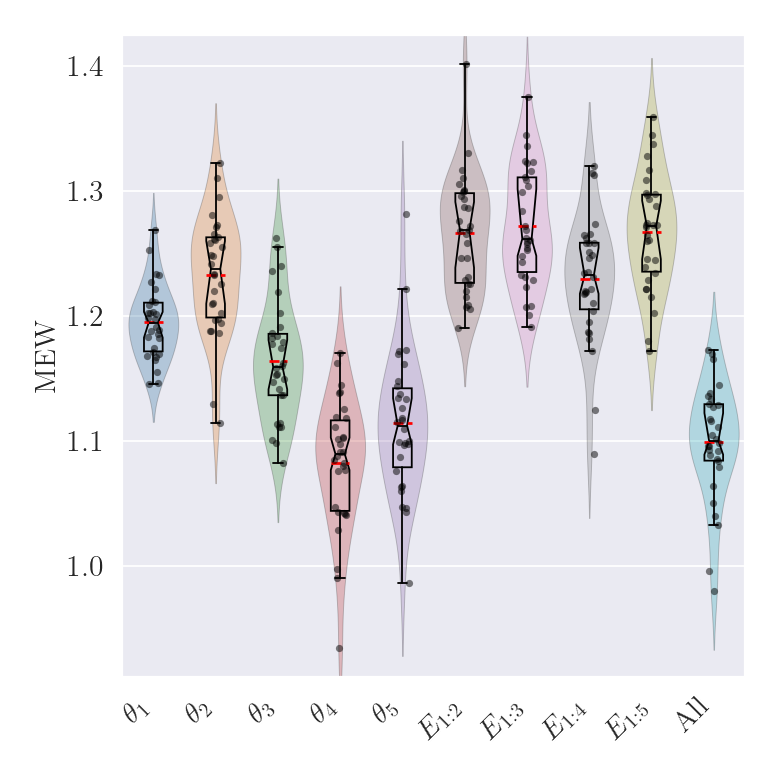}\label{fig:lrhapt_MEW}}%
		\hfill
		\subfloat[ISOLET]{\includegraphics[width=0.24\textwidth]{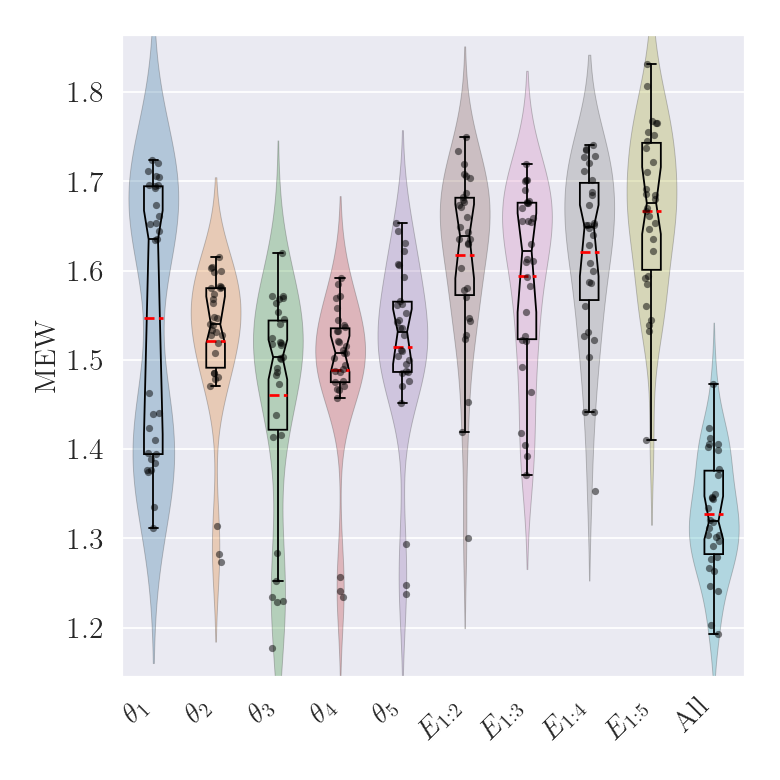}\label{fig:lrisolet_MEW}}%
		\hfill
		\subfloat[PD]{\includegraphics[width=0.24\textwidth]{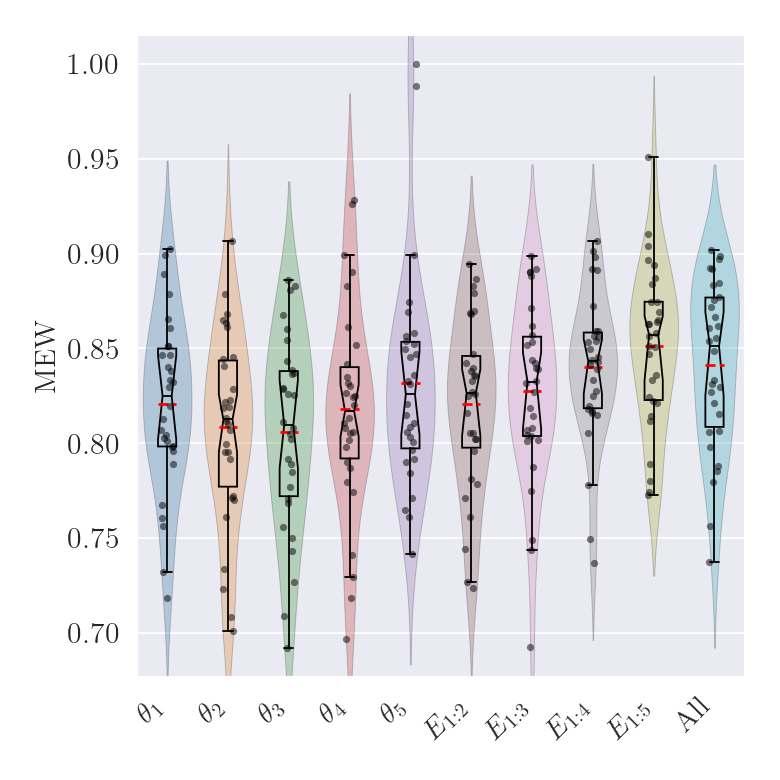}\label{fig:lrpd_MEW}}
		\caption[The distribution of the obtained MEW values for 30 Logistic Regression runs.]{The raincloud plot of MEW results obtained from 30 Logistic Regression runs.}
		
		\label{fig:lr_MEW}
	\end{figure*}
	
	\begin{figure*}[t] 
		\centering
		\subfloat[APSF]{\includegraphics[width=0.24\textwidth]{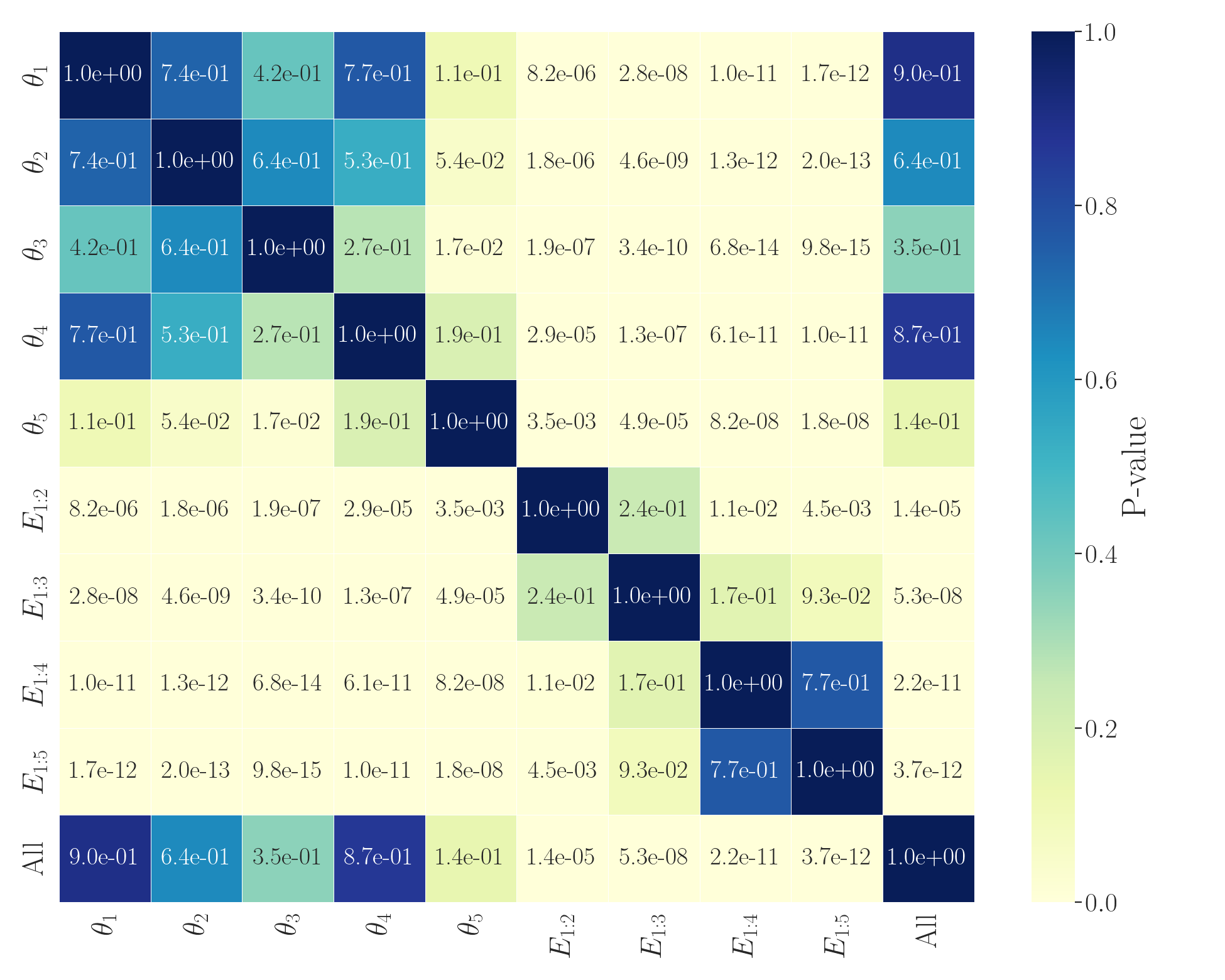}\label{fig:lrnemapsf_MEW}}%
		\hfill
		\subfloat[ARWPM]{\includegraphics[width=0.24\textwidth]{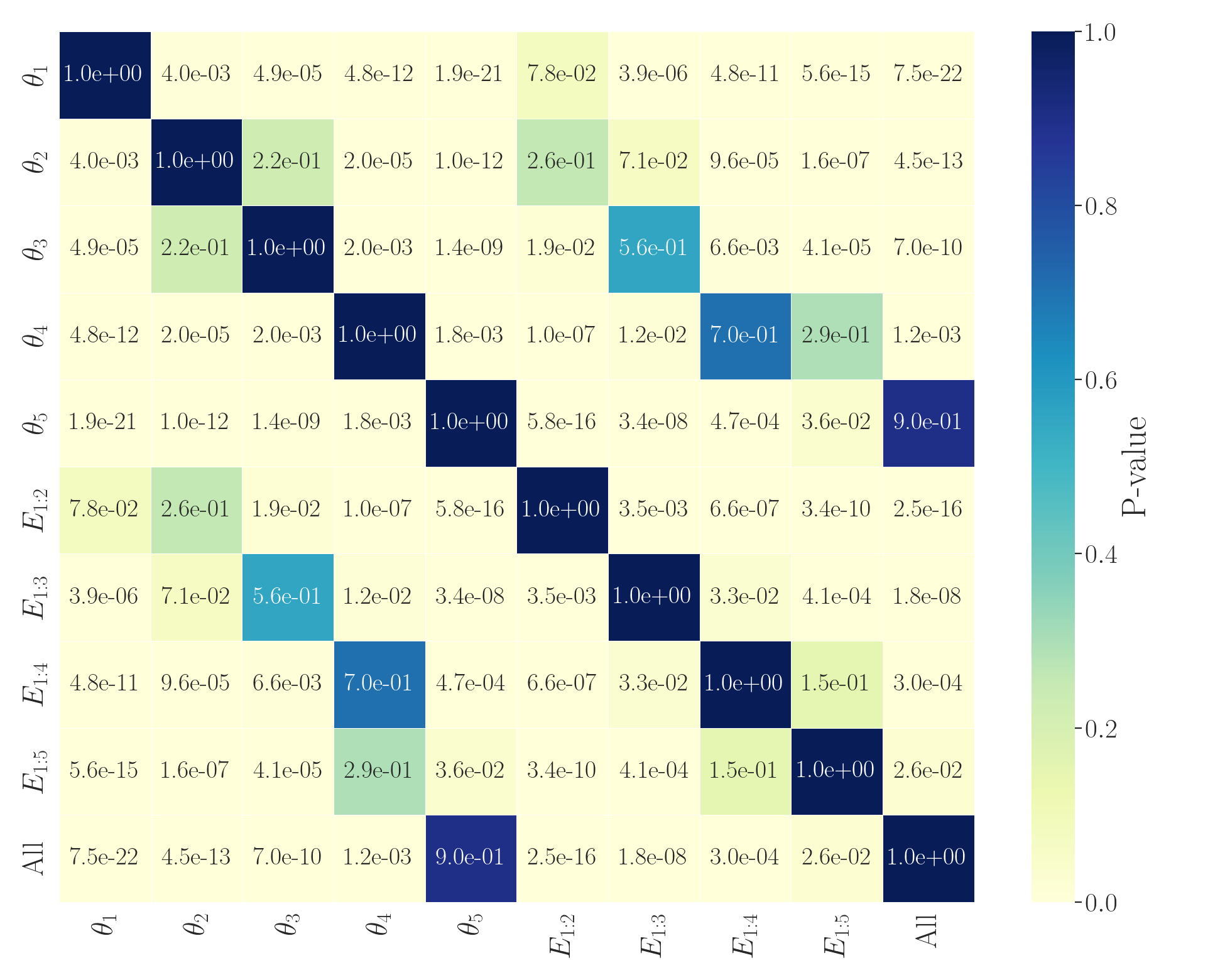}\label{fig:lrnemarwpm_MEW}}%
		\hfill
		\subfloat[GECR]{\includegraphics[width=0.24\textwidth]{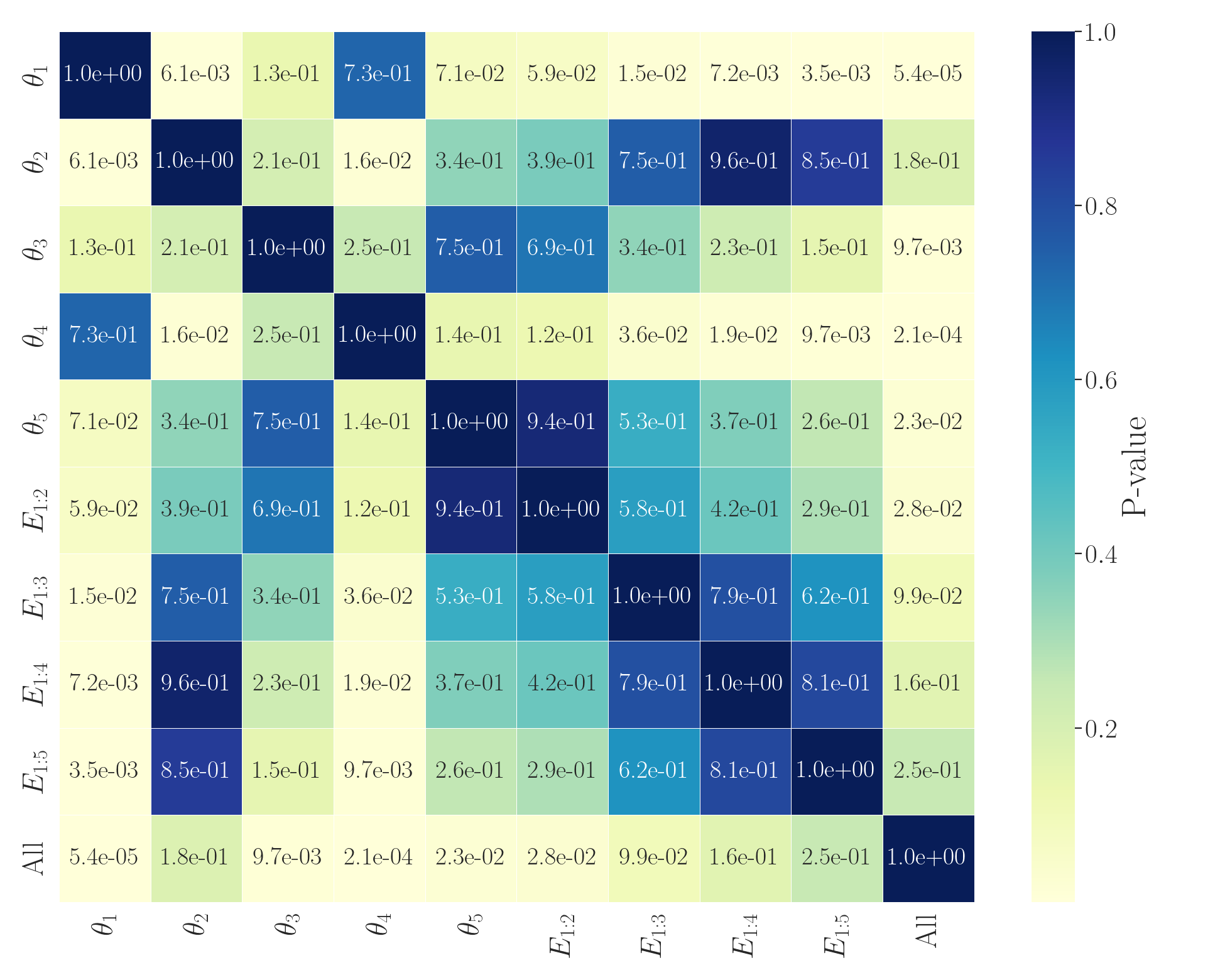}\label{fig:lrnemgecr_MEW}}%
		\hfill
		\subfloat[GFE]{\includegraphics[width=0.24\textwidth]{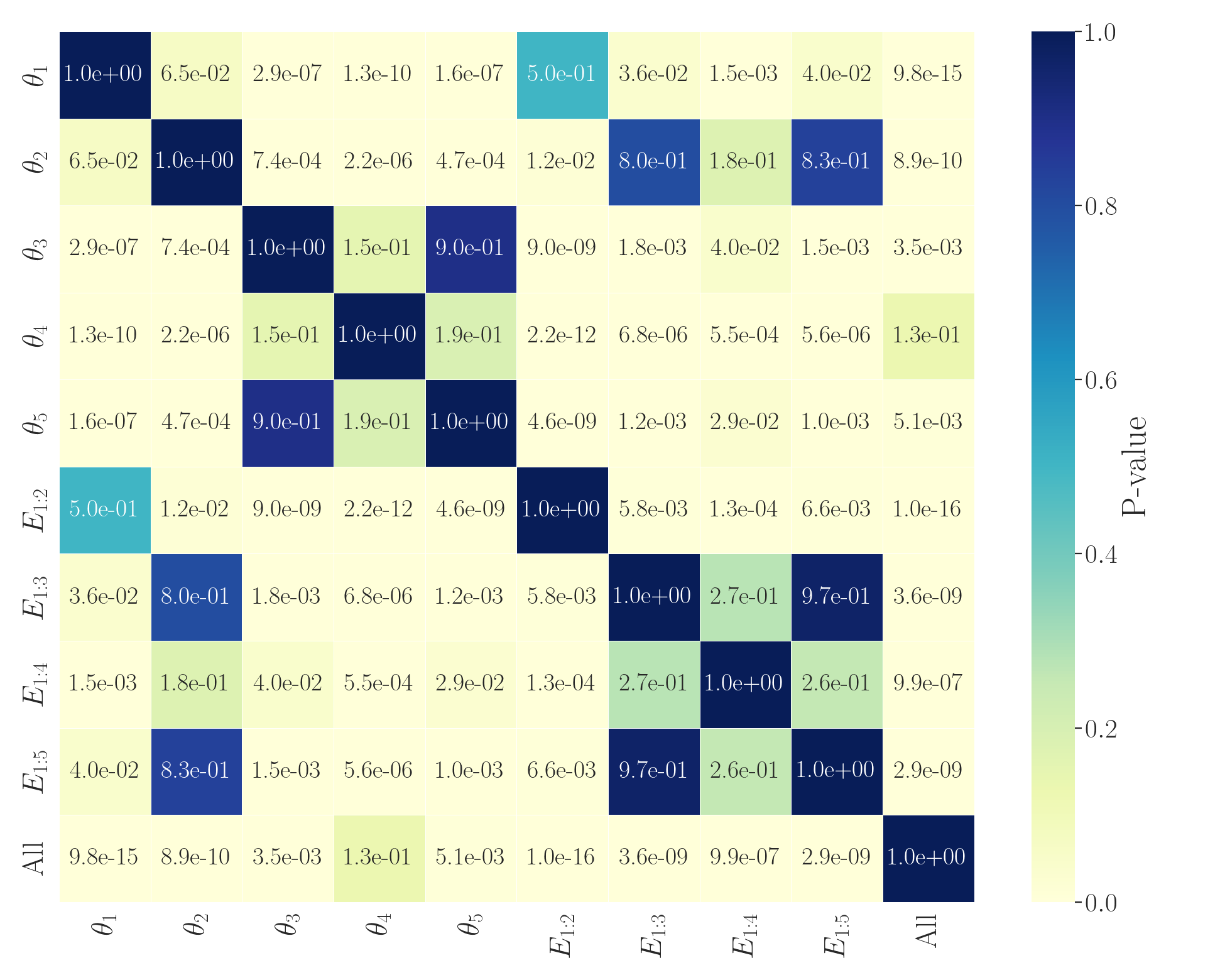}\label{fig:lrnemgfe_MEW}}
		
		\subfloat[GSAD]{\includegraphics[width=0.24\textwidth]{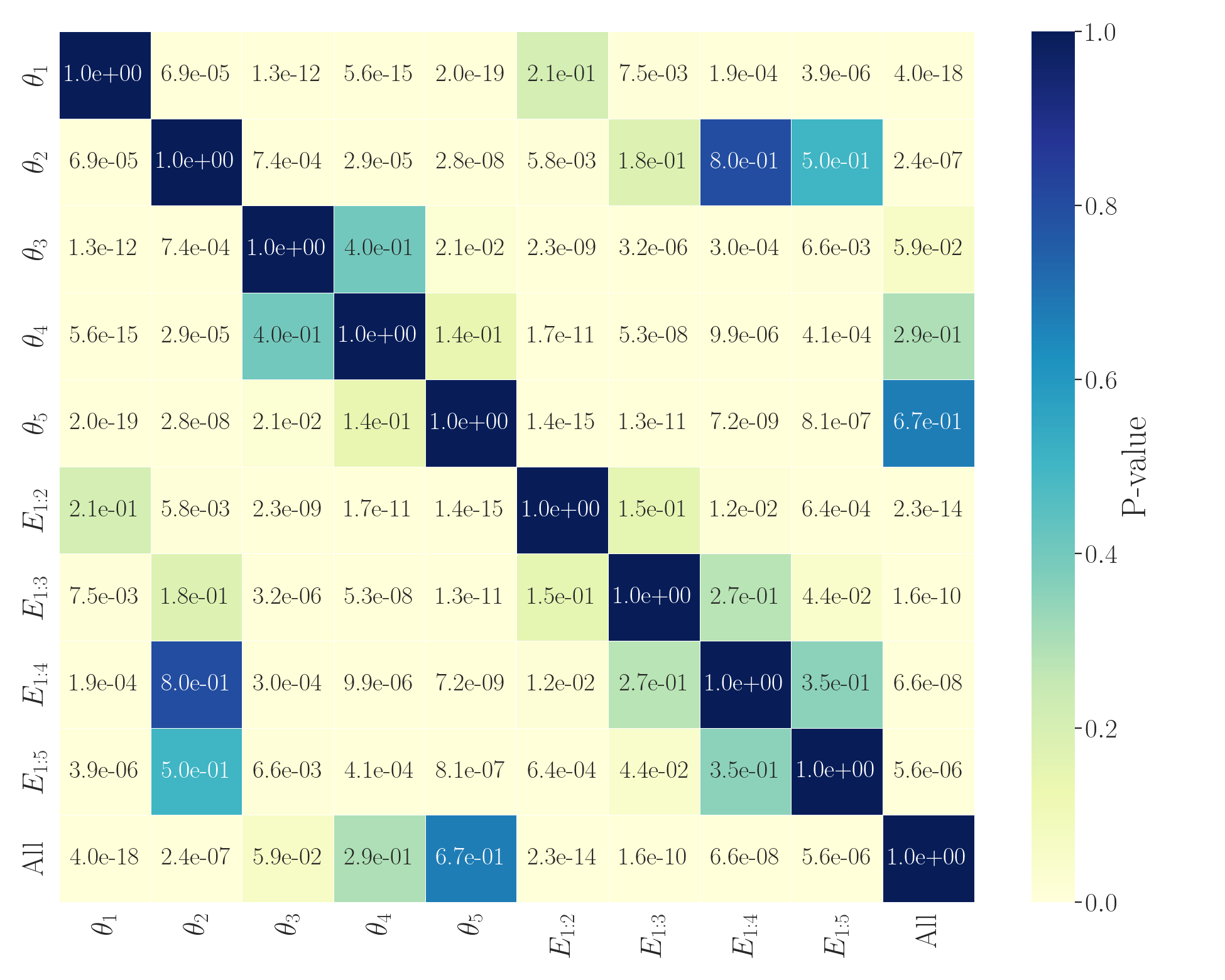}\label{fig:lrnemgsad_MEW}}%
		\hfill
		\subfloat[HAPT]{\includegraphics[width=0.24\textwidth]{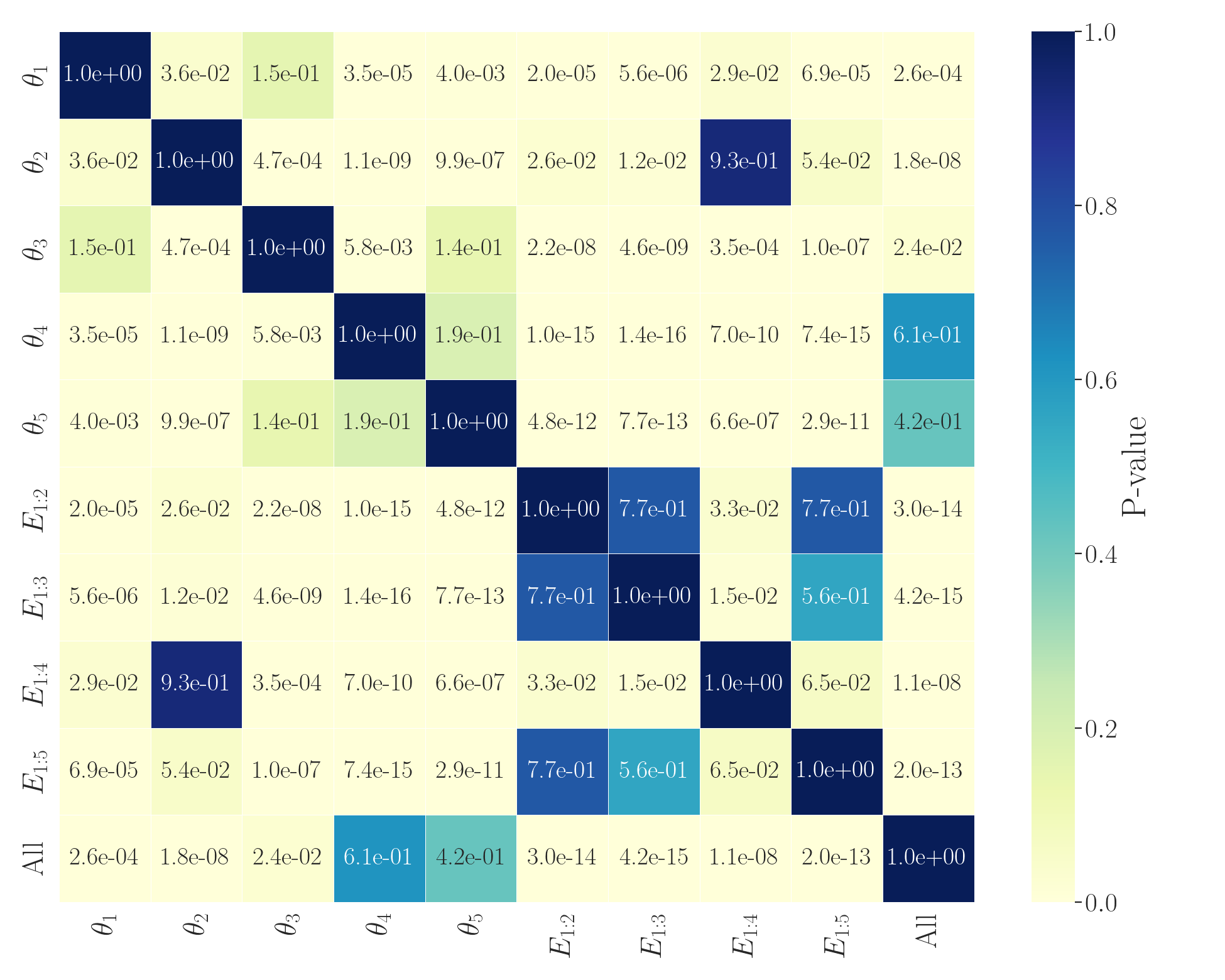}\label{fig:lrnemhapt_MEW}}%
		\hfill
		\subfloat[ISOLET]{\includegraphics[width=0.24\textwidth]{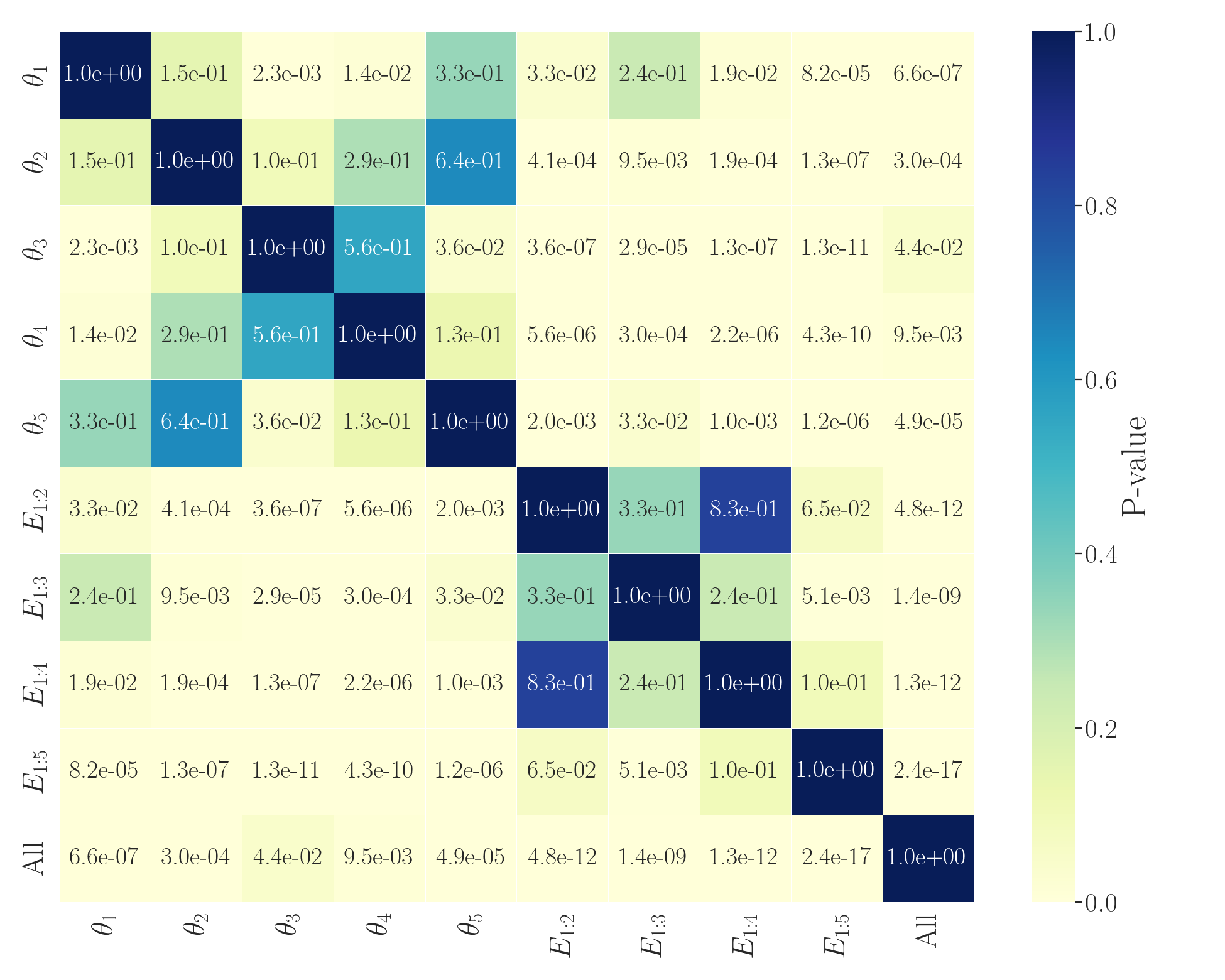}\label{fig:lrnemisolet_MEW}}%
		\hfill
		\subfloat[PD]{\includegraphics[width=0.24\textwidth]{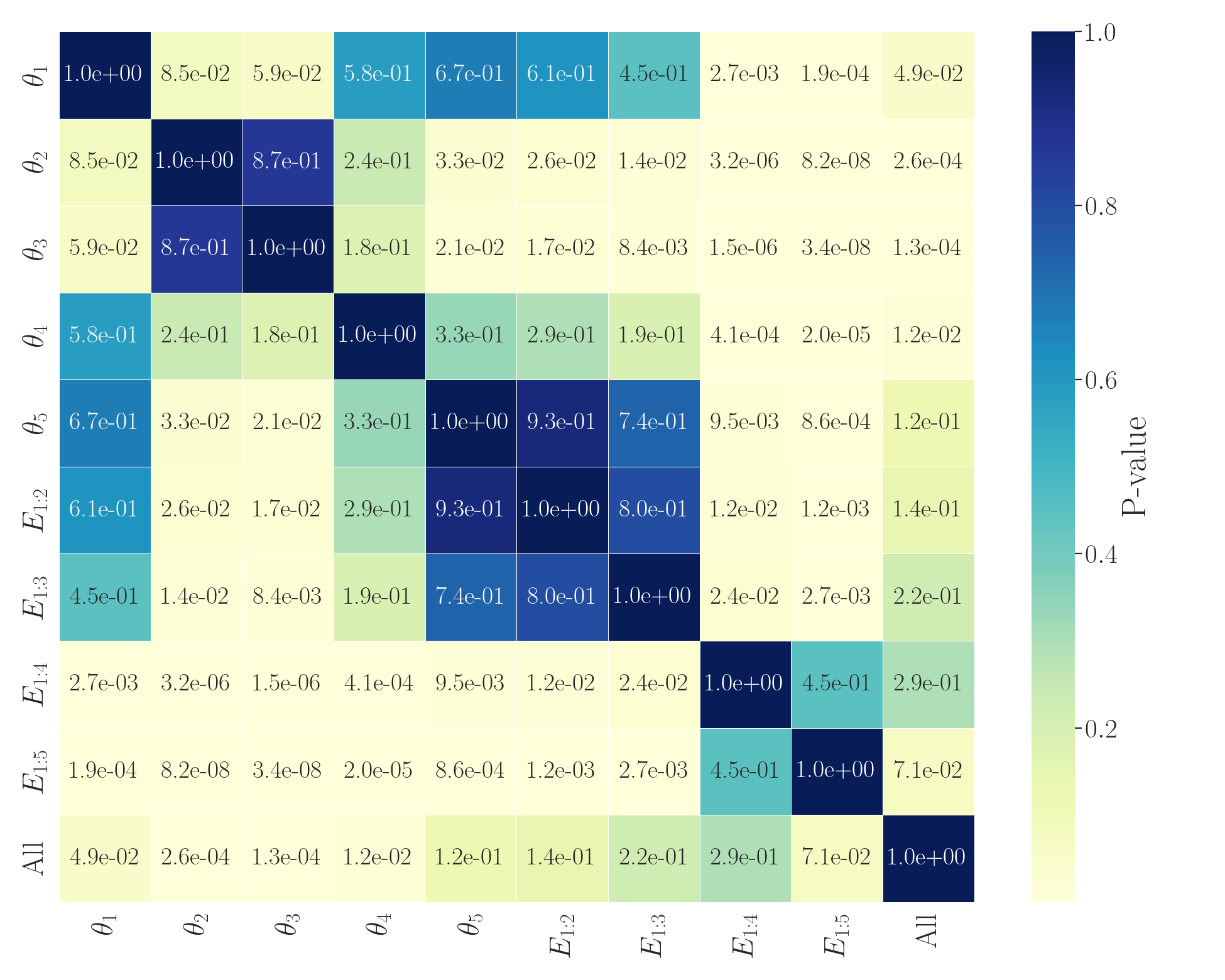}\label{fig:lrnempd_MEW}}
		\caption[The adjusted Conover's P-values for the obtained MEW values in 30 Logistic Regression runs.]{The results of the Conover post-hoc test on testing data’s MEW obtained from 30 Logistic Regression runs.}
		
		\label{fig:lrnem_MEW}
	\end{figure*}
	\FloatBarrier
	
	\begin{figure*}[htbp] 
		\centering
		\subfloat[APSF]{\includegraphics[width=0.24\textwidth]{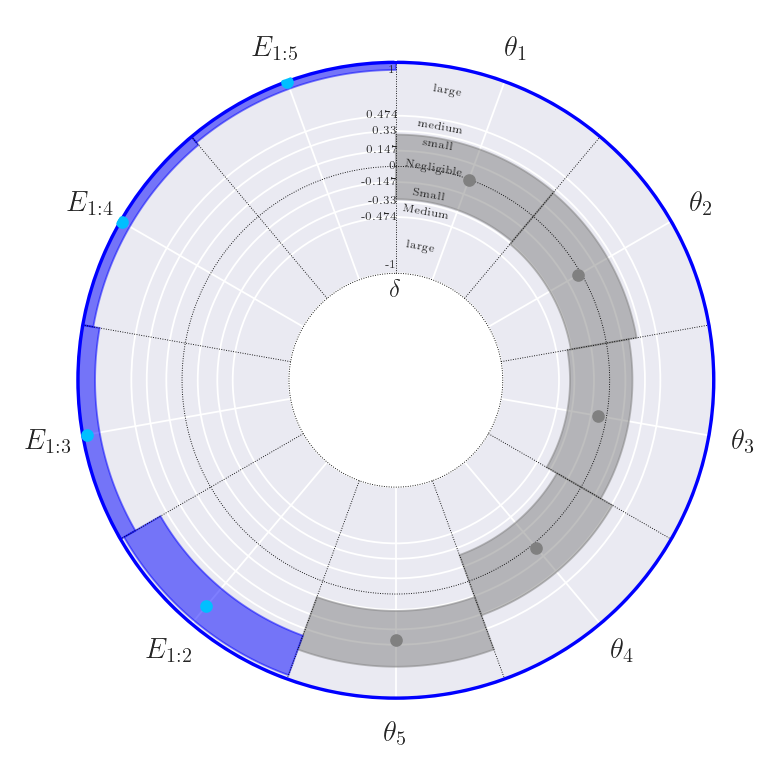}\label{fig:lrcliffapsf_MEW}}%
		\hfill
		\subfloat[ARWPM]{\includegraphics[width=0.24\textwidth]{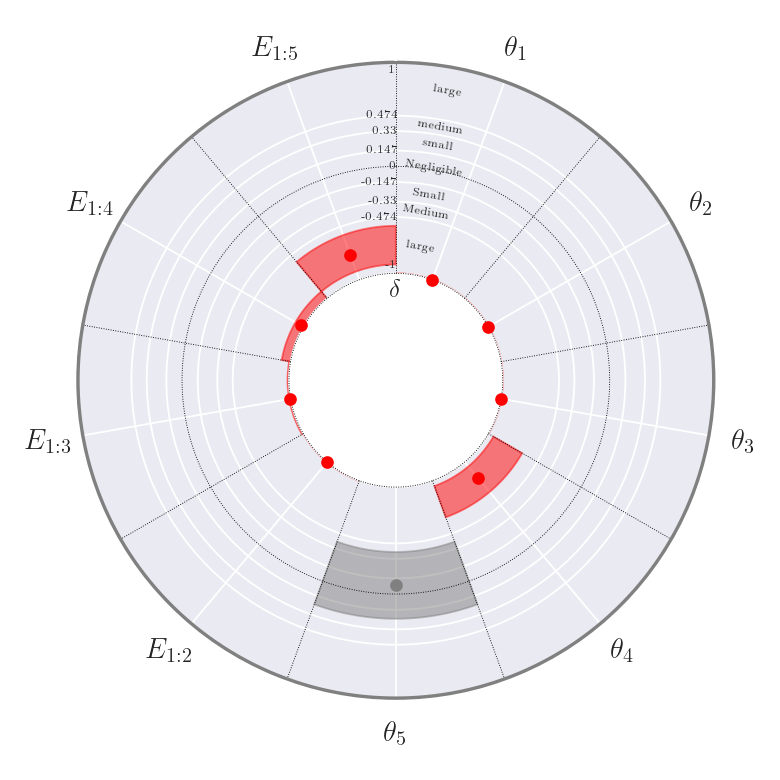}\label{fig:lrcliffarwpm_MEW}}%
		\hfill
		\subfloat[GECR]{\includegraphics[width=0.24\textwidth]{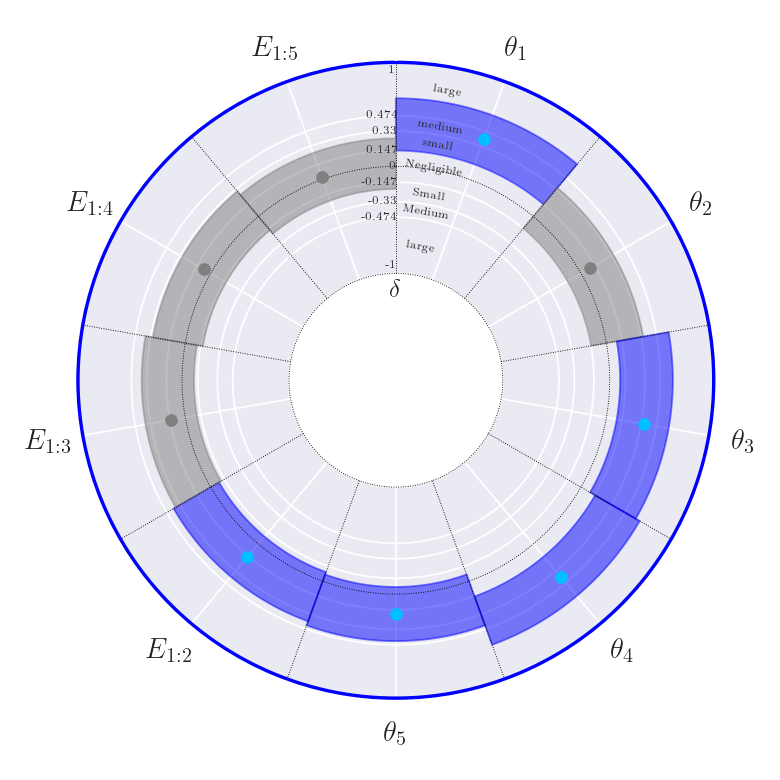}\label{fig:lrcliffgecr_MEW}}%
		\hfill
		\subfloat[GFE]{\includegraphics[width=0.24\textwidth]{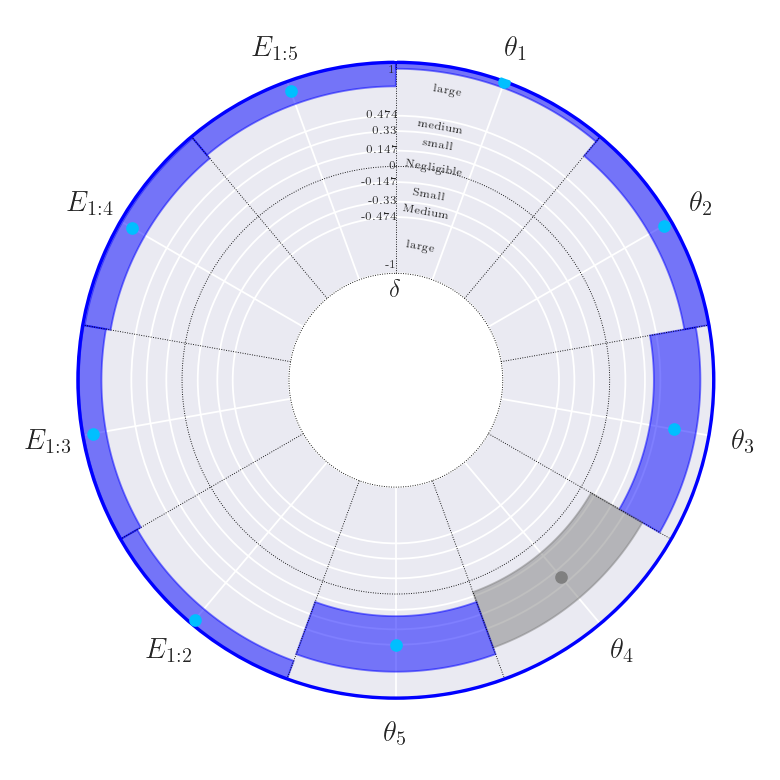}\label{fig:lrcliffgfe_MEW}}
		
		\subfloat[GSAD]{\includegraphics[width=0.24\textwidth]{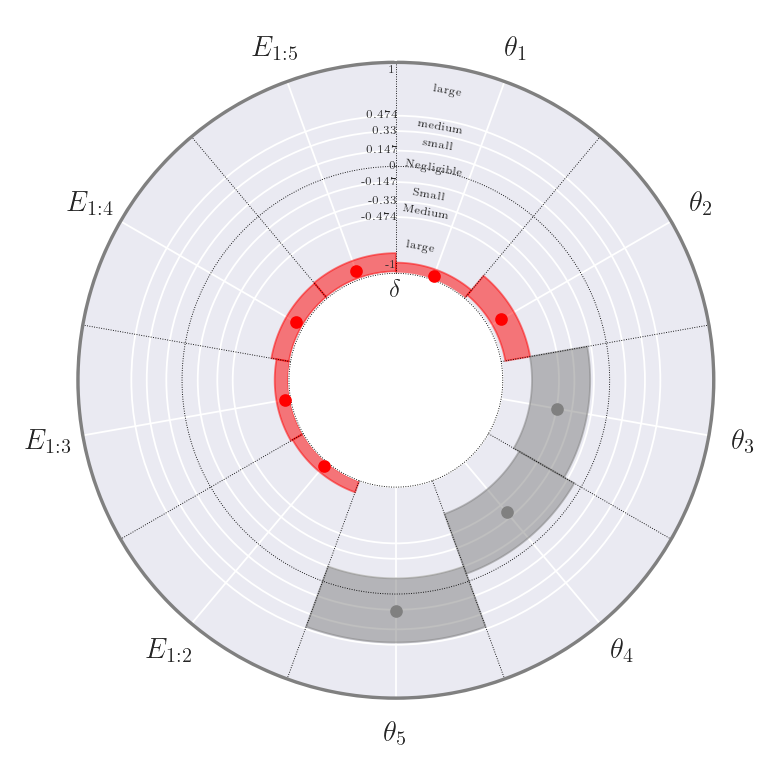}\label{fig:lrcliffgsad_MEW}}%
		\hfill
		\subfloat[HAPT]{\includegraphics[width=0.24\textwidth]{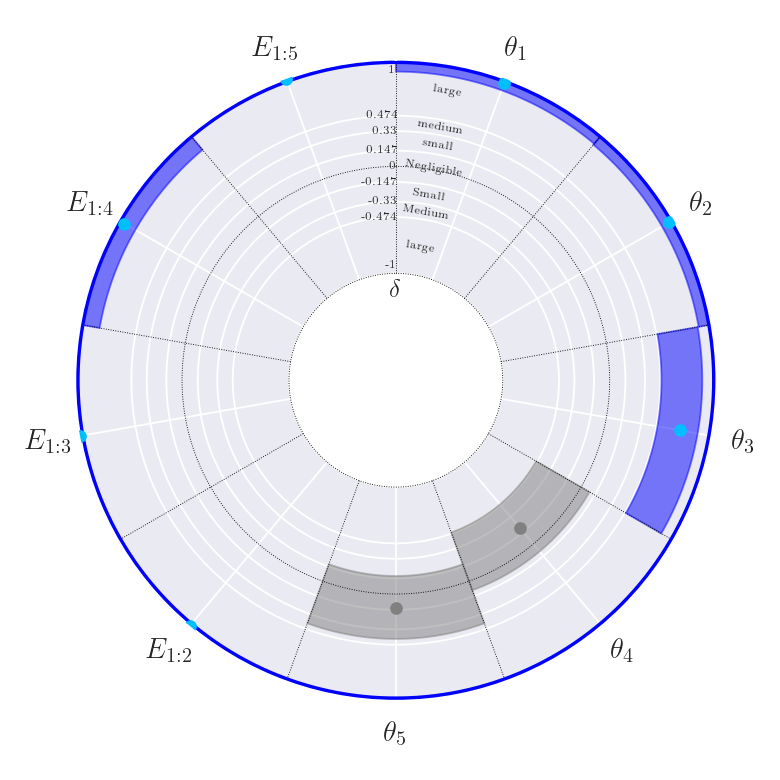}\label{fig:lrcliffhapt_MEW}}%
		\hfill
		\subfloat[ISOLET]{\includegraphics[width=0.24\textwidth]{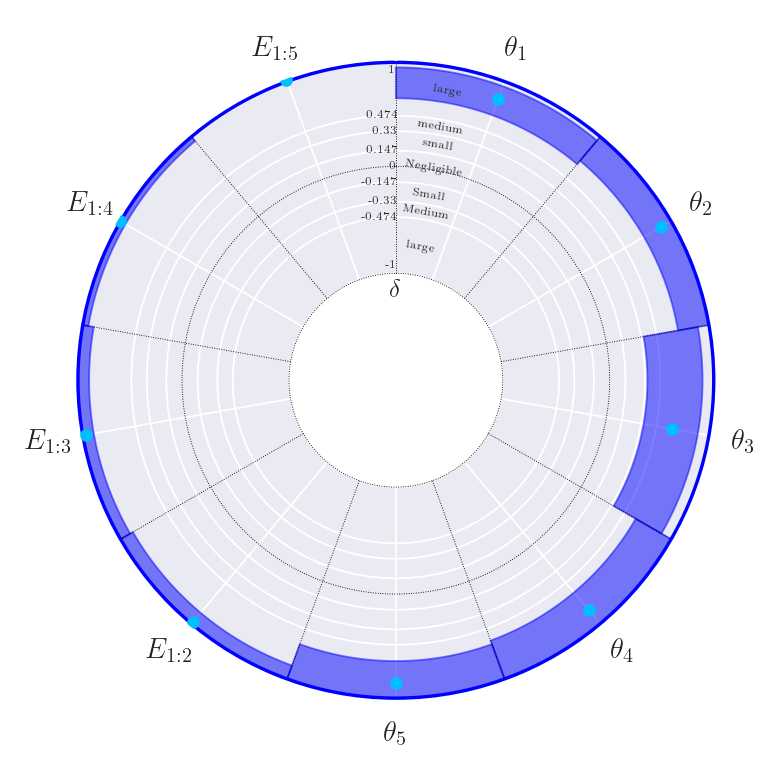}\label{fig:lrcliffisolet_MEW}}%
		\hfill
		\subfloat[PD]{\includegraphics[width=0.24\textwidth]{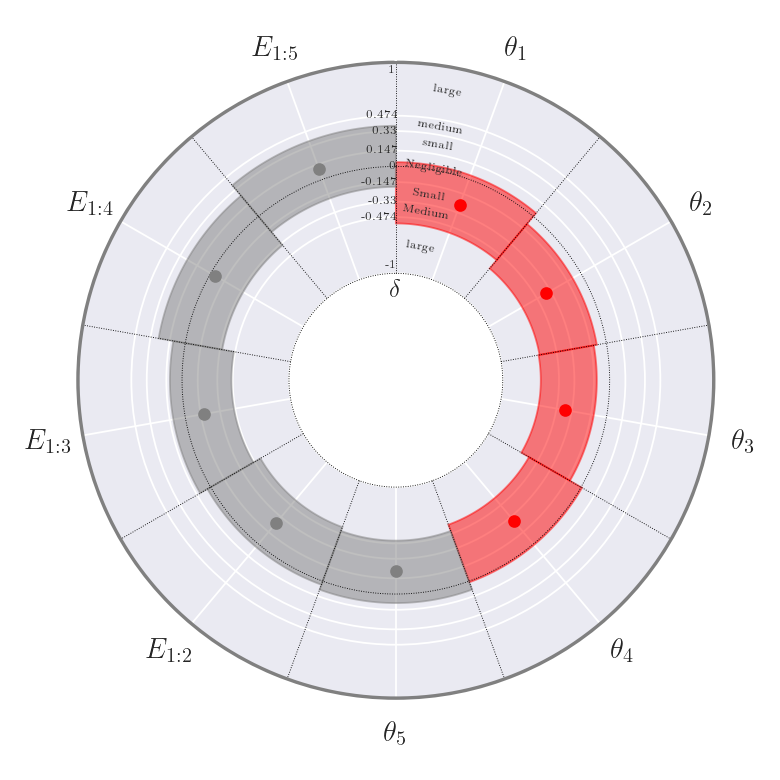}\label{fig:lrcliffpd_MEW}}
		\caption[The Cliff's $\delta$ effect size measure and its 95\% confidence intervals for the MEW values obtained from 30 Logistic Regression runs.]{Effect size analysis of test data MEW across 30 Logistic Regression runs using Cliff's $\delta$. Each point represents the actual value obtained, with segments denoting 95\% confidence intervals based on 10,000 bootstrap resamplings. The outer ring color visualizes the statistical significance: grey illustrates no significant difference (adjusted Friedman's P-value$>0.05$), while color indicates significant differences; blue indicates at least one view and/or ensemble outperforms the benchmark (adjusted Conover's p-value$ < 0.05$, Cliff's $\delta > 0$), and red signifies all views and ensembles underperform relative to the benchmark (adjusted Conover's p-value$ < 0.05$, Cliff's $\delta < 0$). Segment colors show performance difference against the benchmark: grey for no significant difference (adjusted Conover's p-value$  > 0.05$), blue for better performance (Cliff's $\delta > 0$), and red for worse performance (Cliff's $\delta < 0$).}
		
		\label{fig:lrcliff_MEW}
	\end{figure*}
	
	\begin{table*}[htbp]
		\centering
		\caption[The results of Friedman and Conover tests and Cliff's $\delta$ analysis for the obtained MEW values from 30 Logistic Regression runs.]{Statistical comparison of MEW results for testing data obtained from Logistic Regression runs. W, T, and L denote win, tie, and loss based on adjusted Friedman and Conover's p-values. Effect sizes are calculated using Cliff's Delta method and are categorized as negligible, small, medium, or large.}
		\label{tab:lrmew}
		\resizebox{\linewidth}{!}{%
			\begin{tabular}{c|ccccccccc}
				\hline
				\multicolumn{10}{c}{Logistic Regression's MEW}\\
				\hline
				Dataset & $\theta_1$ & $\theta_2$ & $\theta_3$ & $\theta_4$ & $\theta_5$ & $E_{1:2}$ & $E_{1:3}$ & $E_{1:4}$ & $E_{1:5}$ \\
				\hline
				APSF  & T (negligible) & T (negligible) & T (negligible) & T (negligible) & T (medium) & W (large) & W (large) & W (large) & W (large) \\
				ARWPM  & L (large) & L (large) & L (large) & L (large) & T (negligible) & L (large) & L (large) & L (large) & L (large) \\
				GECR  & W (medium) & T (negligible) & W (medium) & W (medium) & W (small) & W (small) & T (negligible) & T (negligible) & T (negligible) \\
				GFE  & W (large) & W (large) & W (large) & T (medium) & W (large) & W (large) & W (large) & W (large) & W (large) \\
				GSAD  & L (large) & L (large) & T (medium) & T (medium) & T (small) & L (large) & L (large) & L (large) & L (large) \\
				HAPT  & W (large) & W (large) & W (large) & T (small) & T (negligible) & W (large) & W (large) & W (large) & W (large) \\
				ISOLET  & W (large) & W (large) & W (large) & W (large) & W (large) & W (large) & W (large) & W (large) & W (large) \\
				PD  & L (small) & L (medium) & L (medium) & L (small) & T (small) & T (small) & T (small) & T (negligible) & T (negligible) \\
				\hline
				W - T - L  & 4 - 1 - 3 & 3 - 2 - 3 & 4 - 2 - 2 & 2 - 4 - 2 & 3 - 5 - 0 & 5 - 1 - 2 & 4 - 2 - 2 & 4 - 2 - 2 & 4 - 2 - 2 \\
				\hline
			\end{tabular}
		}
	\end{table*}
	\FloatBarrier

	\begin{figure*}[t] 
		\centering
		\subfloat[APSF]{\includegraphics[width=0.24\textwidth]{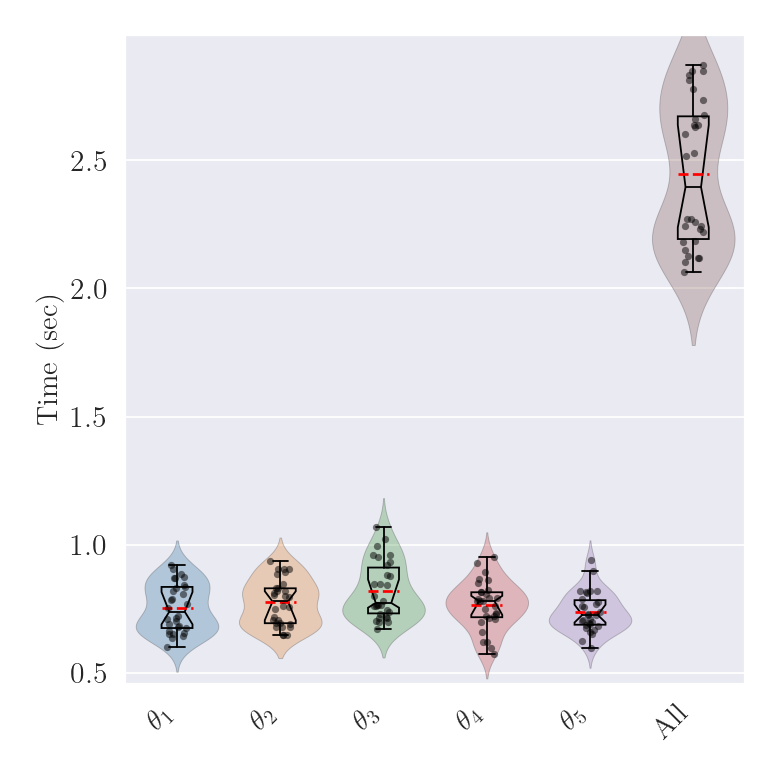}\label{fig:lrapsf_Time}}%
		\hfill
		\subfloat[ARWPM]{\includegraphics[width=0.24\textwidth]{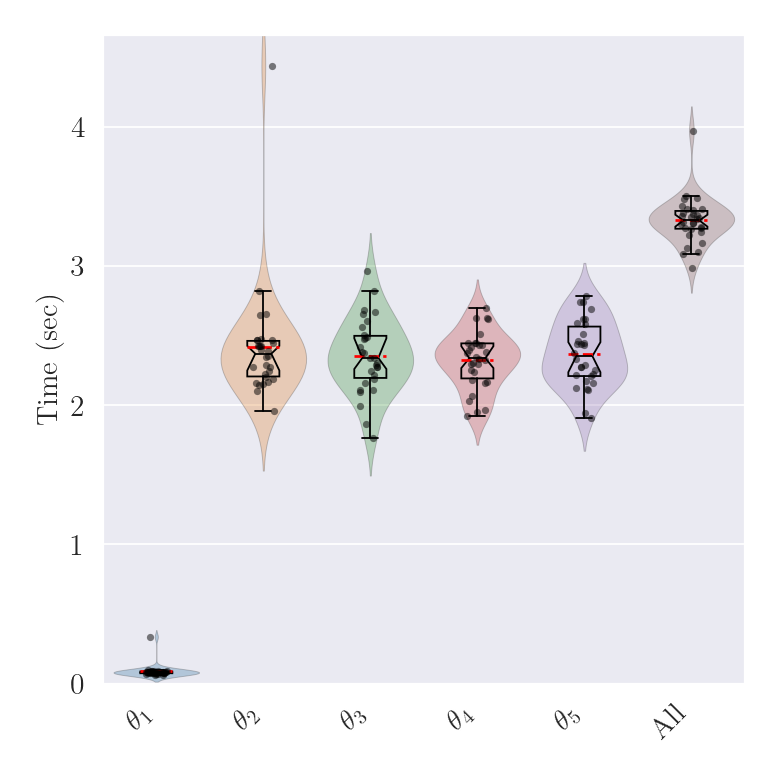}\label{fig:lrarwpm_Time}}%
		\hfill
		\subfloat[GECR]{\includegraphics[width=0.24\textwidth]{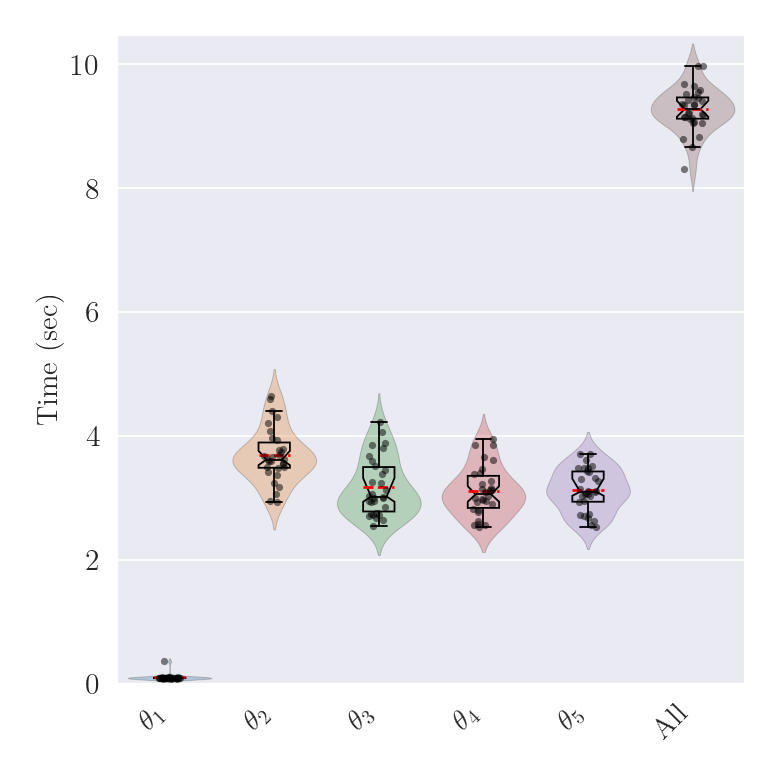}\label{fig:lrgecr_Time}}%
		\hfill
		\subfloat[GFE]{\includegraphics[width=0.24\textwidth]{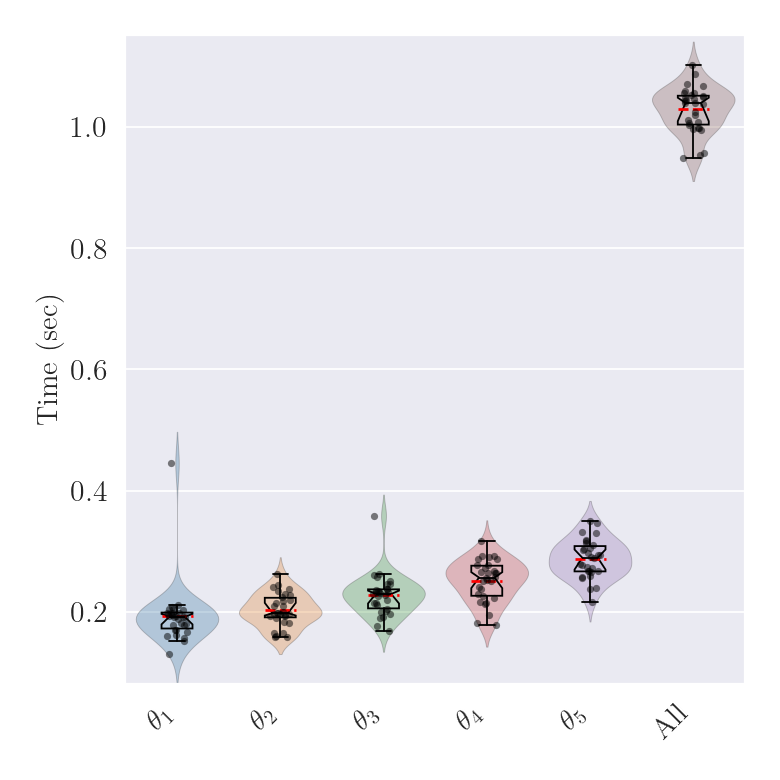}\label{fig:lrgfe_Time}}
		
		\subfloat[GSAD]{\includegraphics[width=0.24\textwidth]{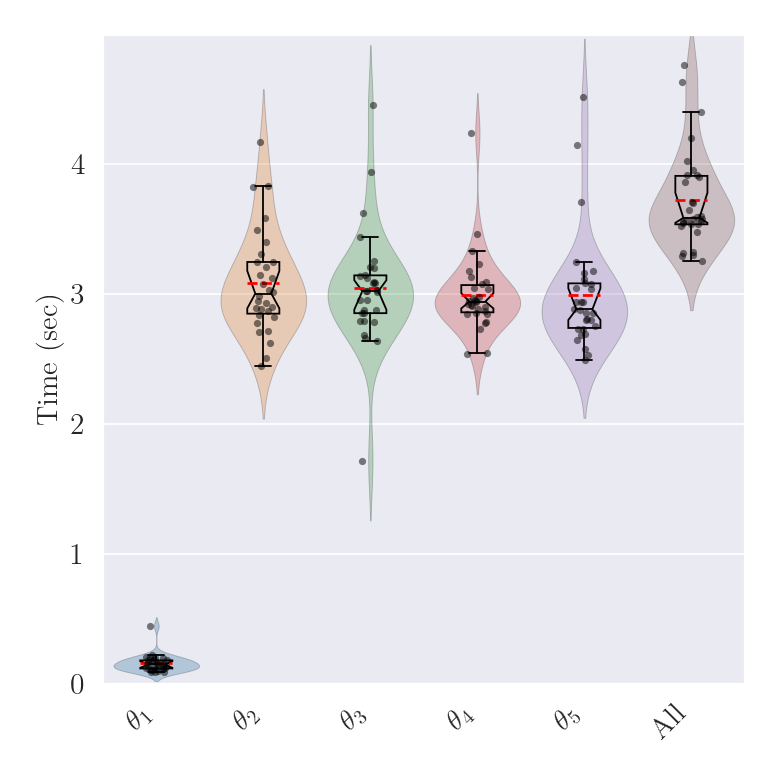}\label{fig:fpgsad_Time}}%
		\hfill
		\subfloat[HAPT]{\includegraphics[width=0.24\textwidth]{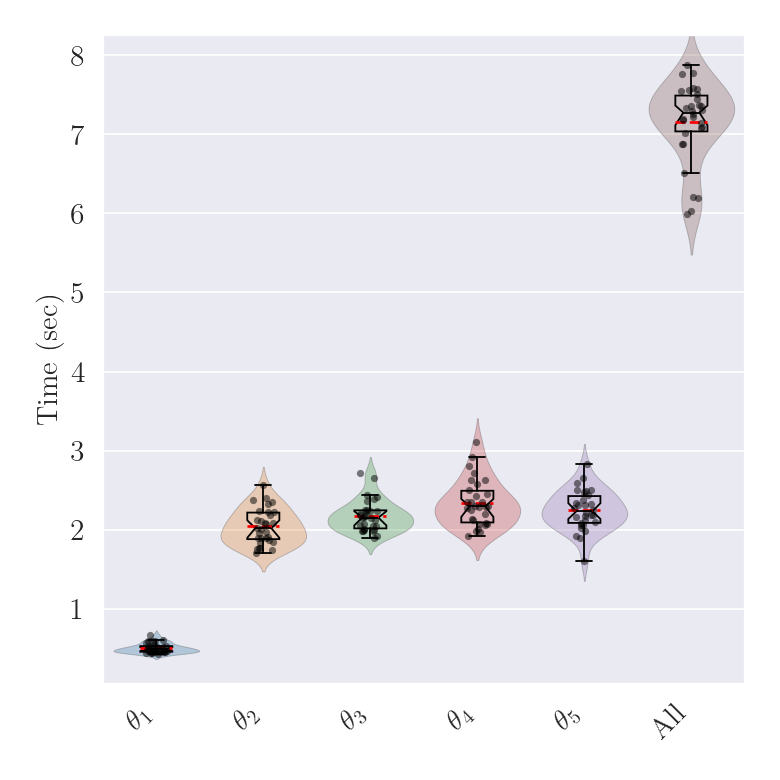}\label{fig:lrhapt_Time}}%
		\hfill
		\subfloat[ISOLET]{\includegraphics[width=0.24\textwidth]{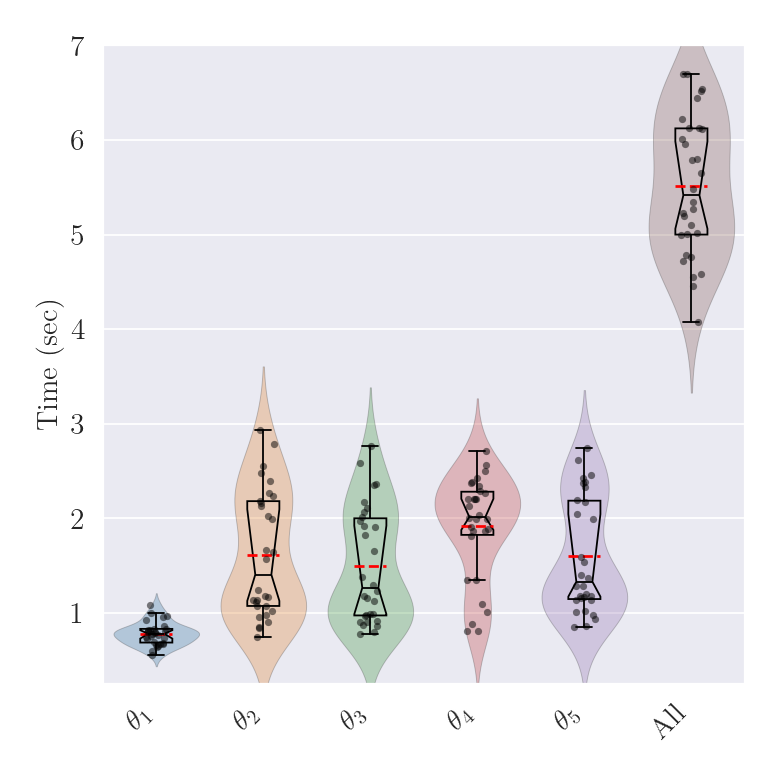}\label{fig:lrisolet_Time}}%
		\hfill
		\subfloat[PD]{\includegraphics[width=0.24\textwidth]{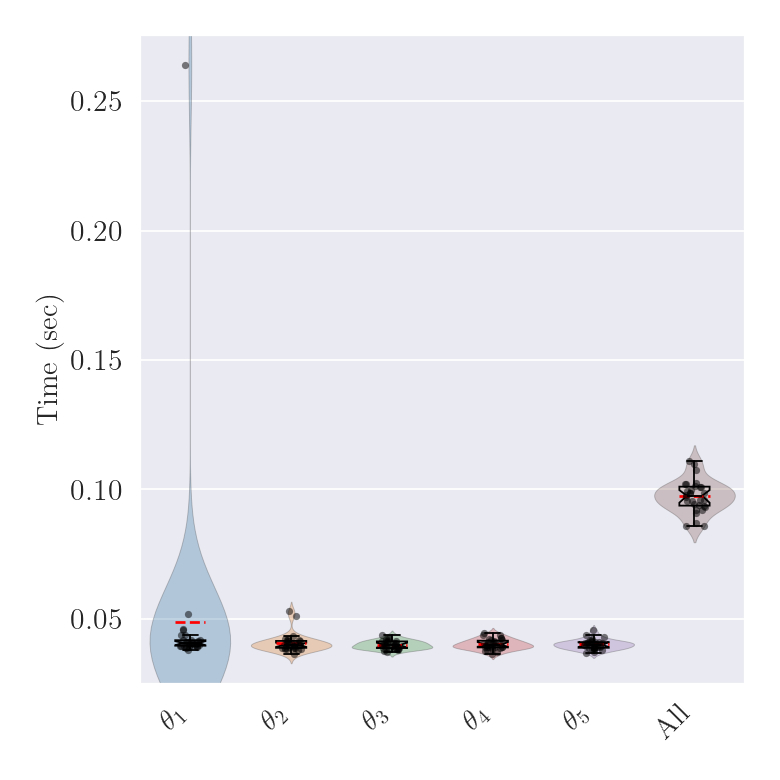}\label{fig:lrpd_Time}}
		\caption[The distribution of the obtained running time for 30 Logistic Regression runs.]{The raincloud plot of running time results obtained from 30 Logistic Regression runs.}
		
		\label{fig:lr_Time}
	\end{figure*}
	
	\begin{figure*}[t] 
		\centering
		\subfloat[APSF]{\includegraphics[width=0.24\textwidth]{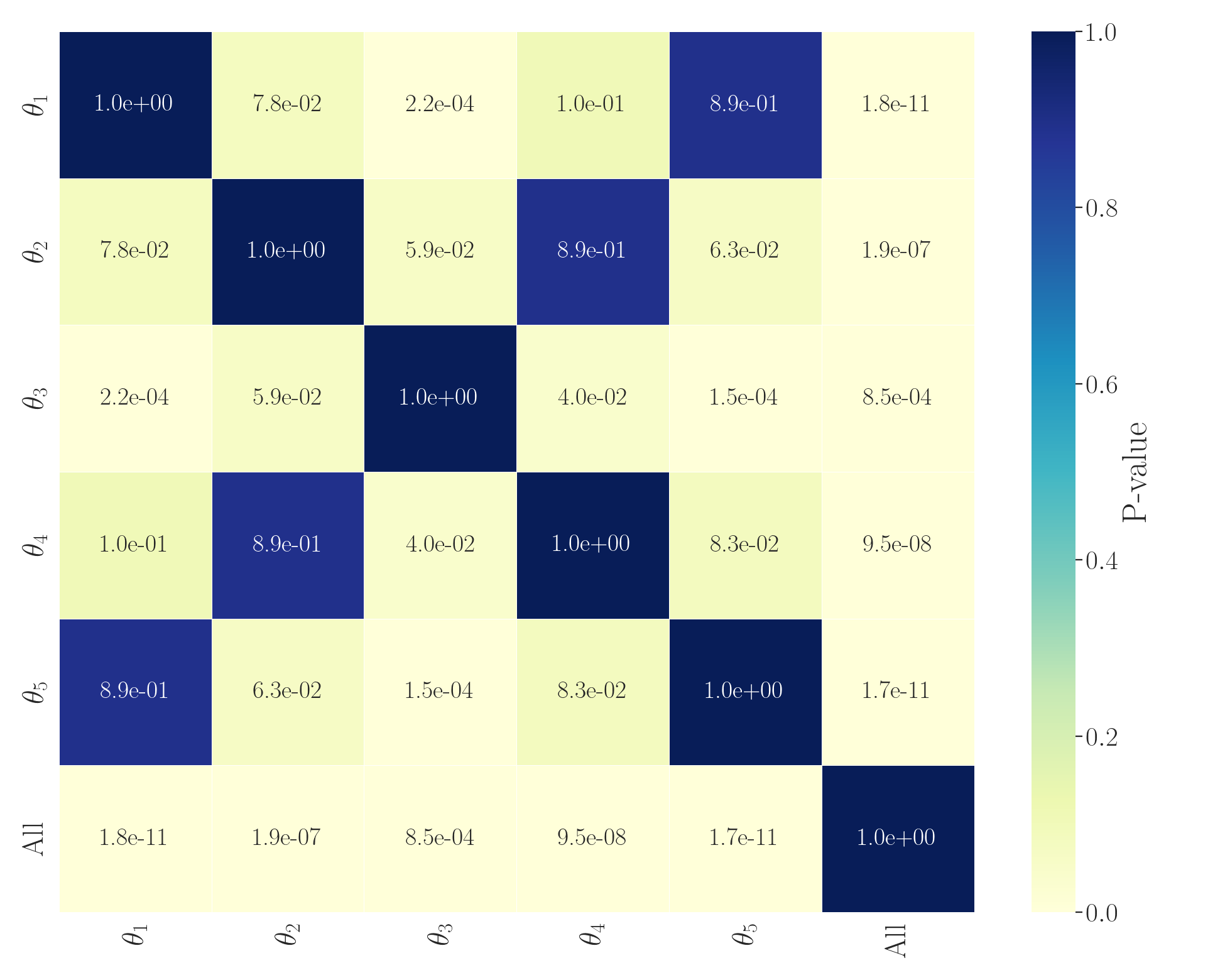}\label{fig:lrnemapsf_Time}}%
		\hfill
		\subfloat[ARWPM]{\includegraphics[width=0.24\textwidth]{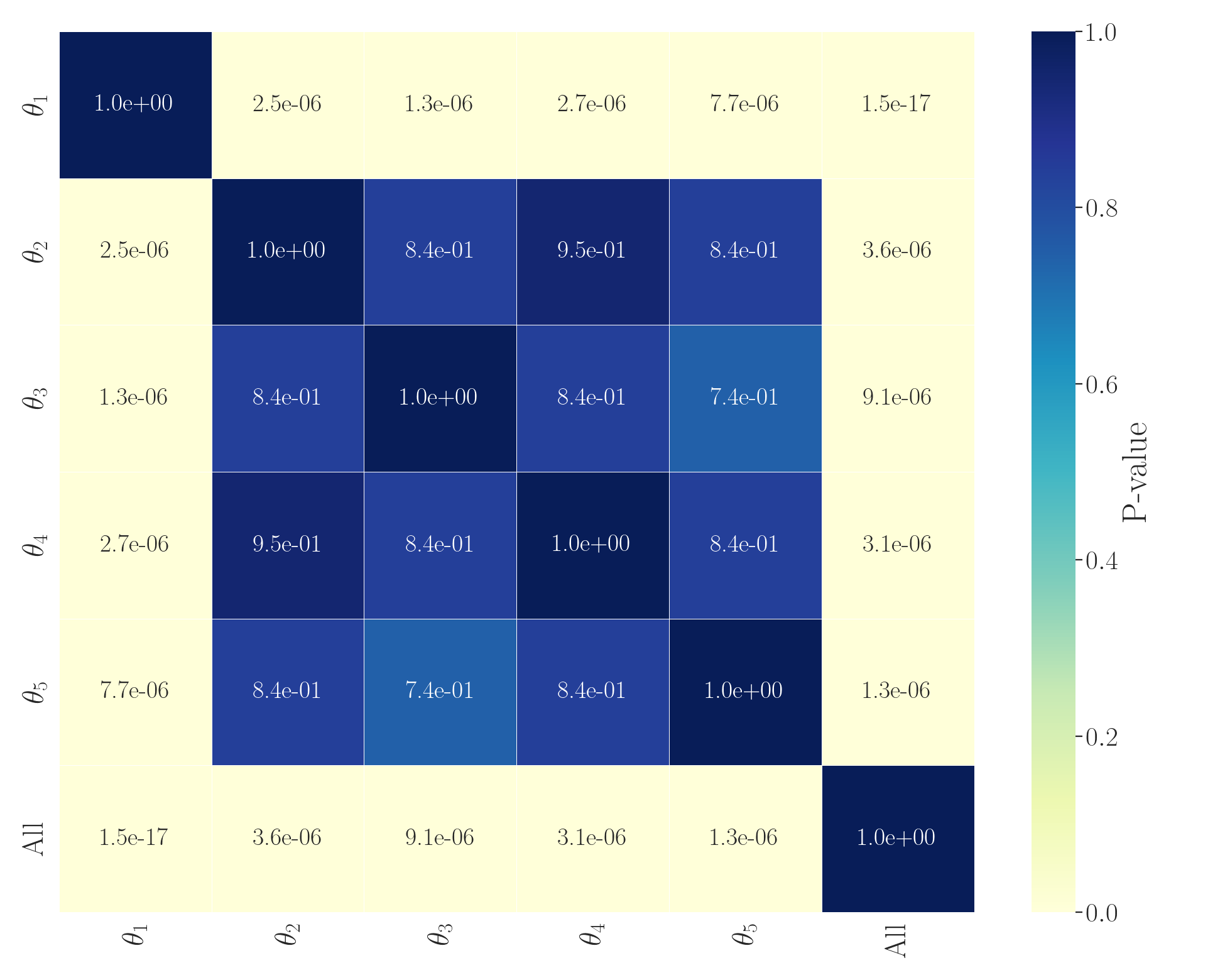}\label{fig:lrnemarwpm_Time}}%
		\hfill
		\subfloat[GECR]{\includegraphics[width=0.24\textwidth]{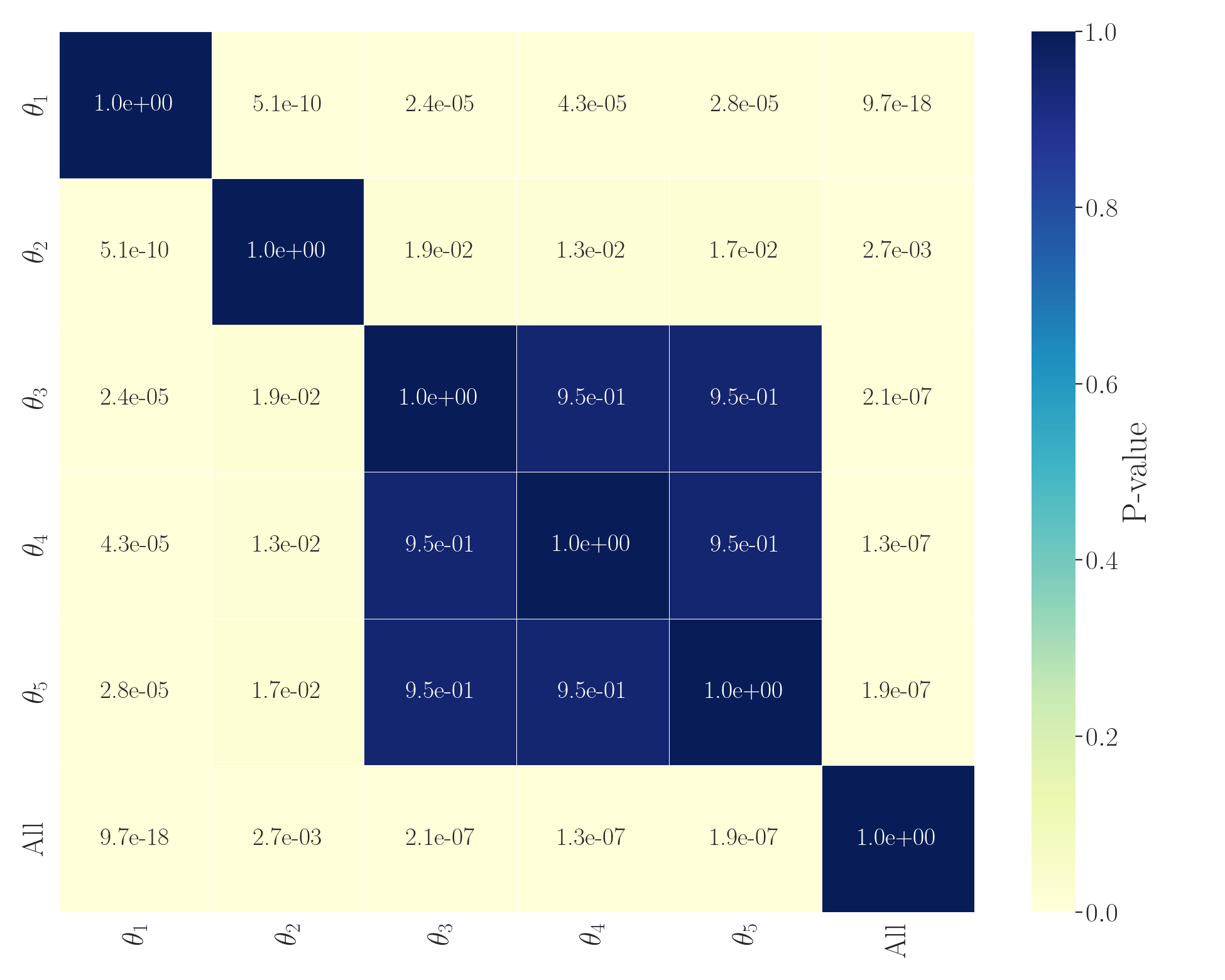}\label{fig:lrnemgecr_Time}}%
		\hfill
		\subfloat[GFE]{\includegraphics[width=0.24\textwidth]{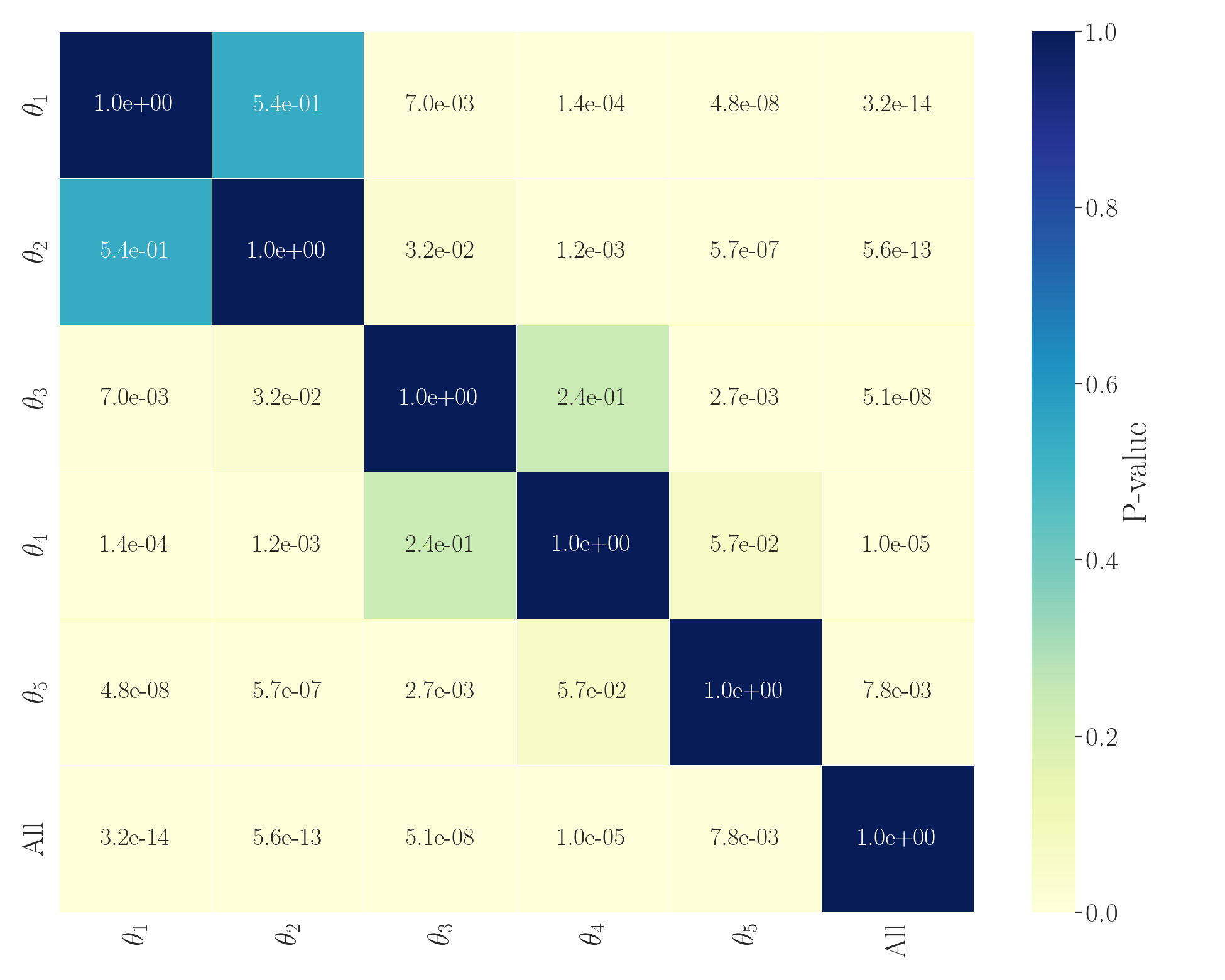}\label{fig:lrnemgfe_Time}}
		
		\subfloat[GSAD]{\includegraphics[width=0.24\textwidth]{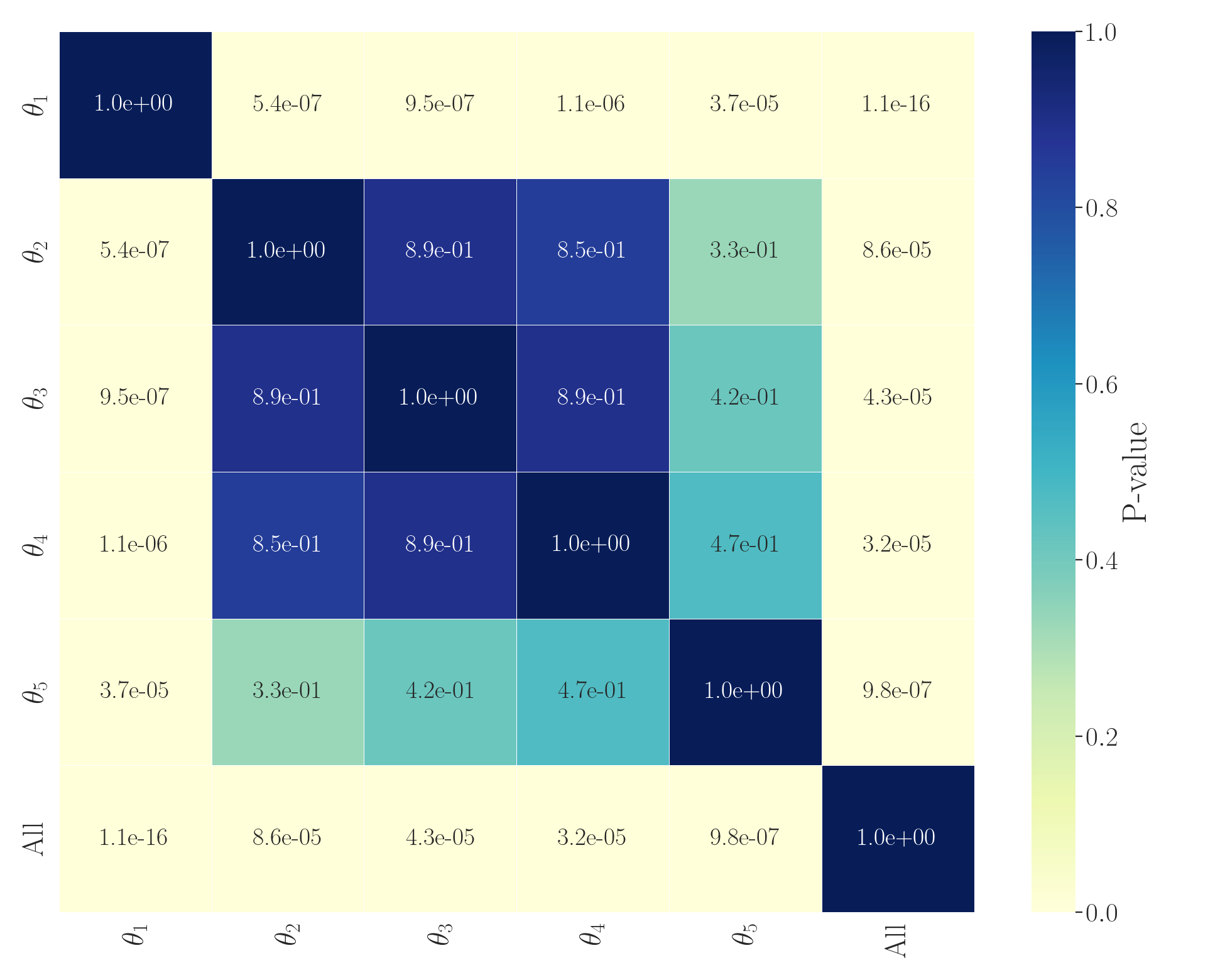}\label{fig:lrnemgsad_Time}}%
		\hfill
		\subfloat[HAPT]{\includegraphics[width=0.24\textwidth]{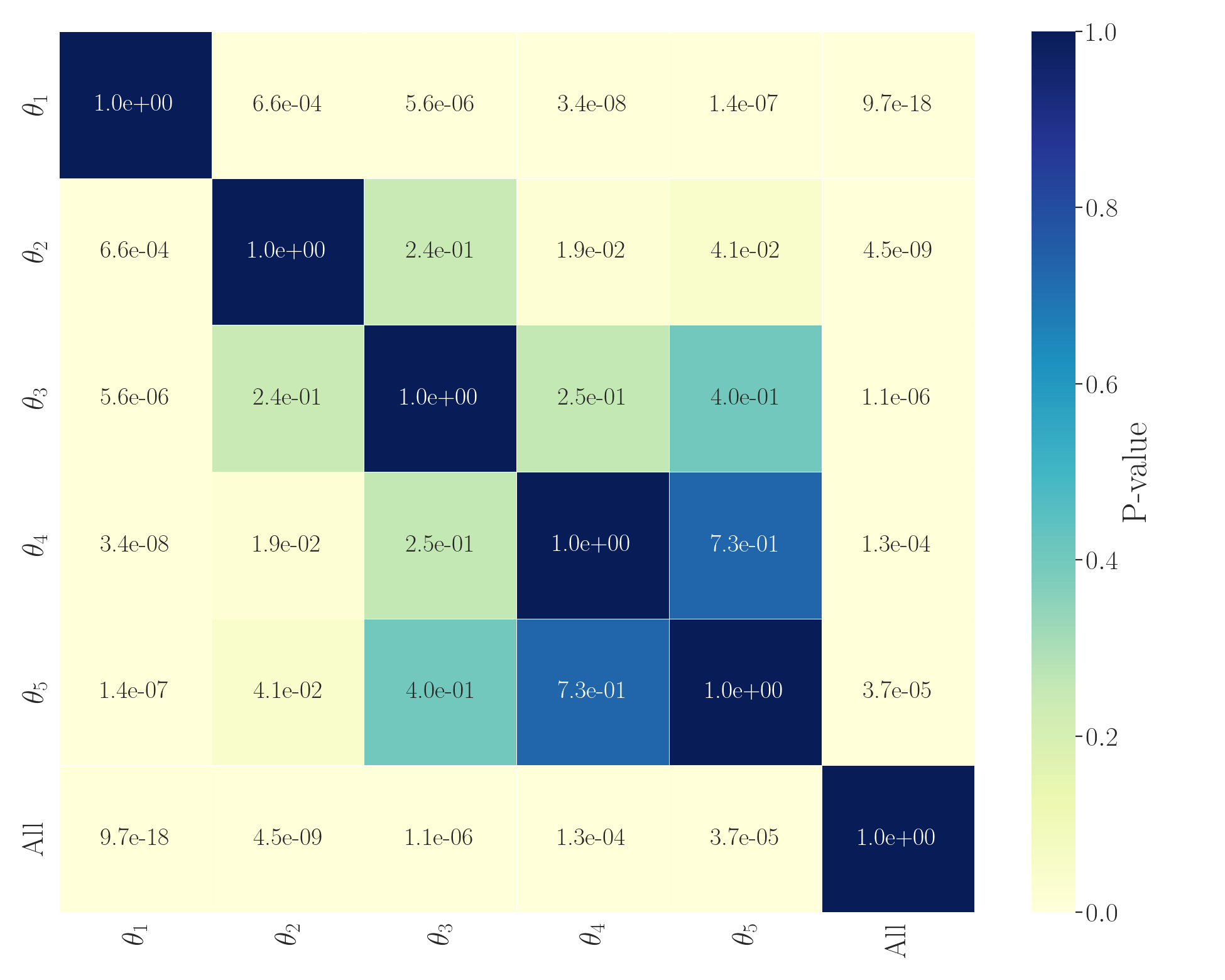}\label{fig:lrnemhapt_Time}}%
		\hfill
		\subfloat[ISOLET]{\includegraphics[width=0.24\textwidth]{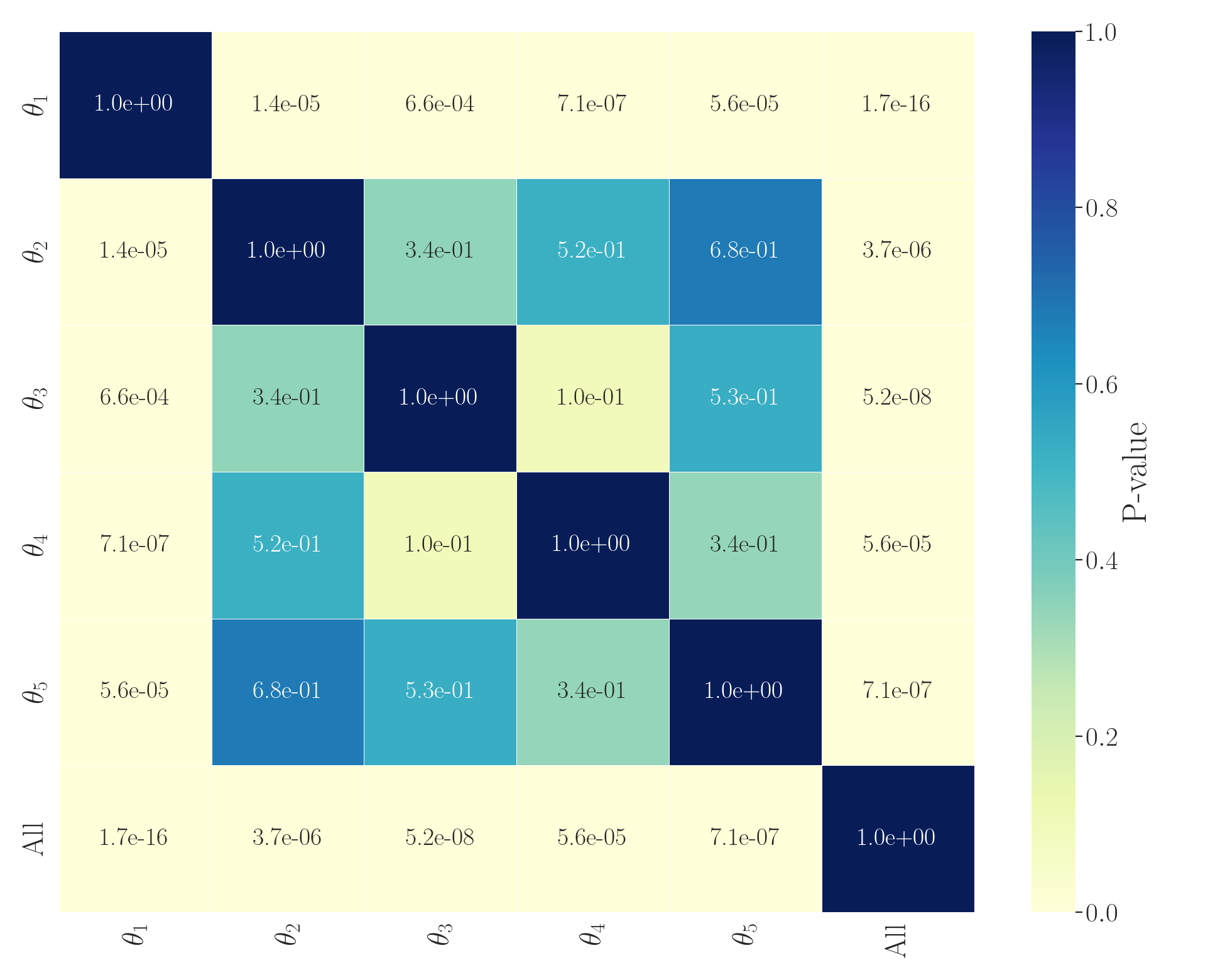}\label{fig:lrnemisolet_Time}}%
		\hfill
		\subfloat[PD]{\includegraphics[width=0.24\textwidth]{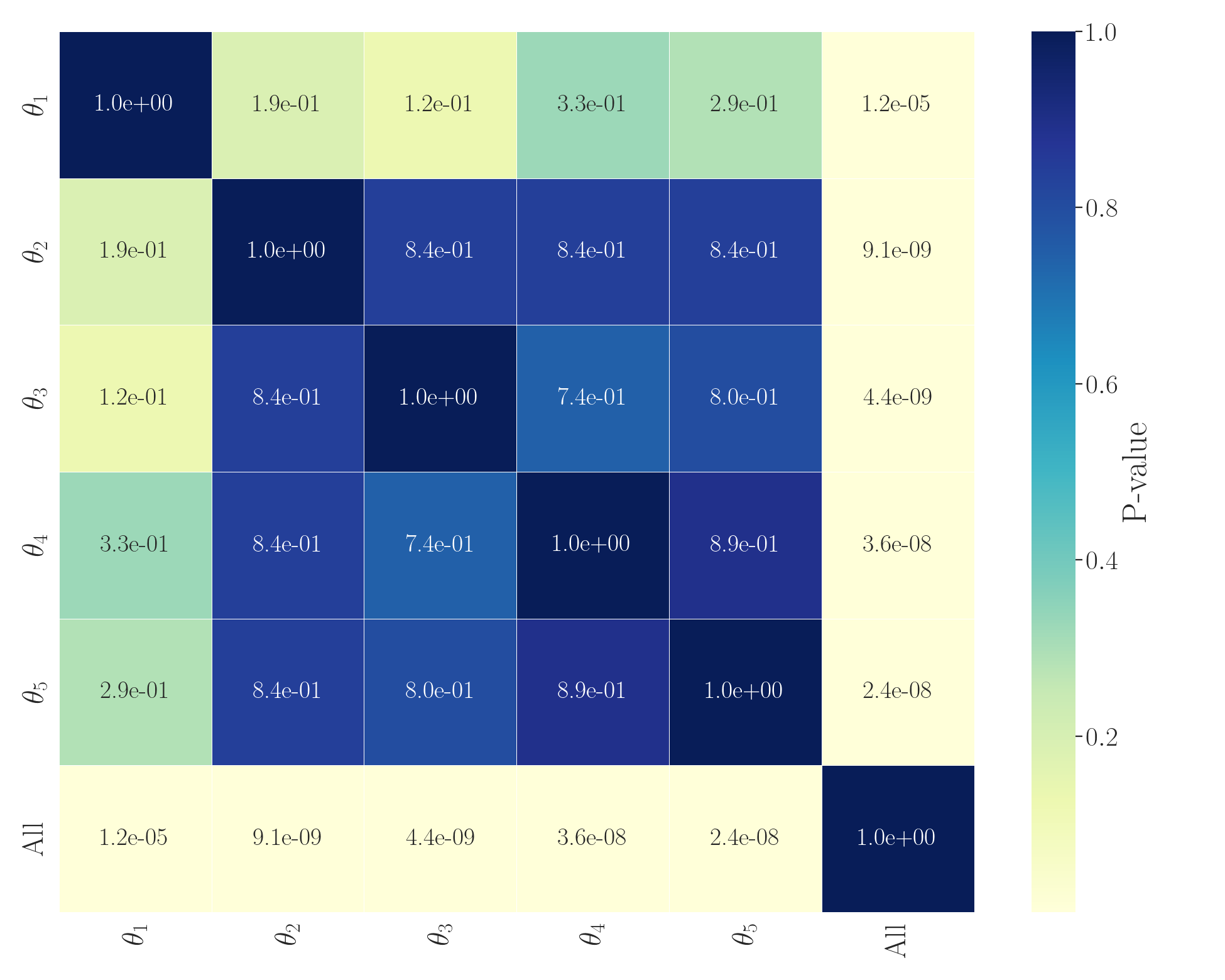}\label{fig:lrnempd_Time}}
		\caption[The adjusted Conover's P-values for the running time of 30 Logistic Regression runs.]{The results of the Conover post-hoc test on testing data’s running time obtained from 30 Logistic Regression runs.}
		
		\label{fig:lrnem_Time}
	\end{figure*}
	\FloatBarrier
	
	\begin{figure*}[htbp] 
		\centering
		\subfloat[APSF]{\includegraphics[width=0.24\textwidth]{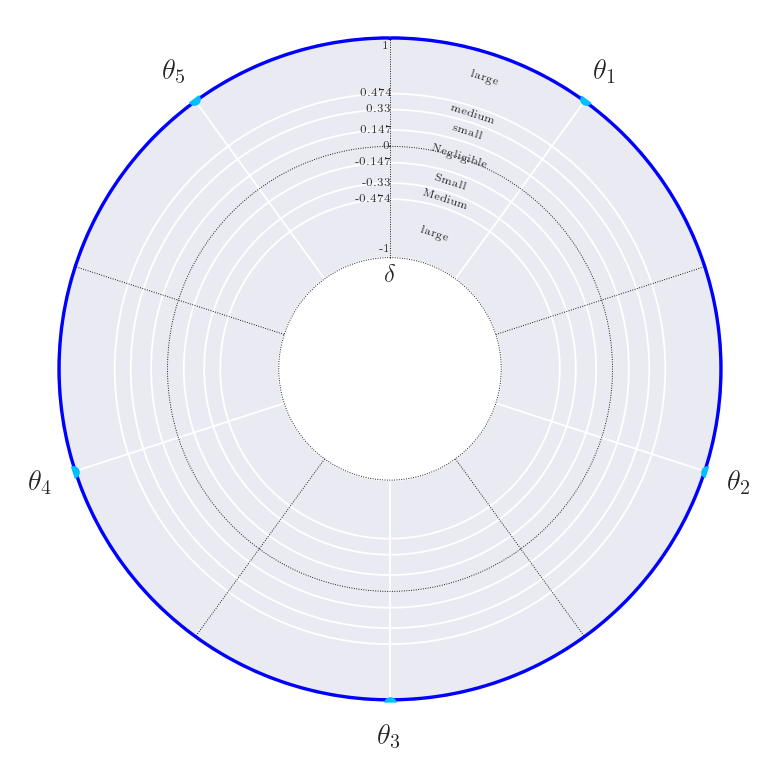}\label{fig:lrcliffapsf_Time}}%
		\hfill
		\subfloat[ARWPM]{\includegraphics[width=0.24\textwidth]{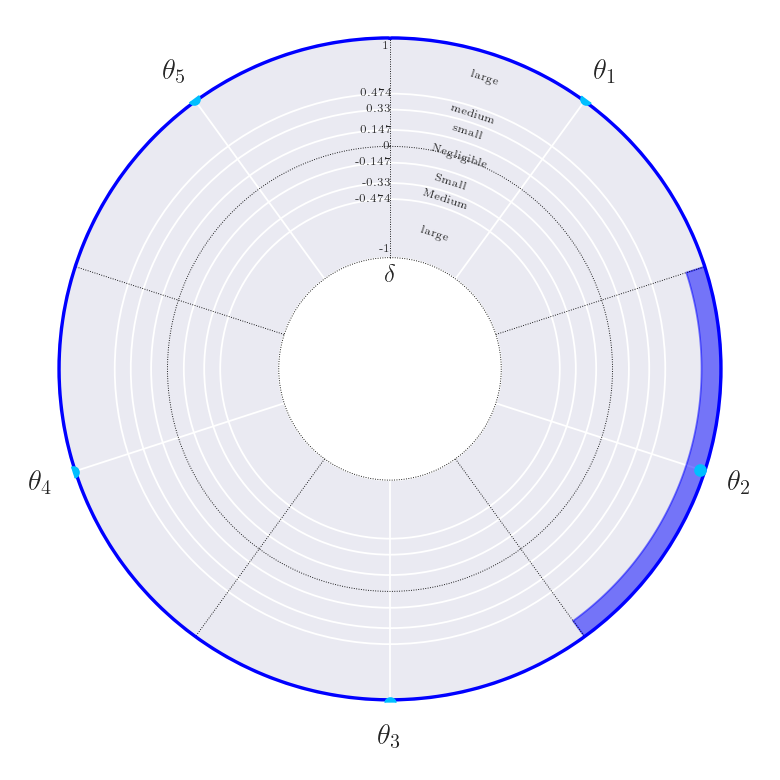}\label{fig:lrcliffarwpm_Time}}%
		\hfill
		\subfloat[GECR]{\includegraphics[width=0.24\textwidth]{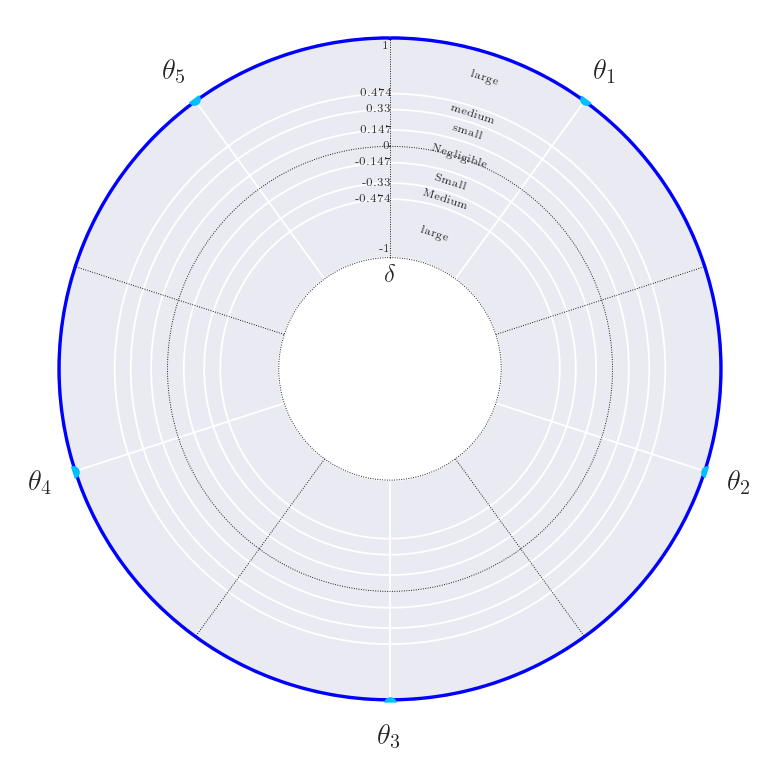}\label{fig:lrcliffgecr_Time}}%
		\hfill
		\subfloat[GFE]{\includegraphics[width=0.24\textwidth]{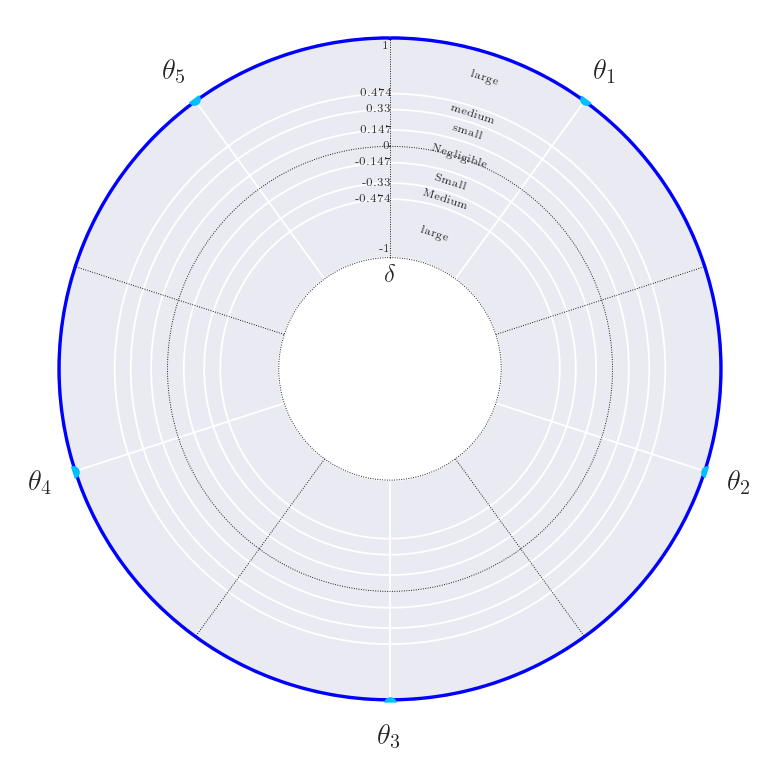}\label{fig:lrcliffgfe_Time}}
		
		\subfloat[GSAD]{\includegraphics[width=0.24\textwidth]{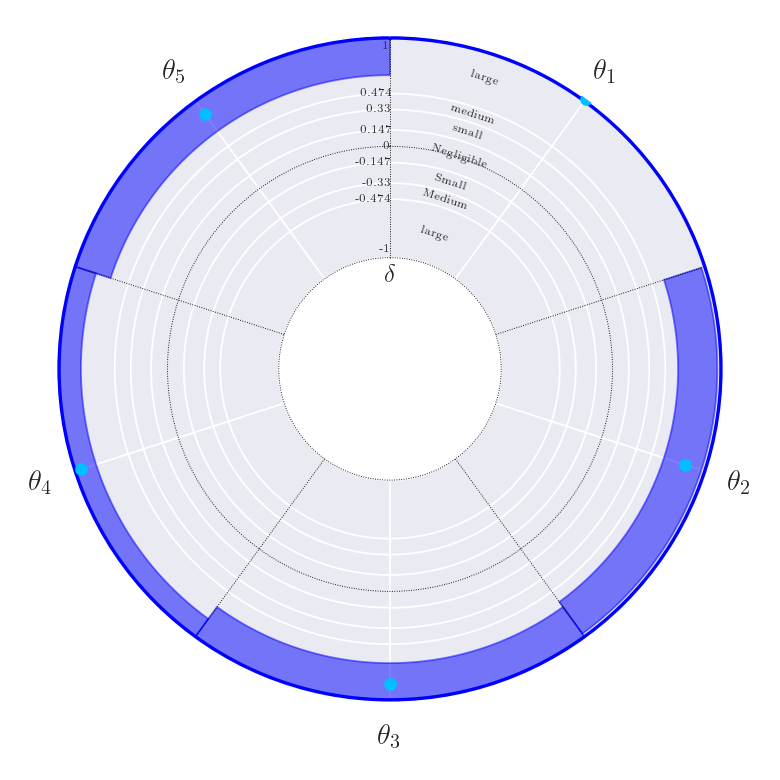}\label{fig:lrcliffgsad_Time}}%
		\hfill
		\subfloat[HAPT]{\includegraphics[width=0.24\textwidth]{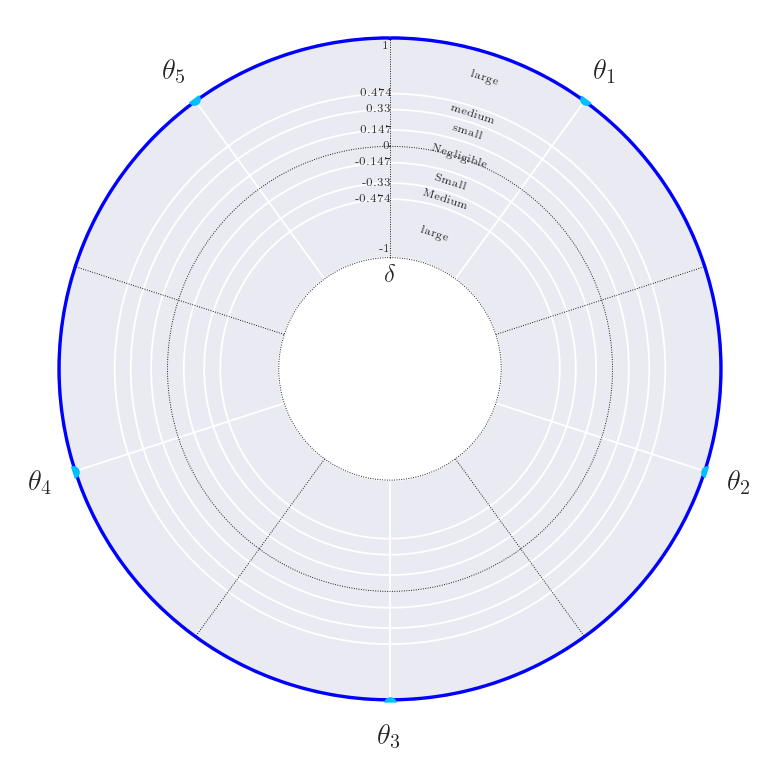}\label{fig:lrcliffhapt_Time}}%
		\hfill
		\subfloat[ISOLET]{\includegraphics[width=0.24\textwidth]{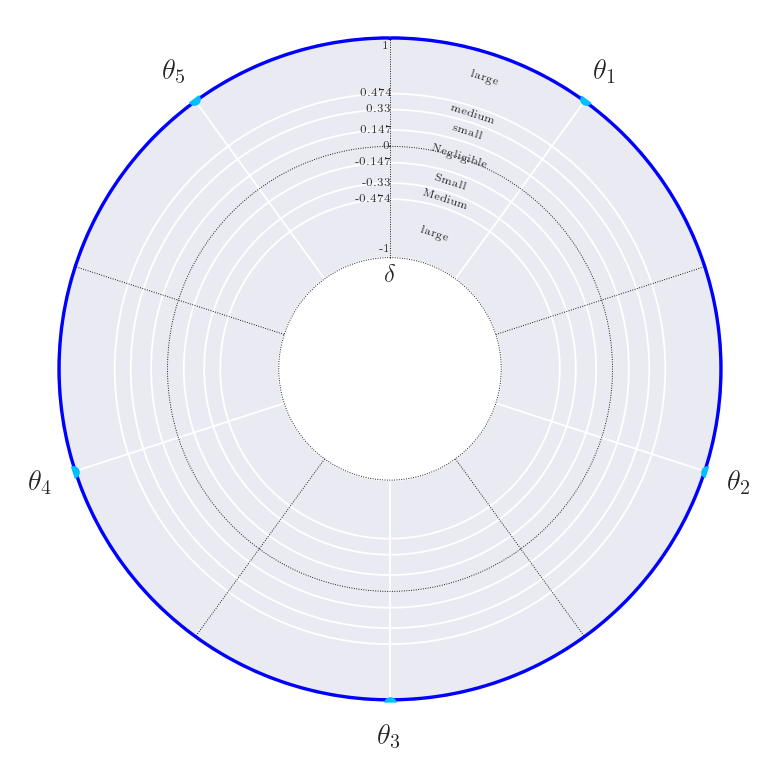}\label{fig:lrcliffisolet_Time}}%
		\hfill
		\subfloat[PD]{\includegraphics[width=0.24\textwidth]{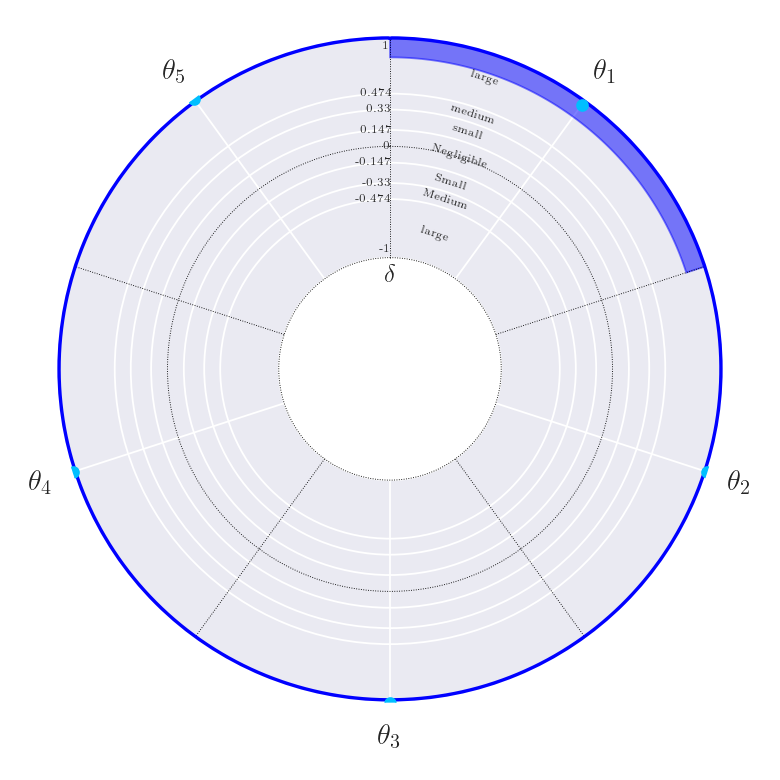}\label{fig:lrcliffpd_Time}}
		\caption[The Cliff's $\delta$ effect size measure and its 95\% confidence intervals for the running time of 30 Logistic Regression runs.]{Effect size analysis of running time across 30 Logistic Regression runs using Cliff's $\delta$. Each point represents the actual value obtained, with segments denoting 95\% confidence intervals based on 10,000 bootstrap resamplings. The outer ring color visualizes the statistical significance: grey illustrates no significant difference (adjusted Friedman's P-value$>0.05$), while color indicates significant differences; blue indicates at least one view outperforms the benchmark (adjusted Conover's p-value$ < 0.05$, Cliff's $\delta > 0$), and red signifies all views underperform relative to the benchmark (adjusted Conover's p-value$ < 0.05$, Cliff's $\delta < 0$). Segment colors show performance difference against the benchmark: grey for no significant difference (adjusted Conover's p-value$  > 0.05$), blue for better performance (Cliff's $\delta > 0$), and red for worse performance (Cliff's $\delta < 0$).}
		
		\label{fig:lrcliff_Time}
	\end{figure*}
	
	\begin{table*}[htbp]
		\centering
		\caption[The results of Friedman and Conover tests and Cliff's $\delta$ analysis for the running time of 30 Logistic Regression runs.]{Statistical comparison of Running Time (seconds) for testing data obtained from Logistic Regression runs. W, T, and L denote win, tie, and loss based on Friedman and Conover's p-values. Effect sizes are calculated using Cliff's Delta method and are categorized as negligible, small, medium, or large.}
		\label{tab:lrtime}
			\begin{tabular}{c|ccccccccc}
				\hline
				\multicolumn{10}{c}{Logistic Regression's Running Time (seconds)}\\
				\hline
				Dataset & $\theta_1$ & $\theta_2$ & $\theta_3$ & $\theta_4$ & $\theta_5$ & $E_{1:2}$ & $E_{1:3}$ & $E_{1:4}$ & $E_{1:5}$ \\
				\hline
				APSF  & W (large) & W (large) & W (large) & W (large) & W (large) & --  & --  & --  & --  \\
				ARWPM  & W (large) & W (large) & W (large) & W (large) & W (large) & --  & --  & --  & --  \\
				GECR  & W (large) & W (large) & W (large) & W (large) & W (large) & --  & --  & --  & --  \\
				GFE  & W (large) & W (large) & W (large) & W (large) & W (large) & --  & --  & --  & --  \\
				GSAD  & W (large) & W (large) & W (large) & W (large) & W (large) & --  & --  & --  & --  \\
				HAPT  & W (large) & W (large) & W (large) & W (large) & W (large) & --  & --  & --  & --  \\
				ISOLET  & W (large) & W (large) & W (large) & W (large) & W (large) & --  & --  & --  & --  \\
				PD  & W (large) & W (large) & W (large) & W (large) & W (large) & --  & --  & --  & --  \\
				\hline
				W - T - L  & 8 - 0 - 0 & 8 - 0 - 0 & 8 - 0 - 0 & 8 - 0 - 0 & 8 - 0 - 0 & -- & -- & -- & -- \\
				\hline
			\end{tabular}
	\end{table*}

	\begin{table*}[htbp]
		\centering
		\caption[The summary of statistical comparison of results obtained from 30 Logistic Regression runs.]{The summary of statistical comparison of results for testing data obtained from Logistic Regression runs. W, T, and L denote win, tie, and loss based on Friedman and Conover's adjusted p-values, and Cliff's $\delta$ effect size analysis.}
		\label{tab:lrwtl}
			\begin{tabular}{cccccccccc}
				\hline
				\multicolumn{10}{c}{Logistic Regression (Win - Tie - Loss)}
				\\
				\hline
				Metric & $\theta_1$ & $\theta_2$ & $\theta_3$ & $\theta_4$ & $\theta_5$ & $E_{1:2}$ & $E_{1:3}$ & $E_{1:4}$ & $E_{1:5}$ \\
				\hline
				$F_{1}$ Score & 1 - 2 - 5 & 1 - 3 - 4 & 0 - 3 - 5 & 0 - 2 - 6 & 0 - 1 - 7 & 1 - 3 - 4 & 3 - 2 - 3 & 3 - 2 - 3 & 4 - 4 - 0 \\
				AUC & 1 - 1 - 6 & 0 - 1 - 7 & 0 - 2 - 6 & 0 - 1 - 7 & 0 - 0 - 8 & 1 - 2 - 5 & 2 - 3 - 3 & 2 - 3 - 3 & 2 - 6 - 0 \\
				Loss & 2 - 1 - 5 & 1 - 1 - 6 & 0 - 3 - 5 & 0 - 3 - 5 & 0 - 2 - 6 & 2 - 2 - 4 & 3 - 1 - 4 & 4 - 0 - 4 & 4 - 1 - 3 \\
				MEC & 3 - 0 - 5 & 3 - 0 - 5 & 2 - 1 - 5 & 2 - 1 - 5 & 1 - 2 - 5 & 3 - 0 - 5 & 3 - 0 - 5 & 1 - 2 - 5 & 1 - 2 - 5 \\
				MEW & 4 - 1 - 3 & 3 - 2 - 3 & 4 - 2 - 2 & 2 - 4 - 2 & 3 - 5 - 0 & 5 - 1 - 2 & 4 - 2 - 2 & 4 - 2 - 2 & 4 - 2 - 2 \\
				Time & 8 - 0 - 0 & 8 - 0 - 0 & 8 - 0 - 0 & 8 - 0 - 0 & 8 - 0 - 0 & -- & -- & -- & -- \\
				\hline
			\end{tabular}
	\end{table*}
	\FloatBarrier
	